\documentclass[12pt]{uwathes}
\usepackage[noepsf,unsectioned]{uwamaths}
\usepackage{graphicx}          
\usepackage{lscape}            
\usepackage{color}

\usepackage{multirow,tabularx,dcolumn,url}
\usepackage{algpseudocode}
\usepackage{algorithm, algorithmicx}

\usepackage[scriptsize]{subfigure}
\usepackage{morefloats}

\usepackage{mathtools,amsmath,amssymb,bm}
\usepackage{amsthm}
\usepackage[square, numbers, comma, sort&compress]{natbib} 

\newcommand{\blankpage}{
\newpage
\mbox{}
\newpage
}

\graphicspath{{./Figures/}} 

\renewcommand{\AtBeginSection}[1]{\numberallin{#1}} 

\oneandhalfspacing             
\makesymbols   
\title{\LARGE Hyperspectral Imaging and Analysis for Sparse Reconstruction and Recognition}     
\author{Zohaib Khan}   
\date{} 
\begin{document}

\maketitle
\textcolor{white}{a}
\thispagestyle{empty}
\vspace{8cm}

\textcolor{white}{a} \hspace{2cm}{\Large \copyright \hspace{2mm} Copyright 2014 by Zohaib Khan}

\clearpage

\textcolor{white}{a}
\thispagestyle{empty}

\vspace{8cm}

\textcolor{white}{a} \hspace{1cm}{\em \large Dedicated to my grandmother, late Bilqees Begum}

\textcolor{white}{a}\clearpage

\setcounter{page}{0}

%
%

\begin{abstract}
Hyperspectral imaging, also known as imaging spectroscopy, captures a data cube of a scene in two spatial and one spectral dimension. Hyperspectral image analysis refers to the operations which lead to quantitative and qualitative characterization of a hyperspectral image. This thesis contributes to hyperspectral imaging and analysis methods at multiple levels.

In a tunable filter based hyperspectral imaging system, the recovery of spectral reflectance is a challenging task due to limiting filter transmission, illumination bias and band misalignment. This thesis proposes a hyperspectral imaging technique which adaptively recovers spectral reflectance from raw hyperspectral images captured by automatic exposure adjustment. A spectrally invariant self similarity feature is presented for cross spectral hyperspectral band alignment. Extensive experiments on an in-house developed multi-illuminant hyperspectral image database show a significant reduction in the mean recovery error.

The huge spectral dimension of hyperspectral images is a bottleneck for efficient and accurate hyperspectral image analysis. This thesis proposes spectral dimensionality reduction techniques from the perspective of spectral only, and spatio-spectral information preservation.
The proposed Joint Sparse PCA selects bands from spectral only data where pixels have no spatial relationship. The joint sparsity constraint is introduced in the PCA regression formulation for band selection. Application to clustering of ink spectral responses is demonstrated for forensic document analysis. Experiments on an in-house developed writing ink hyperspectral image database prove that a higher ink mismatch detection accuracy can be achieved using relatively fewer bands by the proposed band selection method.

Joint Group Sparse PCA is proposed for band selection from spatio-spectral data where pixels are spatially related. The additional group sparsity takes the spatial context into account for band selection. Application to compressed hyperspectral imaging is demonstrated where a test hyperspectral image cube can be reconstructed by sensing only a sparse selection of bands. Experiments on four hyperspectral image datasets including an in-house developed face database verify that the lowest reconstruction error and the highest recognition accuracy is achieved by the proposed compressed sensing technique.

An application of the proposed band selection is also presented in an end-to-end framework of hyperspectral palmprint recognition. An efficient representation and binary encoding technique is proposed for selected bands of hyperspectral palmprint which outperforms state-of-the-art in terms of equal error rates on three databases.
\end{abstract}

\thispagestyle{empty}
\newpage
\setcounter{page}{0}
\begin{acknowledgements}
I begin by thanking Almighty Allah for making me achieve this milestone. I cannot be more thankful in this world than to my parents, grandparents, siblings and relatives who wished and prayed for my success. I also thank my lovely wife who was a great motivation for me to finish my PhD (and get married!).

I am indebted to the unconditional support of my supervisors Dr.~Ajmal Mian and Dr.~Faisal Shafait, through all times, highs and lows. Without their presence, this dream could not be realized. They trained me to undertake research, provided feedback at regular intervals and navigated me through the course of PhD. I am also grateful to Dr.~Yiqun Hu for his co-supervision in the first two years of PhD.

I owe a huge thanks to Prof.~Robyn Owens who provided an insightful directive on my research, crucial to shape the thesis towards the end. I also thank Dr.~Arif Mahmood who reviewed one of my important research contribution. I am grateful to all the anonymous peers who reviewed my numerous submissions to conferences and journals. I am enormously appreciative of the reviewers of this thesis whose timely feedback resulted in great improvement to the final version of this thesis.

One of the most important aspect of this thesis was hyperspectral datasets collection. I thank my supervisors for motivating and encouraging me to collect these datasets. I am extremely thankful to Muhammad Uzair for his support in collection of the hyperspectral face dataset. I thank all the participants, who volunteered for research data collection. I am also grateful to the graduate research coordinator and the head of school for their valuable support as mentors. I thank the administration and support staff at the school who deserve due recognition of their efforts.

I acknowledge the contribution of the external research groups and universities for making their spectral datasets publicly available for research. They are: Carnegie Mellon University (hyperspectral face data), Hong Kong Polytechnic University (multispectral palm data, hyperspectral palm data and hyperspectral face data), Chinese Academy of Sciences Institute of Automation (multispectral palm data), Columbia University (multispectral image data), Harvard University (hyperspectral image data) and Simon Fraser University (hyperspectral illuminant data).

In the end, I would thankfully acknowledge all funding institutions, without whom quality research is inconceivable. This research was sponsored by The Australian Research Council (ARC Grant DP110102399 and DP0881813) and The University of Western Australia (IPRS and UWA Grant 00609 10300067).
\end{acknowledgements}

\thispagestyle{empty}
\newpage
\setcounter{page}{0}      
\tableofcontents
\listoftables         
\listoffigures        
\printsymbols         

\blankpage
\addcontentsline{toc}{chapter}{List of Publications}
\section*{\center \Large List of Publications}

\section*{\normalsize International Journal Publications}

\begin{enumerate}

\item[{[}1{]}] Zohaib~Khan, F.~Shafait and A.~Mian,``Joint Group Sparse PCA for Compressed Hyperspectral Imaging'', \emph{IEEE Trans. Image Processing (under review)}, 2014. (Chapter 5)
\item[{[}2{]}] Zohaib~Khan, F.~Shafait and A.~Mian,``Automatic Ink Mismatch Detection for Forensic Document Analysis'', \emph{Pattern Recognition (under review)}, 2014. (Chapter 6)
\item[{[}3{]}] Zohaib~Khan, F.~Shafait, Y.~Hu and A.~Mian,``Multispectral Palmprint Encoding and Recognition'', \emph{eprint arXiv:1402.2941}, 2014. (Chapter 7)



\hspace{-1.1cm} {\bf \normalsize International Conference Publications (\emph{Fully Refereed})}

\item[{[}6{]}] Zohaib~Khan, F.~Shafait and A.~Mian, ``Adaptive Spectral Reflectance Recovery Using Spatio-Spectral Support from Hyperspectral Images'', \emph{International Conference on Image Processing}, 2014.

\emph{The preliminary ideas and results of this paper were refined and extended to contribute to Chapter 3 of this thesis}.

\item[{[}5{]}] Zohaib~Khan, A.~Mian and Y.~Hu, ``Contour Code: Robust and Efficient Multispectral Palmprint Encoding for Human Recognition'', \emph{International Conference on Computer Vision}, 2011.

\emph{The preliminary ideas and results of this paper were refined and extended to contribute towards} [3]
\emph{which forms Chapter 7 of this thesis}.

\item[{[}6{]}] Zohaib~Khan, F.~Shafait and A.~Mian, ``Hyperspectral Imaging for Ink Mismatch Detection'', \emph{International Conference on Document Analysis and Recognition}, 2013.

\emph{The preliminary ideas and results of this paper were refined and extended to contribute towards} [2]
\emph{which forms Chapter 6 of this thesis}.

\item[{[}7{]}] Zohaib~Khan, Y.~Hu and A.~Mian, ``Facial Self Similarity for Sketch to Photo Matching'', \emph{Digital Image Computing: Techniques and Applications}, 2012.

\emph{The idea of self similarity descriptor in this paper was refined and extended to contribute to Chapter 4 of this thesis}

\item[{[}8{]}] Zohaib~Khan, F.~Shafait and A.~Mian, ``Hyperspectral Document Imaging: Challenges and Perspectives'', \emph{5th International Workshop on Camera-Based Document Analysis and Recognition}, 2013.

\emph{This paper presents an evaluation of the camera based hyperspectral document imaging. The findings of this study contribute towards} [2] {\em which forms Chapter 6 of this thesis}.

\item[{[}9{]}] Zohaib~Khan, F.~Shafait and A.~Mian, ``Towards Automated Hyperspectral Document Image Analysis'', \emph{2nd International Workshop on Automated Forensic Handwriting Analysis}, 2013.
\emph{This paper highlights the potential of hyperspectral imaging in various applications, especially document analysis}.

\end{enumerate}

\textbf{Note}: According to the 2013 ranking of the Computing Research and Education Association of Australasia, CORE, The International Conference on Computer Vision (ICCV) is ranked A$^{*}$ (flagship conference). The International Conference on Document Analysis and Recognition (ICDAR) is ranked A (excellent conference). The International Conference on Image Processing (ICIP) and Digital Image Computing: Techniques and Applications (DICTA) are ranked B (good conference).

%

\mainmatter           

\chapter{Introduction} 
\label{Chapter1} 



The human eye can sense light in the visible range ($\sim$400nm-700nm) of electromagnetic spectrum.
Given its trichromatic design, the human eye is only capable of sensing three primary colors (red, green and blue). This causes metamerism in humans, i.e.~they are unable to distinguish between two apparently similar colors. For instance, two materials with slightly different physical properties may appear identical in color to the naked eye due to metamerism. Moreover, the human eye is only capable of sensing a small range of the electromagnetic spectrum. This limits our ability to seek information beyond the visual range, such as the infra red and the ultraviolet ranges.

Machine vision is free from the limitations of RGB vision. It can benefit from a wide range of the spectrum, both visible and beyond visible range by \emph{hyperspectral imaging}. Hyperspectral imaging levies machine vision from the curse of metamerism and creates opportunities for use in automatic color vision tasks like object detection, segmentation and recognition. It has the capacity to sense more than just three primary colors which offers increased fidelity in sensing the spectral properties of materials.
However, hyperspectral imaging brings its own challenges. Before raw hyperspectral images can be used, a challenge is to separate the true reflectance from the illumination of the scene. This research problem can be termed as the estimation of illumination from hyperspectral images for spectral reflectance recovery.
Unlike RGB images, hyperspectral images are generally captured in a time multiplexed manner, i.e.~each band is captured sequentially, one after the other. During acquisition, small movement of the objects can introduce spatial misalignment of pixels between the consecutive bands which results in spectral noise. Therefore, hyperspectral images cannot be normalized unless the spectral reflectance is recovered and the bands are accurately registered. This thesis investigates preprocessing techniques for normalization of hyperspectral images.

Dimensionality of the data plays a critical role in hyperspectral image analysis. One of the most important question is to see which subset of bands are more informative relative to the rest of the bands. Reduction of bands can subsequently reduce the cost of sensors, the computational cost of analysis, and result in significant performance gains.
This thesis proposes novel band selection techniques for spatial and spatio-spectral hyperspectral image analysis. Application to reconstruction and recognition of objects, biometrics and document analysis are demonstrated to evaluate the superiority of the proposed techniques.

\section{Applications}

Sophisticated hyperspectral imaging systems are open to a number of applications in art, archeology, medical imaging, food inspection, forensics and biometrics. In food quality assessment, hyperspectral imaging can be used to identify premature diseases and defects. For example, rottenness of fruit and meat can usually be detected once visible marks become apparent or a specific odor is released. Hyperspectral imaging can identify such anomalies ahead of time and save huge investments in large scale crops by timely action.
The same quality of hyperspectral imaging can be of benefit in identifying the ripeness of fruit/vegetable in a crop in-vivo. Thus, it avoids the need to pluck out samples and dispatch for analysis in a laboratory.

Hyperspectral imaging is of great value in identification and separation of mineral sources. It can also distinguish writings made in different ink for forensic investigation. Thus unlike destructive forensic examination it allows preservation of a forensic evidence. It can also separate different pigments in a painting or a historical artifact useful for restoration.

Multi-modal biometrics is yet another emerging research area. The ability of hyperspectral imaging to capture the superficial and subsurface information of a human face, palm and fingerprint has translated into research in multispectral biometrics. Such complementary information is relatively more secure and cannot be easily forged to break through a security system.

\begin{figure*}[t]
\centering
\includegraphics[width=1\linewidth]{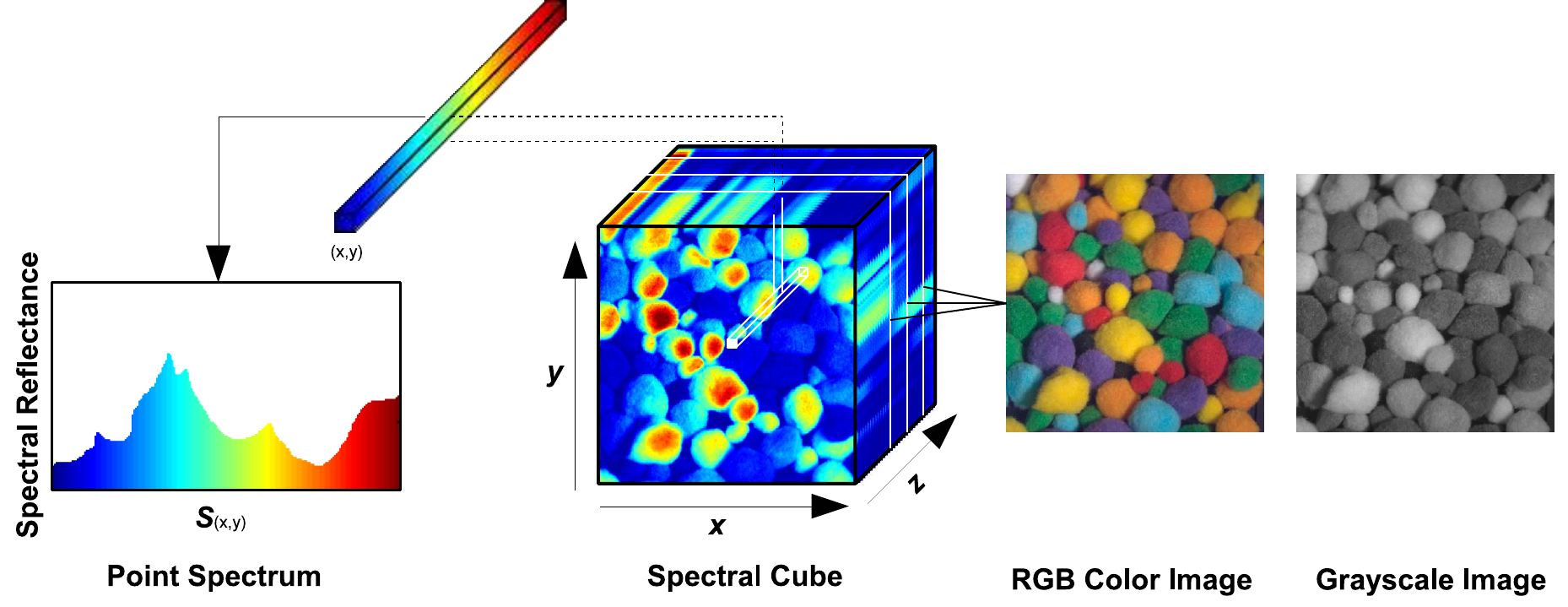}
\caption[Hyperspectral image explained]{A hyperspectral image is represented as a 3D image cube (shown in pseudo-colors). Spectral response at the $(x,y)$ spatial location of the spectral cube. An RGB image is rendered from three bands of the spectral cube. A grayscale image is rendered by averaging the spectral cube.}
\label{fig:MSI}
\end{figure*}

\section{Definitions}

Before outlining the contents of this thesis, it is important to clarify some frequently used terms in hyperspectral image analysis. A \emph{hyperspectral image}, $I(x,y,z)$ has two spatial ($x$ and $y$) and one spectral ($z$) dimension, where $(x,y)$ corresponds to the scene position and $z$ denotes narrow spectral band (see Figure~\ref{fig:MSI}). A \emph{band} refers to a two dimensional slice of a hyperspectral image across the spectral dimension $(z)$. For example, an RGB image has three bands that roughly correspond to the red, green and blue channels of the electromagnetic spectrum. The \emph{spectral response} or \emph{spectral reflectance} is a one dimensional vector of a spatial point on the spectral cube.

%

Spectral images are often classified based on the number of bands. A \emph{multispectral image} has more bands than RGB image, which may lie anywhere in the electromagnetic spectrum. A \emph{hyperspectral image} is a series of contiguous bands, greater in number than multispectral image. The difference between multispectral and hyperspectral is somewhat ambiguous in the literature. There is no consensus in the literature on the number of bands beyond which a multispectral image is considered a hyperspectral image. In this thesis, distinction is made in the use of terms multispectral and hyperspectral mainly with regards to the number of bands. At certain places, the term spectral imaging is used in general to refer to both multi or hyperspectral forms. Table~\ref{tab:differences} lists the major differences between multispectral and hyperspectral images.

\begin{table}[h]
\caption{Differences between multispectral and hyperspectral images.}
\label{tab:differences}
\footnotesize
\begin{center}
\begin{tabular}{r|l}\hline
\textbf{Multispectral image}            & \textbf{Hyperspectral image} \\ \hline
Few bands                               & Many bands            \\
Low spectral resolution (FWHM$\ge10nm$) & High spectral resolution (FWHM$=1$ to $10 nm$)\\
Bands may not be contiguous             & Bands are contiguous \\
Sensor cost and complexity is low       & Sensor cost and complexity is high\\ \hline
\end{tabular}\\
\scriptsize{\emph{FWHM: Full Width at Half Maximum is a measure of the band width.}}
\end{center}
\end{table}

\section{Thesis Structure}

Before each chapter is summarized, an overview of the thesis is presented which is illustrated in Figure~\ref{fig:layout}. In Chapter~\ref{Chapter2} a comprehensive review of hyperspectral imaging and analysis techniques is presented alongside a description of the core concepts in this thesis. Chapter~\ref{Chapter3} presents a hyperspectral imaging and illuminant estimation technique for spectral reflectance recovery. Chapter~\ref{Chapter4} proposes a cross spectral registration method for spatial alignment of hyperspectral images. Chapter~\ref{Chapter5} constitutes a technique for band selection from group structured data with application to compressed hyperspectral imaging and recognition. Chapter~\ref{Chapter6} presents a band selection technique for non-structured data with application to hyperspectral ink mismatch detection. Chapter~\ref{Chapter7} presents a representation and matching technique for hyperspectral palmprint recognition. Chapter~\ref{Chapter8} concludes the thesis with a proposal of future work.

\subsection[Background]{Background (Chapter~2)}

This chapter gives an overview of the hyperspectral imaging and analysis techniques. In the first part of the chapter, some of the most important concepts relevant to foundation of this thesis are briefly discussed. This includes description of regression, regularization and multivariate data analysis to the extent required for the developments in this thesis. In the second part of the chapter, hyperspectral imaging techniques in the current literature are categorized and explained. A taxonomy of hyperspectral imaging methods is presented based on their operating principles and device composition. Some interesting applications of hyperspectral imaging alongside brief discussion of hyperspectral image analysis techniques are presented to highlight the motivation of this research.

\subsection[Spectral Reflectance Recovery]{Spectral Reflectance Recovery from Hyperspectral Images\\(Chapter~3)}

A non-uniform ambient illumination modulates the spectral reflectance of a scene. Tunable filters pose an additional constraint of throughput, which limits the radiant intensity measured by the camera sensor. This results in variable signal-to-noise ratio in spectral bands making accurate recovery of spectral reflectance a challenging task. In this chapter, a novel method for the recovery of spectral reflectance from hyperspectral images is proposed. It adaptively considers the spatio-spectral context of data into account while estimating the scene illumination. The adaptive illumination estimation is improvised by variable exposure imaging which automatically compensates for the SNR of captured hyperspectral images. The proposed spectral reflectance recovery method is evaluated in both simulated and real illumination scenarios. Experiments show that the adaptive illuminant estimation and variable exposure imaging reduce mean error by 13\% and 35\%, respectively.

\begin{landscape}
\begin{figure}
\centering
\includegraphics[width=1\linewidth]{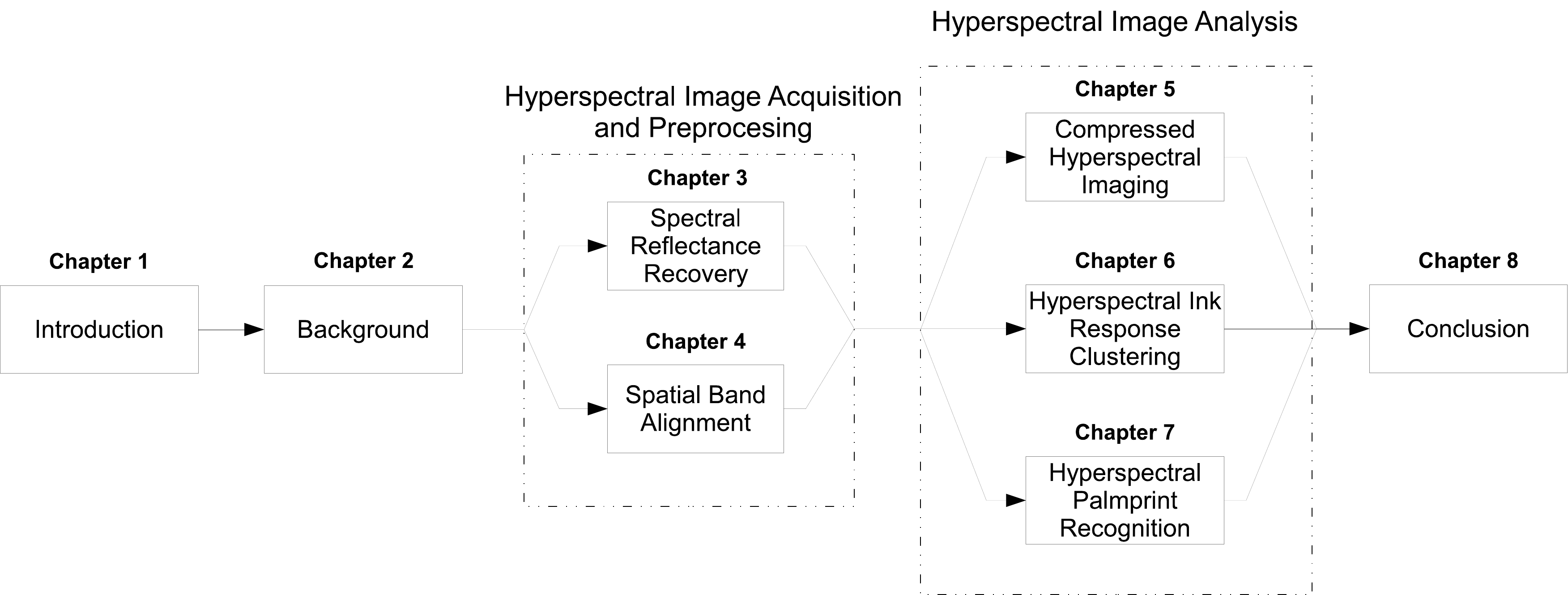}
\caption[Organization of chapters in this thesis]{Organization of chapters in this thesis. The first block comprises Chapters~\ref{Chapter3}~and~\ref{Chapter4} which deal with the normalization of hyperspectral images. The second block comprises of Chapters~\ref{Chapter5},~\ref{Chapter6} and~\ref{Chapter7} on reconstruction and recognition from hyperspectral images which pertain to the major contribution of the thesis.}
\label{fig:layout}
\end{figure}
\end{landscape}

\subsection[Cross-Spectral Registration]{Cross-Spectral Registration of Hyperspectral Face Images\\(Chapter~4)}

Spatial misalignment of hyperspectral images is a challenging phenomenon that can occur during image acquisition of live objects. The consecutive bands of a hyperspectral image are not registered and hence their spectra is not reliable. The spectral variation between bands is the main challenge that poses a hyperspectral image registration problem. In this chapter, a cross spectral similarity based descriptor is proposed for registration of hyperspectral image bands. Self similarity is highly robust to the underlying image modality and hence, particularly useful for hyperspectral images. Experiments are conducted on hyperspectral face images that have misalignment due to movement of subjects. The results indicate that the proposed cross spectral similarity based registration accurately realigns the bands of a hyperspectral face image.

\subsection[Compressed Hyperspectral Imaging]{Joint Group Sparse PCA for Compressed Hyperspectral Imaging (Chapter~5)}

Band selection from hyperspectral images where both spatial and spectral information are contextually important is crucial to hyperspectral image analysis. Current band selection techniques look at one factor at a time, i.e. if the selection relies on the spatial information, the spectral context is ignored and vice versa. In this chapter, this research gap is bridged by proposing a novel band selection technique which applies to spatio-spectral data. Group sparsity is introduced in PCA basis to define spatial context. Joint sparsity is simultaneously enforced to result in spectral band selection. The end result is Joint Group Sparse PCA (JGSPCA) which selects bands based on the spatio-spectral information of the hyperspectral images. The JGSPCA algorithm is validated on the problem of compressed hyperspectral imaging where JGSPCA basis is learned from training data and the hyperspectral images are reconstructed after sensing only a sparse set of bands. Experiments are performed on several publicly available hyperspectral image datasets, including the Harvard and CAVE scene database, CMU and UWA face databases. The reconstruction and recognition results show that the proposed JGSPCA consistently outperforms Sparse PCA and Group Sparse PCA.

\clearpage
\subsection[Hyperspectral Ink Mismatch Detection]{Joint Sparse PCA for Hyperspectral Ink Mismatch Detection\\(Chapter~6)}

In hyperspectral images where the spatial context is not meaningful to reconstruction, a variant of sparse PCA is proposed which solely deals with joint sparsity for band selection from hyperspectral images. A novel joint sparse band selection technique is proposed for hyperspectral ink mismatch detection by clustering of ink spectral responses. Ink mismatch detection provides important clues to forensic document examiners by identifying if some part (e.g. signature) of a note was written with a different ink compared to the rest of the note. An end-to-end camera-based hyperspectral document imaging system is designed for collection of a database of handwritten notes. Algorithmic solutions are presented to the challenges in camera-based hyperspectral document imaging. Extensive experiments show that the proposed technique selects the most fewer and informative bands for ink mismatch detection, compared to a sequential forward band selection approach.

\subsection[Hyperspectral Palmprint Recognition]{Hyperspectral Palmprint Recognition (Chapter~7)}

Palmprints have emerged as a new entity in multi-modal biometrics for human identification and verification. Hyperspectral palmprint images captured in the visible and infrared spectrum not only contain the wrinkles and ridge structure of a palm, but also the underlying pattern of veins; making them a highly discriminating biometric identifier. In this chapter, a representation and encoding scheme for robust and accurate matching of hyperspectral palmprints is proposed. To facilitate compact storage of the feature, a binary hash table structure is designed that allows for efficient matching in large databases. Comprehensive experiments for both identification and verification scenarios are performed on three public datasets -- two captured with a contact-based sensor (PolyU-MS and PolyU-HS dataset), and the third with a contact-free sensor (CASIA-MS dataset). Recognition results in various experimental setups show that the proposed method consistently outperforms existing state-of-the-art methods. Error rates achieved by our method are the lowest reported in literature on all datasets and clearly indicate the viability of hyperspectral imaging in palmprint recognition.

\clearpage

\section{Research Contributions}

The major contributions of the thesis are summarized as follows
\begin{itemize}
\item An automatic exposure adjustment based hyperspectral imaging technique is proposed for illumination recovery. The efficacy of the technique is demonstrated by comparison to traditional fixed exposure imaging in recovery of illumination.
\item An illuminant estimation and reflectance recovery technique from hyperspectral images is presented. The accuracy of the technique is validated in simulated and real illumination hyperspectral scenes of an in-house developed multi illuminant hyperspectral scene database.
\item A self similarity based descriptor is proposed for cross spectral hyperspectral image registration. The algorithm caters for the inter-band misalignments during hyperspectral face image acquisition.
\item Joint Sparse Principal Component Analysis (JSPCA) is proposed which jointly preserves the spectral responses of the hyperspectral images. An application to band selection for hyperspectral ink mismatch detection is demonstrated on an in-house developed database.
\item Joint Group Sparse Principal Component Analysis (JGSPCA) is presented which jointly preserves the spatio-spectral structure of hyperspectral images. An application to compressed hyperspectral imaging and hyperspectral face recognition is demonstrated on various datasets, including an in-house developed hyperspectral face database.
\item A multidirectional feature encoding and binary hash table matching technique is proposed for hyperspectral palmprint recognition. The proposed Joint Group Sparse PCA is used for band selection from hyperspectral palmprint images which outperforms existing band selection techniques.
\end{itemize}


\chapter{Background} 
\label{Chapter2} 



This chapter presents some of the foundational concepts and ideas that are crucial to the understanding of the developments proposed in this thesis. In Section~\ref{sec:sparse}, linear regression, regularization and principal component analysis which are the core ideas concerning reconstruction and recognition techniques are briefly introduced. In Section~\ref{sec:hs-imaging}, the multispectral and hyperspectral imaging techniques developed in the past are presented. This study paves the way for the hyperspectral imaging technique presented in this thesis. In Section~\ref{sec:hs-overview}, a brief survey of the spectral image analysis in computer vision and pattern recognition is provided. The scope of this survey is limited to the multispectral and hyperspectral imaging systems used in ground-based computer vision applications. Therefore, high cost and complex sensors for remote sensing, astronomy, and other geo-spatial applications are excluded from the discussion.

\section{Sparse Reconstruction and Recognition}
\label{sec:sparse}

Supervised learning aims to model the relationship between the observed data $\mathbf{x}$ (predictor) and the external factor $\mathbf{y}$ (response). There are two main tasks in supervised learning, \emph{regression} and \emph{classification}. If the aim is to predict a continuous response variable, the task is known as regression. Otherwise, if the aim of prediction is to classify the observations into a discrete set of labels, the task is classification.

\subsection{Linear Regression}

Linear regression aims to model the relationship between a response variable and one or more predictor variables by adjusting the linear model parameters so as to reduce the sum of squared residuals to a minimum. Consider a data matrix $\mathbf{X} = {[\mathbf{x}_1,\mathbf{x}_2,...,\mathbf{x}_n]}^\intercal, \in \mathbb{R}^{n \times p}$ and its corresponding response vector $\mathbf{y} \in \mathbb{R}^n$. Linear regression ($\mathbf{y} \approx \mathbf{Xw}$) can be cast as a convex optimization problem by minimizing the following objective function
\begin{equation}
\underset{\mathbf{w}}{\arg \min} {\|\mathbf{y}-\mathbf{Xw}\|}^2~,
\end{equation}
where $\mathbf{w} \in \mathbb{R}^p$ are the model parameters or simply regression coefficients. The model$\mathbf{w}$ can be used to predict the response of a new data point. However, this form of linear regression is sensitive to noise and any outlier data sample is likely to bias the model prediction.

\subsection{Regularized Regression}

If the observation matrix is affected with noise or there are less number of predictor variables compared to the number of samples ($p<n$), the regression model is overfitted. One solution in statistical learning is to shrink the regression coefficients by penalizing the norm of $\mathbf{w}$
\begin{equation}
\underset{\mathbf{w}}{\arg \min} {\|\mathbf{y}-\mathbf{Xw}\|}^2+\lambda\|\mathbf{w}\|^2~.
\end{equation}
The added ridge penalty terms shrinks to coefficients corresponding to noisy predictors so as to reduce the residual error. The parameter $\lambda$ controls the bias/variance tradeoff of the model. Higher value of $\lambda$ results in lower bias and higher variance.

Consider a regression problem with $k$ tasks, such that the response variable is a vector $\mathbf{Y} = {[\mathbf{y}_1,\mathbf{y}_2,...,\mathbf{y}_n]}^\intercal, \in \mathbb{R}^{n \times k}$. The target is to seek $k$ regression vectors $\mathbf{W} = [\mathbf{w}_1,\mathbf{w}_2,...,\mathbf{w}_k], \in \mathbb{R}^{p \times k}$ which involves multiple regression tasks. A multi-task regression problem ($\mathbf{Y} \approx \mathbf{XW}$) can be formulated as
\begin{equation}
\underset{\mathbf{W}}{\arg \min} {\|\mathbf{Y}-\mathbf{XW}\|}_F^2+\lambda\|\mathbf{W}\|^2~,
\end{equation}
where $\|.\|_F$ is the Frobenius norm defined as $\sqrt{\sum_i \sum_j w_{ij}^2}$.

\subsection{Sparse Multi-Task Regression}

Each coefficient of a regression vector corresponds to the linear combination of all the predictor variables to get an approximate response. In some instances, it is required to use only a few predictor variables which are most informative to the approximation of a response variable. Sparsity inducing norms allow only a few non-zero coefficients in a regression vector, while achieving the closest approximation to the response variable.
\begin{equation}
\label{eq:sparseMTR}
\underset{\mathbf{W}}{\arg \min} {\|\mathbf{Y}-\mathbf{XW}\|}_F^2+\lambda\,\psi(\mathbf{W})~.
\end{equation}
The first term of the objective function can be interpreted as the reconstruction loss term which minimizes the difference between the data and its approximate representation. The function $\psi(.)$ is a cost function aimed at forcing the representation (linear combination) to be sparse. It could generally be
\begin{itemize}
\item The $\ell_0$ pseudo norm, ${\|\mathbf{w}\|}_0 \triangleq n\{i | \quad w_i \neq 0 \}$ (non-convex)
\item The $\ell_1$ norm, ${\|\mathbf{w}\|}_1 \triangleq \sum_{i=1}^{p} |w_i|$ (convex)
\end{itemize}
The $\ell_0$ norm is non-differentiable and its solution is NP-hard~\cite{natarajan1995sparse}. A convex relaxation to the $\ell_0$ norm in the form of $\ell_1$ norm is most common choice for sparsity~\cite{tibshirani1996regression,zou2006adaptive,meinshausen2006high}.

\subsection{Principal Component Analysis}

Principal Component Analysis (PCA) is a useful transformation for data interpretation and visualization. It highlights the patterns of data distribution, and the interaction of various factors that make the data. It also allows a simplified graphical representation of high dimensional data by reducing the least significant dimensions of the transformed data. The principal components of a data can be computed in many different ways. The Karhunen-Lo$\grave{e}$ve Transform~\cite{jolliffe2005principal}, Singular Value Decomposition(SVD), and the Power Method~\cite{davidson1975iterative} are some of the well known tools.

The SVD based PCA computation is explained further because of its widespread use and better numerical accuracy. Consider a data matrix $\mathbf{X} \in \mathbb{R}^{n \times p}$ of $n$ observations and $p$ features. Each row of $\mathbf{X}$ is an observation, each column corresponds to a feature. Before any further steps, it is important to normalize the data matrix by subtracting the mean $\bar{\mathbf{x}} \in \mathbb{R}^{p}$ from each row of $\mathbf{X}$. This results in a centralized data whose mean is zero. Then, SVD of the data matrix is computed. It is a form of matrix factorization technique and an efficient and accurate tool for computing all the eigenvalues/eigenvectors of a matrix. Many algorithms have been implemented for its efficient computation which are present in statistical libraries of most programming languages.

\begin{figure}[h]
\centering
\subfigure[Data]{\includegraphics[width=0.3\linewidth]{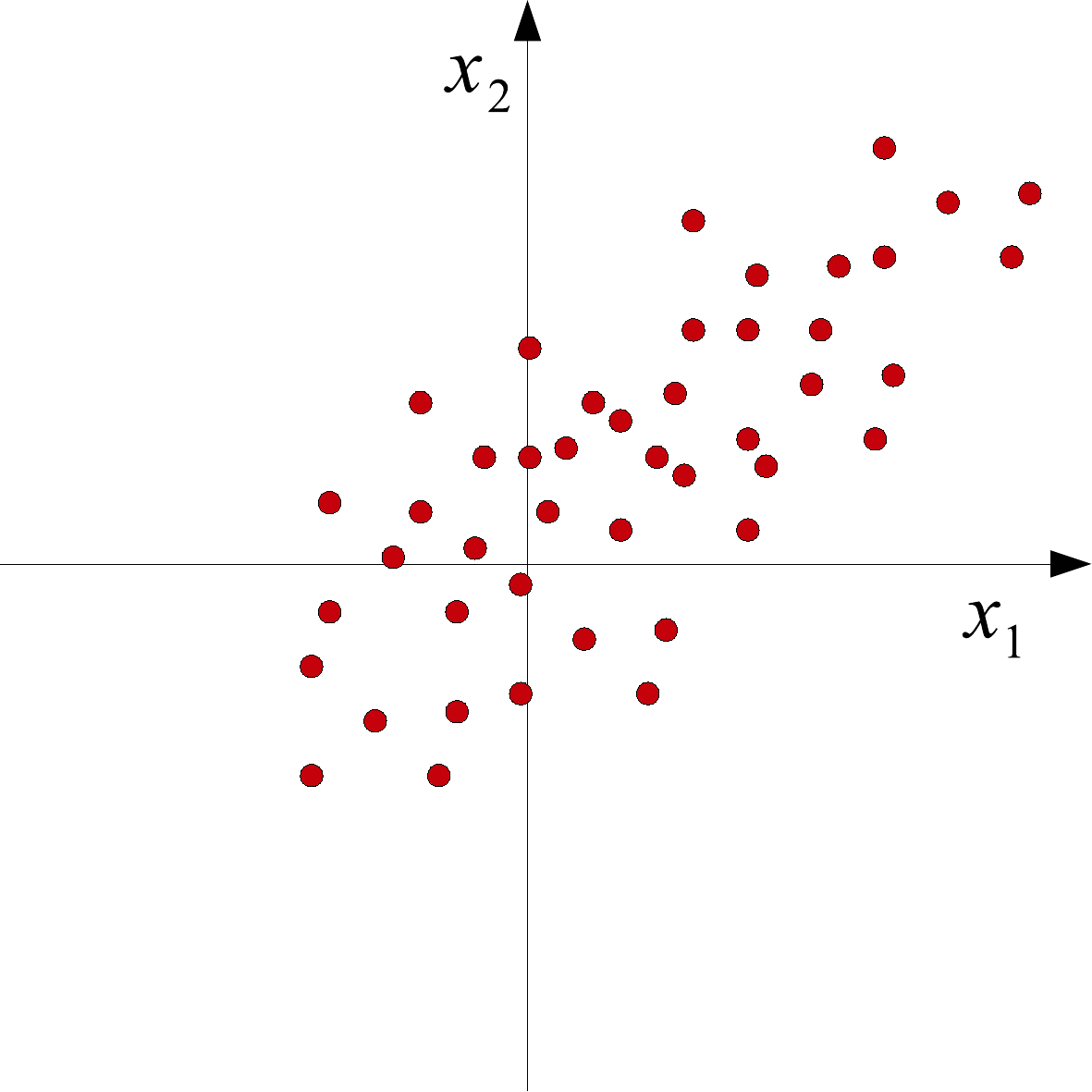}}
\subfigure[PC directions]{\includegraphics[width=0.3\linewidth]{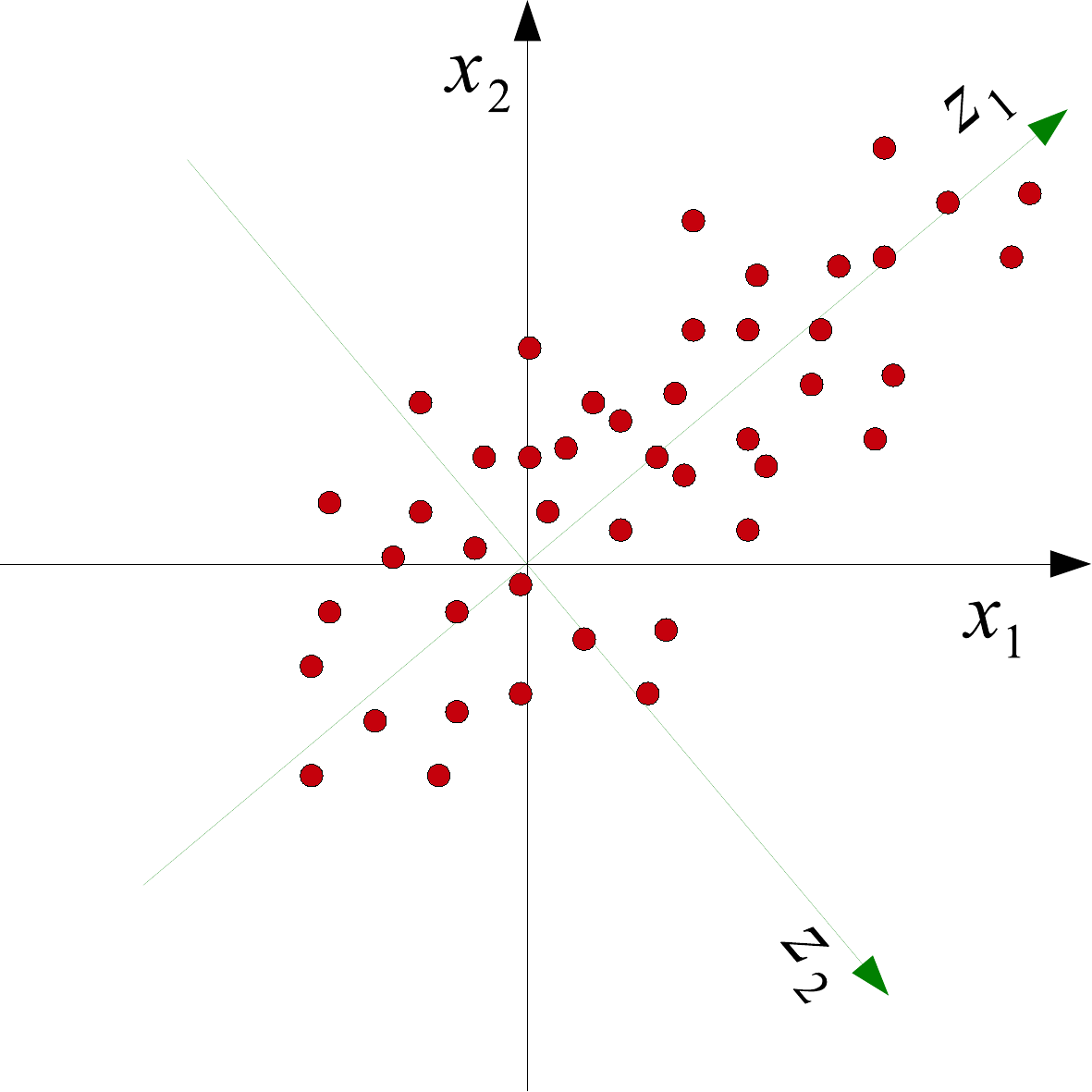}}
\subfigure[Projection on $1^{\textrm{st}}$ PC]{\includegraphics[width=0.3\linewidth]{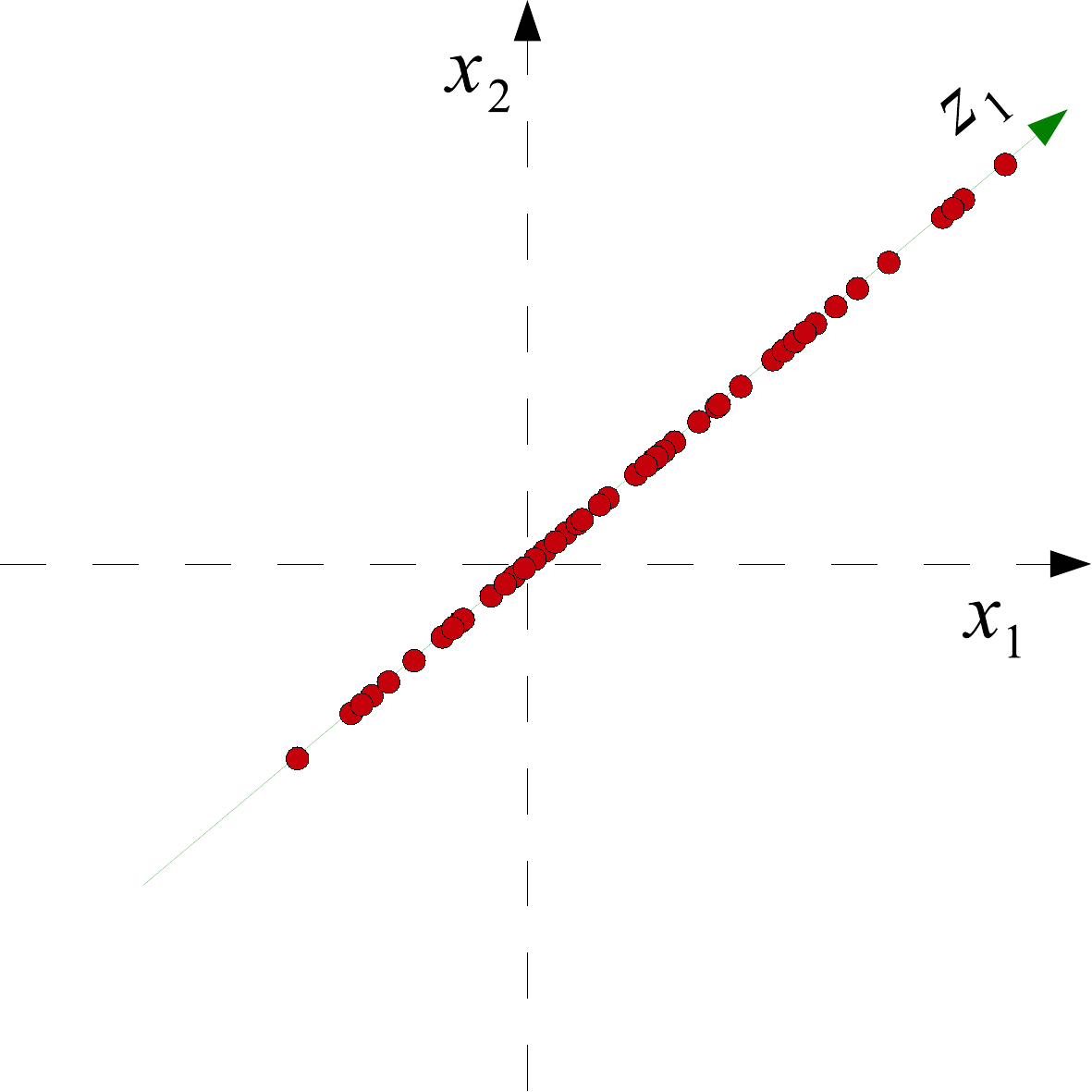}}
\caption[Principal component analysis.]{Principal component analysis. Note that the principal components are orthogonal and the new axes ($z_1,z_2$) is a rotation of original axes ($x_1,x_2$). The first PC direction is aligned with the direction of maximum variation of the data. The second is aligned with the next maximum, which is orthogonal to the first PC direction. Projection of the data on the first PC direction reduces dimensions of the original data.}
\label{fig:pca}
\end{figure}
\clearpage

\noindent The SVD factorizes a data matrix such that
\begin{equation}
\mathbf{X} = \mathbf{USV}^\intercal~,
\end{equation}
where $\mathbf{S}$ is a positive diagonal matrix of singular values (square root of eigenvalues) of $\mathbf{X}$. $\mathbf{U} \in \mathbb{R}^{n \times p}$ is a (row) orthonormal matrix also known as the left singular matrix. $\mathbf{V} \in \mathbb{R}^{p \times p}$ is a (column) orthonormal matrix which has the eigenvectors of matrix $\mathbf{X}$. The eigenvectors are generally referred to as the  basis vectors in the context of PCA. The eigenvalue corresponding to each eigenvector determines the contribution of that principal component in the variance of data.

The original data matrix can be projected on the PCA subspace as $\mathbf{Z}=\mathbf{XV} \in \mathbb{R}^{n \times k}$, where $k$ is the number of PC dimensions to retain. Figure~\ref{fig:pca} shows PCA on an example data.

\subsection{PCA Example: Portland Cement Data}

Let us begin with PCA on an example dataset. The Portland Cement Data~\cite{woods1932effect} contains the relative proportion of 4 ingredients in 13 different samples of cement and the heat of cement hardening after 180 days. The data is given in Table~\ref{tab:data-pca}.

\begin{table}[h]
\caption{The Portland Cement Data.}
\label{tab:data-pca}
\centering
{\footnotesize
\begin{tabular}{|c|c|c|c|c||r|} \hline
Sample& tricalcium& tricalcium & tetracalcium & beta-dicalcium & heat \\
No.& aluminate & silicate & aluminoferrite &  silicate  & \\
$3CaO.Al_2O_3$ &  $3CaO.SiO_2$ &  $4CaO.Al_2O_3.Fe_2O_3$ & $2CaO.SiO_2$ & (cal/gm) \\ \hline
1  & 7  & 26 & 6  & 60& 78.5    \\
2  & 1  & 29 & 15 & 52& 74.3    \\
3  & 11 & 56 & 8  & 20& 104.3   \\
4  & 11 & 31 & 8  & 47& 87.6    \\
5  & 7  & 52 & 6  & 33& 95.9    \\
6  & 11 & 55 & 9  & 22& 109.2   \\
7  & 3  & 71 & 17 & 6 & 102.7   \\
8  & 1  & 31 & 22 & 44& 72.5    \\
9  & 2  & 54 & 18 & 22& 93.1    \\
10 & 21 & 47 & 4  & 26& 115.9   \\
11 & 1  & 40 & 23 & 34& 83.8    \\
12 & 11 & 66 & 9  & 12& 113.3   \\
13 & 10 & 68 & 8  & 12& 109.4   \\ \hline
\end{tabular}}
\end{table}

\begin{figure}[h]
\centering
\includegraphics[width=0.8\linewidth]{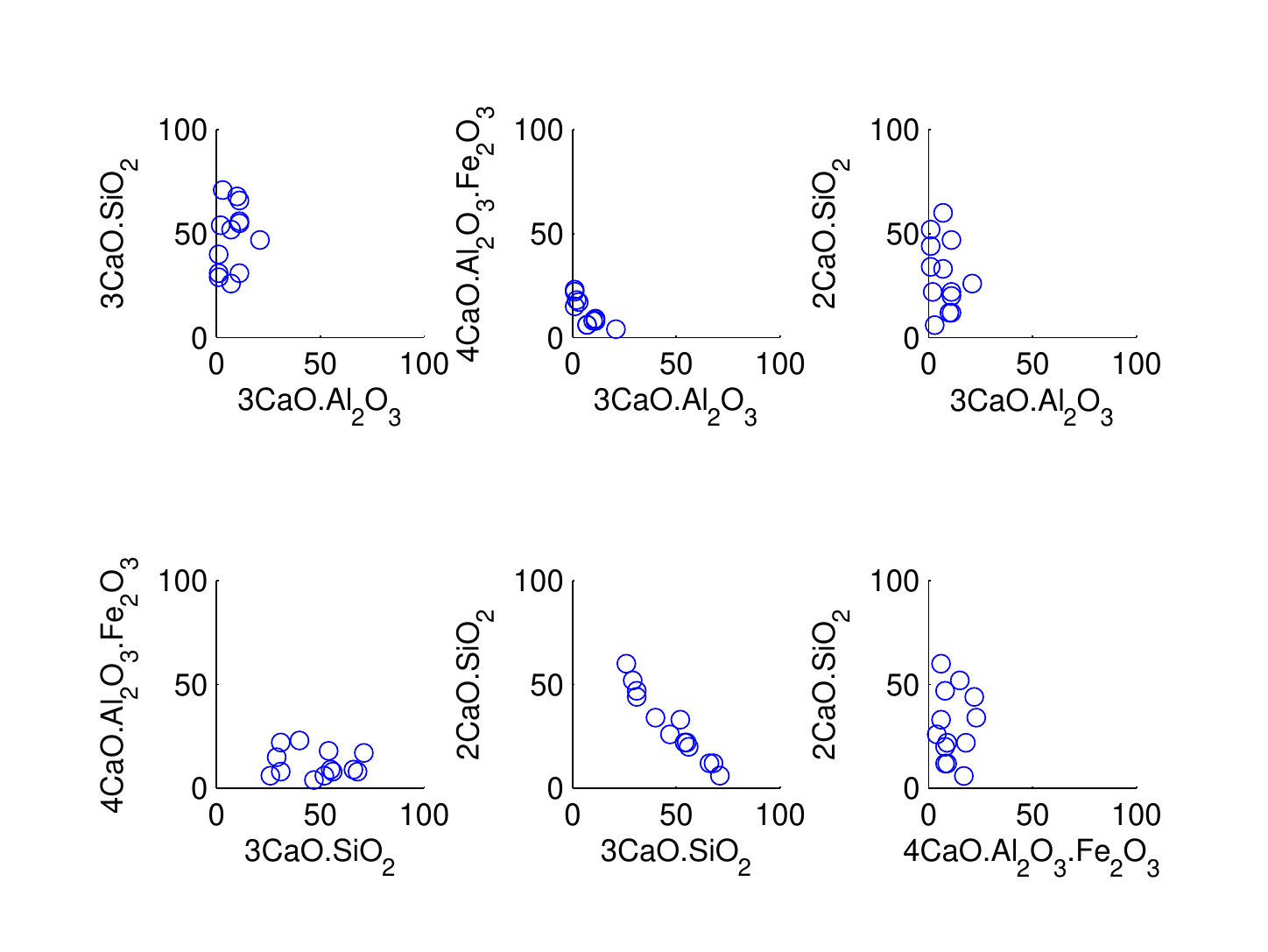}
\caption[Pairwise analysis of the ingredient proportions in data.]{Pairwise analysis of the ingredient proportions in data. Each pair can be interpreted individually but no clear inferences can be extracted on the overall interaction of the ingredients.}
\label{fig:pair}
\end{figure}

Notice that the data is 4 dimensional ($\mathbf{X} \in \mathbb{R}^{13 \times 4}, n=13, p=4$) and it is not possible to graphically observe the distribution of ingredient proportions altogether. It is desirable to know which ingredients are a better indicator of the heat of cement hardening. In order to observe the data graphically, only two (at most 3) ingredient proportions can be observed at a time as shown in Figure~\ref{fig:pair}. A different trend can be observed for each pair of variables. However, an overall picture of the distribution of data cannot be visually perceived. PCA makes it feasible to visualize the interaction of such data by dimensionality reduction.


In order to compute the dimensions required to be retained, a comparison of cumulative variance preserved against the number of principal components used is generally employed. The cumulative variance of the first $k$ PCA basis can be calculated as

\begin{equation}
\sigma_k = \sum_{i=1}^{k} \frac{(S_{ii})^2}{\sum_{j=1}^p(S_{jj})^2}~,
\end{equation}
where $S_{ii}$ is the $i^\textrm{th}$ eigenvalue from the diagonal matrix $\mathbf{S}$. For the cement data, it can be observed in Figure~\ref{fig:var} that the first two principal components are sufficient to explain most variation of the data (98\%).


The original data $\mathbf{X}$ can now be transformed to PCA space by $\mathbf{Z}=\mathbf{XV}$. Now it is possible to graphically represent the transformed data using the first two principal components as shown in Figure~\ref{fig:proj}.

\begin{figure}[h]
\centering
\subfigure[Total explained variance]{\label{fig:var}\includegraphics[width=0.49\linewidth]{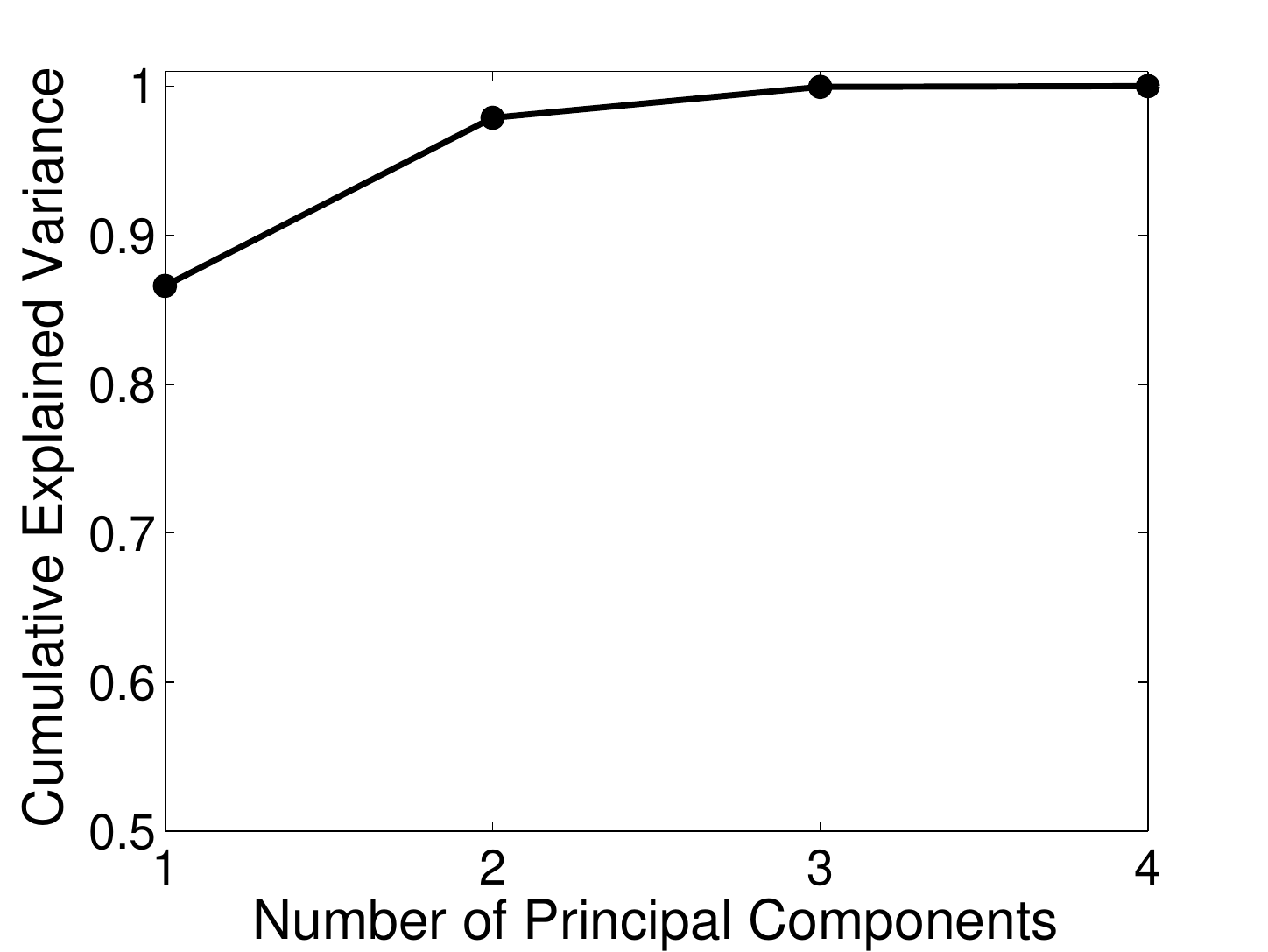}}
\subfigure[Principal component scores]{\label{fig:proj}\includegraphics[width=0.49\linewidth]{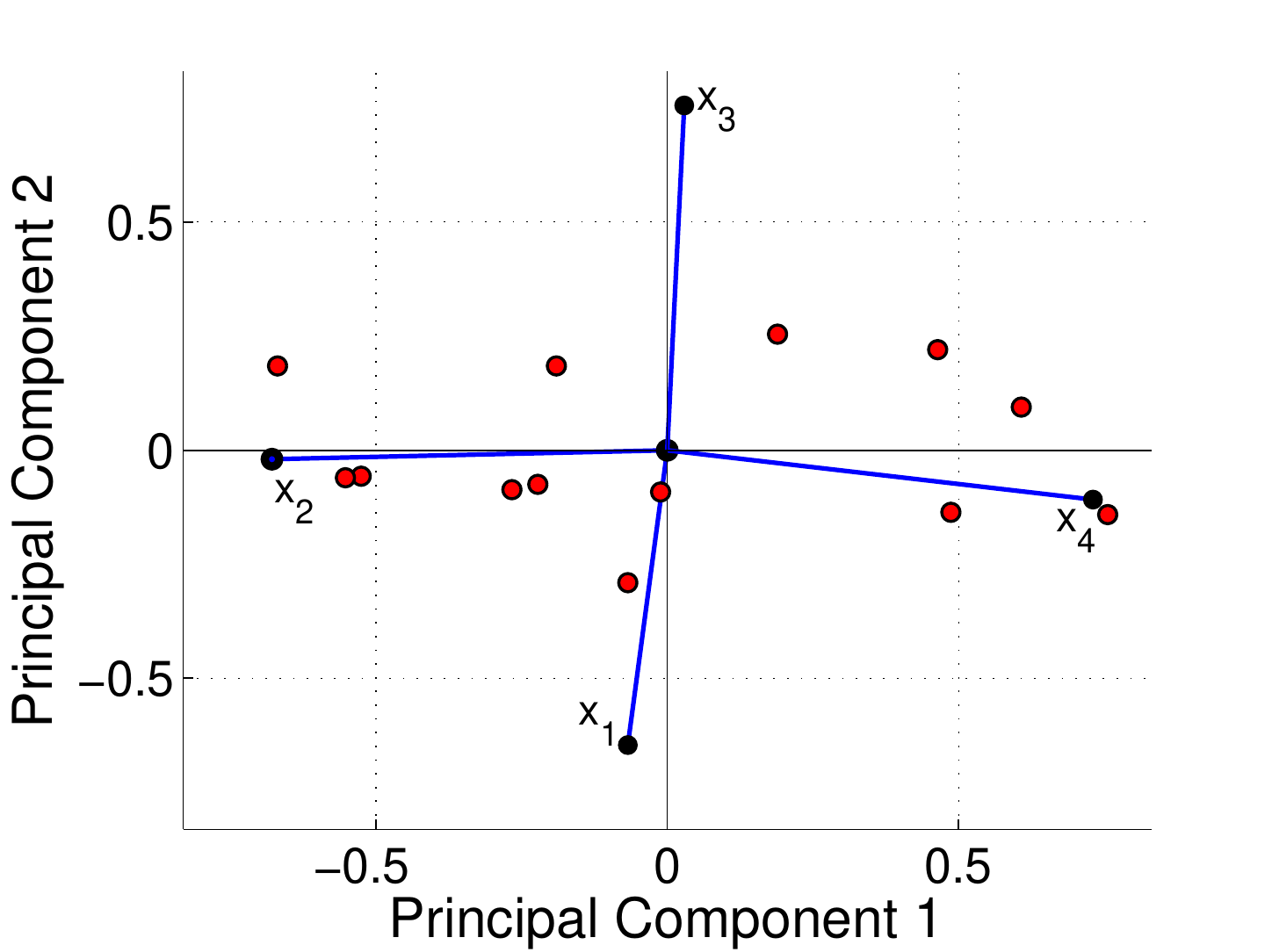}}
\caption[Explained variance and principal component scores]{(a) Cumulative variance explained by the principal components. (b) Plot of the first two principal components of the data. Notice that the second and fourth ingredients contribute the most to the first principal component.}
\label{fig:pca-final}
\end{figure}

\section{Hyperspectral Imaging}
\label{sec:hs-imaging}

The human eye exhibits a trichromatic vision. This is due to the presence of three types of photo-receptors called \emph{Cones} which are sensitive to different wavelength ranges in the visible range of the electromagnetic spectrum~\cite{martin2009retinal}. Conventional imaging sensors and displays (like cameras, scanners and monitors) are developed to match the response of the trichromatic human vision so that they deliver the same perception of the image as in a real scene. This is why an RGB image constitutes three spectral measurements per pixel.

Most of the computer vision systems do not make full use of the spectral information and only consider grayscale or color images for scene analysis. There is evidence that machine vision tasks can take the advantage of image acquisition in a wider range of electromagnetic spectrum and higher spectral resolution by capturing more information in a scene. Hyperspectral imaging captures spectral reflectance of a scene in a wide spectral range. The images can cover visible, infrared, or a combination of both ranges of the electromagnetic spectrum (see Figure~\ref{fig:em-spec}). It also provides selectivity in the choice of frequency bands for specific tasks. Satellite based spectral imaging sensors have long been used in astronomical and remote sensing applications. Due to the high cost and complexity of these sensors, various methods have been introduced to utilize conventional imaging systems combined with a few off-the-shelf optical devices for spectral imaging.

\begin{figure}[h]
\centering
\includegraphics[width=1\linewidth]{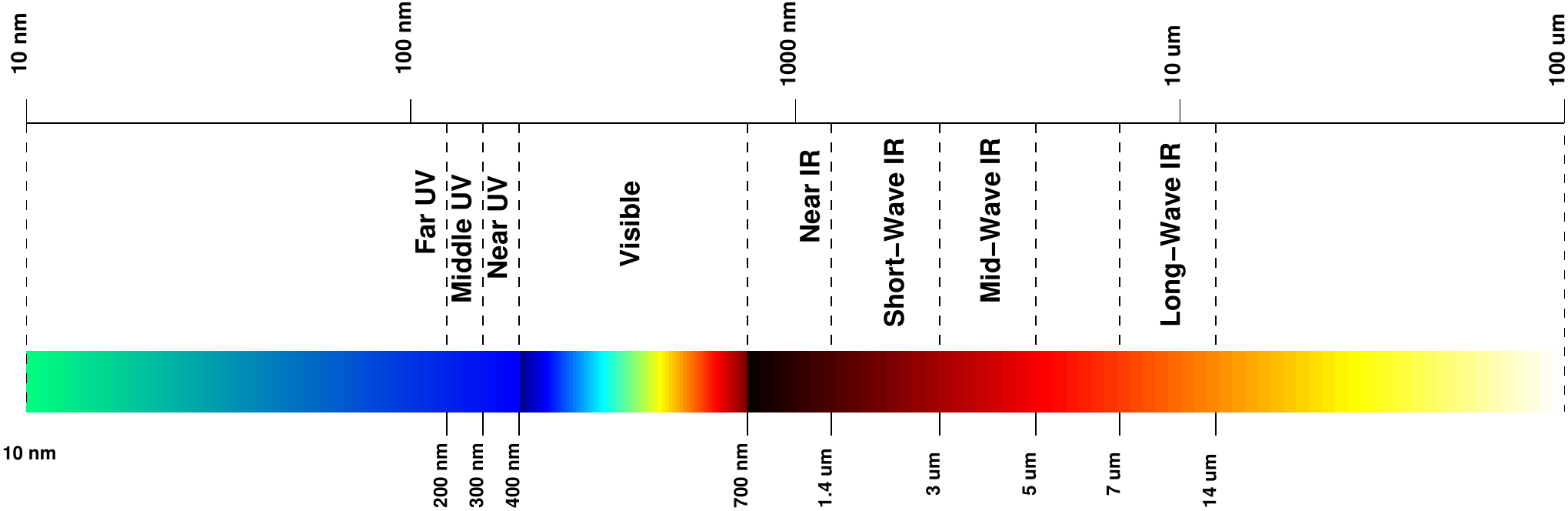}
\caption{The electromagnetic spectrum.}
\label{fig:em-spec}
\end{figure}

%

%

\subsection{Bandpass Filtering}

In filter based approach, the objective is to allow light in a specific wavelength range to pass through the filter and reach the imaging sensor. This phenomenon is illustrated in Figure~\ref{fig:Spec_Filter}. This can be achieved by using optical devices generally named bandpass filters or simply filters. The filters can be categorized into two types depending on the filter operating mechanism. The first type is the \emph{\emph{tunable}} filter or specifically the electrically tunable filter. The pass-band of such filters can be electronically tuned at a very high speed which allows for measurement of spectral data in a wide range of wavelengths. The second type is the \emph{non-tunable} filters. Such filters have a fixed pass-band of frequencies and are not recommended for use in time constrained applications. These filters require physical replacement either manually, or mechanically by a filter wheel. However, they are easy to use in relatively simple and unconstrained applications.


\subsubsection{Tunable Filters}

A common approach to acquire multispectral images is by sequential replacement of bandpass filters between a scene and the imaging sensor. The process of filter replacement can be mechanized by using a wheel of filters. Such filters are useful where time factor is not critical and the goal is to image a static scene. Kise et al.~\cite{kise2010multispectral} developed a three band multispectral imaging system by using interchangeable filter design; two in the visible range (400-700nm) and one in the near infrared range (700-1000nm). The interchangeable filters allowed for selection of three bands. The prototype was applied to the task of poultry contamination detection.

Electronically tunable filters come in different base technologies. One of the most common is the \emph{Liquid Crystal Tunable Filter} (LCTF). The LCTF is characterized by its wide bandwidth, variable transmission efficiency and slow tuning time. On the other hand, the \emph{Acousto-Optical Tunable Filter} (AOTF) is known for narrow bandwidth, low transmission efficiency and faster tuning time. For a detailed description of the composition and operating principles of the tunable filters, the readers are encouraged to read~\cite{gat2000imaging,poger2001multispectral}.

Fiorentin et al.~\cite{fiorentin2009multispectral} developed a spectral imaging system using a combination of CCD camera and LCTF in the visible range with a resolution of 5 nm. The device was used in the analysis of accelerated aging of printing color inks. The system was also applicable of monitoring the variation (especially fading) of color in artworks with the passage of time. The idea can be extended to other materials that undergo spectral changes due to illumination exposure, such as document paper and ink.

Comelli et al.~\cite{comelli2008portable} developed a portable UV-fluorescence spectral imaging system to analyze painted surfaces. The imaging setup comprised a UV-florescence source, an LCTF and a low noise CCD sensor. A total of 33 spectral images in the range (400-720nm) in 10nm steps were captured. The accuracy of the system was determined by comparison with the fluorescence spectra of three commercially available fluorescent samples measured with a bench-top spectro-fluorometer. The system was tested on a 15th century renaissance painting to reveal latent information related to the pigments used for finishing decorations in painting at various times.


%

\begin{figure}[t]
\centering
\includegraphics[width=1\linewidth]{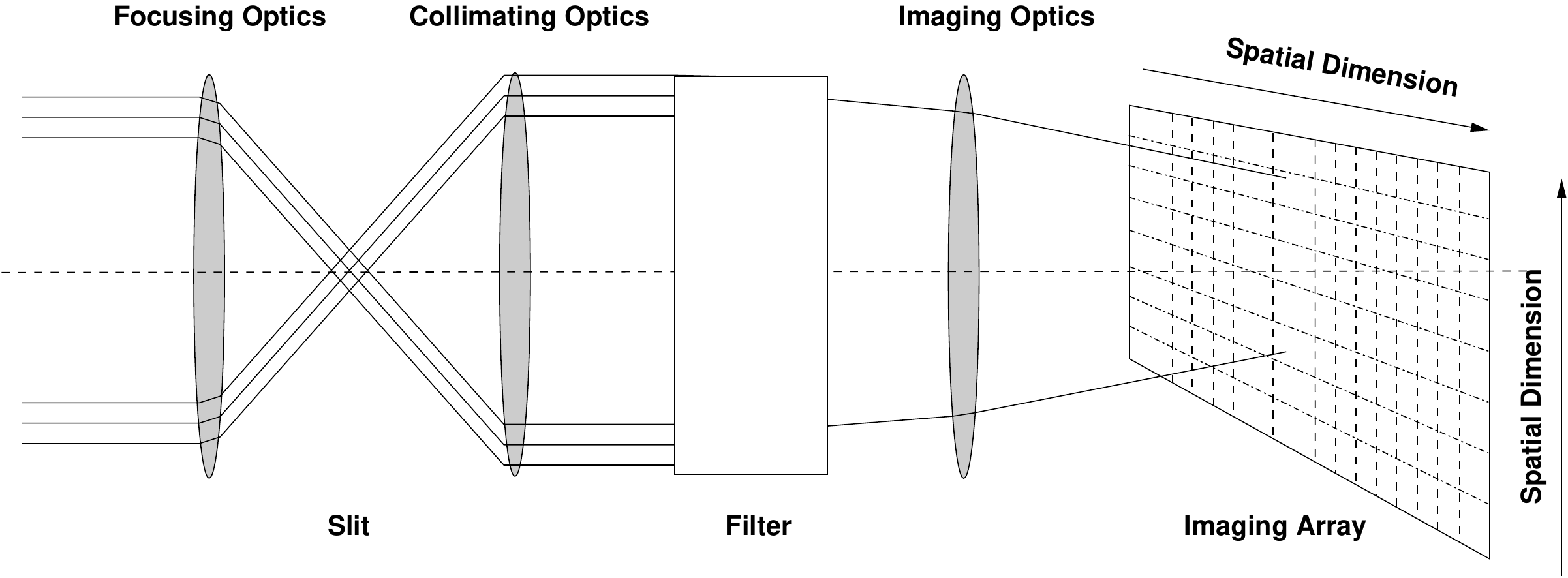}
\caption[Bandpass filtering]{In bandpass filtering, the filter allows only a specific wavelength of light to pass through, resulting in a single projection of the scene at a particular frequency.}
\label{fig:Spec_Filter}
\end{figure}

\subsubsection{Tunable Illumination}

Another approach to acquire multispectral images is by sequential tuning of bandpass filters between a scene and the illumination source. The illumination sources in different spectral bands (colors) are sequentially switched on and off to disseminate light of a specific wavelength. LED illuminations are a useful component of a spectrally variable illumination source. They are commonly available in different colors (wavelengths) for use in economical multispectral imaging systems.

A low cost, high speed system for biomedical spectral imaging is developed by Sun et al.~\cite{sun2010low}. This system comprises of a monochrome CCD camera, a high power LED illumination source and a microcontroller for synchronization. LEDs of different wavelength illumination (Red, Green, Blue) are triggered sequentially at high speeds by the microcontroller to acquire multispectral images. At a full resolution of 640 x 480 pixels, the system can capture 14-bit multispectral images at 90 frames per second. In an experimental trial, images from the cortical surface of a live rat whose brain was injected with a fluorescent calcium indicator were taken to observe its responses to electrical forepaw stimulus.

Another low cost solution to spectral imaging has been developed by Mathews et al.~\cite{mathews2008design}. This system comprised of a single large format CCD and an array of 18 lenses coupled with spectral filters. The system was able to capture multispectral images simultaneously in 17 spectral bands at a maximum resolution of 400 x 400 pixels. It was developed to observe the blood oxygenation levels in tissues for quick assessment of burns.

Park et al.~\cite{park2007multispectral} developed a multispectral imaging system comprising of a conventional RGB camera and two multiplexed illumination sources made up of \emph{white, red, amber, green} and \emph{blue} LEDs to acquire multispectral videos in visible range at 30fps. They showed that the continuous spectral reflectance of a point in a scene can be recovered by using a linear model for spectral reflectance with a reasonable accuracy. The recovered spectral measurements have been applied to the problems of material segmentation and spectral relighting. The system has been implemented in a dark controlled environment with only the multiplexed illumination sources which is likely to degrade in daylight situation.

Tunable illumination sources can also be designed by introducing different color filters in front of a uniform illumination source. Chi et al.~\cite{chi2010multi} presented a novel multispectral imaging technique using an optimized wideband illumination. A set of 16 filters were placed in the front of an illumination source used for active spectral imaging. They showed reconstruction of the spectral reflectance of objects in indoor environment in ambient illumination. Shen et al.~\cite{shen2008optimal} proposed an eigen-vector and virtual imaging based method to recover the spectral reflectance of objects in multispectral images using representative color samples for training.


\subsection{Chromatic Dispersion}

In chromatic dispersion, the objective is to decompose an incoming ray of light into its spectral constituent as shown in Figure~\ref{fig:Chrom_Disp}. This can be achieved by optical devices like diffraction prisms, gratings, \emph{grisms} (grating and prism combined) and interferometers. Chromatic dispersion can be further categorized based on refraction and interference phenomena.

\begin{figure}[h]
\centering
\includegraphics[width=1\linewidth]{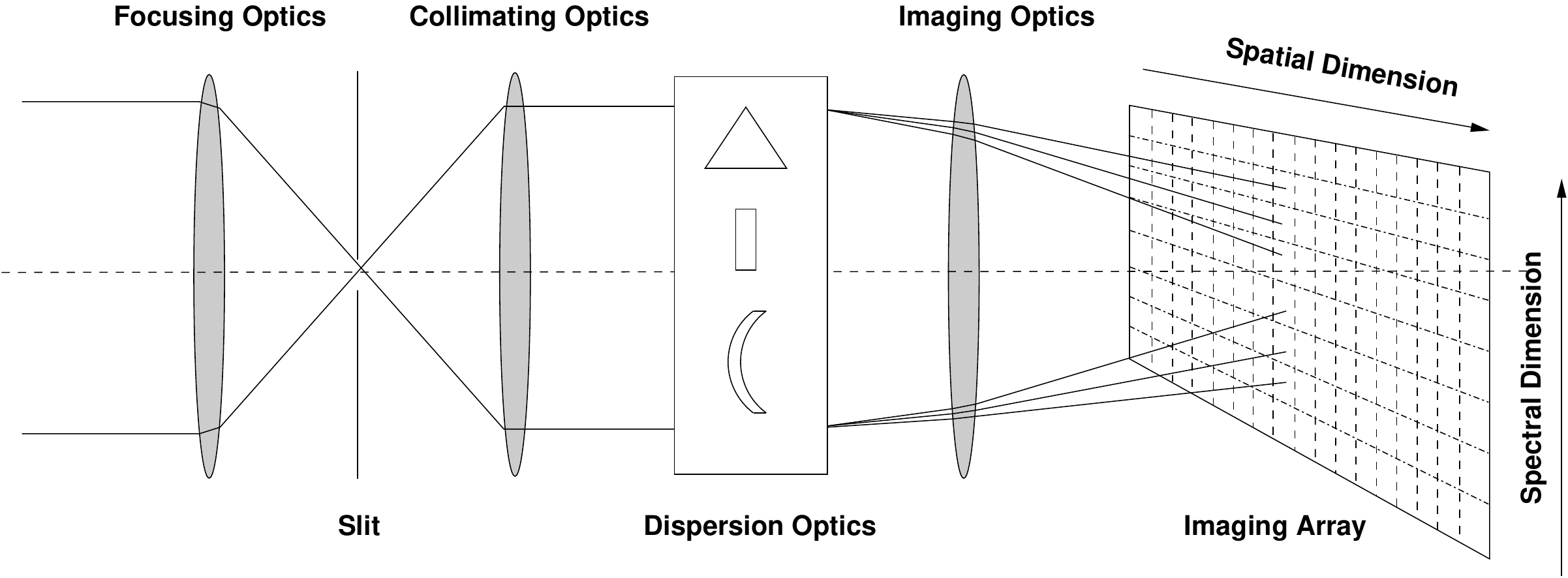}
\caption[Chromatic dispersion]{In chromatic dispersion, the dispersion optics disperses the incoming light into its constituents which are projected onto the imaging plane.}
\label{fig:Chrom_Disp}
\end{figure}

\subsubsection{Refraction Optics}

Refraction is an intrinsic property of glass-like materials such as prisms. A prism separates the incoming light ray into its constituent colors. Du et al.~\cite{du2009prism} proposed a prism-based multispectral imaging system in the visible and infrared bands. The system used an occlusion mask, a triangular prism and a monochromatic camera to capture multispectral image of a scene. Multispectral images were captured at high spectral resolution while trading off the spatial resolution. The use of occlusion mask also reduced the amount of light available to the camera and thus decreased the signal to noise ratio (SNR). The prototype was evaluated for the tasks of human skin detection and material discrimination.


Gorman et al.~\cite{gorman2010generalization} developed an \emph{Image Replicating Imaging Spectrometer} (IRIS) using an arrangement of a \emph{Birefringent Spectral De-multiplexer} (BSD) and off-the-shelf compound lenses to disperse the incoming light into its spectral components. The system was able to acquire spectral images in a snapshot. It could be configured to capture 8, 16 or 32 bands by increasing the number of stages of the BSD. High spectral resolution was achieved by trading-off spatial resolution since a 2D detector was used. The Field-of-View however, was limited by the width of the prism used in the BSD.

\subsubsection{Interferometric Optics}

Optical devices such as interferometers can be used as light dispersion devices by constructive and destructive interference. Burns et al.~\cite{burns1996analysis} developed a seven-channel multispectral imaging device using 50nm bandwidth interference filters and a standard CCD camera. Mohan et al.~proposed the idea of \emph{Agile Spectral Imaging} which used a diffraction grating to disperse the incoming rays~\cite{mohan2008agile}. A geometrical mask pattern allowed specific wavelengths to pass through and reach the sensor.

Descour et al.~\cite{descour1995computed} presented a \emph{Computed Tomography Imaging Spectrometer (CTIS)} using three sinusoidal phase gratings to disperse light into multiple directions and diffraction orders. Assuming the dispersed images to be two dimensional projections of three dimensional multispectral cube, the multispectral cube was reconstructed using maximum-likelihood expectation maximization algorithm. Their prototype was able to reconstruct multispectral images of a simple target in the visible range (470-770nm).

\section{Hyperspectral Image Analysis}
\label{sec:hs-overview}

During the past several years spectral imaging has found its utility in various ground-based applications, some of which are listed in Table~\ref{tab:applications}. The use of spectral imaging in archeological artifacts restoration has shown promising results. It is now possible to read the old illegible historical manuscripts by restoration using spectral imaging~\cite{baronti1997principal}. This was a fairly difficult task for a naked eye due to its capability restricted to the visible spectrum. Similarly, spectral imaging has also been applied to the task of material discrimination. This is because of the physical property of a material to reflect a specific range of wavelengths giving it a spectral signature which can be used for material identification~\cite{thai2002invariant}. The greatest advantage of spectral imaging in such applications is that it is non-invasive and thus does not affect the material under analysis compared to other invasive techniques which inherently affect the material under observation.

\begin{table}[h]
\caption{Applications of spectral imaging in different areas.}
\label{tab:applications}
\begin{center}
\begin{tabular}{l|l}
\hline
\textbf{Areas}      & \textbf{Applications} \\ \hline \hline
Art and Archeology  & Analysis of works of art, historical artifact restoration \\ \hline
Medical Imaging     & MRI imaging, microscopy, biotechnology \\ \hline
Security            & Surveillance, biometrics, forensics \\ \hline
\end{tabular}
\end{center}
\end{table}

\subsection{Security Applications}
\label{sec:spectral-biometrics}

The bulk of computer vision research for security applications revolves around monochromatic imaging. Recently, different biometric modalities have taken advantage of spectral imaging for reliable and improved recognition. The recent work in palmprint, face, fingerprint, and iris recognition using spectral imaging is briefly discussed below.

\subsubsection{Palmprint Recognition}

Palmprints have emerged as a popular choice for human access control and identification. Interestingly, palmprints have even more to offer when imaged under different spectral ranges. The line pattern is captured in the visible range while the vein pattern becomes apparent in the near infrared range. Both line and vein information can be captured using a spectral imaging system such as those developed by Han et al.~\cite{han2008multispectral} or Hao et al.~\cite{hao2008multispectral}.

Multispectral palmprint recognition system of Han et al.~\cite{han2008multispectral} captured images under four different illuminations (red, green, blue and infrared). The first two bands (blue and green) generally showed only the line structure, the red band showed both line and vein structures, whereas the infrared band showed only the vein structure. These images can be fused for subsequent matching and recognition. The contact-free imaging system of Hao et al.~\cite{hao2008multispectral} acquires multispectral images of a palm under six different illuminations. The contact-free nature of the system offers more user acceptability while maintaining a reasonable accuracy. The accuracy achieved by multispectral palmprints is much higher compared to traditional monochromatic systems.

\subsubsection{Fingerprint Recognition}

Fingerprints have established as one of the most reliable biometrics and are in common use around the world. Fingerprints can yield even more robust features when captured under a multispectral sensor. Rowe et al.~\cite{rowe2005multispectral} developed a spectral imaging sensor for fingerprint imaging. The system comprised of illumination source of multiple wavelengths (400, 445, 500, 574, 610 and 660nm) and a monochrome CCD of 640x480 resolution. They showed in comparison to traditional sensors, spectral imaging sensors are less affected by moisture content of skin. Recognition based on multispectral fingerprints outperformed traditional fingerprints.

\subsubsection{Face Recognition}

Face recognition has an immense value in human identification and surveillance. The spectral response of human skin is a distinct feature which is largely invariant to the pose and expression~\cite{pan2003face} variation. Moreover, multispectral images of faces are less susceptible to variations in illumination sources and their directions~\cite{chang2008multispectral}. Multispectral face recognition systems generally use a monochromatic camera coupled with a \emph{Liquid Crystal Tunable Filter} (LCTF) in the visible and/or near-infrared range.


\subsubsection{Iris Recognition}

Iris is another unique biometric used for person authentication. Boyce et al.~\cite{boyce2006multispectral} explored multispectral iris imaging in the visible electromagnetic spectrum and compared it to the near-infrared in a conventional iris imaging systems. The use of multispectral information for iris enhancement and segmentation resulted in improved recognition performance.

\subsection{Material Identification}
\label{sec:spectral-material}

Naturally existing materials show a characteristic spectral response to incident light. This property of a material can distinguish it from other materials. The use of multispectral techniques for imaging the works of arts like paintings allows segmentation and classification of painted parts. This is based on the pigment physical properties and their chemical composition~\cite{baronti1997principal}.

\subsubsection{Pigment Identification}

Pelagotti et al.~\cite{pelagotti2008multispectral} used multispectral imaging for analysis of paintings. They collected multispectral images of a painting in UV, Visible and Near IR band. It was possible to differentiate among different color pigments which appear similar to the naked eye based on spectral reflectance information.

\subsubsection{Ice Accumulation Detection}

Gregoris et al.~\cite{gregoris2004multispectral} exploited the characteristic reflectance of ice in the infrared band to detect ice on various surfaces which is difficult to inspect manually. The developed prototype called \emph{MD Robotics' Spectral Camera system} could determine the type, level and location of the ice contamination on a surface. The prototype system was able to estimate thickness of ice ($<$0.5mm) in relation to the measured spectral contrast. Such system may be of good utility for aircraft/space shuttle ice contamination inspection and road condition monitoring in snow conditions.

\subsubsection{Medical Image Analysis}

Multispectral imaging has critical importance in magnetic resonance imaging. Multispectral magnetic resonance imagery of brain is in wide use in medical science. Various tissue types of the brain are distinguishable by virtue of multispectral imaging which aids in medical diagnosis~\cite{taxt1994multispectral}.

\subsubsection{Concrete Moisture Estimation}

Clemmensen et al.~\cite{clemmensen2010comparison} used multispectral imaging to estimate the moisture content of sand used in concrete. It is a very useful technique for non-destructive in-vivo examination of freshly laid concrete. A total of nine spectral bands was acquired in both visual and near infrared range. Zawada et al.~\cite{zawada2003image} proposed a novel underwater multispectral imaging system named \emph{LUMIS} (Low light level Underwater Multispectral Imaging System) and demonstrated its use in study of phytoplankton and bleaching experiments.

\subsubsection{Food Quality Inspection}

Fu et al.~\cite{fu2011discriminant} identified optimal absorption band segments to characterize the dissimilarity between materials using probabilistic and supervised learning. They optimal absorption feature band segments were used for discrimination of normal and rusted wheat and classification of dry fruit via hyperspectral imaging.

Spectrometry techniques can identify the fat content in meat, as it can be economical, efficient and non-invasive compared to traditional analytical chemistry methods~\cite{thodberg1996review}. For this purpose, near-infrared spectrometers have been used to measure the spectrum of light transmitted through a sample of minced meat.

\chapter[Spectral Reflectance Recovery from Hyperspectral Images]{Spectral Reflectance Recovery from Hyperspectral Images} 

\label{Chapter3} 

The appearance of a scene changes with the spectrum of ambient illumination~\cite{golz2002influence}. The human visual system has an intrinsic capability of recognizing colored objects under different illuminations~\cite{land1977retinex}. In machine vision systems, it is desirable to remove the effect of illumination, so as to measure the true spectral reflectance of the objects in a scene\cite{katravsnik2013method}. This is because in object detection, segmentation~\cite{brill1990image} and recognition~\cite{healey1994global}, an illumination invariant view of the object is critical to achieving accurate results~\cite{gegenfurtner2003cortical}.

Color constancy refers to the removal of the extrinsic color cast by an illumination in a scene~\cite{gijsenij2011computational}. It is synonymously viewed as the recovery of spectral reflectance under certain assumptions on scene illumination~\cite{maloney1986color}. A bulk of color constancy research is focused on dealing with the trichromatic images~\cite{li2011evaluating,van2007edge,van2007using,forsyth1990novel,gevers2000color}. With the advances in sensor technology, hyperspectral imaging is claiming profound interest in medicine, art and archeology, and computer vision~\cite{wang2003detection,calcagni2011multispectral,baronti1997principal,pelagotti2008multispectral}. In hyperspectral imaging, recovery of spectral reflectance remains a challenge in a much higher dimension~\cite{gu2013efficient}. Figure~\ref{fig:sample-images} illustrates this phenomenon in the analysis of art works, such as paintings. Although, color constancy has been explored for remotely sensed hyperspectral images~\cite{wiemker1997color}, only few studies investigated the problem with focus on ground based hyperspectral imaging systems~\cite{hashimoto2011multispectral,fauch2010recovery,shen2007reflectance}.

In contrast to color imaging, hyperspectral imaging involves complex optical components that capture the reflectance spectra in narrow bands. The basic principle of spectral imaging is to disperse/filter incoming light with dispersion optics or bandpass filters. Chromatic dispersion using prisms~\cite{du2009prism}, grating~\cite{mohan2008agile} or interferometers~\cite{burns1996analysis} separates light into its constituent colors and simultaneously acquires a spatial and a spectral dimension. The second spatial dimension is acquired by moving the imaging system. Therefore, it involves motion which inherently suffers from noise. In contrast, filter based spectral imaging simultaneously acquires two spatial dimensions, whereas the spectral dimension is sequentially acquired by tuning the filter frequency. This method is suitable for static objects, which are of interest pertaining to ground based hyperspectral imaging systems.

\clearpage

\begin{figure}[h]
\centering
\includegraphics[width=0.26\linewidth]{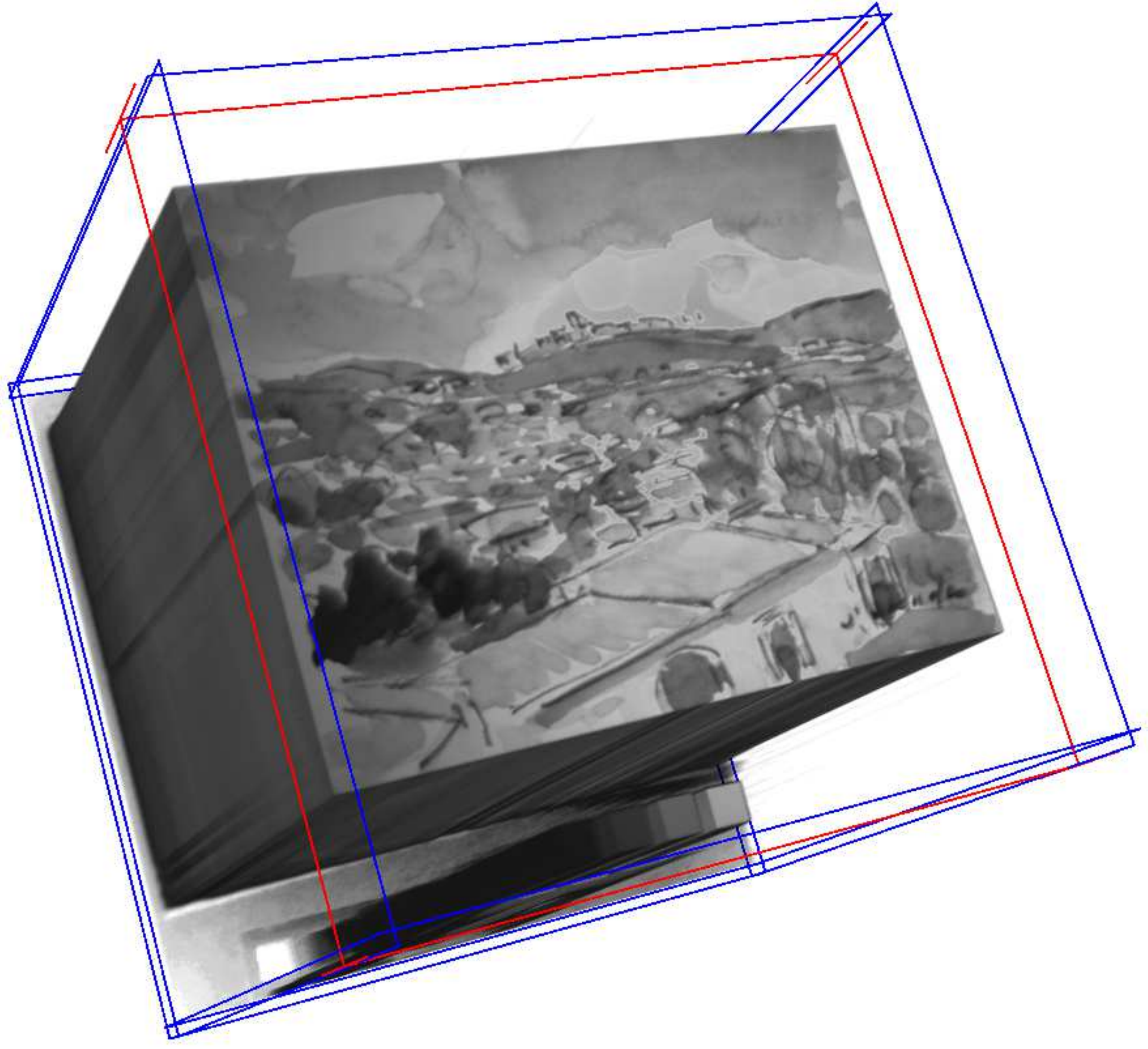}
\subfigure[Uniform]{\includegraphics[width=0.2\linewidth]{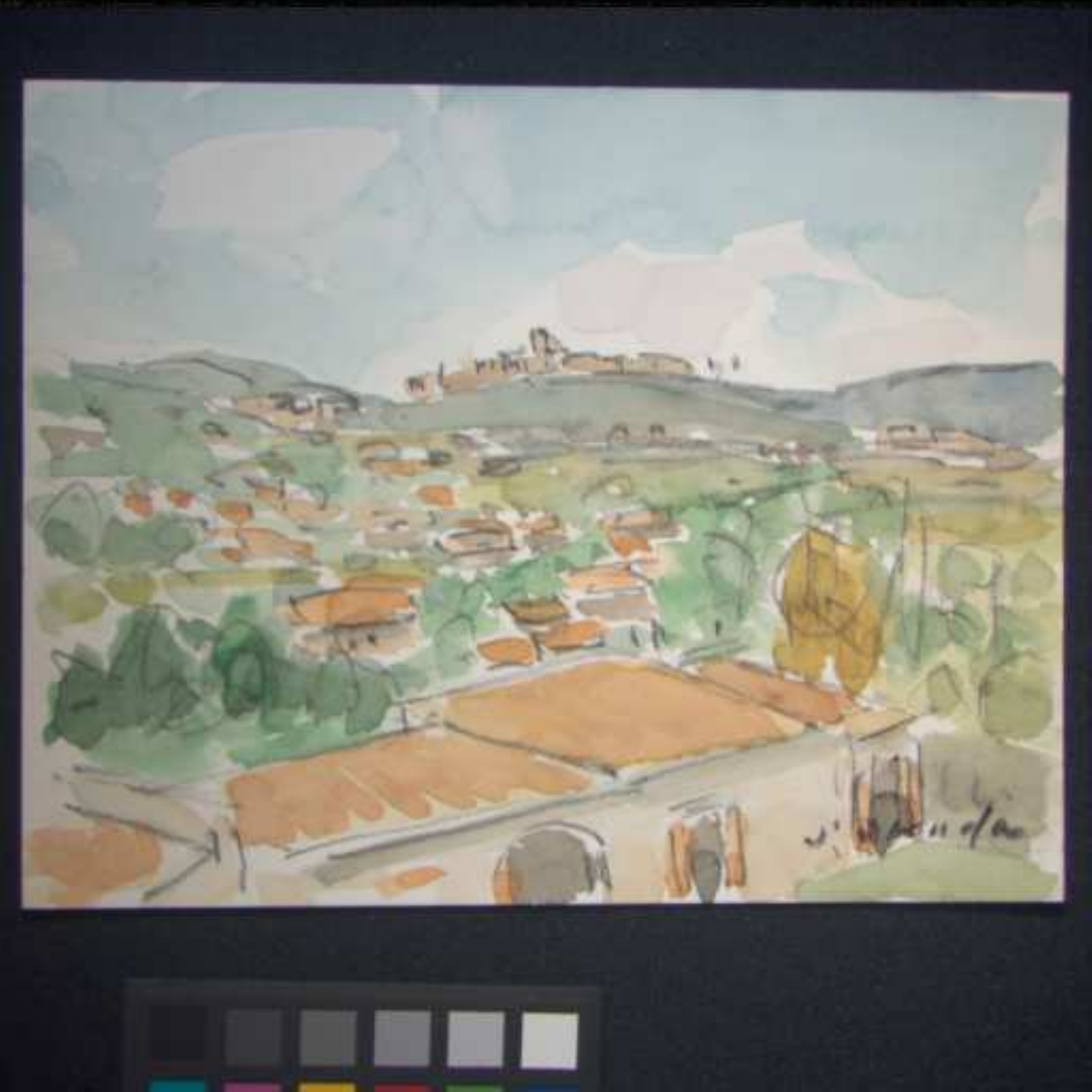}}
\subfigure[Fluorescent]{\includegraphics[width=0.2\linewidth]{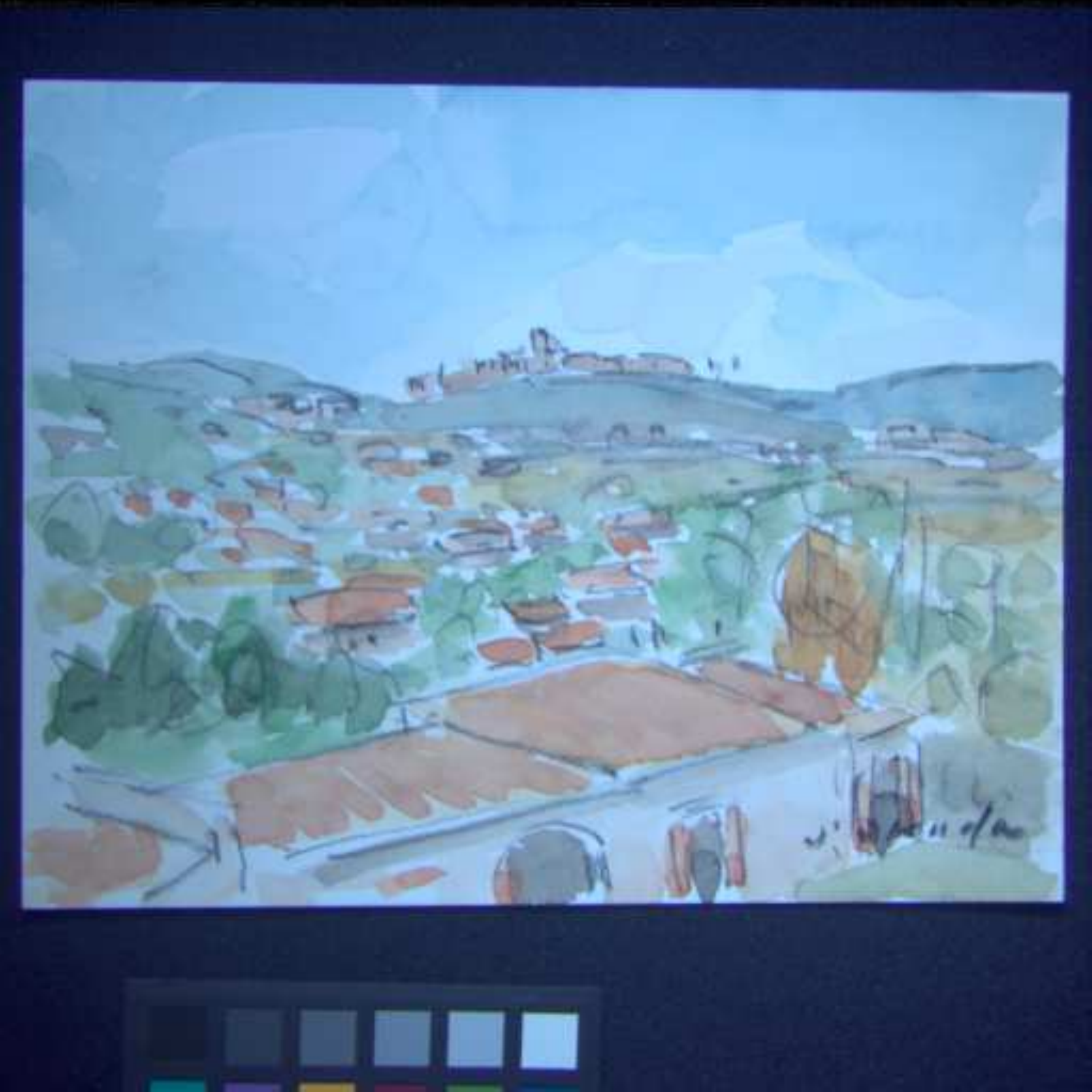}}
\subfigure[Halogen]{\includegraphics[width=0.2\linewidth]{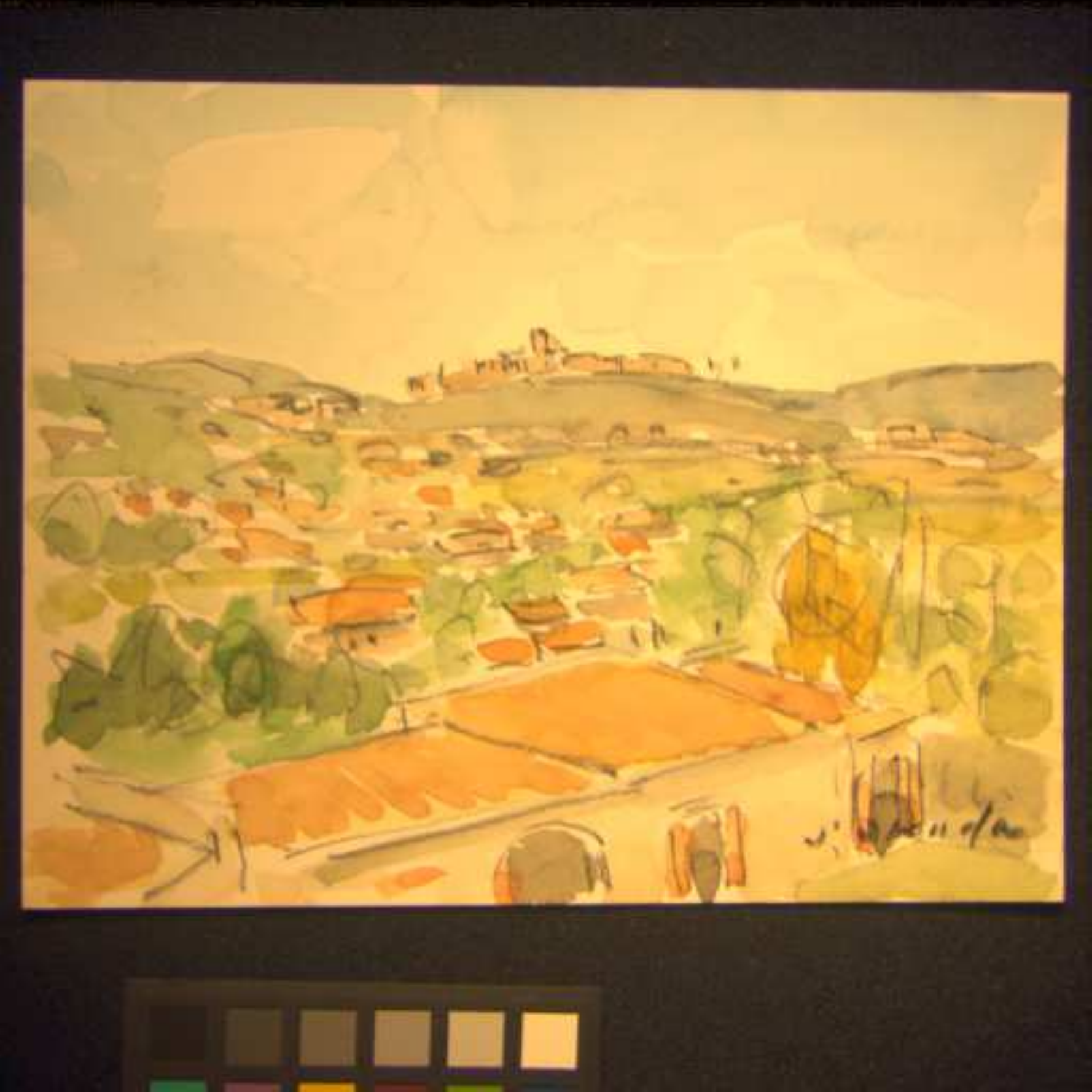}}
\caption[Hyperspectral image of a painting under different illuminations]{Hyperspectral image cube of a watercolor painting. Appearance of the painting under different illuminations is visualized as RGB rendering. Observe that the apparent colors of the painting under non-uniform illumination are significantly different from the actual colors viewed under a uniform illumination.}
\label{fig:sample-images}
\end{figure}

Electronically tunable filters, such as the Liquid Crystal Tunable Filter (LCTF)~\cite{alboon2008flat,zheng2011three}, are primarily designed for ground based hyperspectral imaging systems~\cite{gat2000imaging}. However, the LCTF suffers from very low transmission at shorter wavelengths (blue region) and very high transmission at longer wavelengths (red region) of the visible spectrum. Due to this modulating factor, the radiant energy received at the sensor varies with respect to the wavelength. This eventually degrades illuminant estimation through color constancy in the affected wavelength ranges. Therefore, radiometric compensation of an LCTF hyperspectral imaging system is crucial for accurate recovery of the spectral reflectance of a scene.

In this chapter, we propose a method for accurate spectral reflectance recovery from hyperspectral images. First, we show how illumination in a hyperspectral image can be estimated by color constancy. We then improve illuminant estimation based on two important properties of hyperspectral images, correlation between the nearby bands and apriori identification of illuminant type from the image. Second, we show how illuminant estimation in the first step can be improved by a modified form of hyperspectral imaging. We propose a variable exposure hyperspectral imaging technique for measurement of the scene spectral reflectance. The variable exposure compensates for the non-linearities of the optical components in a hyperspectral imaging system. The technique improves signal-to-noise ratio of hyperspectral images which subsequently results in better illuminant recovery through color constancy. We evaluate and compare the algorithms on two hyperspectral image databases and present a thorough experimental analysis. Experiments on real and simulated data show better reflectance recovery using the proposed imaging and illuminant estimation technique.


\section{Hyperspectral Color Constancy}
\label{sec:color}

\subsection{Adaptive Illuminant Estimation}
\label{sec:meth}

Assuming Lambertian (diffused) surface reflectance, the hyperspectral image of a scene can be modeled as follows. The formation of an $\lambda$ band hyperspectral image $I(x,y,z),z=1,2,...,\lambda$ of a scene is mainly dependent on three physiological factors i.e.~the illuminant spectral power distribution (SPD) $L(x,y,z)$, the scene spectral reflectance $S(x,y,z)$, and the system response $C(x,y,z)$ which combines both the sensor spectral sensitivity $q(x,y,z)$ (quantum efficiency) and the filter transmission $F(z)$ such that $C(x,y,z)=q(x,y,z)F(z)$. Considering the illumination and the sensor spectral sensitivity to be spatially invariant, one can concisely represent them as $L(z)$ and $C(z)$
\begin{equation}
\label{eq:hsi-model}
I(x,y) = \int_{z} L(z)S(x,y,z)C(z)dz~.
\end{equation}

Van de Weijer et al.~\cite{van2007edge} proposed a unified representation for a variety of color constancy methods. The illuminant spectra is estimated by different parameter values of the following formulation
\begin{equation}
\label{eq:unified}
\hat{L}(z:n,p,\sigma) = \frac{1}{\kappa}\int_y\int_x {\|\nabla^n I_\sigma(x,y)\|}_p\,dx\,dy~,
\end{equation}
where $n$ is the order of differential, $\|.\|_p$ is the Minkowski norm and $\sigma$ is the scale of the Gaussian filter such that $I_\sigma(x,y) = I(x,y) \ast G(x,y:\sigma)$ is the gaussian filtered image. Simply put, the Minkowsky norm of the aggregate gradient magnitude (e.g.~$n=2$) of each smoothed band is considered as its illumination value
\begin{equation}
\hat{L}(z:n,p,\sigma) = \frac{1}{\kappa}\left(\sum_{x}\sum_{y}\left(\sqrt{\frac{\partial^2 I_\sigma(x,y)}{\partial x^2}+\frac{\partial^2 I_\sigma(x,y)}{\partial y^2}}\right)^p\right)^{\frac{1}{p}}.
\end{equation}
The parameter $\kappa$ is a constant, valued such that the estimated illuminant spectra has a unit $\ell_2$ norm.

Figure~\ref{fig:illum} shows SPD of some common illumination sources, both artificial and natural. It can be observed that some illuminants are highly differentiable from others based on their SPD pattern. These SPDs can be broadly categorized into smooth or spiky. Most illumination sources generally exhibit smooth SPD (e.g.~daylight) where the spectral power gradually varies across consecutive bands. This implies that illumination estimated from neighboring bands is strongly related and can provide an improved illumination estimate. In contrast, for spiky illumination sources (e.g.~fluorescent), the spectral power undergoes sharp variation in certain bands. Therefore, the illumination estimated from nearby bands are weakly related.

\begin{figure}[h]
\centering
\includegraphics[trim = 75pt 3pt 95pt 20pt, clip, width=0.4\linewidth]{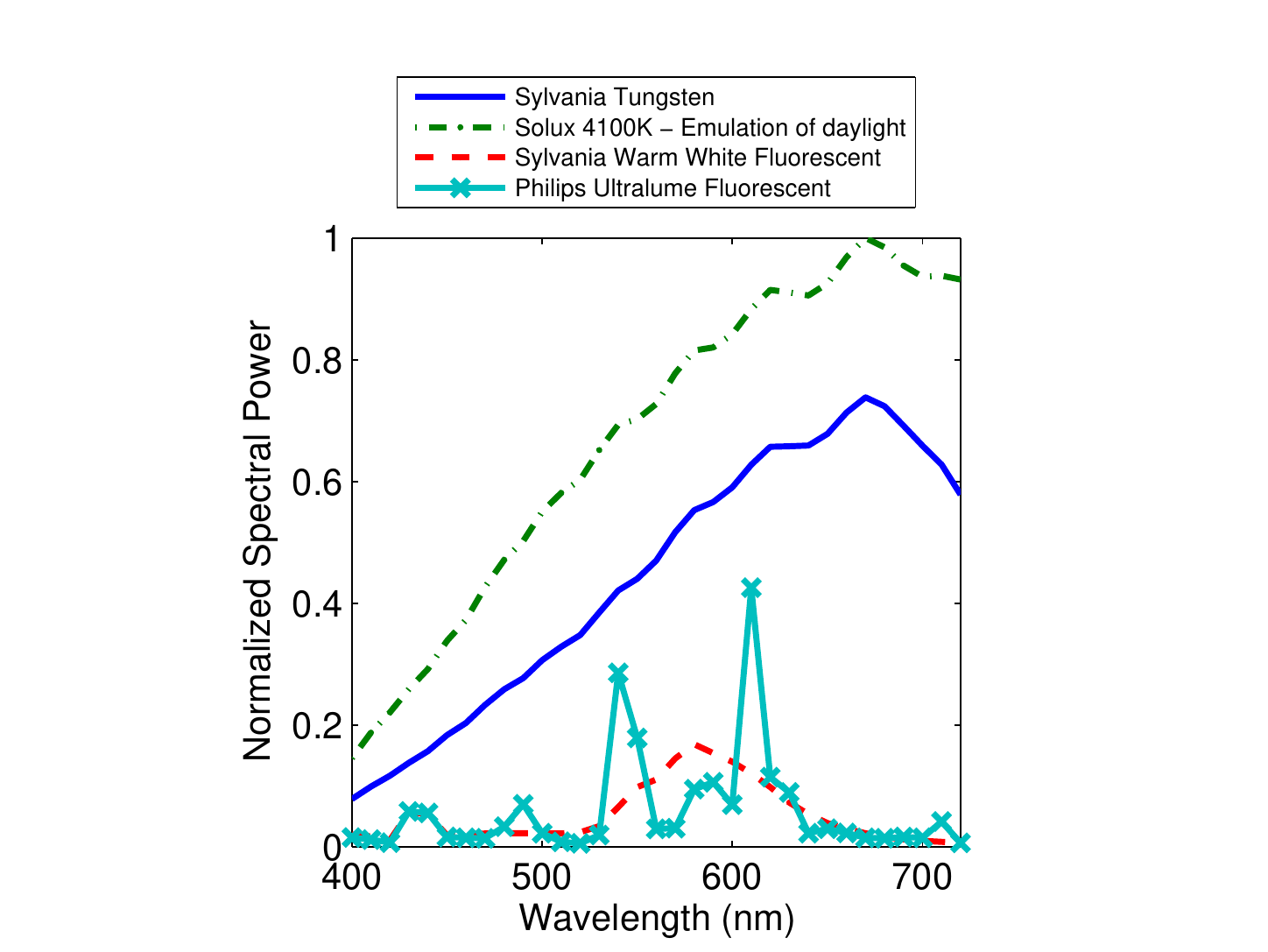}
\includegraphics[trim = 75pt 3pt 95pt 20pt, clip, width=0.4\linewidth]{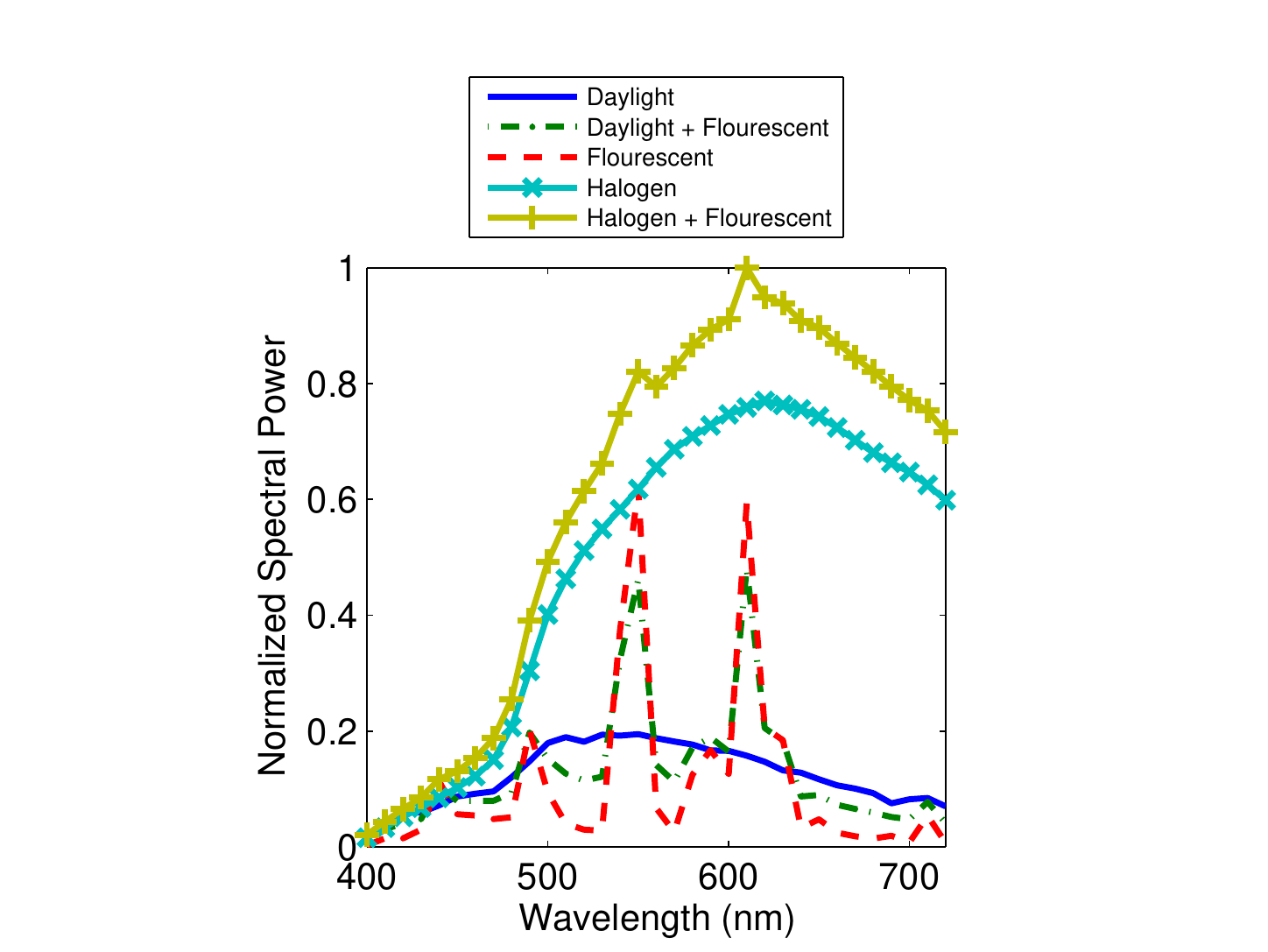}
\caption[Spectral Power Distribution of the illuminations]{SPD of the illuminations in simulated and real data. Illuminants in SFU data to generate simulated scenes of CAVE data (left). Illuminants measured in the real scenes of UWA data (right). Observe the diversity of illumination spectra in both cases.}
\label{fig:illum}
\end{figure}

To exploit this illumination differentiating characteristic, we devise an adaptive illumination estimation approach. First, an initial estimate of the illumination in a hyperspectral image is achieved using Equation~\ref{eq:unified}. Then, to detect whether the scene is lit by a smooth or a spiky illumination source, this initial estimate is then fed to a classifier. The classification is performed by a linear Support Vector Machine (SVM) which is trained on a set of illumination sources labeled as smooth or spiky. If the illumination is classified as smooth, the information in neighboring bands is used for an improved illumination estimate as follows.

Spatio-spectral information in hyperspectral images is useful for improving spectral reproduction and restoration~\cite{murakami2008color,mian2012hyperspectral}. We define a \emph{spatio-spectral support}, where each spectral band $I(x,y,z_{i})$ is supported by the neighboring bands $I(x,y,z_{i-\omega...,i+\omega})$, where $\omega={0,1,2,...}$ is the spectral support width. It is so called because the bands are spatially collated in the spectral dimension. An illumination estimate using spatio-spectral support can be achieved by modifying Equation (\ref{eq:unified})
\begin{equation}
L(z:n,p,\sigma) = \frac{1}{\kappa} \int_y\int_x {\|\nabla^n I^{z_\omega}_{\sigma}(x,y)\|}_p\,dx\,dy~,
\end{equation}
where $I^{z_\omega}=\{I^{z_0},I^{z_{\pm1}},...,I^{z_{\pm\omega}}\}$ is the set of neighboring bands, forming the spatio-spectral support as shown in Figure~\ref{fig:support}. Furthermore, it is intuitive to form a weighted spatio-spectral support such that the nearby bands carry more weight, whereas the bands farther away bear proportionally lesser weights with respect to the distance from the central band. Thus, by introducing weighting, the spatio-spectral support is updated as $I^{z_\omega}=\{w_0I^{z_0},w_1I^{z_{\pm1}},...,w_\omega I^{z_{\pm\omega}}\}$. A standard normal function is applied as weights for the spatio-spectral support.

\begin{figure}[t]
\centering
\includegraphics[width=0.75\linewidth]{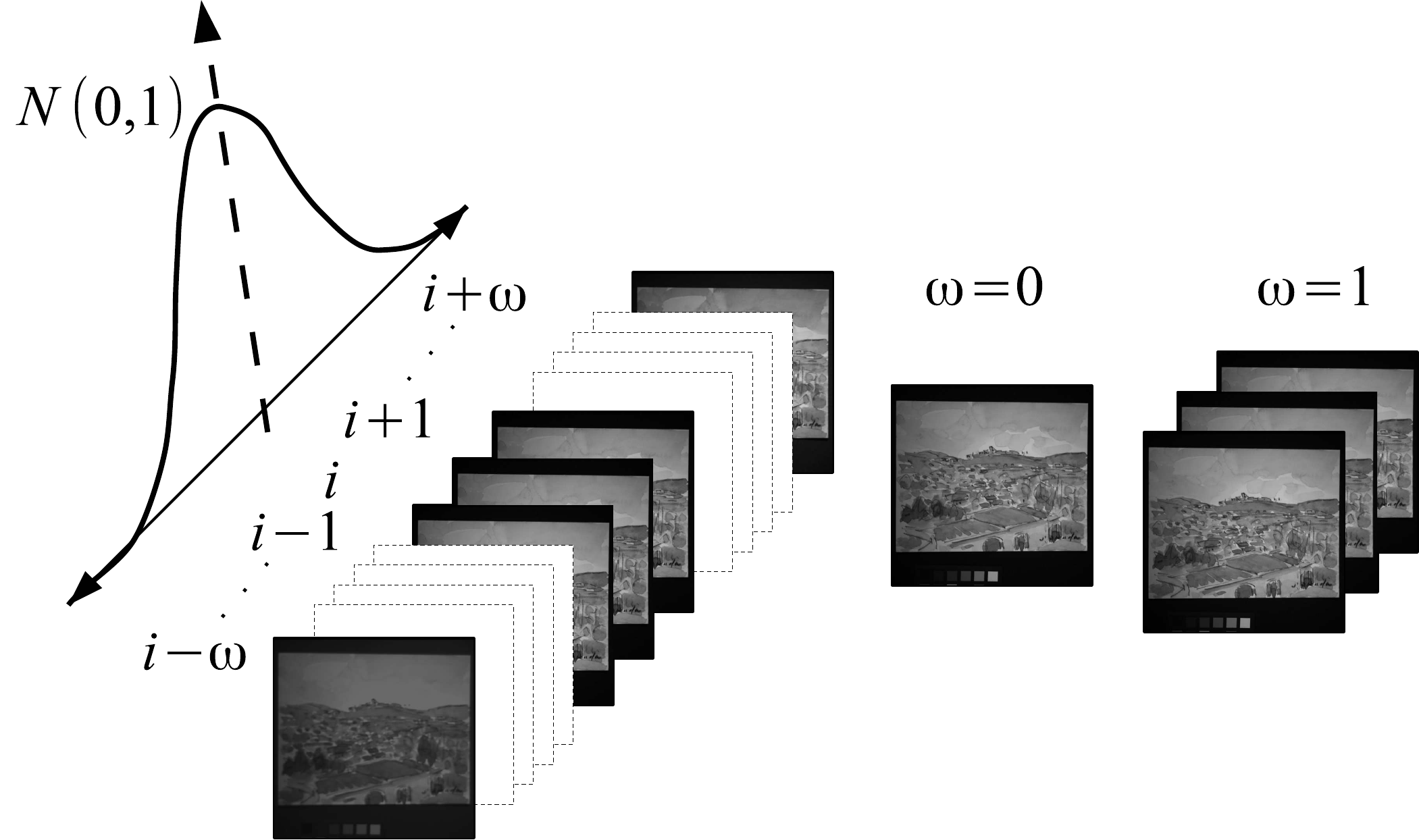}
\caption[Structure of spatio-spectral supports]{Structure of spatio-spectral supports of band $i$ for different instants of $\omega$. The spatio-spectral support is weighted by a standard normal distribution function.}
\label{fig:support}
\end{figure}

Generally, a simplified linear transformation is used to obtain an illumination corrected hyperspectral image.
\begin{equation}
\hat{I}(x,y,z) = M~I(x,y,z)~,
\end{equation}
where $M \in \mathbb{R}^{\lambda \times \lambda}$ is a diagonal matrix such that
\begin{equation}
M_{i,j} =
\begin{cases}
1 / \hat{L}(z_i) & \text{if $i=j$}\\
0 & \text{otherwise}
\end{cases}
\end{equation}
The \emph{angular error}~\cite{hordley2004re} is a widely used metric for benchmarking color constancy techniques. It has been shown to be a good perceptual indicator of the performance of color constancy algorithms~\cite{gijsenij2009perceptual}. The angular error is defined as the angle (in degrees) between the estimated illuminant spectra $(\hat{L})$, and the ground truth illuminant spectra $(L)$
\begin{equation}
\mathbf{\epsilon}= \arccos \left( \frac{L \cdot \hat{L}}{\|L\|\|\hat{L}\|} \right)~.
\end{equation}
The angular error ($\mathbf{\epsilon}$) is used for the evaluation of all algorithms presented in this work.

\subsection{Individual Color Constancy Methods}
\label{sec:ind}

Different combinations of the parameters $(n,p,\sigma)$, signify a unique hypothesis and translate into different illuminant estimation algorithms. \emph{Gray World} (GW)~\cite{buchsbaum1980spatial} assumes that the average image spectra is flat (uniform) while \emph{Gray Edge} (GE)~\cite{van2007edge} assumes that the mean spectra of the edges is flat so that the illuminant spectra can be estimated as the shift from respective deviation. Two common variants of the GE algorithm are the $1^\textrm{st}$ order gray edge (GE1) and the $2^\textrm{nd}$ order gray edge (GE2). \emph{White Point} (WP)~\cite{land1974retinex} assumes the presence of a white patch in the scene such that the maximum value in each band is the reflection of the illuminant from the white patch. \emph{Shades-of-Gray} (SoG)~\cite{finlayson2004shades} assumes that the $p^\textrm{th}$ norm of a scene is a shade of gray whereas the \emph{general Gray World} (gGW)~\cite{buchsbaum1980spatial} considers the $p^\textrm{th}$ norm of a scene after smoothing to be flat.

Although, a number of other algorithms can emanate from more sophisticated instantiations of the parameters $(n,p,\sigma)$, we restrict our scope only to the above mentioned widely accepted algorithms. A list of these algorithms along with their parameter values widely used in the literature~\cite{gijsenij2011computational,bianco2012color} are given in Table \ref{tab:algos}. We used the original authors' implementations of these algorithms\footnote{Color Constancy Algorithms:\\ \url{http://lear.inrialpes.fr/people/vandeweijer/code/ColorConstancy.zip}} after extension for use with hyperspectral images.

\begin{table}[h]
\caption[Parameters of color constancy methods]{Color constancy methods from different instantiations of parameters in Equation~\ref{eq:unified}}
\label{tab:algos}
\vspace{9pt}
\centering
\footnotesize
\begin{tabular}{|l|c|c|c|}
  \hline
  \textbf{Methods}                                          &  $n$  &  $p$   & $\sigma$ \\ \hline \hline
  Gray World (GW)~\cite{buchsbaum1980spatial}               &   0   &   1    &   0      \\ \hline
  White Point (WP)~\cite{land1974retinex}                   &   0   &$\infty$&   0      \\ \hline
  Shades of Gray (SoG)~\cite{finlayson2004shades}           &   0   &   4    &   0      \\ \hline
  general Gray World (gGW)~\cite{buchsbaum1980spatial}      &   0   &   9    &   9      \\ \hline
  $1^\textrm{st}$ order Gray Edge (GE1)~\cite{van2007edge}  &   1   &   1    &   6      \\ \hline
  $2^\textrm{nd}$ order Gray Edge (GE2)~\cite{van2007edge}  &   2   &   1    &   1      \\ \hline
\end{tabular}
\end{table}


\subsection{Combinational Color Constancy Methods}
\label{sec:cmb}
We also investigate few strategies to combine the outputs of different algorithms for extensive evaluation~\cite{shen2007improved,bianco2010automatic}. A simple combination is the average of the estimate of all $P$ individual algorithms ($P=6$). The assumption of such an algorithm would be that if majority of the $P$ algorithms produce correct estimate, the average estimate would also be close to the ground truth and vice versa. The average estimated illumination will therefore be,
\begin{equation}
\hat{L}_{_{\textrm{AVG}}} = \frac{1}{P}\sum_{i=1}^{P} \hat{L^i}~.
\end{equation}


It is possible to combine the outputs of all the algorithms, excluding the worst performing algorithm. The algorithm with the largest aggregate angular error between its estimate and the estimates obtained by the rest of the algorithms, is left out. Then the average of the rest of the algorithms is the L1O estimate~\cite{li2011evaluating}.

\begin{equation}
\hat{L}_{_{\textrm{L1O}}} = \frac{1}{P-1}\sum_{i=1}^{P} \hat{L^i},~~~i\neq\operatorname*{arg\,max}_{j}\left({\sum_{i=1}^{P-1} \mathbf{\varepsilon}_{i,j} }\right)~,
\end{equation}
where $\varepsilon \in \mathbb{R}^{P-1\times P}$ is the matrix obtained by computing the angular errors between all $P$ individual algorithms.

Different individual algorithms are likely to produce a dissimilar illumination estimate on a particular image depending on the illumination and scene contents. It is not known a priori, which algorithm suits a particular scenario. Correlation based combinations are deemed beneficial if they posses the following desirable properties. First, the algorithms' outputs should be uncorrelated. Second, both algorithms should be accurate overall. Thus, illumination estimates $\hat{L}_{_\textrm{X}}$ and $\hat{L}_{_\textrm{Y}}$ from two algorithms $X$ and $Y$ are combined as
\begin{equation}
\hat{L}_{_{\textrm{CbC}}} = \frac{\hat{L}_{_\textrm{X}}+\hat{L}_{_\textrm{Y}}}{2}~,
\end{equation}
where $\hat{L}_{_{\textrm{CbC}}}$ would be robust in case either $X$ or $Y$ produces an outlier estimate. Selection of algorithm $X$ and algorithm $Y$ is discussed later in experiments.

\section{Hyperspectral Imaging by Automatic Exposure Time Adjustment}
\label{sec:hs-autexpimg}

LCTF based hyperspectral imaging systems exhibit extremely low transmission levels at shorter wavelengths and the image sensor has a variable quantum efficiency in the visible range as shown in Figure~\ref{fig:sensors}. These factors result in dark and noisy images at shorter wavelengths due to very low energy received at the sensor. Therefore, color constancy algorithms are unable to accurately recover spectral reflectance, especially in bands having low signal to noise ratio. In order to radiometrically compensate the system in the affected wavelengths we present an automatic exposure time adjustment imaging technique. We investigate the exposure-intensity relationship to introduce a variable exposure factor in the basic hyperspectral image model. The variable exposure allows higher energy in shorter wavelengths and lower energy at longer wavelengths to achieve a net uniform energy received at the sensor. In this way, radiometric compensation is achieved, which results in better spectral recovery using color constancy methods.

\begin{figure}[t]
\footnotesize
\centering
\includegraphics[trim = 33pt 31pt 30pt 23pt, clip, width=1\linewidth]{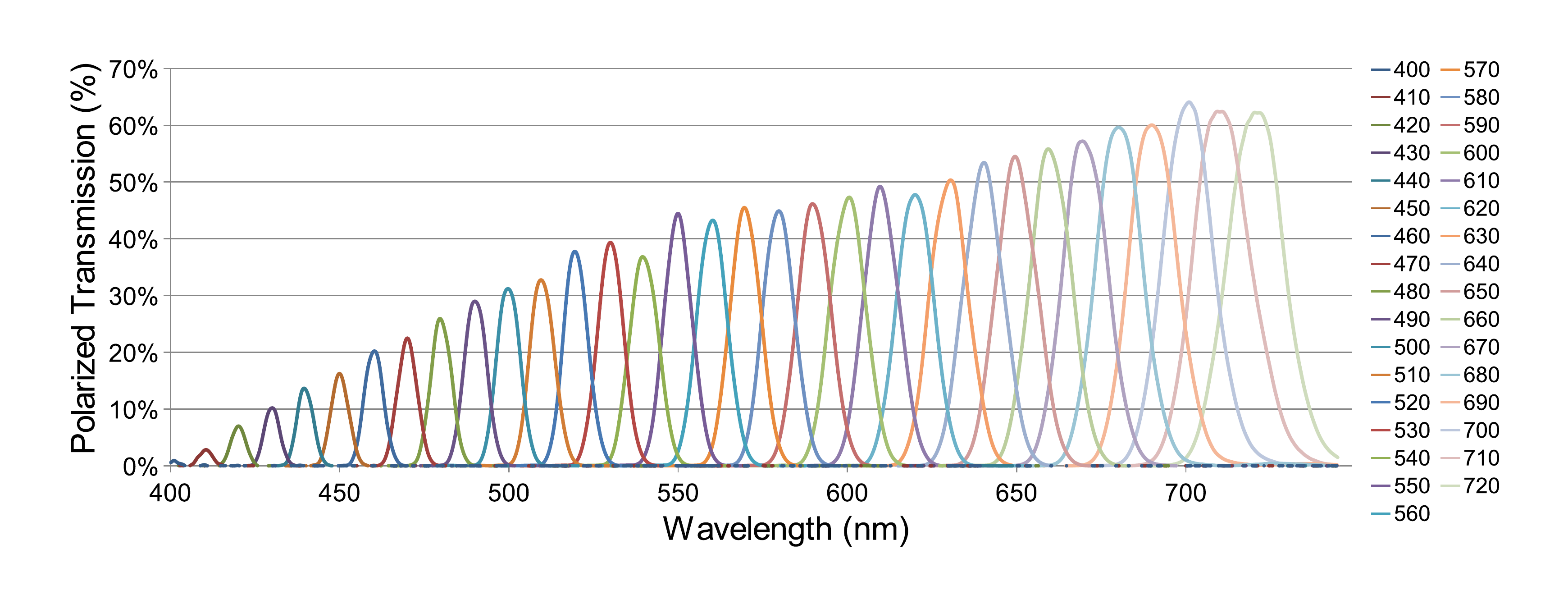}
\includegraphics[width=0.4\linewidth]{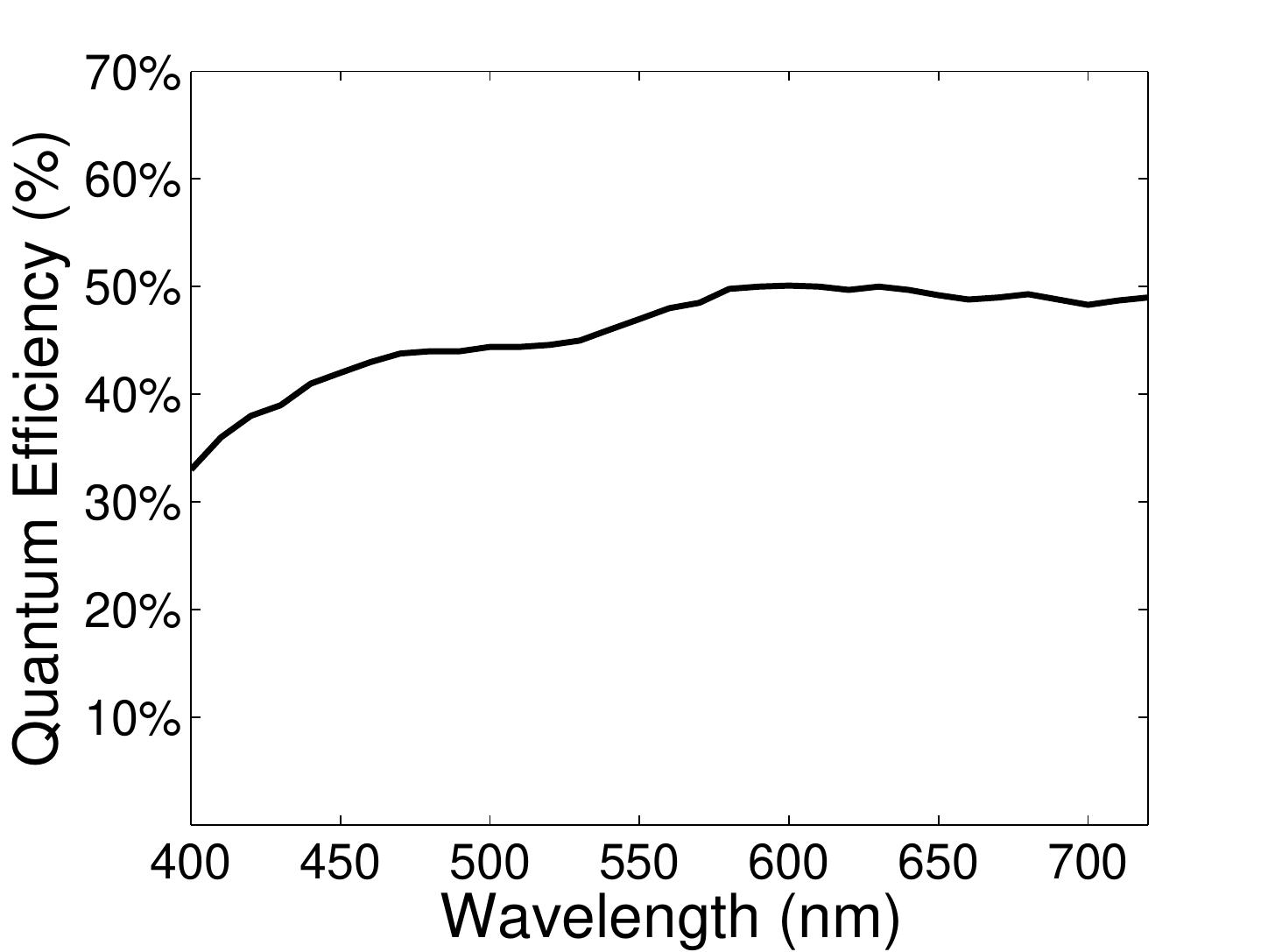}
\caption[Transmission functions of filter and sensor quantum efficiency]{Transmission functions of the LCTF at 10nm wavelength step and the quantum efficiency of the camera CCD versus wavelength.}
\label{fig:sensors}
\end{figure}

\subsection{Exposure-Intensity Relationship}
\label{sec:exp-int}

The relationship between measured intensity and the exposure time, photon flux and quantum efficiency of the camera sensor is
\begin{equation}
I(x,y) = t \int_{z} P(x,y,z) q(z)\,dz~.
\end{equation}
In the above equation, the factor $t$ (exposure time) is fixed and independent of $z$, $P(x,y,z)$ is the photon flux incident on the image sensor array and $q(z)$ is the quantum efficiency of the sensor. In order to control image intensity with exposure time, we can make $t$ variable such that it is a function of $z$. Therefore, the above equation changes to
\begin{equation}
I(x,y) = \int_{z} t(z) P(x,y,z) q(z)\,dz~.
\end{equation}
The exposure time is linearly related to the amount of photon flux incident on the sensor, given the illumination does not vary instantaneously. Moreover, the sensor quantum efficiency remains nearly constant, provided the sensor temperature is held fixed. In summary, if the exposure time is linearly varied, so does in effect, the radiance measured at the sensor pixel. We experimentally validate this relationship. In this experiment, the exposure time was linearly varied from minimum to maximum in discrete steps. In each step, an image of a white patch is acquired and the average response of the pixels is calculated. The minimum exposure is set as the device exposure lower limit $t_{\textrm{min}}$. The maximum exposure $t_{\textrm{max}}$ is a value such that the white pixels average just equals the absolute intensity scale maximum (255). The procedure is repeated for discrete center wavelengths.

In Figure~\ref{fig:exp-lin}-\ref{fig:exp-log} the plots of radiance against exposure times (both linear and log scale) are shown for 4 different wavelengths. We observe a linear relationship between the exposure time and the measured radiance. Furthermore, it can be seen that this relationship remains linear regardless of the wavelength of the incident light. Note that all curves are linear and the slope of the lines is a direct function of the wavelength.

\begin{figure}[h]
\footnotesize
\centering
\subfigure[]{\label{fig:exp-lin}\includegraphics[trim = 5pt 0pt 15pt 5pt, clip, width=0.43\linewidth]{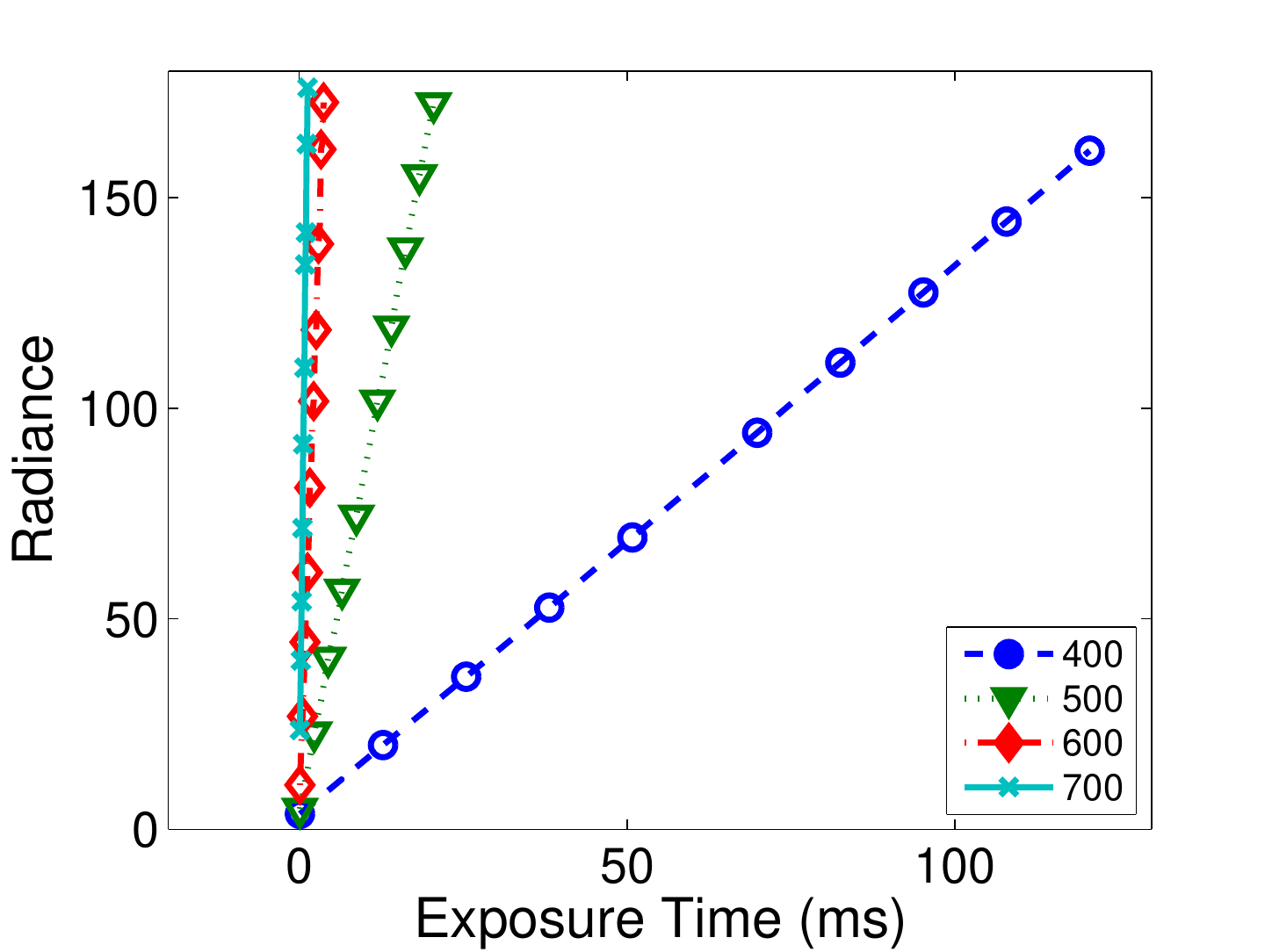}}
\subfigure[]{\label{fig:exp-log}\includegraphics[trim = 5pt 0pt 15pt 5pt, clip, width=0.43\linewidth]{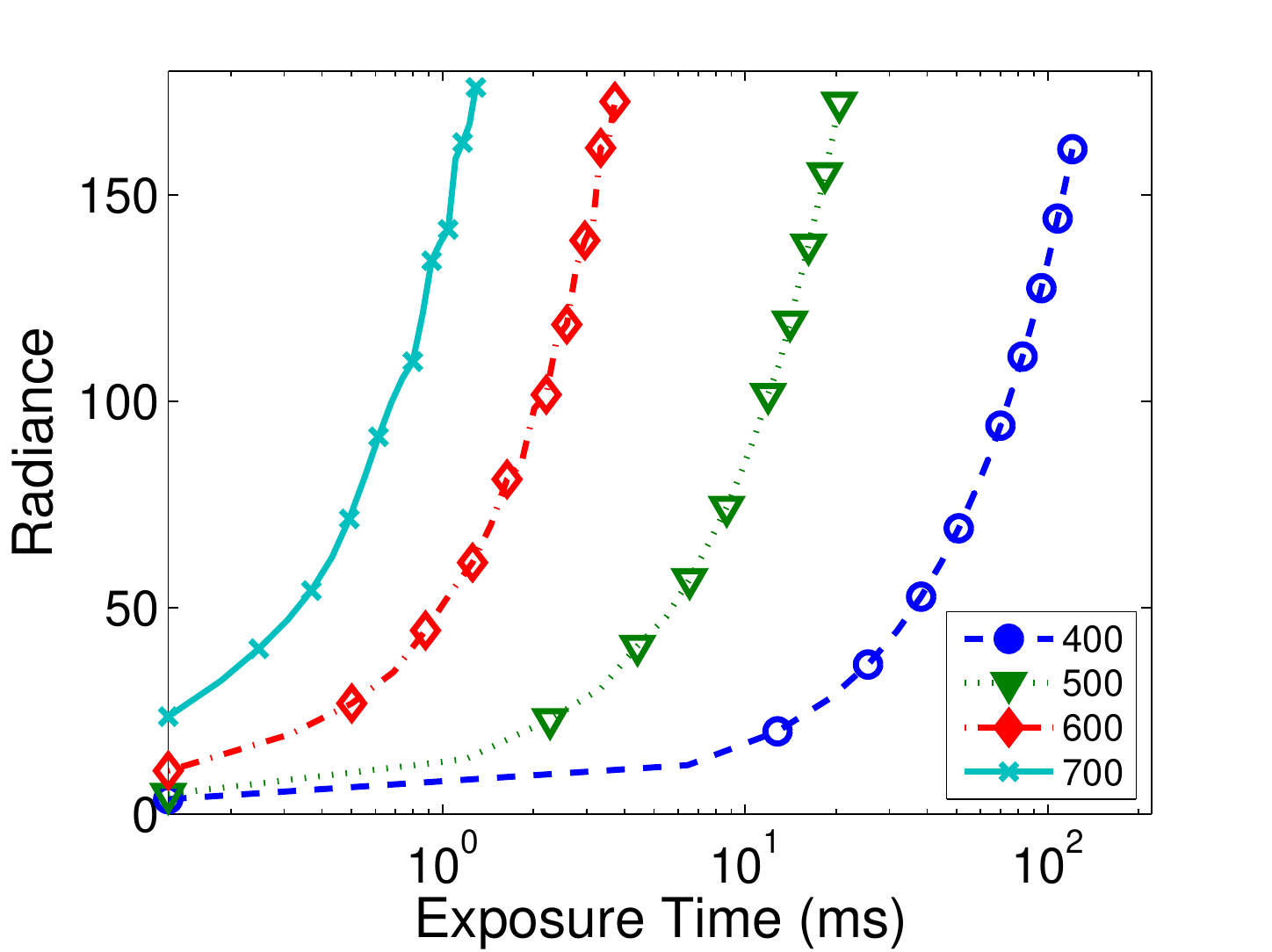}}
\subfigure[]{\label{fig:camera-exp-cc}\includegraphics[trim = 5pt 0pt 15pt 5pt, clip, width=0.43\linewidth]{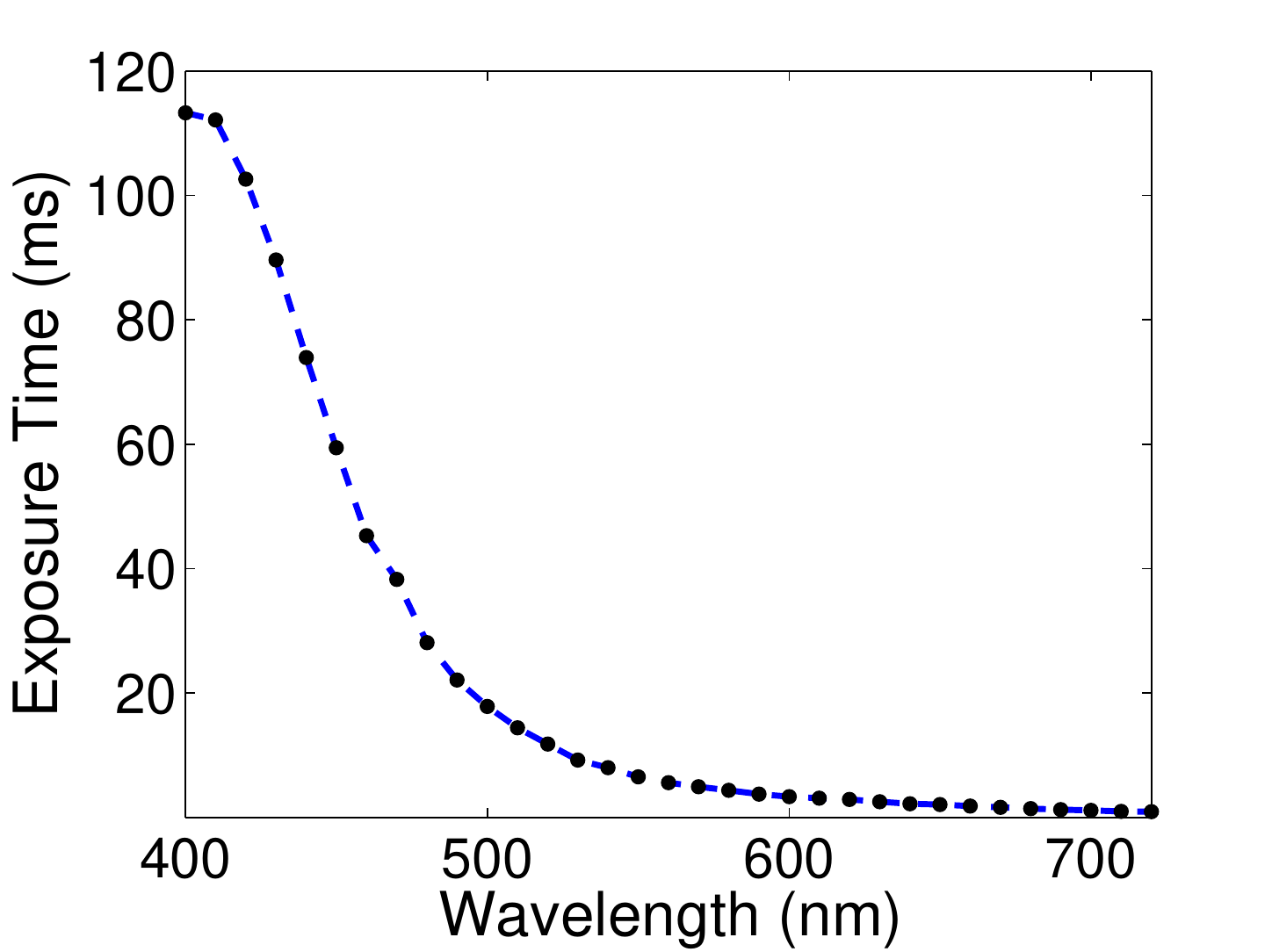}}
\caption[Exposure versus radiance relationship]{Exposure versus radiance relationship (a) linear scale and (b) log scale. (c) The exposure function against wavelength computed by Algorithm~\ref{alg:AETA} for a scene under halogen illumination.}
\label{fig:exposure}
\end{figure}

Based on the validated linear relationship, we now add the variable exposure factor in Equation~\ref{eq:hsi-model} to form our variable exposure hyperspectral image model
\begin{equation}
I(x,y) = \int_{z} t(z)L(z)S(x,y,z)C(z)\,dz~.
\end{equation}
In order to recover the true spectral responses, the exposure time should be an inverse function of the system response and the illuminant spectral power distribution.
\begin{align}
\label{eq:t-prop}
t(z) \propto \frac{1}{L(z)C(z)} \text{\quad or \quad} t(z) = \frac{b}{L(z)C(z)}
\end{align}
where $b$ is a constant of proportionality. In the following we present an automatic exposure time computation algorithm which implicitly computes $t(z)$ that compensates for the factors in Equation~\ref{eq:t-prop}.

\subsection{Automatic Exposure Time Computation}

We propose a bisection search algorithm for automatically computing variable exposure time that compensate for the factors in Equation~\ref{eq:t-prop}. For each band, the bisection search (Algorithm~\ref{alg:AETA}) returns an exposure time ($t_j$). This exposure time would result in a nearly flat response of the white patch. As a result of a flat response for a white patch, the true spectral reflectance of any object in the scene is captured.

The exposure search starts in the range $t_{\textrm{min}}, t_{\textrm{max}}$ which are the absolute camera exposure limits. It is assumed that the exposure time to achieve the required intensity value $\mu_{\textrm{req}}$ exists within this range. Then, for the $j^\textrm{th}$ band, $t_{\textrm{low}}, t_{\textrm{high}}$ are initialized with the absolute exposure limits. A bisection search then begins with the computation of a test exposure value $t_j$ which is the average of $t_{\textrm{high}}$ and $t_{\textrm{low}}$. An image of a white patch is captured with the test exposure $t_j$. If the average value of the patch $\mu_j$ is higher than the required value $\mu_{\textrm{req}}$, then $t_{\textrm{high}}$ is reduced, otherwise, $t_{\textrm{low}}$ is increased and the process is repeated. If the difference between the achieved and required averages is less than a tolerance ($e$) or the number of iterations is exhausted, the search is discontinued and the required exposure for the band is the last test exposure value. In general, the bisection search could easily converge in 4-10 iterations for each band.

\begin{algorithm}[h]  
\caption{Automatic Exposure Time Adjustment Algorithm}          
\label{alg:AETA}                           
\begin{minipage}[c]{0.75\linewidth}
\begin{algorithmic}                    
\Require {$t_{\textrm{min}}, t_{\textrm{max}}$} \Comment{absolute exposure time limits}
\State $\mu_{\textrm{req}}$ \Comment{required average intensity of white patch}
\State $\lambda$ \Comment{total number of bands}
\State $e$ \Comment{tolerance value}
\State $i_{\textrm{max}}$ \Comment{maximum no. of iterations per band}
\For{$j=1$ to $\lambda$}
\State $t_{\textrm{low}} \gets t_{\textrm{min}},\,t_{\textrm{high}} \gets t_{\textrm{max}}$
\State $t_j \gets t_{\textrm{low}}+(t_{\textrm{high}}-t_{\textrm{low}})/2$
\State $i \gets 0$
\Repeat \Comment{bisection search}
\State $I_j \in {\mathbb{R}^{m\times n}}=\texttt{getWhitePatch}(t_j)$
\State $\mu_{j} \gets \frac{1}{mn}\underset{x,y}{\sum} I_j$
\If{$\mu_{j} > \mu_{\textrm{req}}$} \Comment{exposure too high, decrease exposure}
\State $t_{\textrm{high}} \gets t_j$
\State $t_j \gets t_{\textrm{low}}+(t_j-t_{\textrm{low}})/2$
\Else                               \Comment{exposure too low, increase exposure}
\State $t_{low} \gets t_j$
\State $t_j \gets t_j+(t_{\textrm{high}}-t_j)/2$
\EndIf
\State $i\gets i+1$
\Until{$(|\mu_{j}-\mu_{\textrm{req}}|\leq e) \vee (i=i_{\textrm{max}})$}
\EndFor
\Ensure {$t \in \mathbb{R}^{\lambda}$} \Comment{computed exposure times}
\end{algorithmic}
\end{minipage}
\end{algorithm}

Although, for a given imaging system, $C(z)$ is fixed, $L(z)$ usually varies in different capture setups. Thus, an exposure function is automatically estimated for a given illumination to capture the true spectral response of a scene in-situ. Therefore, a single automatic calibration sequence is required to compute the exposure function. Figure~\ref{fig:camera-exp-cc} provides the computed exposure times for a scene under halogen illumination. Observe how larger exposure values are obtained for shorter wavelengths and vice versa, resulting in radiometric compensation of the system response and illumination.

\section{Hyperspectral Image Rendering for Visualization}

Hyperspectral images are essentially made of more than three bands. Since, human eye can only sense three colors (commonly referred to RGB), the hyperspectral images are visualized in two ways. One way is to use pseudo-color maps, which are mostly used in remotely sensed satellite captured hyperspectral images. A pseudo-color rendering is sufficient and favorable for visualizing the different class of materials that can be recognized from such images, e.g., land, water, vegetation, fire, roads and pavements etc. However, in ground based hyperspectral imaging of real world objects, it is preferred to visualize the hyperspectral images as they normally appear to a human observer, i.e. in RGB colors.

For rendering hyperspectral images into RGB, there are many different options. The first, which is somewhat standard is to use the CIE 1931 color space transformation created by International Commission on Illumination (CIE). The CIE XYZ standard color matching functions are analogous to the human cone responses which were experimentally computed in 1931~\cite{smith1931cie}.

\begin{equation}
X = \int_{z_1}^{z_2} I(z)\,\overline{x}(z)\,dz,~~
Y = \int_{z_1}^{z_2} I(z)\,\overline{y}(z)\,dz,~~
Z = \int_{z_1}^{z_2} I(z)\,\overline{z}(z)\,dz
\end{equation}
where $\bar{x}(z),\bar{y}(z)$ and $\bar{z}(z)$ are the chromatic response functions of a standard observer. $(z_1,z_2)$ is the spectral range, generally (400nm,720nm).

To visualize the images on color displays, a predefined linear transformation converts the images from the XYZ color space to the sRGB color space.

\begin{equation}
\begin{bmatrix}R\\G\\B\end{bmatrix}=
\begin{bmatrix}
3.2406&-1.5372&-0.4986\\
-0.9689&1.8758&0.0415\\
0.0557&-0.2040&1.0570
\end{bmatrix}
\begin{bmatrix}X\\Y\\Z\end{bmatrix}
\end{equation}

The second approach is to use a custom XYZ transformation function. The reason to use non-standard XYZ functions is that the standard CIE XYZ functions may not provide the optimal transformation for visualization of images. This is because different hyperspectral imaging systems suffer from camera and filter noise. This introduces deviation in measurement from real spectra that exists in the real world. Thus, a perceptually correct visualization in RGB requires specialized transformation functions. Such transformations may be characterized by Gaussian filters occurring in the vicinity of R,G and B. The mean and spread of the Gaussians determine the relative proportion of the red, green and blue colors.

\begin{equation}
\bar{x}(z) = a_{1}~\exp\left({-\frac{x-\mu_{z_1}}{\sqrt{2}\sigma_1}}\right)^2~,\\
\bar{y}(z) = a_{2}~\exp\left({-\frac{y-\mu_{z_2}}{\sqrt{2}\sigma_2}}\right)^2~,\\
\bar{z}(z) = a_{3}~\exp\left({-\frac{z-\mu_{z_3}}{\sqrt{2}\sigma_3}}\right)^2~,
\end{equation}
where $a$ is the peak filter response, $\mu_{z}$ is the center wavelength and $\sigma$ denotes the spread of a filter.
In this work, the parameters of custom XYZ to RGB transformation are, mean: $(\mu_1,\mu_2,\mu_3) = (640,550,450)$, variance: $(\sigma_1,\sigma_2,\sigma_3 = (40,40,40)$, peak: $(a_1,a_2,a_3) = (1,0.88,0.8)$

\begin{figure}[h]
\centering
\includegraphics[width=0.4\linewidth]{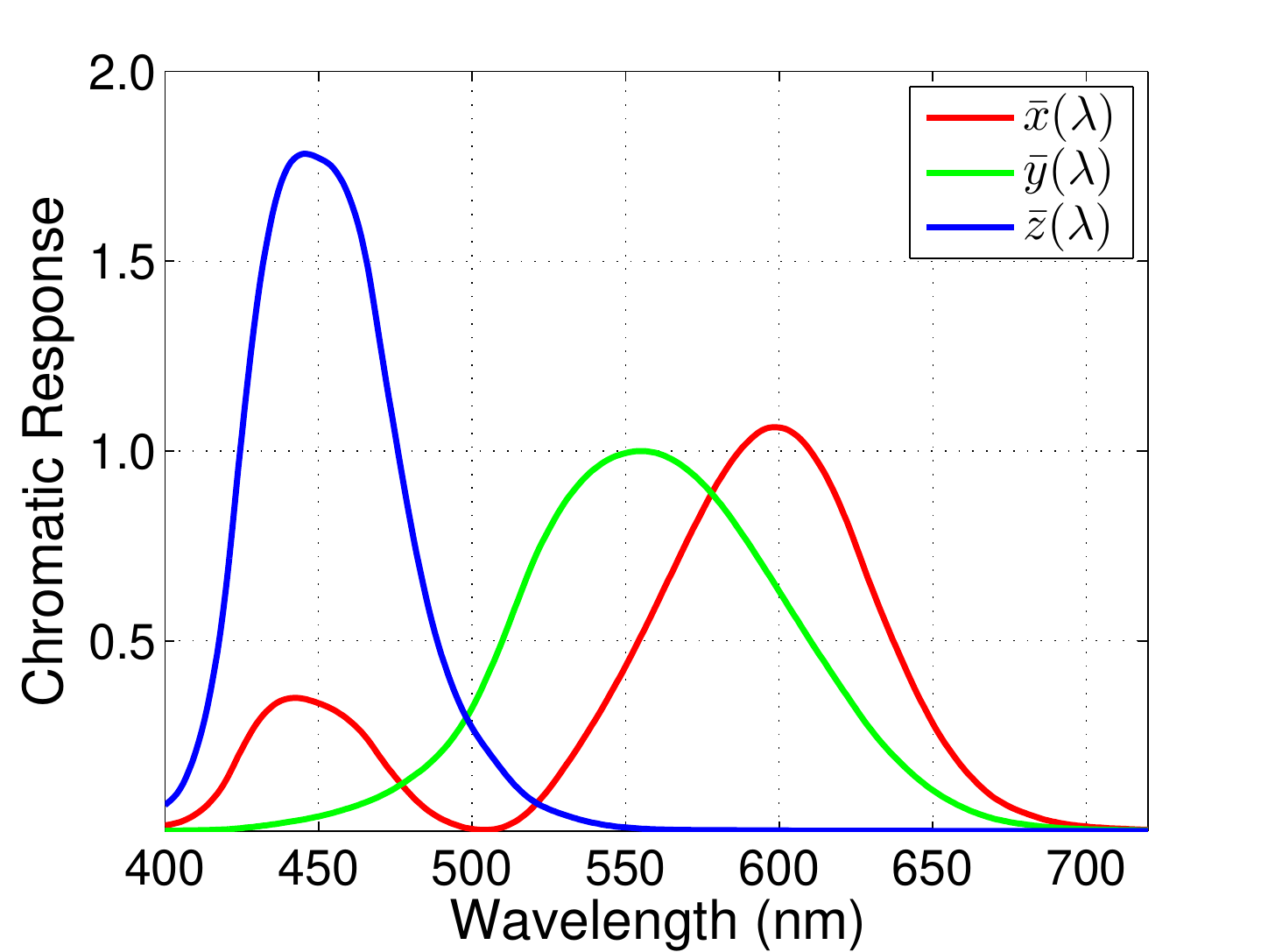}
\includegraphics[width=0.4\linewidth]{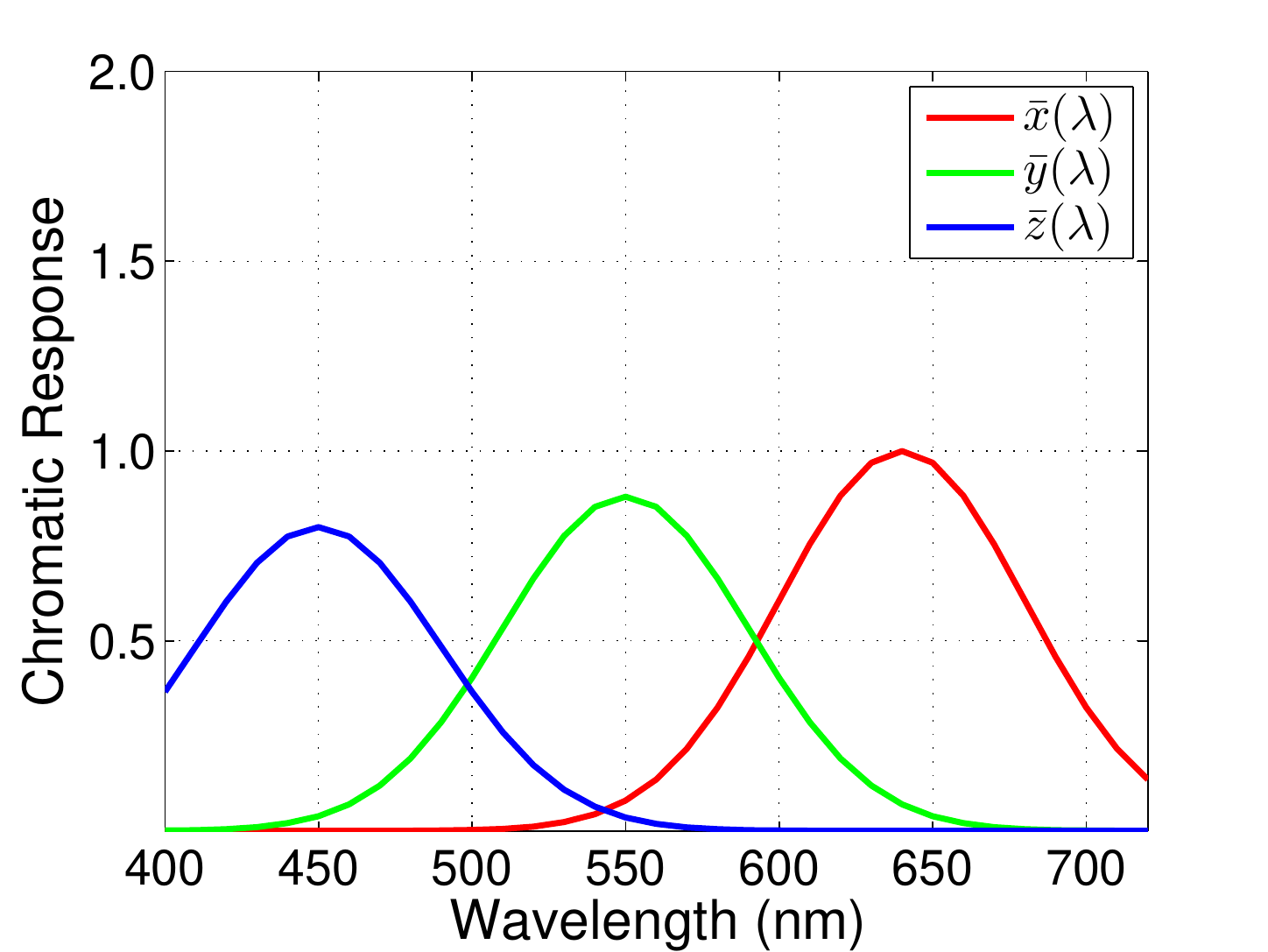}
\caption[Color matching functions]{The CIE standard observer color matching functions and the \emph{custom} color matching functions ($\bar{x}(z),\bar{y}(z),\bar{z}(z)$)}
\label{fig:XYZfcn}
\end{figure}

\section{Experimental Setup}
\label{sec:exp-setup}

\subsection{Imaging Setup}

The hyperspectral image acquisition setup is illustrated in Figure~\ref{fig:setup}. The system consists of a monochrome machine vision CCD camera from \emph{Basler Inc.} with a native resolution of $752\times 480$ pixels (8-bit). In front of the camera is a focusing lens (\emph{Fujinon} 1:1.4/25mm) followed by a \emph{VariSpec Inc.} liquid crystal tunable filter, operable in the range of 400-720 nm. The average tuning time of the filter is 50 ms. The filter bandwidth, measured in terms of the \emph{Full Width at Half Maximum (FWHM)} is 7 to 20nm which varies with the center wavelength. The scene is illuminated by a choice of different illuminations. For automatic exposure computation, the white patch from a 24 patch color checker from \emph{Xrite Inc.} was utilized. Note that the white patch is not utilized to spectrally calibrate a hyperspectral image by dividing each band by corresponding band of the white patch. This is because it would mean using the ground truth illumination of the scene.

\begin{figure}[h]
\centering
\includegraphics[width=0.48\linewidth]{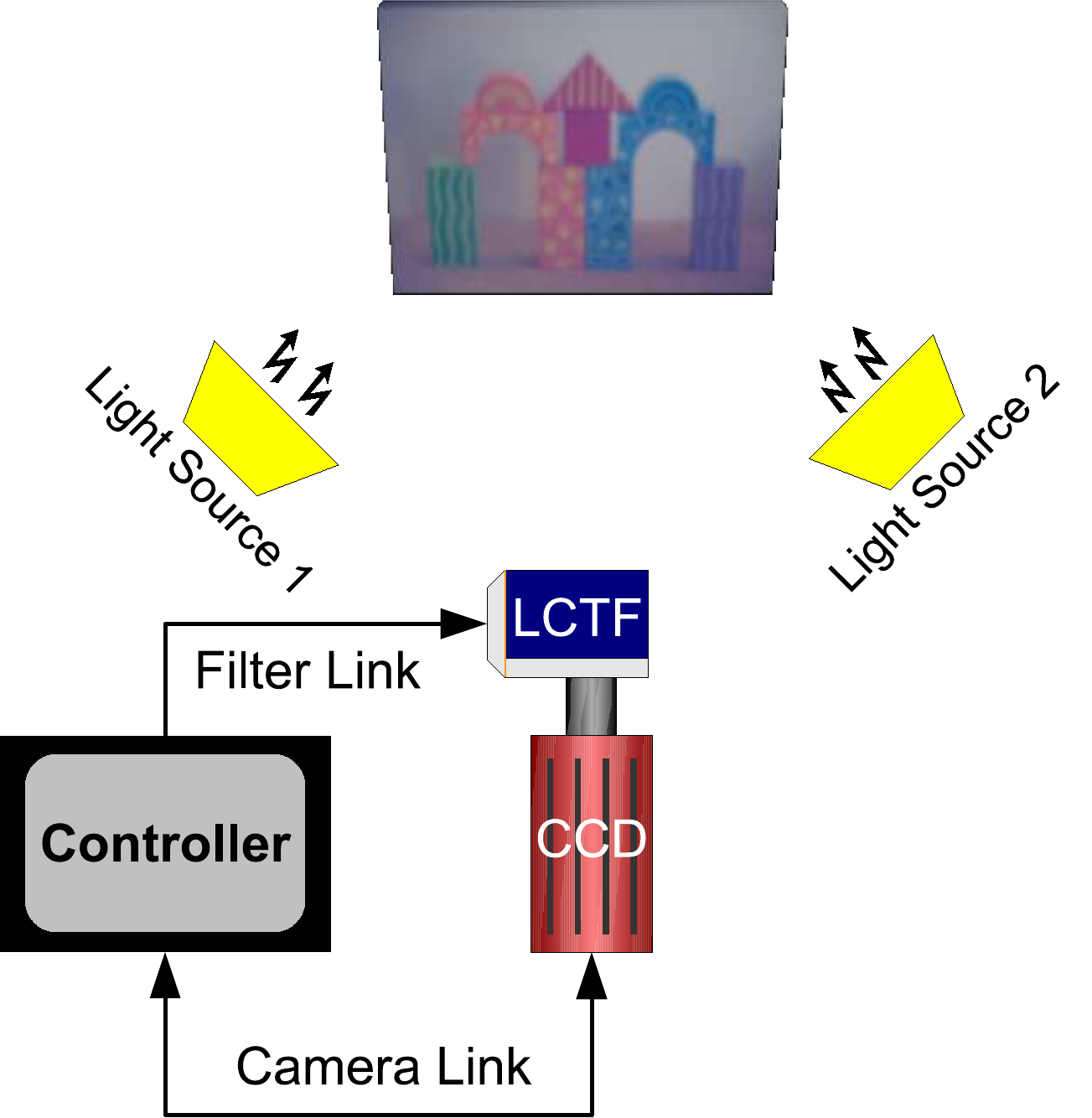}
\caption[Hyperspectral imaging setup]{The proposed LCTF based hyperspectral image acquisition setup.}
\label{fig:setup}
\end{figure}

\subsubsection{Variable Exposure Imaging}

Once the calibrated exposure times for each band are obtained using Algorithm~\ref{alg:AETA}, the hyperspectral images can be captured in a time multiplexed manner. The filter is tuned to the desired wavelength, followed by setting the camera to the required exposure for that wavelength. An image is acquired and the cycle is repeated for the rest of the bands. The whole hyperspectral cube can be acquired in around 6 seconds. Sample images captured using automatic exposure time adjustment are shown in Figure~\ref{fig:auto-exp-images}. Observe that the captured images are much close in visual appearance to the real world, implying that the bias of illumination, sensor and filter have been compensated to a great extent.

\begin{figure}[h]
\centering
\footnotesize{Scene 1~~~~~~~~~~~~~~~~~~~~~~~~~~~~Scene 2~~~~~~~~~~~~~~~~~~~~~~~~~~~~Scene 3} \\
\subfigure[Hyperspectral images captured with variable exposure]{\label{fig:auto-exp-images}
\includegraphics[width=0.3\linewidth]{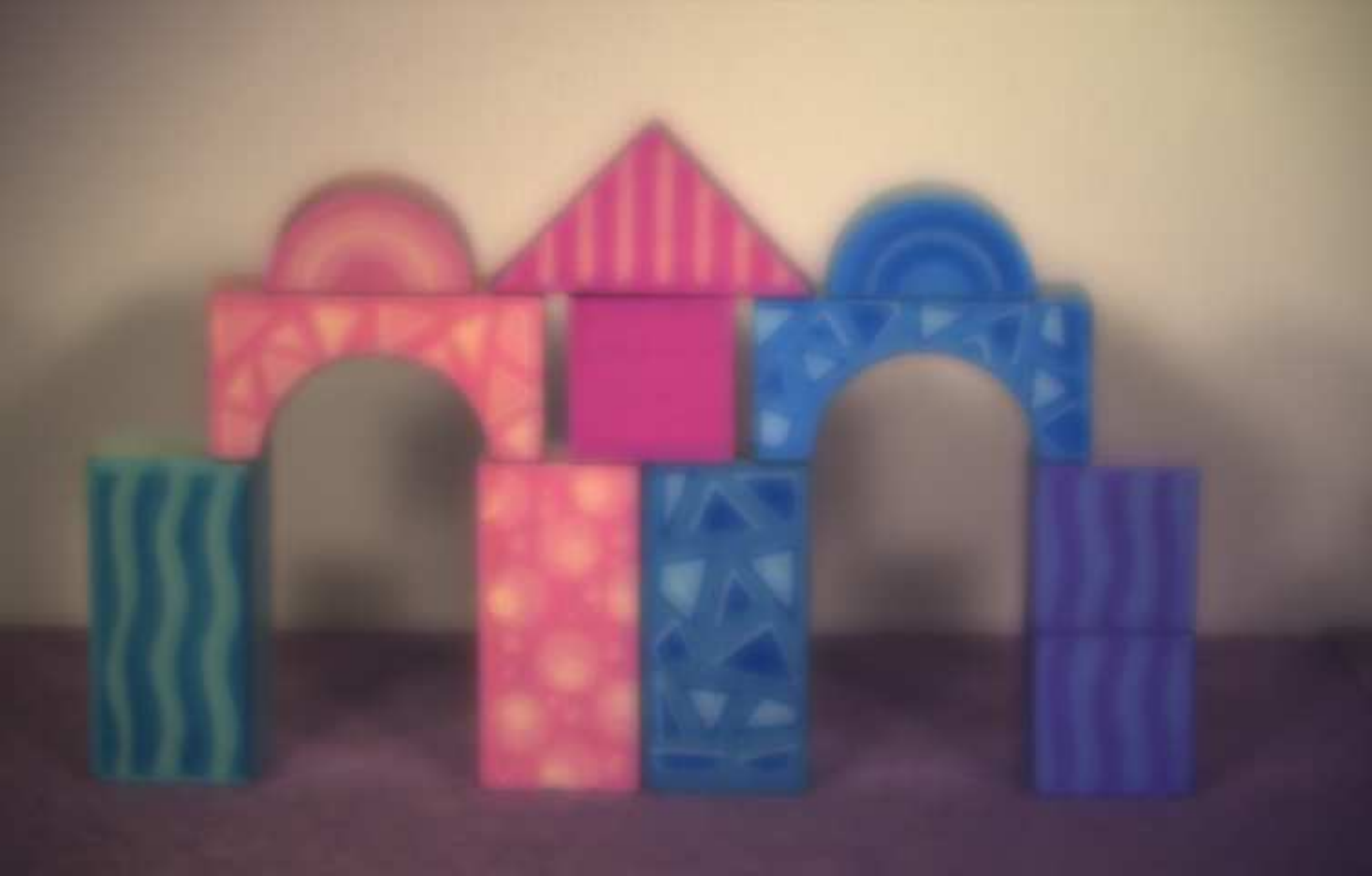}
\includegraphics[width=0.3\linewidth]{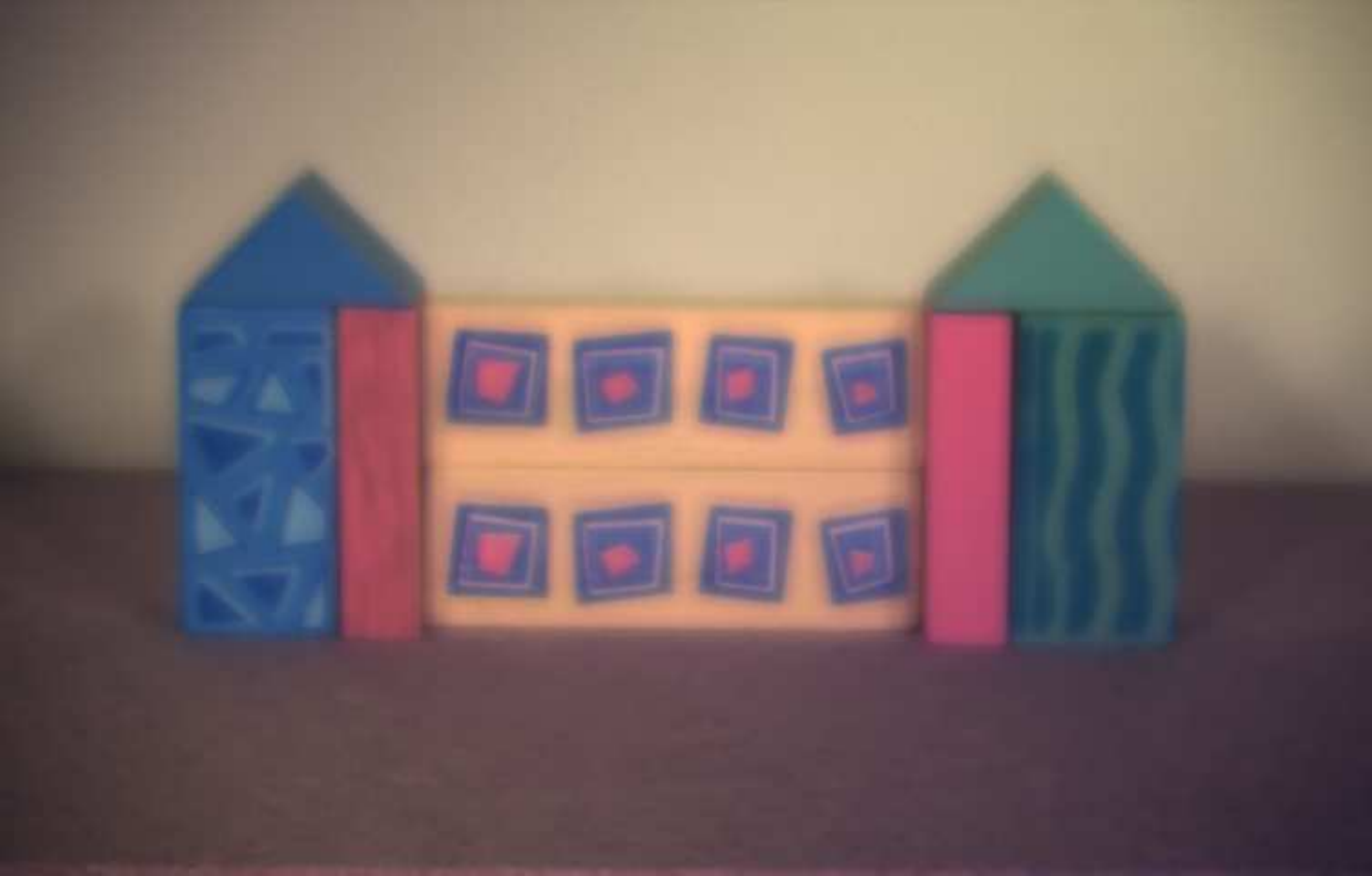}
\includegraphics[width=0.3\linewidth]{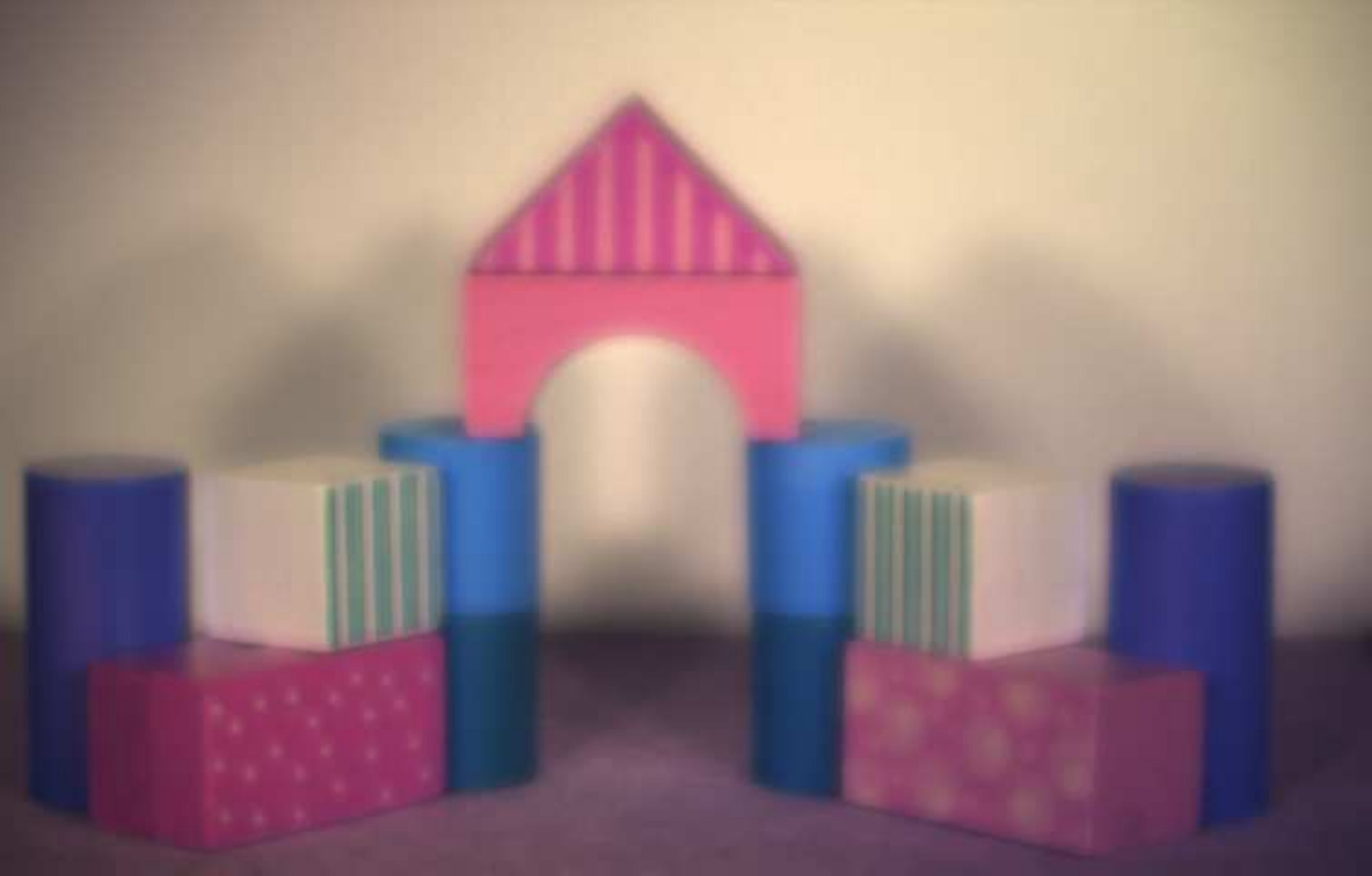}}\\
\subfigure[Hyperspectral images captured with fixed exposure]{\label{fig:const-exp-images}
\includegraphics[width=0.3\linewidth]{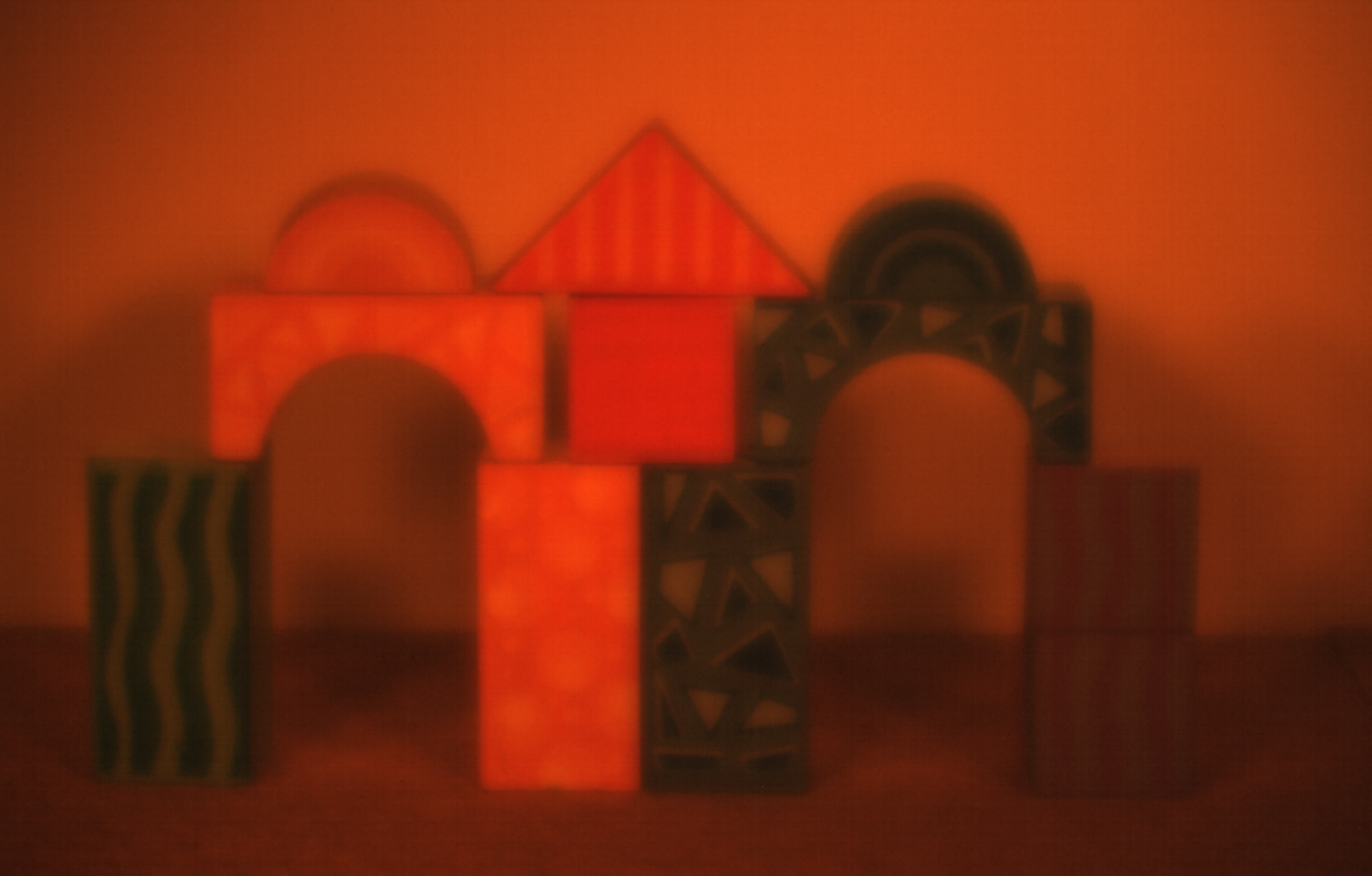}
\includegraphics[width=0.3\linewidth]{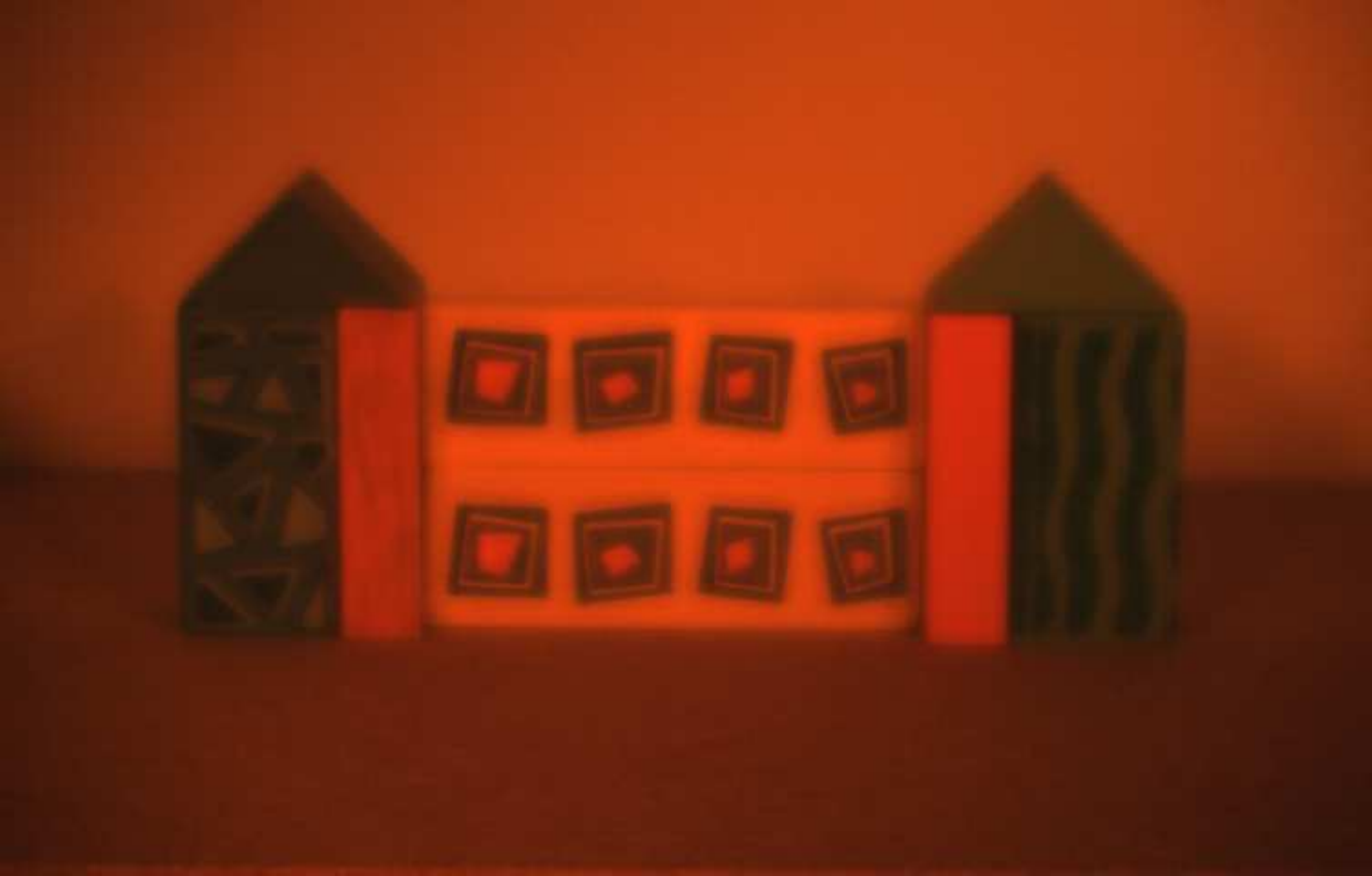}
\includegraphics[width=0.3\linewidth]{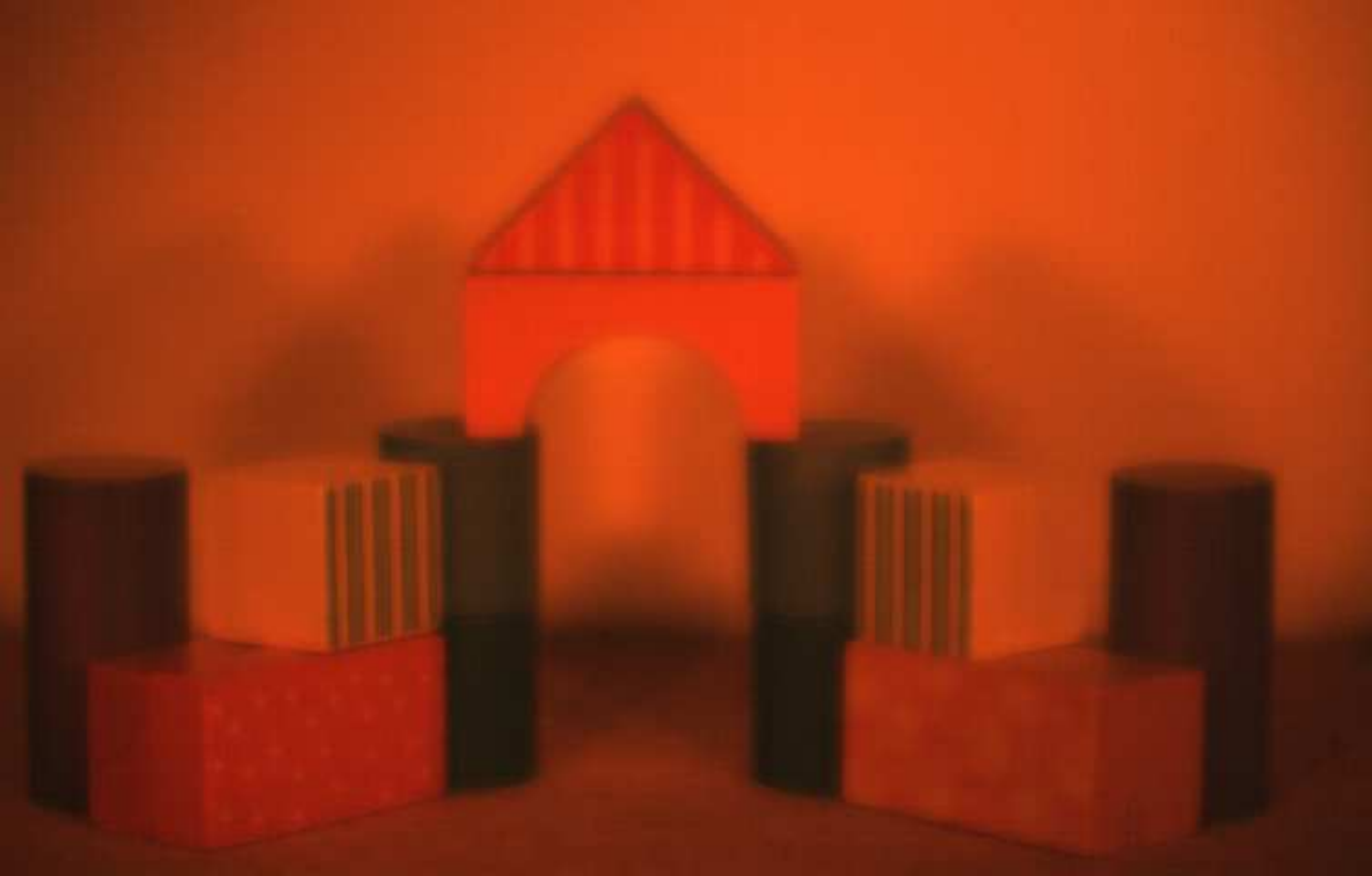}}\\
\subfigure[RGB images captured with a standard digital camera]{\label{fig:rgb-images}
\includegraphics[width=0.3\linewidth]{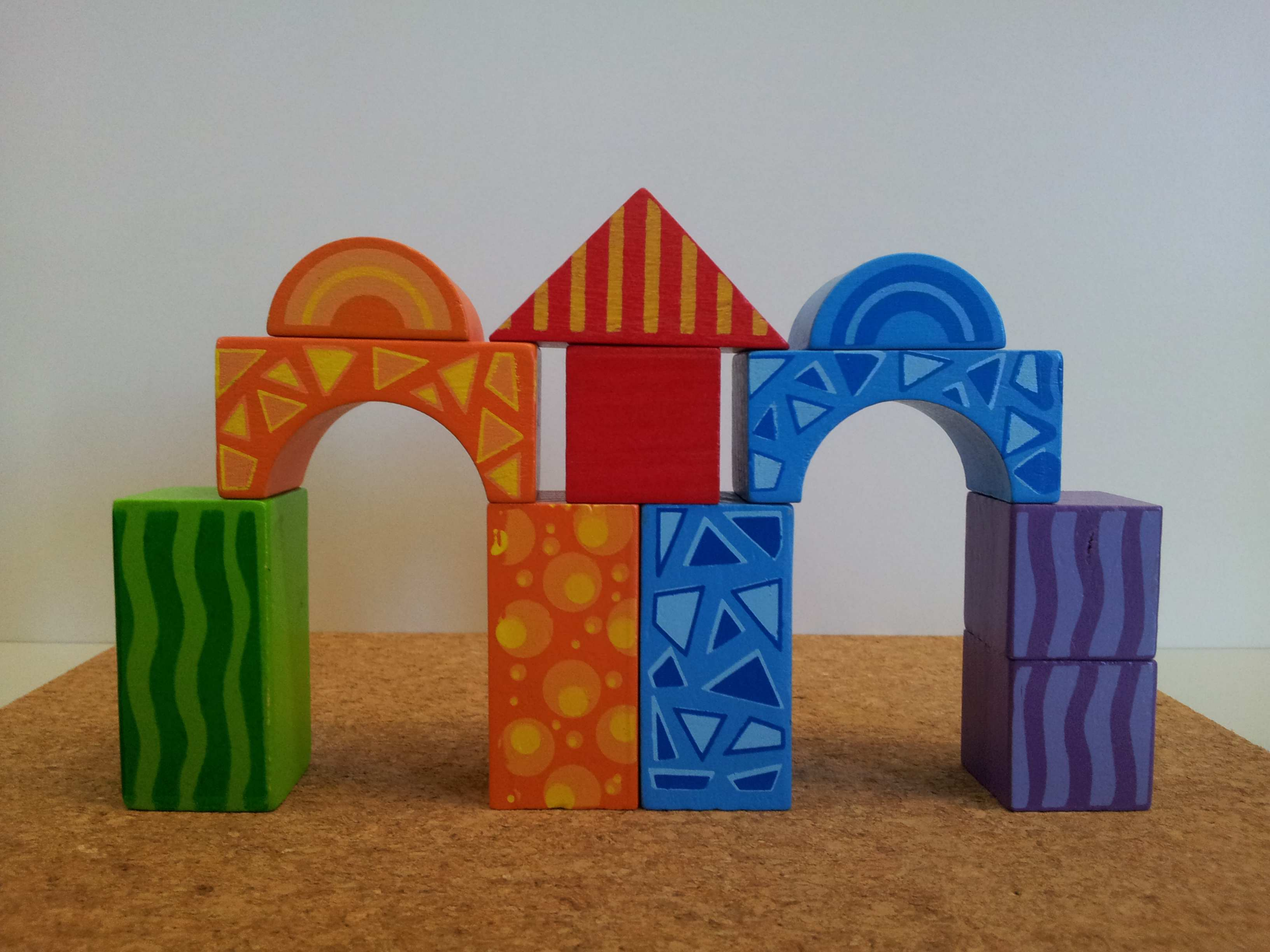}
\includegraphics[width=0.3\linewidth]{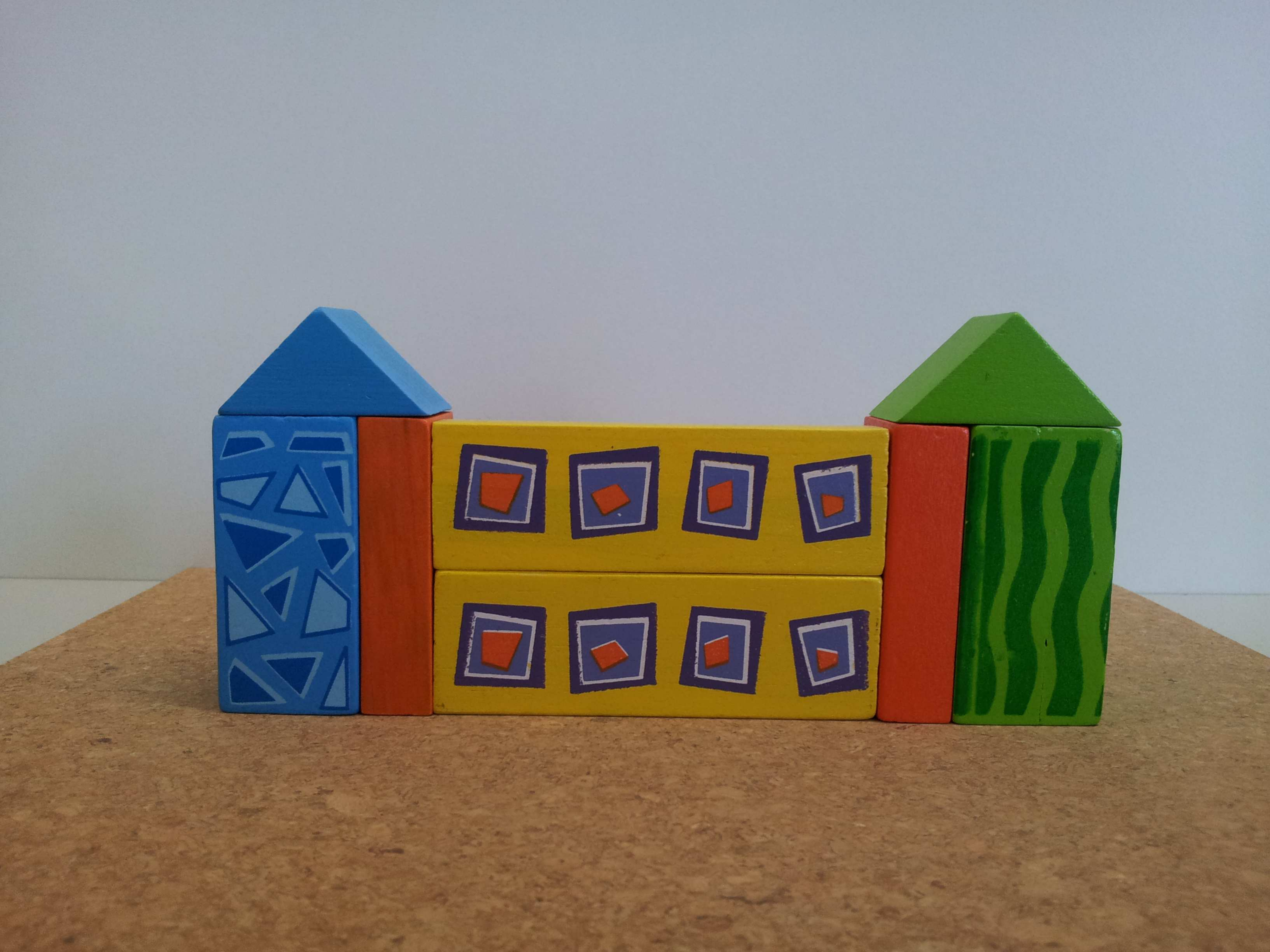}
\includegraphics[width=0.3\linewidth]{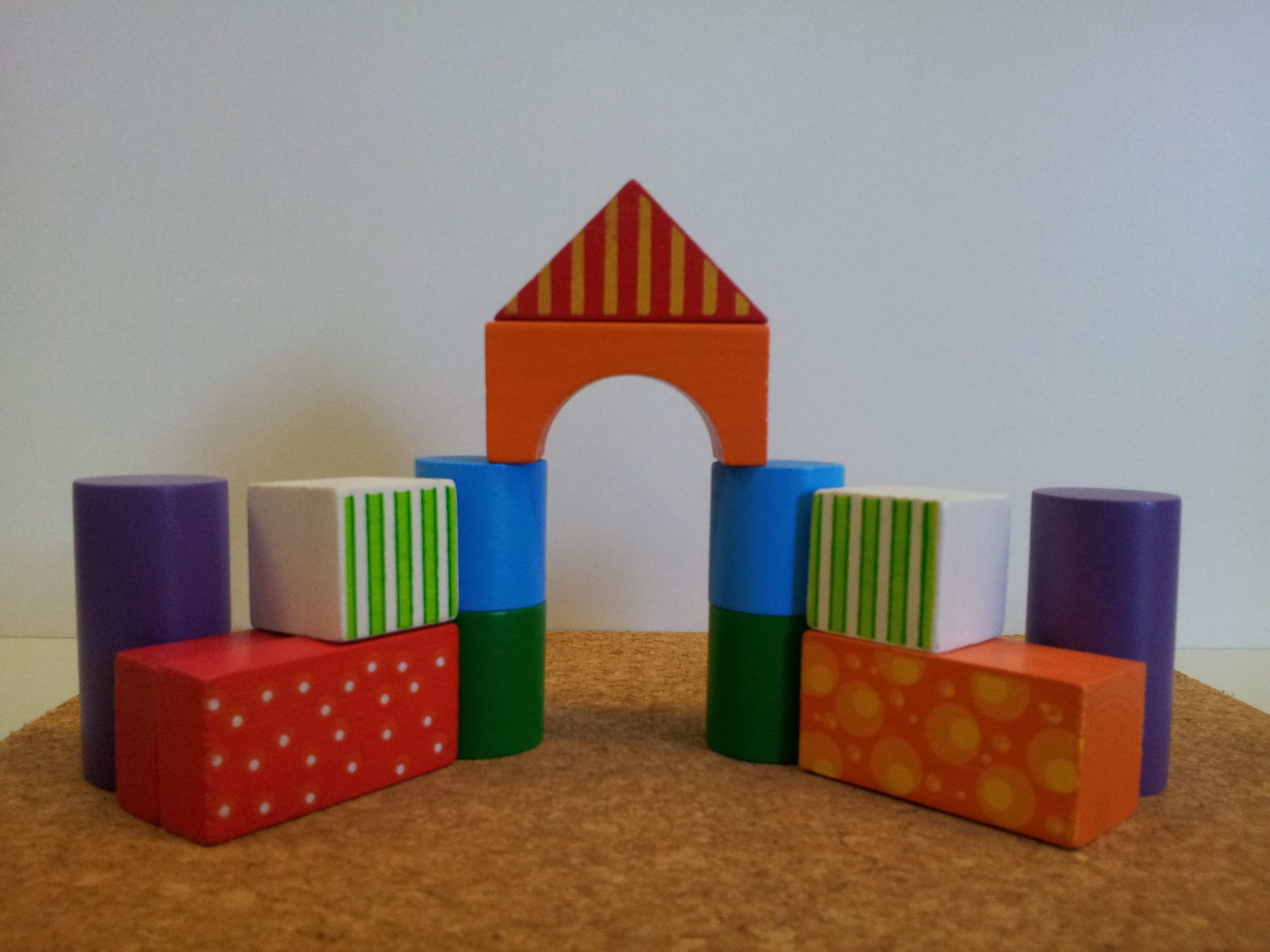}}
\caption[Sample hyperspectral images]{Hyperspectral scenes in the UWA multi-illuminant hyperspectral scene data captured by different methods.}
\label{fig:exp-images}
\end{figure}

\subsubsection{Fixed Exposure Imaging}

An important factor in fixed exposure imaging is to control the exposure time such that the image pixels do not saturate in any band. In a LCTF hyperspectral imaging system, various factors affect the final radiance value measured at the sensor pixel. If these factors are not taken into account, saturation may occur which results in loss of valuable information. Thus for capturing images, we expose the scene for the maximum possible exposure time, that just avoids saturation in any band. Thus, if $t_1, t_2,...,t_\lambda$ are the maximum allowable exposures for each band, then $t_{\textrm{opt}} = \min (t_1,...,t_\lambda)$ is the fixed exposure value for capturing all bands. Theoretically this ensures that no pixel in any band exceeds the camera intensity scale. In order to allow successful image capture in most lighting conditions, we keep the saturation intensity threshold to 180 which is much less than the hardware threshold value (255). This ensures a real intensity value per image pixel, even in adverse lighting conditions and in presence of noise.

Figure~\ref{fig:const-exp-images} shows the rendered hyperspectral images using fixed exposure. It can be observed that the rendered images are not visually similar in appearance to real world RGB images shown in Figure~\ref{fig:rgb-images}. There are two main reasons for this that introduce a combined effect. First is the illumination spectral power distribution which is low at shorter wavelengths and high at larger wavelengths. Second, the LCTF has a variable filter transmission, such that there is less transmission at shorter wavelengths and more at longer wavelengths. Due to this, there is more red illuminant power and the resulting images exhibit a reddish tone. The quantum efficiency of the camera sensor is also variable but not a limiting factor in this case.

\clearpage

\subsection{Dataset Specifications}

\subsubsection{Simulated Data}
In order to evaluate the hyperspectral color constancy algorithms, we perform experiments on simulated in addition to the real data. Experiments on simulated data are important because the true illumination is known for comparison with the estimated illumination. The hyperspectral images of simulated illumination scenes are synthesized from the publicly available CAVE multispectral image database\footnote{CAVE Multispectral Image Database\\ \url{www1.cs.columbia.edu/CAVE/projects/gap_camera/}} which contains true spectral reflectance images. It has 31 band hyperspectral images (420-720nm with 10nm steps) of 32 scenes consisting of a variety of objects at a resolution of $512 \times 512$ pixels. Each image has a color checker chart in place, masked out to avoid bias in illuminant estimation. The advantage of using this dataset is that the true spectral reflectance of the scenes is known, so that a scene can be illuminated by any light source of a known SPD.

The Simon Fraser University (SFU) hyperspectral dataset\footnote{{SFU Hyperspectral Set\\ \url{www.cs.sfu.ca/~colour/data/colour_constancy_synthetic_test_data/index.html}}}~\cite{barnard2002data} contains SPDs of 11 real illuminants in (\emph{Set A}) and 81 real illuminants in (\emph{Set B}). We make use of the \emph{Set A} illuminants to simulate real life lighting scenarios for the images in the CAVE database. The \emph{Set B} illuminants are used to train the SVM to classify the illuminants in \emph{Set A} as smooth or spiky. Each image of the database is illuminated by a source and the estimated illumination is recovered using all the algorithms. The difference between the estimated illuminant spectra compared to ground truth is then measured in terms of the angular error.

\subsubsection{Real Data}

Using the imaging setup described previously, we collected a hyperspectral image dataset of real world scenes. The UWA multi-illuminant hyperspectral scenes dataset contains images of different scenes captured under five real illuminations namely daylight, halogen, fluorescent, and two mixed illuminants, daylight-fluorescent and halogen-fluorescent. Each hyperspectral image has 33 bands in the range (400-720nm with 10nm steps). To create spatially and spectrally diverse scenes, we selected blocks of various shape and color, arranged to form 3 distinct structures.

\begin{figure}[t]
\centering
\includegraphics[trim = 100pt 15pt 115pt 2pt, clip, width=0.34\linewidth]{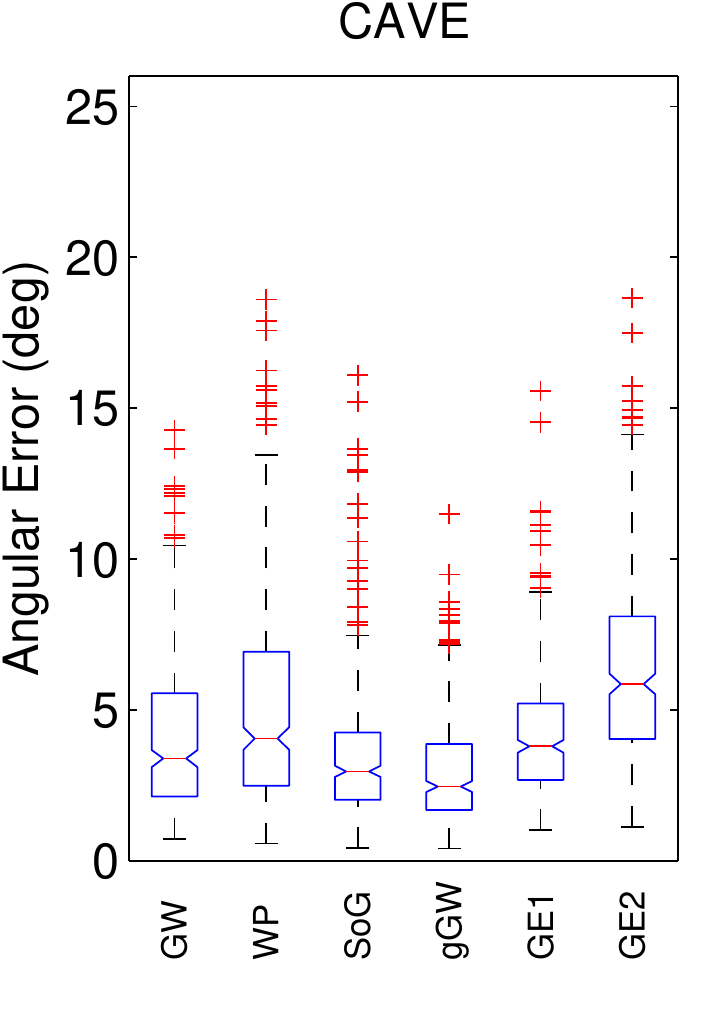}
\includegraphics[trim = 100pt 15pt 115pt 2pt, clip, width=0.34\linewidth]{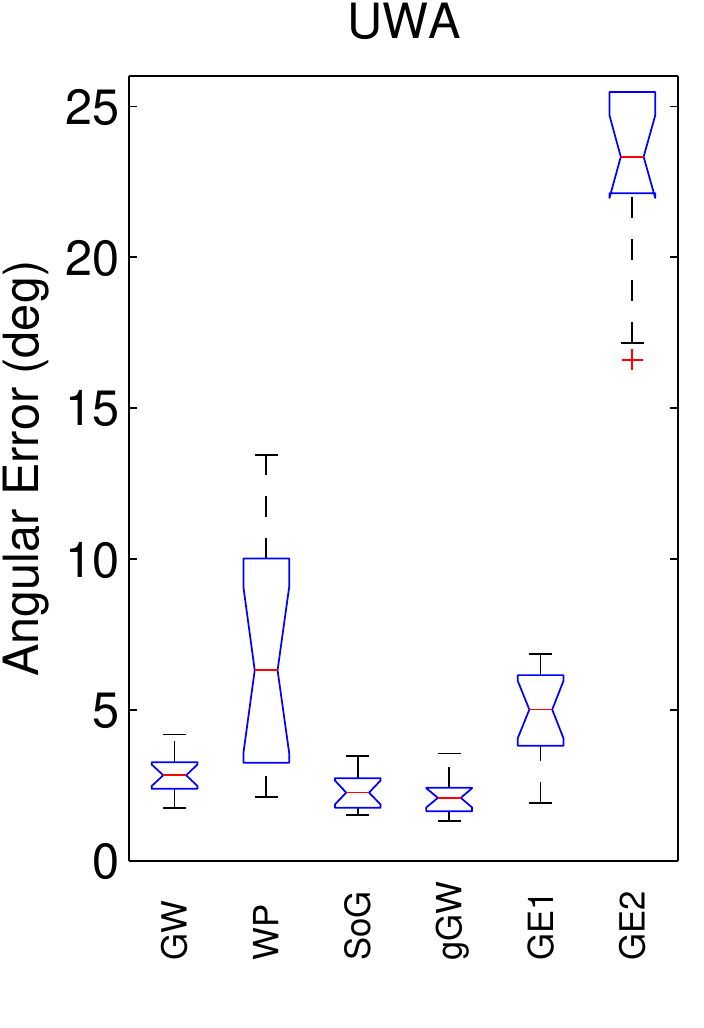}
\caption[Angular errors for individual algorithms]{Distribution of angular errors in simulated and real datasets. Observe that the GW, SoG and gGW algorithms achieve the lowest mean angular errors on both databases.}
\label{fig:boxplot-ind}
\end{figure}

\section{Results and Discussion}
\label{sec:results}

\subsection{Individual and Combinational Color Constancy Methods}
\label{sec:sing-comb}


In experiments, we first present the angular error distributions in the form of a boxplot\footnotemark~for all color constancy algorithms (see Figure~\ref{fig:boxplot-ind}). The results are without the adaptive spatio-spectral support.
We observe that the gGW algorithm achieves the lowest mean angular error (MAE). Analysis of the edge based color constancy algorithms GE1 and GE2 indicates that the first order derivative assumption holds better compared to the second order derivative. Overall, GW, WP, SoG and gGW exhibit comparable performances with slight variation.
\footnotetext{Boxplot: On each box, the central mark is the median, the lower and upper edge of the box are the 25th and 75th percentiles, respectively. The whiskers extend to the most extreme data points not considered outliers, and the outliers are plotted individually as red crosses. Two medians are significantly different at the $5\%$ significance level if their intervals do not overlap. Interval endpoints are the extremes of the notches.}

\begin{figure}[h]
\centering
\includegraphics[trim = 40pt 2pt 45pt 15pt, clip, width=0.66\linewidth]{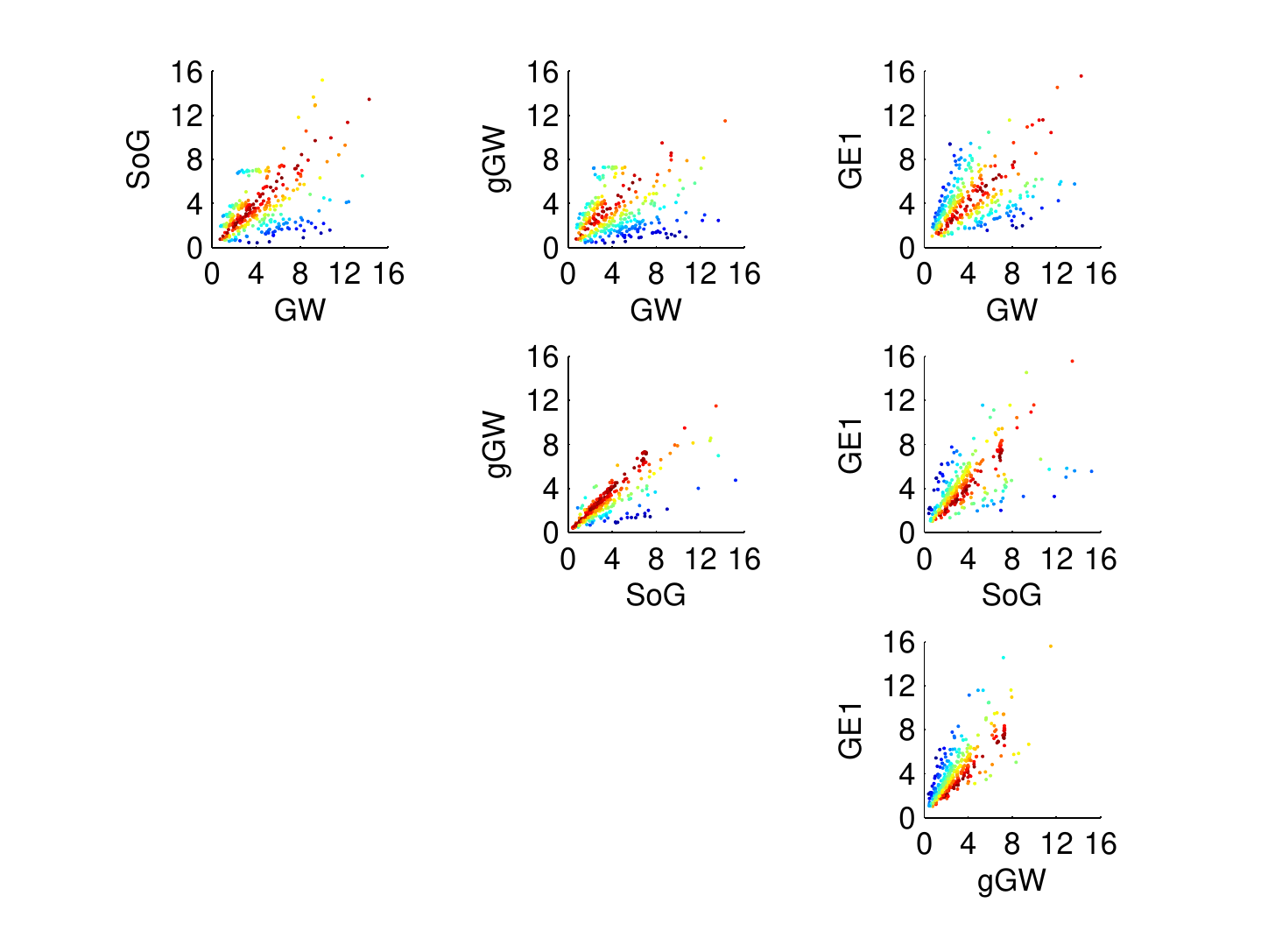}
\caption[Correlation of the angular errors for the best individual algorithm pairs]{Correlation of the angular errors for the best individual algorithm pairs on simulated data. Notice the errors of GW algorithm (top row) are least correlated with other algorithms (SoG, gGW, GE1) making them a preferred choice for CbC.}
\label{fig:corr-sing}
\end{figure}

\begin{figure}[h]
\centering
\includegraphics[trim = 100pt 15pt 115pt 2pt, clip, width=0.34\linewidth]{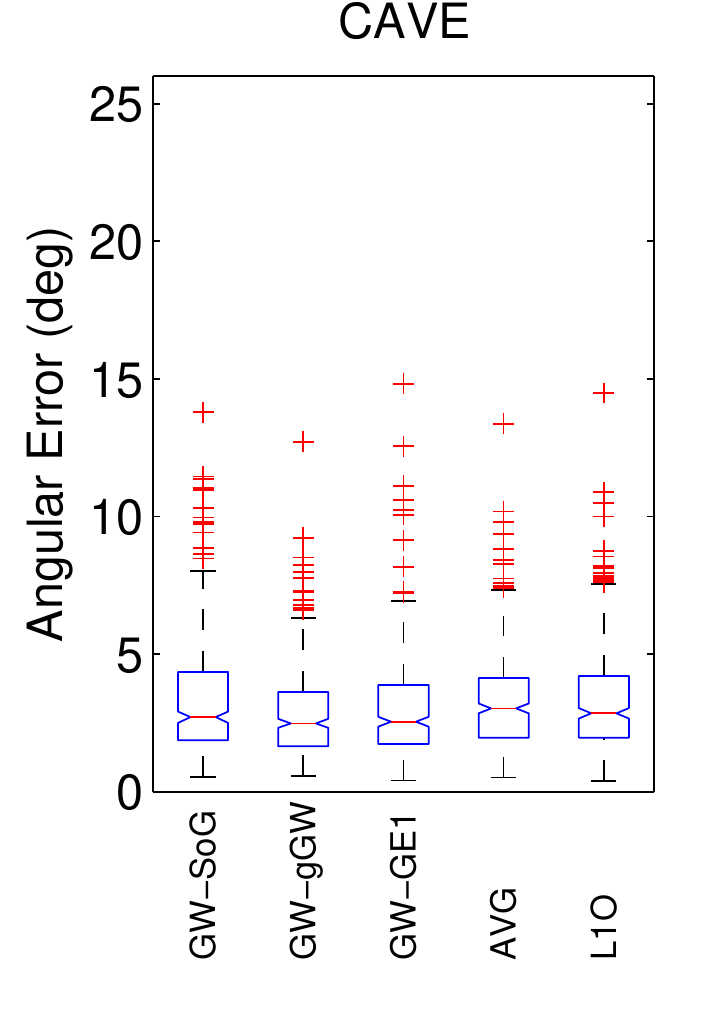}
\includegraphics[trim = 100pt 15pt 115pt 2pt, clip, width=0.34\linewidth]{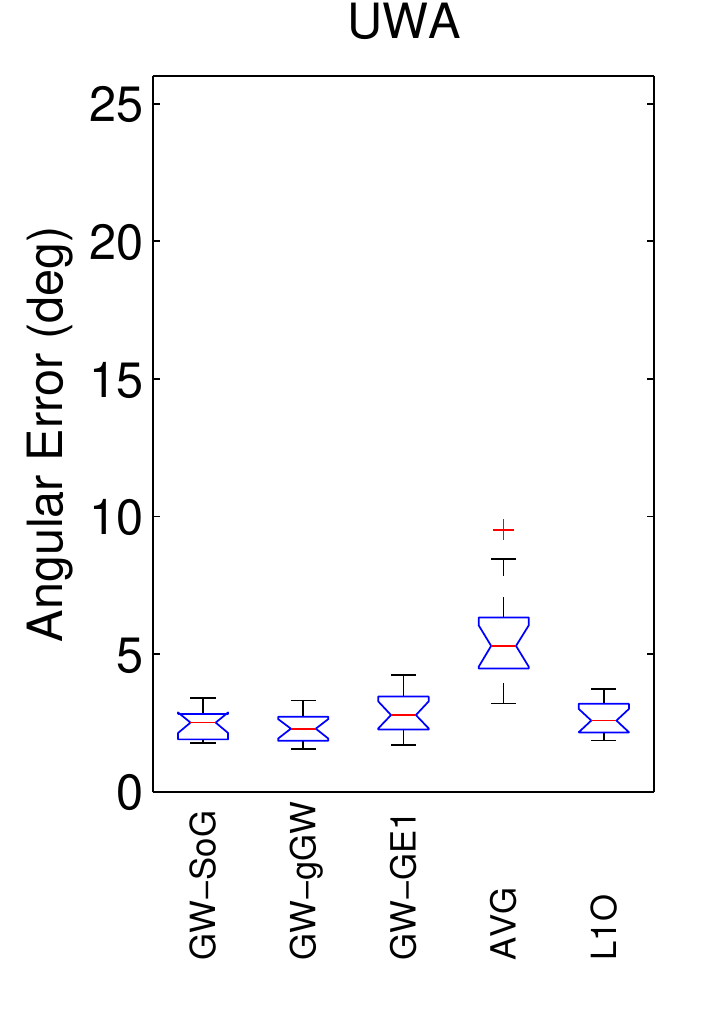}
\caption[Angular errors for combinational algorithms]{Distribution of angular errors for combinational algorithms on simulated and real data. The correlation based combinations outperform other combinational strategies.}
\label{fig:boxplot-cmb}
\end{figure}


As mentioned previously in Section~\ref{sec:cmb}, we analyze the error correlation of the four best performing individual algorithms. A close analysis of the scatter plots in Figure~\ref{fig:corr-sing} reveals that the output of GW algorithm is relatively less correlated with that of the other algorithms. On the other hand, the errors of SoG, gGW and GE1 are more correlated. This leads to the inference that GW combined with any of the other three algorithms should yield better illumination estimates. Therefore, we devise three correlation based combinations, GW-SoG, GW-gGW and GW-GE1 to include in the list of combinational algorithms. The distribution of angular errors for all combinational algorithms is shown in Figure~\ref{fig:boxplot-cmb}. The MAE of any combinational method is either better or equal to that of its respective individual algorithm. Another observation is that the correlation based combinational algorithms are robust to outlier prediction, compared to all other algorithms.

A qualitative comparison of the individual and combinational color constancy algorithms is shown in Figure~\ref{fig:qual-sing-comb}. In these examples, we observe that the error of the best performing combinational algorithm is smaller than the error of the best individual algorithm. Interestingly, the error of the worst performing combinational algorithm is much smaller than the error of the worst individual algorithm. This is a clear advantage of using combinational algorithms with small minimum and maximum error bounds for robust illumination estimates.

\begin{figure}[!h]
\renewcommand{\thesubfigure}{\relax}
\centering
\subfigure[Original ($18.77^{\circ}$)]{\includegraphics[width=0.155\linewidth]{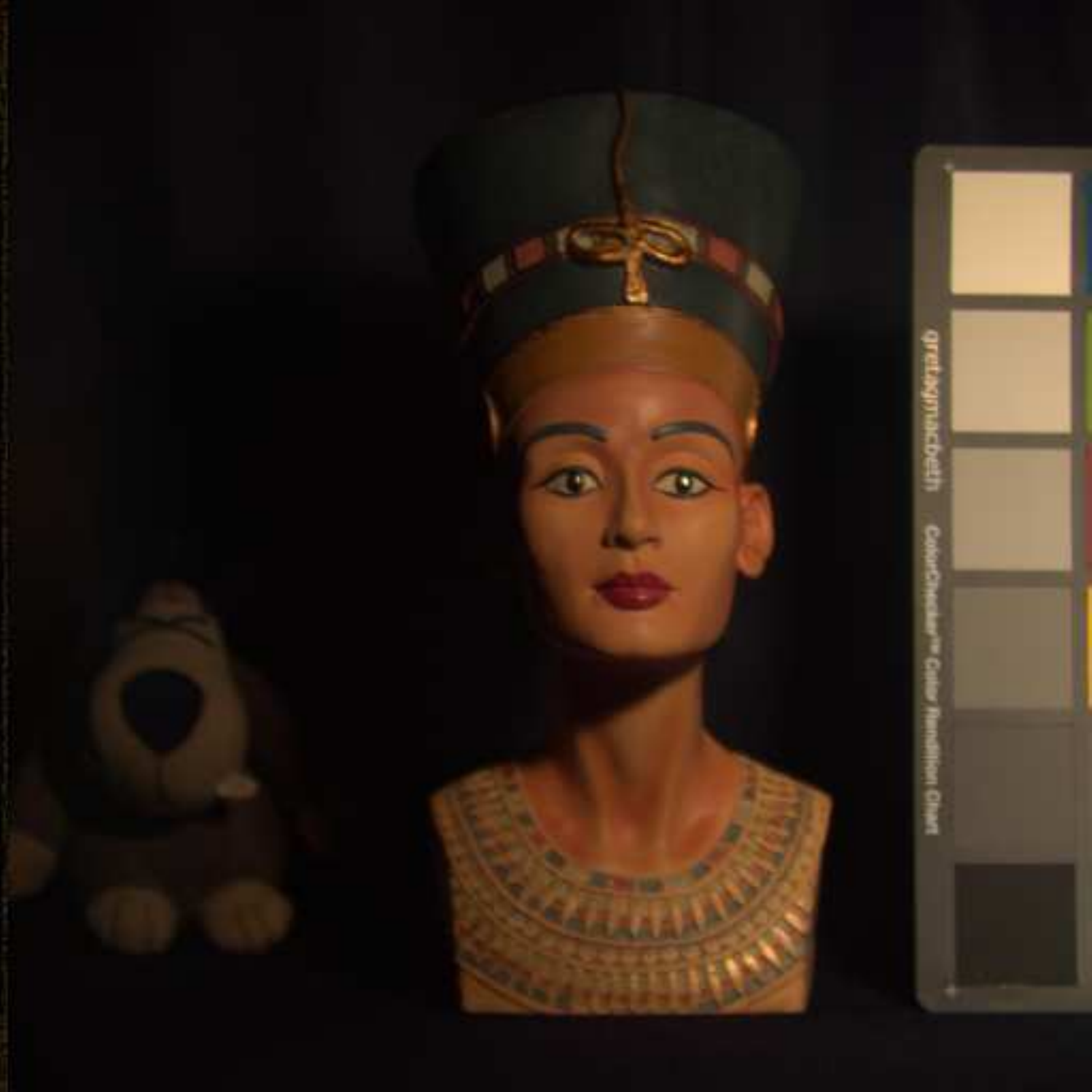}}
\subfigure[Ideal ($0^{\circ}$)]{\includegraphics[width=0.155\linewidth]{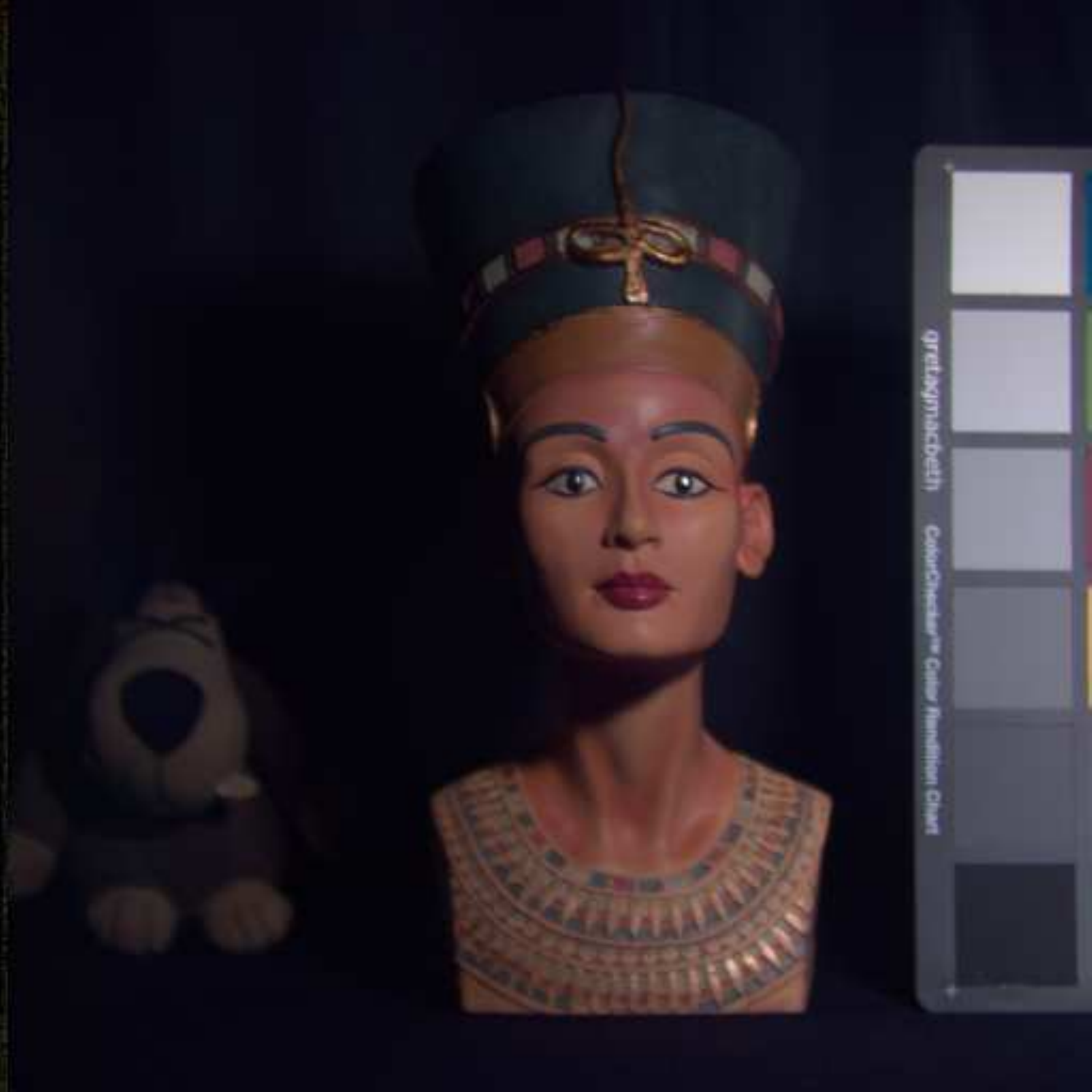}}
\subfigure[gGW ($1.73^{\circ}$)]{\includegraphics[width=0.155\linewidth]{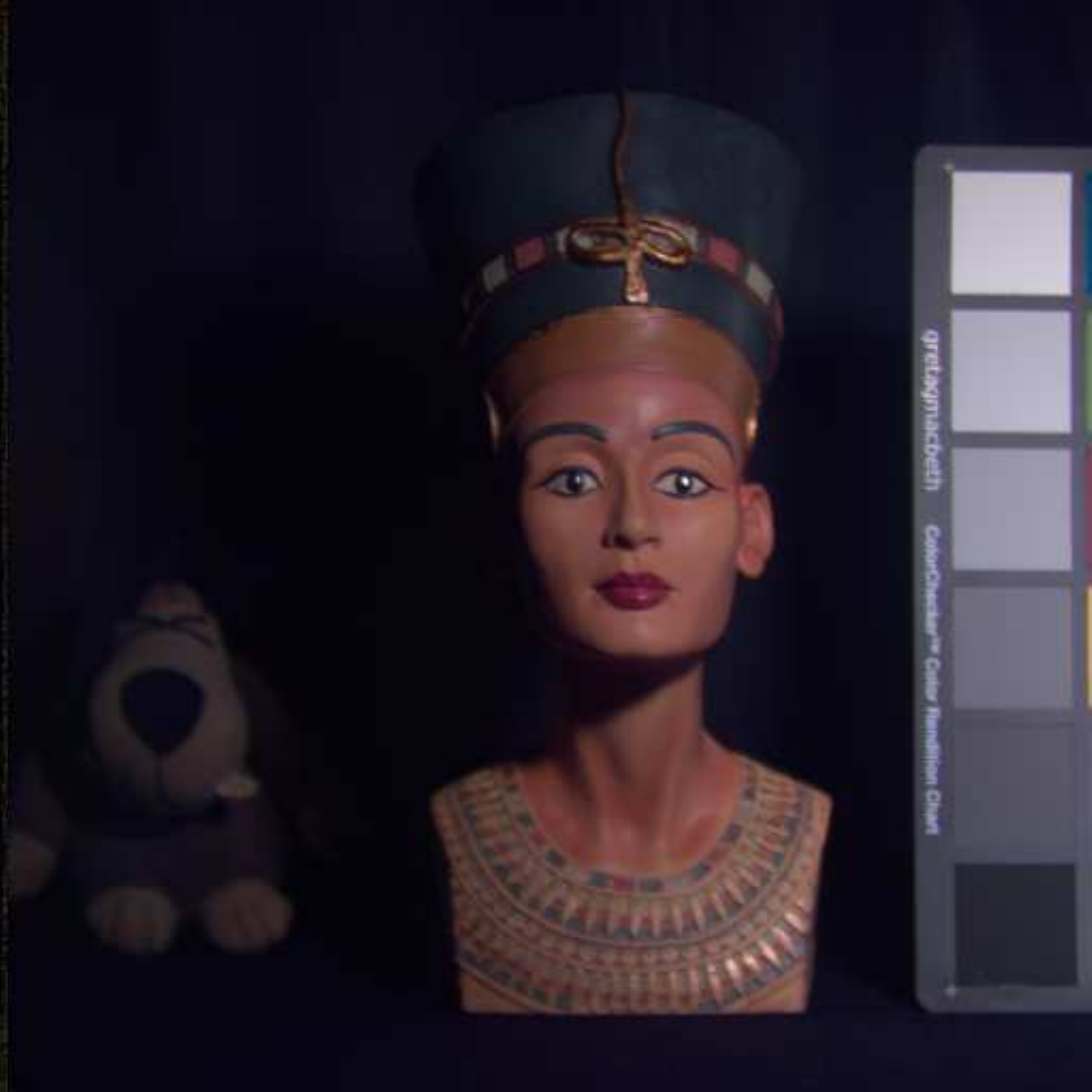}}
\subfigure[GW-gGW($2.32^{\circ}$)]{\includegraphics[width=0.155\linewidth]{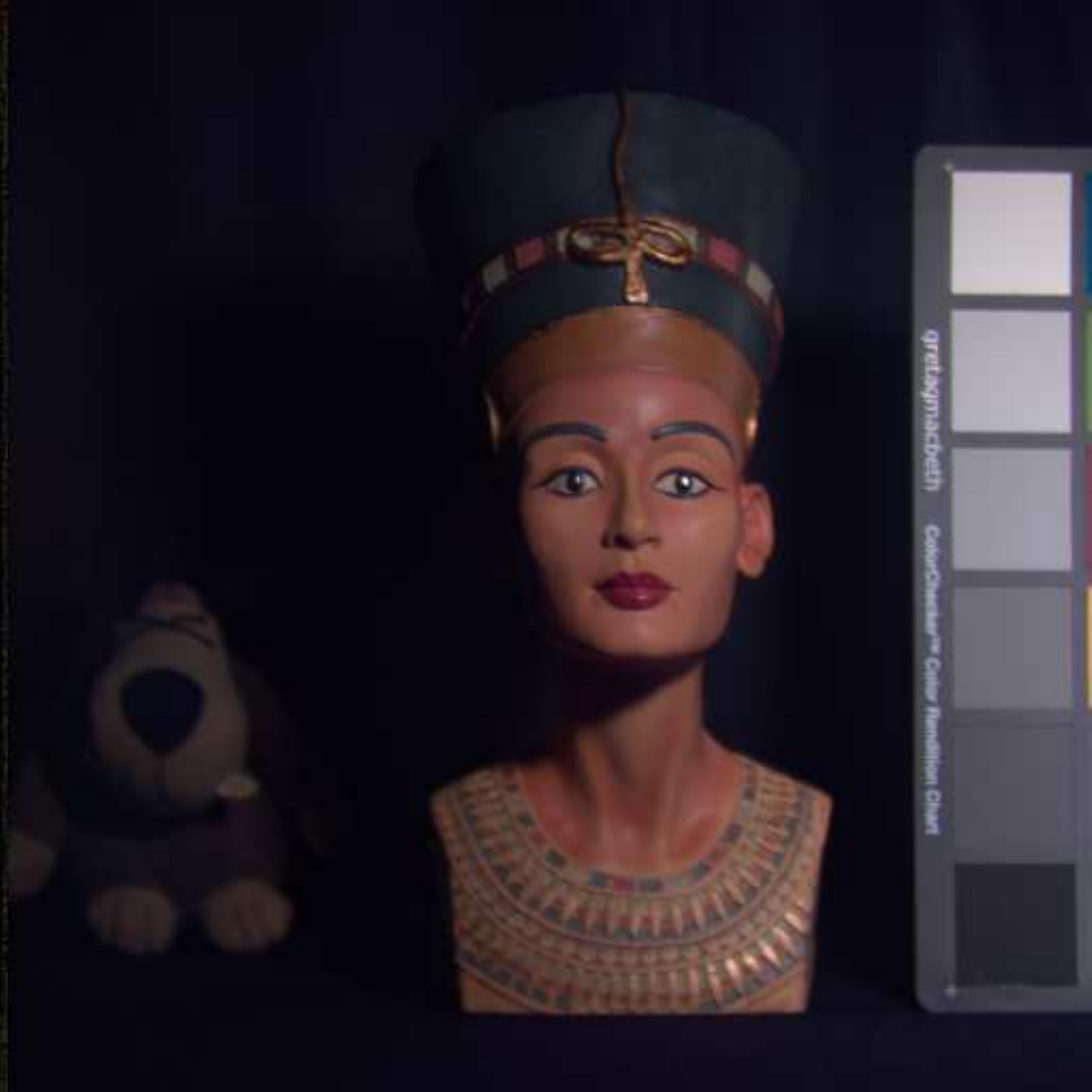}}
\subfigure[WP ($6.46^{\circ}$)]{\includegraphics[width=0.155\linewidth]{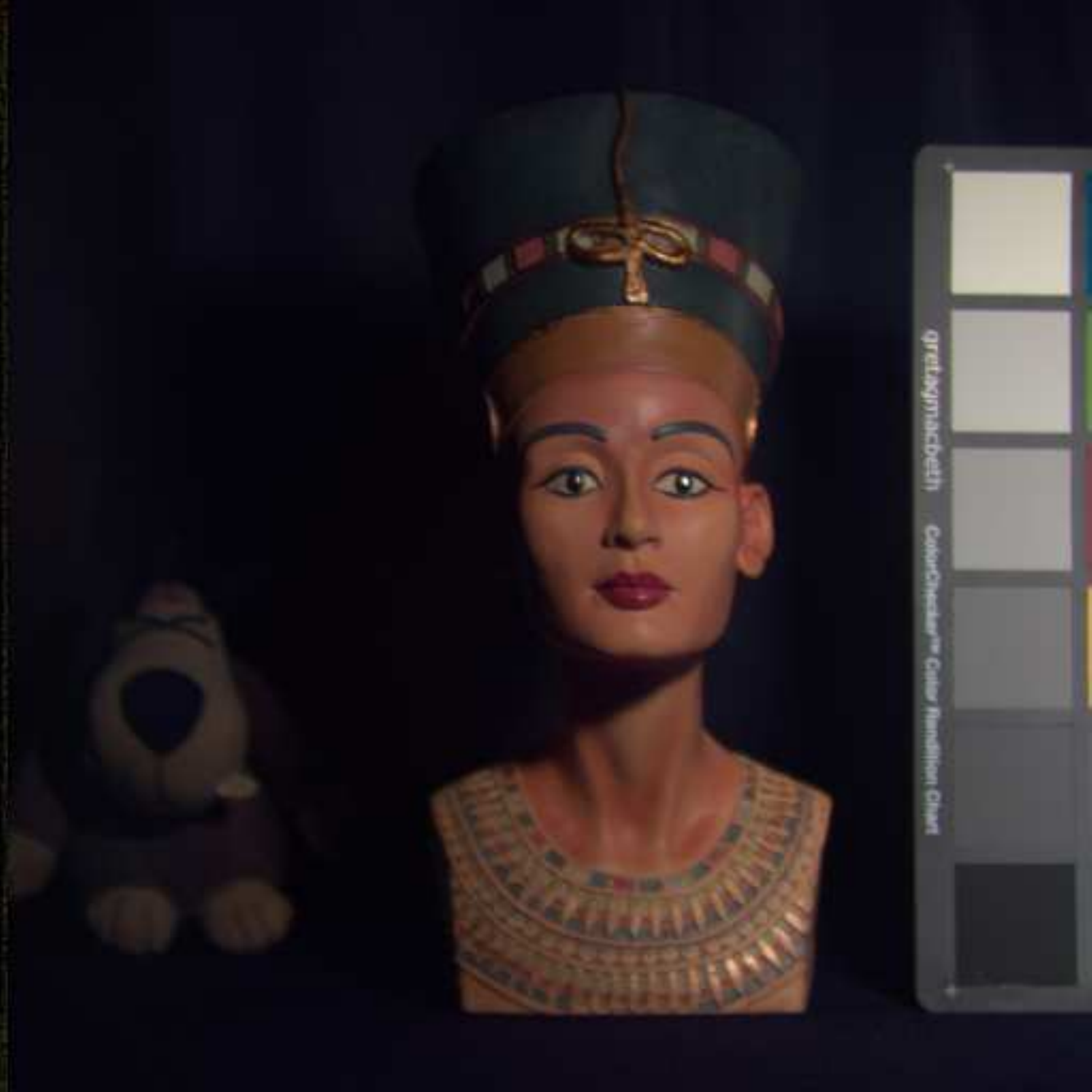}}
\subfigure[GW-SoG ($3.43^{\circ}$)]{\includegraphics[width=0.155\linewidth]{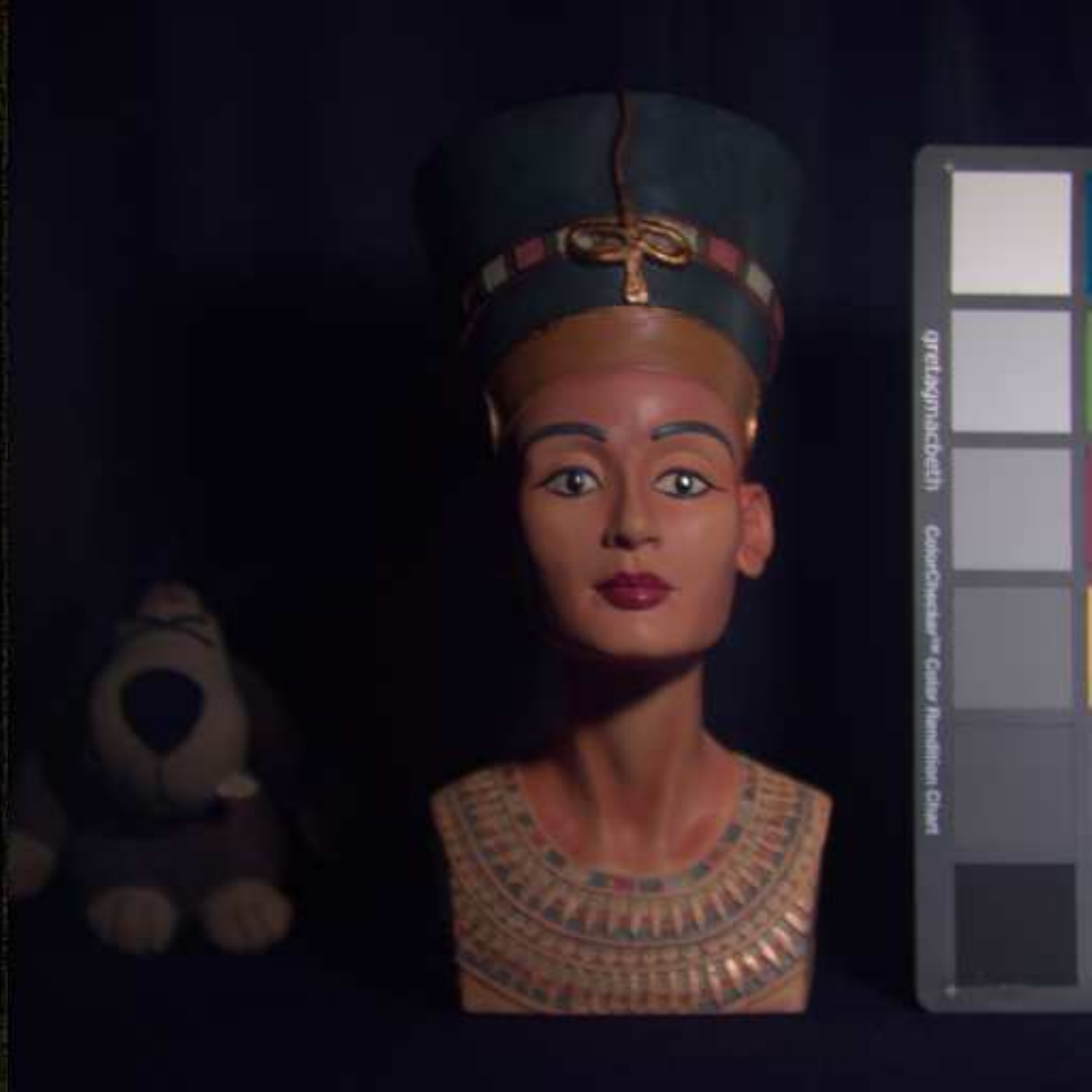}}\\
\vspace{10pt}
\subfigure[Original ($24.4^{\circ}$)]{\includegraphics[width=0.155\linewidth]{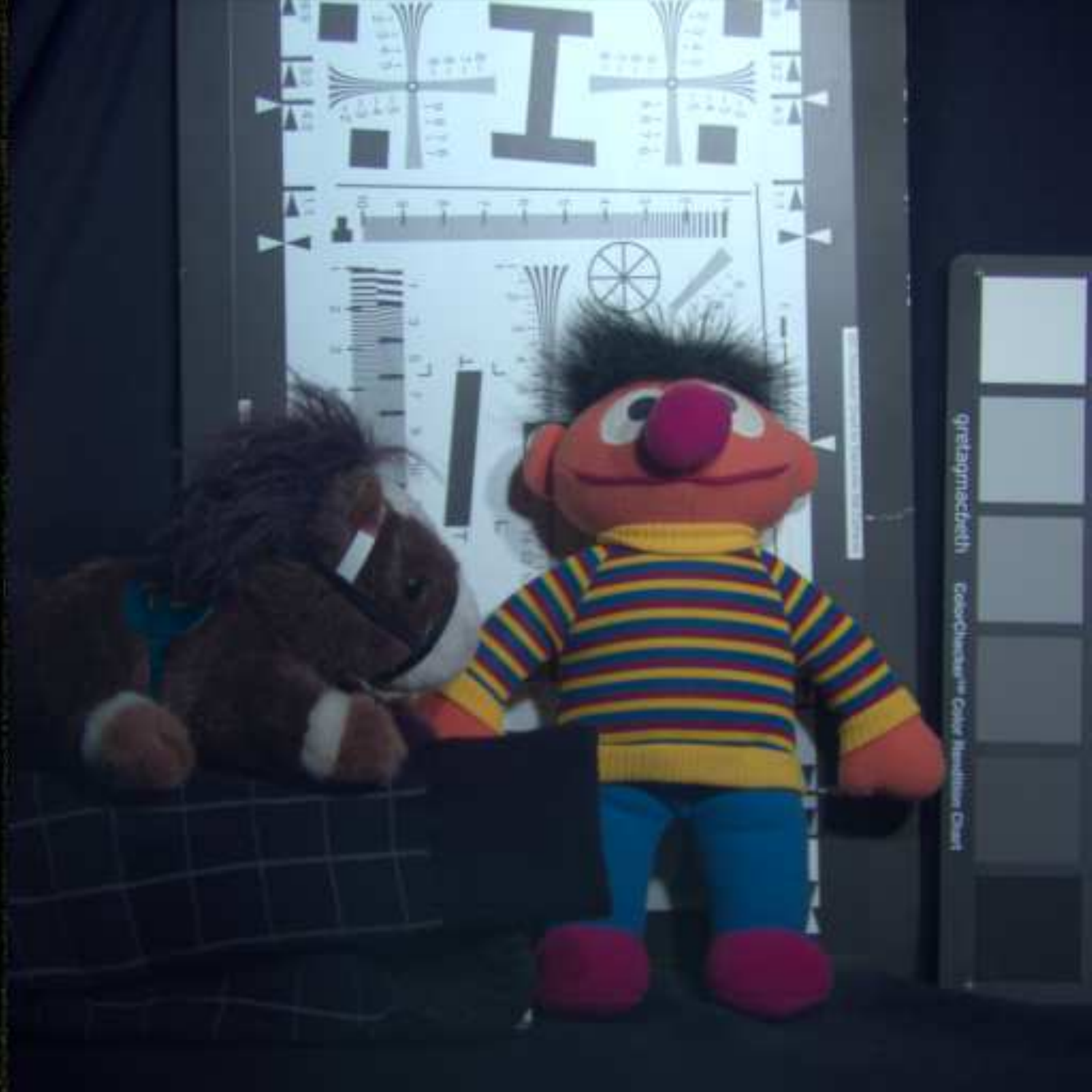}}
\subfigure[Ideal ($0^{\circ}$)]{\includegraphics[width=0.155\linewidth]{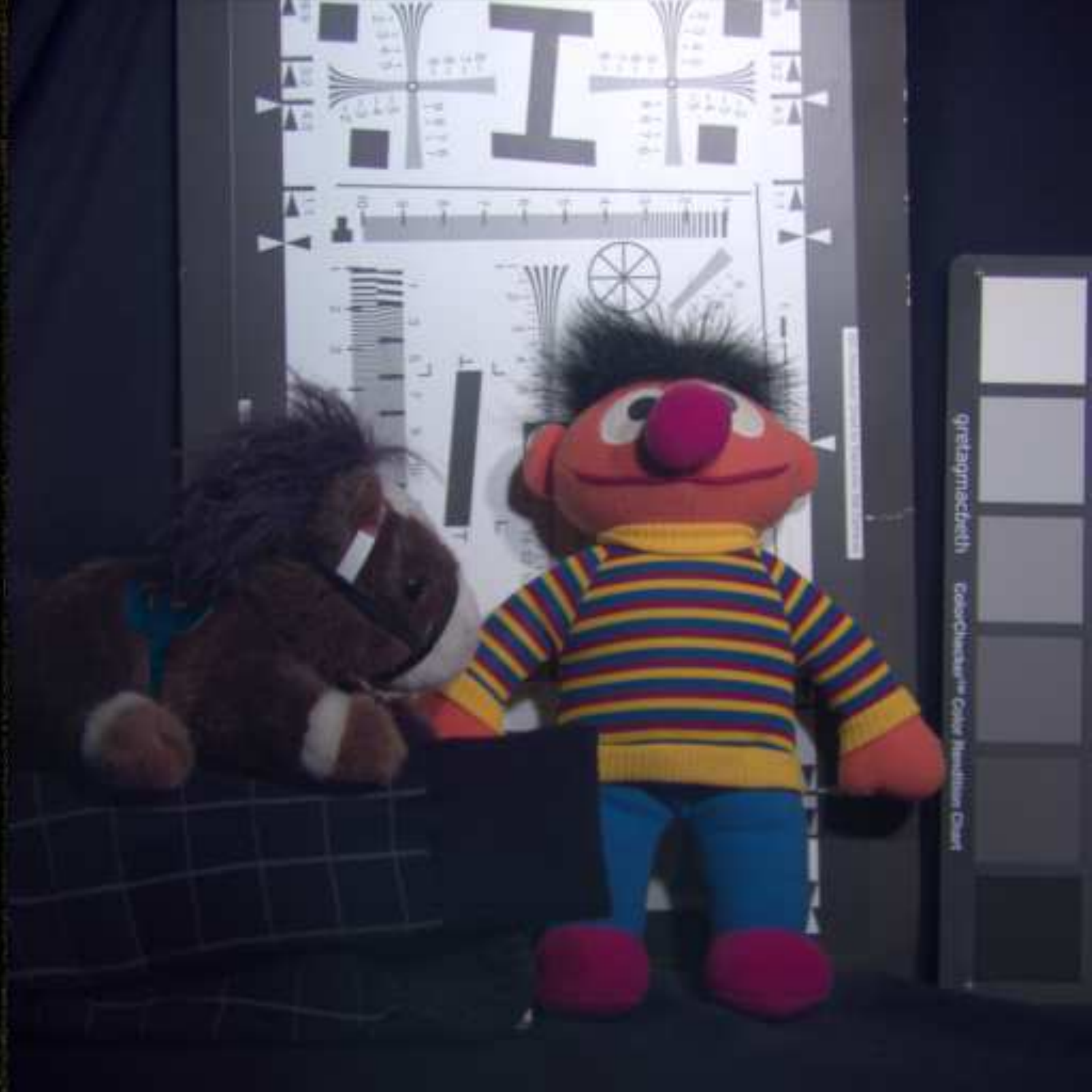}}
\subfigure[GW ($2.8^{\circ}$)]{\includegraphics[width=0.155\linewidth]{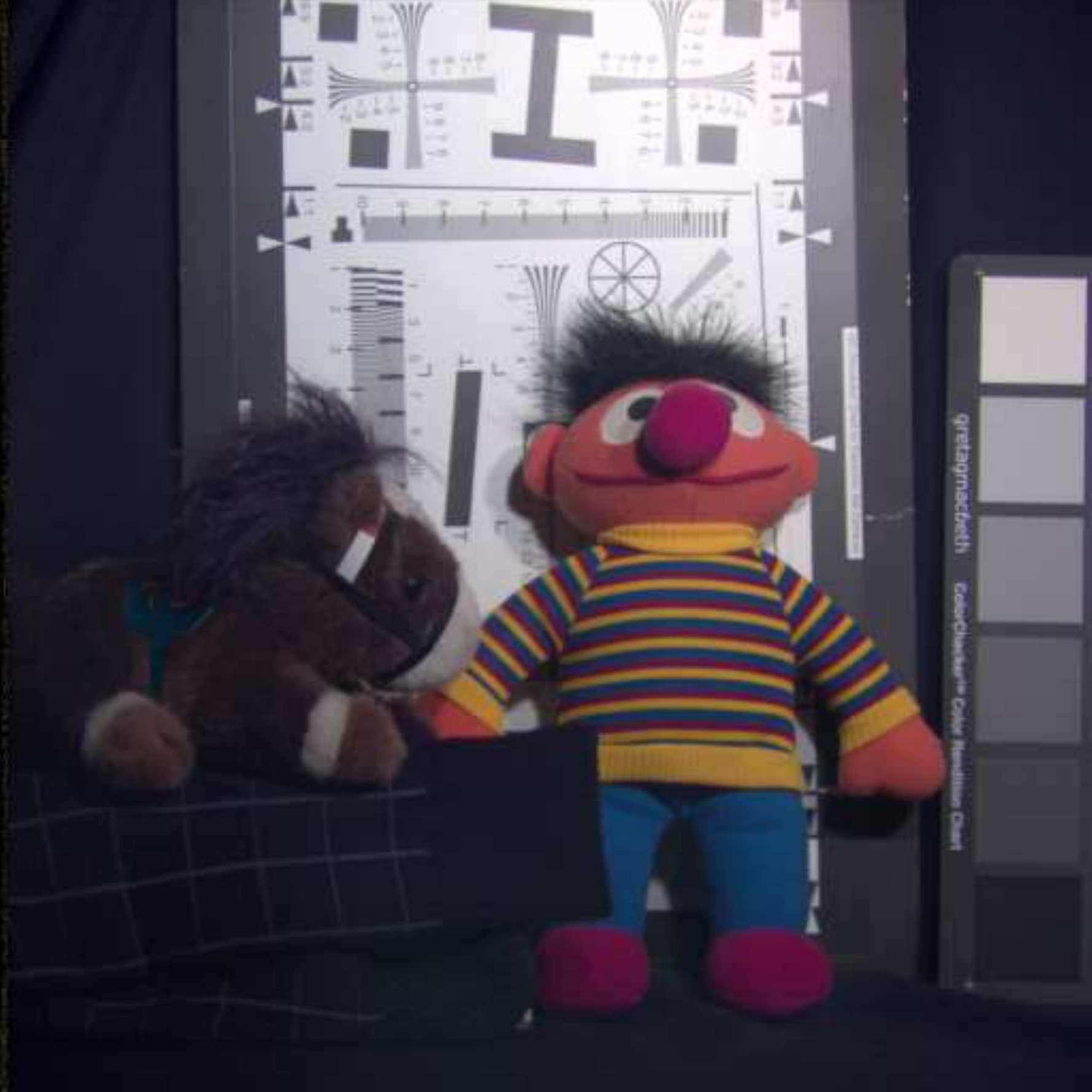}}
\subfigure[GW-gGW ($2.5^{\circ}$)]{\includegraphics[width=0.155\linewidth]{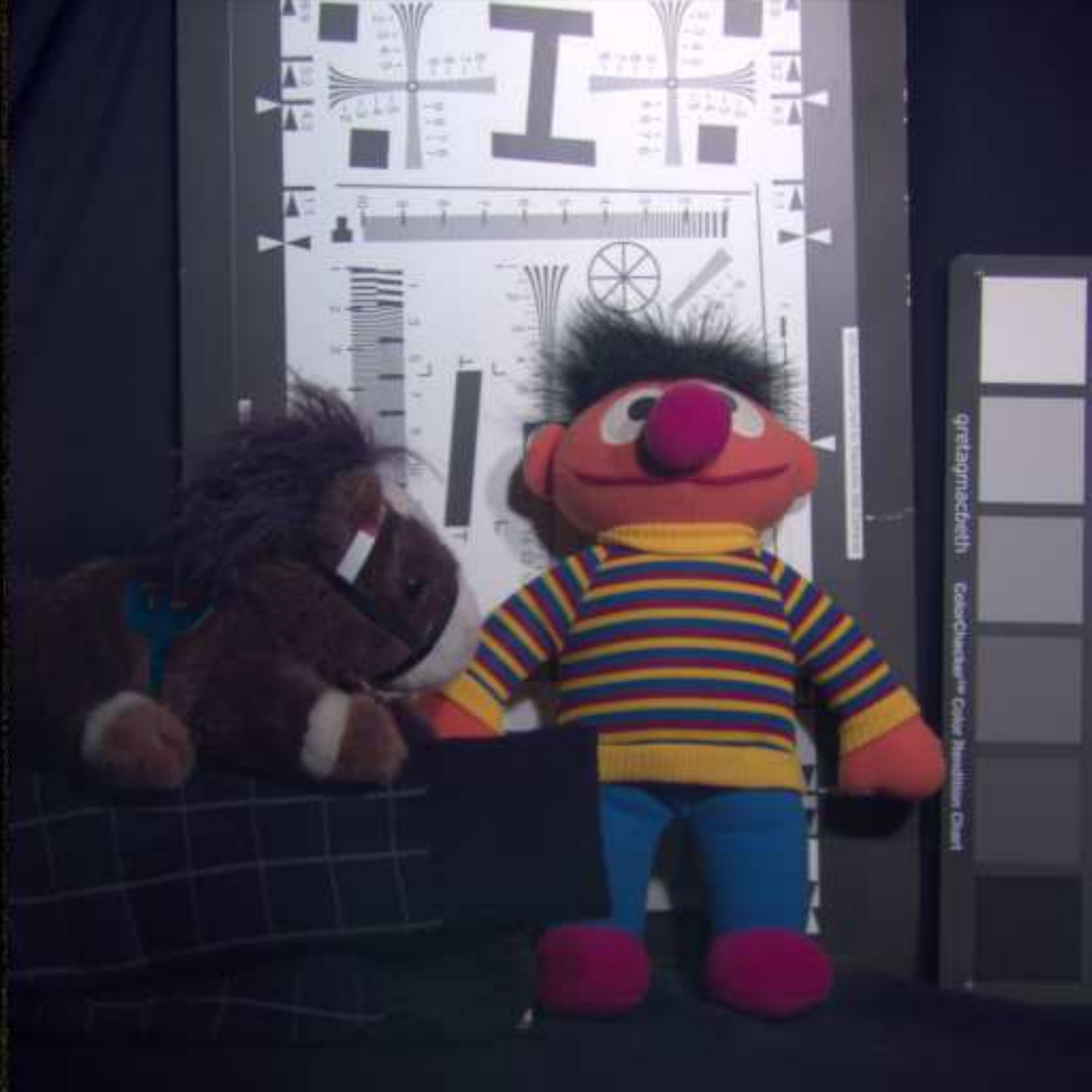}}
\subfigure[WP ($9.0^{\circ}$)]{\includegraphics[width=0.155\linewidth]{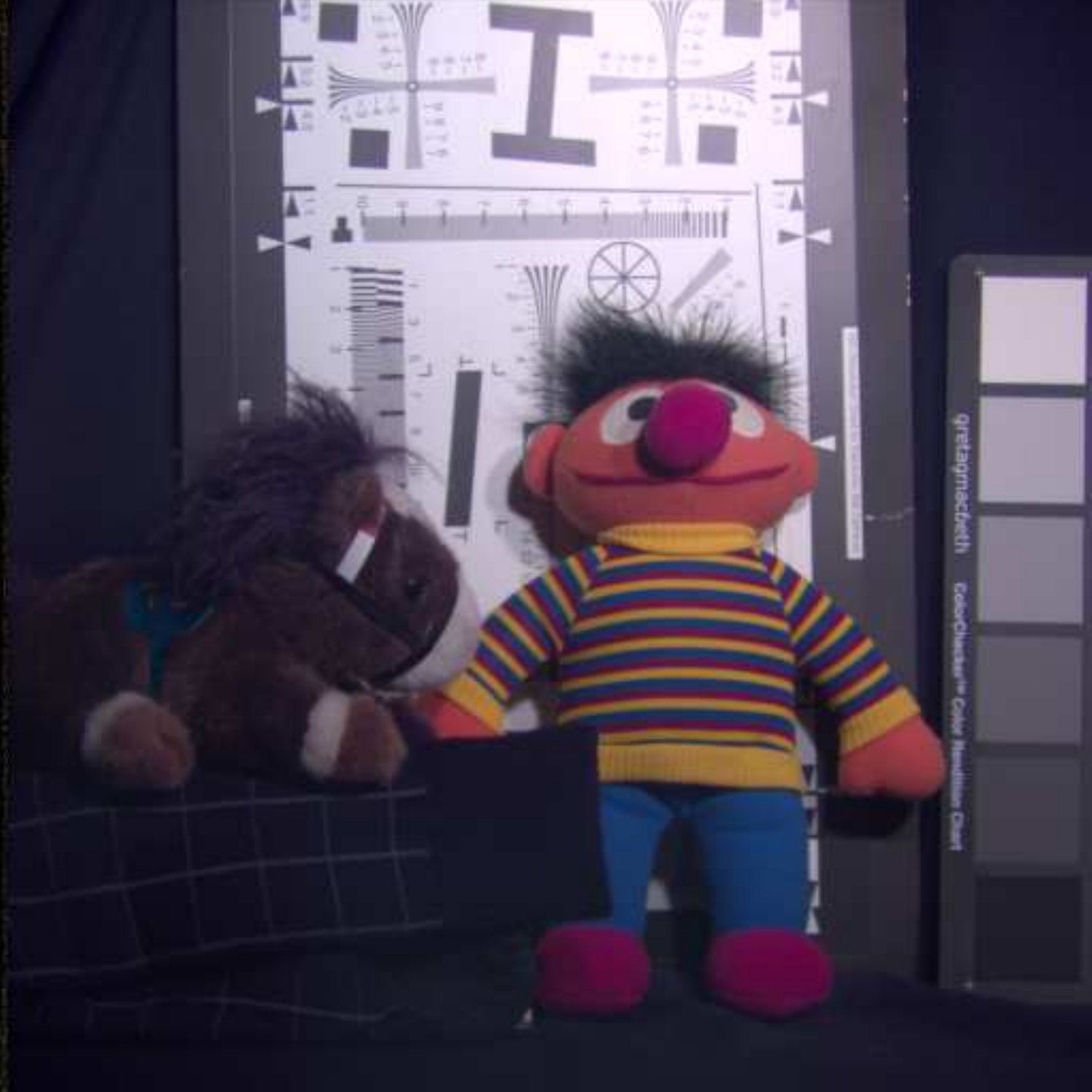}}
\subfigure[AVG ($4.4^{\circ}$)]{\includegraphics[width=0.155\linewidth]{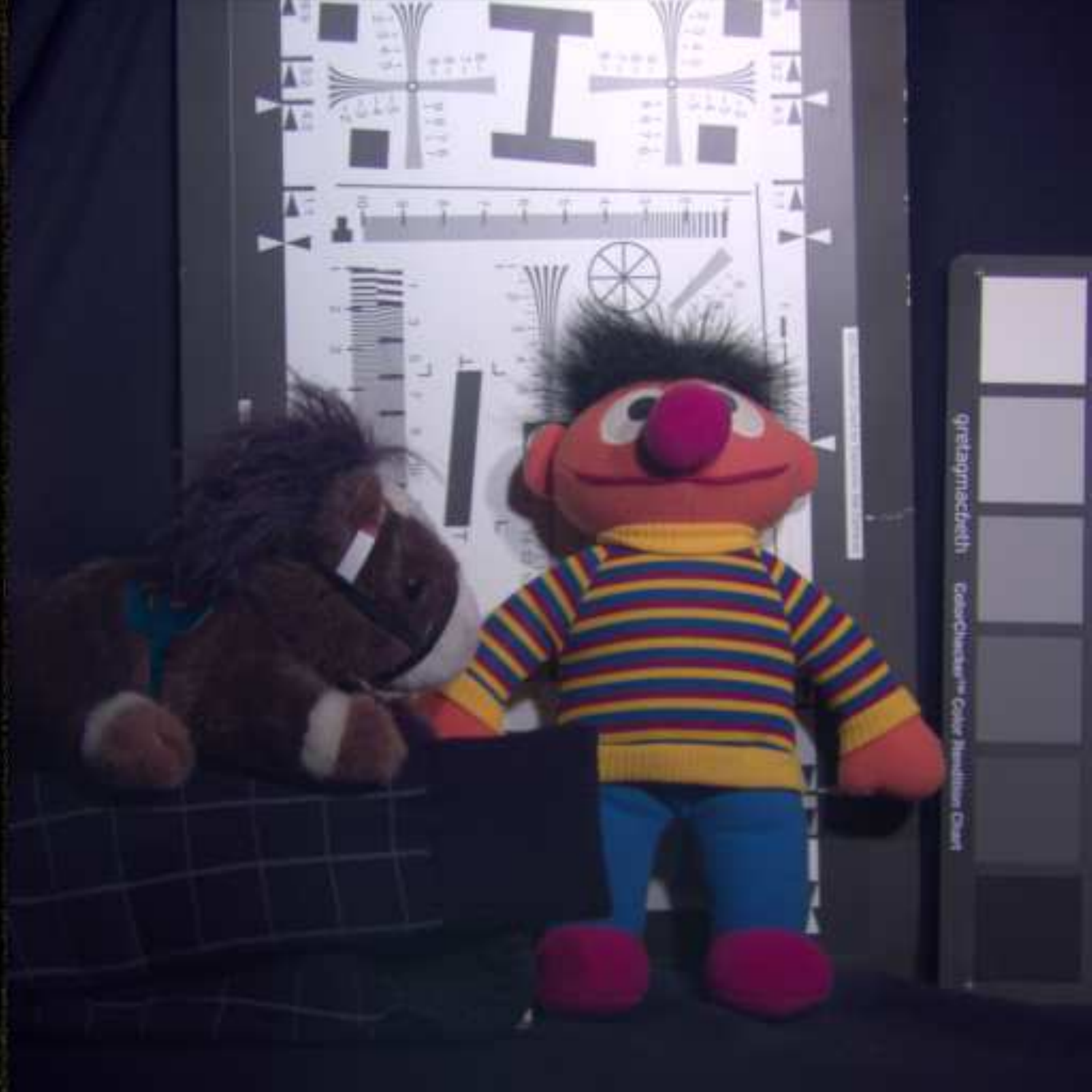}}\\
\vspace{10pt}
\subfigure[Original ($8.2^{\circ}$)]{\includegraphics[width=0.155\linewidth]{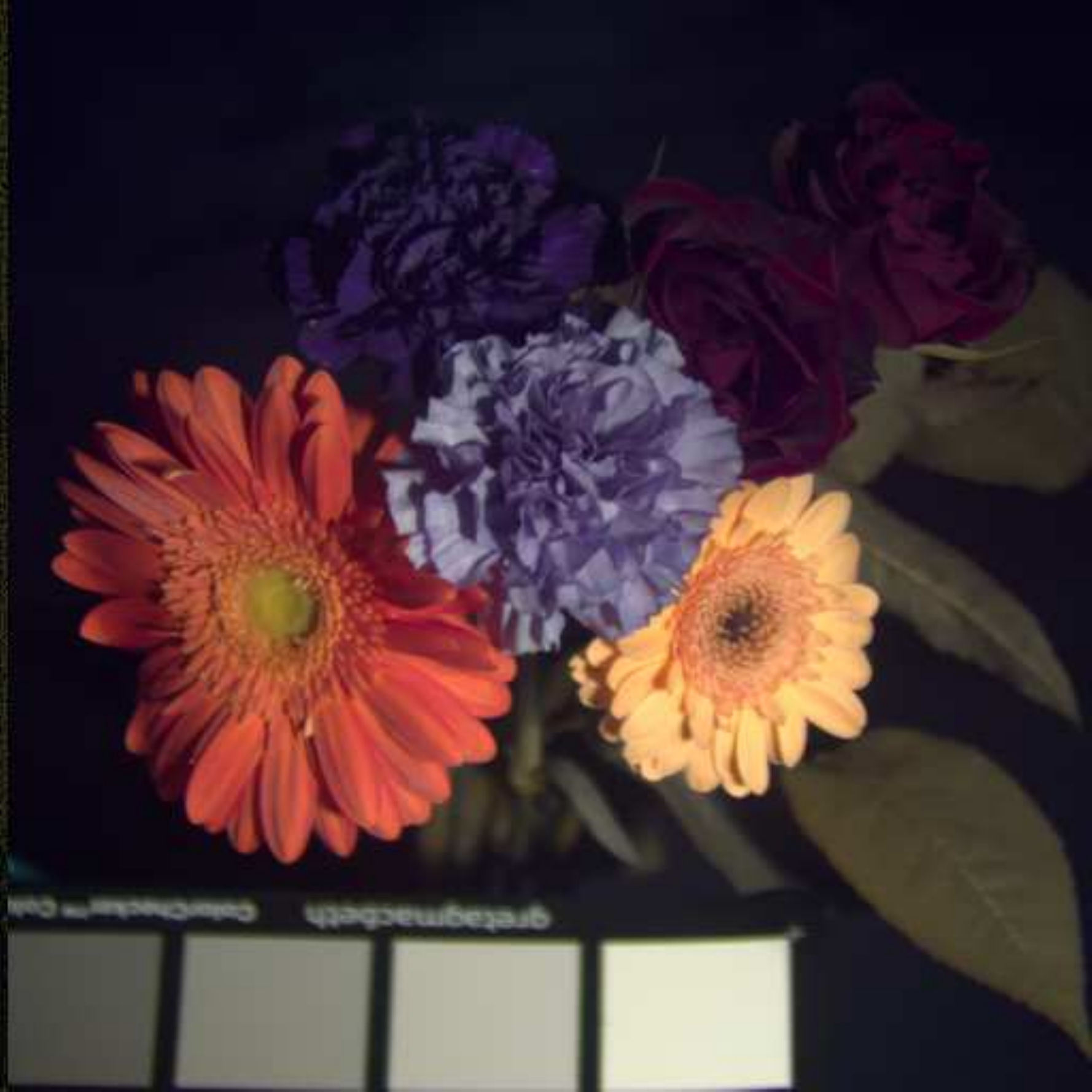}}
\subfigure[Ideal ($0^{\circ}$)]{\includegraphics[width=0.155\linewidth]{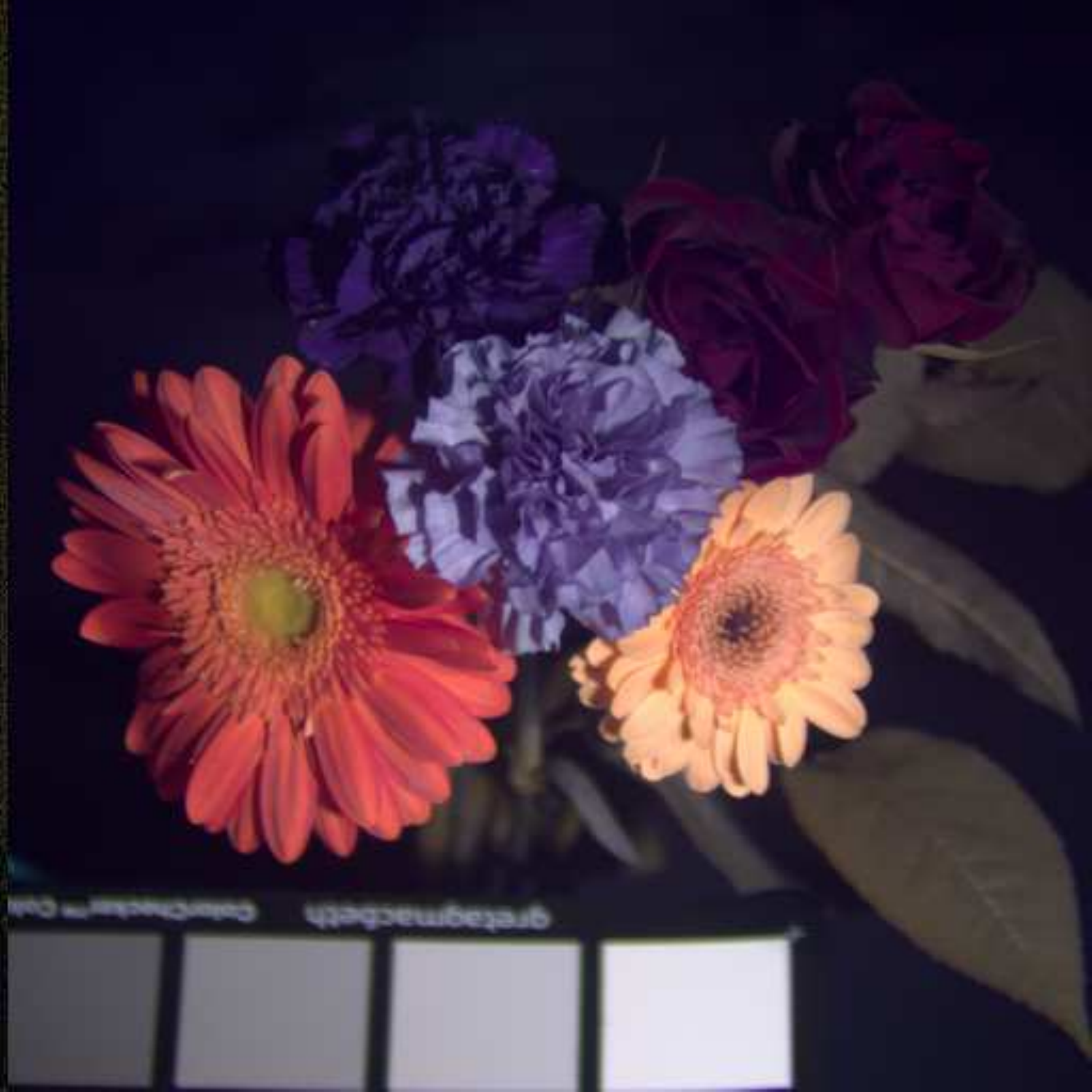}}
\subfigure[GW ($2.9^{\circ}$)]{\includegraphics[width=0.155\linewidth]{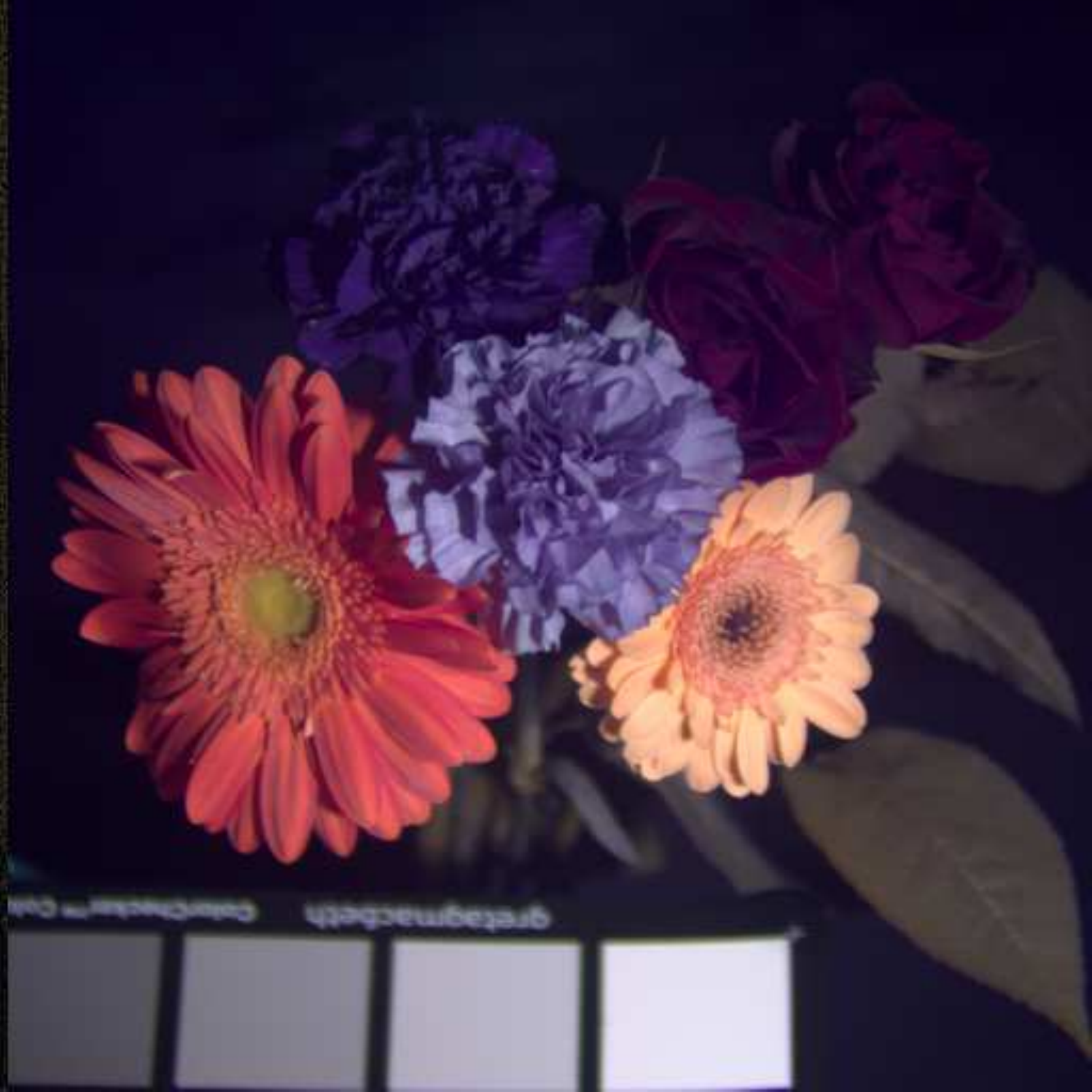}}
\subfigure[AVG ($2.9^{\circ}$)]{\includegraphics[width=0.155\linewidth]{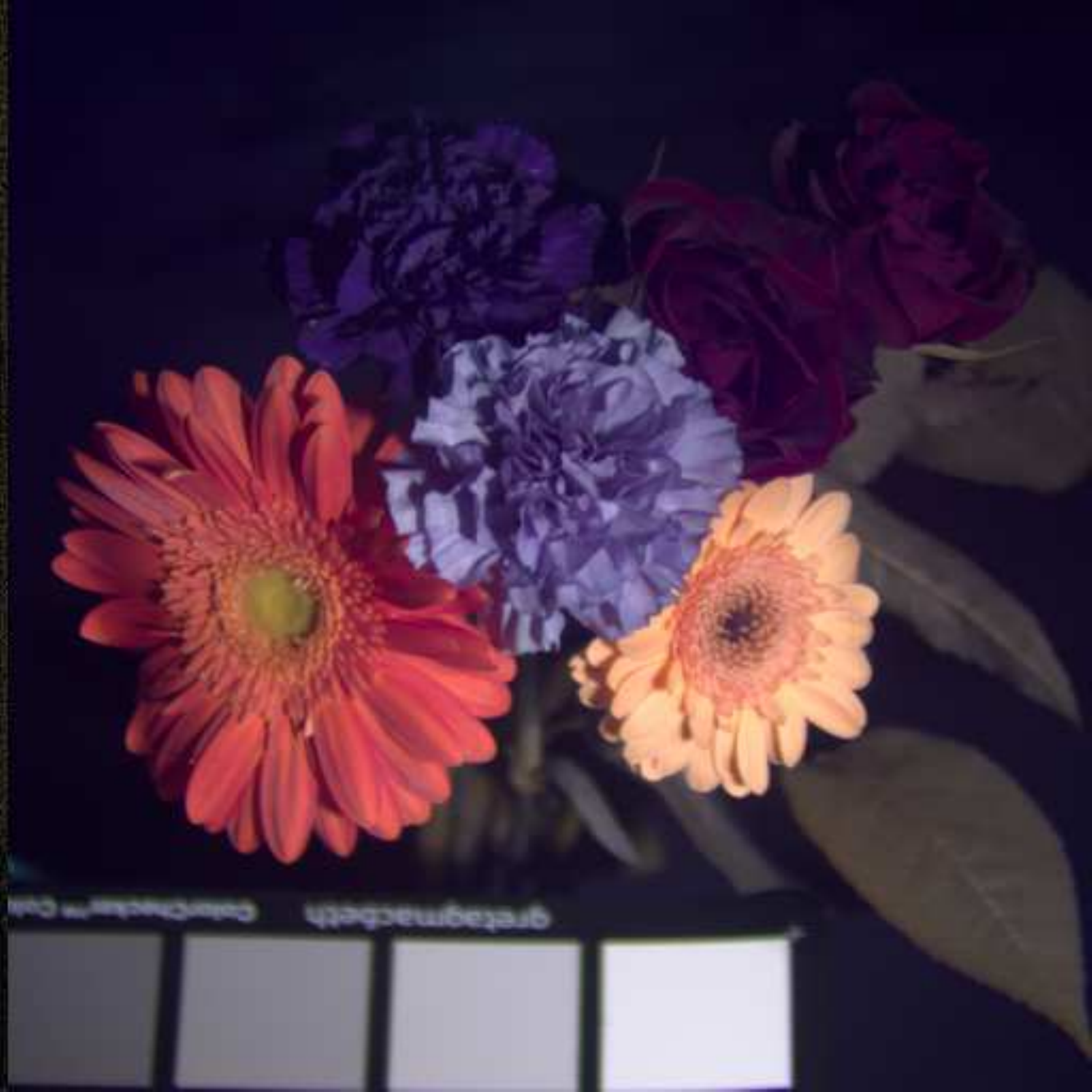}}
\subfigure[GE2 ($5.9^{\circ}$)]{\includegraphics[width=0.155\linewidth]{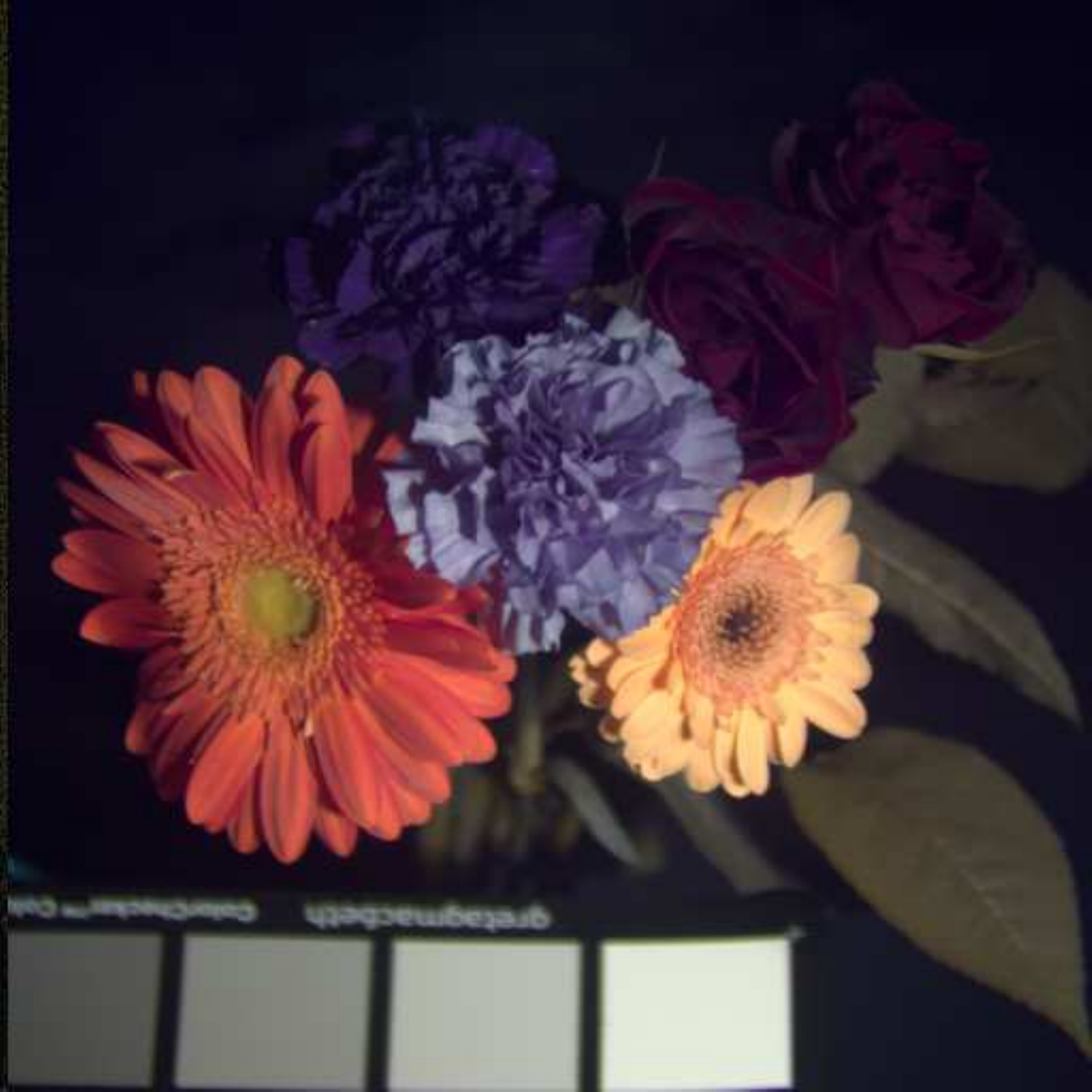}}
\subfigure[L1O ($3.6^{\circ}$)]{\includegraphics[width=0.155\linewidth]{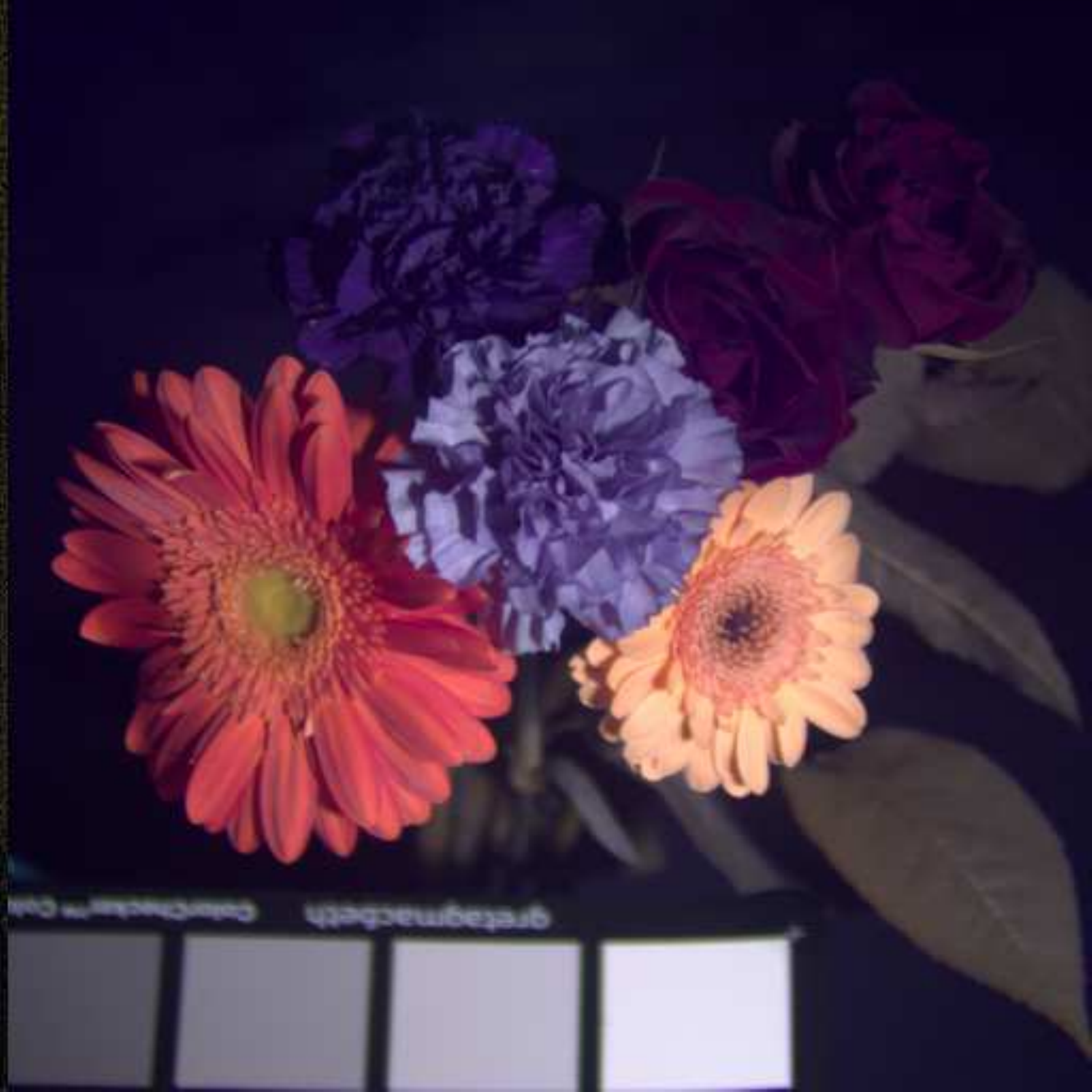}}\\
\vspace{10pt}
\subfigure[Original ($24.39^{\circ}$)]{\includegraphics[width=0.155\linewidth]{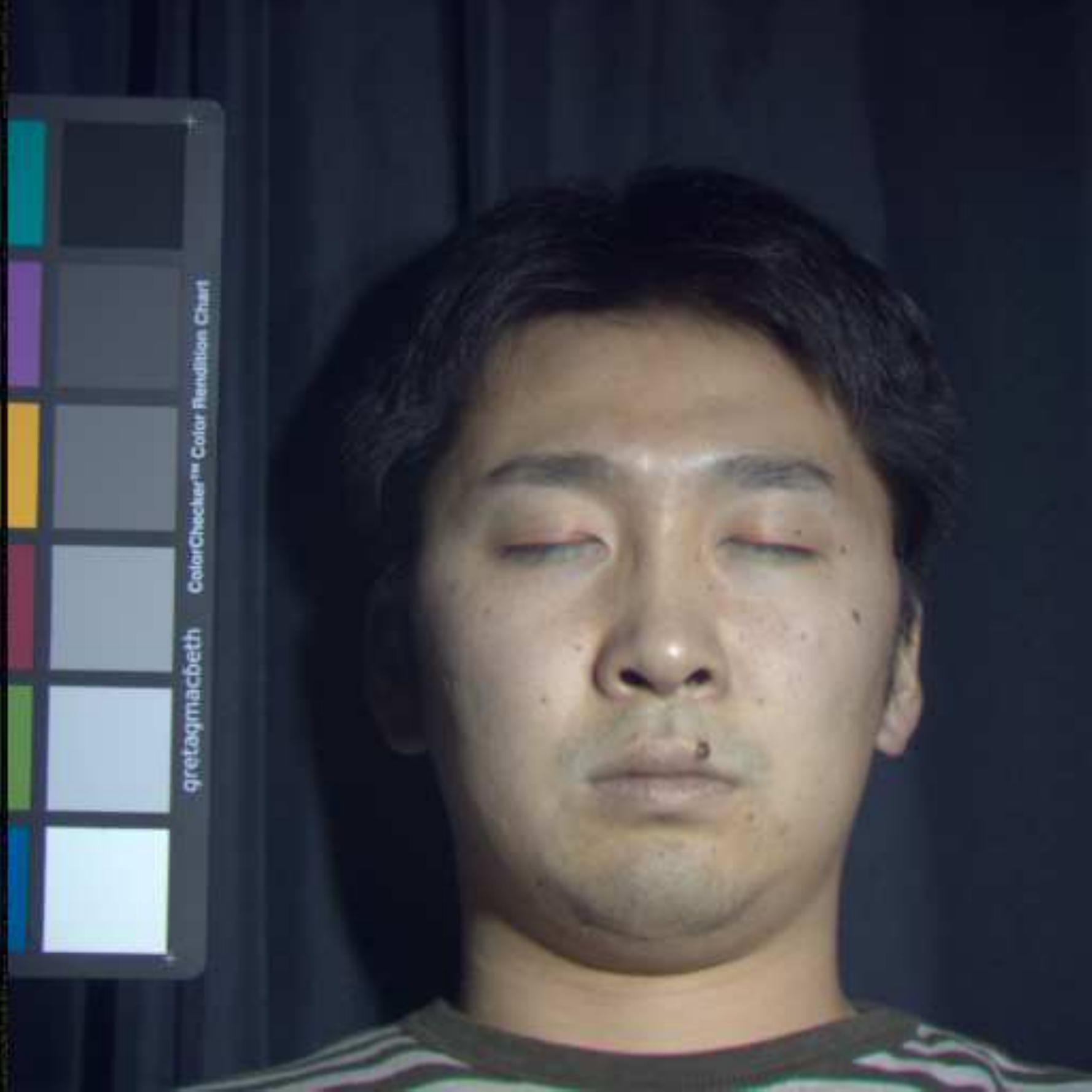}}
\subfigure[Ideal ($0^{\circ}$)]{\includegraphics[width=0.155\linewidth]{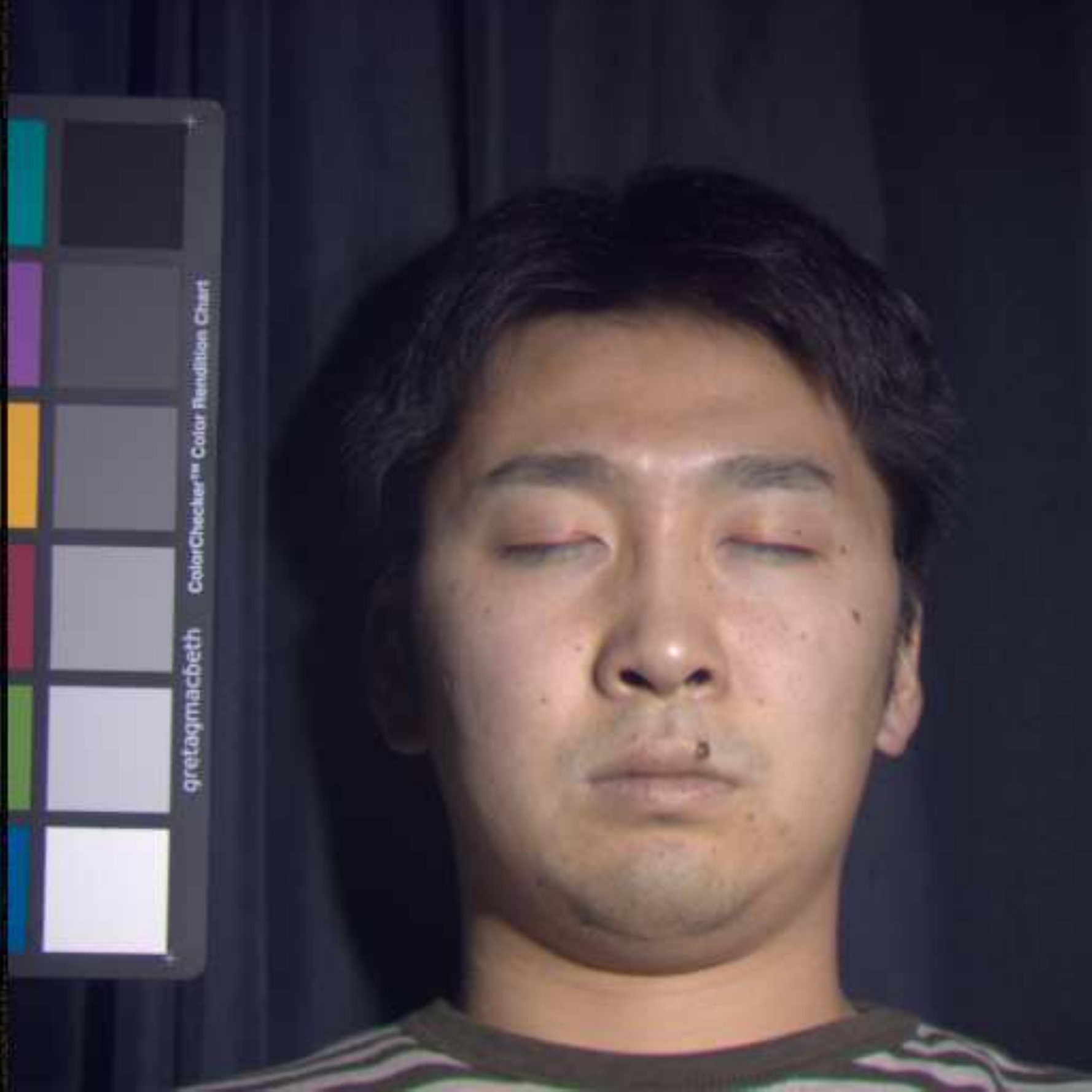}}
\subfigure[gGW ($3.02^{\circ}$)]{\includegraphics[width=0.155\linewidth]{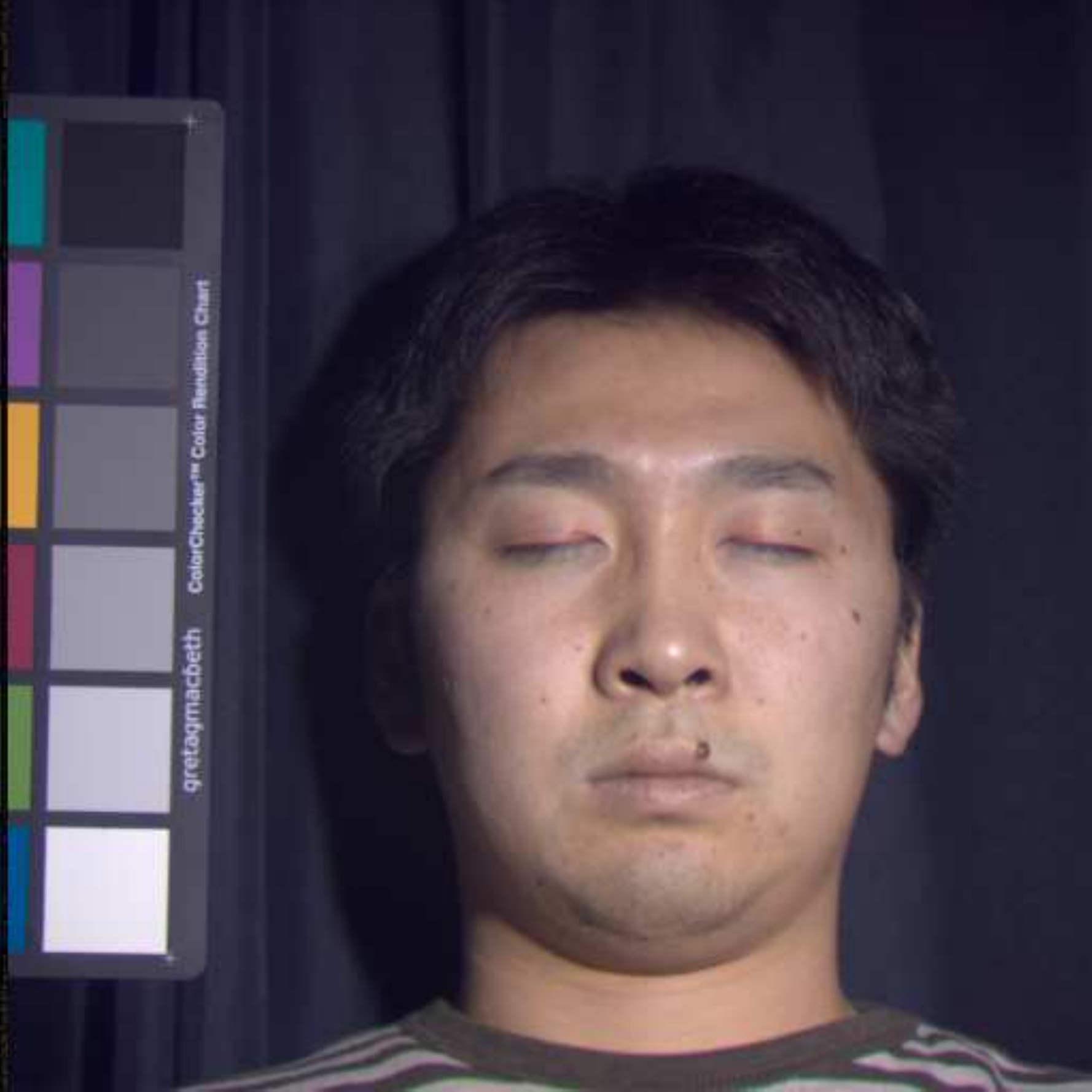}}
\subfigure[GW-gGW ($3.42^{\circ}$)]{\includegraphics[width=0.155\linewidth]{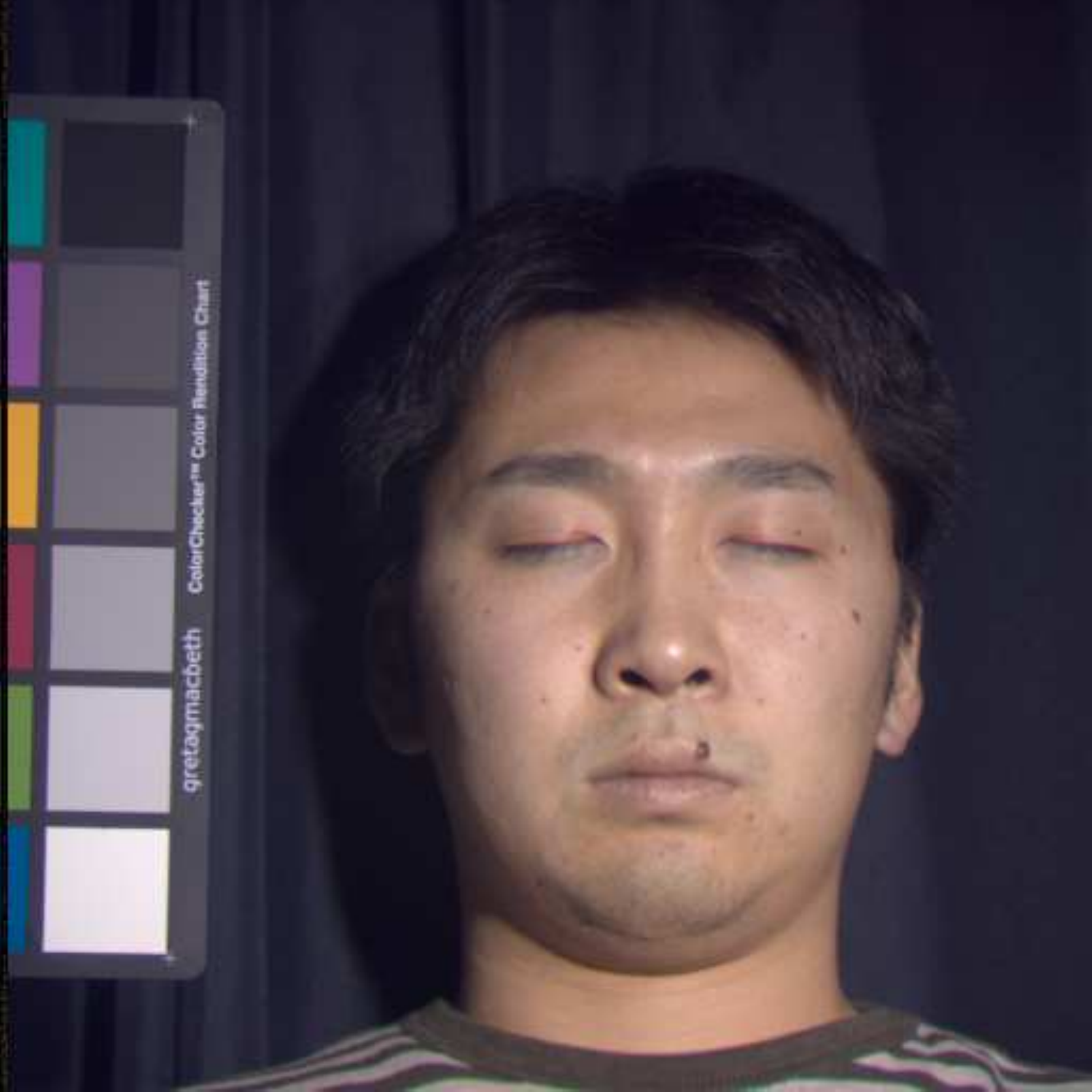}}
\subfigure[WP ($11.22^{\circ}$)]{\includegraphics[width=0.155\linewidth]{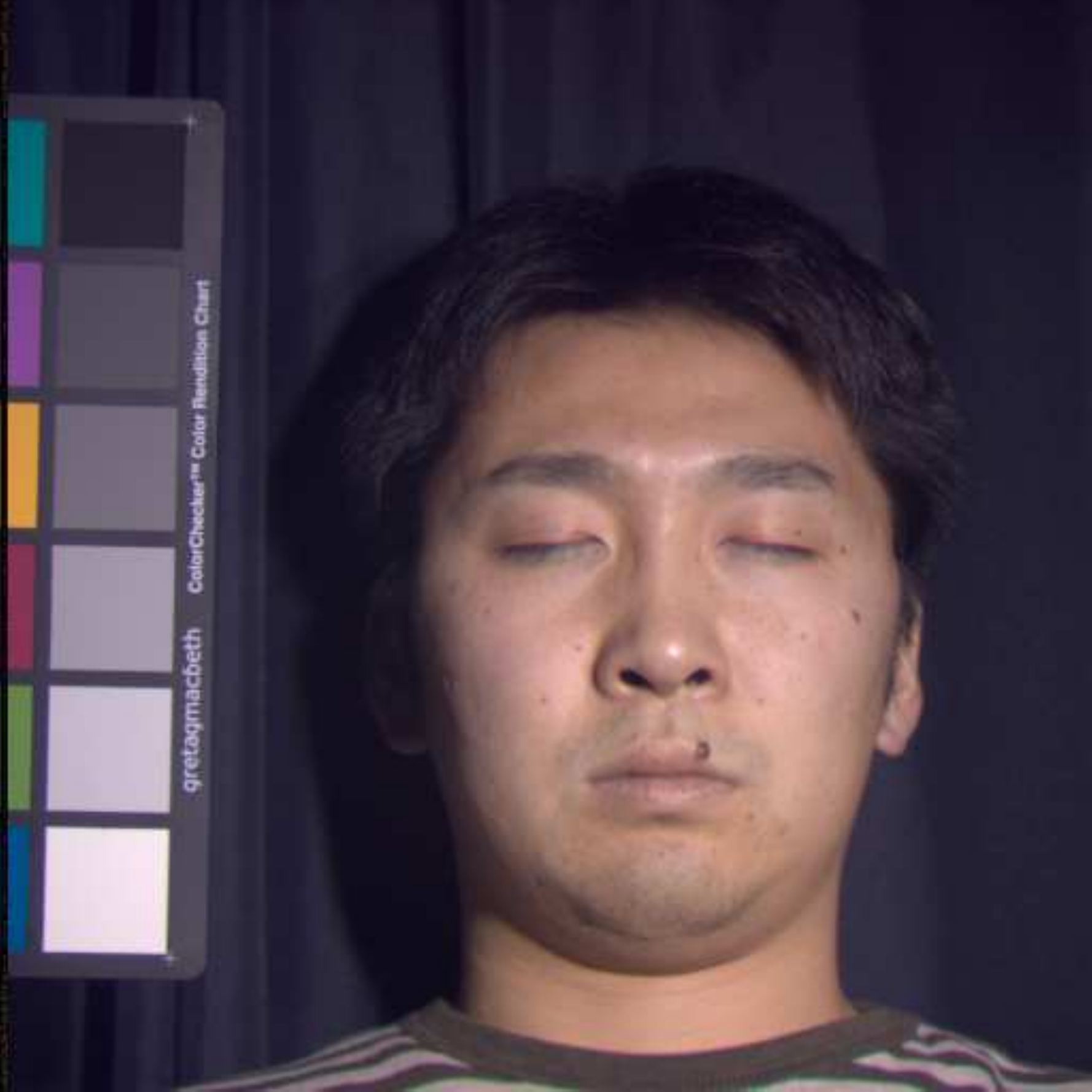}}
\subfigure[AVG ($5.27^{\circ}$)]{\includegraphics[width=0.155\linewidth]{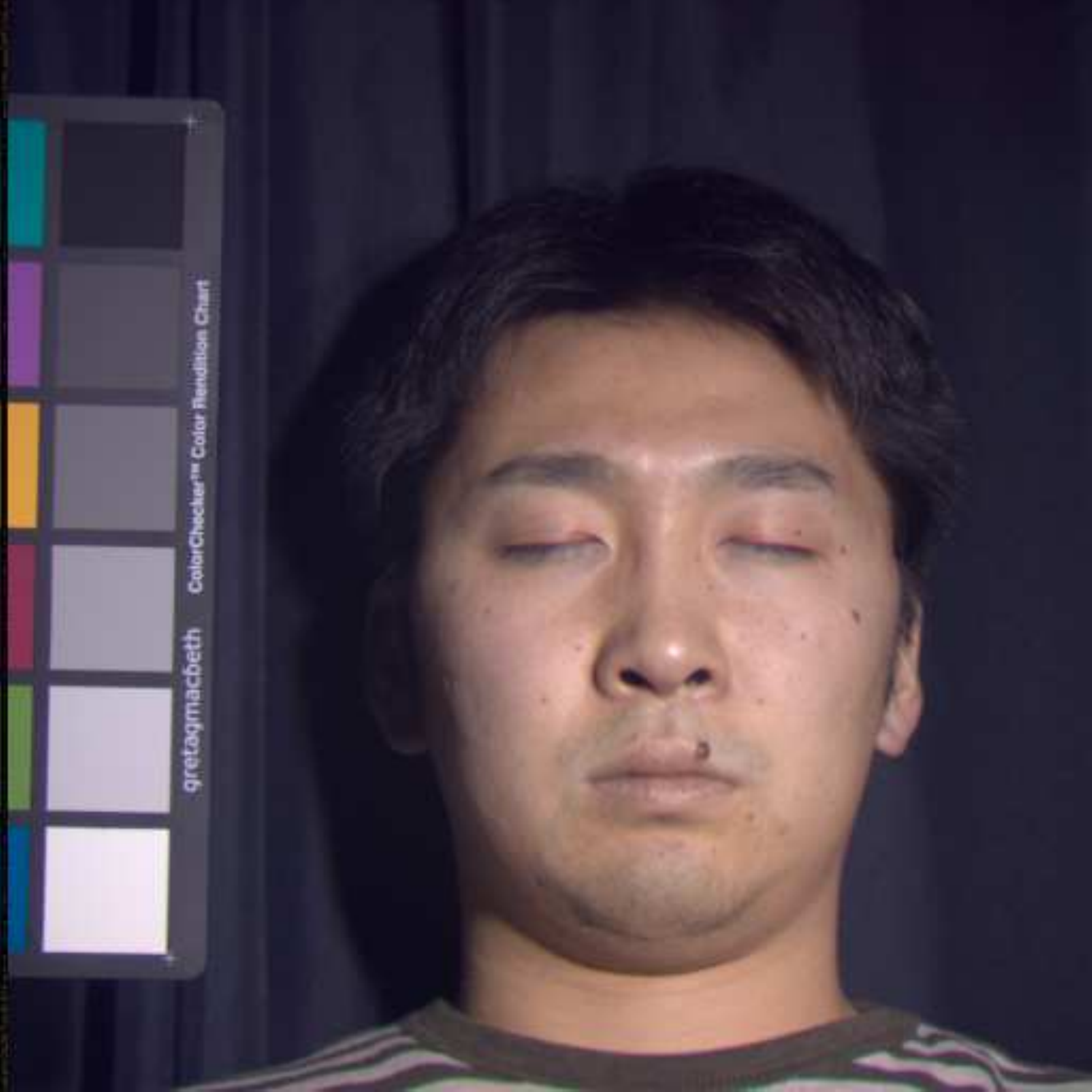}}\\
\vspace{10pt}
\subfigure[Original ($13.3^{\circ}$)]{\includegraphics[width=0.155\linewidth]{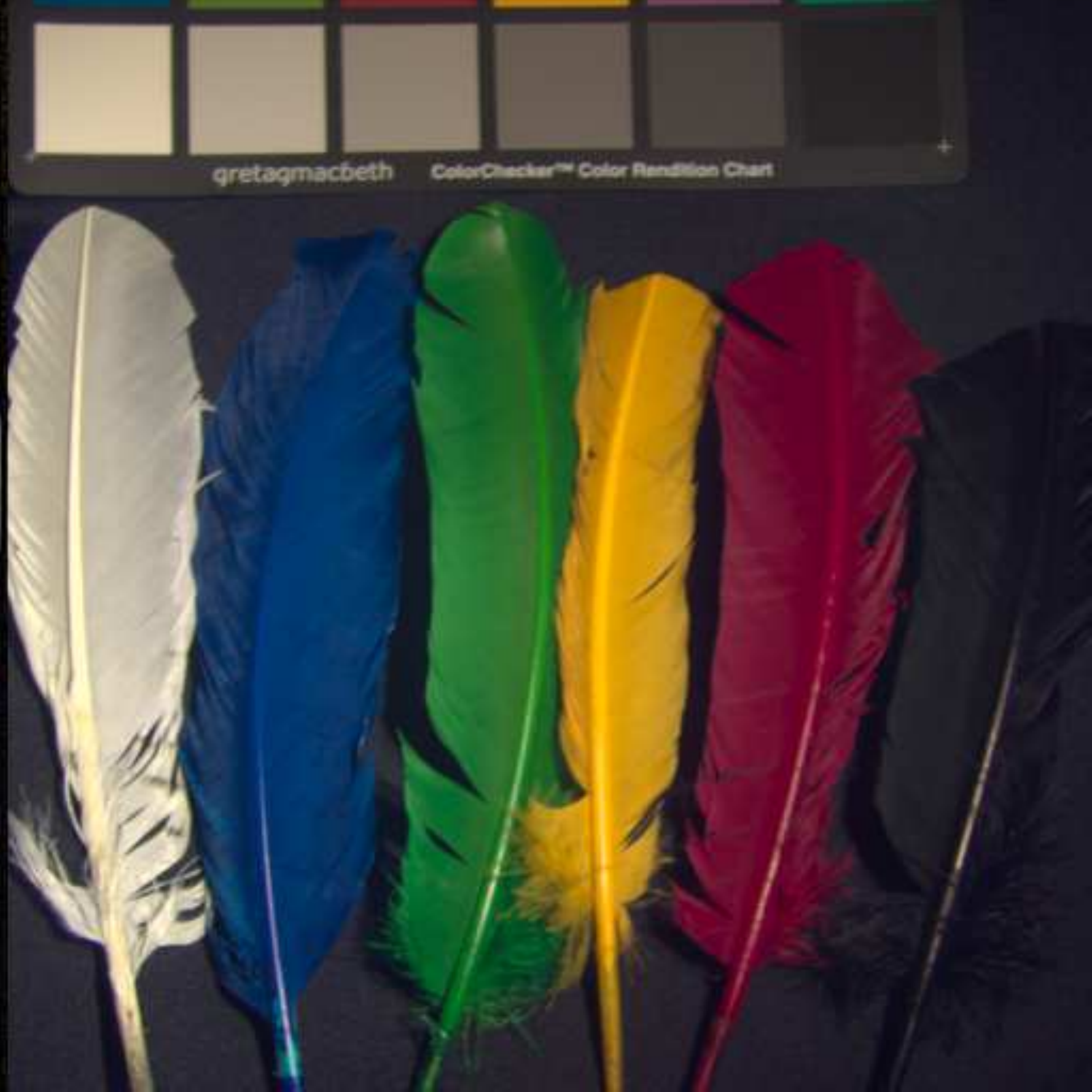}}
\subfigure[Ideal ($0^{\circ}$)]{\includegraphics[width=0.155\linewidth]{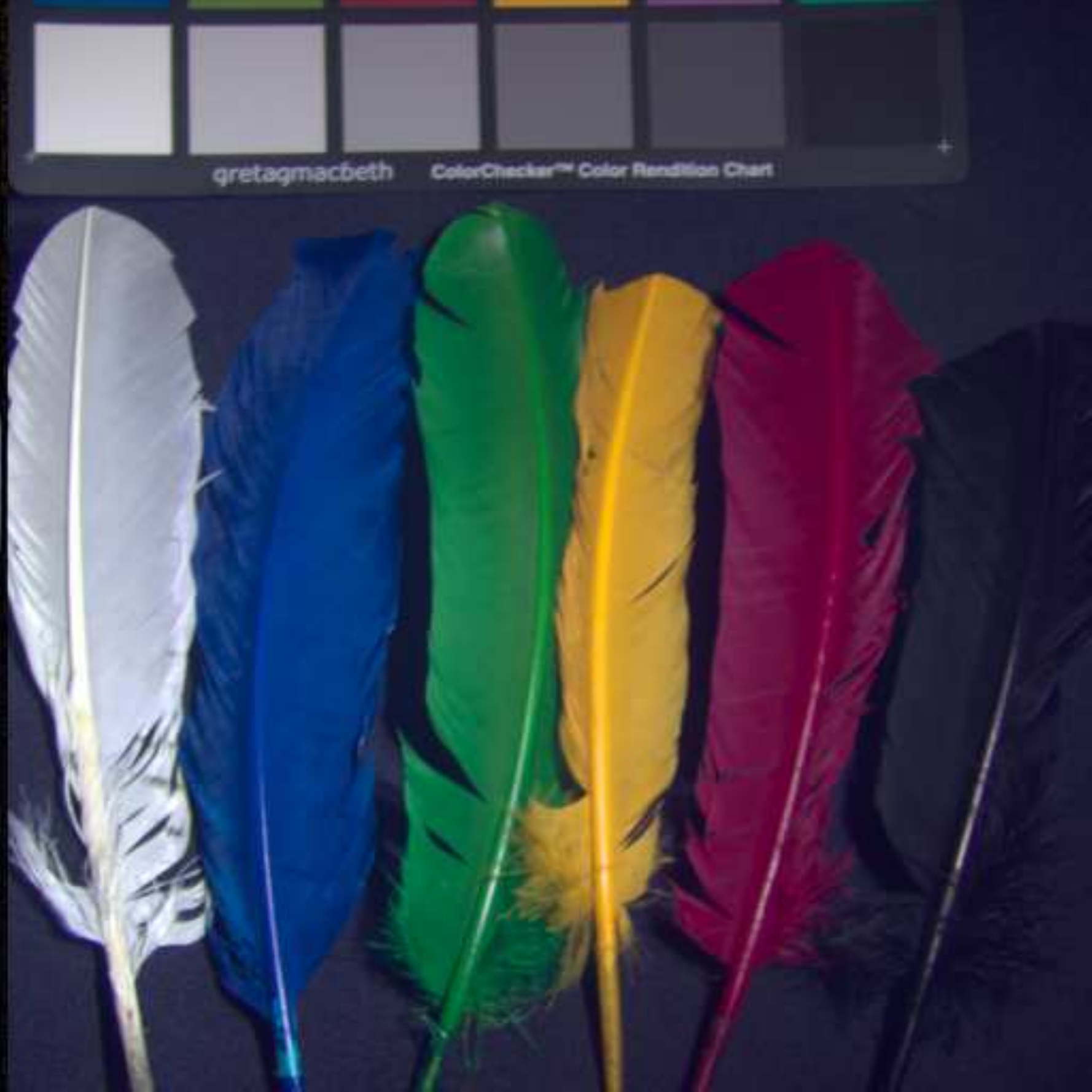}}
\subfigure[gGW ($2.2^{\circ}$)]{\includegraphics[width=0.155\linewidth]{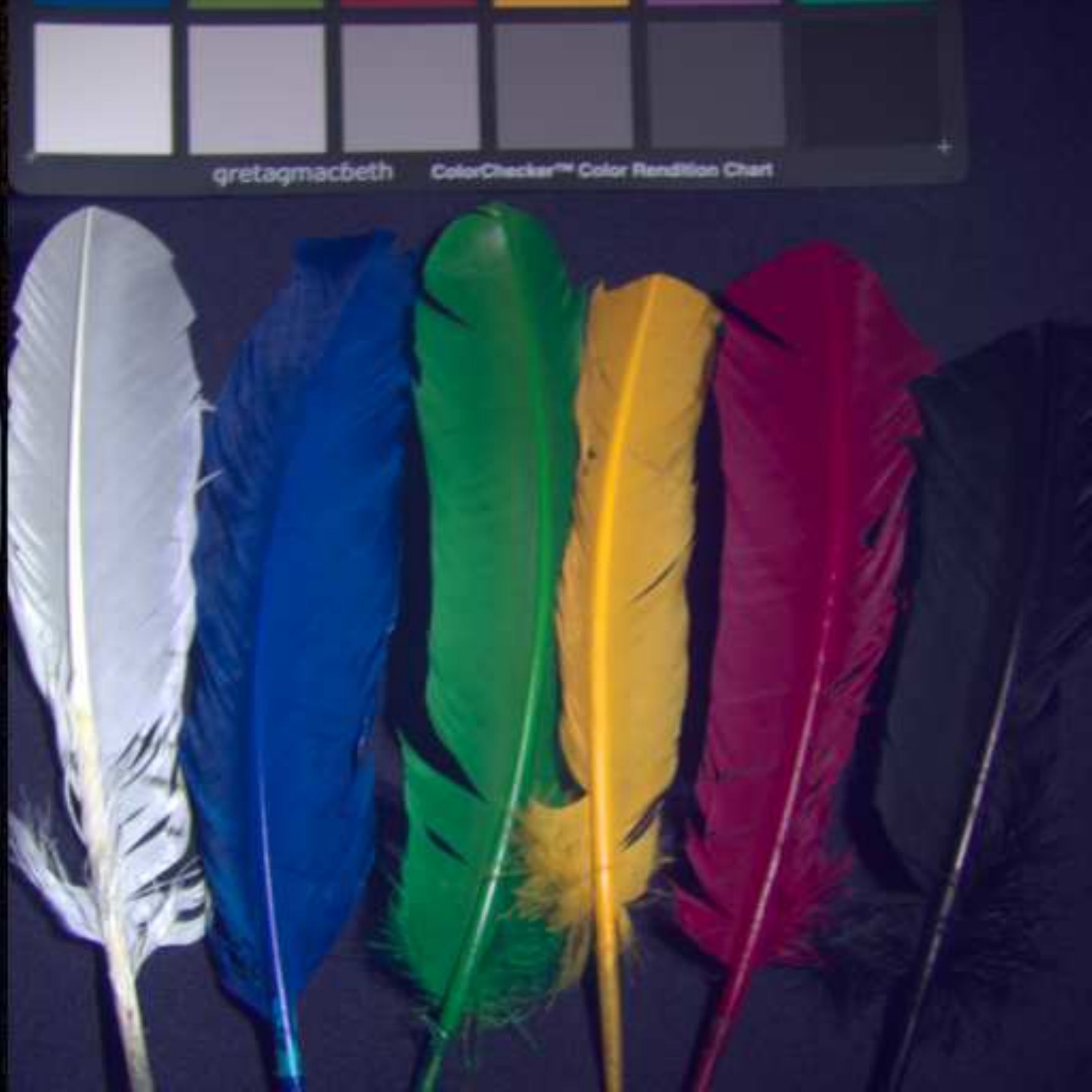}}
\subfigure[GW-gGW ($1.8^{\circ}$)]{\includegraphics[width=0.155\linewidth]{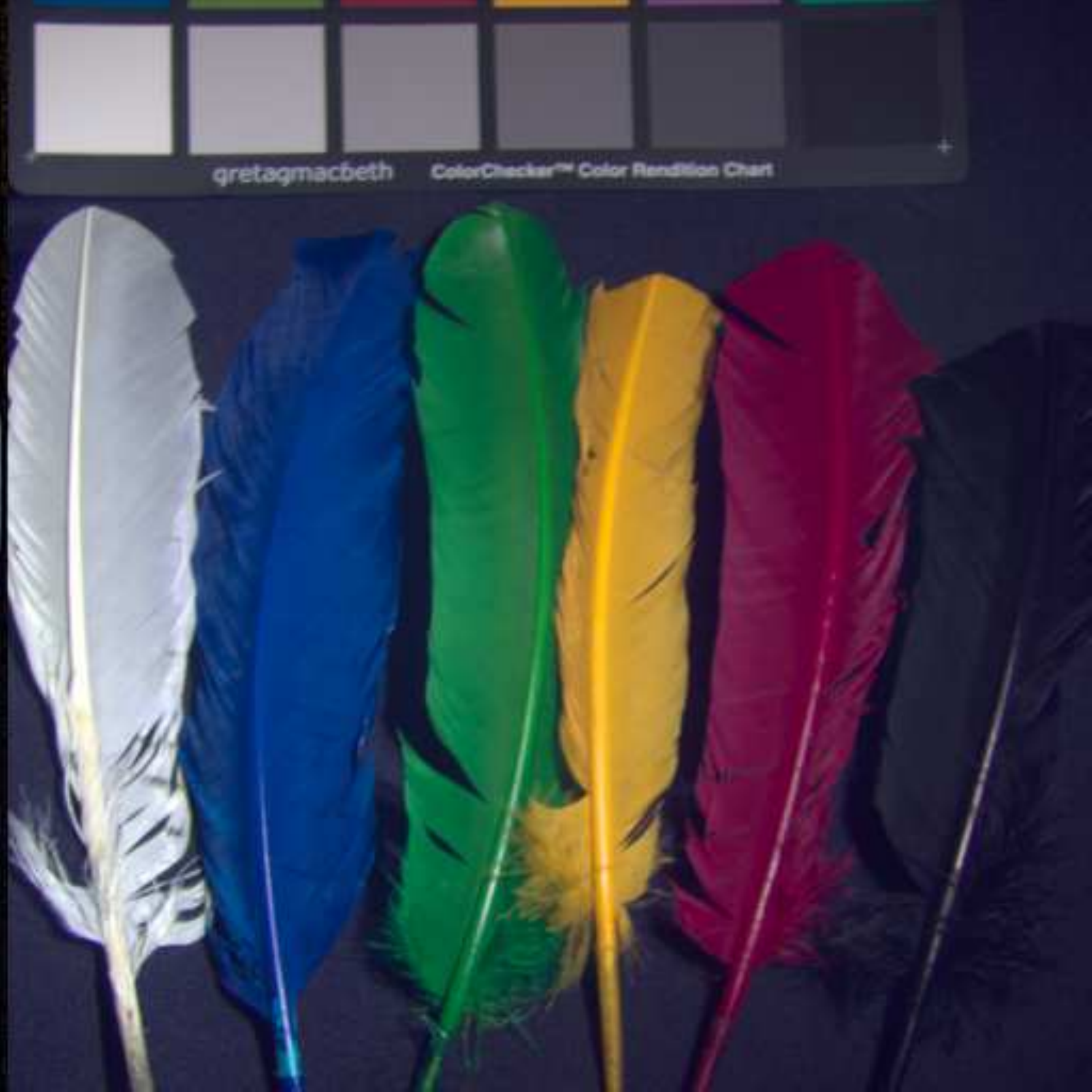}}
\subfigure[WP ($5.8^{\circ}$)]{\includegraphics[width=0.155\linewidth]{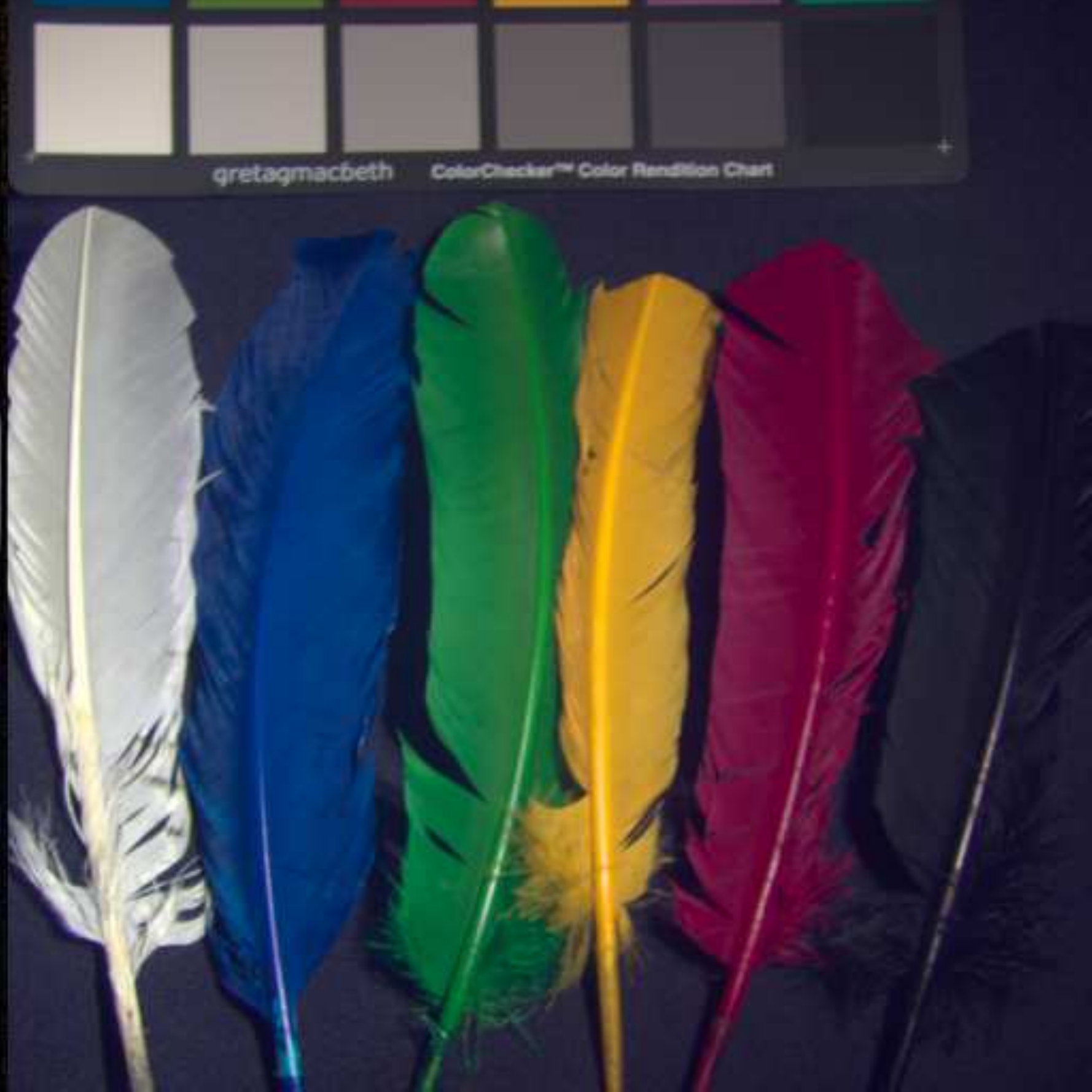}}
\subfigure[AVG ($2.7^{\circ}$)]{\includegraphics[width=0.155\linewidth]{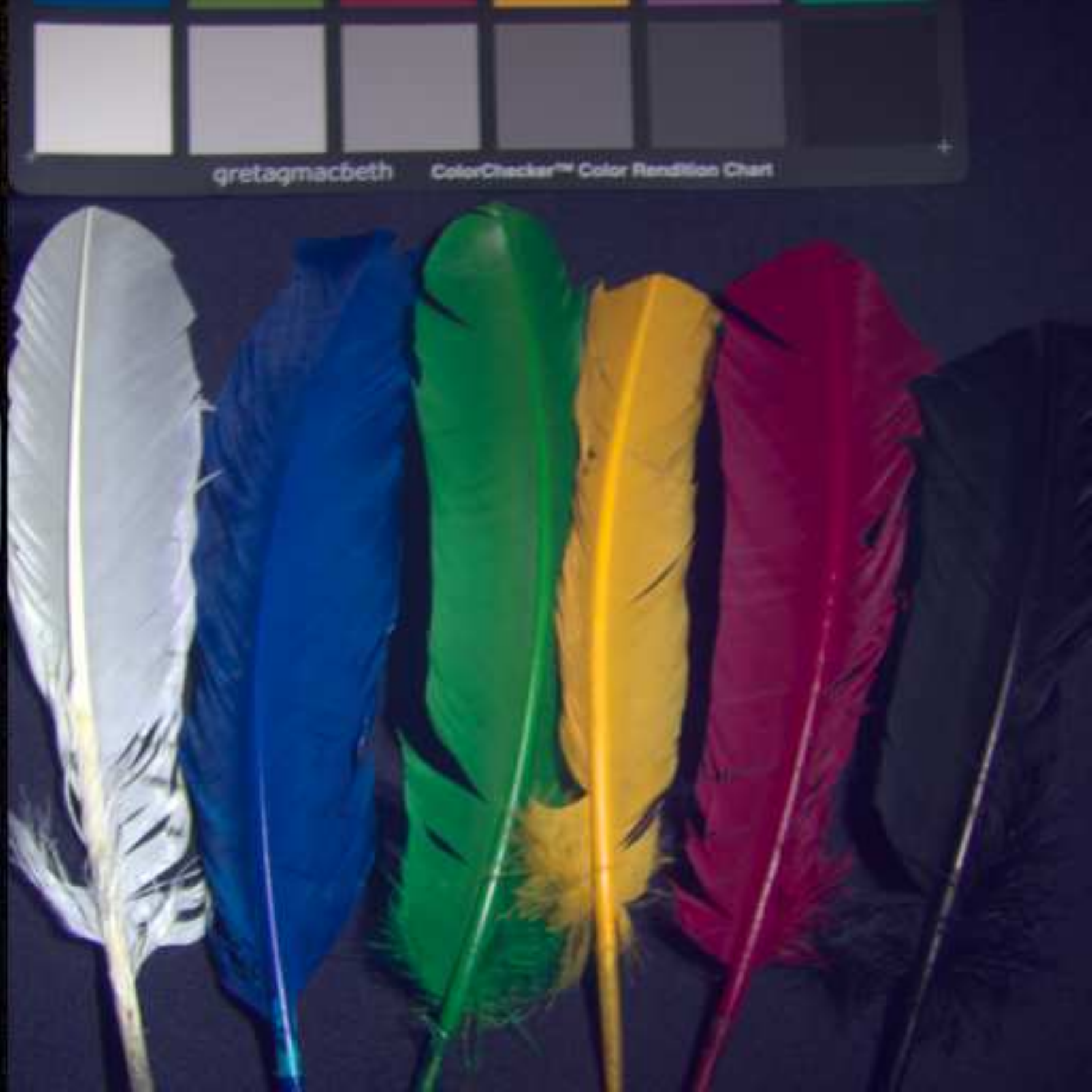}}
\caption[Qualitative comparison of individual and combinational algorithms]{(left to right) Original image with illumination bias and ideal recovery based on ground truth. Recovery with the \emph{best individual} algorithm, \emph{best combinational} algorithm, \emph{worst individual} algorithm and the \emph{worst combinational} algorithm, respectively.}
\label{fig:qual-sing-comb}
\end{figure}

\begin{figure}[h]
\centering
\includegraphics[trim = 90pt 15pt 125pt 2pt, clip, width=0.34\linewidth]{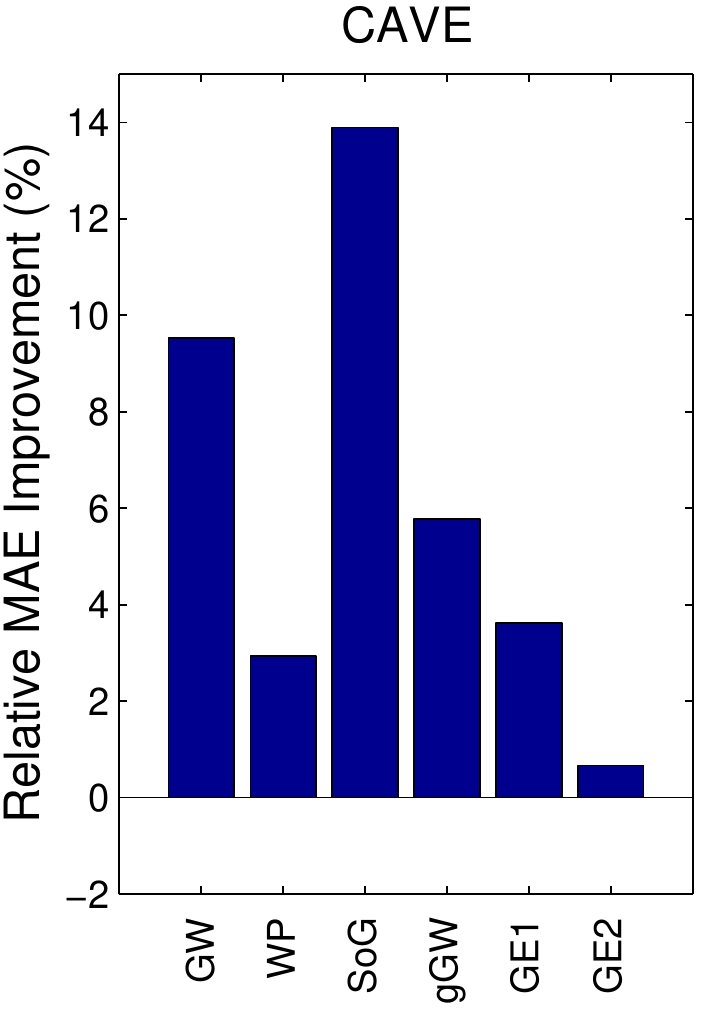}
\includegraphics[trim = 90pt 15pt 125pt 2pt, clip, width=0.34\linewidth]{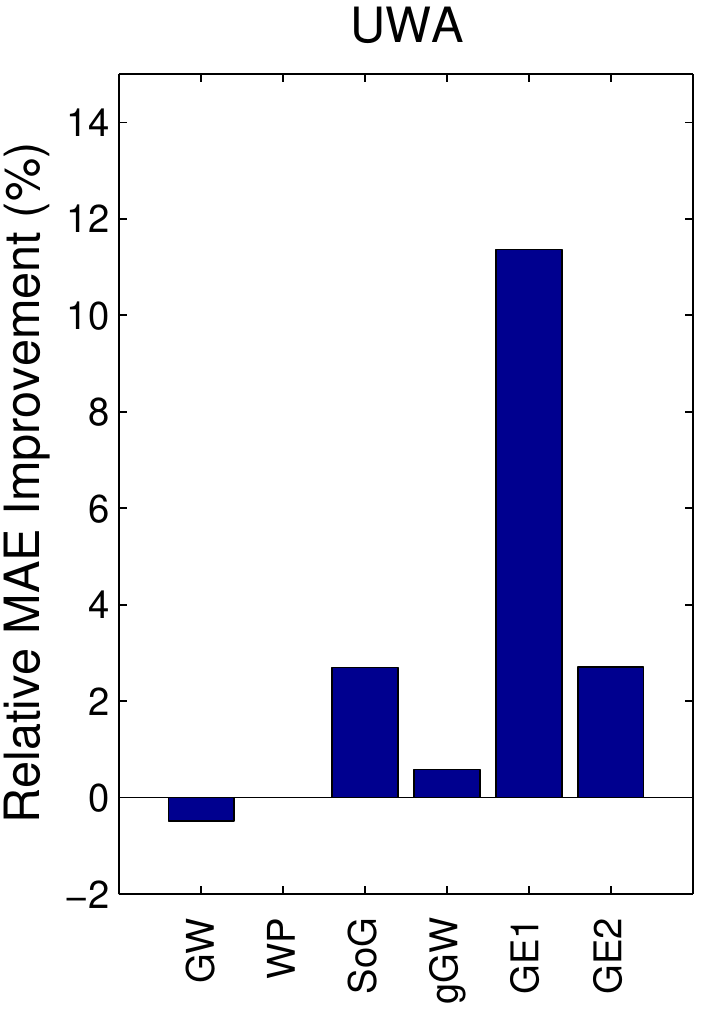}
\caption[Angular errors of non-adaptive and adaptive spatio-spectral support]{Relative MAE improvement between non-adaptive and adaptive spatio-spectral support on CAVE and UWA database}
\label{fig:rmaed-adapt-fixed}
\end{figure}

\subsection{Adaptive and Non-Adaptive Illuminant Estimation}
\label{sec:adapt-nonadapt}

We now analyze the effect of introducing the adaptive spatio-spectral support. We define relative MAE as the improvement in the mean angular error after introduction of adaptive spatio-spectral support
\begin{equation}
\Delta\bar{\epsilon}_{_{\textrm{rel}}}~(\%)= \frac{\bar{\epsilon}_{_{\textrm{nad}}}-\bar{\epsilon}_{_{\textrm{adp}}}}{\bar{\epsilon}_{_{\textrm{nad}}}} \times 100~.
\end{equation}
where $\bar{\epsilon}_{_{\textrm{nad}}}$ and $\bar{\epsilon}_{_{\textrm{adp}}}$ are the mean angular errors of non-adaptive and adaptive illuminant estimation respectively. A positive $\Delta\bar{\epsilon}_{_{\textrm{rel}}}$ indicates a decrease in the MAE (i.e.~improvement), by adaptive spatio-spectral support and vice versa. Figure~\ref{fig:rmaed-adapt-fixed} shows the relative MAE improvement for all algorithms. It can be observed that the algorithms show up to 13\% improvement after the introduction of adaptive spatio-spectral supports on the UWA and CAVE datasets. The superiority of adaptive spatio-spectral support is consistently demonstrated in most algorithms with different degrees of improvement. One can deduce that the color constancy assumptions of these algorithms is supportive for smooth illuminant estimation when the information from neighboring bands is integrated. In some instances, the adaptive spatio-spectral support brings no improvement. This can be attributed to the unpredictable SPD of illuminants in the real world, even though the classifier is trained on a subset of real world illuminants.


We also qualitatively analyze the color constancy results for the sample images shown in Figure~\ref{fig:qual-adapt-fixed}. The advantage of the adaptive approach is subtle but visually appreciable and numerically more prominent. A close look at the illumination SPD plots reveals that the adaptive approach results in a smoother estimate, closer to the ground truth. Even for the failing algorithm (WP) on an image, the adaptive approach still recovers better illumination estimate and results in lower angular error. Observe the high illumination bias in the real illuminant images compared to the simulated illuminant ones, whereas the illumination is almost perfectly recovered by both methods. Note that there is no difference in the angular errors in this case because the illuminant is correctly classified as spiky and both algorithms result in a similar estimate.

\begin{figure}[!h]
\renewcommand{\thesubfigure}{\relax}
\subfigure[Orig~$13.3^{\circ}$]{\includegraphics[width=0.168\linewidth]{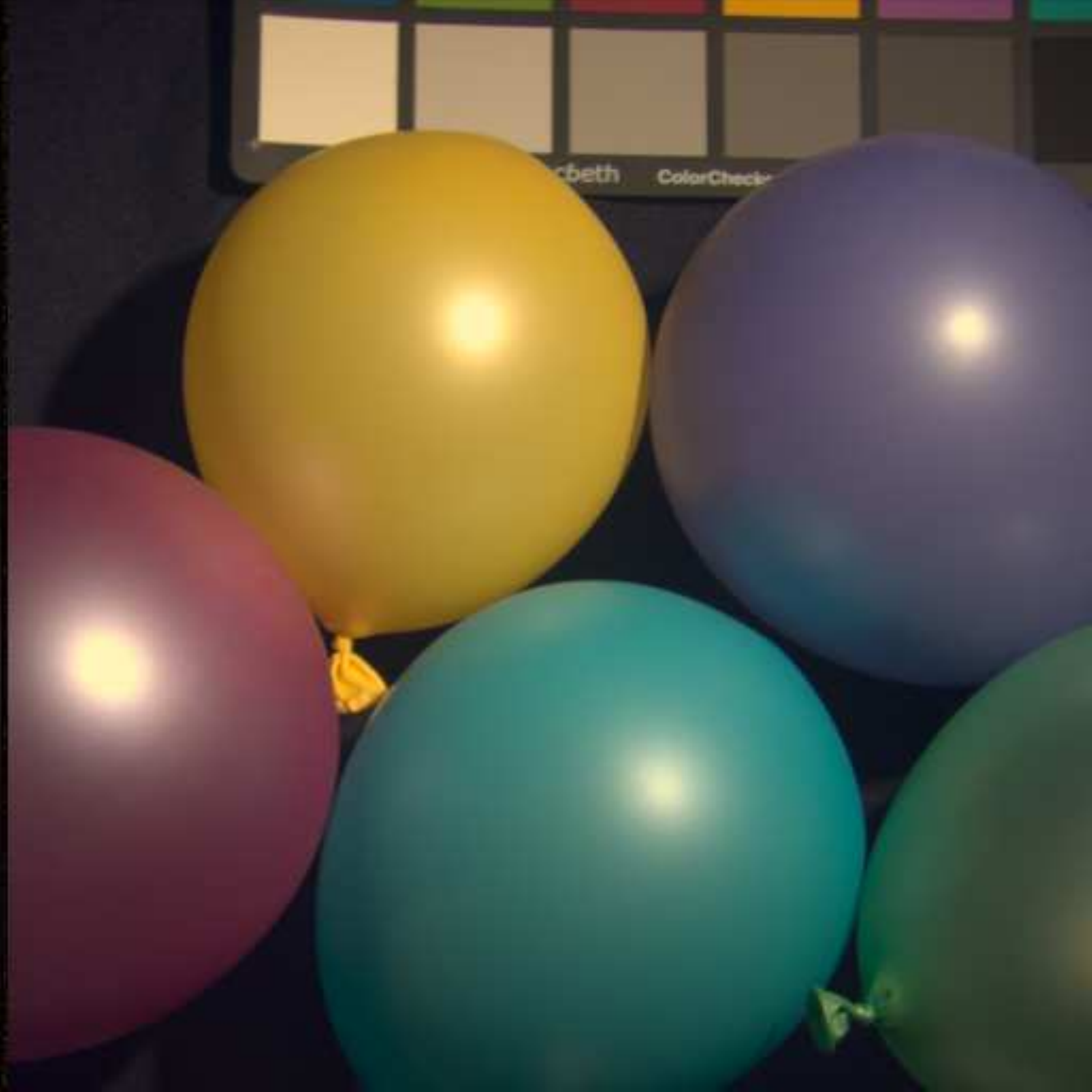}}
\subfigure[Ideal~$0^{\circ}$]{\includegraphics[width=0.168\linewidth]{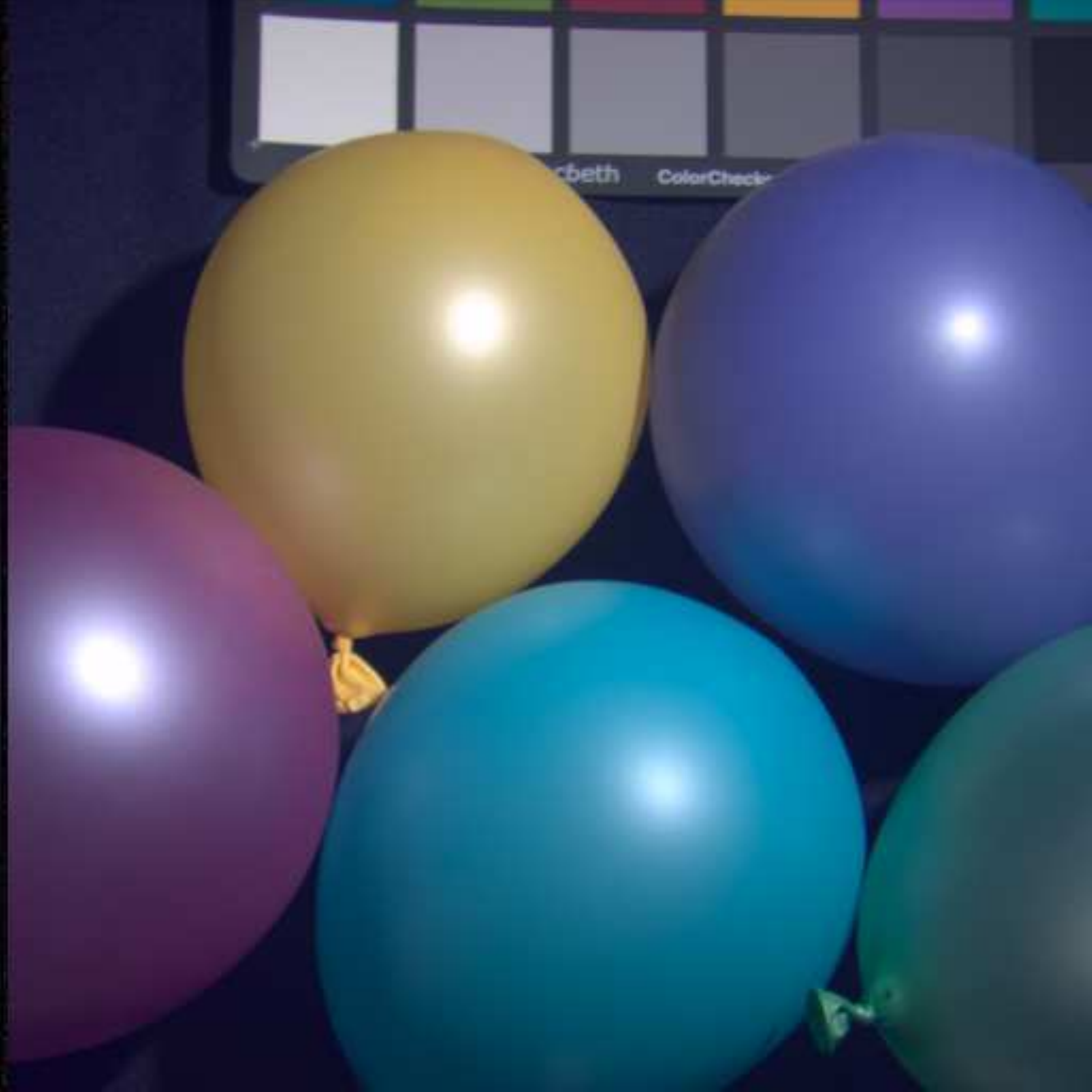}}
\subfigure[GW~$4.9^{\circ}$]{\includegraphics[width=0.168\linewidth]{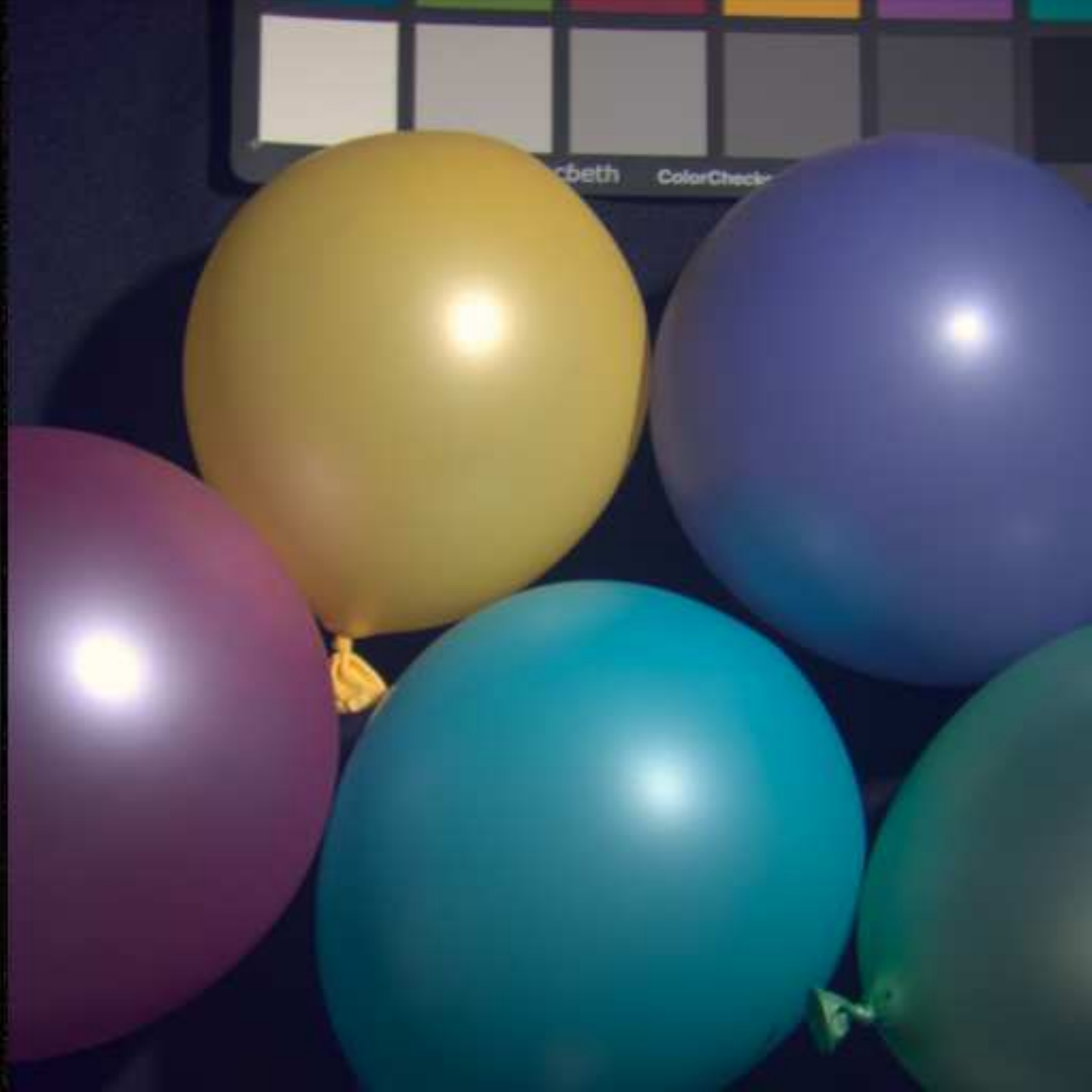}}
\subfigure[GW-a~$1.8^{\circ}$]{\includegraphics[width=0.168\linewidth]{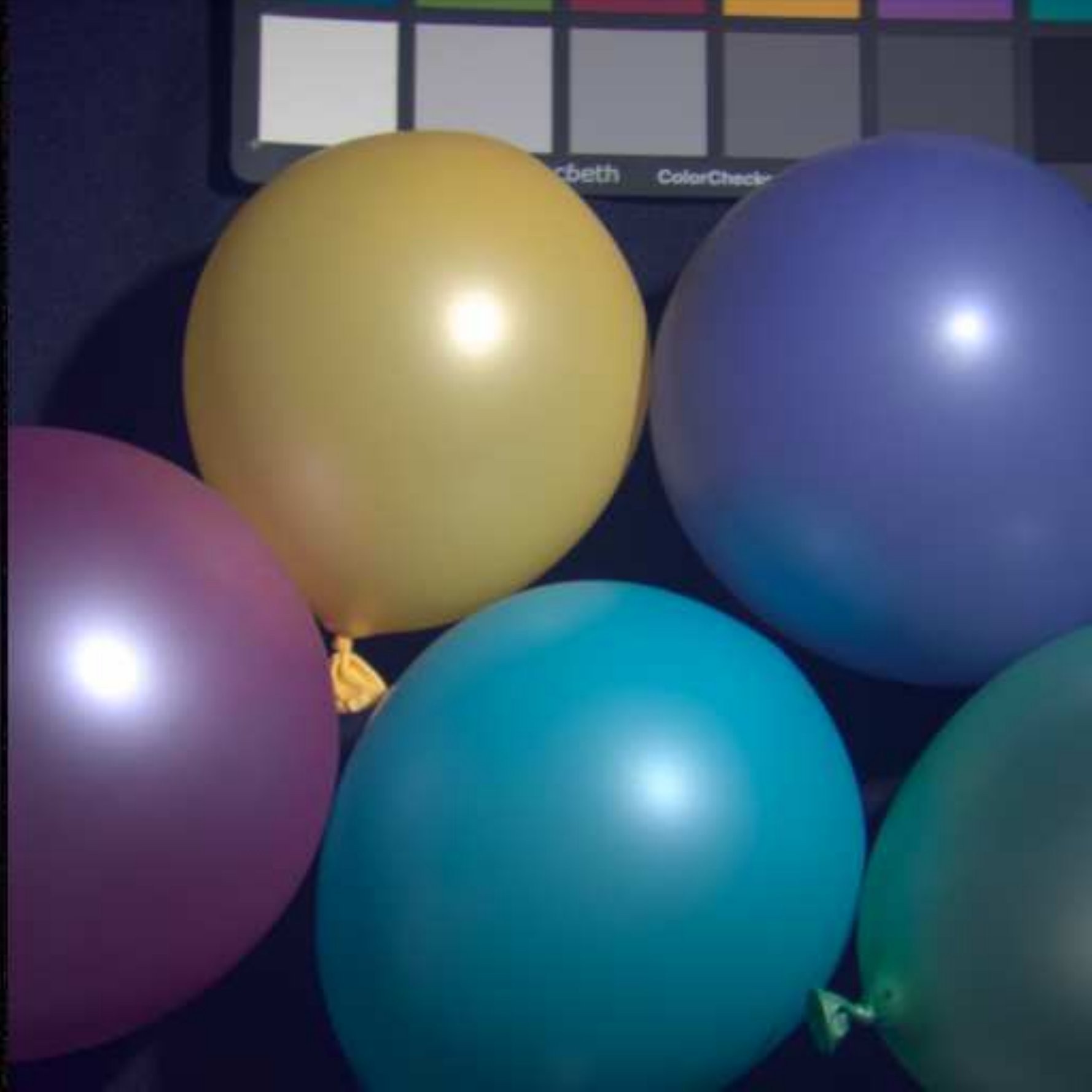}}
\subfigure[]{\includegraphics[trim = 0pt 50pt 0pt 0pt, width=0.30\linewidth]{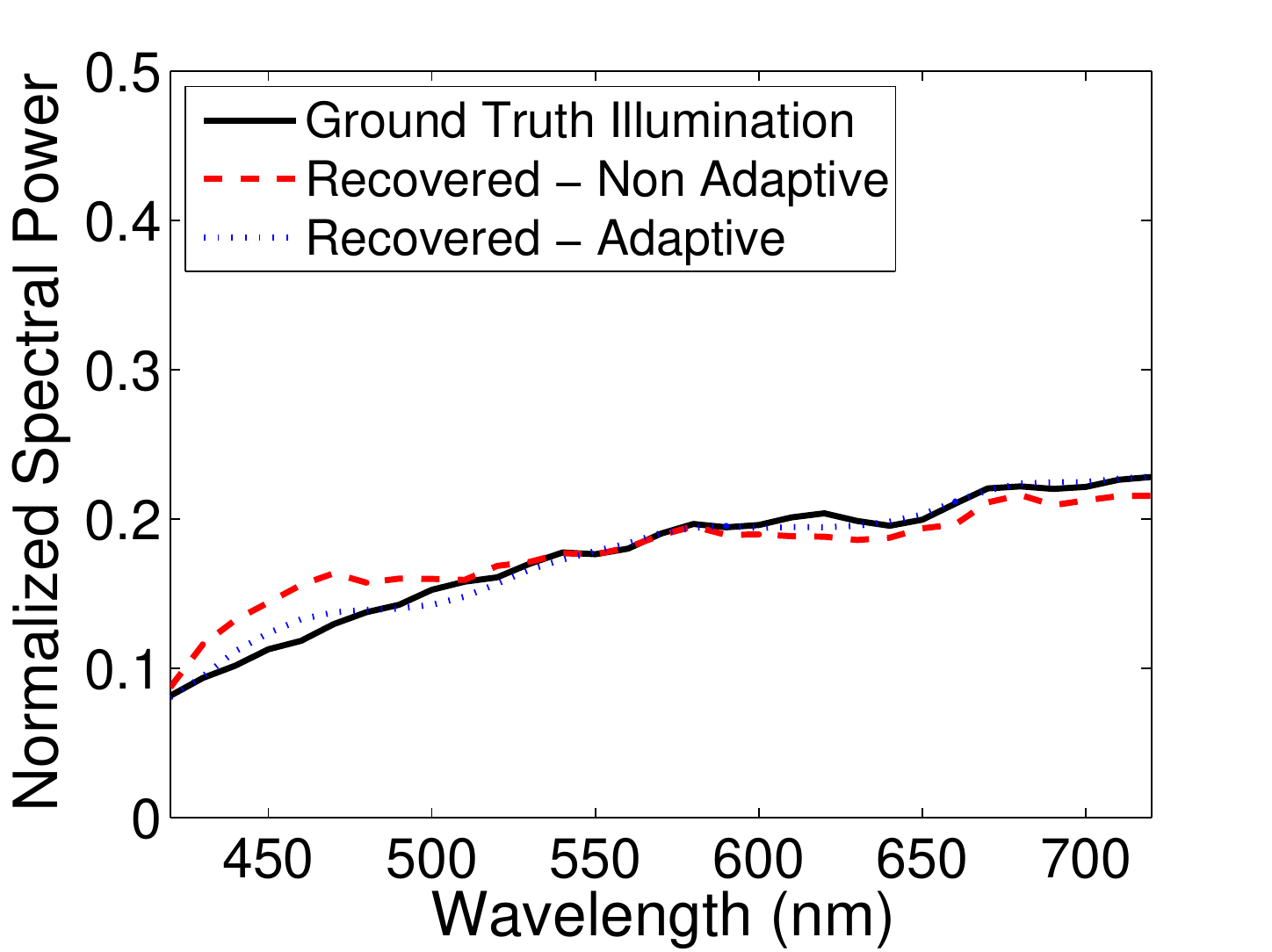}}\\ [-3pt]
\subfigure[Orig~$18.8^{\circ}$]{\includegraphics[width=0.168\linewidth]{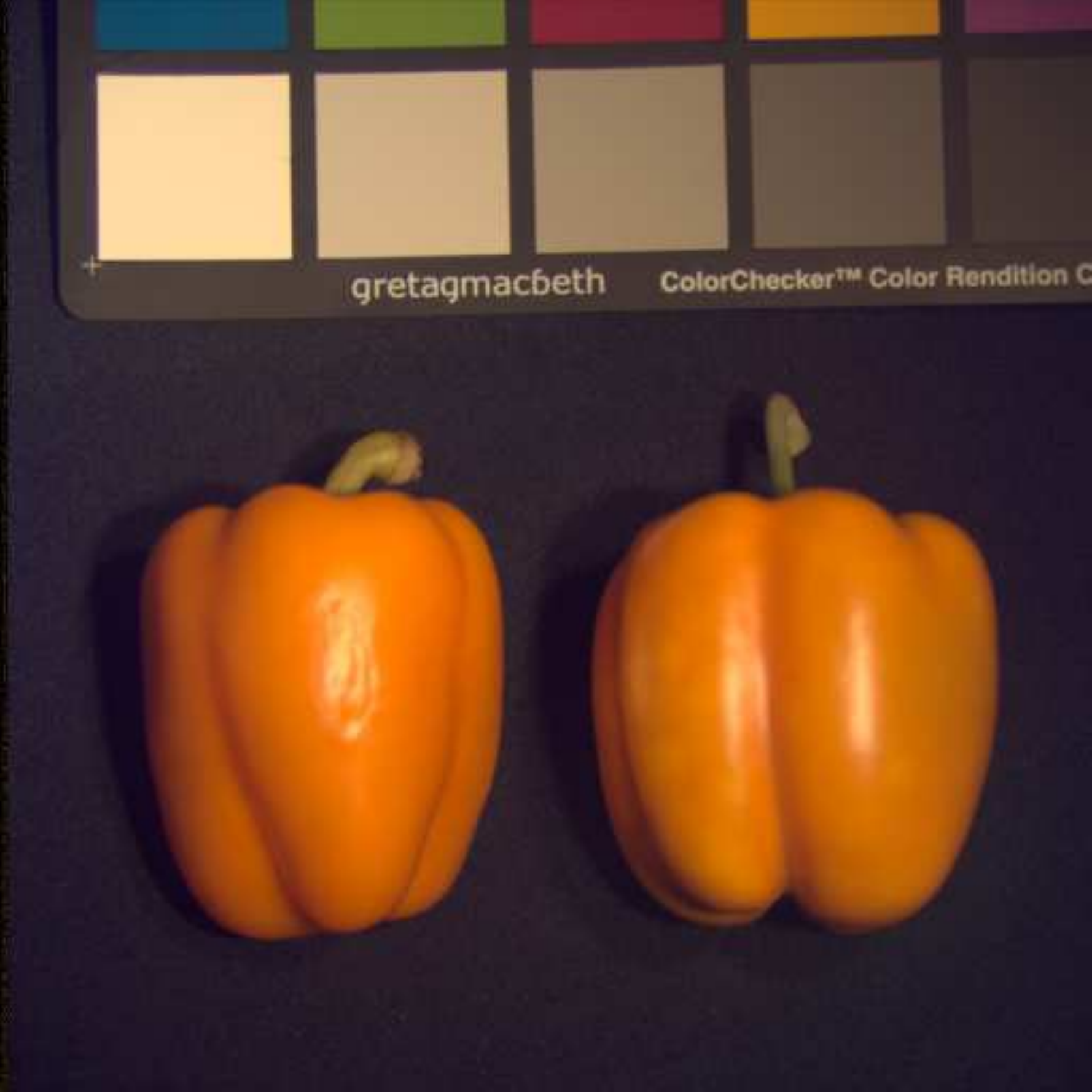}}
\subfigure[Ideal~$0^{\circ}$]{\includegraphics[width=0.168\linewidth]{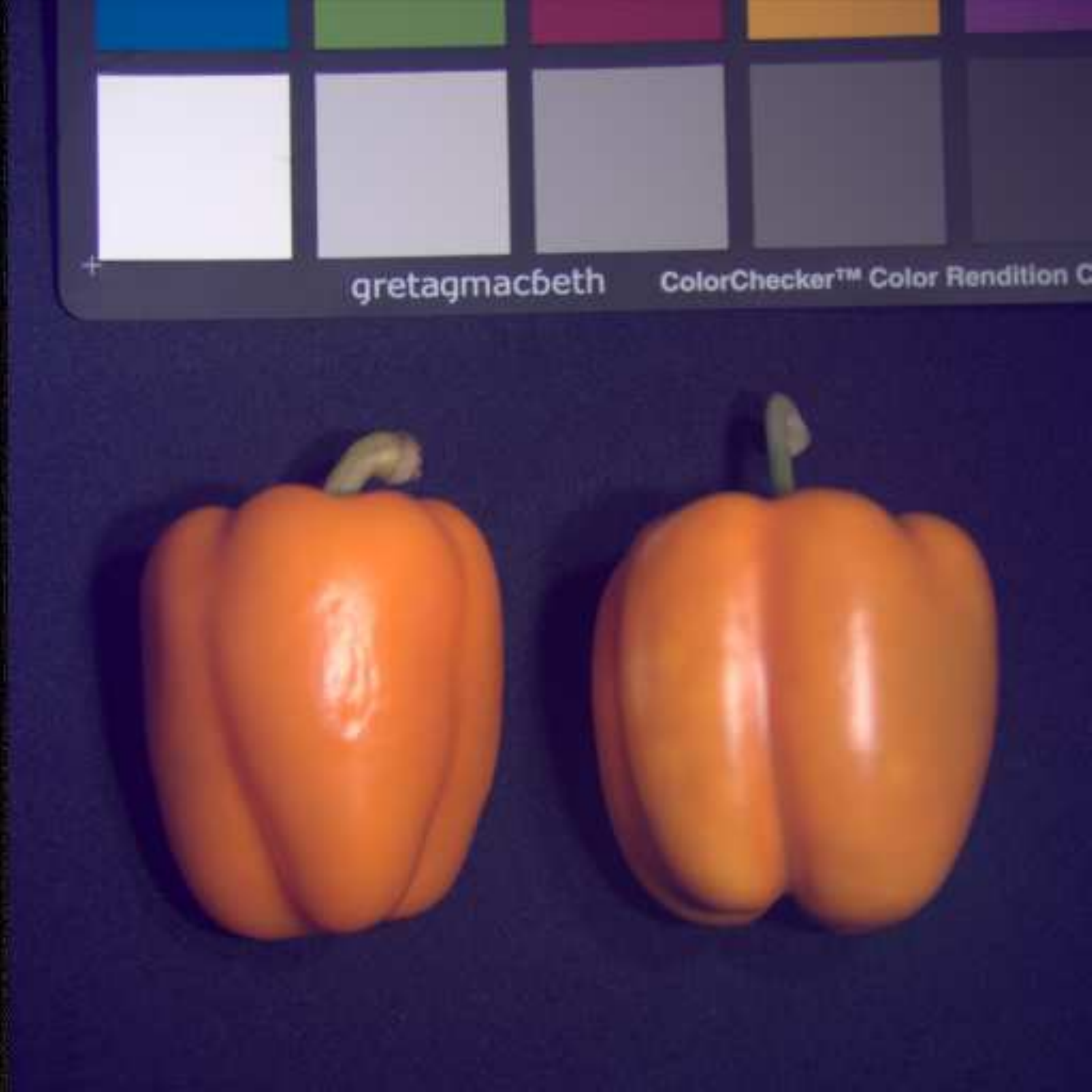}}
\subfigure[SoG~$4.3^{\circ}$]{\includegraphics[width=0.168\linewidth]{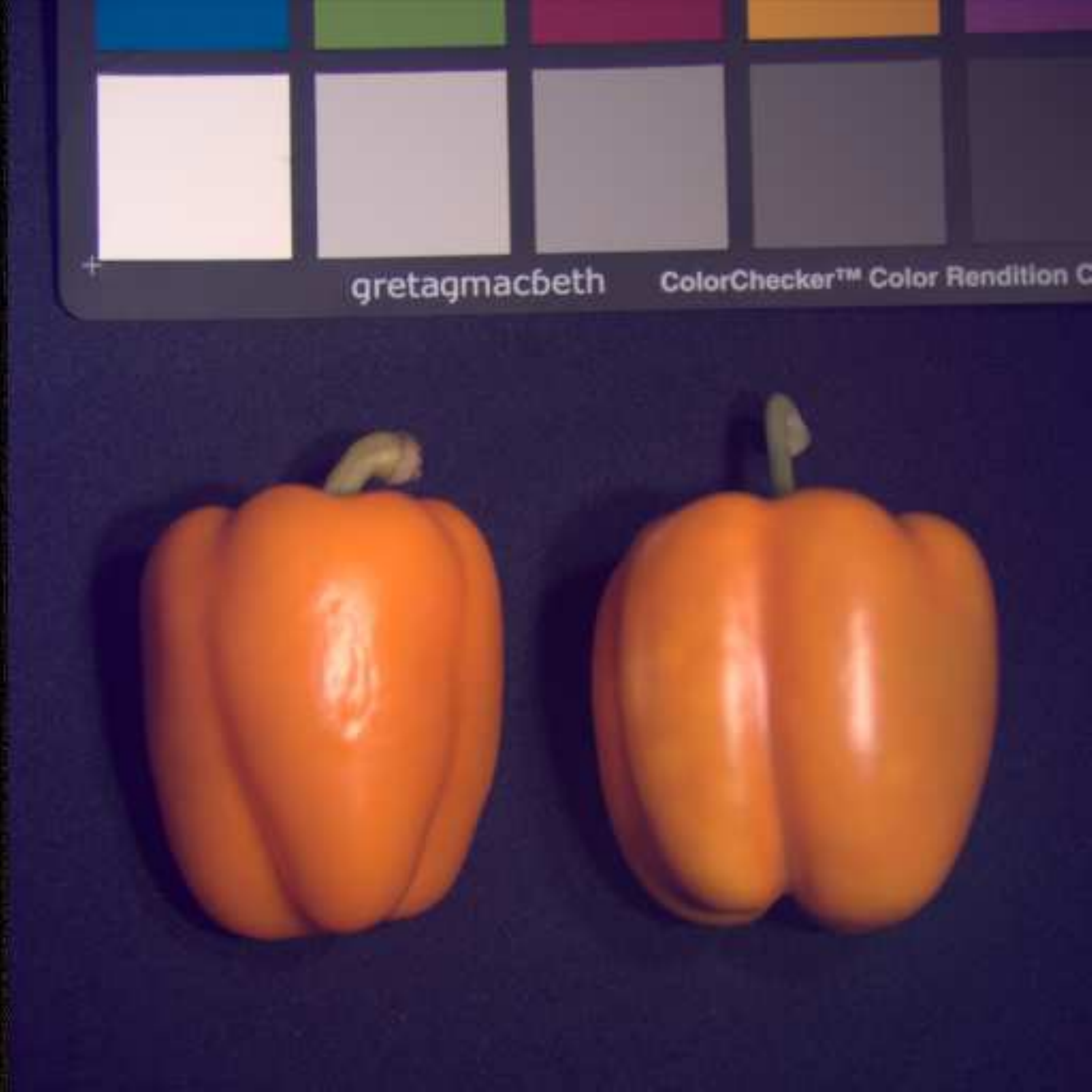}}
\subfigure[SoG-a~$1.4^{\circ}$]{\includegraphics[width=0.168\linewidth]{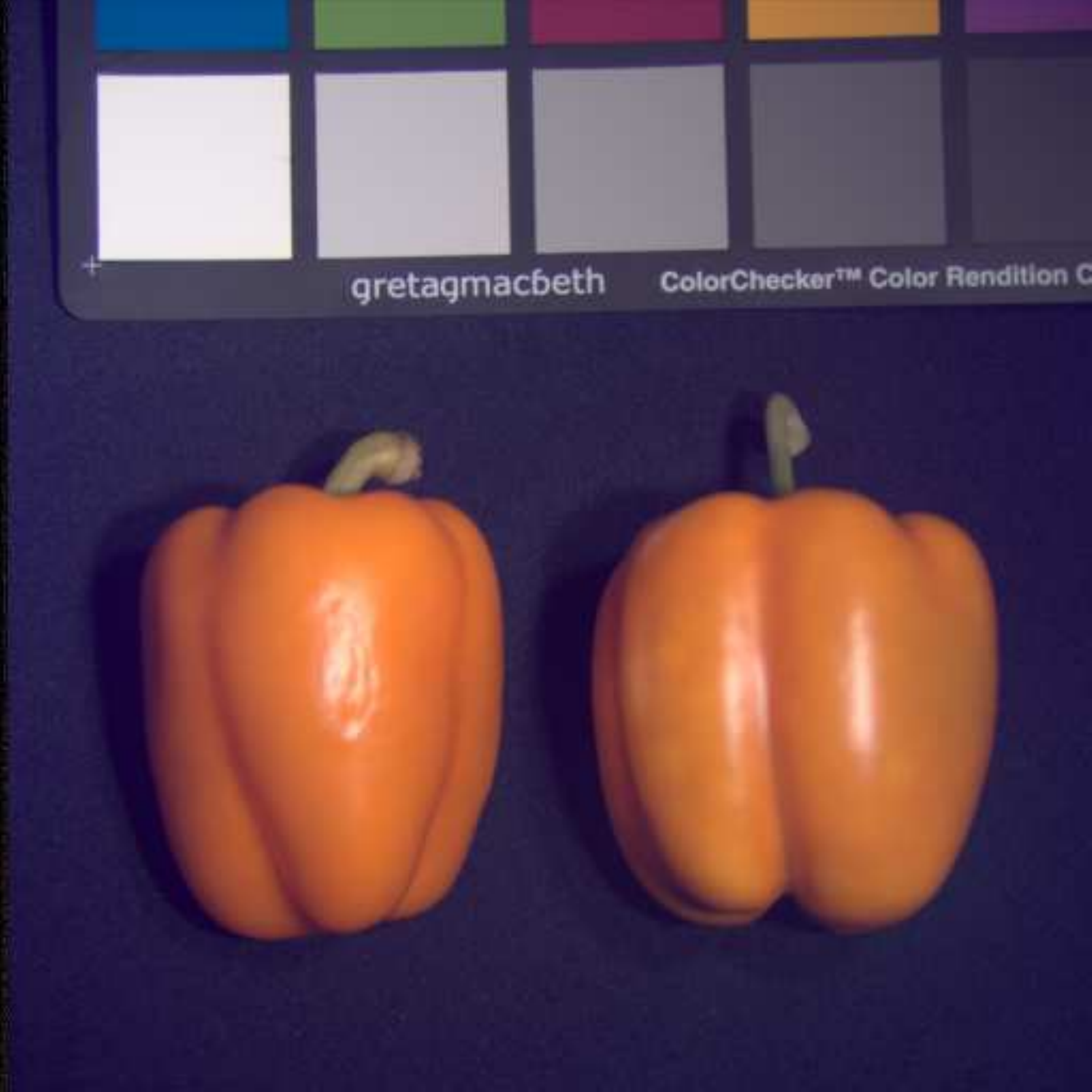}}
\subfigure[]{\includegraphics[trim = 0pt 50pt 0pt 0pt, width=0.30\linewidth]{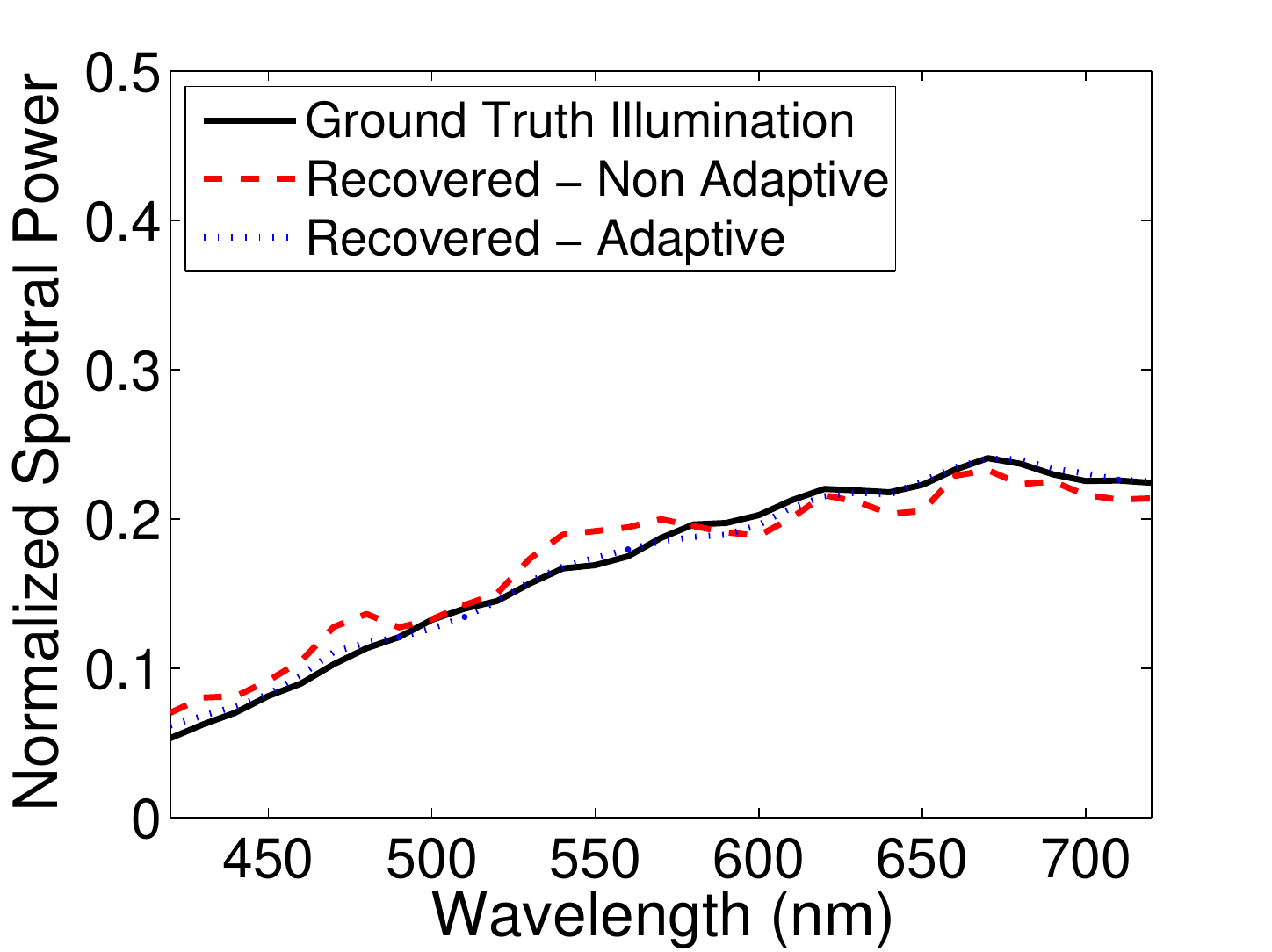}}\\ [-3pt]
\subfigure[Orig~$18.8^{\circ}$]{\includegraphics[width=0.168\linewidth]{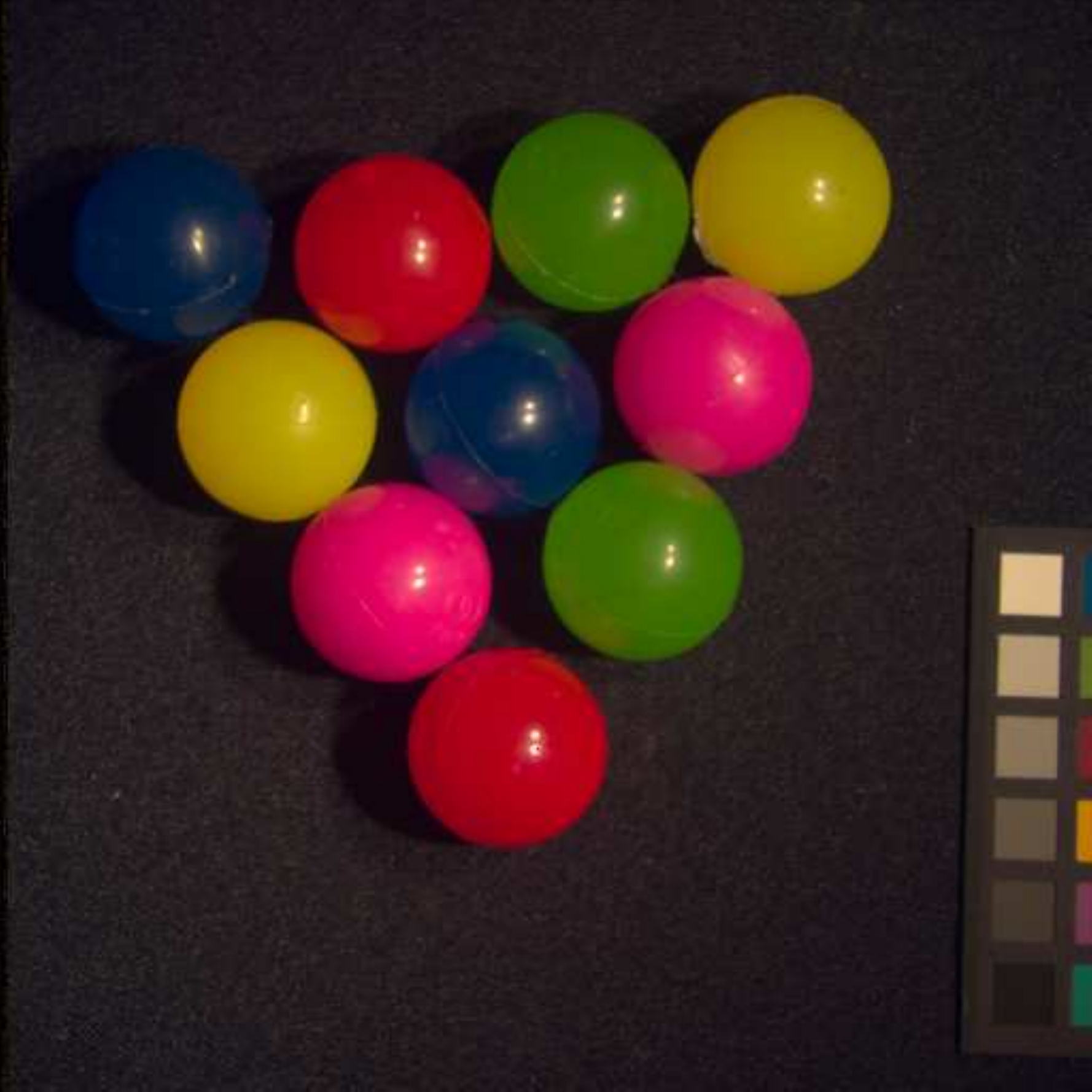}}
\subfigure[Ideal~$0^{\circ}$]{\includegraphics[width=0.168\linewidth]{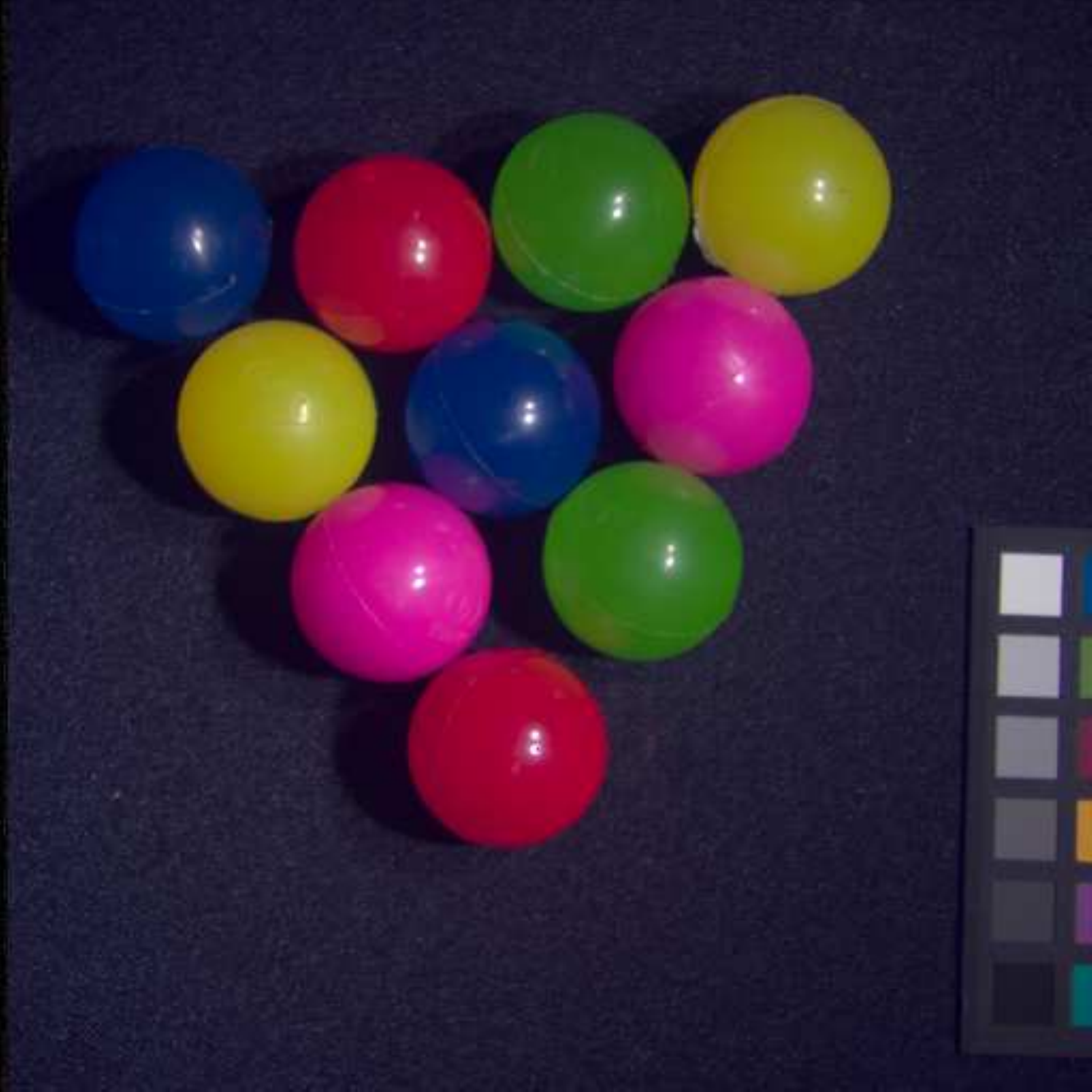}}
\subfigure[WP~$15.7^{\circ}$]{\includegraphics[width=0.168\linewidth]{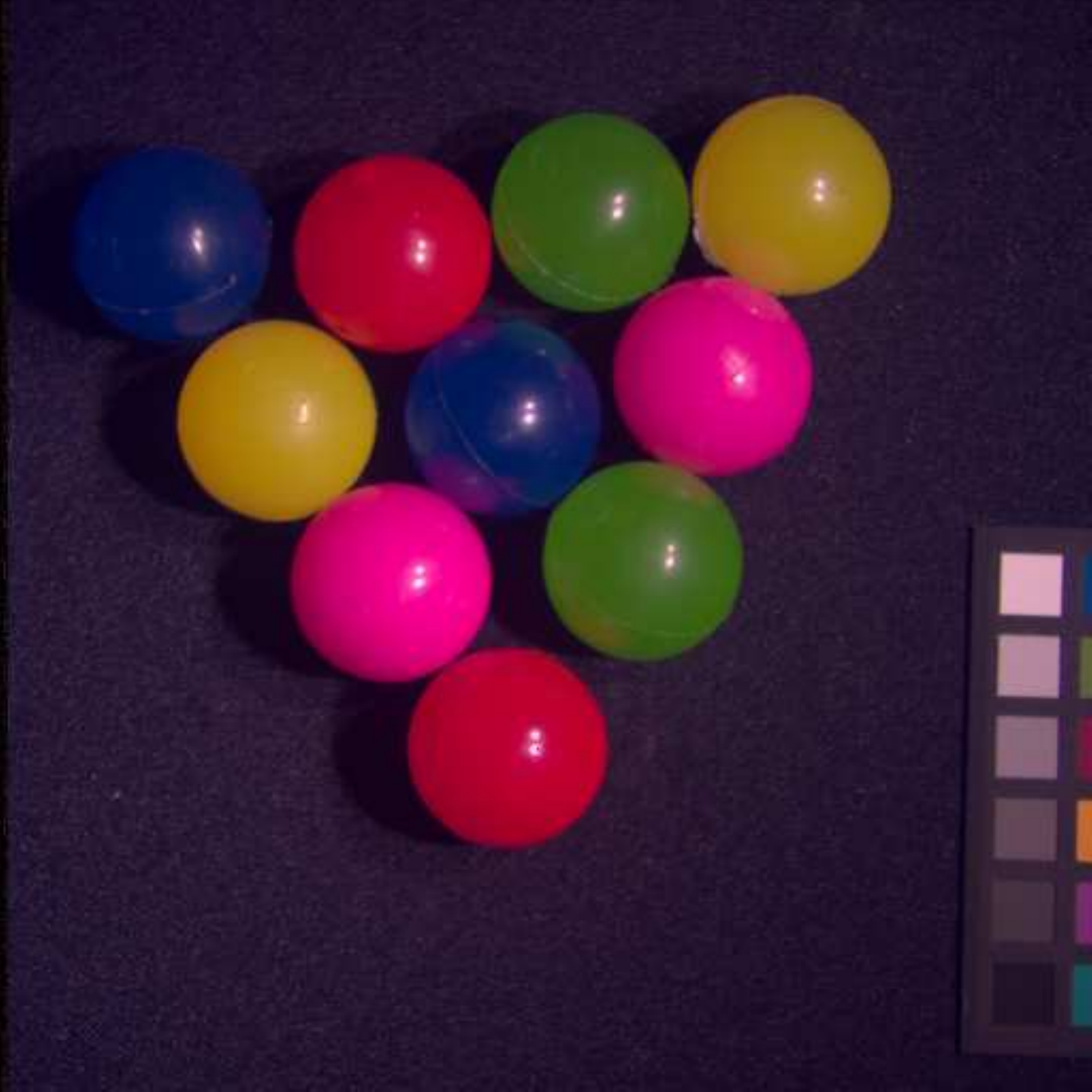}}
\subfigure[WP-a~$9.9^{\circ}$]{\includegraphics[width=0.168\linewidth]{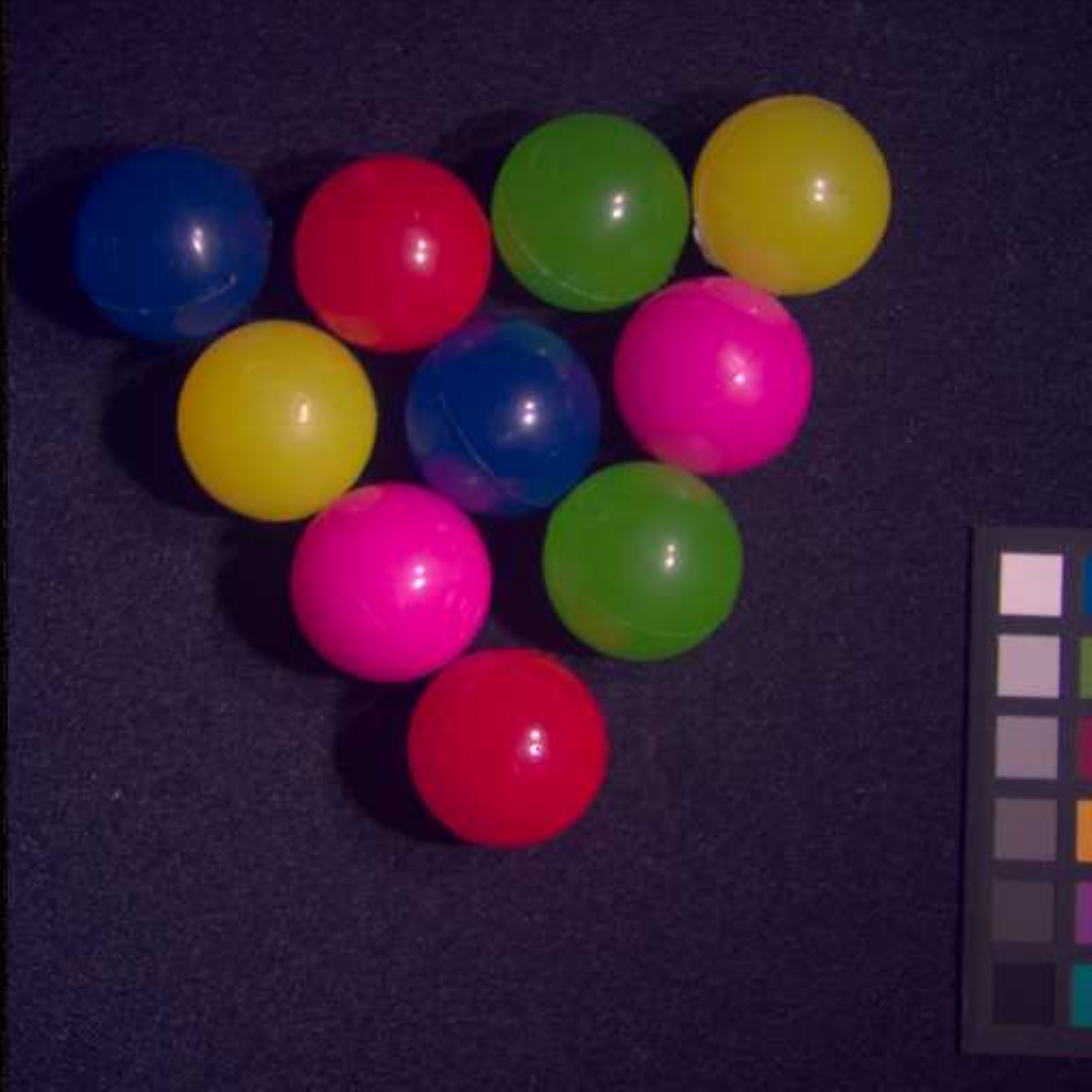}}
\subfigure[]{\includegraphics[trim = 0pt 50pt 0pt 0pt, width=0.30\linewidth]{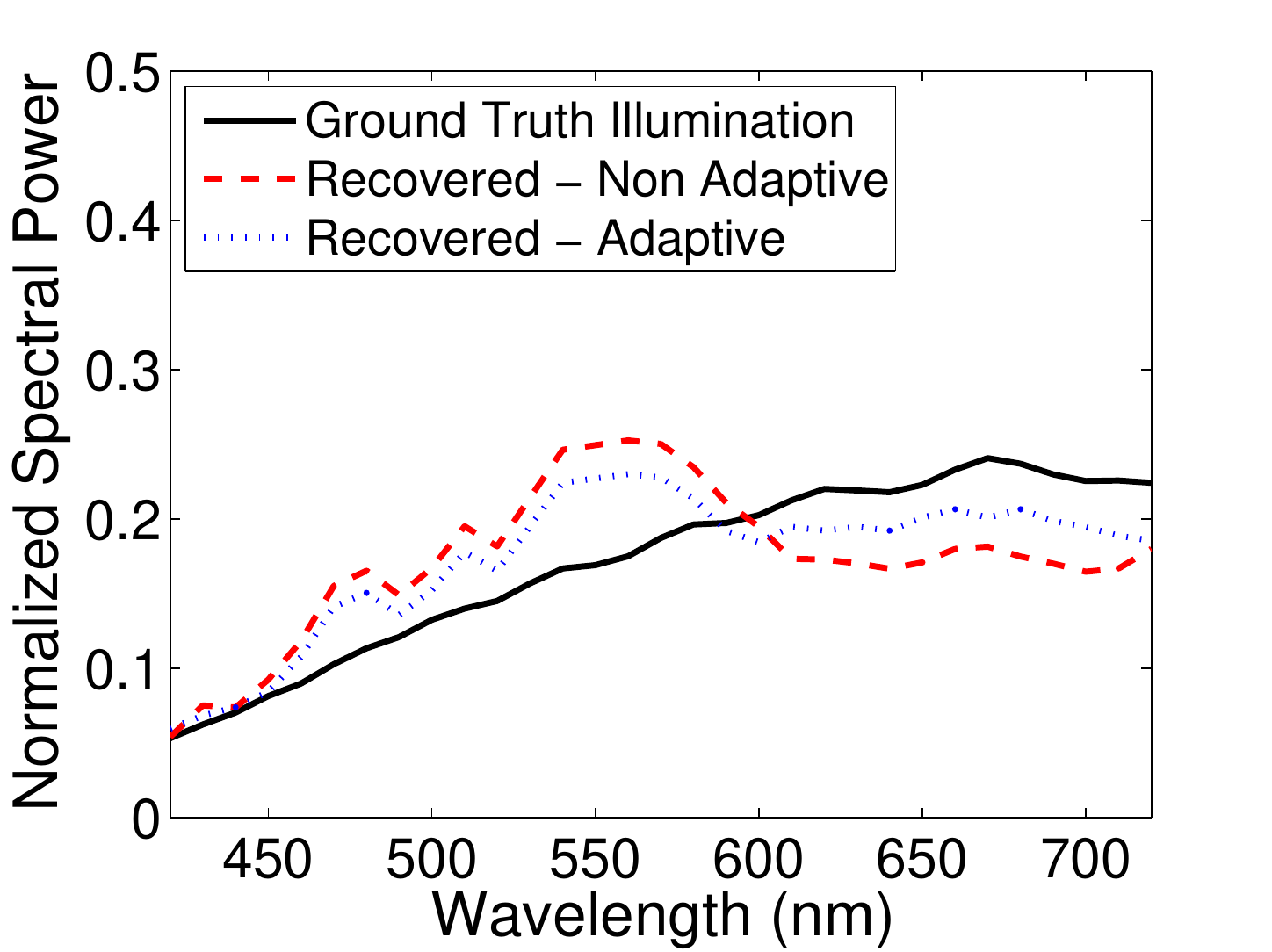}}\\ [-3pt]
\subfigure[Orig~$39.7^{\circ}$]{\includegraphics[width=0.168\linewidth]{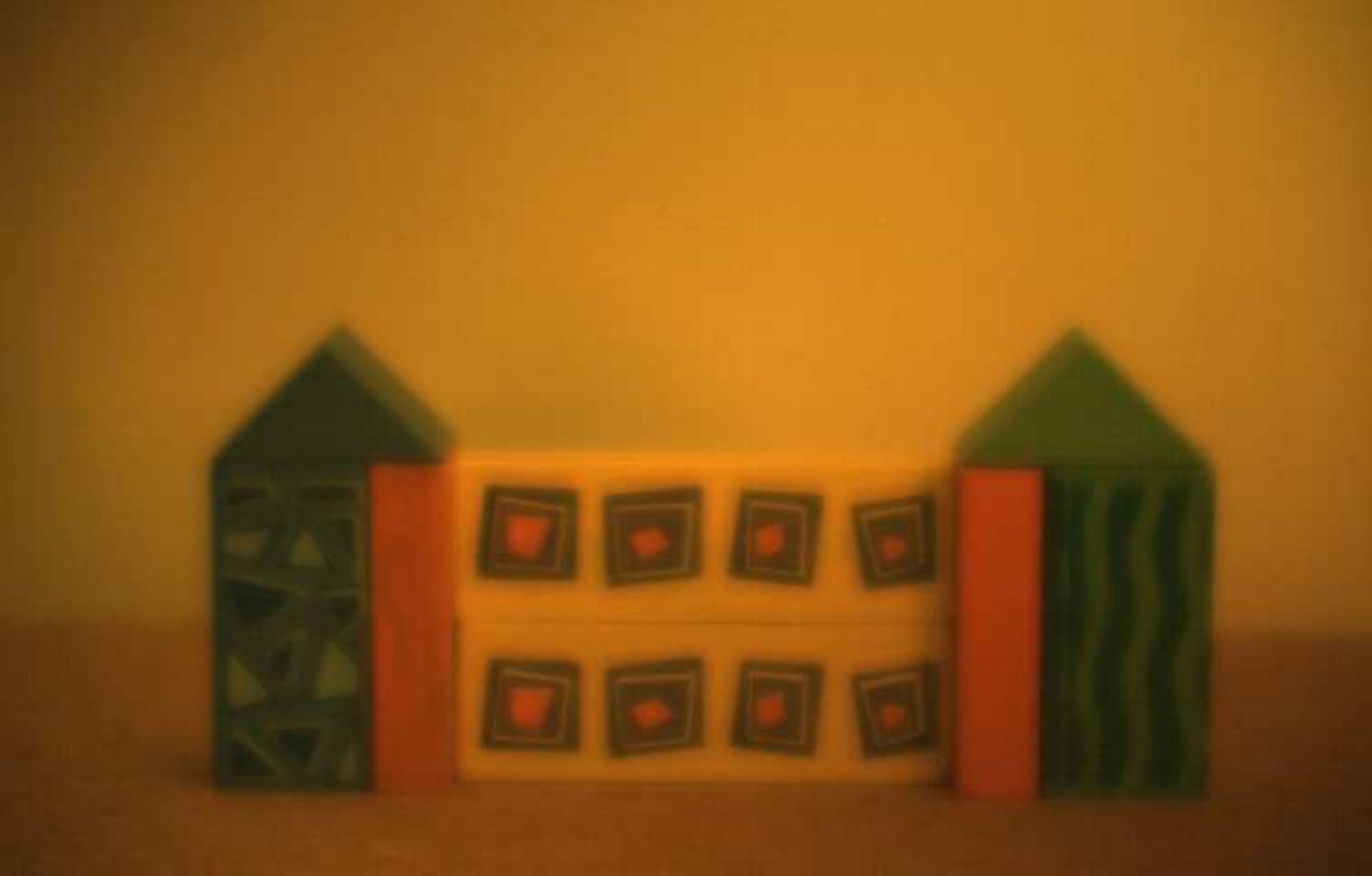}}
\subfigure[Ideal~$0^{\circ}$]{\includegraphics[width=0.168\linewidth]{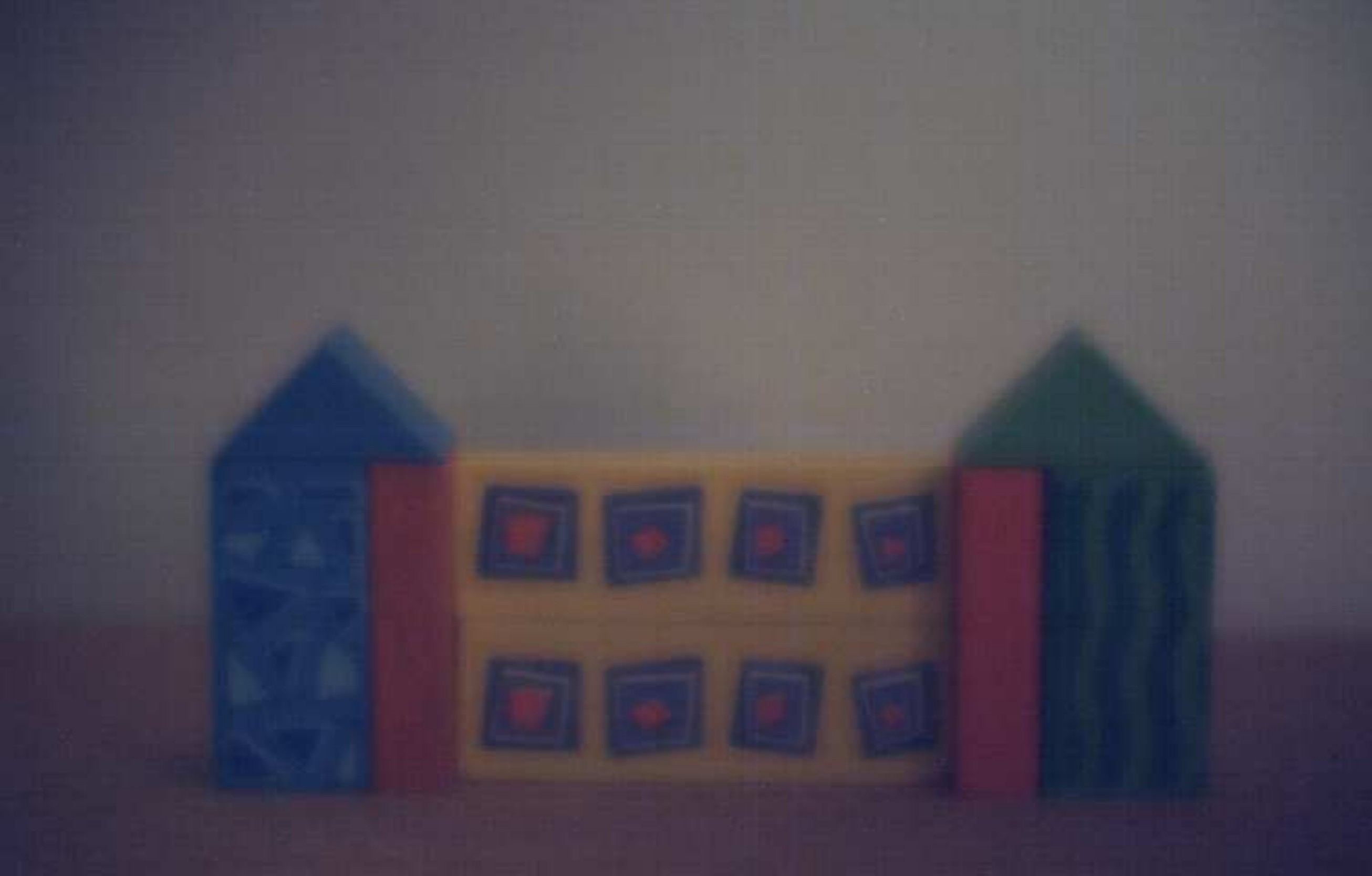}}
\subfigure[GW~$3.1^{\circ}$]{\includegraphics[width=0.168\linewidth]{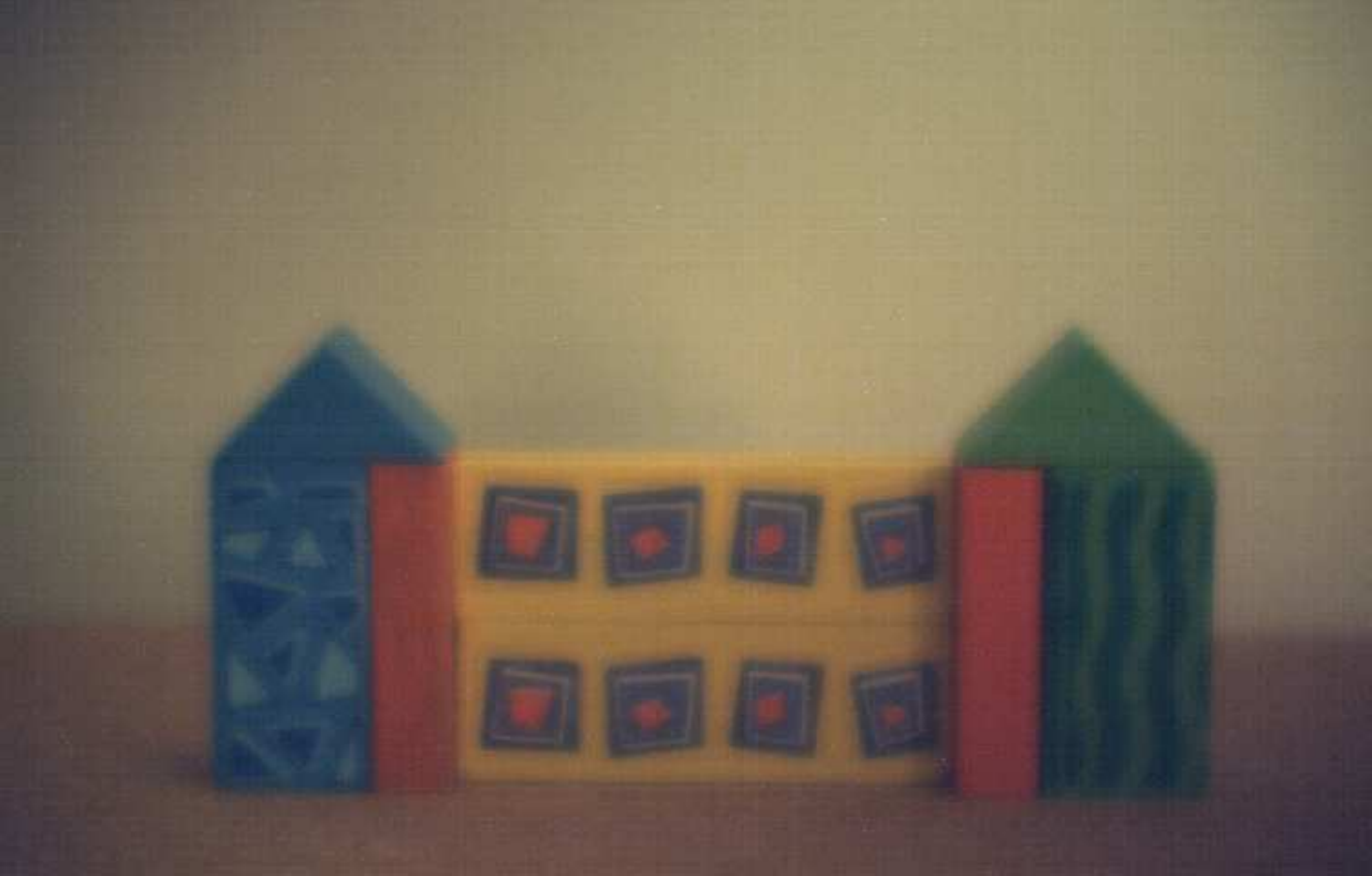}}
\subfigure[GW-a~$3.1^{\circ}$]{\includegraphics[width=0.168\linewidth]{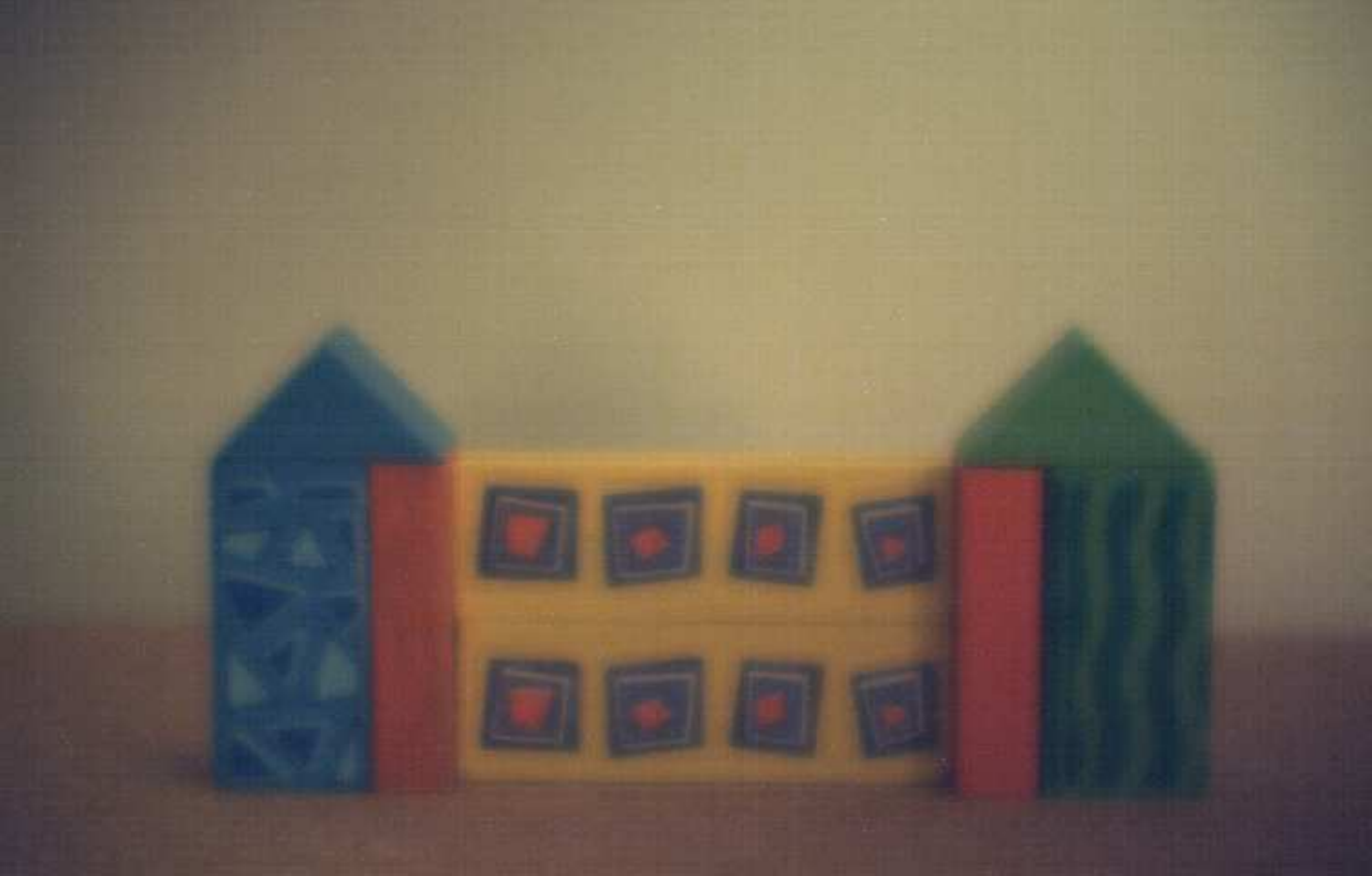}}
\subfigure[]{\includegraphics[trim = 0pt 100pt 0pt 0pt, width=0.30\linewidth]{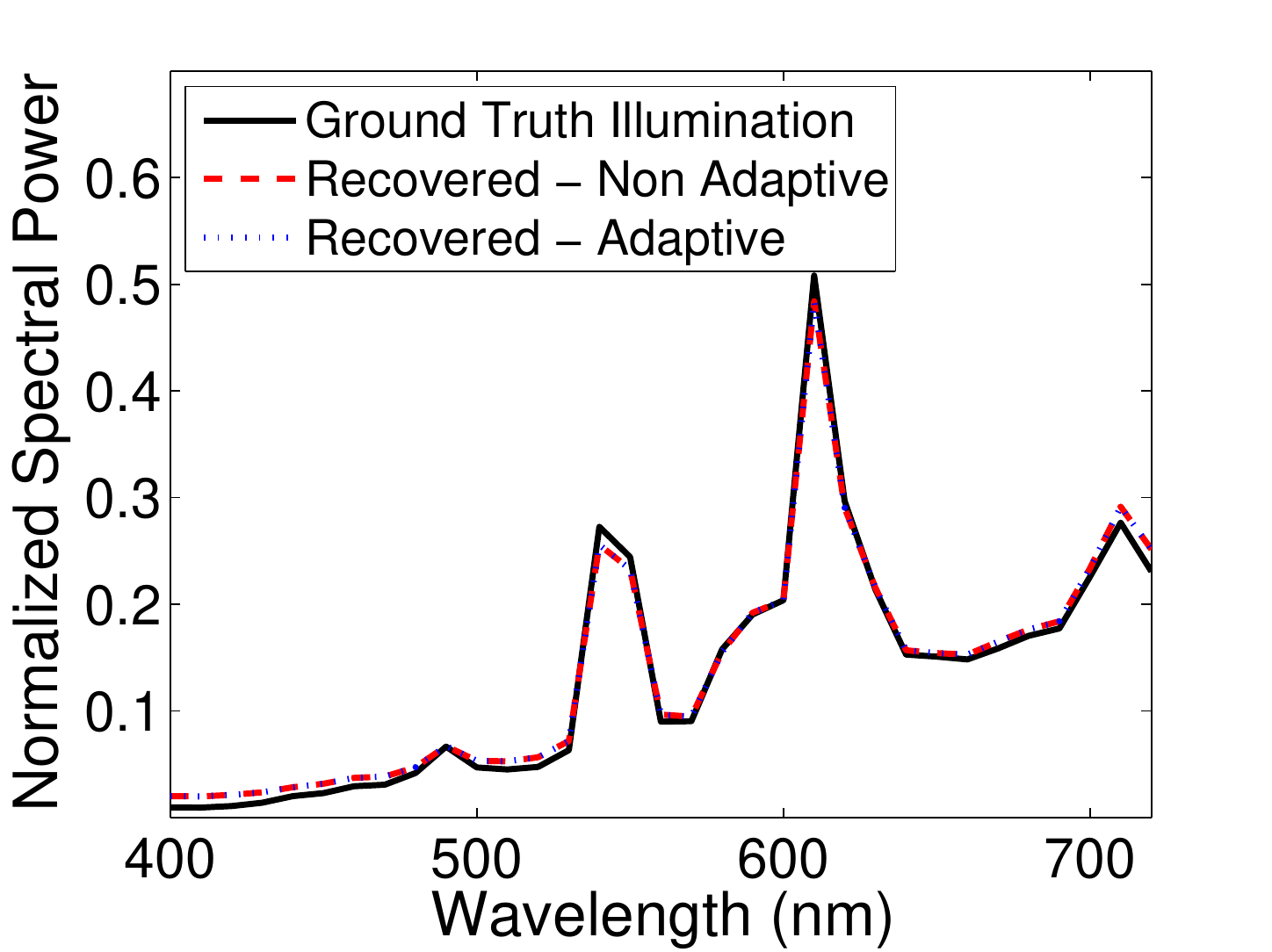}}\\
\subfigure[Orig~$44.5^{\circ}$]{\includegraphics[width=0.168\linewidth]{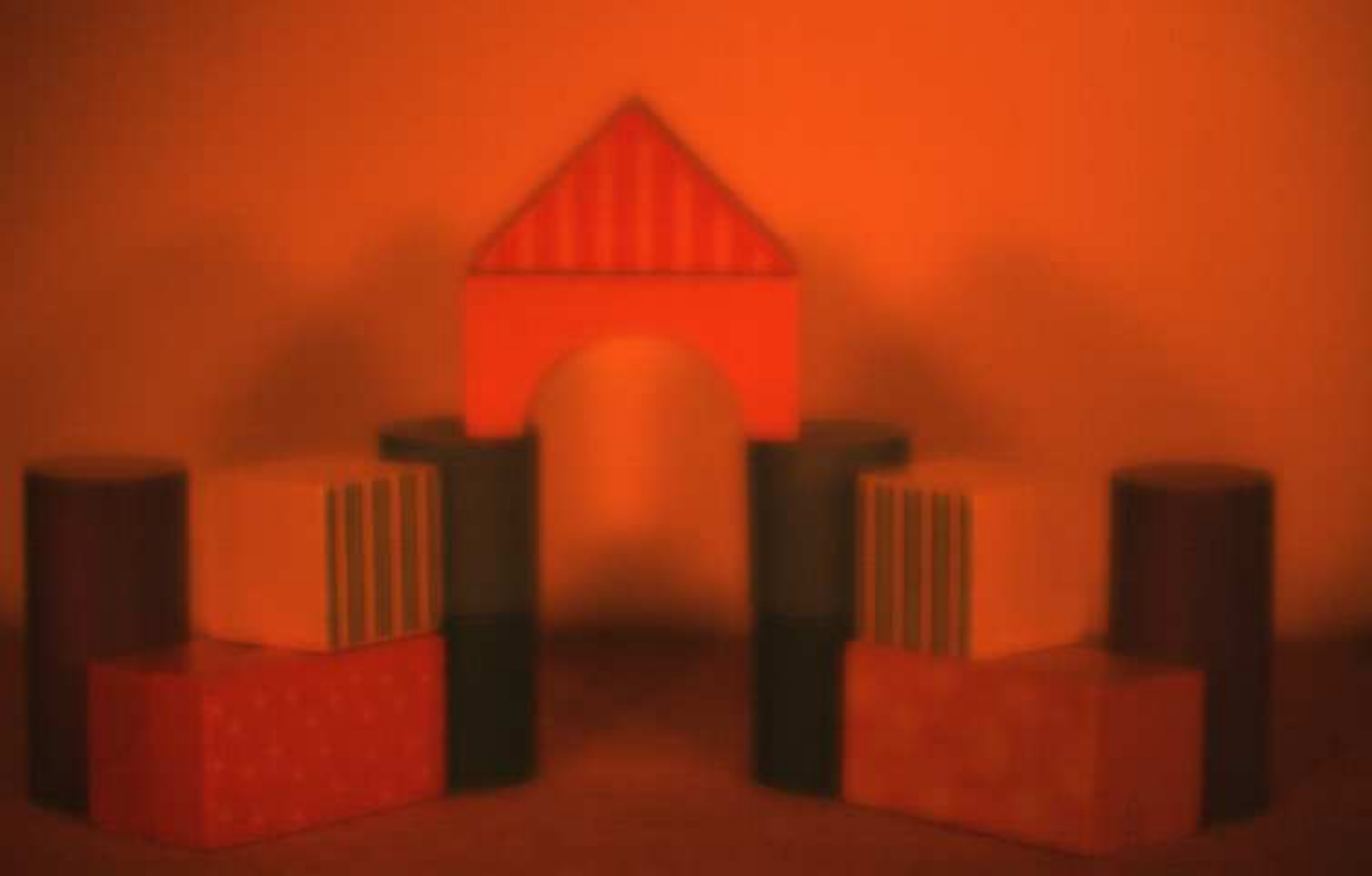}}
\subfigure[Ideal~$0^{\circ}$]{\includegraphics[width=0.168\linewidth]{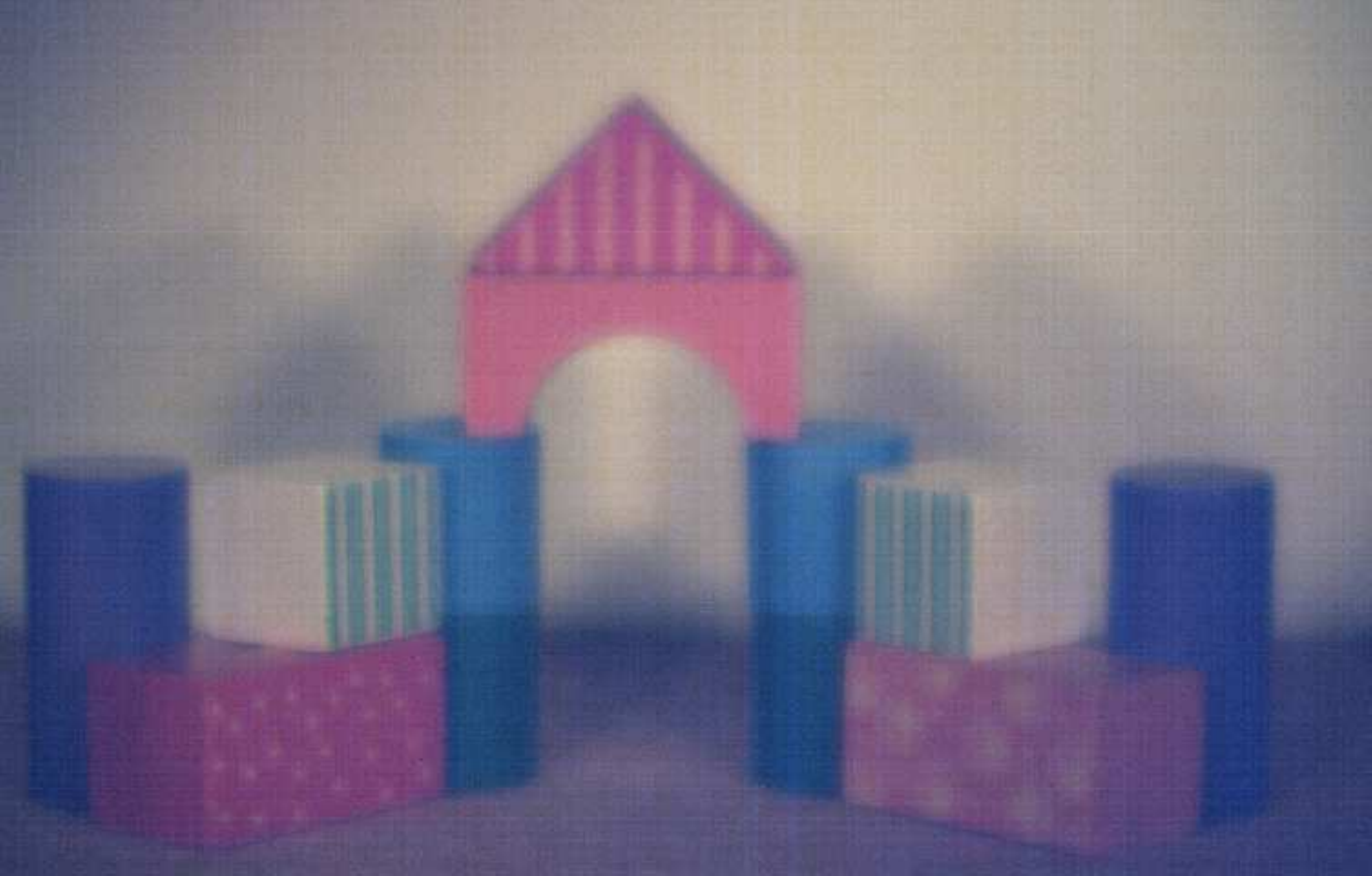}}
\subfigure[SoG~$2.3^{\circ}$]{\includegraphics[width=0.168\linewidth]{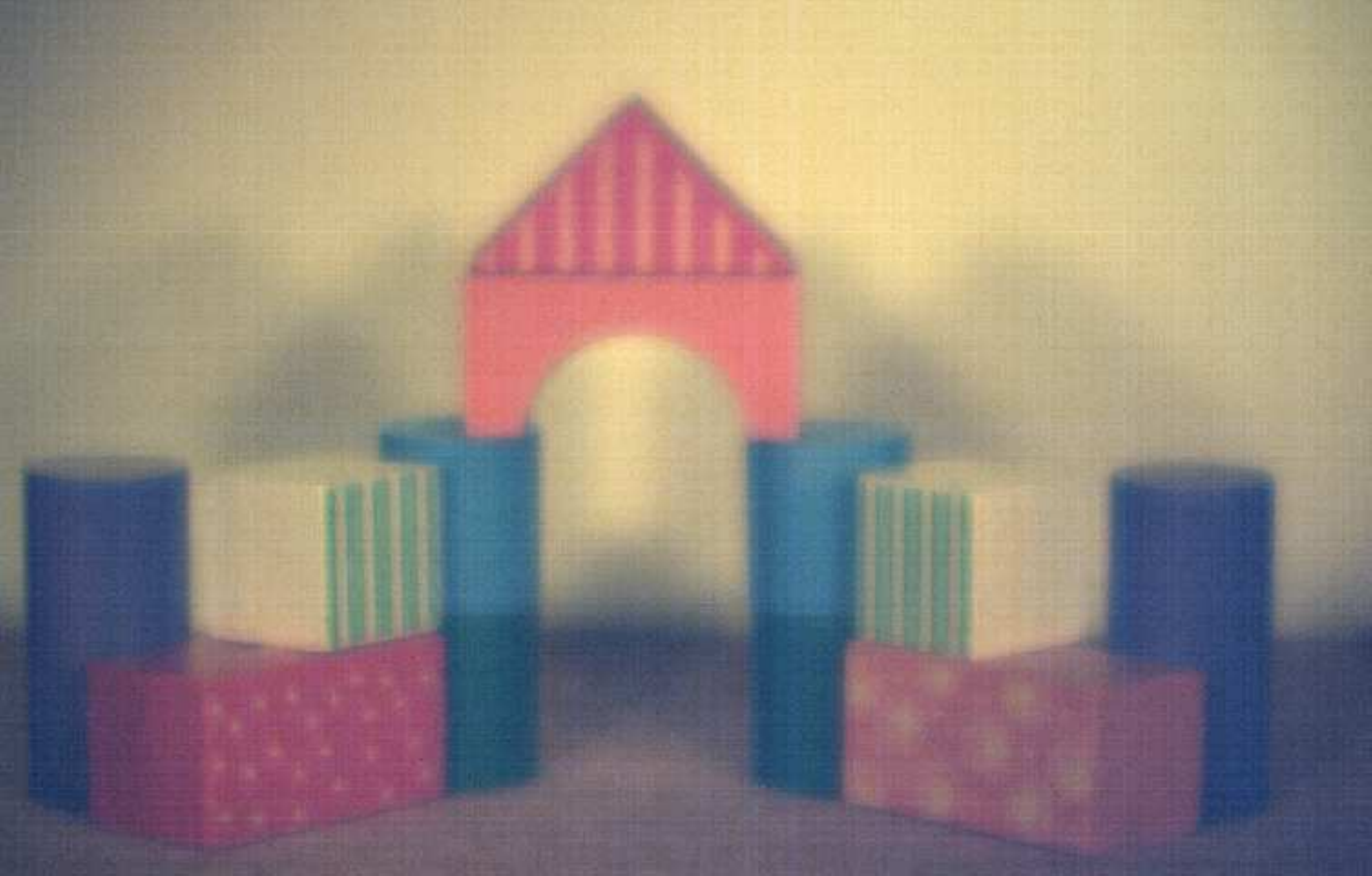}}
\subfigure[SoG-a~$1.9^{\circ}$]{\includegraphics[width=0.168\linewidth]{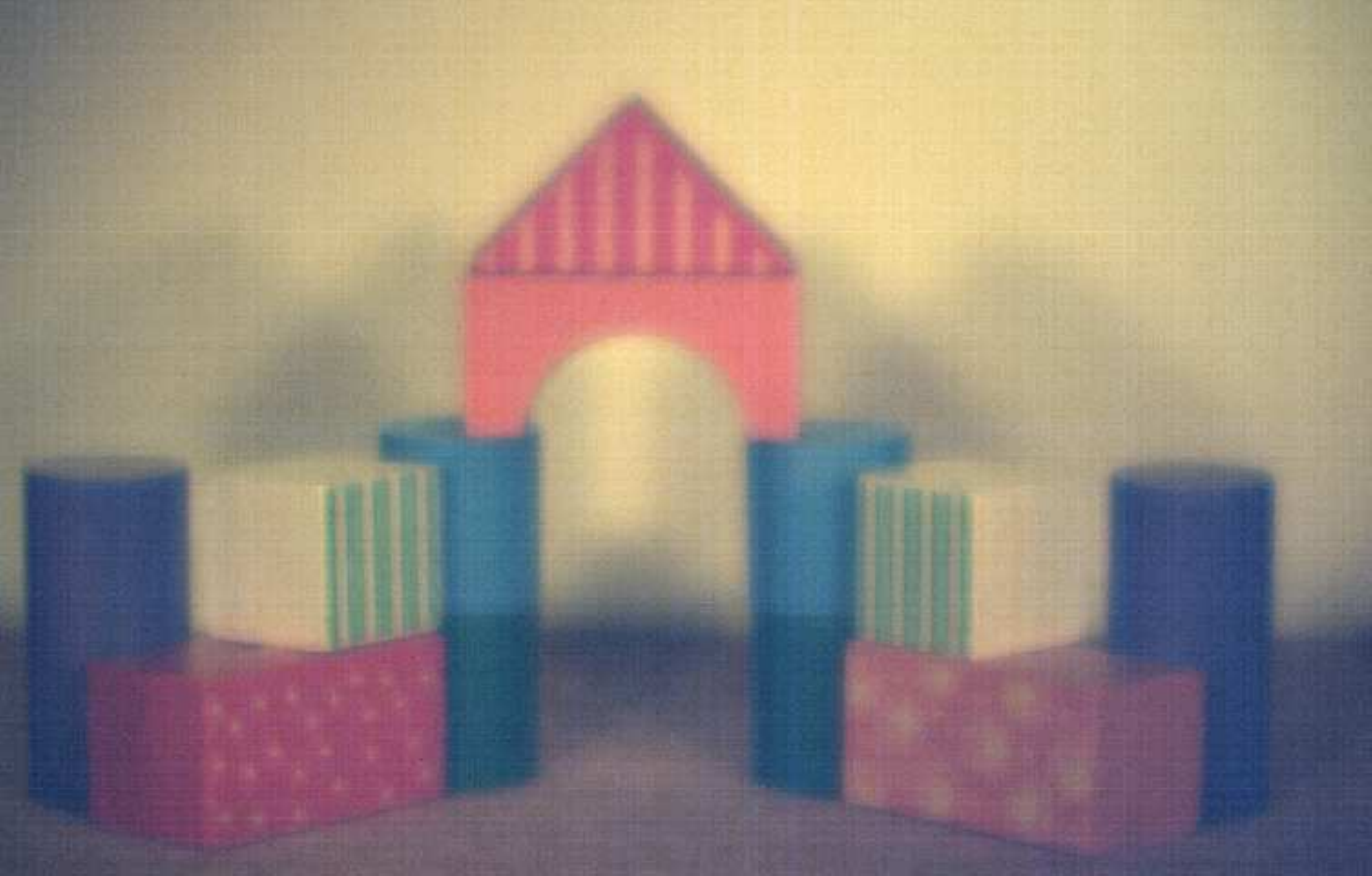}}
\subfigure[]{\includegraphics[trim = 0pt 100pt 0pt 0pt, width=0.30\linewidth]{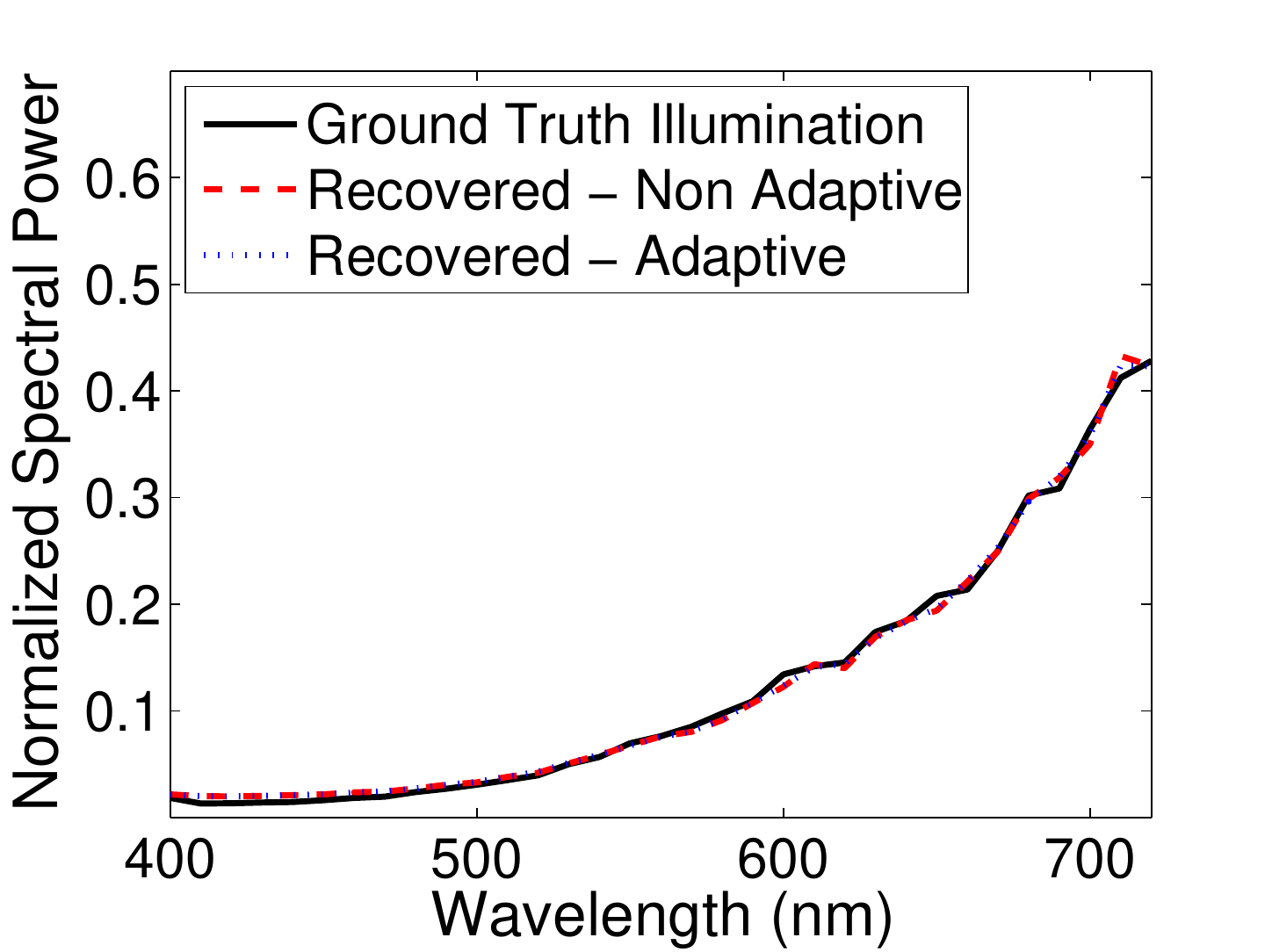}}
\caption[Qualitative comparison of spatio-spectral support]{(left to right) Original image with illumination bias, ideal recovery based on ground truth, recovery with non adaptive and adaptive (-a) spatio-spectral support, and plot of ground truth and recovered illumination SPDs}
\label{fig:qual-adapt-fixed}
\vspace{-10pt}
\end{figure}

\subsection{Color Constancy in Fixed and Variable Exposure Imaging}
\label{sec:cons-auto}

Finally, we evaluate how the two imaging techniques perform in recovering illumination spectra. For this purpose, color constancy experiments were performed on both fixed and variable exposure images, separately. The Relative MAE improvement between the two imaging methods is shown in Figure~\ref{fig:rmae-exposure} which is computed in this experiment as
\begin{equation}
\Delta\bar{\epsilon}_{_{\textrm{rel}}}~(\%)= \frac{\bar{\epsilon}_{_{\textrm{fix}}}-\bar{\epsilon}_{_{\textrm{var}}}}{\bar{\epsilon}_{_{\textrm{fix}}}} \times 100~.
\end{equation}
where $\bar{\epsilon}_{_{\textrm{fix}}}$ and $\bar{\epsilon}_{_{\textrm{var}}}$ are the mean angular errors on images captured by fixed and variable exposure methods respectively. It can be observed that there is always an improvement after using automatic exposure adjustment technique. Moreover, the improvement is highly appreciable in the gray edge based color constancy algorithms. One main reason which can be attributed to this improvement is that the edge based color constancy methods are sensitive to noise, which is much higher in fixed exposure images compared to variable exposure.

\begin{figure}[h]
\centering
\includegraphics[trim = 90pt 15pt 110pt 2pt, clip, width=0.34\linewidth]{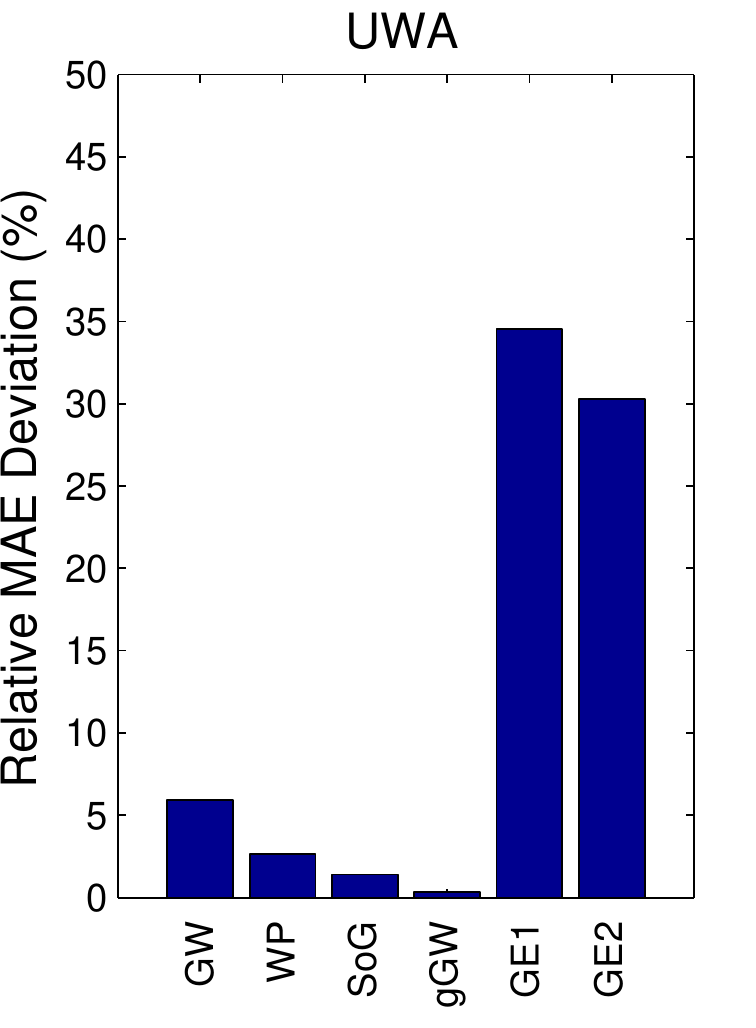}
\caption[Relative MAE improvement between fixed and variable exposure]{Relative MAE improvement between fixed and variable exposure on real data.}
\label{fig:rmae-exposure}
\end{figure}

The illumination estimated from an image $\hat{L}$ is normalized by the ground truth illumination $L$. Ideally, one should recover a uniform illumination, if the estimated illumination is exactly the same as the ground truth illumination. However, in reality there is always a difference between the two which can be measured in terms of the angular error against a uniform illuminant
\begin{equation}
\mathbf{\epsilon}= \arccos \left( \frac{L_{\textrm{uni}} \cdot \hat{L}_{\textrm{norm}}}{\|L_{\textrm{uni}}\|\|\hat{L}_{\textrm{norm}}\|} \right),~\hat{L}_{\textrm{norm}}=\frac{L}{\|L\|} \oslash \frac{\hat{L}}{\|\hat{L}\|}~,
\end{equation}
where $\oslash$ is the element by element division operator. The above error metric measures the deviation of color constancy normalized images from a uniform (flat) illumination.

\begin{figure}[!h]
\renewcommand{\thesubfigure}{\relax}
\subfigure[Orig~$44.6^{\circ}$]{\includegraphics[width=0.168\linewidth]{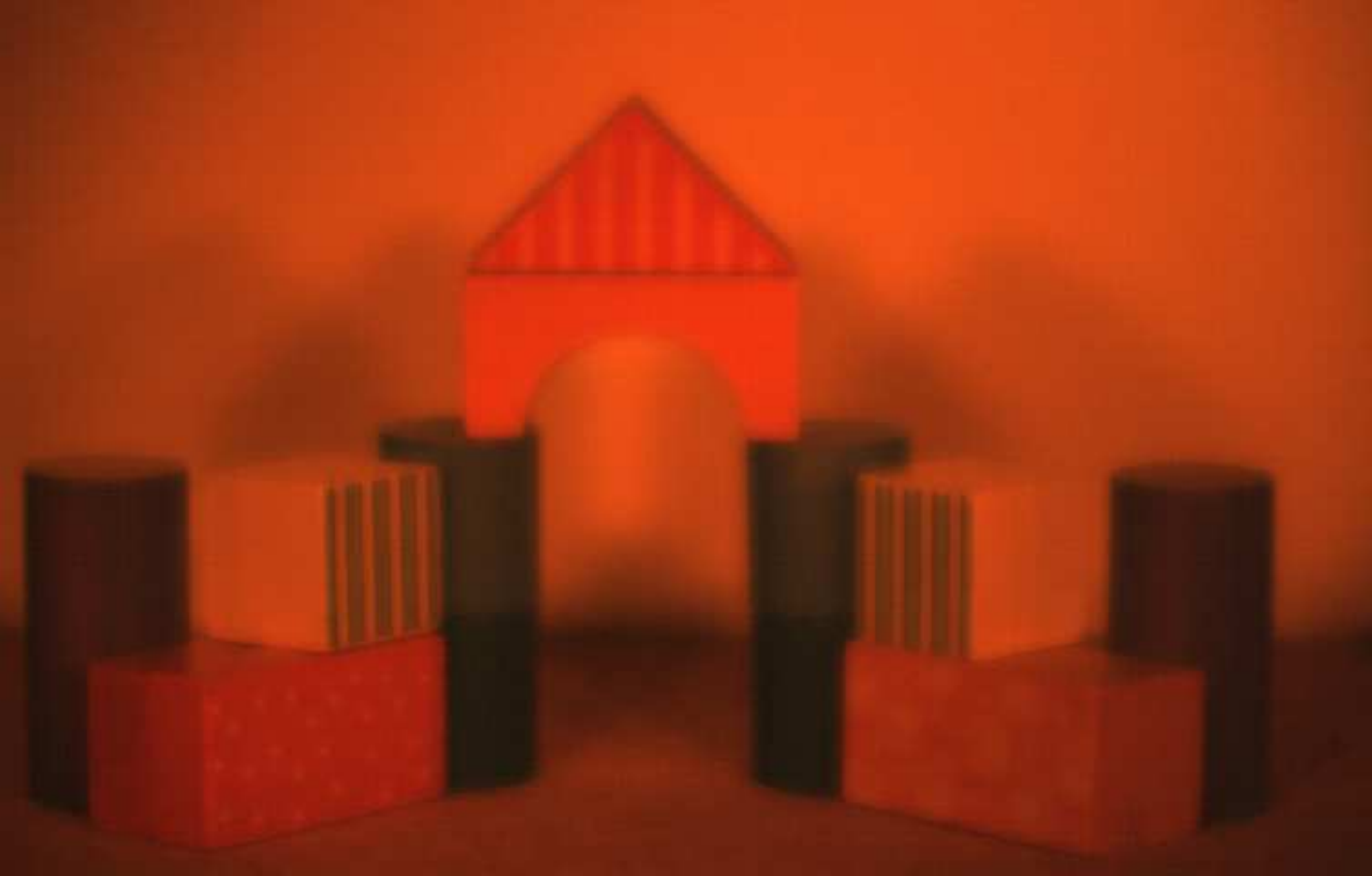}}
\subfigure[Ideal~$0^{\circ}$]{\includegraphics[width=0.168\linewidth]{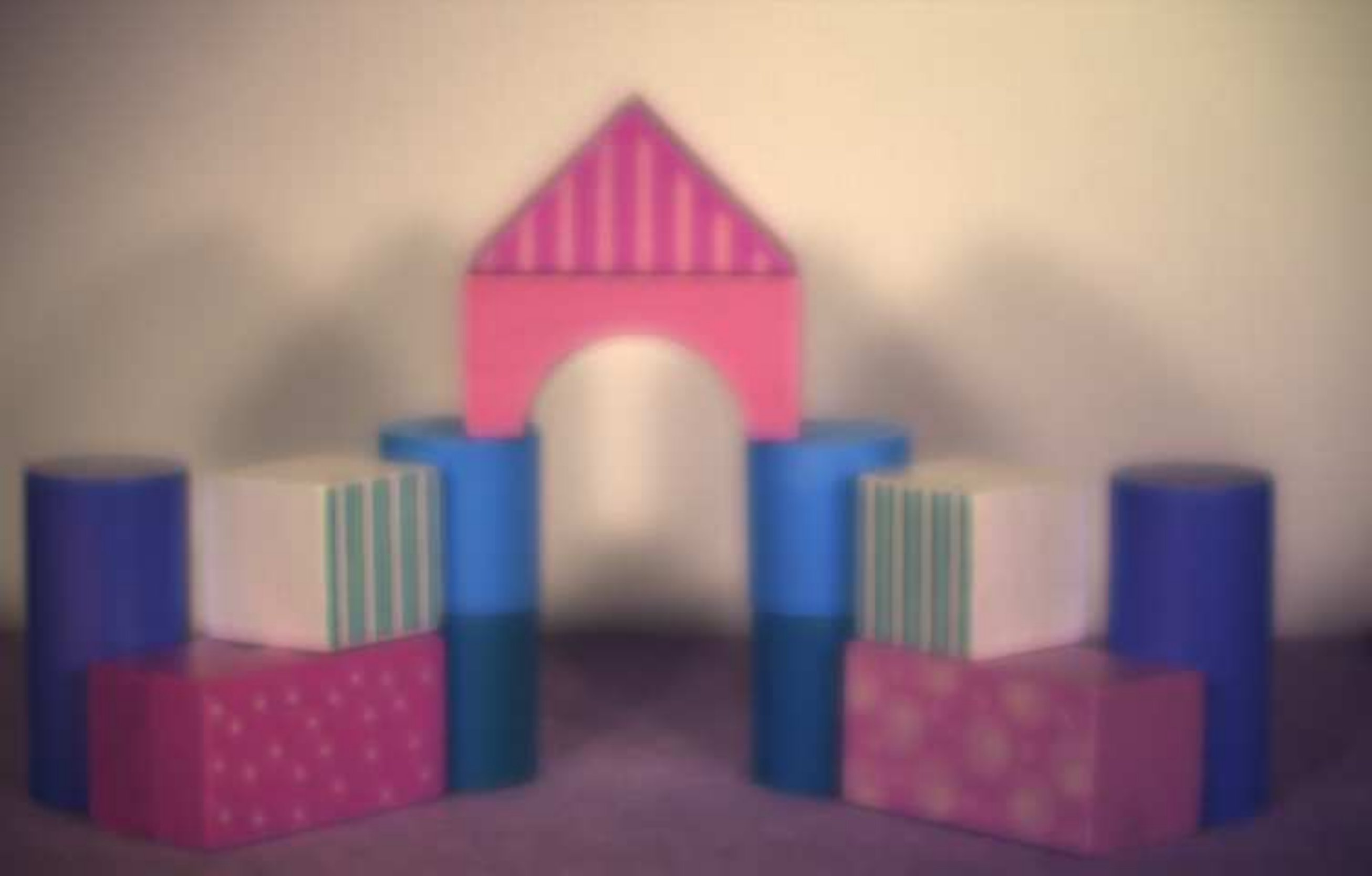}}
\subfigure[GE1-f~$18.7^{\circ}$]{\includegraphics[width=0.168\linewidth]{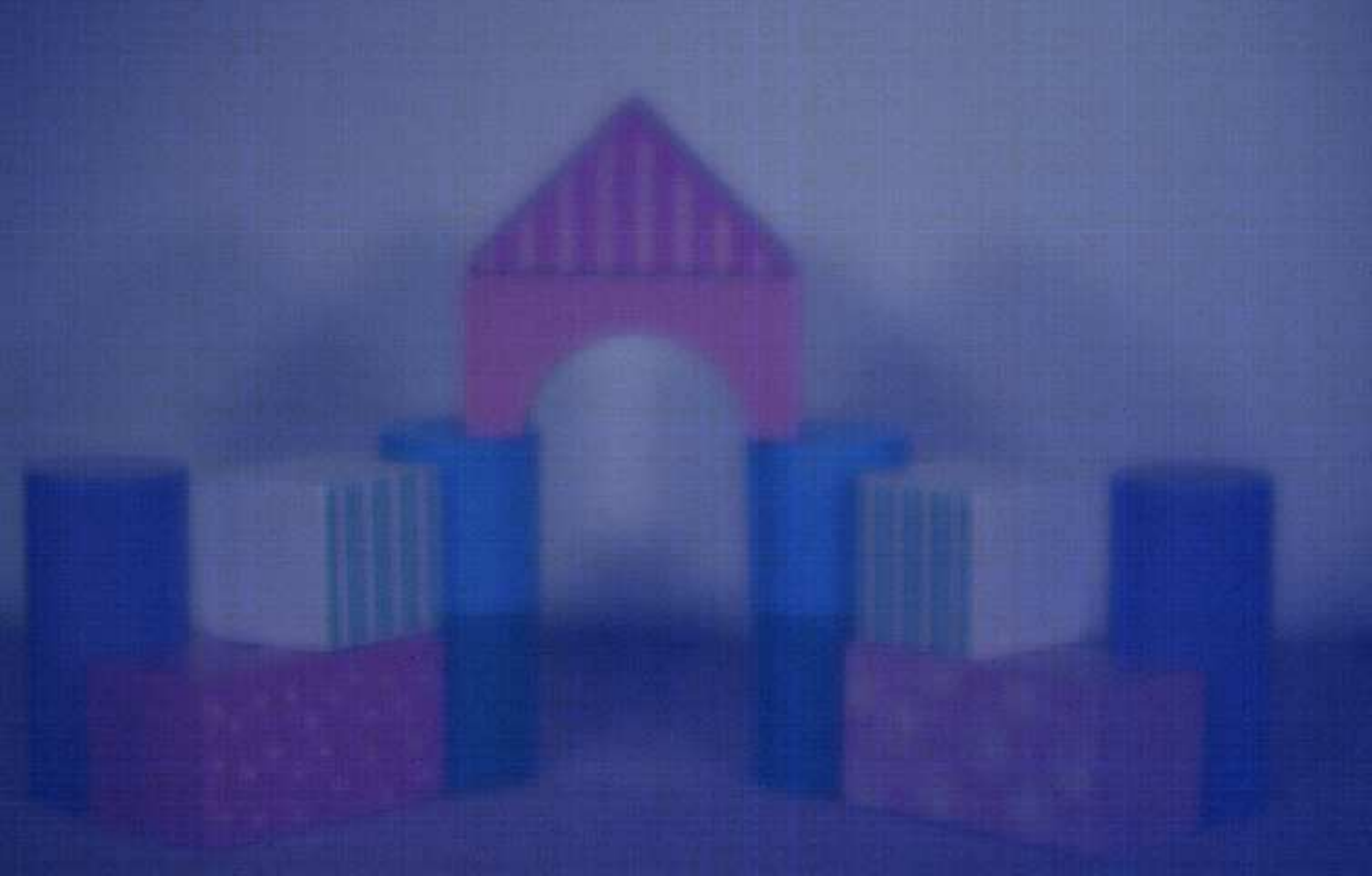}}
\subfigure[GE1-v~$6.9^{\circ}$]{\includegraphics[width=0.168\linewidth]{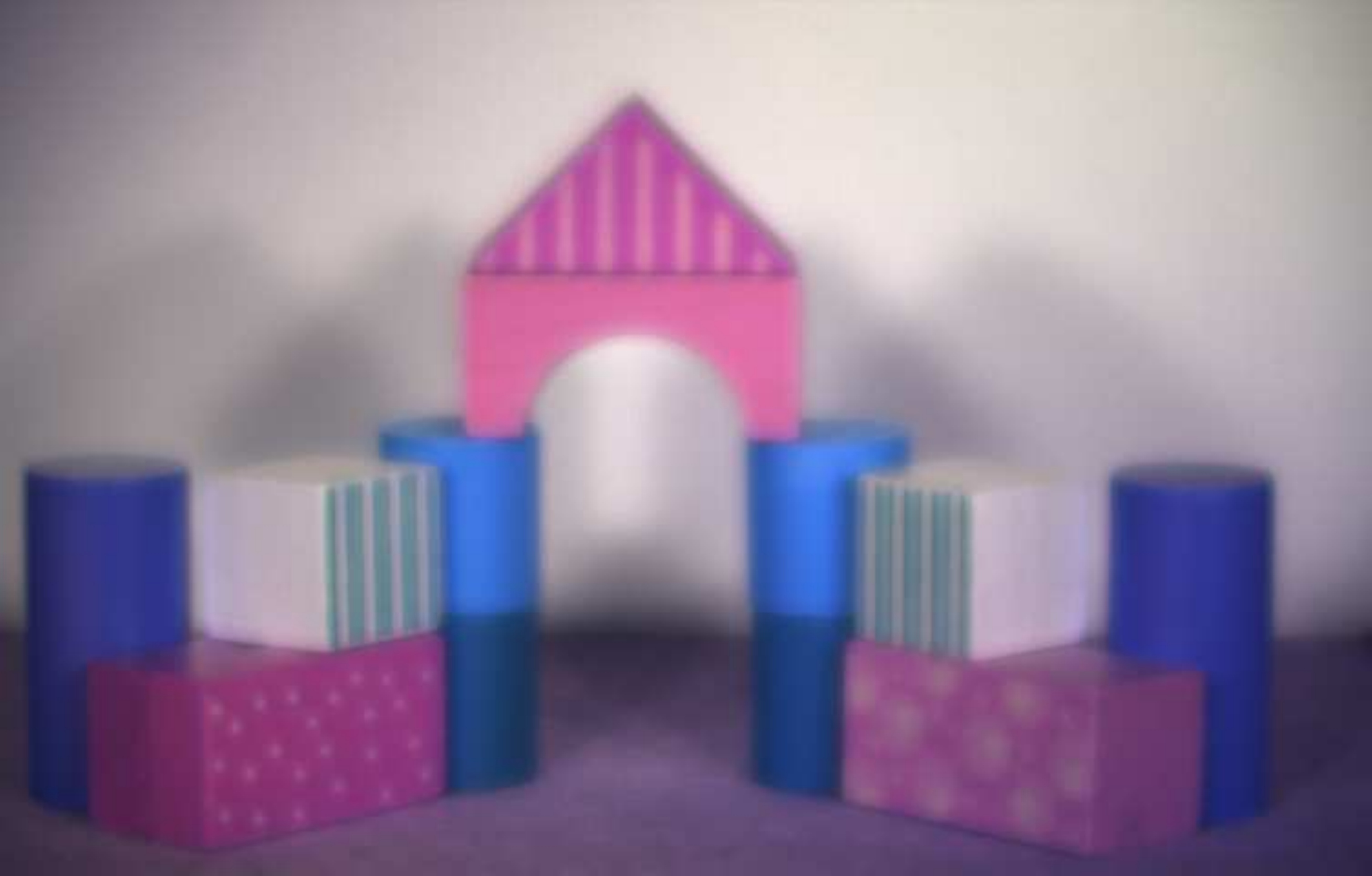}}
\subfigure[]{\includegraphics[trim = 0pt 100pt 0pt 0pt, width=0.30\linewidth]{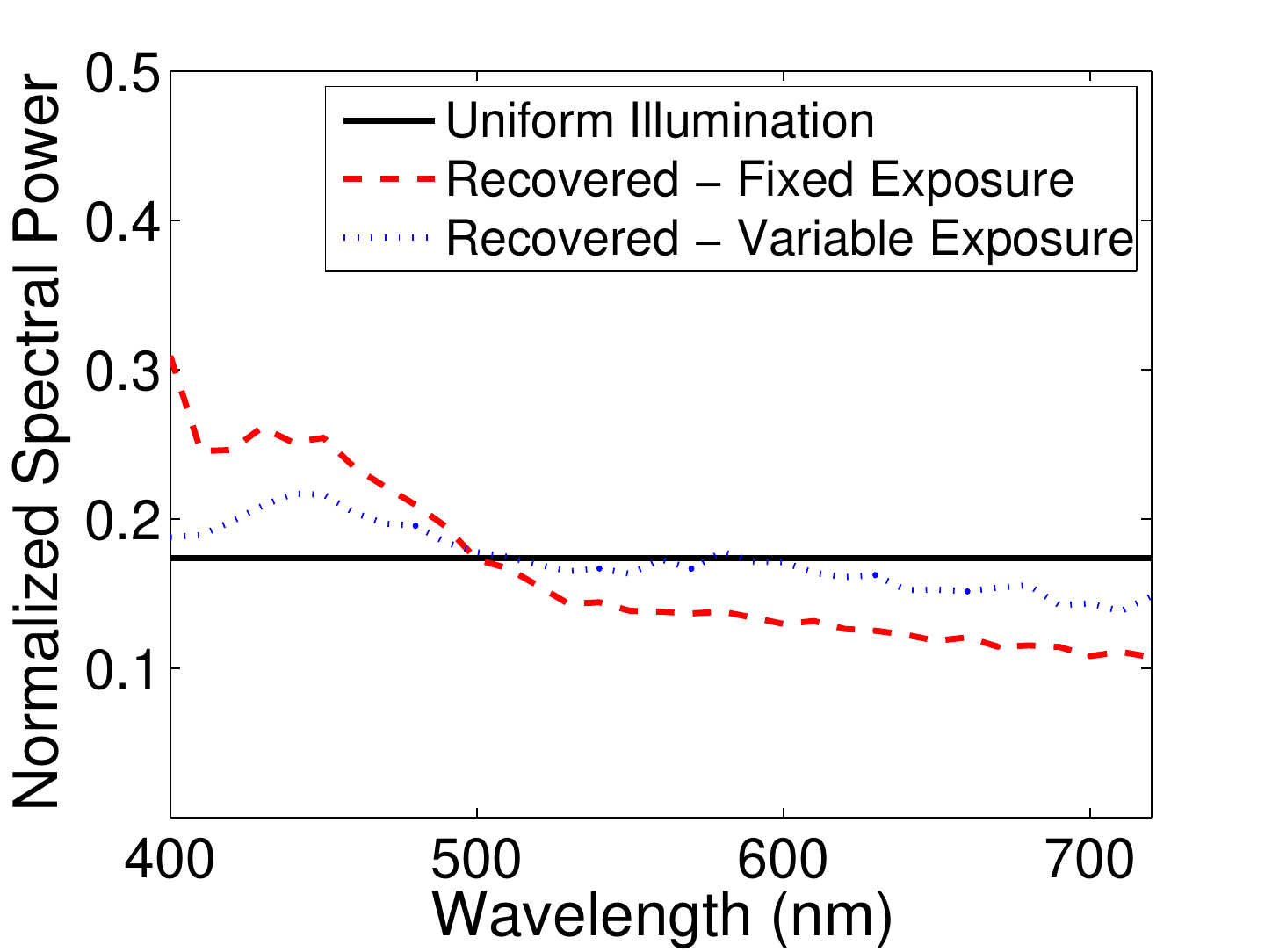}}\\ [4pt]
\subfigure[Orig~$44.5^{\circ}$]{\includegraphics[width=0.168\linewidth]{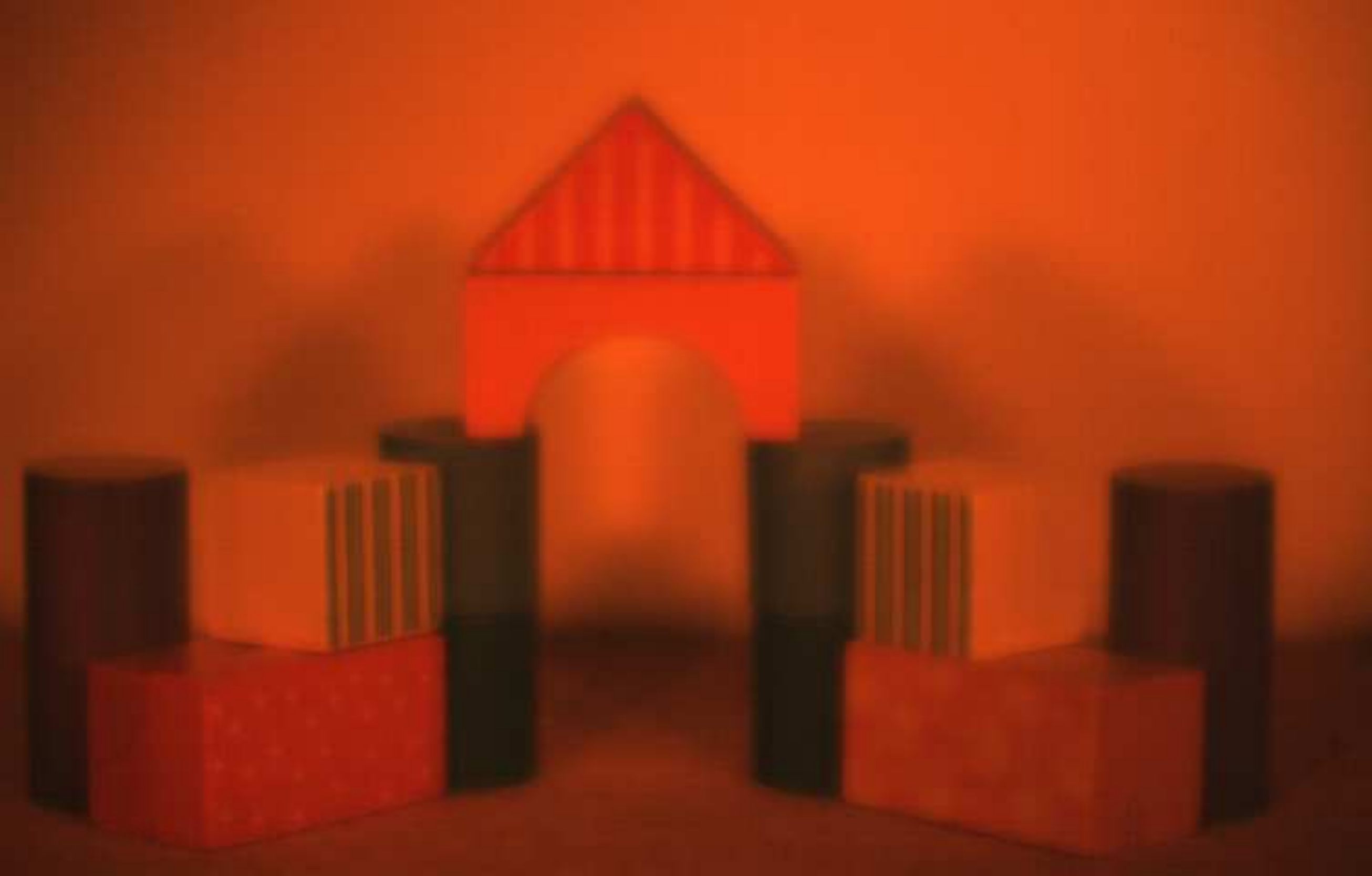}}
\subfigure[Ideal~$0^{\circ}$]{\includegraphics[width=0.168\linewidth]{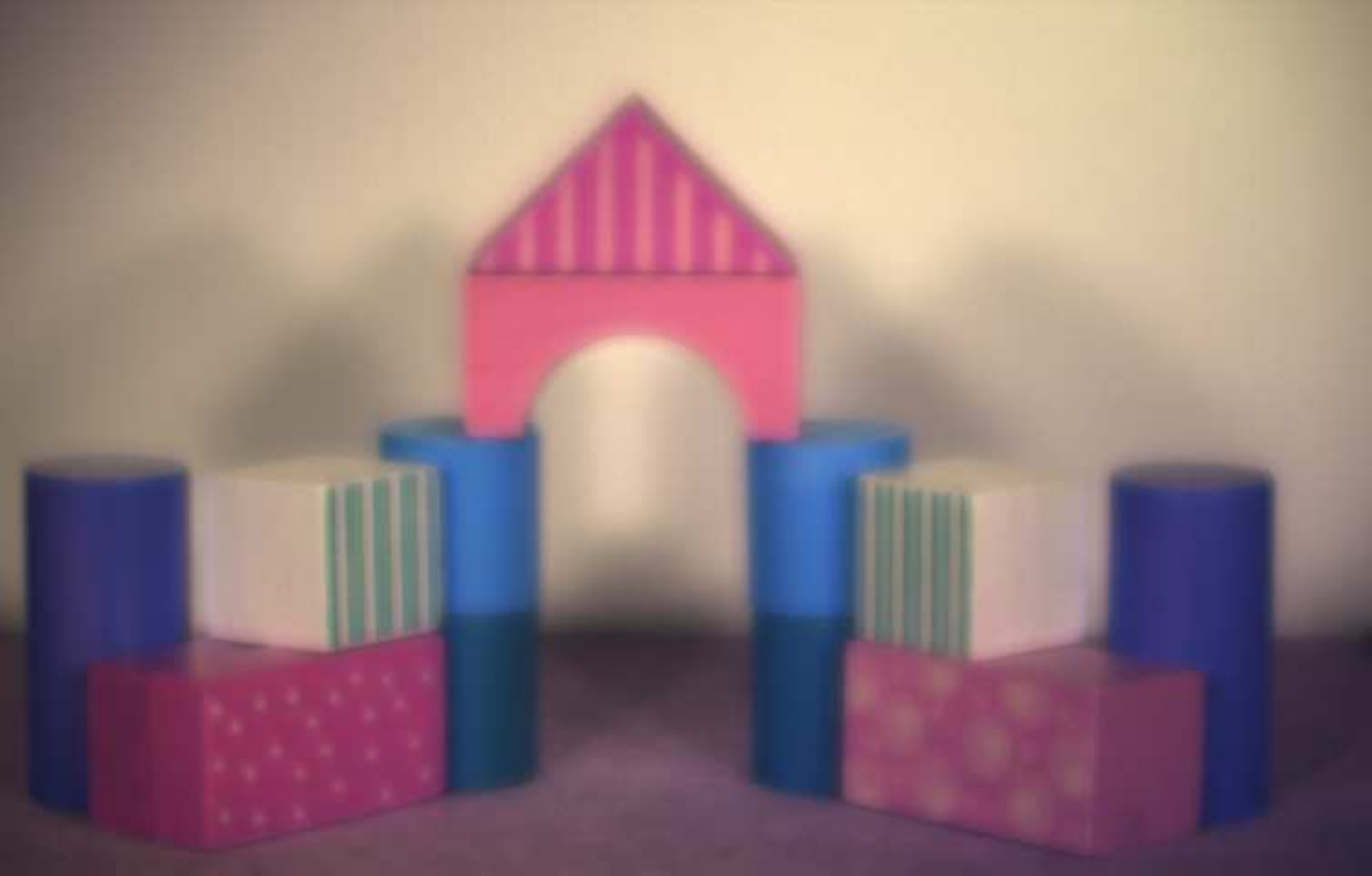}}
\subfigure[GW-f~$10.0^{\circ}$]{\includegraphics[width=0.168\linewidth]{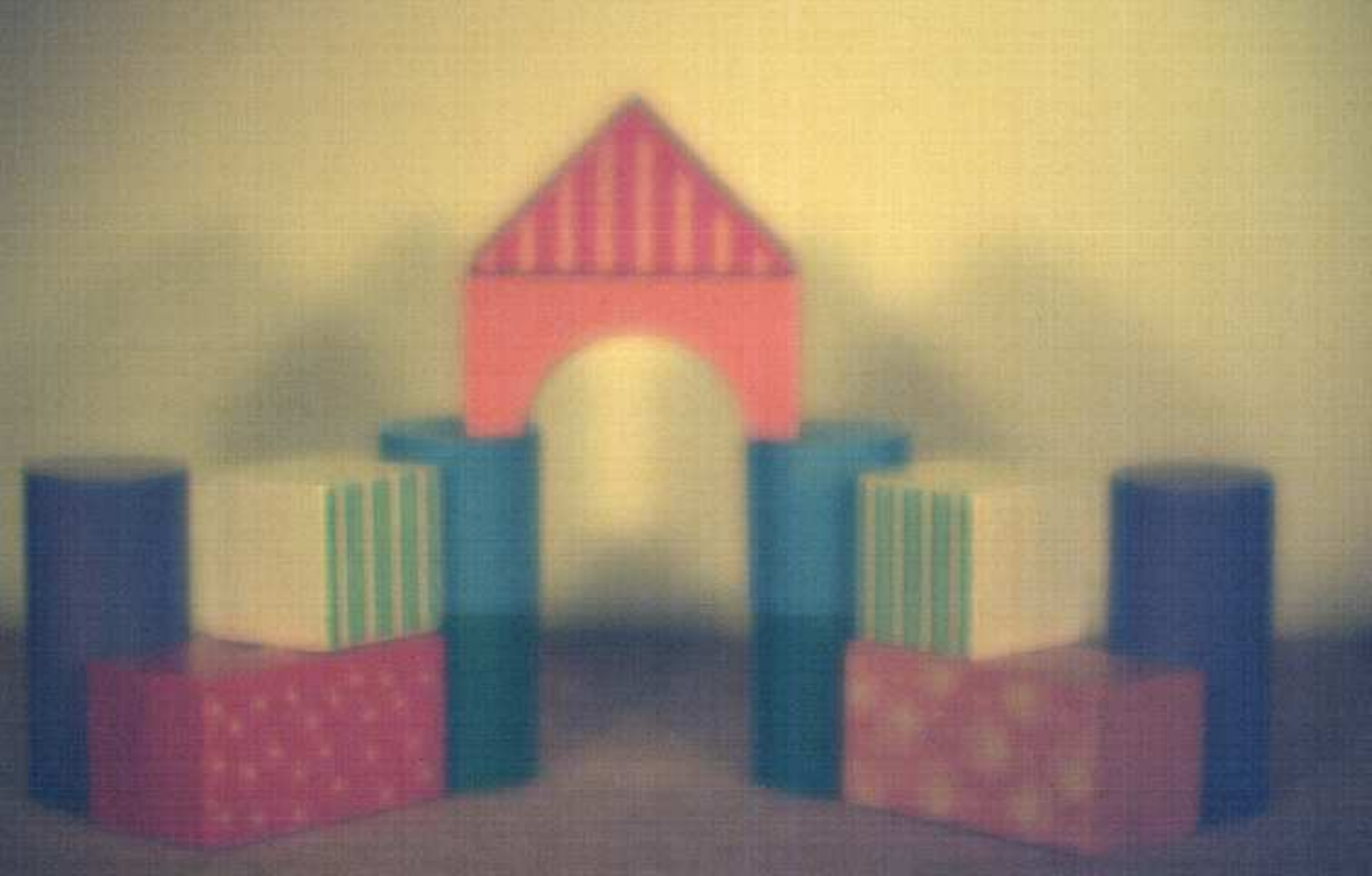}}
\subfigure[GW-v~$6.7^{\circ}$]{\includegraphics[width=0.168\linewidth]{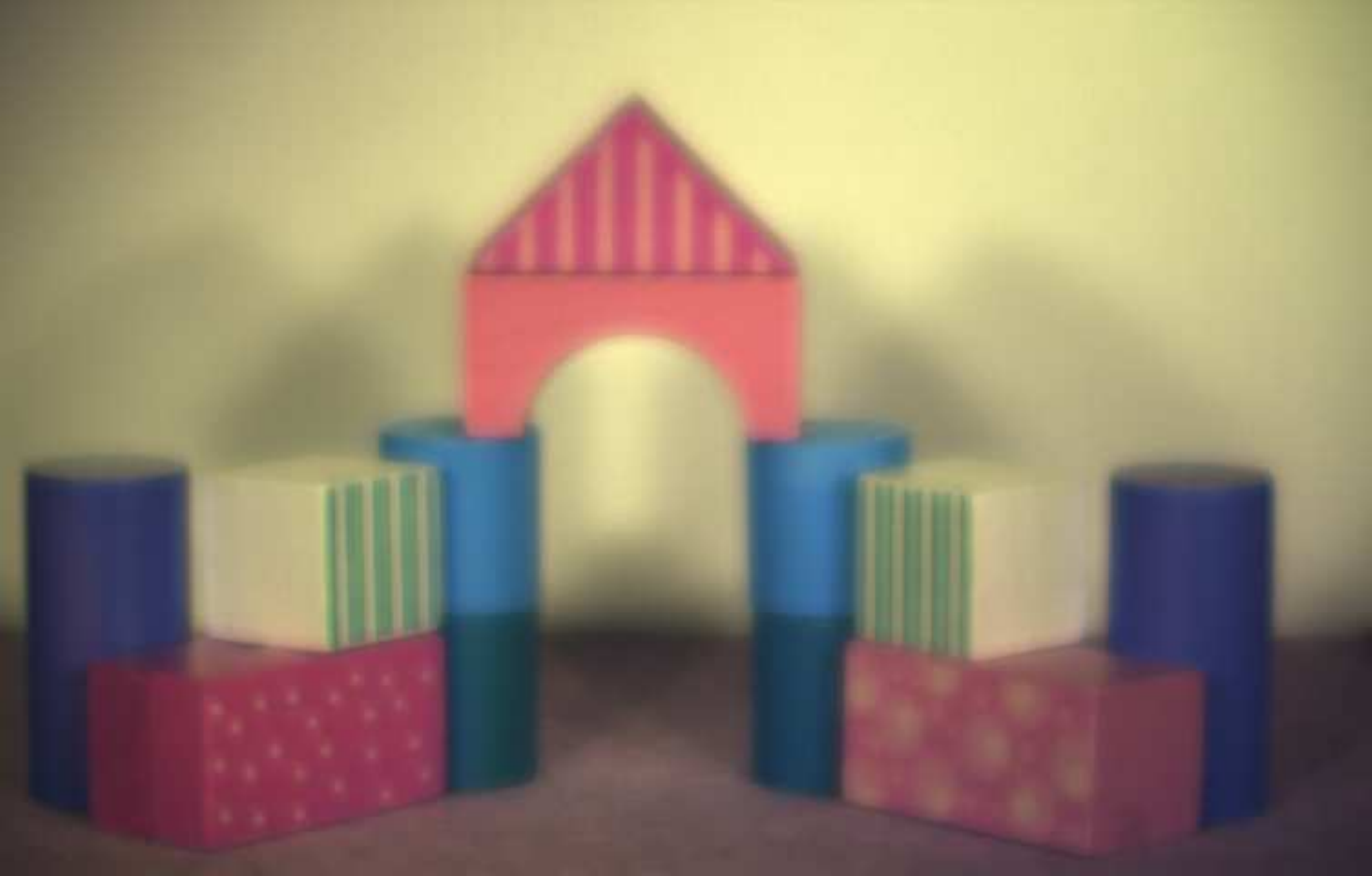}}
\subfigure[]{\includegraphics[trim = 0pt 100pt 0pt 0pt, width=0.30\linewidth]{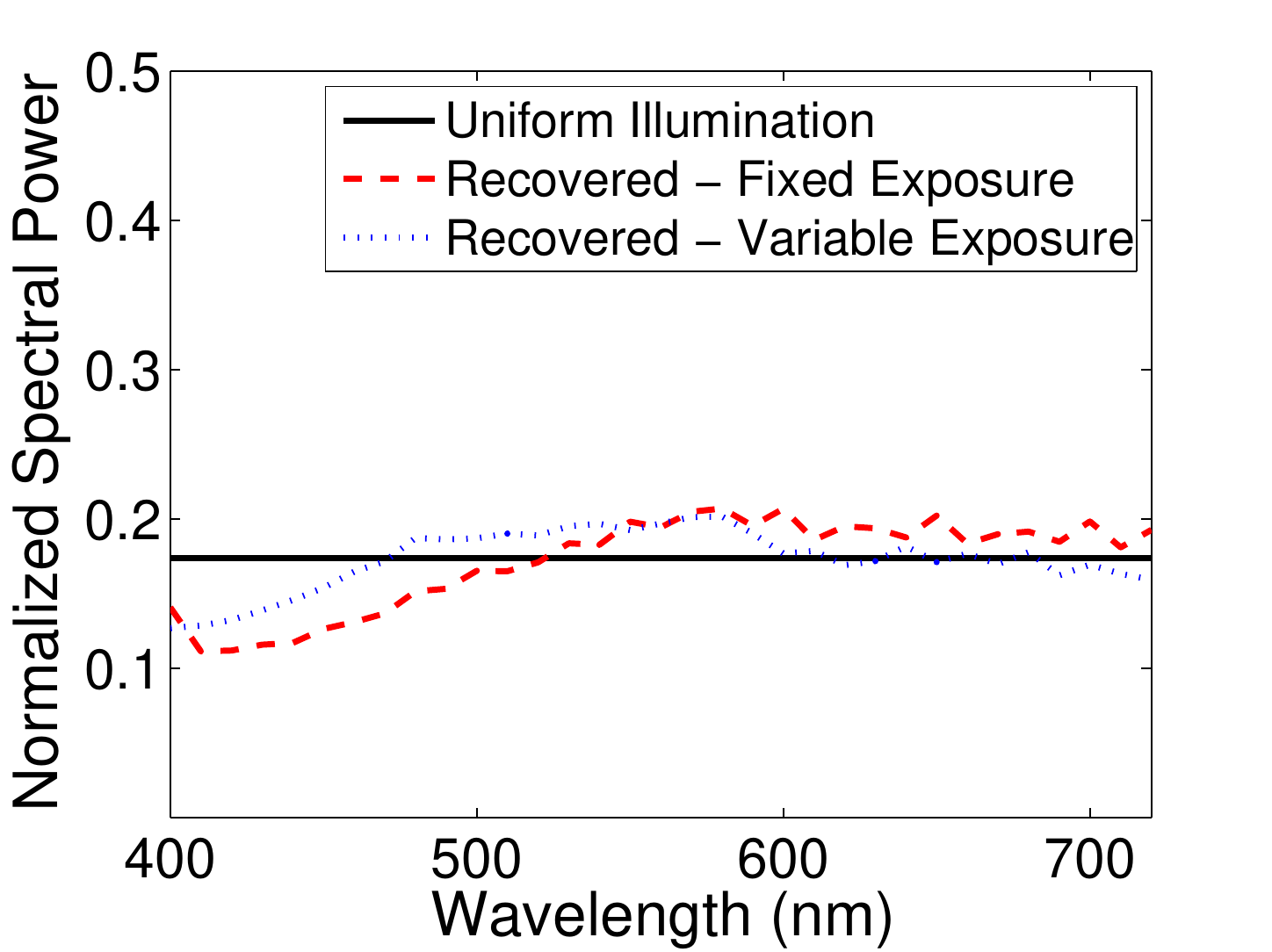}}\\ [4pt]
\subfigure[Orig~$44.5^{\circ}$]{\includegraphics[width=0.168\linewidth]{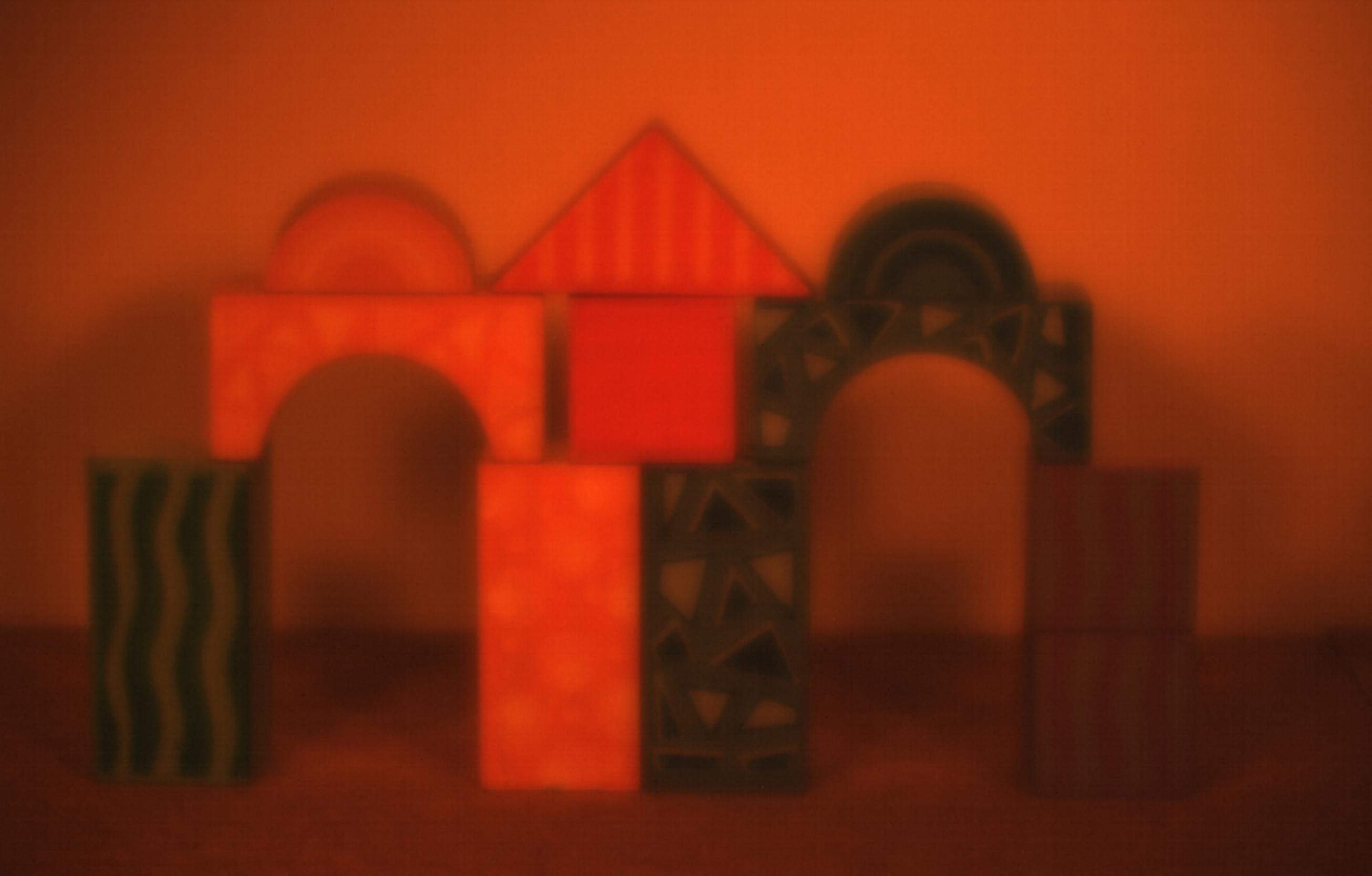}}
\subfigure[Ideal~$0^{\circ}$]{\includegraphics[width=0.168\linewidth]{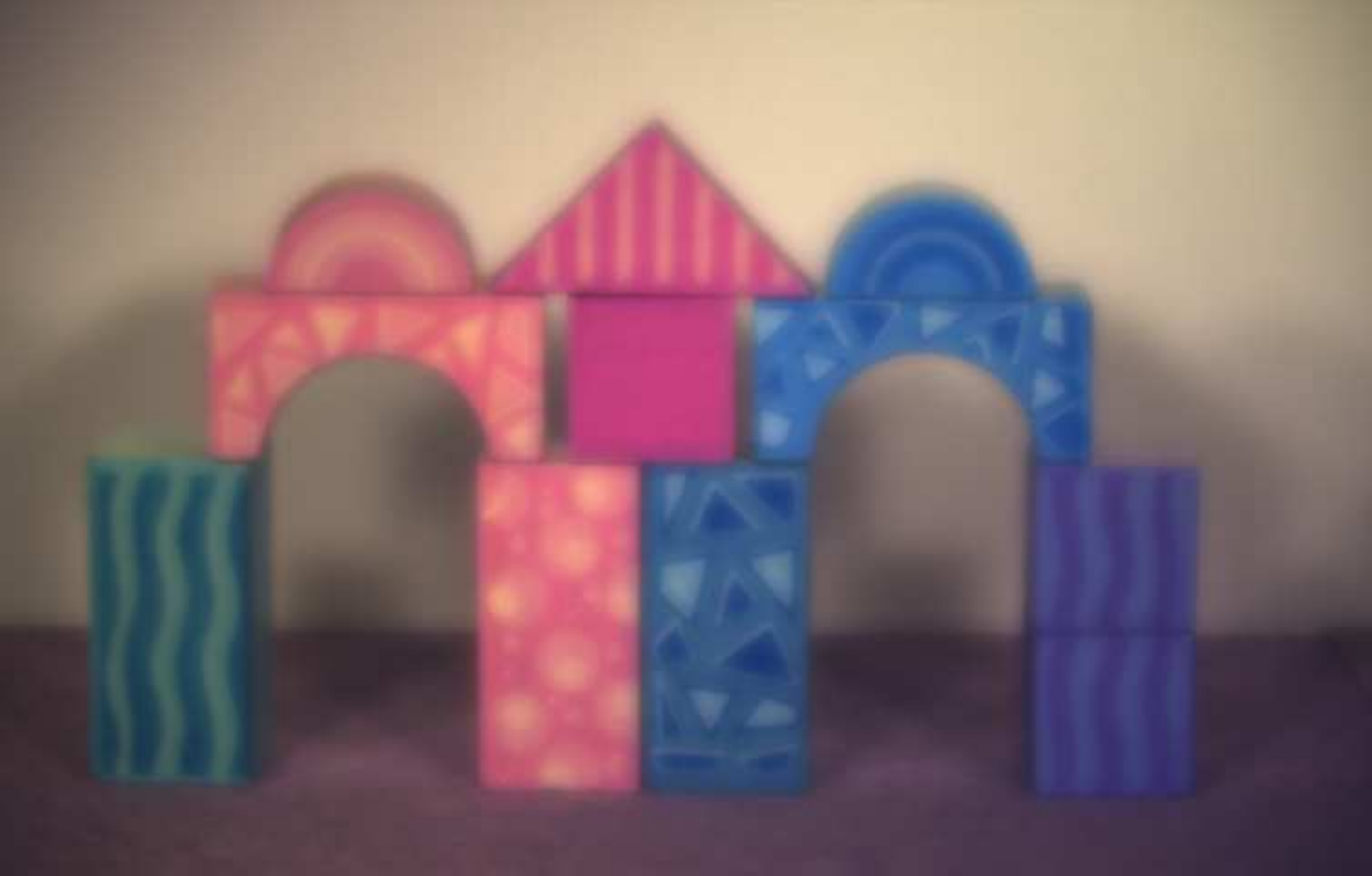}}
\subfigure[GE2-f~$33.5^{\circ}$]{\includegraphics[width=0.168\linewidth]{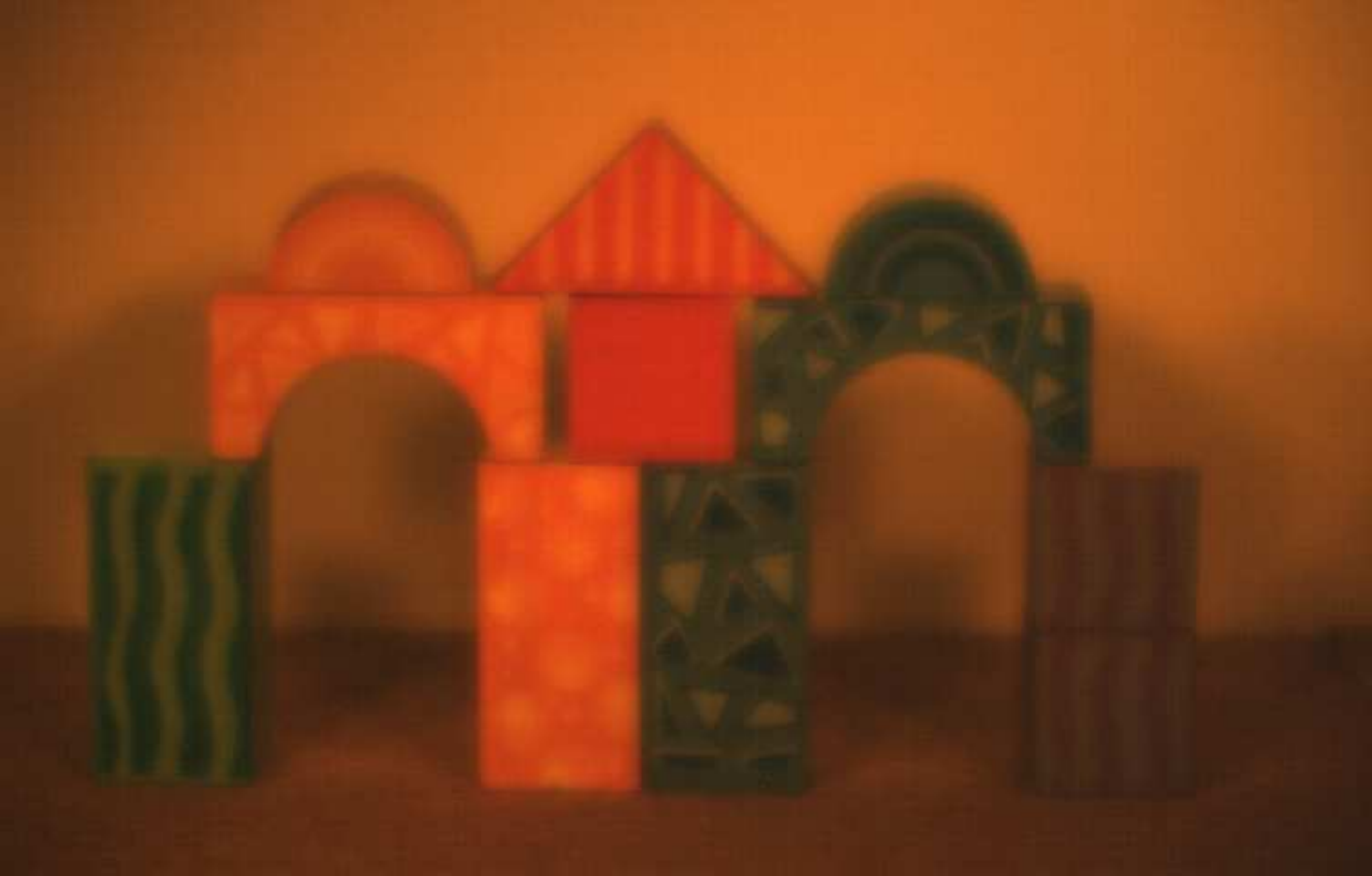}}
\subfigure[GE2-v~$8.4^{\circ}$]{\includegraphics[width=0.168\linewidth]{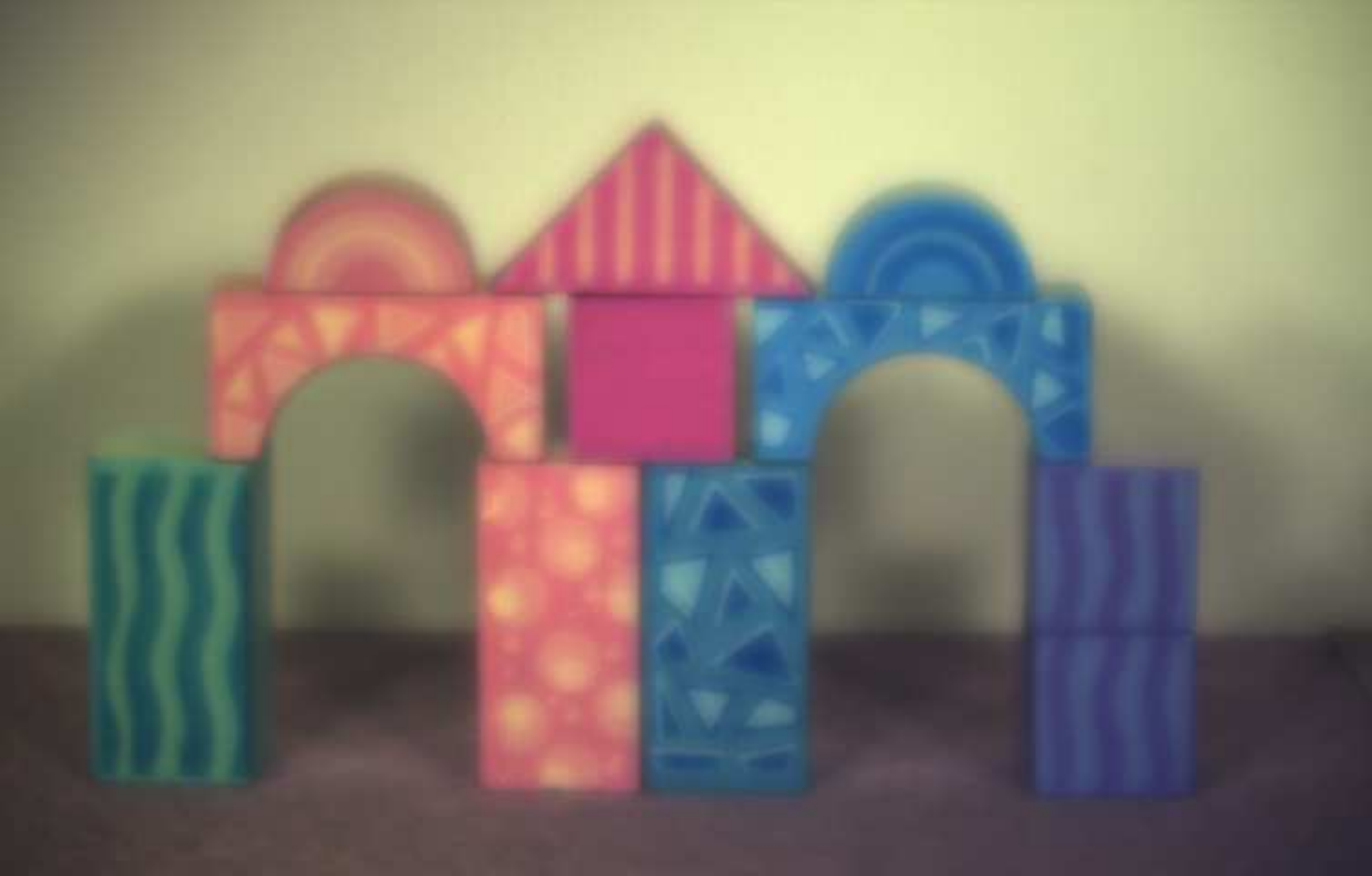}}
\subfigure[]{\includegraphics[trim = 0pt 100pt 0pt 0pt, width=0.30\linewidth]{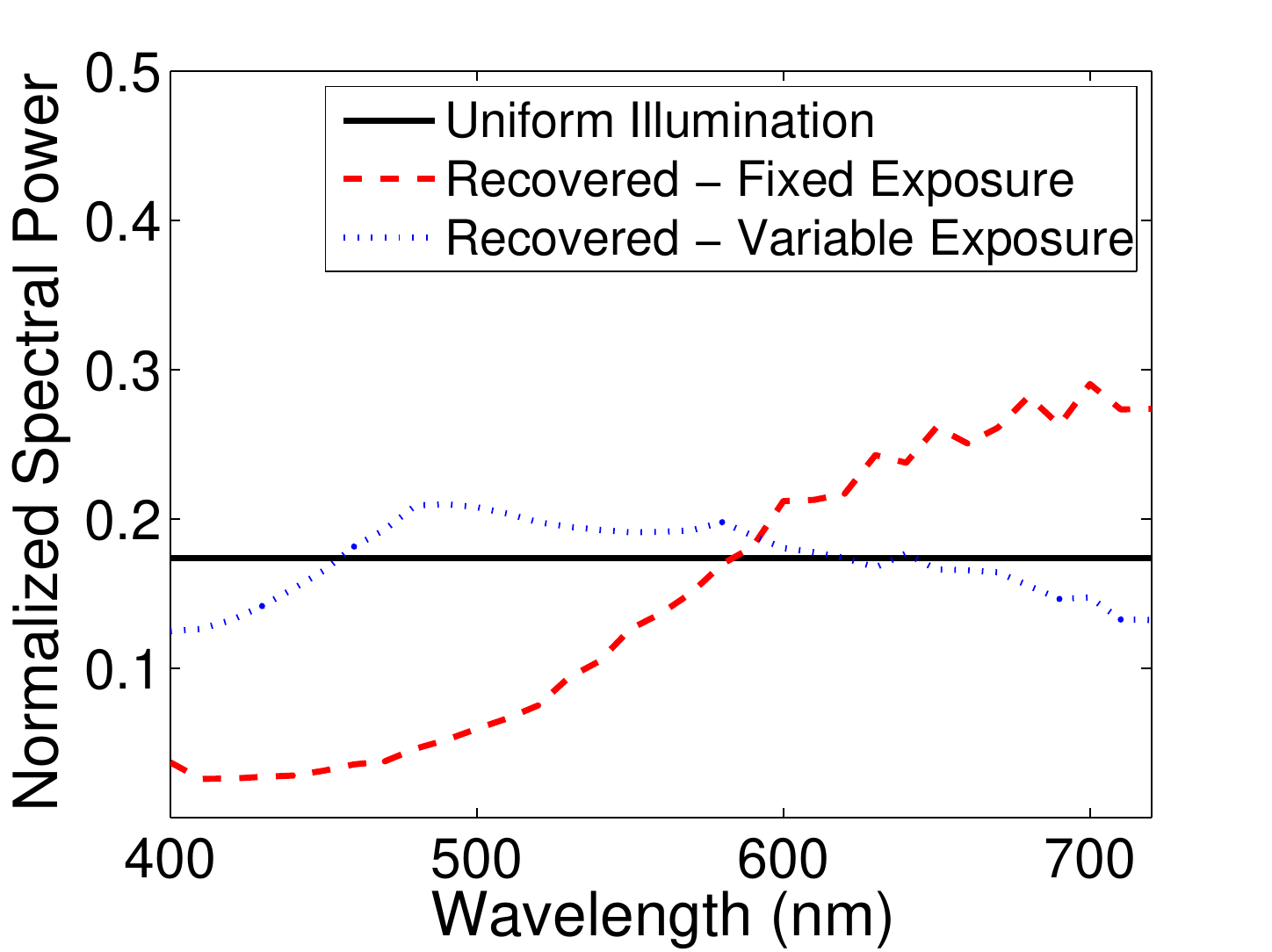}}\\ [4pt]
\subfigure[Orig~$44.5^{\circ}$]{\includegraphics[width=0.168\linewidth]{chapter_3/001_illum_05_orig_const}}
\subfigure[Ideal~$0^{\circ}$]{\includegraphics[width=0.168\linewidth]{chapter_3/001_illum_05_ideal_auto}}
\subfigure[WP-f~$12.8^{\circ}$]{\includegraphics[width=0.168\linewidth]{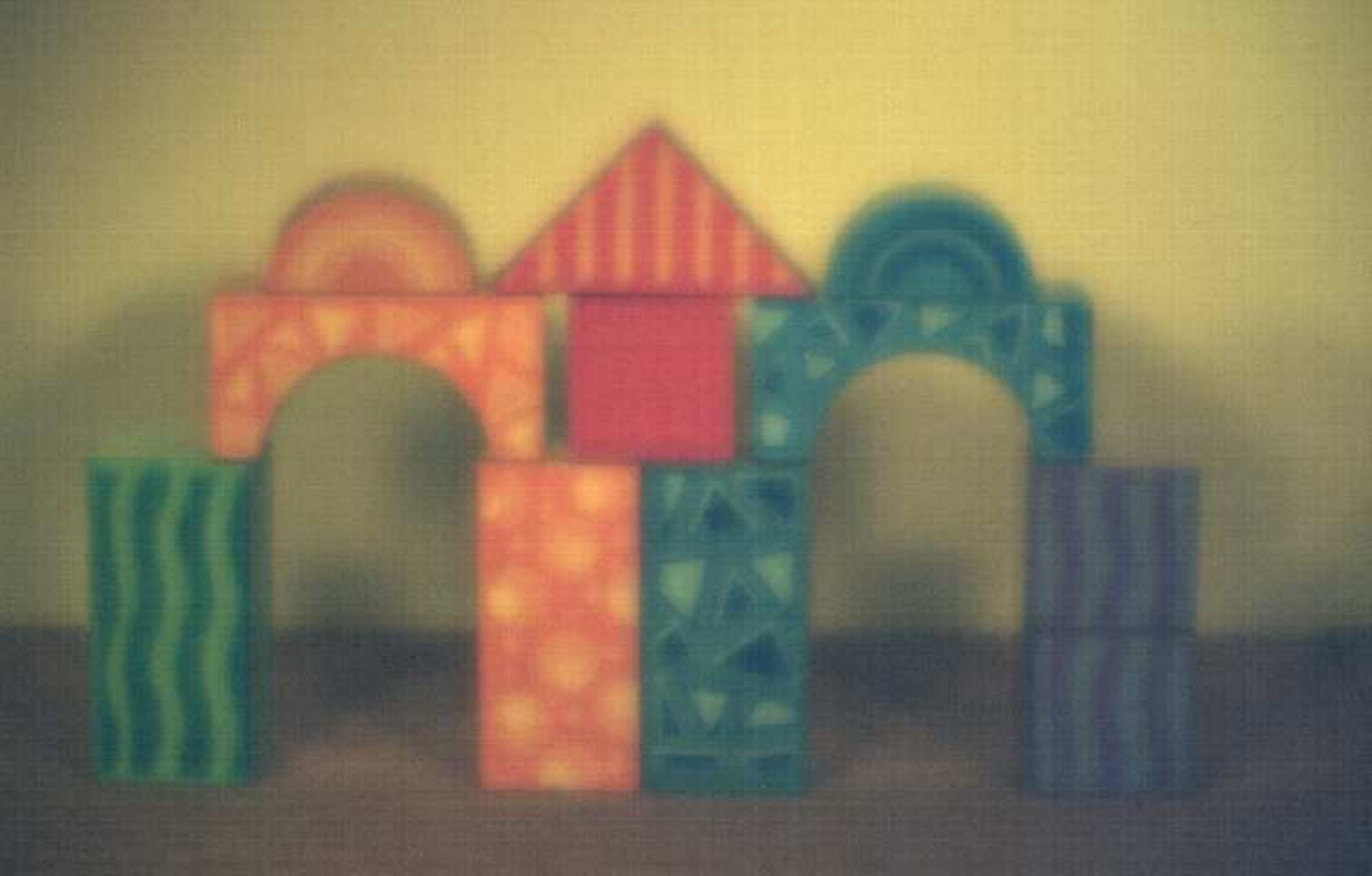}}
\subfigure[WP-v~$4.5^{\circ}$]{\includegraphics[width=0.168\linewidth]{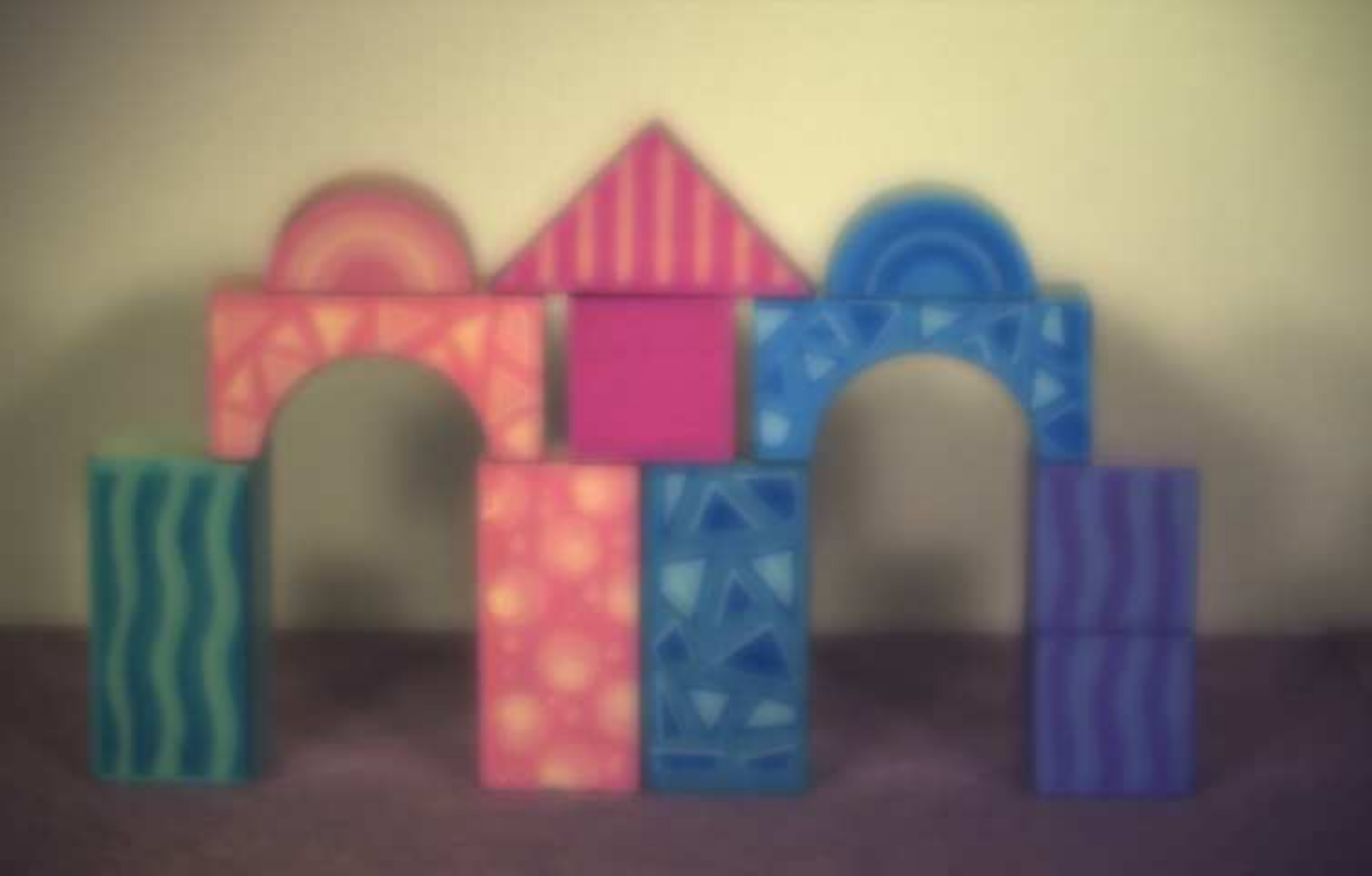}}
\subfigure[]{\includegraphics[trim = 0pt 100pt 0pt 0pt, width=0.30\linewidth]{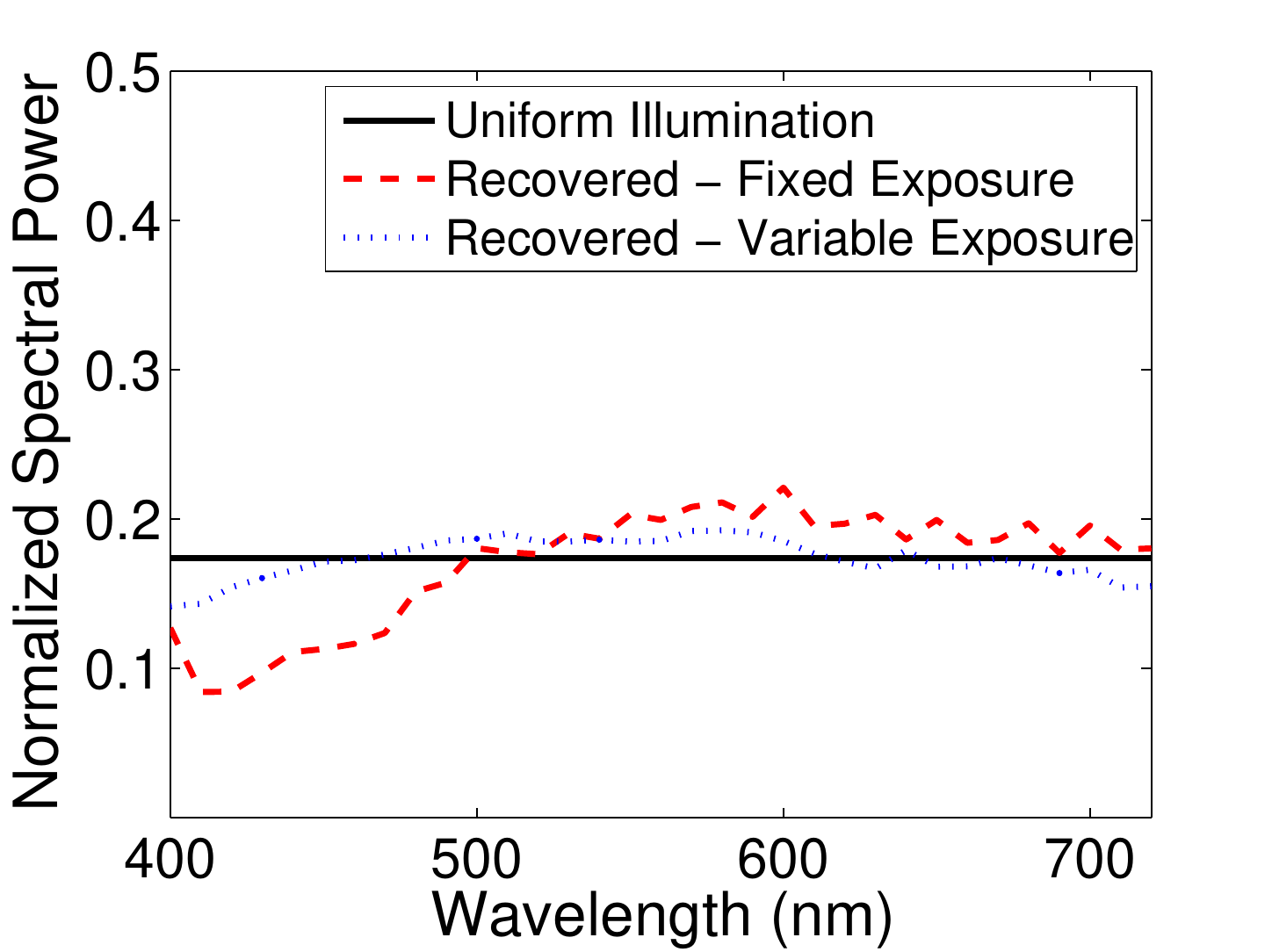}}
\caption[Qualitative comparison of fixed and variable exposure imaging]{(left to right) Original image with illumination bias (fixed exposure), ideal recovery based on ground truth, recovery with color constancy for fixed (-f) and variable (-v) exposure, and SPD of uniform and recovered illuminations.}
\label{fig:qual-const-auto}
\end{figure}

Figure~\ref{fig:qual-const-auto} shows sample scenes and their color constancy corrected images based on fixed and variable exposure.
We select the edge based color constancy algorithms (GE1 and GE2), which demonstrate highest Relative MAE improvement from fixed to variable exposure imaging. It can be observed that the variable exposure images are closer to the uniform illuminant after recovery by the same color constancy methods. The marked difference in image quality can also be observed in the corrected images.

\section{Conclusion}
\label{sec:conclusion}

We proposed a method for accurate recovery of spectral reflectance from an LCTF based hyperspectral imaging system. We investigated color constancy for illuminant estimation and proposed an adaptive illumination estimation technique, exploiting the properties of hyperspectral images. We also proposed an automatic exposure adjustment technique for compensating the bias of various optical factors involved in an LCTF based hyperspectral imaging system. Experiments were performed on an in house developed and a publicly available database of a variety of objects in simulated and real illumination conditions. It was observed that the identification of the illuminant a priori, is particularly useful for estimating illuminant sources with a smooth spectral power distribution. Our findings also suggest that automatic exposure adjustment based imaging followed by color constancy improves spectral reflectance recovery under different illuminations.

\chapter[Cross Spectral Registration of Hyperspectral Face Images]{Cross Spectral Registration\\of Hyperspectral Face Images} 

\label{Chapter4} 


A hyperspectral image contains multiple contiguous bands of a scene in narrow wavelength sections. Hyperspectral images are captured by either line scan or area scan sensors. Line scan spectral images are acquired by moving a line scanner across a scene to  sequentially capture each line of pixels. Area scan spectral imaging systems filter the incoming light in sections of the spectrum and acquire image of a scene. In line scan spectral imaging, each consecutive spectral line may be misaligned due to the non-uniform movement of the sensor. In a similar manner, in area scan spectral imaging, each consecutive captured band may be misaligned due to the movement of the scene. The misalignments may be of different nature depending on the nature of objects in a scene.

\begin{itemize}
\item Rigid: objects are of a definite shape, such that the distance between any two points remains the same under the influence of a force.
\item Non-Rigid: objects bear an indefinite shape such and are flexible when subjected to a force.
\end{itemize}

Cross spectral registration is a complex task in that each band is a different modality and a direct correspondence between the pixel intensity values cannot be made. This is because any given material has a different response to the incident light in each band of the electromagnetic spectrum. Thus, a major challenge in registration of spectral bands is to deal with the cross spectral differences. This problem can be regarded as a subset of heterogenous image registration where one image is in a different modality from the other image.

In this chapter, we focus on cross-spectral alignment of hyperspectral images of human faces. In our approach, we extract self similarity based features individually from local regions of the face. Self similarity features are obtained by correlating a small image patch within its larger neighborhood and therefore, remain relatively invariant to the cross spectral variation. For example, in Figure~\ref{fig:teaser}, the self correlation between the inner and outer patches of one band is more similar to the correlation between the inner and outer patches in the other band compared to a mere image difference. The proposed \emph{Cross Spectral Similarity (CSS)} descriptor implicitly reduces the cross spectral distance by using the notion of self similarity.

\begin{figure}[t]
\centering
\includegraphics[width=0.5\linewidth]{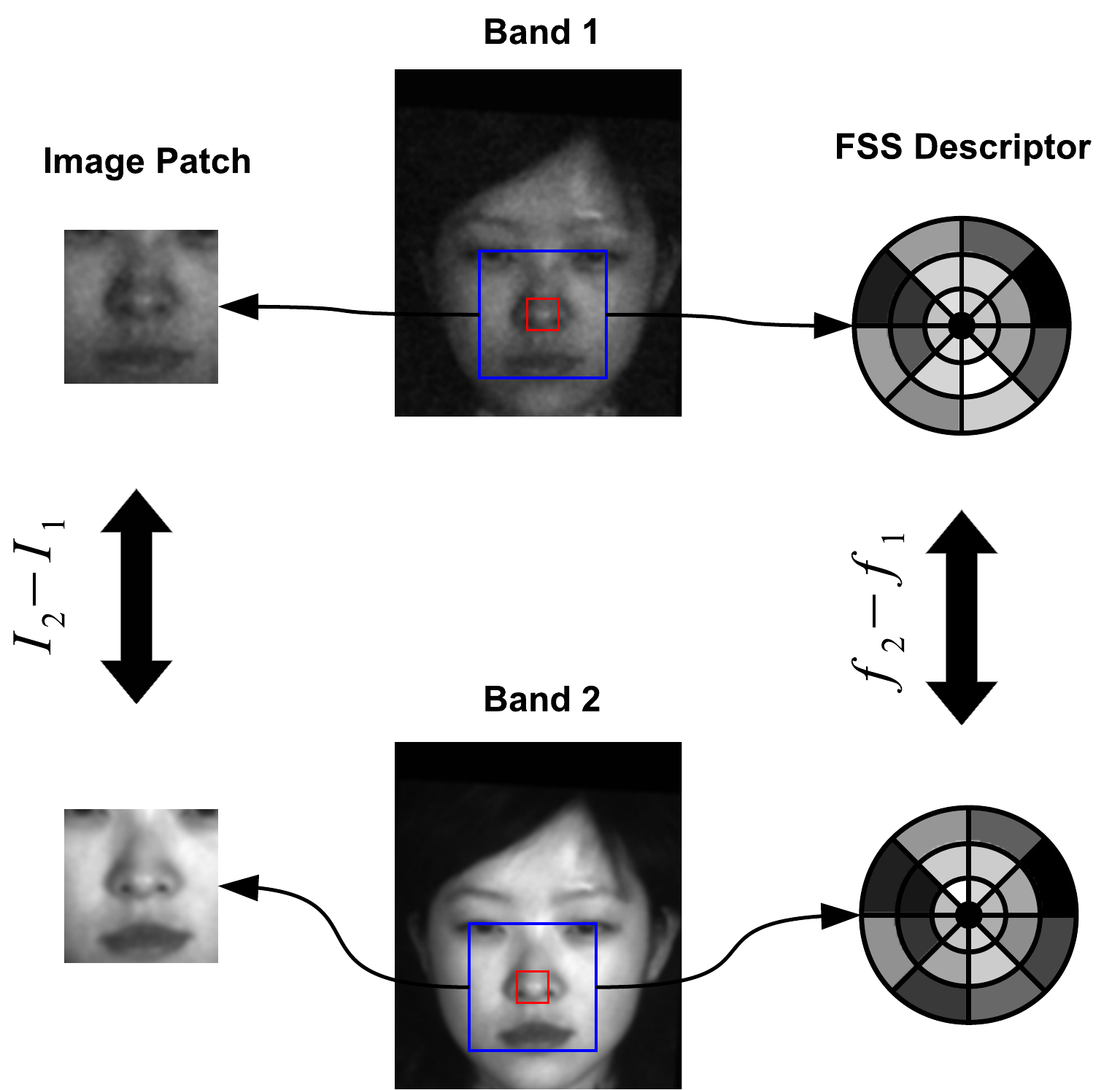}
\caption[Hyperspectral image spectral differences]{Different bands of a hyperspectral image significantly vary from each other. High photometric variation can be observed between the selected bands of a hyperspectral face image. Despite the spectral variation, the proposed \emph{Cross Spectral Similarity (CSS)} descriptor is similar at corresponding locations.}
\label{fig:teaser}
\end{figure}



Image registration has immense potential in hyperspectral image analysis~\cite{zitova2003image,maintz1998survey}. In image registration, a source image is registered to a target image. The source image is related to the target image by a transformation subject to some deformation constraints outlined by the registration technique. The registration outcome is dependent on the type of transformation under consideration. Affine transformations are restricted to translation, scale, rotation, shear and perspective between the target and source images. When such transformations are limited to rotation and translation, it is called rigid transformation. Non-rigid transformations can be arbitrary and are either a form of locally linear transformation or elastic deformation.

Phase correlation was used by Erives and Fitzgerald for hyperspectral image registration captured by a liquid crystal tunable filter~\cite{erives2006automatic}. They enhanced phase correlation for subpixel correspondence in cross spectral registration. Their results showed a 9.5\% improvement in normalized correlation compared to only phase correlation. They also extended the technique to cater for nonrigid misalignments by introducing localized phase correlation measure and geometric transformations~\cite{erives2007automated}. Zhao et al.~proposed an optical configuration which simultaneously captured diffracted and non-diffracted beams of an acousto optical tunable filter~\cite{zhao2013development}. The use of non-diffracted beam allows measurement of motion in different bands and improves the registration accuracy of spectral images.

Stone and Wolpov proposed a nonlinear prefiltering and thresholding technique that enhances the cross spectral correlation given there are significant similarities across the spectra~\cite{stone2002blind}. Their algorithm could correctly register 90\% of the images with a 10\% false positive rate. Fang et al.~presented an elastic registration method based on mutual information criteria and thin-plate spline interpolation~\cite{fang2012multi}. They experimented with different block sizes which controlled the allowable extent of deformation. An optimal block size resulted in the maximum mutual information and the highest registration accuracy in cross spectral registration.

One of the areas which is not well explored in cross spectral registration is the potential of feature descriptors. Most of the registration techniques either directly use image pixel based distance measure, correlation~\cite{pratt1974correlation}, or mutual information criteria~\cite{pluim2003mutual} for computation of registration error between two images. These criteria are reasonable for intensity images, however, spectral images are prone to photometric noise which makes the use of intensity images less effective. Local features provide a compact representation compared to image pixels and allow for efficient image registration. Moreover, they offer robustness to noise which is a major issue in hyperspectral images.

In our approach, we propose the \emph{Cross Spectral Similarity (CSS)} feature which is robust to the spectral differences between consecutive bands. It should be noted that our notion of self similarity is different to that of Shectman and Irani's~\cite{shechtman2007matching} which was mainly used for the task of image retrieval and associated applications.

%

\section{Proposed Method}

\subsection{Preprocessing}

A visual observation of the spectral variation between the different bands of a hyperspectral image suggests significant photometric variation. Moreover, due to the low throughput of filter in a spectral imaging system, some bands are affected with image noise. In order to reduce this noise, we consider spectral image smoothing using a Gaussian filter. The filtered image $\mathbf{D}$ is obtained by convolving the input image $\mathbf{I}$ with a Gaussian filter.
\begin{equation}
\mathbf{D}(x,y) = \mathbf{G}(x,y,\sigma) * \mathbf{I}(x,y)~,
\end{equation}
where $\sigma$ is the standard deviation of the Gaussian filter.

\subsection{Cross Spectral Similarity (CSS) Descriptor}
\label{sec:CSS}

The procedure to compute the CSS descriptor comprises of two main steps. The first step is to compute a self similarity surface at a location. The second step is to convert the similarity surface into a polar histogram. A detailed description of the proposed CSS descriptor computation is presented as follows.

\subsubsection{Self Similarity Surface Computation}

A uniform rectangular grid with a spacing of $\delta$ pixels is overlayed on the image. At each grid location $(x,y)$, a square window $\mathbf{A} \in \mathbb{R}^{w \times w}$ is sampled from the filtered image $\mathbf{D}$ as shown in Figure~\ref{fig:CSS-misalign}. Subsequently, a square patch $\mathbf{V} \in \mathbb{R}^{p \times p}, p<w$ is sampled from the center of $\mathbf{A}$. The correlation between $\mathbf{A}$ and $\mathbf{V}$ can be computed by various forms of correlation functions (sum of absolute differences, normalized cross correlation, sum of squared distances) which capture different order of the similarity. The \emph{sum of squared differences} (SSD) accomplishes the task efficiently and captures sufficient fidelity of correlation~\cite{shechtman2007matching}. The sum of squared differences $\mathbf{S}\in\mathbb{R}^{w\times w}$ between two images is computed as
\begin{equation}
\mathbf{S}(x,y) = \sum_{i=-\frac{p}{2}}^{\frac{p}{2}}\sum_{j=-\frac{p}{2}}^{\frac{p}{2}} (\mathbf{V}(i,j)-\mathbf{A}(i+x,j+y))^2~,
\end{equation}
which is a distance surface, however, we are interested in the computation of a similarity surface. Therefore, $\mathbf{S}$ is transformed into a correlation surface $\mathbf{C} \in \mathbb{R}^{w \times w}$ as
\begin{equation}
\mathbf{C}(x,y) = \exp \left(-\frac{\mathbf{S}(x,y)}{\max(\sigma_n,\sigma_a)} \right)~,
\end{equation}
where $\sigma_n$ is the estimated photometric noise variance which is an estimate of the average noise in the image. The parameter $\sigma_a$ is the local variance computed from the local window $\mathbf{A}$. Due to the spectral intensity variation between the consecutive bands, the correlations of similar regions may be differently scaled. Thus, $\mathbf{C}(x,y)$ is scaled by dividing with its Frobenius norm.
\begin{equation}
\hat{\mathbf{C}} = \frac{\mathbf{C}}{\|\mathbf{C}\|}_F~.
\end{equation}

\begin{figure}[!h]
\centering
\includegraphics[width=1\linewidth]{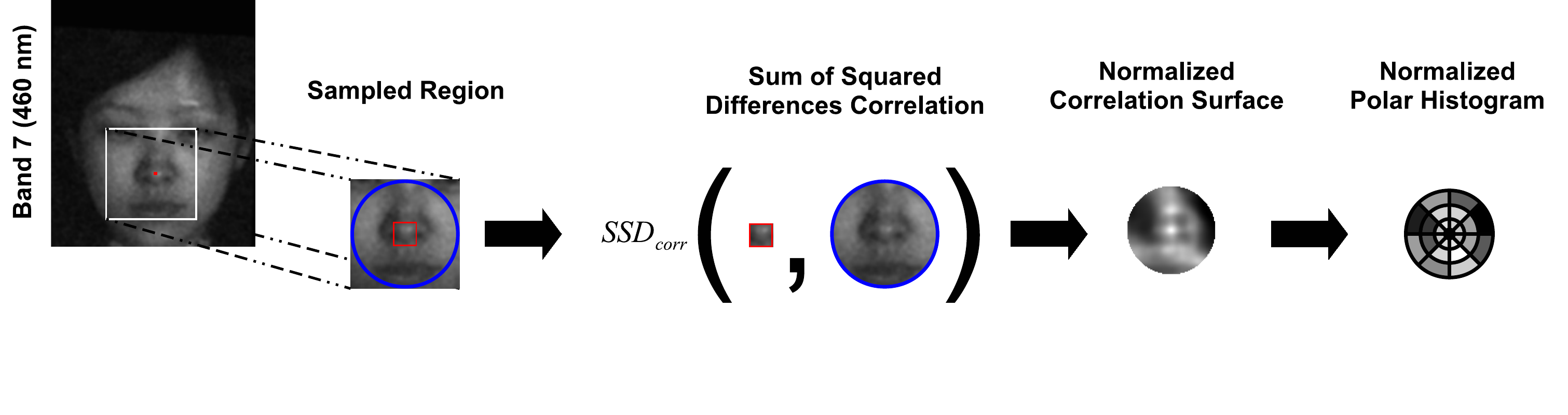}
\includegraphics[width=1\linewidth]{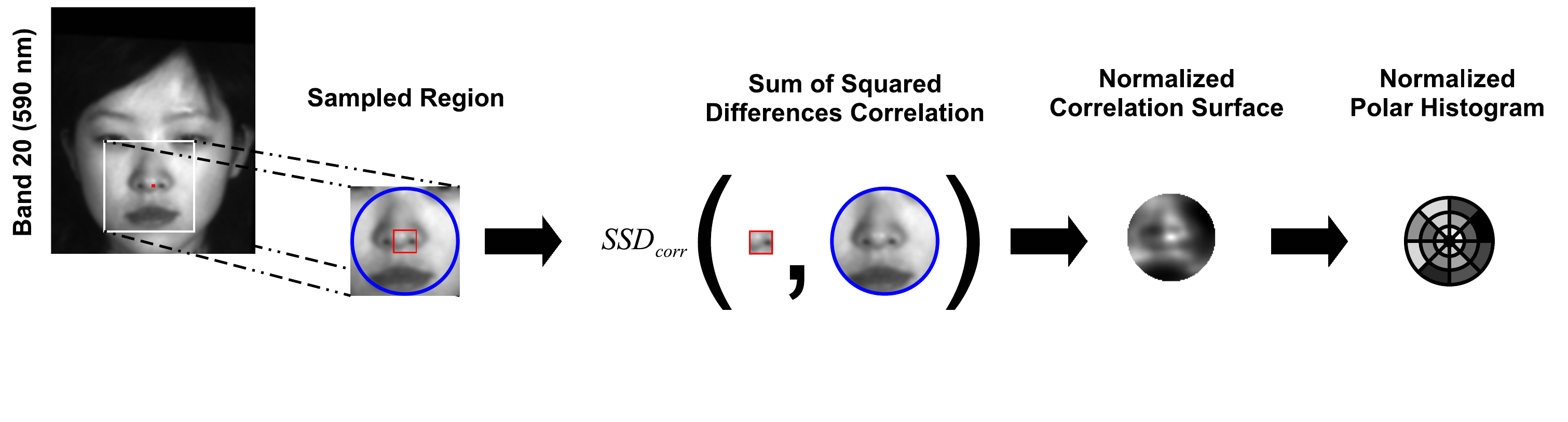}
\includegraphics[width=1\linewidth]{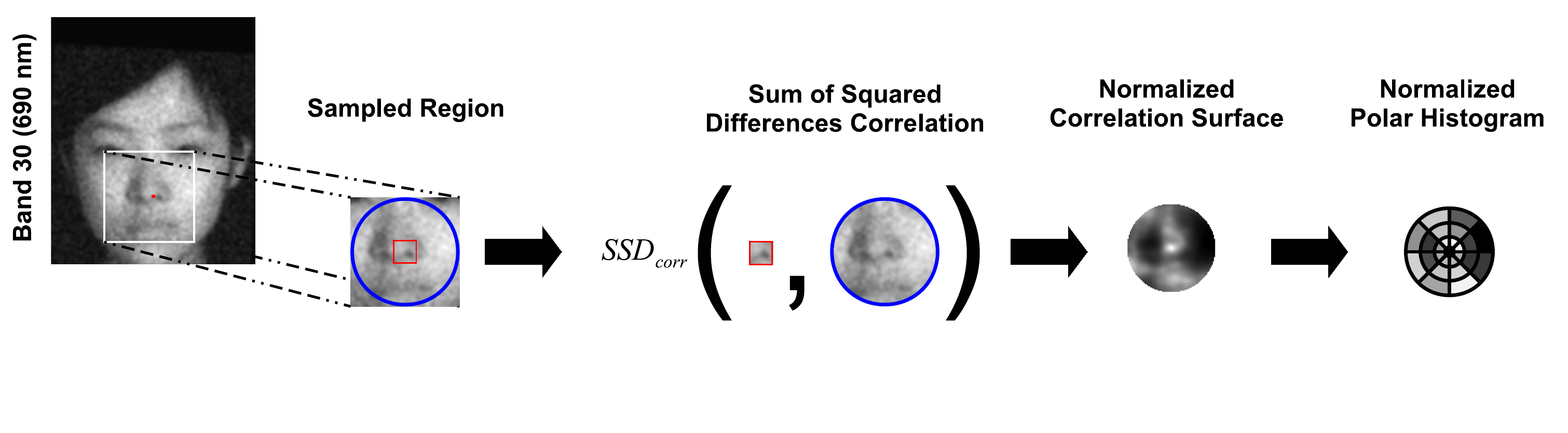}
\caption[CSS descriptor computation at different facial locations]{An example illustration of the CSS descriptor computation at the same absolute image locations of face in three different bands of a hyperspectral image. Notice that the extracted CSS descriptor in Bands 20 and 30 is different from that of Band 7 because the bands are slightly misaligned and the descriptors are extracted from different location. An inner patch (red) is correlated with the outer window (blue) in a circular vicinity. The resulting correlation surface is normalized and divided into $\theta$ angular and $\rho$ radial intervals (here, $\theta=8$ and $\rho=3$). Each bin value in the final descriptor is the average of the correlation pixels in that bin. The descriptor is normalized via min-max rule.}
\label{fig:CSS-misalign}
\end{figure}

\begin{figure}[!h]
\centering
\includegraphics[width=1\linewidth]{chapter_4/ex2_7}
\includegraphics[width=1\linewidth]{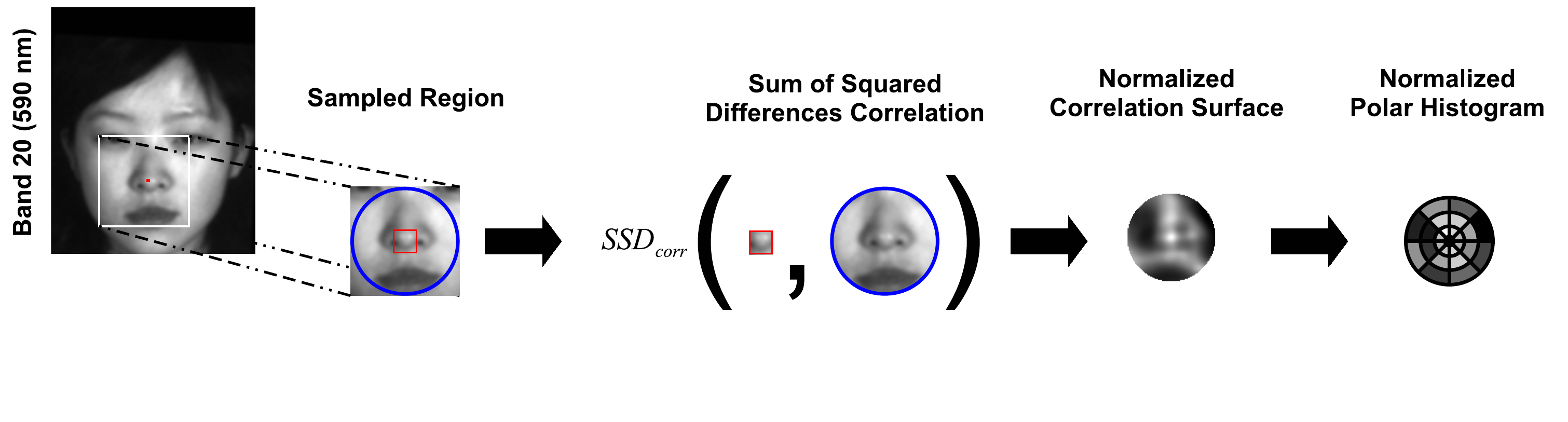}
\includegraphics[width=1\linewidth]{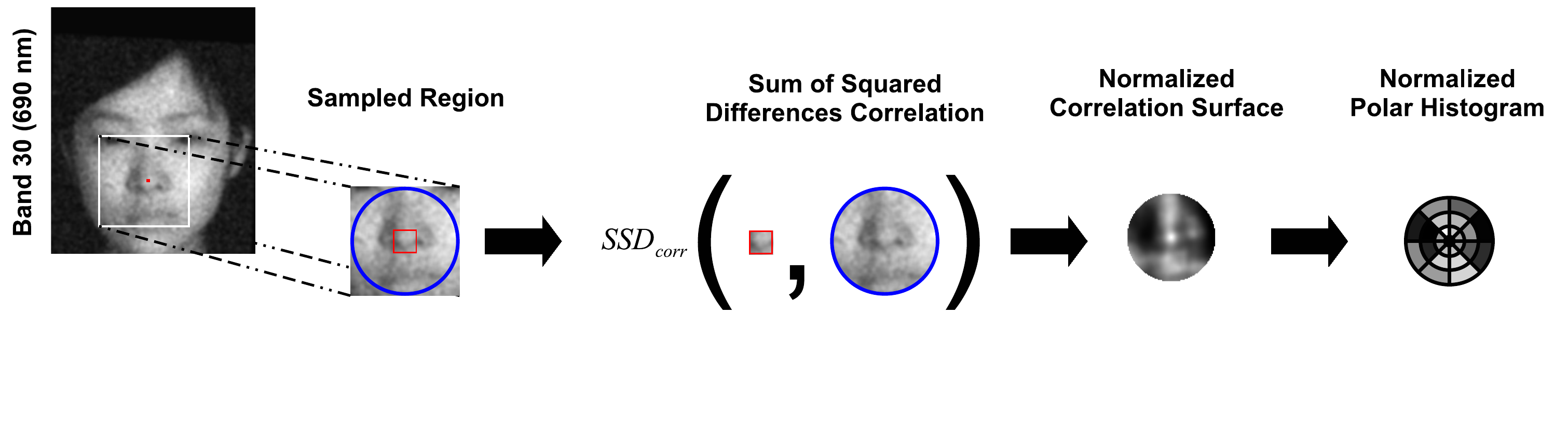}
\caption[CSS descriptor computation at same facial locations]{An illustration of the CSS descriptor computation at the same face locations in three different bands of a hyperspectral image. Observe that the final descriptors are much similar to each other.}
\label{fig:CSS-align}
\end{figure}
\subsubsection{Polar Histogram Conversion}

Registration of hyperspectral image of a face requires accurate spatial correspondence across spectral bands. Hence, we use a \verb"polar" representation to capture the spatial deformation of faces within a radial spatial distance. The similarity surface is pooled into a polar grid of $\theta$ angular and $\rho$ radial intervals as shown in Figure~\ref{fig:CSS-misalign}. The total number of descriptor bins is thus $b = \theta \times \rho$. We deviate from the original methodology of local self similarity~\cite{shechtman2007matching} due to limitations of the LSS descriptor in the scenario of face spectral image registration. We propose the usage of \verb"mean" instead of \verb"max" for pooling correlation pixels in histogram bins in order to make it sensitive to a local region instead of a single pixel (as \verb"max" would do). We use the \verb"mean" value of the correlation surface pixels falling in a bin, so as to get a vector $\mathbf{f} \in \mathbb{R}^b$,
\begin{equation}
\mathbf{f} = \frac{1}{b_i}\sum_{(x,y) \in i}\mathbf{\hat{\mathbf{C}}(x,y)}~,
\end{equation}
where $b_i$ is the number of pixels falling in the $i^\textrm{th}$ bin. The feature descriptor is then normalized in the range $[0,1]$ using the \verb"min-max" rule
\begin{equation}
\hat{\mathbf{f}} = \frac{\mathbf{f}-\min(\mathbf{f})}{\max(\mathbf{f})-\min(\mathbf{f})}~,
\end{equation}
where $\hat{\mathbf{f}}$ is a normalized CSS descriptor. It can be seen in Figure~\ref{fig:CSS-align} that the descriptors are much similar when extracted from similar facial locations in different bands.

Computation of the descriptor at all grid points gives $\mathbf{F}=\{\hat{\mathbf{f}}_1, \hat{\mathbf{f}}_2,...,\hat{\mathbf{f}}_n\}$ which is a global representation of all local CSS features.  We do not eliminate the \emph{uniform} and \emph{salient} descriptor points as in~\cite{shechtman2007matching} because all points on a face are important for registration. The \emph{uniform} points come from patches sampled from homogenous regions of a face, whereas, the \emph{salient} points exhibit low self similarity within a local neighborhood. This may provide enhanced results in objects of assorted categories but for objects of a specific class (e.g. faces) elimination of certain points is not beneficial.

\subsection{Grid Based Registration}

Grid based registration is one of the most successful technique for cross modality registration~\cite{rueckert1999nonrigid}. The registration starts with definition of an initial grid of control points over the images. A \emph{source image} is registered to a \emph{target image} by iteratively transforming the grid and computing the registration error between the target and source image. The source image is deformed using transformation matrix which locally maps the source image to the target image. The image registration error is iteratively minimized by an optimization algorithm.

In case of cross spectral registration, no band can be specifically designated as the target image because all bands are equally probable of misalignment. We follow an intuitive approach for defining the target image. Suppose $p$ bands of a spectral image need to be registered. We sequentially select the $i^\textrm{th}$ band ($i=2,3,...,p$) as the source image, and the mean of the remaining $p-i$ bands as the target image. Based on our observation, this is a relatively robust strategy as compared to sequentially registering $i^\textrm{th}$ band to its preceding $i-1^\textrm{th}$ band only. Since some spectral bands are highly affected by noise, sequential registration of those bands is erroneous. However, registration to the `mean of the remaining bands' is found robust to noise.

\subsection{CSS Matching}

Use of image self similarity has mainly been oriented towards indoor and outdoor object retrieval from real world scenes. Such objects exhibit high intra-class variations in addition to the inter-class variation. In case of cross spectral face registration all images are frontal faces in a close to neutral expression. Consequently, the inter spectral variation is expected to be low and hence, an accurate spatial correspondence is required between the descriptors sampled from the grid locations of consecutive bands.

The CSS descriptors are compared using the Euclidean distance measure as opposed to the computationally expensive \emph{ensemble matching}~\cite{shechtman2007matching} or \emph{Hough transform based voting}~\cite{chatfield2009efficient}. Given $\mathbf{f}_i,\mathbf{f}_j \in {\mathbb{R}}^b$, the descriptors for local image patches $i$ and $j$ respectively, the distance $y$ between the two descriptors is determined as
\begin{equation}
y_{ij} = \sum (\mathbf{f}_i-\mathbf{f}_j)^2~,
\end{equation}
which computes the Euclidean distance between $\mathbf{f}_i$ and $\mathbf{f}_j$.

%

\section{Experimental Results}

\subsection{Database}


The PolyU Hyperspectral Face database~\cite{di2010studies} is comprised of 151 hyperspectral image cubes of 47 individuals. The frontal faces of each subject were captured in several different sessions using a Varispec LCTF and halogen lights. Each image comprises 33 bands in 400 to 720nm range (10nm steps) with a spatial resolution of $220 \times 180$ pixels. The first 6 and last 3 bands of each hyperspectral image are discarded because of very high noise in these bands~\cite{di2010studies}. Therefore the actual hyperspectral images used in registration experiments contain 24 bands in 460 to 690nm range.

\subsection{Registration Results}

In this experiment we selected 113 images in the PolyU hyperspectral face database which had no visually noticeable misalignments across spectral bands. Now consider a raw spectral image $\mathbf{X}$ whose bands are aligned. The image $\mathbf{X}$ undergoes a simulated rigid transformation with forward transformation parameters ${X(t_{{x}_i},t_{{y}_i},\theta_i)}_i=1^p$ selected from a scaled random normal distribution to get the misaligned image $\mathbf{Z}$. The transformations are limited to a rigid (translation and rotation) as only slight variation is expected during acquisition of faces in consecutive bands. Then, the proposed registration algorithm registers the bands of $\mathbf{Z}$ and returns the reverse transformation parameters ${Y(t_{{x}_i},t_{{y}_i},\theta_i)}_i=1^p$ to get the registered image $\mathbf{Y}$. The registration error between $\mathbf{X}$ and $\mathbf{Y}$ is then computed as

\begin{align}
\label{eq:sim_errors}
e_\theta=& |X(\theta_i)-Y(\theta_i)|\\
e_r     =& \sqrt{(|X(t_{{x}_i})-Y(t_{{x}_i})|^2+|X(t_{{y}_i})-Y(t_{{y}_i})|^2)}
\end{align}
where $e_\theta$ is the rotational error and $e_r$ is the radial displacement error.

Figure~\ref{fig:reg-results-sim} shows the results of registration using the proposed method on sample image of PolyU database. Observe that the cross spectral misalignment has been significantly reduced. In Figure~\ref{fig:reg-error-sim} we present the distribution of rotation and translational variation between misaligned and aligned hyperspectral images using~\ref{eq:sim_errors}. It can be seen that the rotational errors are symmetric about $0^\circ$ and the displacement errors are a right tailed skewed distribution with a peak at 5 pixels which indicates improvement in registration errors.

\begin{landscape}
\begin{figure*}[t]
\centering
\subfigure[]{
\begin{minipage}{0.1\linewidth}{
\includegraphics[width=1\linewidth]{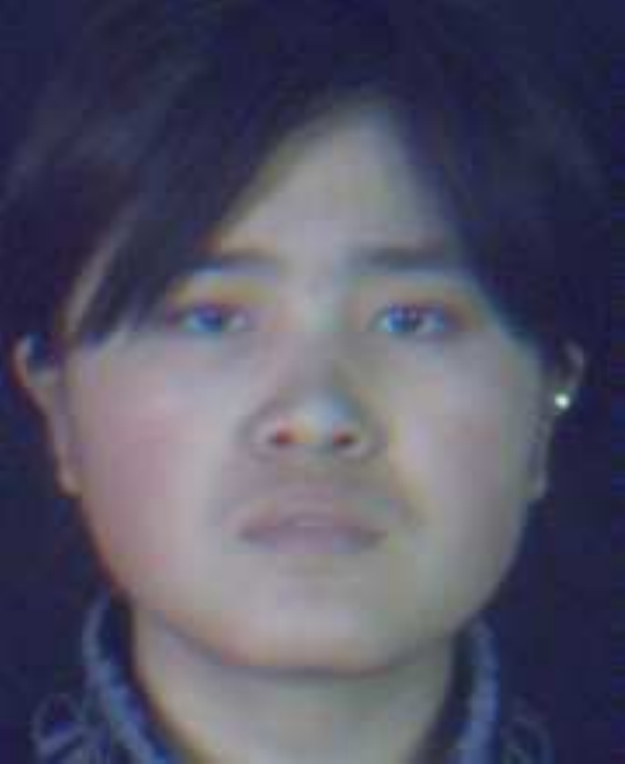}\\
\includegraphics[width=1\linewidth]{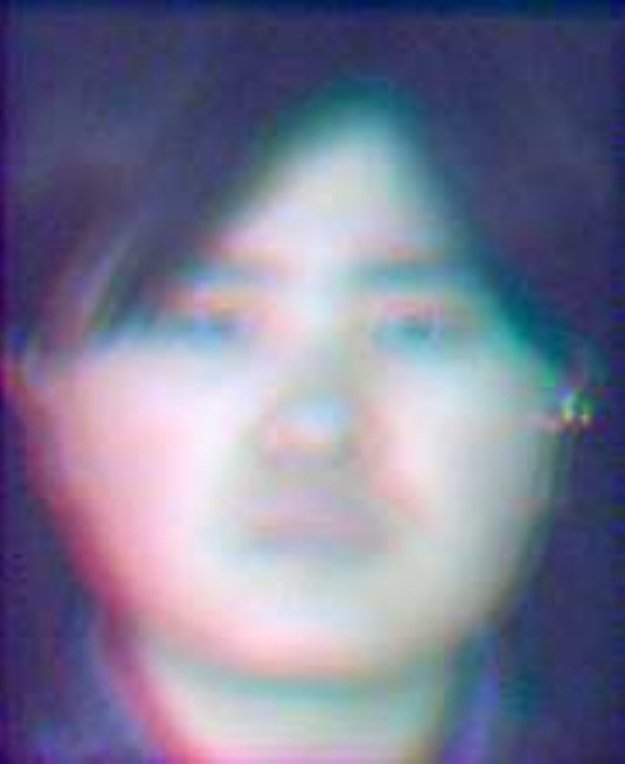}\\
\includegraphics[width=1\linewidth]{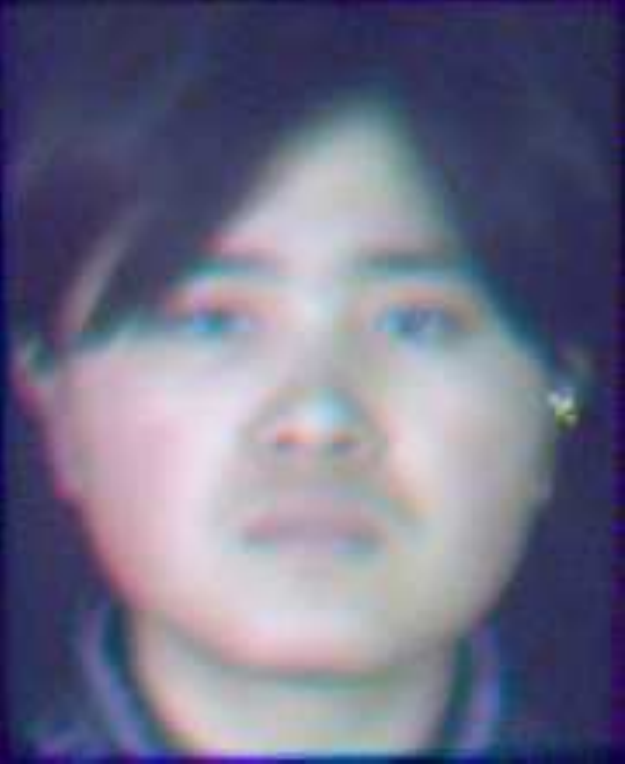}}
\end{minipage}}
\subfigure[]{
\begin{minipage}{0.1\linewidth}{
\includegraphics[width=1\linewidth]{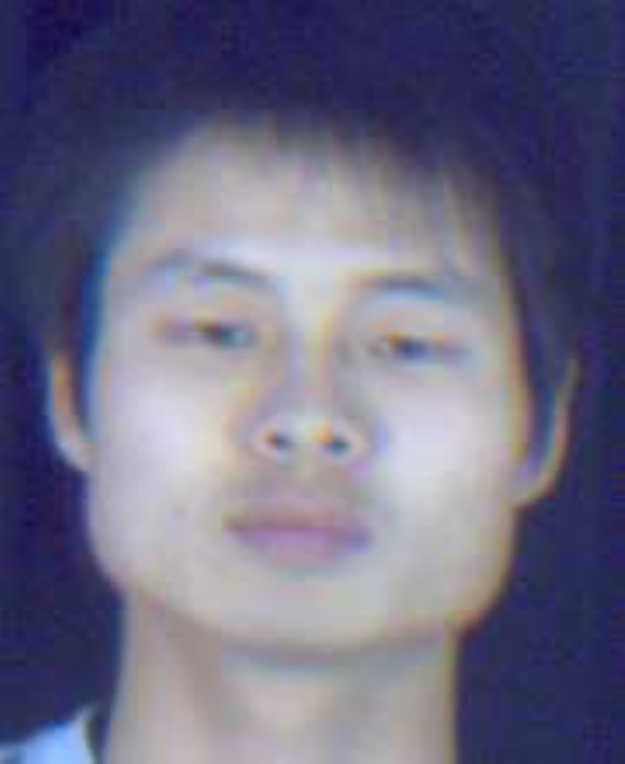}\\
\includegraphics[width=1\linewidth]{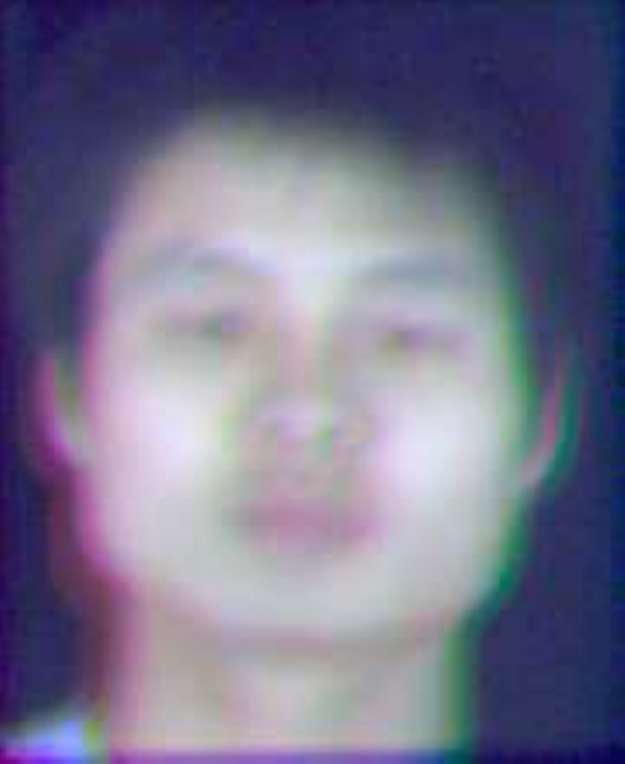}\\
\includegraphics[width=1\linewidth]{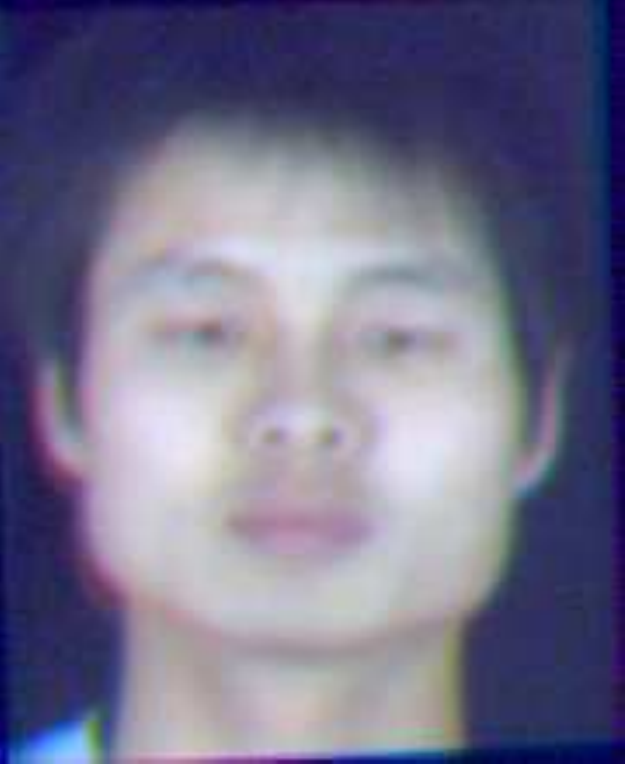}}
\end{minipage}}
\subfigure[]{
\begin{minipage}{0.1\linewidth}{
\includegraphics[width=1\linewidth]{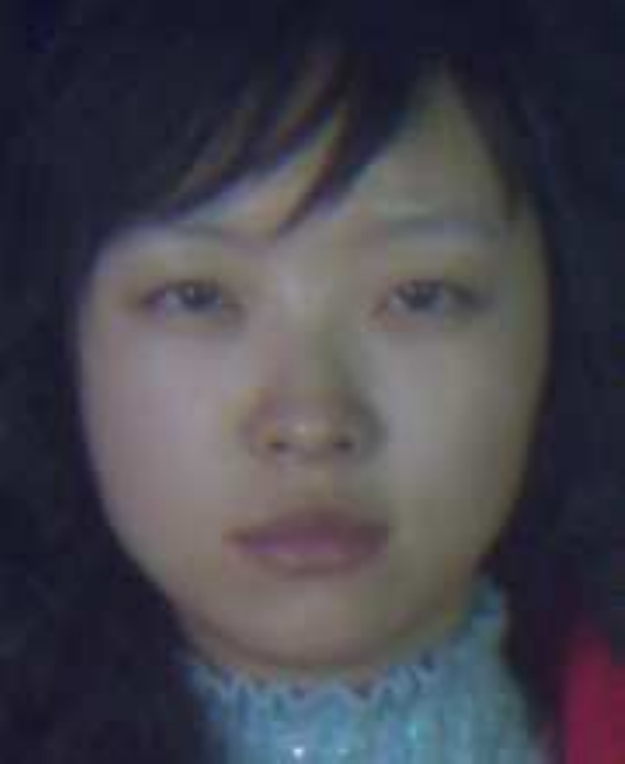}\\
\includegraphics[width=1\linewidth]{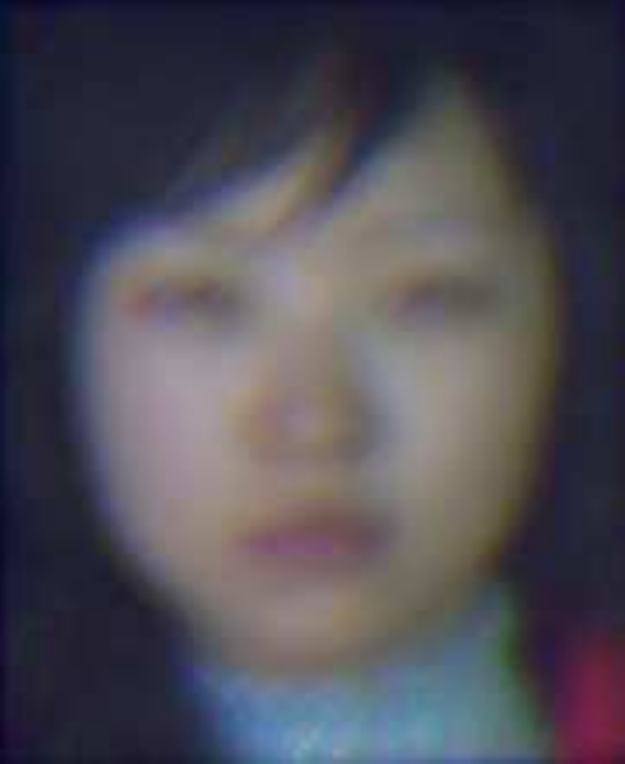}\\
\includegraphics[width=1\linewidth]{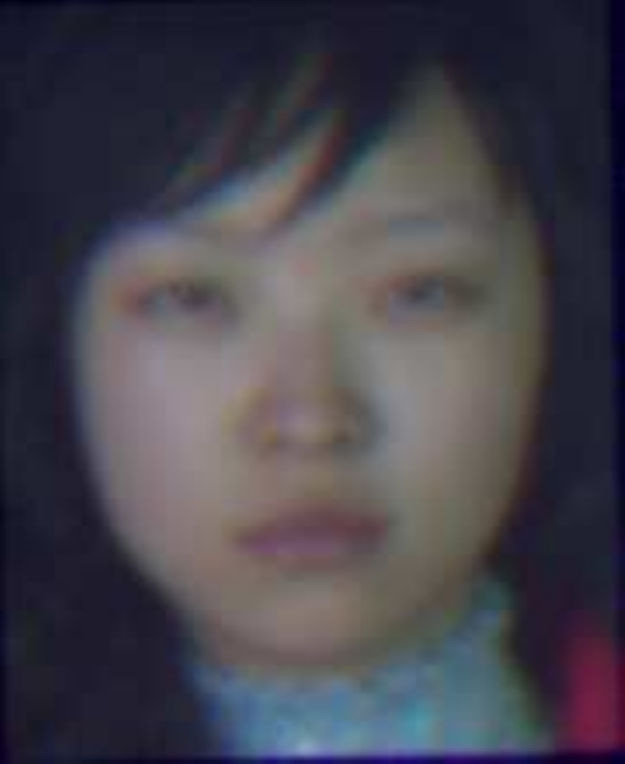}}
\end{minipage}}
\subfigure[]{
\begin{minipage}{0.1\linewidth}{
\includegraphics[width=1\linewidth]{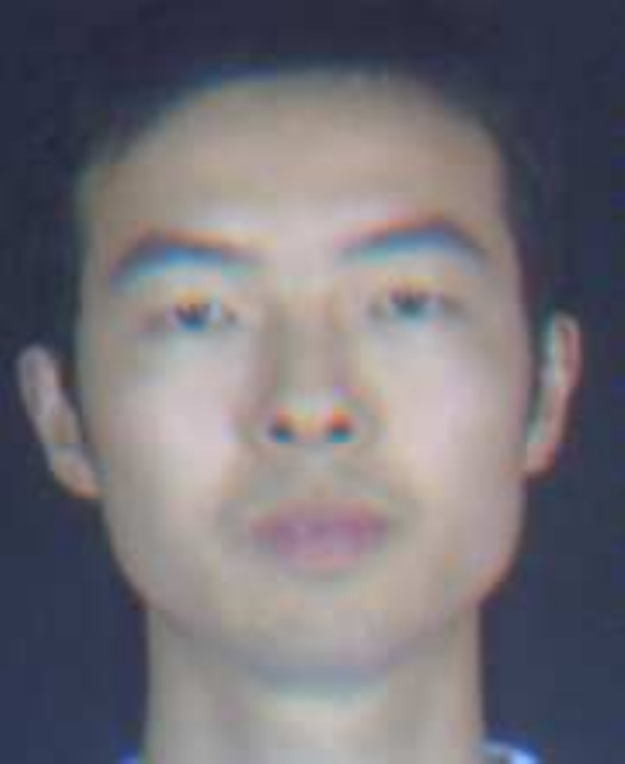}\\
\includegraphics[width=1\linewidth]{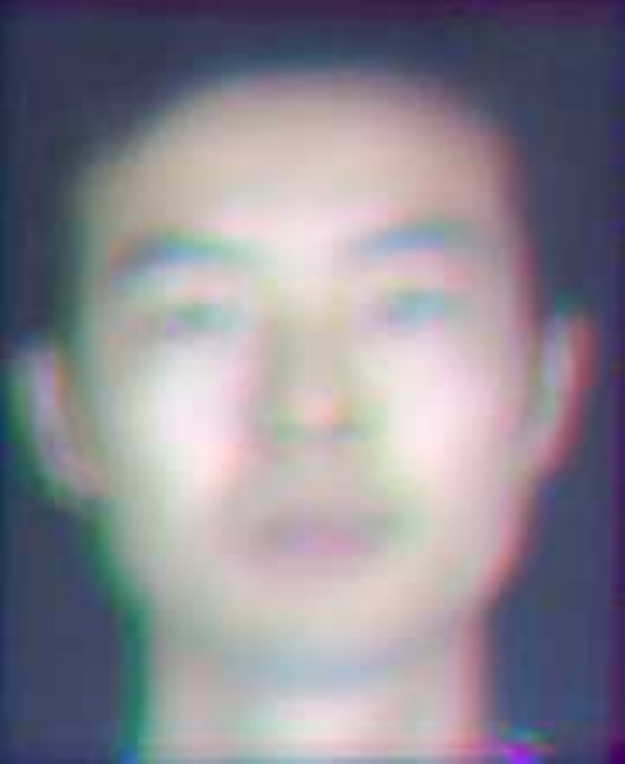}\\
\includegraphics[width=1\linewidth]{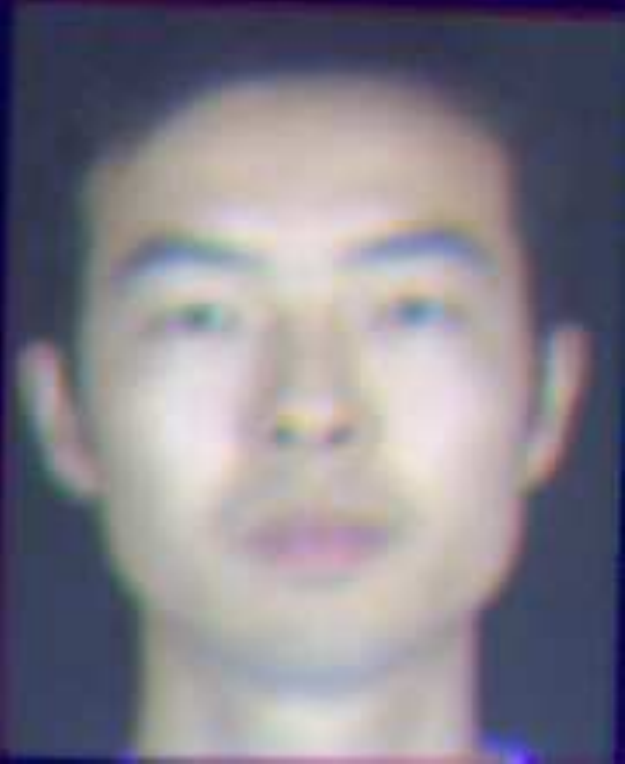}}
\end{minipage}}
\subfigure[]{
\begin{minipage}{0.1\linewidth}{
\includegraphics[width=1\linewidth]{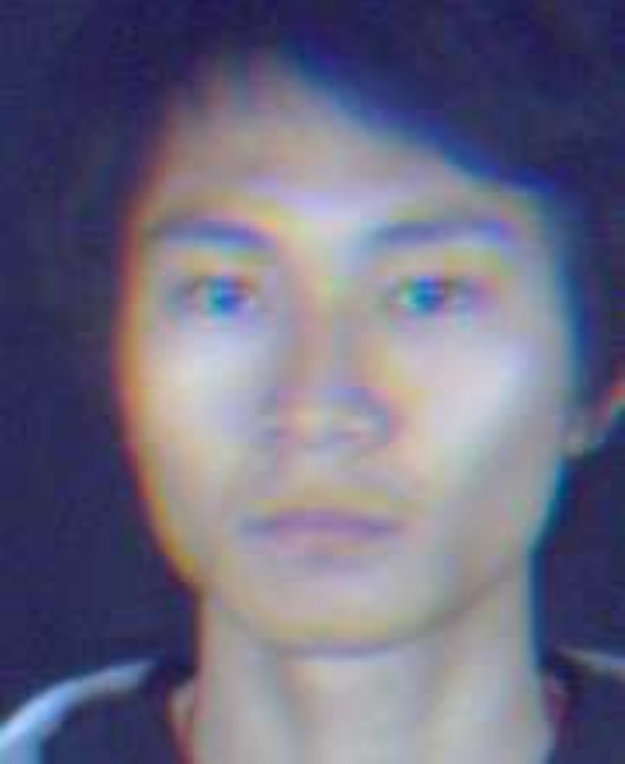}\\
\includegraphics[width=1\linewidth]{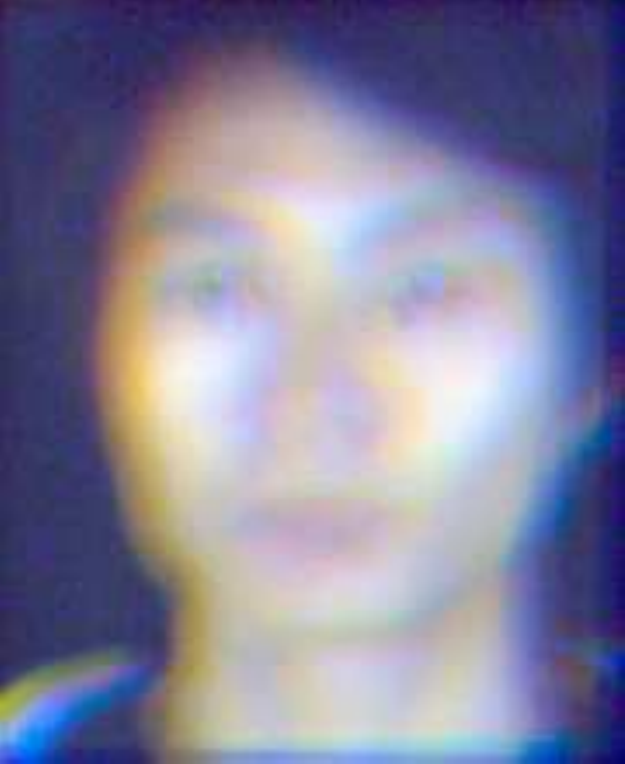}\\
\includegraphics[width=1\linewidth]{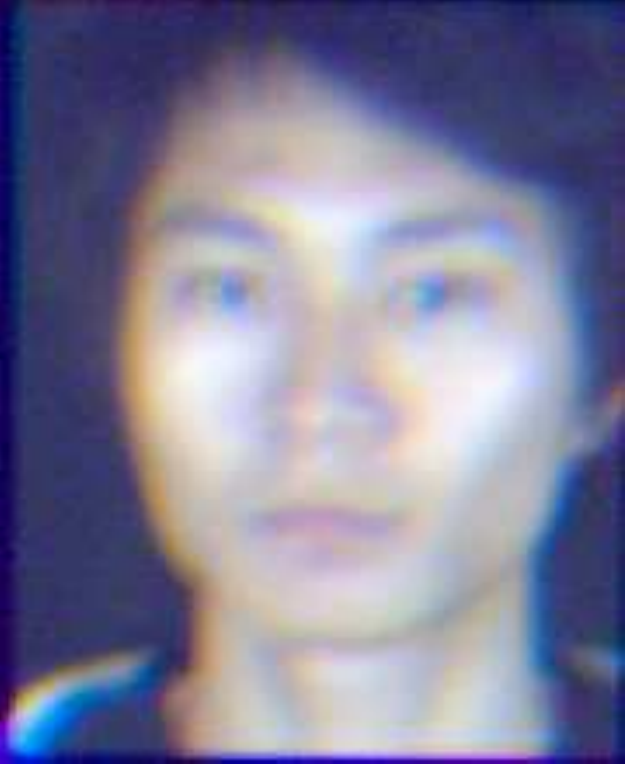}}
\end{minipage}}
\subfigure[]{
\begin{minipage}{0.1\linewidth}{
\includegraphics[width=1\linewidth]{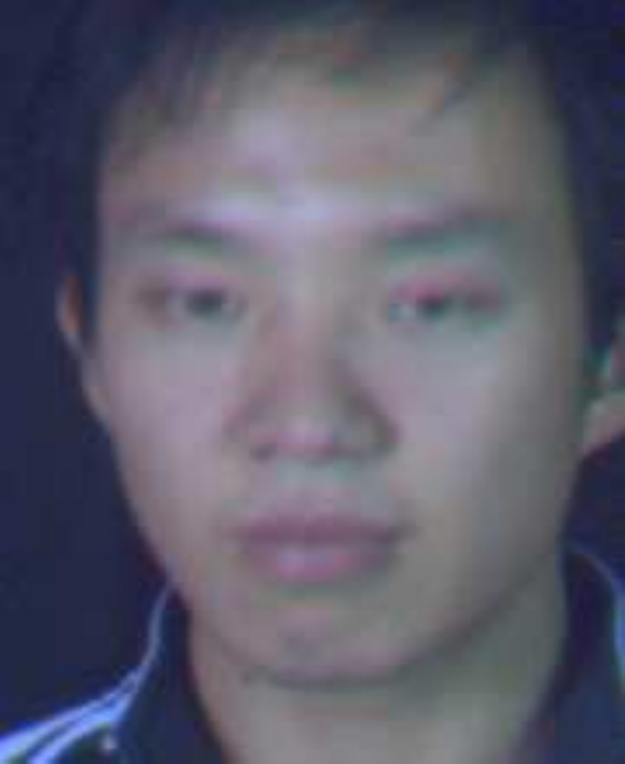}\\
\includegraphics[width=1\linewidth]{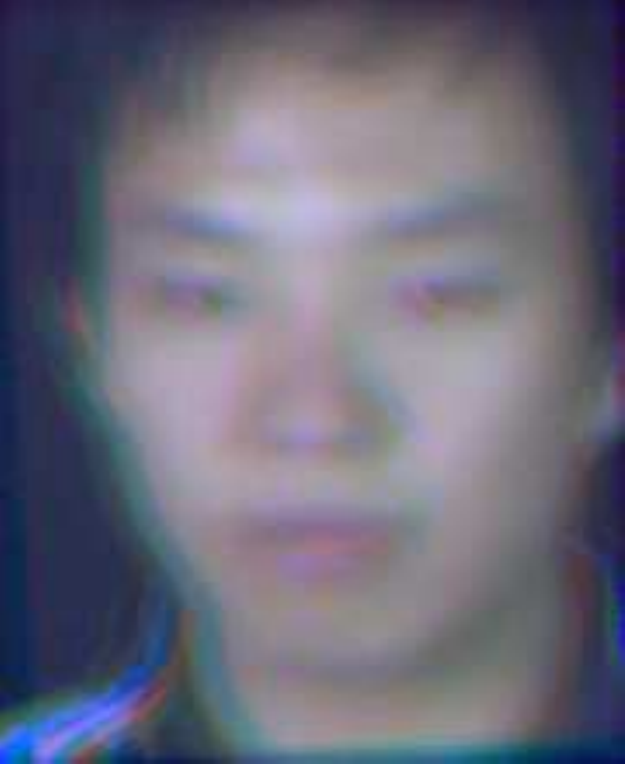}\\
\includegraphics[width=1\linewidth]{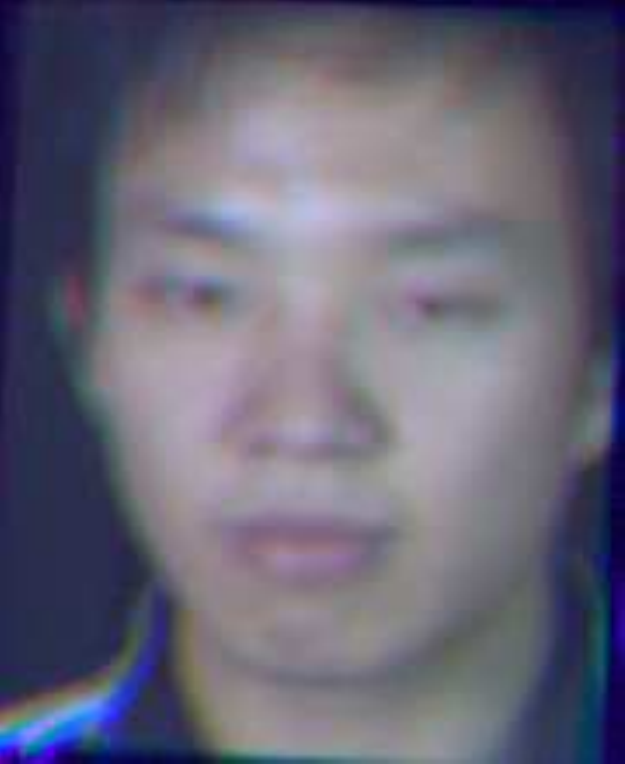}}
\end{minipage}}
\subfigure[]{
\begin{minipage}{0.1\linewidth}{
\includegraphics[width=1\linewidth]{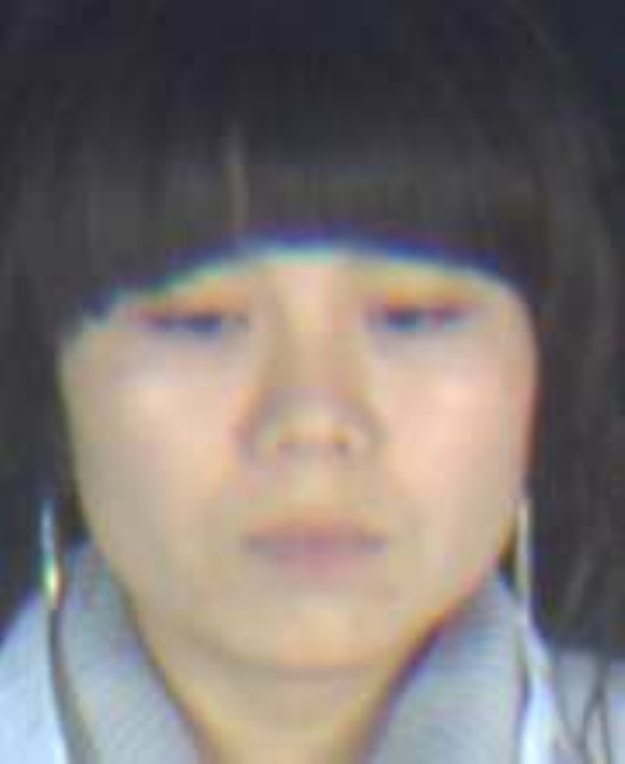}\\
\includegraphics[width=1\linewidth]{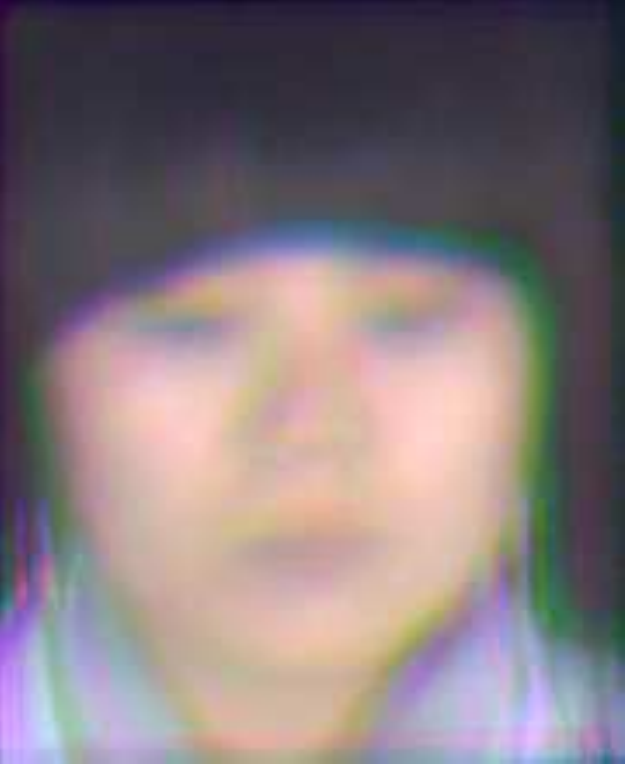}\\
\includegraphics[width=1\linewidth]{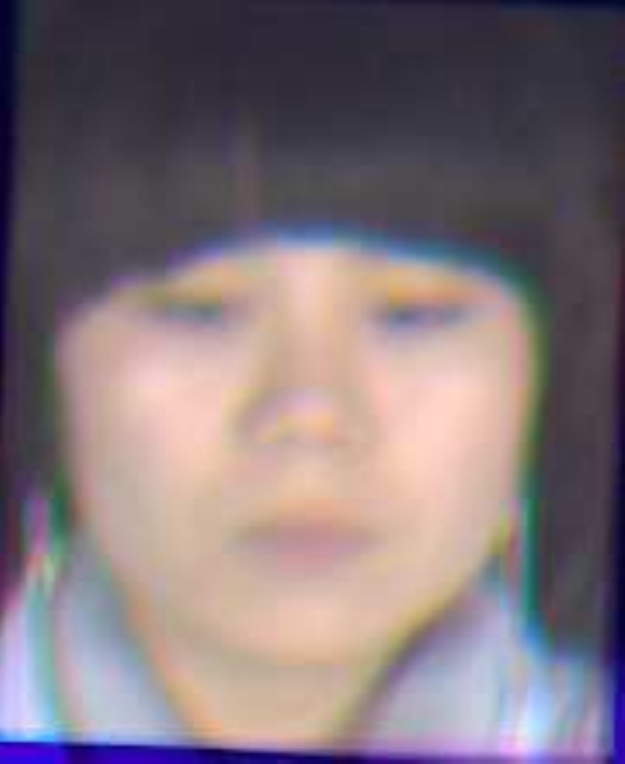}}
\end{minipage}}
\subfigure[]{
\begin{minipage}{0.1\linewidth}{
\includegraphics[width=1\linewidth]{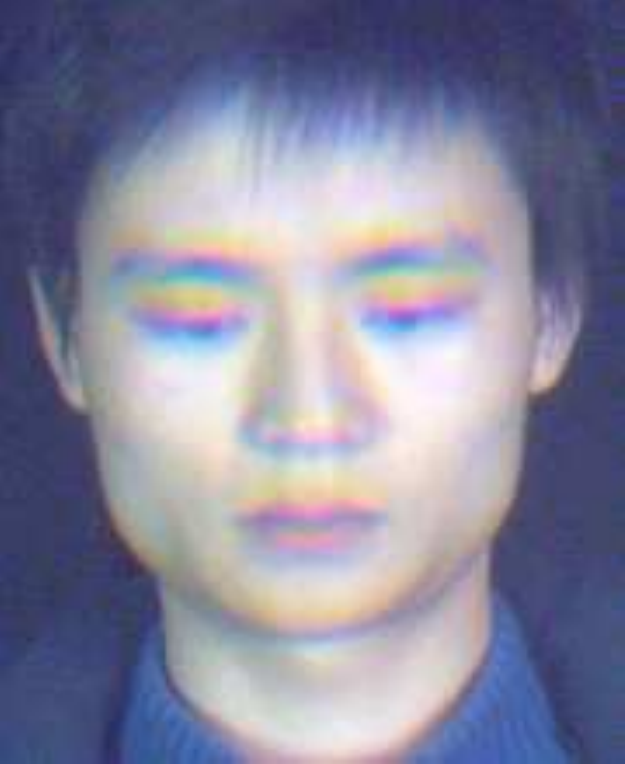}\\
\includegraphics[width=1\linewidth]{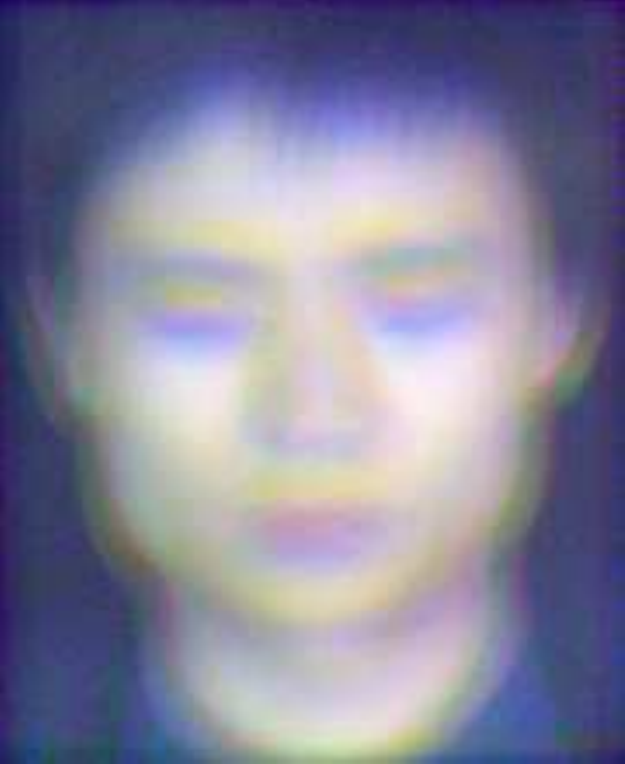}\\
\includegraphics[width=1\linewidth]{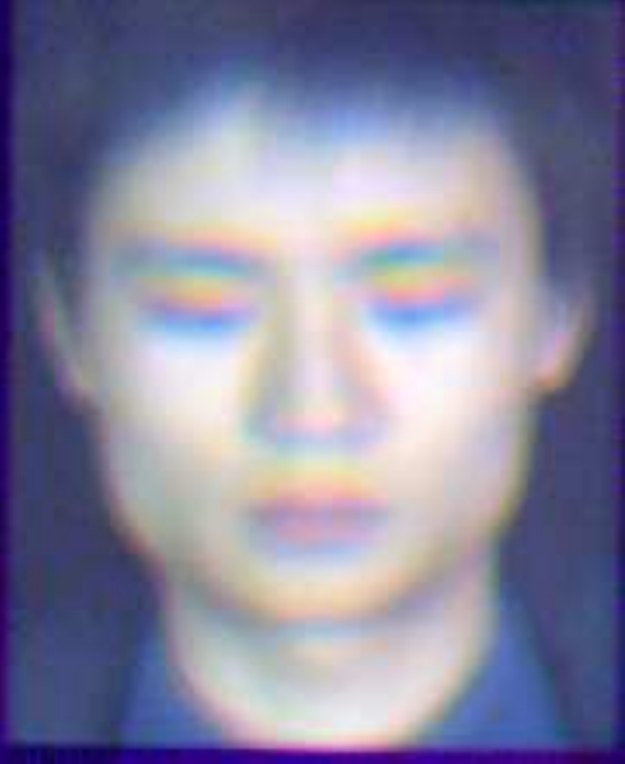}}
\end{minipage}}
\subfigure[]{
\begin{minipage}{0.1\linewidth}{
\includegraphics[width=1\linewidth]{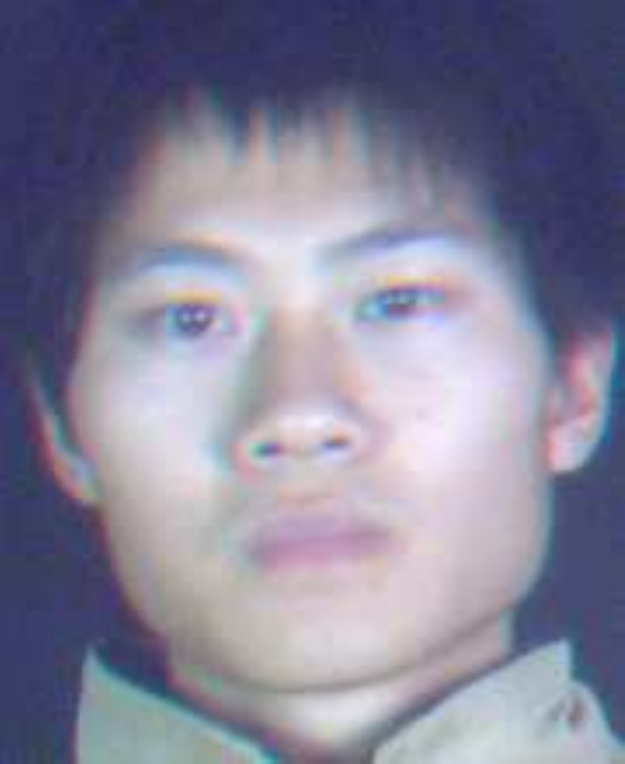}\\
\includegraphics[width=1\linewidth]{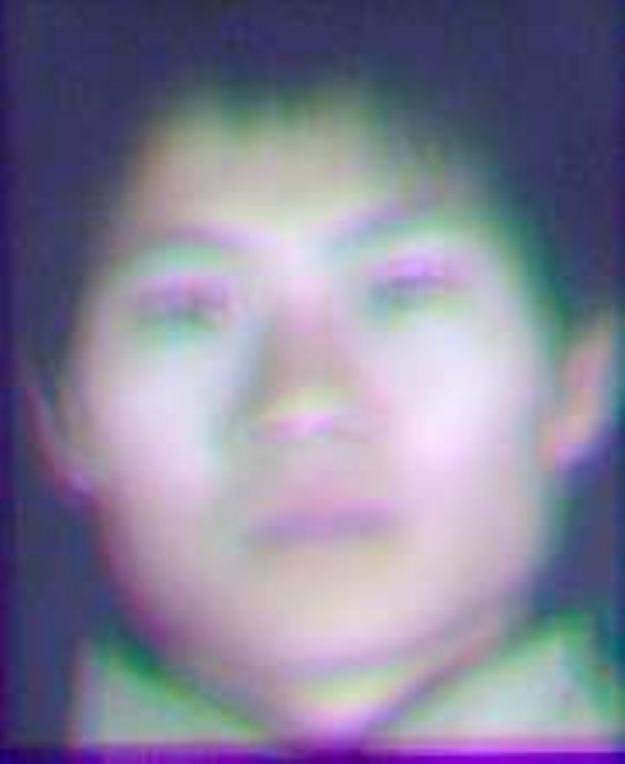}\\
\includegraphics[width=1\linewidth]{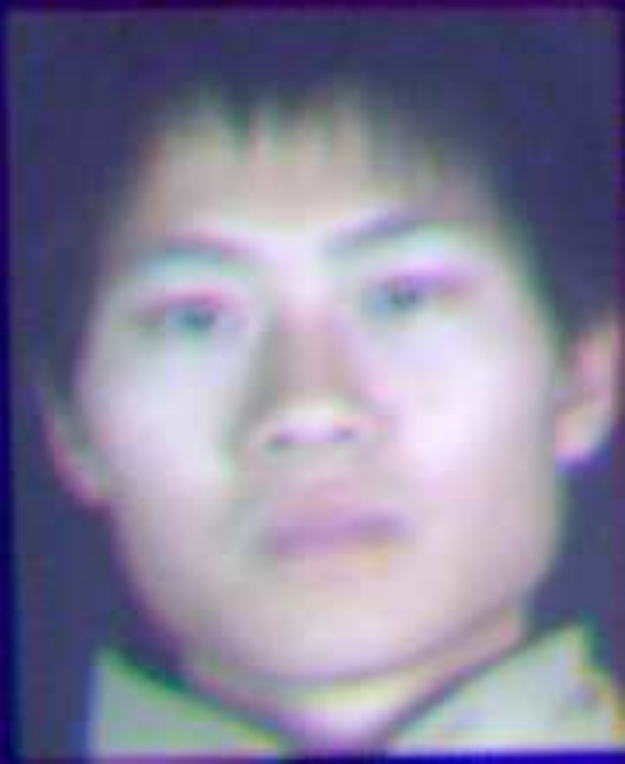}}
\end{minipage}}
\caption[Registration results in simulated experiment]{The registration result of a hyperspectral face image can be visually observed by rendering it as RGB image. Original image (top), misaligned image (center), and registered image (bottom). The registered images are sharper because the bands are aligned to each other compared to misaligned images.}
\label{fig:reg-results-sim}
\end{figure*}
\end{landscape}

\begin{landscape}
\begin{figure*}[t]
\centering
\subfigure[]{
\begin{minipage}{0.10\linewidth}{
\includegraphics[width=1\linewidth]{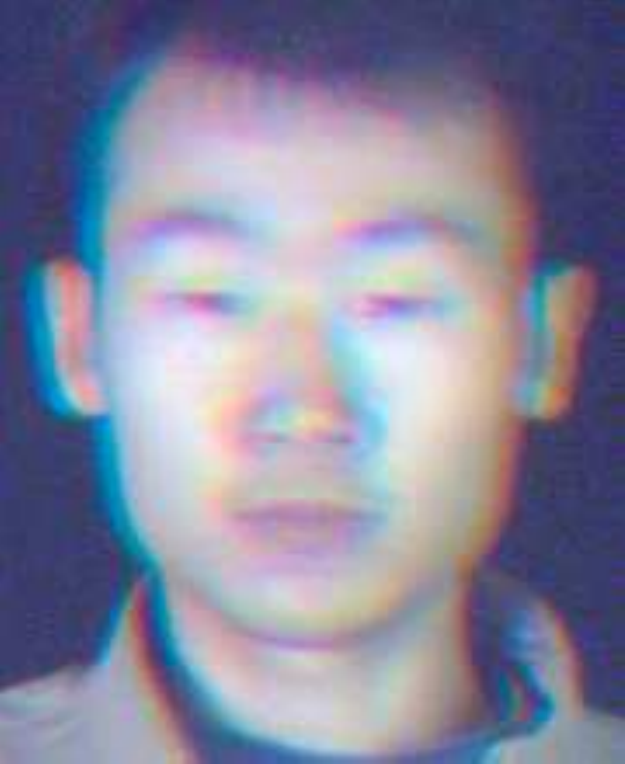}\\
\includegraphics[width=1\linewidth]{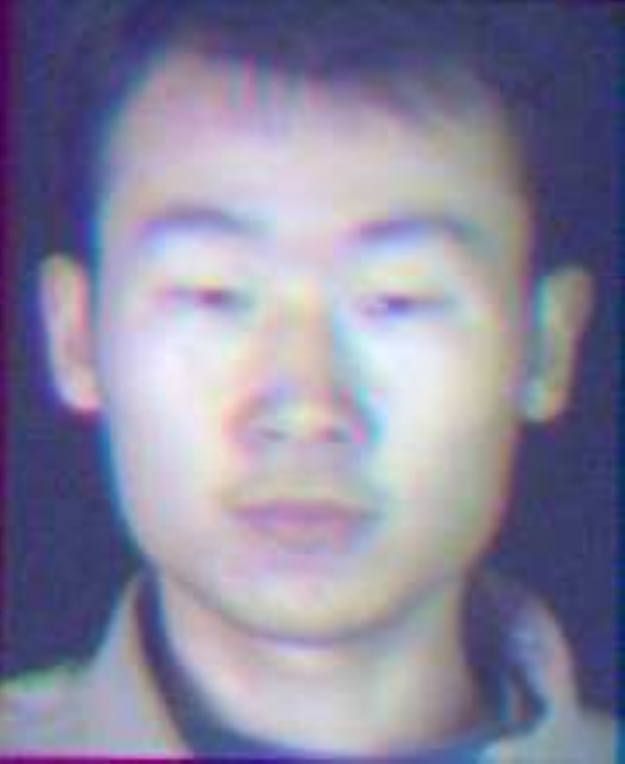}}
\end{minipage}}
\subfigure[]{
\begin{minipage}{0.10\linewidth}{
\includegraphics[width=1\linewidth]{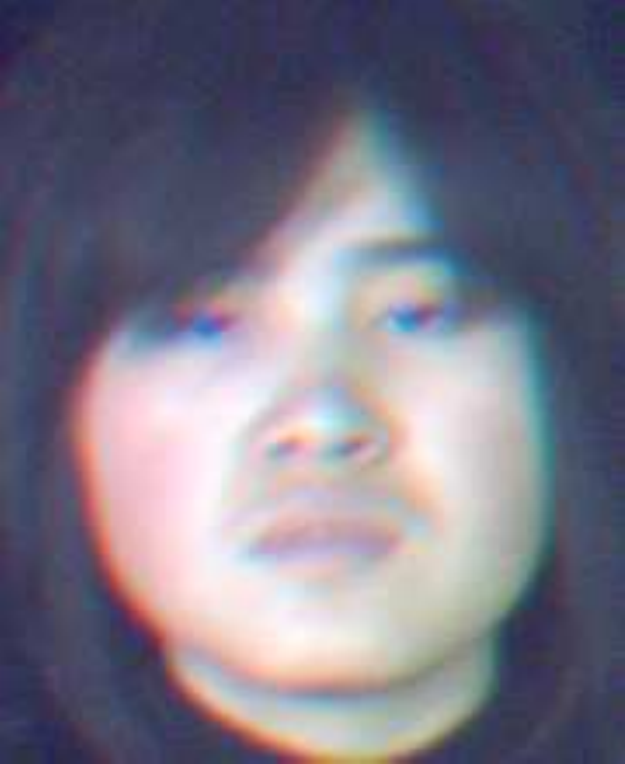}\\
\includegraphics[width=1\linewidth]{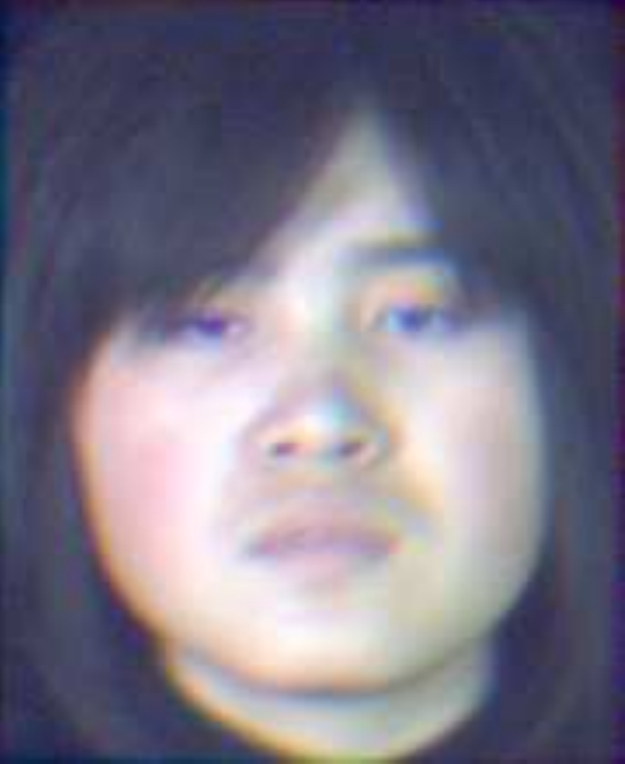}}
\end{minipage}}
\subfigure[]{
\begin{minipage}{0.10\linewidth}{
\includegraphics[width=1\linewidth]{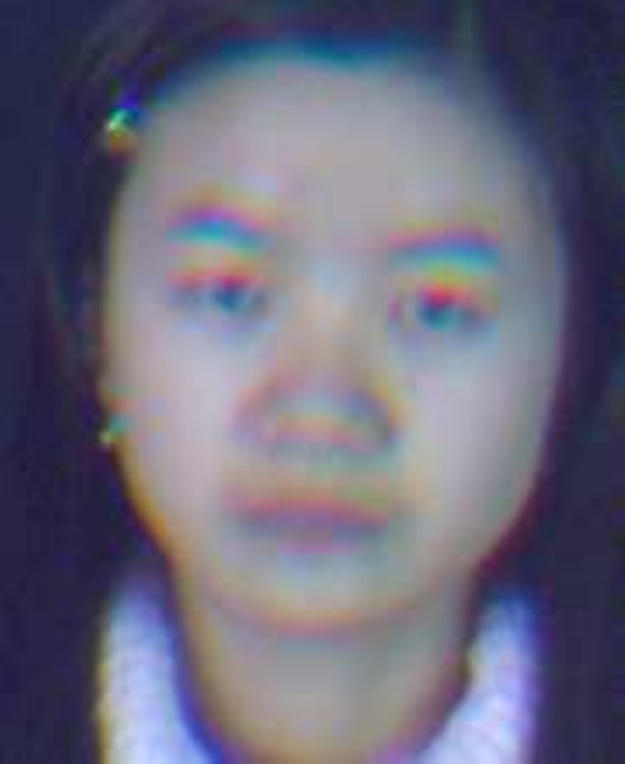}\\
\includegraphics[width=1\linewidth]{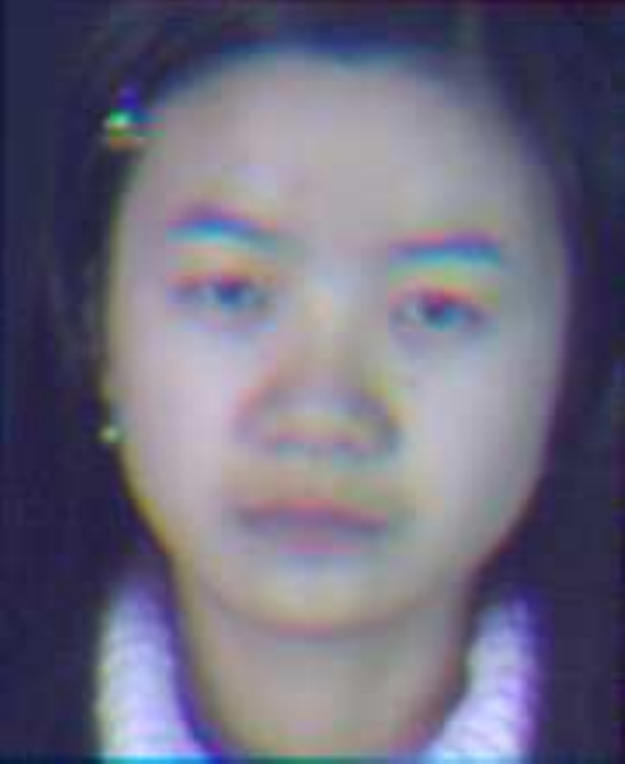}}
\end{minipage}}
\subfigure[]{
\begin{minipage}{0.10\linewidth}{
\includegraphics[width=1\linewidth]{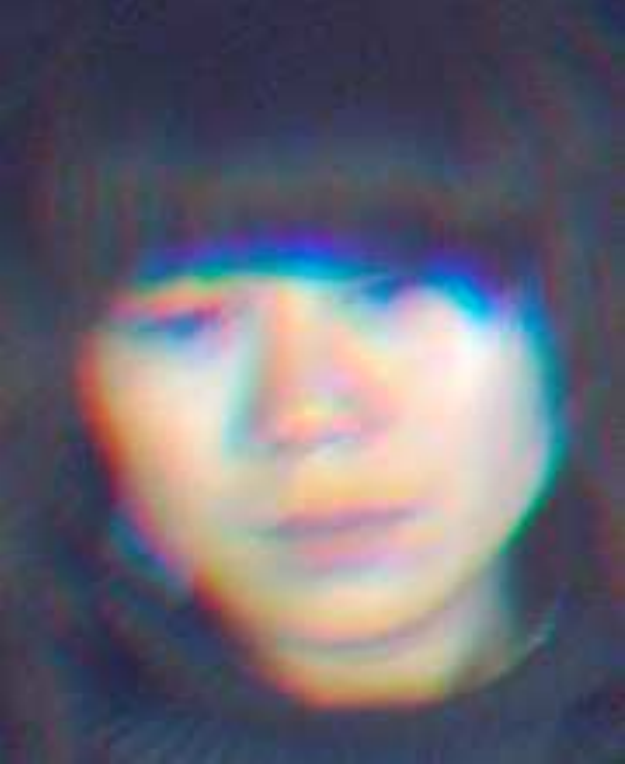}\\
\includegraphics[width=1\linewidth]{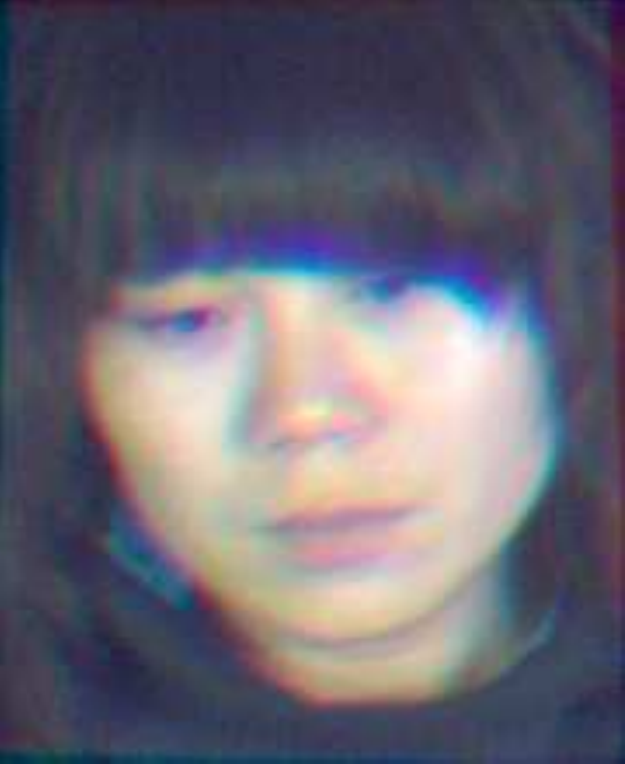}}
\end{minipage}}
\subfigure[]{
\begin{minipage}{0.10\linewidth}{
\includegraphics[width=1\linewidth]{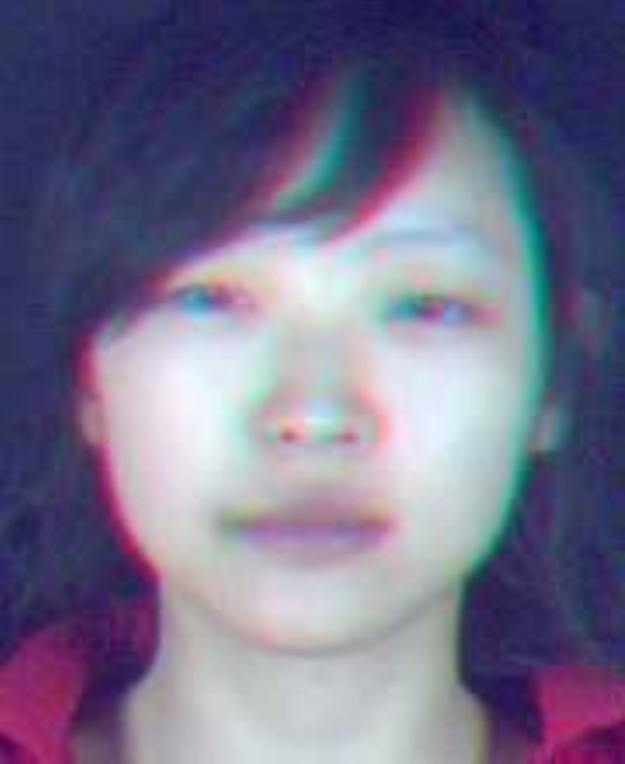}\\
\includegraphics[width=1\linewidth]{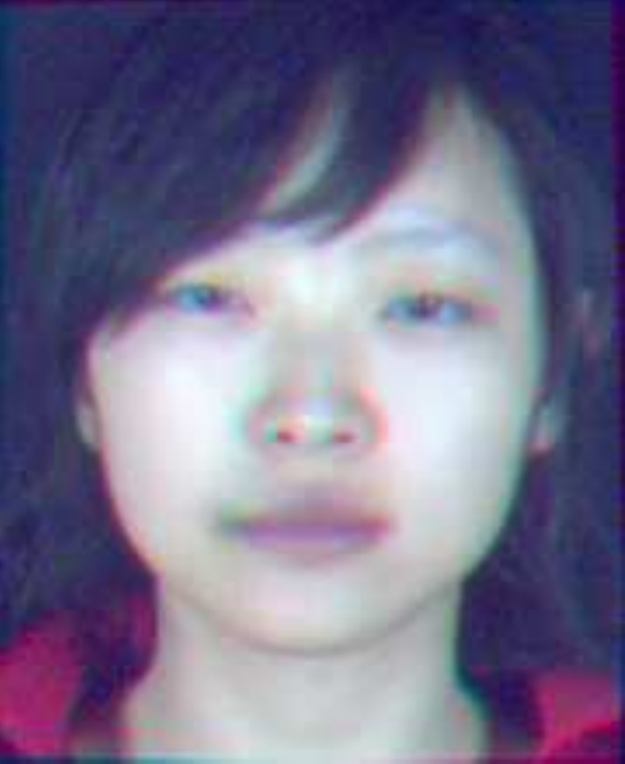}}
\end{minipage}}
\subfigure[]{
\begin{minipage}{0.10\linewidth}{
\includegraphics[width=1\linewidth]{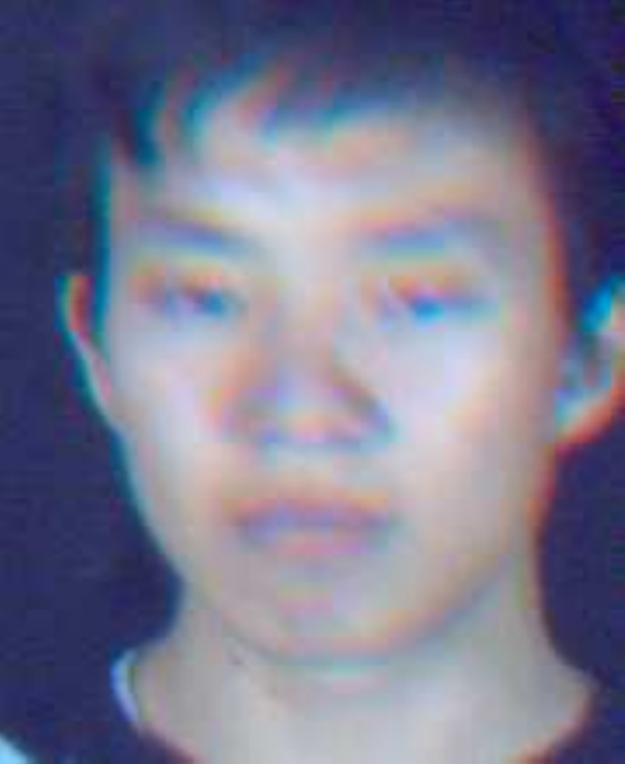}\\
\includegraphics[width=1\linewidth]{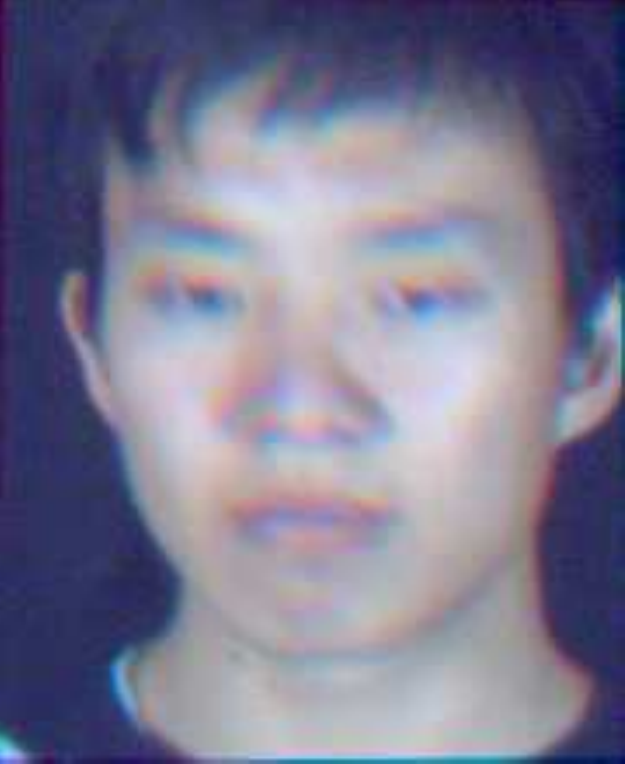}}
\end{minipage}}
\subfigure[]{
\begin{minipage}{0.10\linewidth}{
\includegraphics[width=1\linewidth]{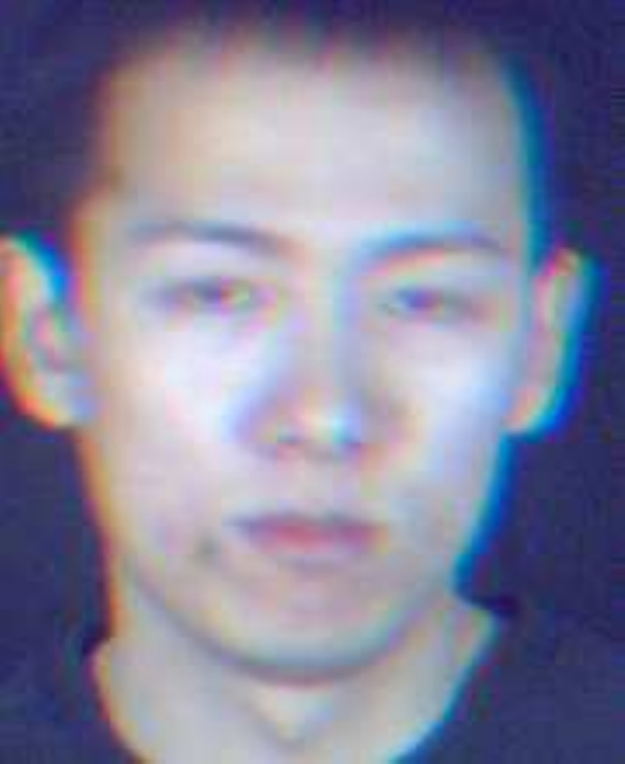}\\
\includegraphics[width=1\linewidth]{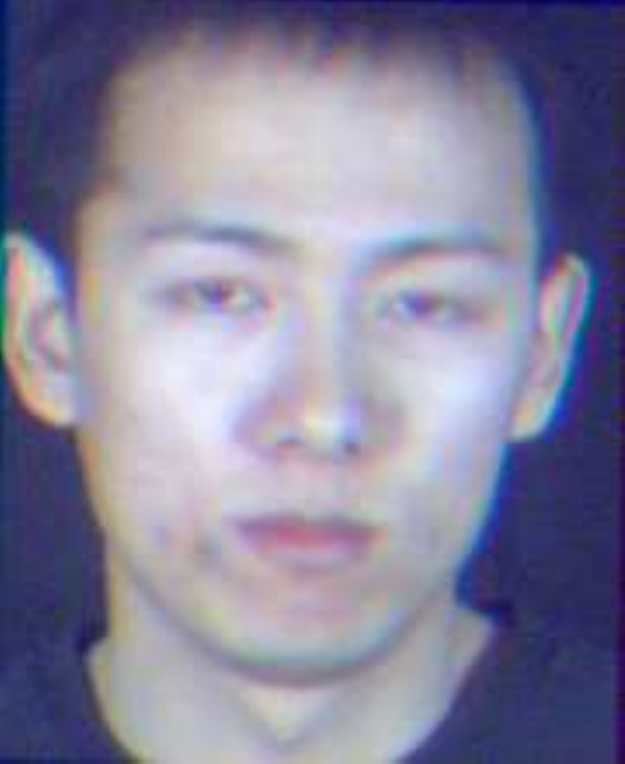}}
\end{minipage}}
\subfigure[]{
\begin{minipage}{0.10\linewidth}{
\includegraphics[width=1\linewidth]{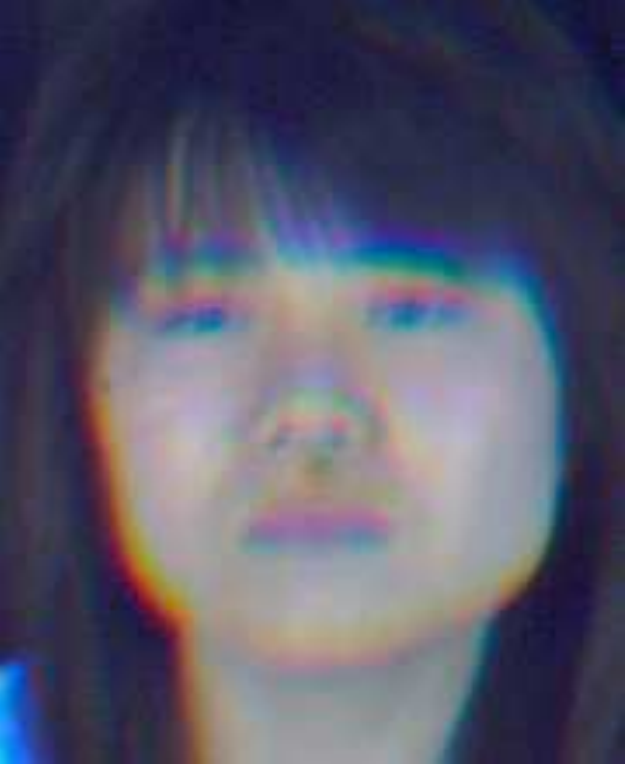}\\
\includegraphics[width=1\linewidth]{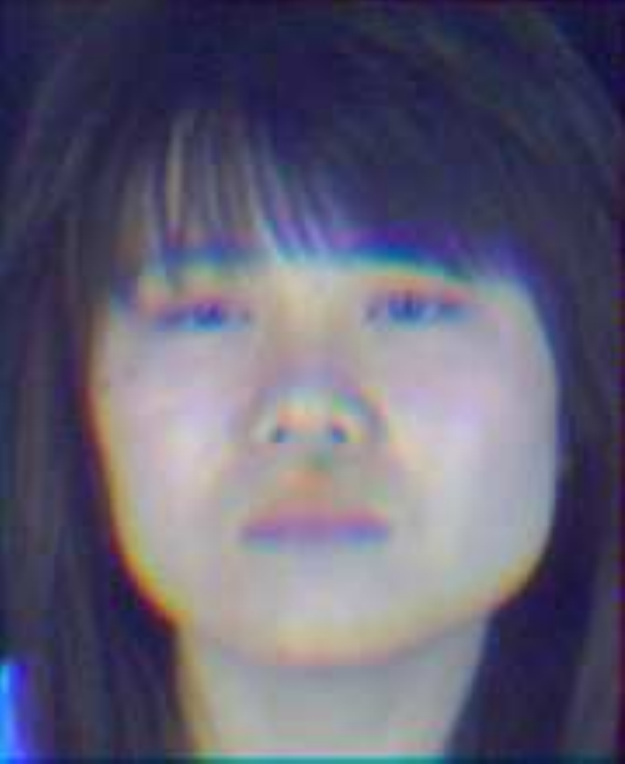}}
\end{minipage}}
\subfigure[]{
\begin{minipage}{0.10\linewidth}{
\includegraphics[width=1\linewidth]{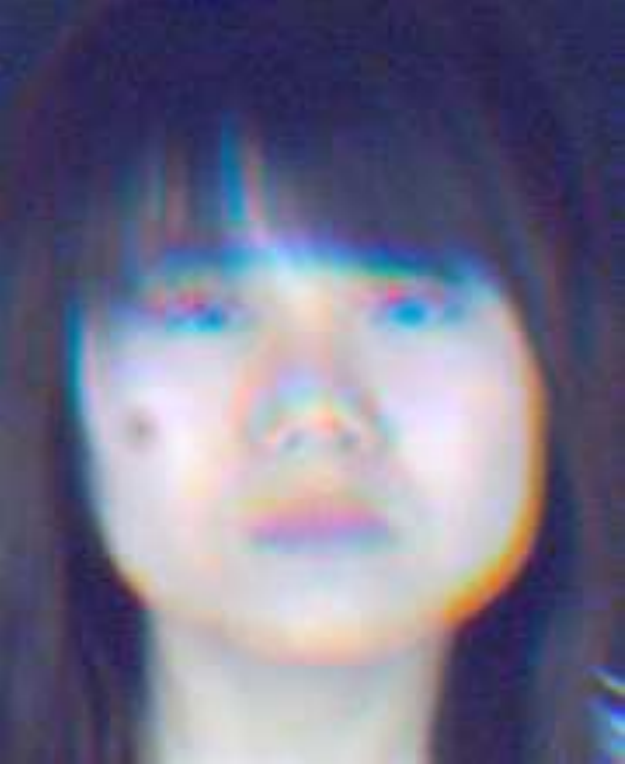}\\
\includegraphics[width=1\linewidth]{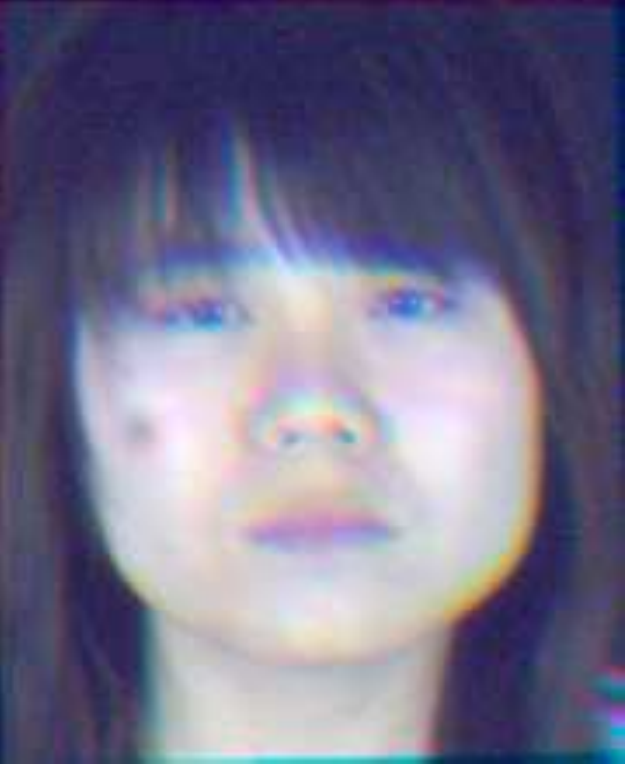}}
\end{minipage}}
\caption[Registration results in real experiment]{The registration result of a hyperspectral face image can be visually observed by rendering it as RGB image. The misaligned images are blurry especially in the areas of movement (facial outline, hair etc.). The registered images are sharper because the bands are relatively aligned.}
\label{fig:reg-results-real}
\end{figure*}
\end{landscape}

\begin{figure}[t]
\centering
\includegraphics[width=0.48\linewidth]{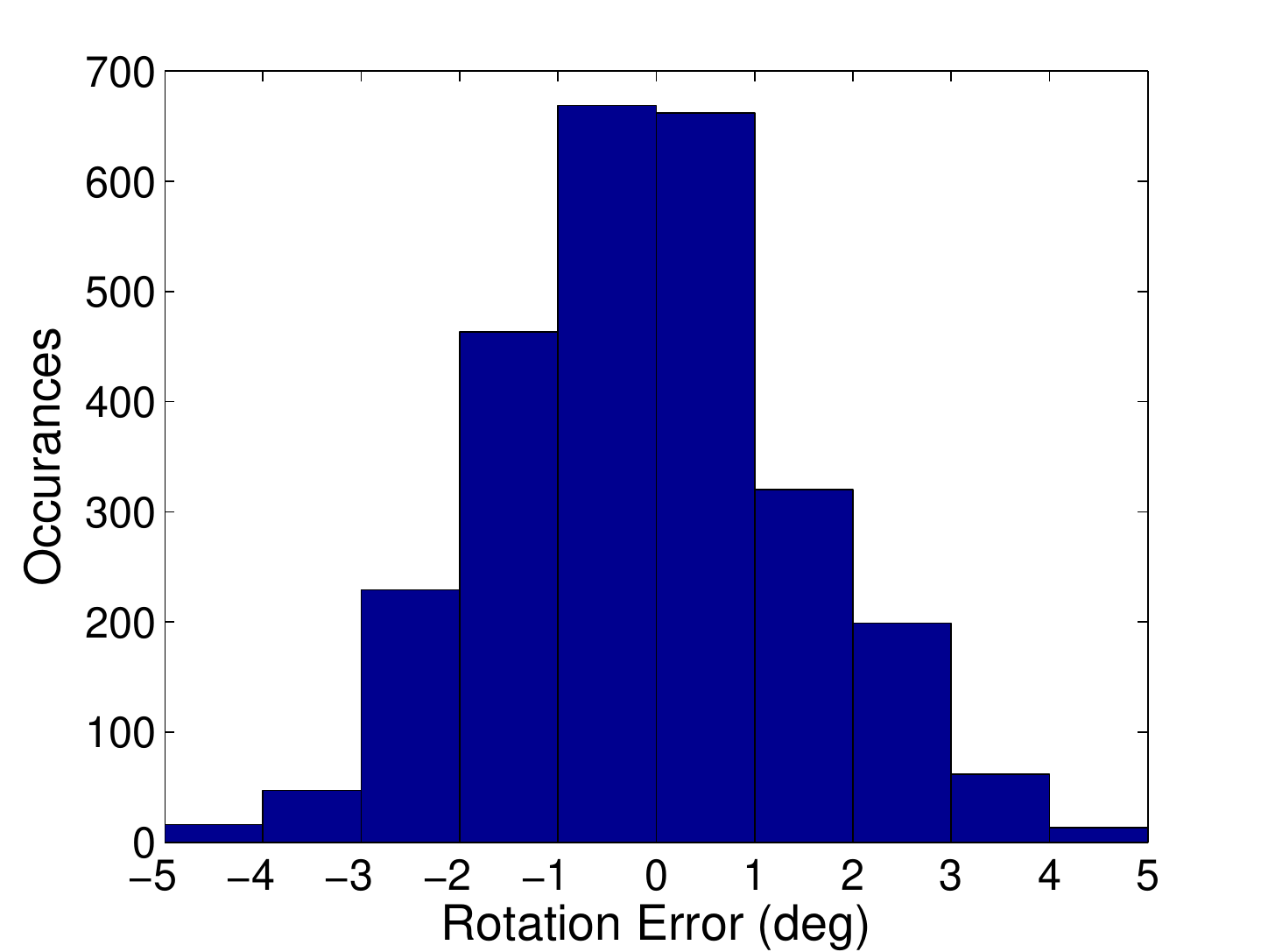}
\includegraphics[width=0.48\linewidth]{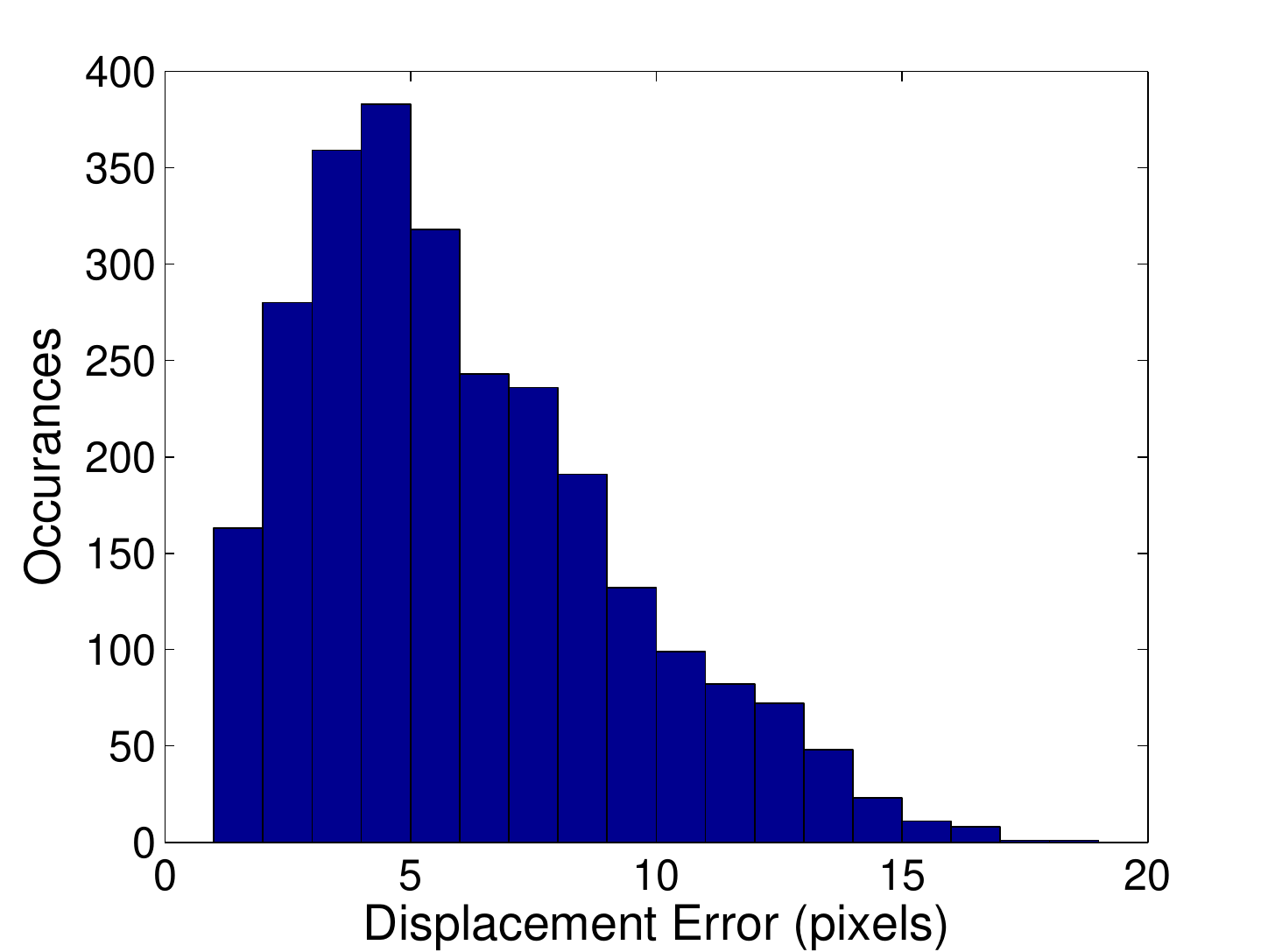}
\caption[Registration result of simulated images]{Errors in rotation and translation between the bands of 113 misaligned and aligned hyperspectral images. The rotational errors peak at zero degrees and translational error peak at 5 indicate improvement in cross spectral alignment.}
\label{fig:reg-error-sim}
\end{figure}

For the next experiment, we carefully selected 38 images from the PolyU hyperspectral face database which had noticeable misalignments during acquisition. Consider the raw spectral image $\mathbf{X}$ whose bands are misaligned, and the registered spectral image $\mathbf{Y}$ whose bands are aligned by the proposed algorithm. The proposed registration algorithm aligns the bands of hyperspectral face image to give the registered image $\mathbf{Y}$. From each cube $\mathbf{X}$ and $\mathbf{Y}$ we compute the sum of squared difference of the target image with the source image for each $k^\textrm{th}$ band, $k = 1,2,...,p$.
\begin{equation}
\mathbf{e}_x = \sum_i \sum_j {|\mathbf{X}_k(i,j)-\bar{\mathbf{X}}_k(i,j)|}^2~.
\end{equation}
Similarly, for registered cube $\mathbf{Y}$
\begin{equation}
\mathbf{e}_y = \sum_i \sum_j {|\mathbf{Y}_k(i,j)-\bar{\mathbf{Y}}_k(i,j)|}^2~,
\end{equation}
where $\mathbf{e}_x,\mathbf{e}_y \in \mathbb{R}^p$.
We compute the registration improvement from an unregistered spectral image $\mathbf{X}$ to its registered version $\mathbf{Y}$ as
\begin{equation}
\mathbf{e}_r = \frac{\mathbf{e}_x-\mathbf{e}_y}{\mathbf{e}_x}~.
\end{equation}

Figure~\ref{fig:reg-results-real} shows the registration results of real misaligned images. The improvement is observable after a close analysis. In order to numerically observe this improvement, we plot the improvement in sum of squared difference between consecutive bands. Figure~\ref{fig:reg-error-real} shows the improvement in $\mathbf{e}_r$ between consecutive bands of hyperspectral face images. It can be seen that most of the bands demonstrate a positive $\mathbf{e}_r$ which clearly indicates an improvement in cross spectral alignment.

\clearpage

\begin{figure}[h]
\centering
\includegraphics[width=0.4\linewidth]{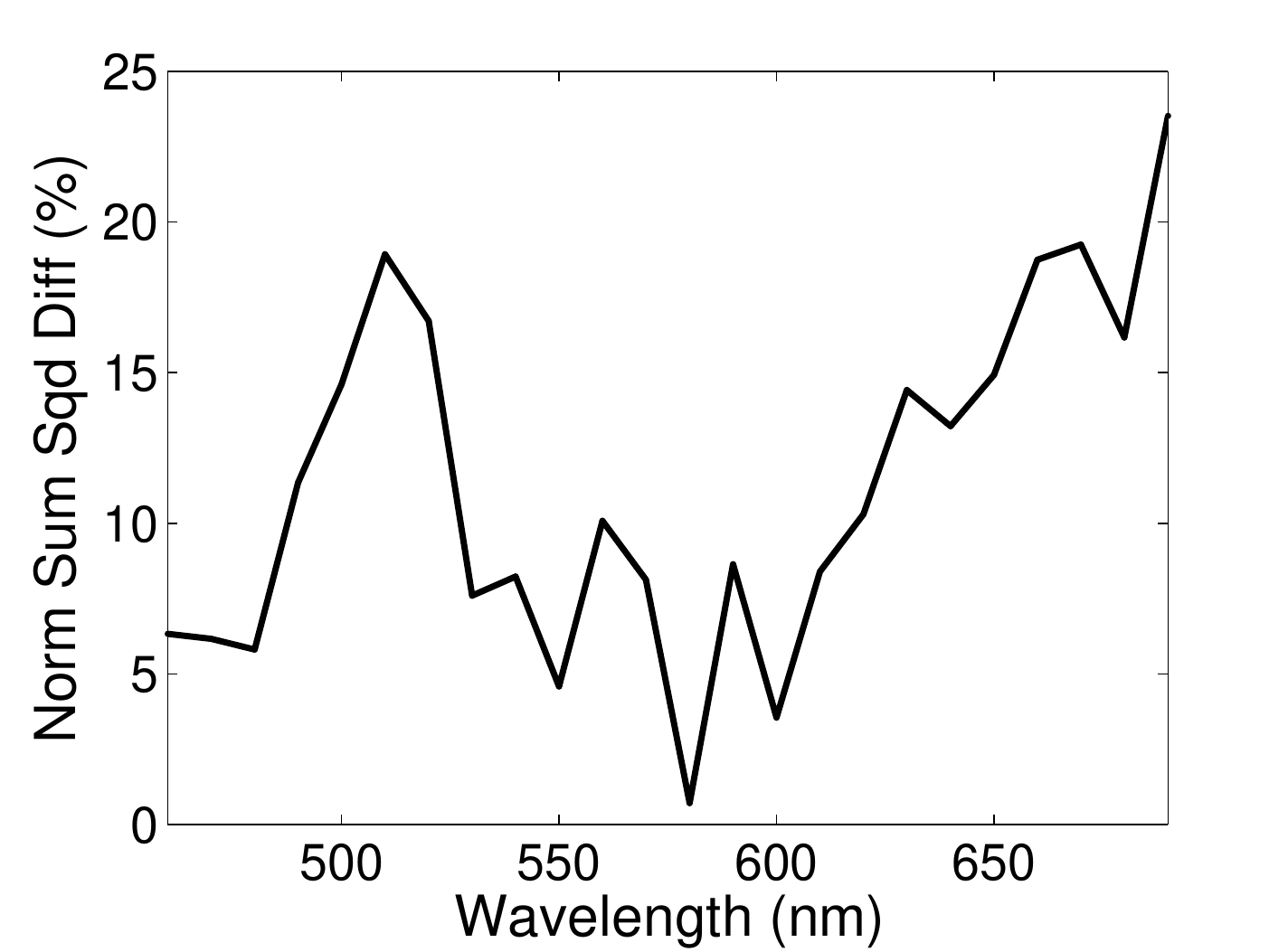}
\caption[Registration result of real images]{Average improvement in registration error between consecutive bands of 38 hyperspectral face images.}
\label{fig:reg-error-real}
\end{figure}

\section{Conclusion}
We presented the Cross Spectral Similarity feature for cross spectral image registration. We demonstrated that the extraction of self similarities across spectral bands holds promise in accurate alignment of hyperspectral images. Experiments on simulated and real misaligned hyperspectral images from the PolyU hyperspectral face database show the efficacy of the proposed approach in cross spectral registration with low registration errors. The proposed descriptor can be extended to other heterogenous image registration scenarios as it is extracted independently from the different modalities.


\chapter[Joint Group Sparse Principal Component Analysis]{Joint Group Sparse PCA\\for Compressed Hyperspectral Imaging} 

\label{Chapter5} 



Principal Component Analysis (PCA) is a powerful tool for unsupervised data analysis and visualization. PCA gives an orthogonal basis aligned with the directions of maximum variance of the data. It is useful for projecting the data onto a subspace defined by the most significant basis vectors. However, each principal component is a linear combination of \emph{all} features which makes measurement of all features essential. In some applications (e.g. spectral imaging), each sensed feature (band) may come at an additional cost of acquisition, processing and storage. Moreover, not all measured features may be important to the potential application with regards to the information in the signal and the relative noise. Therefore, it is desirable to select the most informative subset of the features to consequently reduce the cost of sensing the additional less informative features.

Sparse PCA enforces sparsity on the linear combination of input features used to compute the PCA basis. Zhou et al.~\cite{zou2006sparse} cast Sparse PCA as a regression type optimization problem and imposed the lasso constraint~\cite{tibshirani1996regression} to approximate the data with a sparse linear combination of the input features. Such an approach is good for interpretation of the data but would still require the measurement of all the input features. Other algorithms for computing Sparse PCA include the SCoTLASS algorithm~\cite{jolliffe2003modified} which aims at maximizing the Rayleigh quotient of the covariance matrix of the data using a non-convex optimization, the DSPCA algorithm~\cite{d2007direct} which solves a convex relaxation of the sparse PCA problem, the low rank matrix approximation method with sparsity constraint~\cite{shen2008sparse}, Sparse PCA with positivity constraints~\cite{zass2006nonnegative} and the generalized power method~\cite{journee2010generalized}. In all of these methods, the computation of each basis vector is dealt as an independent problem, the basis vectors are individually sparse but may not be jointly sparse.

Another aspect overlooked by Sparse PCA is the structure of the data in terms of groups of correlated features~\cite{huang2011learning}. For example, image pixels are organized on a rectangular grid exhibiting some sort of connectivity and neighborhood relationship. Similarly, gene expression data involves groups of genes corresponding to the same biological processes or sets of genes that are physical neighbors. It is sometimes desirable to encode relationship of the features in Sparse PCA, so that sparsity follows the group structure. Standard sparse solutions do not offer the incorporation of feature groups.

\clearpage

A rather obvious extension of the lasso formulation in Sparse PCA to Group Sparse PCA is to introduce the group lasso penalty ~\cite{yuan2006model,friedman2010note}. Group lasso uses the $\ell_1/\ell_2$ mixed vector norm to shrink all features in predefined groups with small magnitude to zero. Guo et al.~\cite{guo2010principal} proposed Sparse Fused PCA which derives group structures from feature correlation. They augment the Sparse PCA formulation~\cite{zou2006sparse} by an additional penalty term that encourages the coefficients of highly correlated features to be similar and subsequently fused. However, their solution does not directly result in sparsity, but only forces the coefficients to a similar value which may or may not be close to zero. Jenatton et al.~\cite{jenatton2010structured} used the non-convex $\ell_\alpha/\ell_2$ quasi-norm (where $\alpha \in (0,1)$) for structured sparse PCA. Rectangular patterns are rotated to obtain a larger set of convex patterns for defining groups. They showed the benefits of using structured sparsity in image denoising and face recognition tasks. Grbovic et al.~introduced two types of grouping constraints into the Sparse PCA problem to ensure reliability of the resulting groups~\cite{grbovic2012sparse}. Jacob et al.~\cite{jacob2009group} proposed a new penalty function that allowed potentially overlapping groups, whereas, Huang et al.~\cite{huang2011learning} generalized the group sparsity to accommodate arbitrary structures.

While group sparsity accounts for the data structure, it still does not guarantee joint sparsity of the complete PCA basis with respect to the input features. We present Joint Group Sparse PCA (JGSPCA) which forces the basis coefficients corresponding to a group of features to be jointly sparse. Joint sparsity ensures that the complete data be reconstructed from only a sparse set of input features whereas the group sparsity ensures that the structure of the correlated features is maximally preserved.

An important application of Sparse PCA and Group Sparse PCA is data interpretation through dimensionality reduction. However, the proposed Joint Group Sparse PCA (JGSPCA) can also be used for model based compressed sensing. Classical compressed sensing does not assume any prior model over the data and is based on the restricted isometry property (see~\cite{donoho2006compressed} and the references therein). In other words they are not learning based. On the other hand, the proposed JGSPCA algorithm is learning based and is closer to model based compressive sensing theory~\cite{baraniuk2007compressive}.

We validate the proposed JGSPCA algorithm on the problem of compressed hyperspectral imaging and recognition. A hyperspectral image is a data cube comprising two spatial and one spectral dimension. Since the spectra of natural objects is smooth, their variations can be approximated by a few basis vectors~\cite{nascimento2005psychophysical}. Besides, there is a high correlation among neighboring pixels in the spatial domain. In a compact representation of such a data, structure needs to be preserved in the spatial dimension, while sparsity is desirable in the spectral dimension. Pixels from local spatial neighborhood are grouped together, while sparsity is induced in the spectral dimension. This redundancy in the data makes hyperspectral images a good candidate for sparse representation~\cite{chakrabarti2011statistics} as well as compressed sensing~\cite{golbabaee2012hyperspectral}. We present the Joint Group Sparse PCA algorithm in Section~\ref{sec:JGSPCA}. Description of the experimental setup, evaluation protocol, and datasets used in the experiments are given in Section~\ref{sec:exp-JGSPCA}. The results of compressed sensing and recognition experiments are presented in Section~\ref{sec:results}. The chapter is concluded in Section~\ref{sec:conc-JGSPCA}.

\section{Joint Group Sparse PCA}
\label{sec:JGSPCA}


Let $\mathbf{X}={[\mathbf{x}_1,\mathbf{x}_2,\ldots,\mathbf{x}_n]}^{\intercal} \in \mathbb{R}^{n \times p}$ be a data matrix which comprises $n$ observations $\mathbf{x}^i \in\mathbb{R}^p$, where $p$ is the number of features. Assume that the sample mean $\bar{\mathbf{x}} \in \mathbb{R}^{p}$ has been subtracted from all $n$ observations so that the columns of $\mathbf{X}$ are centered. Generally, a PCA basis can be computed by singular value decomposition of the data matrix.
\begin{equation}
\label{eq:svd-JGSPCA}
\mathbf{X} = \mathbf{USV}^{\intercal}~,
\end{equation}
where $\mathbf{V} \in \mathbb{R}^{p \times p}$ are the PCA \emph{basis vectors} (loadings) and $\mathbf{S}$ is the diagonal matrix of \emph{eigenvalues}. $\mathbf{V}$ is an orthonormal basis such that $\mathbf{v}_i^\intercal\mathbf{v}_j=0~\forall~i\neq j$ and $\mathbf{v}_i^\intercal\mathbf{v}_j=1~\forall~i=j$. If $\mathbf{X}$ is low rank, it is possible to significantly reduce its dimensionality by using the $k$ most significant basis vectors. The projection of data $\mathbf{X}$ upon the first $k$ basis vectors of $\mathbf{V}$ gives the \emph{principal components} (scores). An alternative formulation treats PCA as a regression type optimization problem
\begin{equation}
\label{eq:PCAregress-JGSPCA}
\underset{\hat{\mathbf{A}}}{\arg \min} {\|\mathbf{X}-\mathbf{XAA}^{\!\intercal}\|}_F^2 \qquad \text{subject to}\, \mathbf{A}^{\!\intercal}\mathbf{A}=\mathbf{I}_k~,
\end{equation}
where ${\|.\|}_F$ is the Frobenius norm, $\mathbf{A} \in \mathbb{R}^{p \times k}$ is an orthonormal basis $\{\bm{\alpha}_1,\bm{\alpha}_2,$ $\ldots,\bm{\alpha}_k\}$. Here, $\mathbf{A}$ is equivalent to the first $k$ columns of $\mathbf{V}$. Each principal component is derived by a linear combination of all $p$ features and consequently $\bm{\alpha}$ is non-sparse. In order to obtain a sparse PCA basis, a regularization term is usually included in the regression formulation~\eqref{eq:PCAregress-JGSPCA}. Inclusion of a sparse penalty reduces the number of features involved in each linear combination for obtaining the principal components. One way to obtain sparse basis vectors is by imposing the $\ell_0$ constraint upon the regression coefficients (basis vectors)~\cite{zou2006sparse}.
\begin{multline}
\label{eq:SPCACriterion-JGSPCA}
\underset{\hat{\mathbf{A}},\hat{\mathbf{B}}}{\arg \min} {\|\mathbf{X}-\mathbf{XBA}^{\!\intercal}\|}_F^2+ \lambda\sum_{j=1}^{k}{\|\bm{\beta}_j\|}_0\\ \text{subject to} \qquad \mathbf{A}^{\!\intercal}\mathbf{A}=\mathbf{I}_k~,
\end{multline}
where $\mathbf{B} \in \mathbb{R}^{p \times k}$ corresponds to the required sparse basis $\{\bm{\beta}_1,\bm{\beta}_2,\ldots,\bm{\beta}_k\}$. The $\ell_0$-norm regularization term penalizes the number of non-zero coefficients in $\bm{\beta}$, whereas the loss term simultaneously minimizes the reconstruction error ${\|\mathbf{X}-\mathbf{XB}\mathbf{A}^{\!\intercal}\|}_F^2$. If $\lambda$ is zero, the problem reduces to finding the ordinary PCA basis vectors, equivalent to~\eqref{eq:PCAregress-JGSPCA}. When $\lambda$ is large, most coefficients of $\bm{\beta}_j$ will shrink to zero, resulting in sparsity as shown in Figure~\ref{fig:pat1-JGSPCA}.

The above formulation allows us to individually determine informative features. However, it may not account for the structural relationship among multiple features. It is sometimes desirable that the sparsity patterns in the computed basis be similar for correlated group of features. This means that the features should exhibit a sparsity structure which improves the interpretation of their underlying sources. To address this issue, we reconsider our problem from the view of grouping correlated features. The grouping of features can be known either a priori from domain information, or computed directly from the data by utilizing correlation.

Consider the $p$ features are now divided into $g$ mutually exclusive groups. Let $\mathcal{G}_i$ be the set of indices of features corresponding to the $i^{\textrm{th}}$ group. The number of features in the $i^{\textrm{th}}$ group is $p_i=|\mathcal{G}_i|$ such that the total number of features $p=\sum_{i=1}^{g} p_i$. Hence, $\mathbf{X}$ can be considered a horizontal concatenation of $g$ submatrices $[\mathbf{X}_{\centerdot{\mathcal{G}}_1},\mathbf{X}_{\centerdot{\mathcal{G}}_2},...,\mathbf{X}_{\centerdot{\mathcal{G}}_g}]$. Each $\mathbf{X}_{\centerdot\mathcal{G}_i} \in \mathbb{R}^{n \times p_i}$ contains data (columns of $\mathbf{X}$) corresponding to the features of the $i^{\textrm{th}}$ group. The group lasso regularization penalizes $\ell_2$-norm of the coefficients corresponding to a \emph{feature group}~\cite{yuan2006model}. It enforces sparsity on a group of coefficients, instead of individual coefficients. The group lasso constraint can be incorporated into~\eqref{eq:SPCACriterion-JGSPCA}, to achieve the Group Sparse PCA (GSPCA) criterion

\begin{multline}
\label{eq:GSPCACriterion}
\underset{\hat{\mathbf{A}},\hat{\mathbf{B}}}{\arg \min} \|\mathbf{X}-\sum_{i=1}^g \mathbf{X}_{\centerdot\mathcal{G}_i}\mathbf{B}^{\centerdot\mathcal{G}_i}\mathbf{A}^{\!\intercal}\|_F^2 +\lambda\sum_{j=1}^{k}\sum_{i=1}^{g}\eta_i{\|\bm{\beta}^{\centerdot\mathcal{G}_i}_j\|}_2\\ \text{subject to} \qquad \mathbf{A}^{\intercal}\mathbf{A}=\mathbf{I}_k~,
\end{multline}
where $\|.\|_2$ is the Euclidean norm and $\eta_i$ is the weight of the $i^{\textrm{th}}$ group. $\mathbf{B}^{\centerdot\mathcal{G}_i} \in\mathbb{R}^{p_i \times p}$ denotes the submatrix corresponding to the $i^{\textrm{th}}$ group of features in $\mathbf{B}$. The group lasso penalty $\sum_{i=1}^{g}\eta_i{\|\bm{\beta}^{\centerdot\mathcal{G}_i}\|}_2$ induces sparsity at the group level, i.e.~if the coefficients of the $i^{\textrm{th}}$ group are non-zero, the entire $p_i$ features of the group will be selected and vice versa~\cite{friedman2010note}. It is important to note that the factor $\eta_i$ will only affect the regularization penalty for differently sized groups (typically $\eta_i = \sqrt{p_i}$). In case of equally sized groups, the factor can be ignored altogether (or assumed $\eta_i=1$).

Notice that the $\ell_0$ penalty in~\eqref{eq:SPCACriterion-JGSPCA} has been replaced with an $\ell_{2,1}$ penalty in~\eqref{eq:GSPCACriterion}. This formulation can be considered to be a generalized form for group and non-group structured data. A group may even consist of a single feature, if it is not highly correlated with other features. Hence, in the extreme case of an uncorrelated data, each group will contain a single feature, i.e.~$g=p$.

Equation~\eqref{eq:GSPCACriterion} gives a sparse basis which is able to account for the group structure of the data. When the group constraint is enforced, the basis coefficients become sparse in a group-wise manner. Imposing the additional group constraint generally results in reduced sparsity within the feature groups. This phenomenon is illustrated for an example basis in Figure~\ref{fig:patterns-JGSPCA}. Figure~\ref{fig:pat1-JGSPCA} depicts a sparse basis obtained by the SPCA criterion~\eqref{eq:SPCACriterion-JGSPCA} which does not take the group structure into account. Figure~\ref{fig:pat2-JGSPCA} gives a group sparse basis obtained by the GSPCA criterion~\eqref{eq:GSPCACriterion} for the same data. Consider for instance the null coefficients within the groups $\mathcal{G}_i$ of an SPCA basis vector $\bm{\beta}_j$ . As a consequence of enforcing the group constraint, some of the coefficients that were null in the SPCA basis within the groups become non-zero in the GSPCA basis. Since the group sparsity is independently achieved in the basis vectors, each vector is sparse for a different group of features and the complete basis may still end up using all groups of features.

\begin{figure}[!h]
\footnotesize
\centering
\subfigure[SPCA Basis]{\label{fig:pat1-JGSPCA}
\begin{minipage}[b]{0.3\linewidth}
\centering
\includegraphics[trim = 232pt 32pt 240pt 2pt, clip, width=0.85\linewidth]{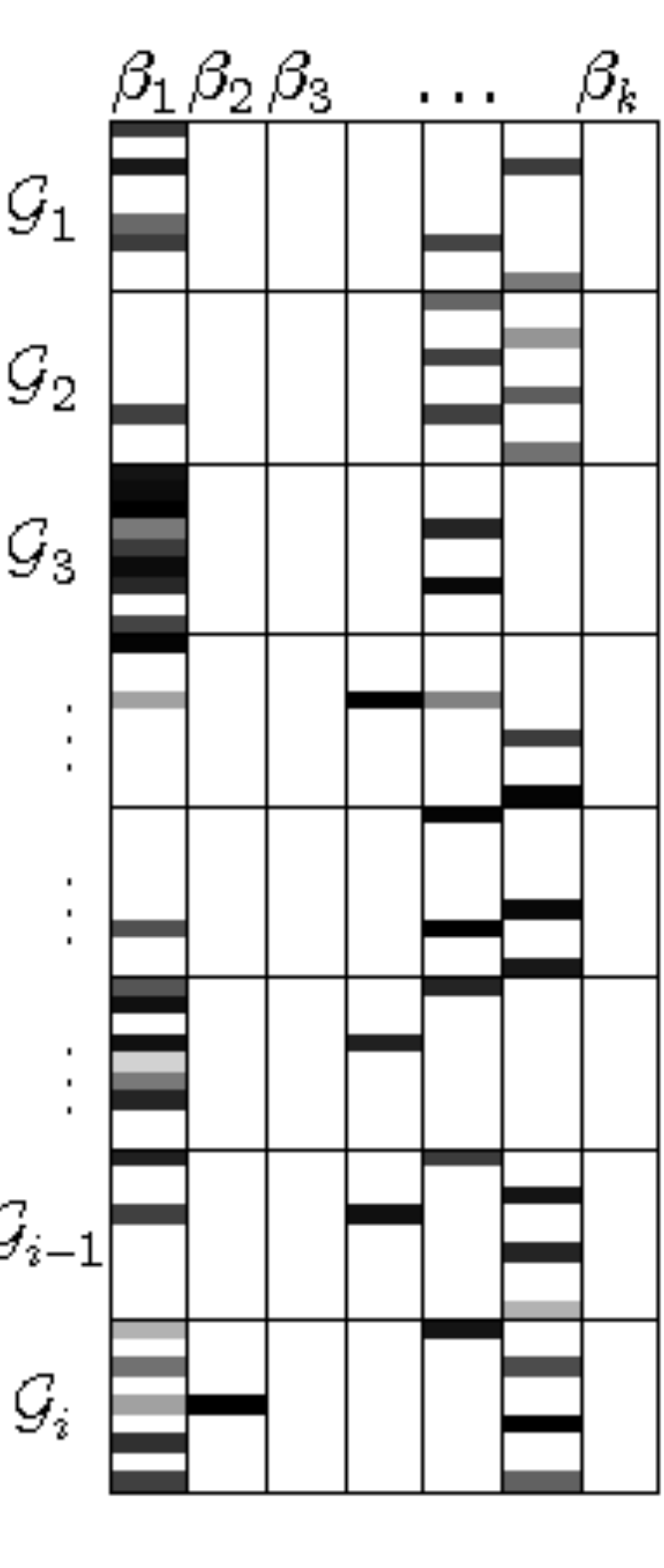}
\end{minipage}} \hfill
\subfigure[GSPCA Basis]{\label{fig:pat2-JGSPCA}
\begin{minipage}[b]{0.3\linewidth}
\centering
\includegraphics[trim = 232pt 32pt 240pt 2pt, clip, width=0.85\linewidth]{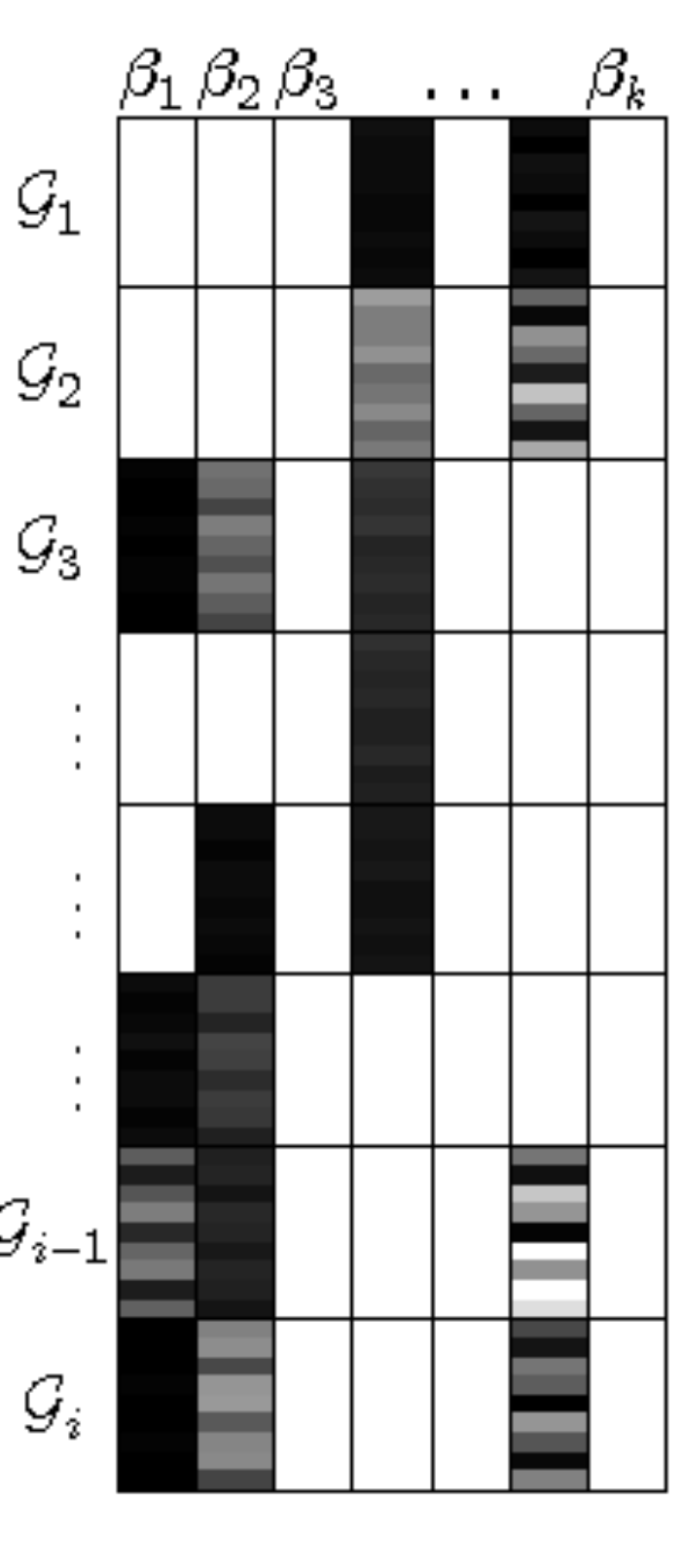}
\end{minipage}} \hfill
\subfigure[JGSPCA Basis]{\label{fig:pat3}
\begin{minipage}[b]{0.3\linewidth}
\centering
\includegraphics[trim = 232pt 32pt 240pt 2pt, clip, width=0.85\linewidth]{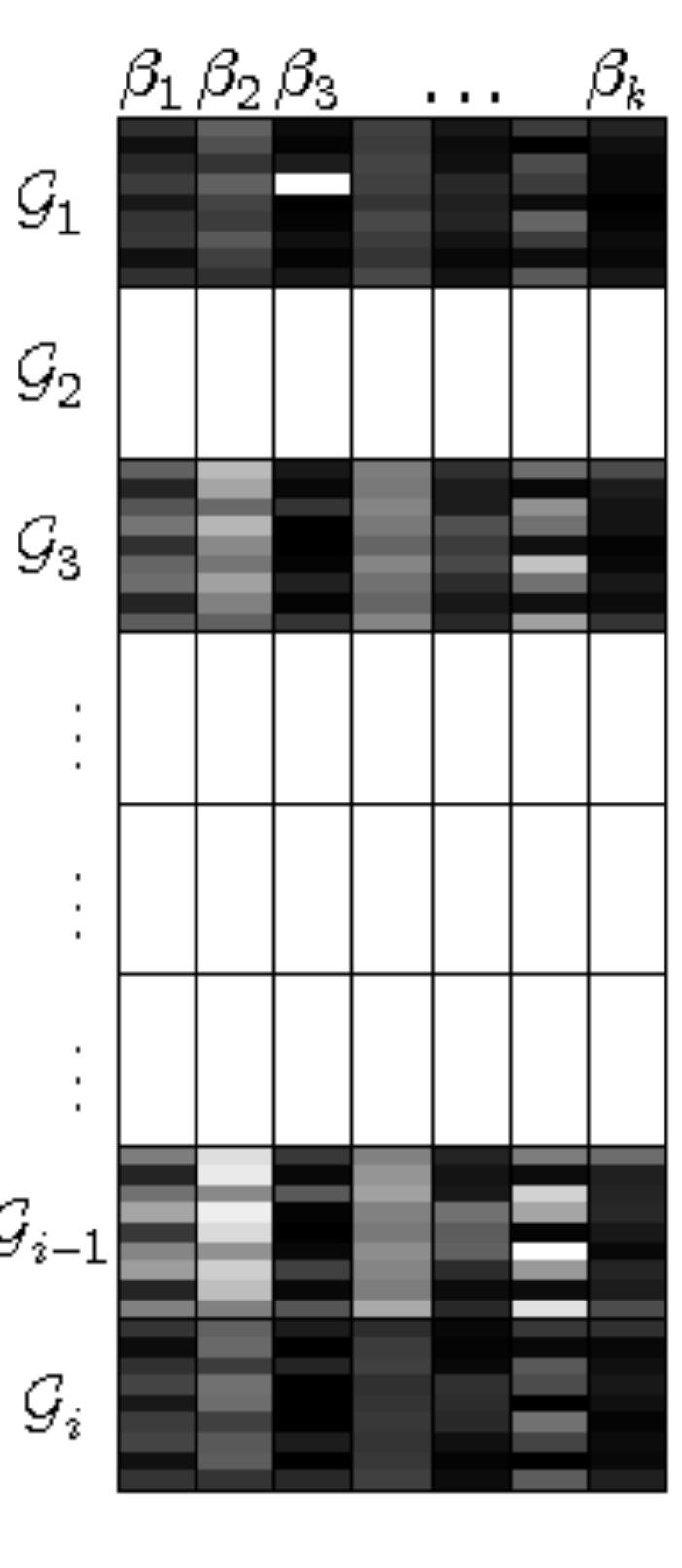}
\end{minipage}}
\caption[Basis vectors sparsity patterns]{This example illustrates the basis vectors $\bm{\beta}_j$ computed on a data $\mathbf{X}$ consisting of 8 feature groups of 9 features each ($g=8, p_i=9, k=7$). Dark rectangles are non-zero coefficients. The group sparsity applies across the groups of features in each basis vector, individually. The joint group sparsity ensures both sparsity among the groups and joint selection of groups across the basis vectors.}
\label{fig:patterns-JGSPCA}
\end{figure}

In several applications, it is desirable to perform feature selection such that the selected features explain the major variation of the data. This is particularly true for data consisting of large number of redundant features or where measurement of features is expensive. To achieve this goal, we expect all basis vectors $\bm{\beta}_j$ to end up using the same \emph{groups of features}. This kind of sparsity is called \emph{joint sparsity}~\cite{duarte2005joint,lee2012subspace}. Joint sparsity is neither considered by SPCA nor GSPCA, since they independently solve~\eqref{eq:SPCACriterion-JGSPCA} and~\eqref{eq:GSPCACriterion} for individual basis vectors $\bm{\beta}_j$. We propose to directly optimize for $\mathbf{B}$ to achieve joint sparsity while simultaneously achieving group sparsity. In other words, the coefficients corresponding to some groups of rows of $\mathbf{B}$ should altogether be null, as shown in Figure~\ref{fig:pat3}. Our proposed joint group sparsity can be obtained by imposing the following regularization penalty
\begin{equation}
\label{eq:ellF1}
\ell_{F_g,1}(\mathbf{B}) = \sum_{i=1}^{g} \eta_i \|\mathbf{B}^{\centerdot\mathcal{G}_i}\|_F~.
\end{equation}
The minimization of $\ell_1$-norm on the Frobenius norm of sub-basis $\mathbf{B}^{\centerdot\mathcal{G}_i}$ will force some of the sub-basis (group of rows of $\mathbf{B}$) to be null. This will result in joint group sparsity over the complete basis. The nullified groups directly correspond to the feature groups of $\mathbf{X}$ with minimum contribution in explaining the data. By including the joint group sparse regularization penalty~\eqref{eq:ellF1} in~\eqref{eq:SPCACriterion-JGSPCA}, the proposed Joint Group Sparse PCA criterion is obtained as

\begin{multline}
\label{eq:JGSPCACriterion-ellF1}
\underset{\hat{\mathbf{A}},\hat{\mathbf{B}}}{\arg \min} \|\mathbf{X}-\sum_{i=1}^{g}\mathbf{X}_{\centerdot\mathcal{G}_i}\mathbf{B}^{\centerdot\mathcal{G}_i}{\mathbf{A}^{\!\intercal}\|}_F^2+\lambda\sum_{i=1}^{g} \eta_g \|\mathbf{B}^{\centerdot\mathcal{G}_i}\|_F\\ \text{subject to} \qquad \mathbf{A}^{\!\intercal}\mathbf{A}=\mathbf{I}_k~,
\end{multline}
For sufficiently large values of $\lambda$, some group of rows of $\mathbf{B}$ will vanish, resulting in a jointly group sparse basis.

Although, the above formulation ensures a joint group sparse basis, simultaneous minimization for $\mathbf{A}$ and $\mathbf{B}$ makes the problem non-convex. If one of the two matrices is known, the problem becomes convex over the second unknown matrix. Hence, a locally convex solution of~\eqref{eq:JGSPCACriterion-ellF1} can be obtained by iteratively minimizing $\mathbf{A}$ and $\mathbf{B}$. Therefore, the joint group sparse PCA formulation in~\eqref{eq:JGSPCACriterion-ellF1} is dissociated into two independent optimization problems. In the first optimization problem, $\mathbf{A}$ is initialized with $\mathbf{V}$ obtained from~\eqref{eq:svd-JGSPCA} and the minimization under the joint group sparsity constraint on $\mathbf{B}$ is formulated as
\begin{equation}
\label{eq:JGSPCA-B}
\underset{\hat{\mathbf{B}}}{\arg \min} {\|\mathbf{XA}-\mathbf{XB}\|}_F^2+\lambda\sum_{i=1}^{g} \eta_i \|\mathbf{B}^{\centerdot\mathcal{G}_i}\|_F~,
\end{equation}
which is similar to a multi-task regularized regression problem~\cite{mairal2010network} with grouping constraints
\begin{equation}
\label{eq:multi-task-regress}
\underset{\hat{\mathbf{W}}}{\arg \min} {\|\mathbf{Q}-\mathbf{XW}\|}_F^2+ \psi(\mathbf{W})~,
\end{equation}
where $\mathbf{Q}=\mathbf{XA}$ is the response matrix, $\mathbf{W}=\mathbf{B}$ is the matrix of regression coefficients and $\psi$ is any convex matrix norm. An optimization problem of the form of~\eqref{eq:multi-task-regress} can be efficiently solved by proximal programming methods~\cite{jenatton2010proximal}.
\begin{theorem}\label{thm:X-XBAt}
The loss term ${\|\mathbf{X}-\sum_{i=1}^{g}\mathbf{X}_{\centerdot\mathcal{G}_i}\mathbf{B}^{\centerdot\mathcal{G}_i}\mathbf{A}^{\!\intercal}\|}_F^2$ in~\eqref{eq:JGSPCACriterion-ellF1} is equivalent to ${\|\mathbf{XA}-\mathbf{XB}\|}_F^2$ given $\mathbf{A}^{\!\intercal}\mathbf{A}=\mathbf{I}_k$.
\end{theorem}
\begin{proof}
The proof of this theorem follows from the specification of the summation $\sum_{i=1}^{g}\mathbf{X}_{\centerdot\mathcal{G}_i}\mathbf{B}^{\centerdot\mathcal{G}_i}\mathbf{A}^{\!\intercal}$.
\begin{lemma}
For non-overlapping feature groups, $\mathcal{G}_i\cap\mathcal{G}_j=\emptyset$ $\forall i\neq j$,
\begin{equation}
\sum_{i=1}^{g}\mathbf{X}_{\centerdot\mathcal{G}_i}\mathbf{B}^{\centerdot\mathcal{G}_i} = \mathbf{X}\mathbf{B}~.
\end{equation}
\end{lemma}
\begin{proof}
The proof follows from the expansion of the sum
\begin{align}
\sum_{i=1}^{g}\mathbf{X}_{\centerdot\mathcal{G}_i}\mathbf{B}^{\centerdot\mathcal{G}_i} &= \mathbf{X}_{\centerdot\mathcal{G}_1}\mathbf{B}^{\centerdot\mathcal{G}_1} + \mathbf{X}_{\centerdot\mathcal{G}_2}\mathbf{B}^{\centerdot\mathcal{G}_2} \ldots \mathbf{X}_{\centerdot\mathcal{G}_g}\mathbf{B}^{\centerdot\mathcal{G}_g}\\
&=[\mathbf{X}_{\centerdot\mathcal{G}_1} \ldots \mathbf{X}_{\centerdot\mathcal{G}_g}] {[{(\mathbf{B}^{\centerdot\mathcal{G}_1})}^\intercal \ldots {(\mathbf{B}^{\centerdot\mathcal{G}_g})}^\intercal]}^\intercal\\
&=\mathbf{X}{(\mathbf{B}^\intercal)}^\intercal=\mathbf{XB}
\end{align}
\end{proof}
Lemma~1 results in simplification of the loss function to~$\|\mathbf{X}-\mathbf{XBA}^{\!\intercal}\|_F^2$. Given orthogonality constraints on $\mathbf{A}$, the minimization of the simplified loss function will require the difference
\begin{align}
\mathbf{X}-\mathbf{XBA}^{\!\intercal} &= \mathbf{C} & \mathbf{C} \in \mathbb{R}^{n\times p} \text{ is a residue matrix, } \mathbf{C}\neq0\\
\mathbf{XA}-\mathbf{XBA}^{\!\intercal}\mathbf{A} &=\mathbf{CA} &\text{multiplying by }\mathbf{A} \\
\mathbf{XA}-\mathbf{XB} &=\mathbf{CA}&\text{since } \mathbf{A}^{\!\intercal}\mathbf{A}=\mathbf{I}_k\\
\mathbf{XA}-\mathbf{XB} &=\mathbf{E}&\text{since } \mathbf{A} \text{ is fixed, } \mathbf{E}=\mathbf{CA} \text{ is also constant}
\label{eq:cor1}
\end{align}
\end{proof}
Once a solution for $\mathbf{B}$ is found in~\eqref{eq:JGSPCA-B}, the next step is to solve the second problem, i.e., optimizing with respect to $\mathbf{A}$. For a known $\mathbf{B}$ the regularization penalty in~\eqref{eq:JGSPCACriterion-ellF1} becomes irrelevant for the optimization with respect to $\mathbf{A}$. Therefore, the following objective function is required to be minimized
\begin{equation}
\label{eq:JGSPCA-A}
\underset{\hat{\mathbf{A}}}{\arg \min} {\|\mathbf{X}-\mathbf{X}\mathbf{B}\mathbf{A}^{\!\intercal}\|}_F^2 \qquad \text{subject to}\, \mathbf{A}^{\!\intercal}\mathbf{A}=\mathbf{I}_k~.
\end{equation}
A closed form solution for minimizing~\eqref{eq:JGSPCA-A} can be obtained by computing a reduced rank procrustes rotation~\cite{zou2006sparse}.
\begin{theorem}\label{thm:procrustes}
$\hat{\mathbf{A}} = \mathbf{\dot{U}\dot{V}}^{\intercal}$ is the closed form solution of~\eqref{eq:JGSPCA-A}.
\end{theorem}
\begin{proof}
We first expand the Frobenius norm
\begin{align}
{\|\mathbf{X}-\mathbf{XB}\mathbf{A}^{\!\intercal}\|}_F^2 &=\textrm{Tr}((\mathbf{X}-\mathbf{XB}\mathbf{A}^{\!\intercal})^{\!\intercal}(\mathbf{X}-\mathbf{XB}\mathbf{A}^{\!\intercal}))\\
&=\textrm{Tr}((\mathbf{X}^{\intercal}-\mathbf{AB}^{\intercal}\mathbf{X}^{\intercal})(\mathbf{X}-\mathbf{XB}\mathbf{A}^{\!\intercal}))\\
&=\textrm{Tr}(\mathbf{X}^{\intercal}\mathbf{X})-\textrm{Tr}(\mathbf{AB}^{\intercal}\mathbf{X}^{\intercal}\mathbf{X})-\textrm{Tr}(\mathbf{X}^{\intercal}\mathbf{XBA}^{\!\intercal})+\textrm{Tr}(\mathbf{AB}^{\intercal}\mathbf{X}^{\intercal}\mathbf{XBA}^{\!\intercal})\\
&=\textrm{Tr}(\mathbf{X}^{\intercal}\mathbf{X})-2\textrm{Tr}(\mathbf{X}^{\intercal}\mathbf{XBA}^{\!\intercal})+\textrm{Tr}(\mathbf{B}^{\intercal}\mathbf{X}^{\intercal}\mathbf{XB})
\label{eq:simplify-frob}
\end{align}
The middle term in \eqref{eq:simplify-frob} arises due to the fact that the trace of a matrix and its transpose are equal. The last term has been simplified to $\mathbf{B}^{\intercal}\mathbf{X}^{\intercal}\mathbf{XB}$ due to orthogonality constraints on $\mathbf{A}$. Therefore, minimizing the loss function requires maximizing the trace of $\mathbf{X}^{\intercal}\mathbf{XBA}^{\!\intercal}$, since it is negative. Assume the SVD of $\mathbf{X}^{\intercal}\mathbf{XB}$ is $\dot{\mathbf{U}}\dot{\mathbf{S}}\dot{\mathbf{V}}^{\intercal}$, then
\begin{align}
\textrm{Tr}(\mathbf{X}^{\intercal}\mathbf{XBA}^{\!\intercal})&=\textrm{Tr}(\dot{\mathbf{U}}\dot{\mathbf{S}}\dot{\mathbf{V}}^{\intercal}\mathbf{A}^{\!\intercal})\\
 &=\textrm{Tr}(\dot{\mathbf{V}}^{\intercal}\mathbf{A}^{\!\intercal}\dot{\mathbf{U}}\dot{\mathbf{S}})
\label{eq:simplify-proc}
\end{align}
Equation~\eqref{eq:simplify-proc} is drawn from the cyclic nature of the trace of a product of matrices. Since, $\dot{\mathbf{S}}$ is a diagonal matrix, $\textrm{Tr}(\dot{\mathbf{V}}^{\intercal}\mathbf{A}^{\!\intercal}\dot{\mathbf{U}}\dot{\mathbf{S}})$ is maximized when the diagonal of $\dot{\mathbf{V}}^{\intercal}\mathbf{A}^{\!\intercal}\dot{\mathbf{U}}$ is maximum. This is true when $\dot{\mathbf{V}}^{\intercal}\mathbf{A}^{\!\intercal}\dot{\mathbf{U}}= \mathbf{I}$. Therefore, after simplification, $\mathbf{A}^{\!\intercal} \dot{\mathbf{U}}=\dot{\mathbf{V}}$ or $\mathbf{A}=\dot{\mathbf{U}}\dot{\mathbf{V}}^\intercal$ is the closed form solution of~\eqref{eq:JGSPCA-A}.
\end{proof}
The alternating optimization process is repeated until convergence or until a specified number of iterations is reached. Algorithm~\ref{alg:JGSPCA} summarizes the procedure for Joint Group Sparse PCA.

\begin{algorithm}[h]  
\caption{Joint Group Sparse PCA}       
\label{alg:JGSPCA}                     
\begin{algorithmic}                    
\Require $\mathbf{X} \in \mathbb{R}^{n \times p}, \{\mathcal{G}_i\}_{i=1}^g, \eta_i, \lambda, j_{\textrm{max}}$
\State $\textbf{Initialize:}~ j \gets 1,$ converge $\gets$ \texttt{false}
\State $\mathbf{USV}^{\intercal} \gets \mathbf{X}$ 
\State $\mathbf{A} \gets \mathbf{V}_{\centerdot \{1:k\}}$ 
\While {$j\le j_{\textrm{max}} \wedge \neg \textrm{converge}$}
\State $\hat{\mathbf{B}} \gets \underset{\mathbf{B}}{\min} {\|\mathbf{XA}-\mathbf{XB}\|}_F^2 + \lambda {\|\mathbf{B}\|}_{F_g,1}$ 
\State $\dot{\mathbf{U}} \dot{\mathbf{S}} \dot{\mathbf{V}}^{\intercal} \gets {\mathbf{X}}^{\intercal} \mathbf{X}\hat{\mathbf{B}}$
\State $\hat{\mathbf{A}} \gets \dot{\mathbf{U}}\dot{\mathbf{V}}^{\intercal}$ 
\If {$\|\mathbf{B}-\hat{\mathbf{B}}\|_F < \epsilon $} 
\State $\textrm{converge} \gets \texttt{true}$
\Else
\State $\mathbf{B} \gets \hat{\mathbf{B}}, \mathbf{A} \gets \hat{\mathbf{A}}$ 
\State $ j \gets j+1$ 
\EndIf
\EndWhile
\Ensure $\hat{\mathbf{A}},\hat{\mathbf{B}}$
\end{algorithmic}
\end{algorithm}

\subsection{Model Tree Search}

It is imperative that the sparsity of a basis is dependent on the regularization penalty parameter $\lambda$. The higher the value of $\lambda$, the lower the cardinality of the basis. We define the group cardinality $r$ of a basis $\mathbf{B}$ as its number of non-zero group of rows (which directly corresponds to the number of feature groups, $r \in [0,g]$). The obtained joint group sparse basis $\mathbf{B}$ has therefore $r$ non-zero feature groups. A conventional grid search over a range of $\lambda$ offers an ill-posed problem, since the group cardinality of a model for a particular value of $\lambda$ is not known apriori.

We propose an intuitive tree search that seeks for the value of $\lambda$ to achieve a model with desired group cardinality over a range $[r_{\textrm{min}},r_{\textrm{max}}]$. Each node of the tree corresponds to a different value of $\lambda$. The tree search explores intelligently selected nodes to identify further nodes that lead to a desired model $\mathcal{M}_r$. We briefly explain the major steps involved in a model tree search (Algorithm~\ref{alg:HRFP}).

\subsubsection{Initialize Root Nodes}

The exploration of models in the tree is executed as illustrated in Figure~\ref{fig:tree}. The tree search begins with a populating values of $\lambda$ at the root nodes. The values are linearly sampled in the range $[\lambda_{\textrm{max}},\lambda_{\textrm{min}}]$.
Then, all root nodes are initialized as active $\mathcal{I}_j=1 \forall j$. At this stage it is relatively cost effective to activate all nodes because the root level has the least number of nodes.

\begin{landscape}
\begin{figure}
\footnotesize
\centering
\includegraphics[width=1\linewidth]{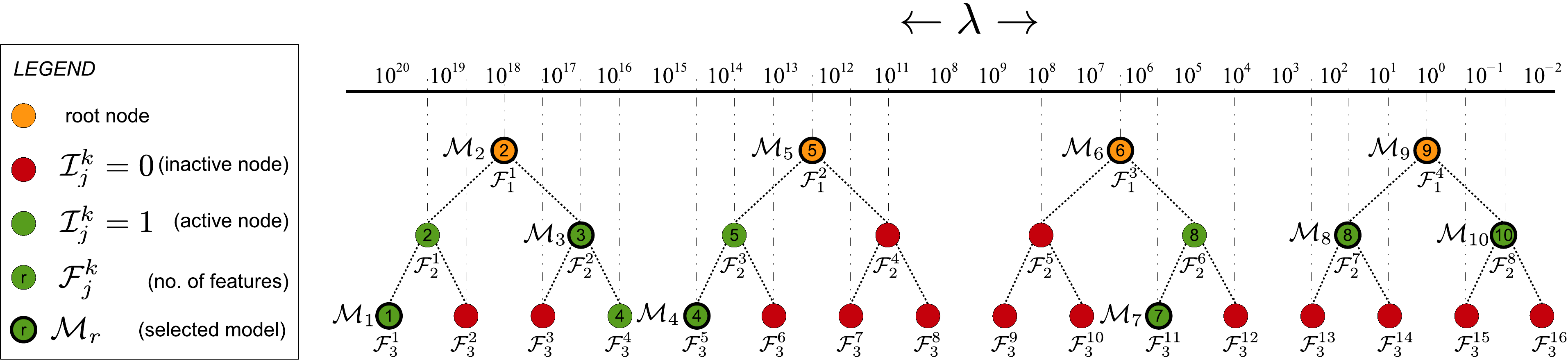}
\caption[An illustration of model tree search]{A three level tree search for compressive sensing models with feature $r=\{1,2,...,10\}$ is presented. The yellow circles denote the root nodes, green circles are the active nodes, whereas the red circles are the inactive nodes. The $k^\textrm{th}$ node at the $j^\textrm{th}$ level is associated to a $\lambda$ parameter value. The tree is initialized with four root nodes which return models for $r=\{2,5,6,9\}$. At the next level, nodes 4 and 5 are deactivated since no new models are expected in between. Observe that fewer features are selected with a higher penalty value and vice versa. Moreover, the number of active nodes continues to decrease at each level, thereby reducing the computational cost of search. When the same number of features occur at multiple nodes, the node with lower $\lambda$ (and lower reconstruction error) is selected. As early as all models corresponding to $r$ are explored, further exploration of the tree is discontinued.}
\label{fig:tree}
\end{figure}
\end{landscape}

\subsubsection{Compute Model For Each Node}

A model $\mathcal{M}$ is learned sequentially at each node $k$ with the node parameter value $\lambda_j^k$, where $j$ is the current level of tree. If the number of features ($r$) at a node is equal to the maximum number of features, the following nodes at the same level are deactivated. This is because any further node search would result in models with the same number of selected features.

\subsubsection{Update Child Nodes}

The number of nodes at each level is equal to twice the number of nodes at the preceding level. Therefore, in between consecutive parent nodes are exactly two child nodes. There can be two possibilities for updating the activation of child nodes.

\begin{enumerate}
\item If the difference between number of features on consecutive parent nodes is more than one, it indicates that a further solution may possibly exist in the child nodes. (both child nodes activated)
\item If the difference between number of features on consecutive parent nodes is less than or equal to one, there is no advantage in searching over their child nodes for further solution. (both child nodes deactivated)
\end{enumerate}

\subsubsection{Check Convergence}

If the current level is lower than the depth of tree, steps 2-3 are repeated, otherwise the search is terminated. The convergence of the tree search is almost always guaranteed, if the relationship between $r$ and $\lambda$ is considered monotonically inverse. In occasional instances, this relationship may not hold true which translates into tree search anomalies. This phenomenon is presented in Figure~\ref{fig:anomalies} which illustrates the relationship between $r$ and $\lambda$. We observed that such anomalies are data dependent and can be catered for by cross-validated model search.

\begin{figure}[!h]
\footnotesize
\centering
\vspace{-5pt}
\subfigure{\includegraphics[trim = 3pt 2pt 30pt 2pt, clip, width=0.37\linewidth]{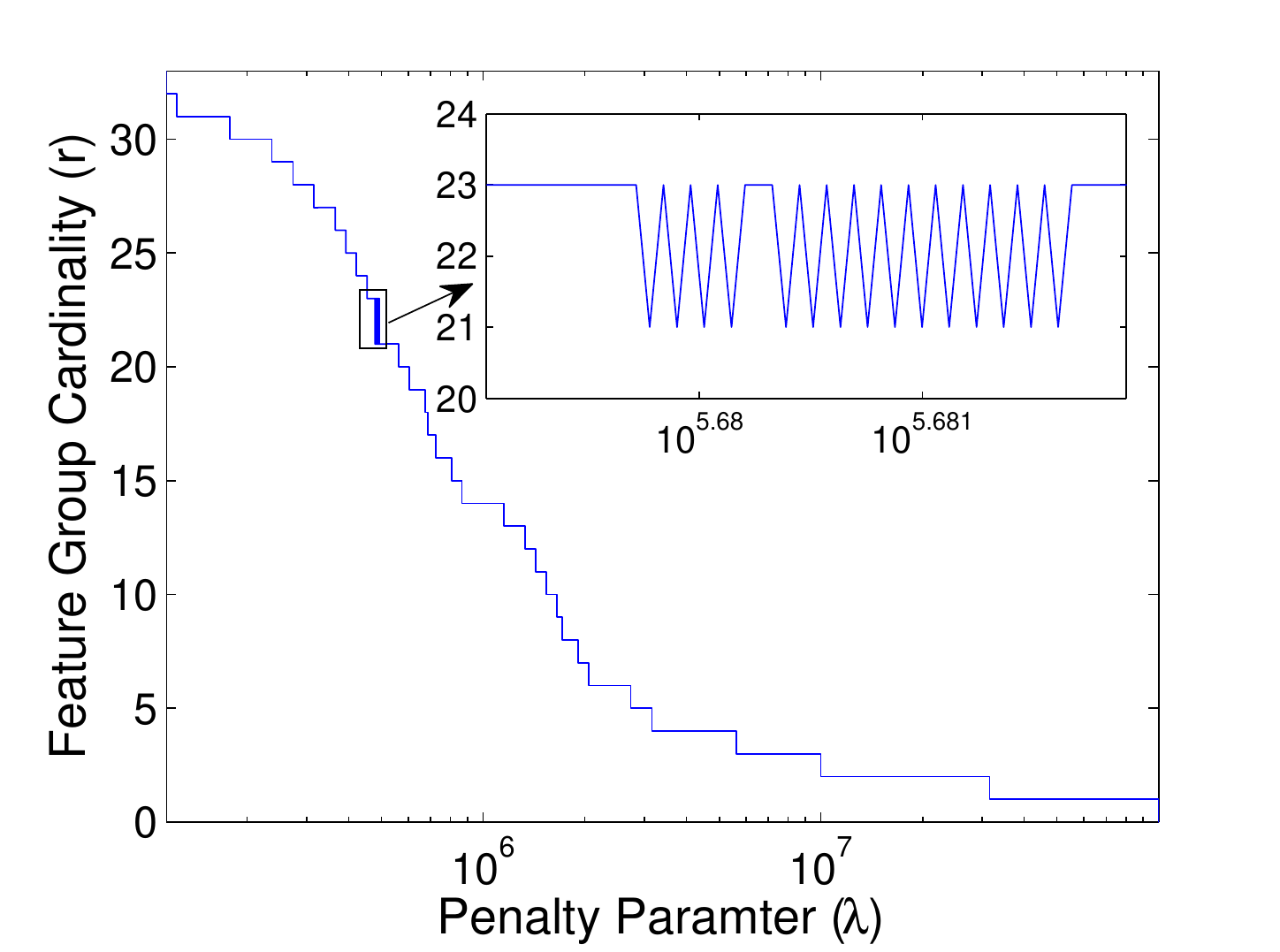}}
\subfigure{\includegraphics[trim = 3pt 2pt 30pt 2pt, clip, width=0.37\linewidth]{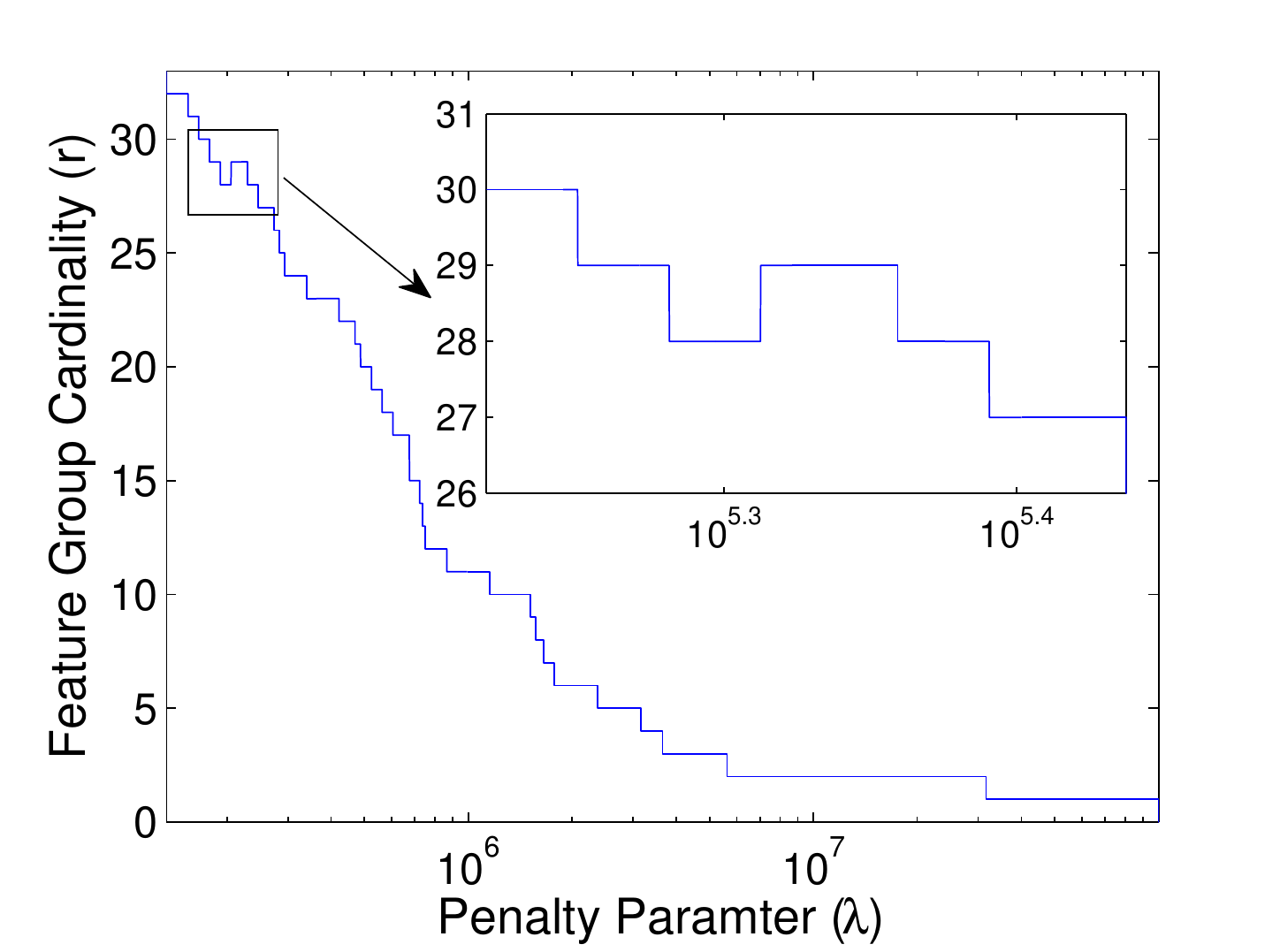}}
\caption[Relationship between $r$ and $\lambda$]{(a) A model with the required feature group cardinality is not found between two $\lambda$. Here, $\mathcal{M}_{22}$ is not found between $\mathcal{M}_{21}$ and $\mathcal{M}_{23}$ . (b) A model with higher group cardinality is found between two monotonically increasing $\lambda$ values. Here, $\mathcal{M}_{29}$ is found between $\lambda=10^{5.3}$ and $\lambda = 10^{5.4}$.}
\label{fig:anomalies}
\end{figure}

\begin{algorithm}[!h]  
\caption{Model Tree Search}          
\label{alg:HRFP}                           
\begin{algorithmic}                    
\Require $\mathbf{X} \in \mathbb{R}^{n \times p}$ \Comment {data matrix}
\State $\{\mathcal{G}_i\}_{i=1}^g$ \Comment {feature group indices sets}
\State $\{\lambda_j\}_{j=1}^d, \lambda_j \in \mathbb{R}^{2^{j-1}b}$ \Comment {regularization parameter values}
\State $d,b$ \Comment {depth of tree, number of root nodes}
\Ensure${\{\mathcal{M}_i\}}_{i=1}^g$ \Comment {set of learned models on $i$ features}
\State \textbf{Initialize:}
\State${\{\mathcal{M}_i\}}_{i=1}^g \gets \emptyset$
\State $\mathcal{I}_j \in \mathbb{R}^{2^{j-1}b}, \mathcal{I}_1 \gets \mathbf{1} $ \Comment{activate root nodes}
\State ${\{\mathcal{F}_j\}}_{j=1}^d, \mathcal{F}_j \in \mathbb{R}^{2^{j-1}b}$ \Comment {selected features at each node}
\State $j \gets 1, r \gets 0$ \Comment {level counter, no. of features in model}
\Repeat \Comment {for each level}
\State $k \gets 1, h \gets \sum \mathcal{I}_j$ \Comment {node counter, no. of active nodes}
\While {$k \le h \wedge r \neq g$}
\State $\mathbf{A},\mathbf{B} \gets $\textsc{Algorithm1}($\mathbf{X}, \lambda_j^k$)
\State $r \gets |\{i | \quad \|\mathbf{B}^{\centerdot\mathcal{G}_i}\|_{F,1}\neq 0\}|$ \Comment {find selected features}
\If {$r>0$} \State $\mathcal{F}_j^k \gets r, \mathcal{M}_r \gets \{\mathbf{A},\mathbf{B}\}$ \Comment {save model}
\EndIf
\State $k \gets k+1$
\EndWhile
\State $\mathcal{I}_{j+1}, {\mathcal{F}}_{j+1} \gets $ \textsc{UpdateChildNodes}($\mathcal{F}_j$)
\State $j \gets j+1$
\State $q \gets |\{i | \quad \mathcal{M}_i \neq \emptyset \}|$ \Comment {number of explored models}
\Until {$j \le d \vee q \neq g$}
\State {--------------------------------------------------------------------}
\Function{UpdateChildNodes}{$\mathcal{F}_j$}
\State \textbf{Initialize:} $\mathcal{I}_{j+1} \gets \mathbf{0},{\mathcal{F}}_{j+1} \gets \emptyset$
\For {$k=1$ to $h-1$}
\If {$\mathcal{F}_j^{k+1}-\mathcal{F}_j^k>1$} \Comment {condition 1}
\State ${\mathcal{F}}_{2k-1}^{j+1} \gets \mathcal{F}_j^k$ \Comment {replicate to left child node}
\State $\mathcal{I}_{j+1}^{2k:2k+1} \gets 1$ \Comment {activate child nodes}
\Else \Comment {condition 2}
\State ${\mathcal{F}}_{j+1}^{2k-1:2k} \gets \mathcal{F}_j^k$ \Comment {replicate to both children}
\EndIf
\EndFor \\
\Return $\mathcal{I}_{j+1}, {\mathcal{F}}_{j+1}$
\EndFunction
\end{algorithmic}
\end{algorithm}

\clearpage

\subsection{Implementation}

We used the Fast-Iterative Shrinkage and Thresholding Algorithm (FISTA) library~\cite{jenatton2010proximal} to solve the optimization problems in~\eqref{eq:SPCACriterion-JGSPCA}, \eqref{eq:GSPCACriterion} and \eqref{eq:JGSPCACriterion-ellF1}. Note that we modified the FISTA library to solve for the $\ell_{F_g,1}$ joint group regularization penalty in~\eqref{eq:ellF1}. All source codes for the algorithms (including modified FISTA library) will soon be released.

\section{Experiments}
\label{sec:exp-JGSPCA}

\subsection{Evaluation Criteria}

The data matrix $\mathbf{X}$ is first created by sampling non-overlapping spatio-spectral volumes of dimension $\sqrt{p_i} \times \sqrt{p_i} \times g, p_i=9~\forall~i$ (after vectorizing) from all training hyperspectral images, where $g$ is the total number of bands. A model is learned from the training data using Algorithm~\ref{alg:HRFP}. At most $g$ models are learned with each algorithm, one for each $r=1,2,...,g$ number of bands. To evaluate compressive sensing performance of the $r^{\textrm{th}}$ learned model $\mathcal{M}_r$ with orthonormal basis $\mathbf{A}$ and the corresponding sparse basis $\mathbf{B}$, the reconstruction error is computed as
\begin{equation}
e_r = \frac{{\|\mathbf{XV}_{\centerdot\{1:k\}}\mathbf{V}_{\centerdot\{1:k\}}^\intercal - \mathbf{XB}\mathbf{A}^{\!\intercal}\|}_F}
                        {{\|\mathbf{XV}_{\centerdot\{1:k\}}\mathbf{V}_{\centerdot\{1:k\}}^\intercal\|}_F}~,
\end{equation}
where $k$ is the number of basis vectors. We choose $k=p$ for all experiments as we are interested in reconstructing complete hyperspectral image cubes. This makes $\mathbf{V}_{\centerdot\{1:k\}}\mathbf{V}_{\centerdot\{1:k\}}^\intercal=1$.

\subsection{Databases}

Hyperspectral imaging has progressed with the recent advances in electronically tunable filters~\cite{gat2000imaging} based hyperspectral cameras for capturing nearby objects. We have used publicly available datasets of indoor/outdoor scenes and faces. A few sample hyperspectral images are shown in Figure~\ref{fig:sample}. A summary of specifications for all hyperspectral datasets used in the experiments is provided in Table~\ref{tab:datasets}. A brief description of each dataset used in our experiments is as follows.


\subsubsection{Harvard Scene Dataset}

A hyperspectral image database of 50 indoor and outdoor scenes under daylight illumination~\cite{chakrabarti2011statistics}. The images were captured using a commercial grade hyperspectral camera (Nuance FX, CRI Inc.), which is based on a liquid crystal tunable filter design. The dataset consists of a diverse range of objects, materials and structures and is a good representative of real world spatio-spectral interactions. The training and testing dataset consist of 10 and 40 images, respectively. All images were spatially resized to $105 \times 141$ pixels.

\newcolumntype{C}{>{\centering\arraybackslash}X}
\begin{table*}[t]
\caption[An overview of hyperspectral image databases]{An overview of hyperspectral image databases used in the experiments. Our newly developed UWA face database is a low noise hyperspectral face database in the visible range.}
\label{tab:datasets}
\centering
\begin{tabularx}{1\linewidth}{|l|C|C|C|C|}
  \hline
  \textbf{Database }    &\textbf{Harvard} &\textbf{CAVE}    & \textbf{CMU}   & \textbf{UWA}    \\ 
  \hline
  Spectral Range (nm)   & 420 - 720       & 410 - 710       & 450 - 1090     & 400 - 720       \\ 
  Number of Bands       & 31 (VIS)        & 31 (VIS)        & 65 (VIS-NIR)   & 33 (VIS)        \\ 
  Spatial Resolution    & 1392$\times$1040& 512$\times$512  & 640$\times$480 & 1024$\times$1024\\ 
  Images/Subjects       & 50              & 32              & 48             & 70              \\ 
  Acquisition Time      & 60 sec          & -               & 8 sec          & 6 sec           \\ 
  Noise Grade           & Low             & Low             & High           & Low             \\ 
  \hline
\end{tabularx}
\end{table*}

\subsubsection{CAVE Scene Dataset}

The CAVE multispectral image database contains true spectral reflectance images of 32 scenes consisting of a variety of objects in an indoor setup~\cite{yasuma2010generalized}. It has 31 band hyperspectral images (420-720nm with 10nm steps) at a resolution of $512 \times 512$ pixels. All images contain the true spectral reflectance of a scene i.e., they are corrected for ambient illumination. We used 20 images for training and 10 images for testing. Each band was spatially resized to $120 \times 120$ pixels.

\subsubsection{CMU Face Dataset}

The CMU hyperspectral face database~\cite{denes2002hyperspectral} contains facial images of 48 subjects captured in multiple sessions over a period of about two months. The images cover both visual and near infrared spectrum and span the spectral range from 450nm to 1090 nm (at 10nm steps). The data was obtained using a prototype CMU-developed spectro-polarimetric camera mainly comprising of Acousto Optical Tunable Filter. For experiments, single sample per subject is used for the training set and the remaining samples make the test set. Specifically, 48 samples were used for training and 103 for testing. All faces are spatially resized to $24\times 21$ pixels after normalization.

%

\begin{figure}[t]
\footnotesize
\centering
\subfigure[Harvard Data]{\includegraphics[width=1\linewidth]{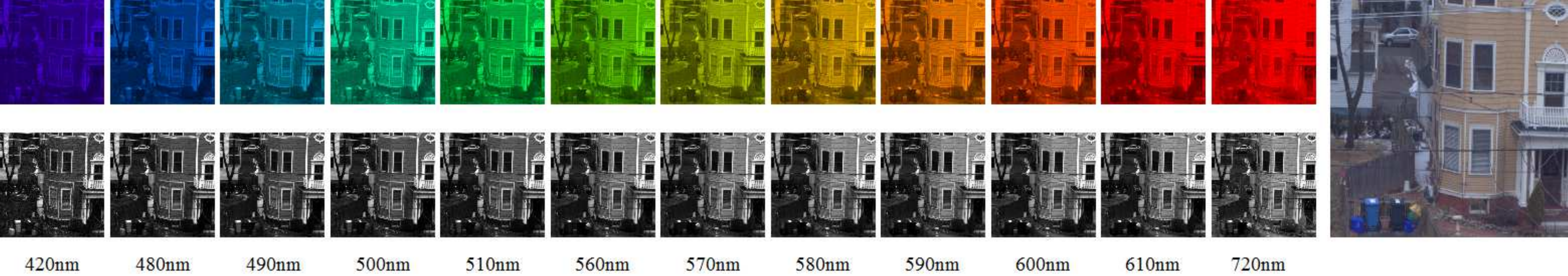}}
\subfigure[CAVE Data]{\includegraphics[width=1\linewidth]{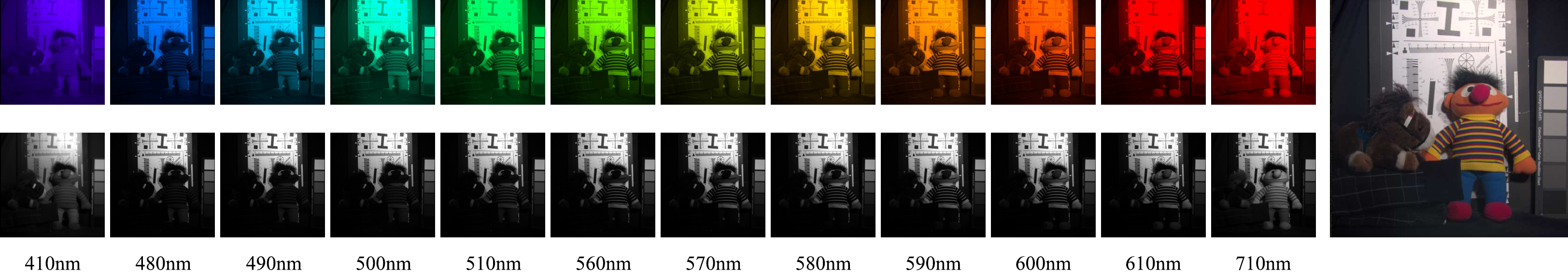}}
\subfigure[CMU Data]{\includegraphics[width=1\linewidth]{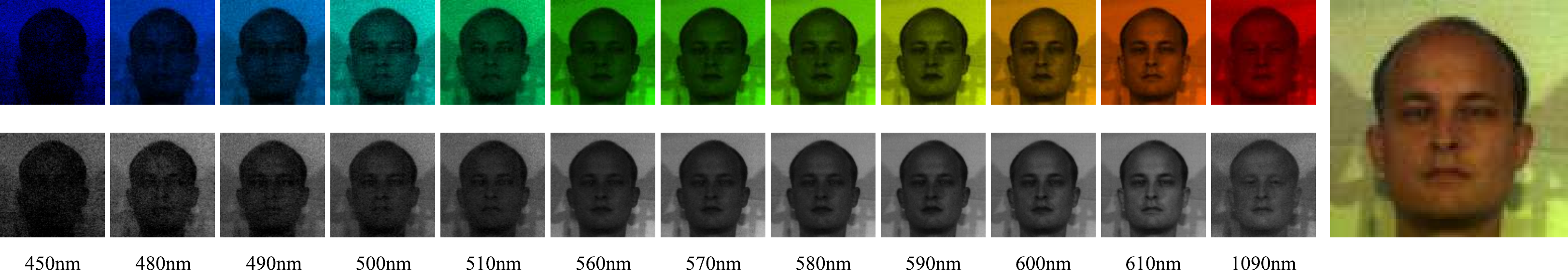}}
\subfigure[UWA Data]{\includegraphics[width=1\linewidth]{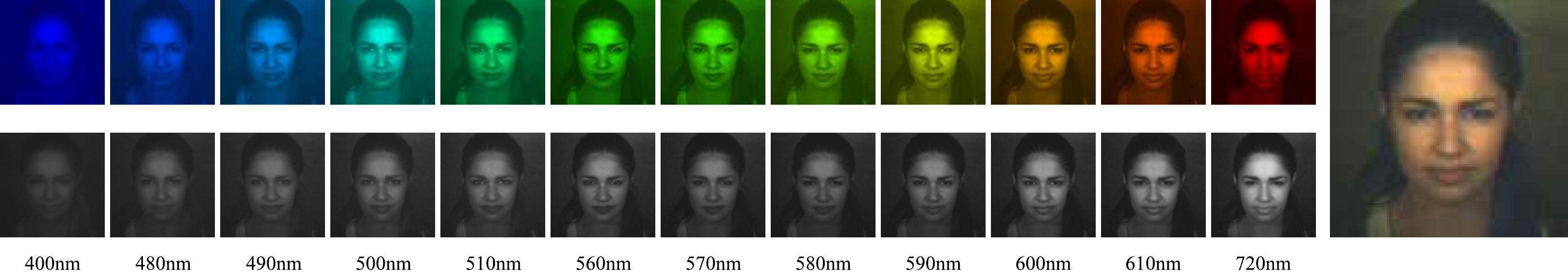}}
\caption[Sample hyperspectral images in different datasets]{Sample hyperspectral images in different datasets. Each image is shown as a series of bands in pseudo color and grayscale (only a subset of bands is shown here). Also shown is their corresponding RGB rendered image}
\label{fig:sample}
\vspace{-10pt}
\end{figure}

\subsubsection{UWA Face Dataset}

The hyperspectral face database collected in our lab contains 110 hyperspectral images of 70 subjects of different ethnicity, gender and age. Each subject was imaged in different sessions, separated by a duration between a week to two months. The system consists of a monochrome machine vision camera with a focusing lens (1:1.4/25mm) followed by a Liquid Crystal Tunable Filter (LCTF) which is tunable in the range of 400-720 nm. The average tuning time of the filter is 50 ms. The filter bandwidth, measured in terms of the \emph{Full Width at Half Maximum (FWHM)} is 7 to 20nm which varies with the center wavelength. The scene is illuminated by twin-halogen lamps on both sides of the subject. The illumination is left partially uncontrolled as it is mixed with indoor lights and occasionally daylight, varying with the time of image capture. For spectral response calibration, the white patch from a standard 24 patch color checker was utilized.

The training and testing sets consist of 70 and 40 images, respectively. The database will soon be made available for research~\footnote{UWA Hyperspectral Face Database:\\http://www.csse.uwa.edu.au/\%7Eajmal/databases.html}.

%

\section{Results and Discussion}
\label{sec:results}

\subsection{Compressed Hyperspectral Imaging}

In the first experiment, we examine the compressive sensing performance of all algorithms (SPCA, GSPCA and JGSPCA) in terms of reconstruction error. In the following text, we refer a feature as the \emph{band} of a hyperspectral image. When the number of bands in a model increases, the reconstruction error should decrease. We expect an algorithm to be relatively better for compressed sensing if it achieves lower reconstruction error with fewer bands sensed.

The reconstruction error on the test data with different algorithms is provided in Figure~\ref{fig:reconError}. Interesting results are obtained for the reconstruction errors on the Harvard and CAVE scene datasets. The first few bands equally explain the data with either SPCA, GSPCA or JGSPCA. When the more bands are added into the model, significant improvement in the reconstruction error is achieved with JGSPCA. We observe that GSPCA alone is only slightly better than SPCA, whereas the JGSPCA consistently achieves lower reconstruction error and outperforms both SPCA and GSPCA. The results on hyperspectral face datasets are slightly different from the datasets of scenes. The JGSPCA consistently outperforms SPCA and GSPCA on both CMU and UWA dataset. It reconstructs the data with lower error from the first band up until the last band on UWA dataset.

\begin{figure}[t]
\footnotesize
\centering
\includegraphics[trim = 3pt 2pt 30pt 2pt, clip, width=0.44\linewidth]{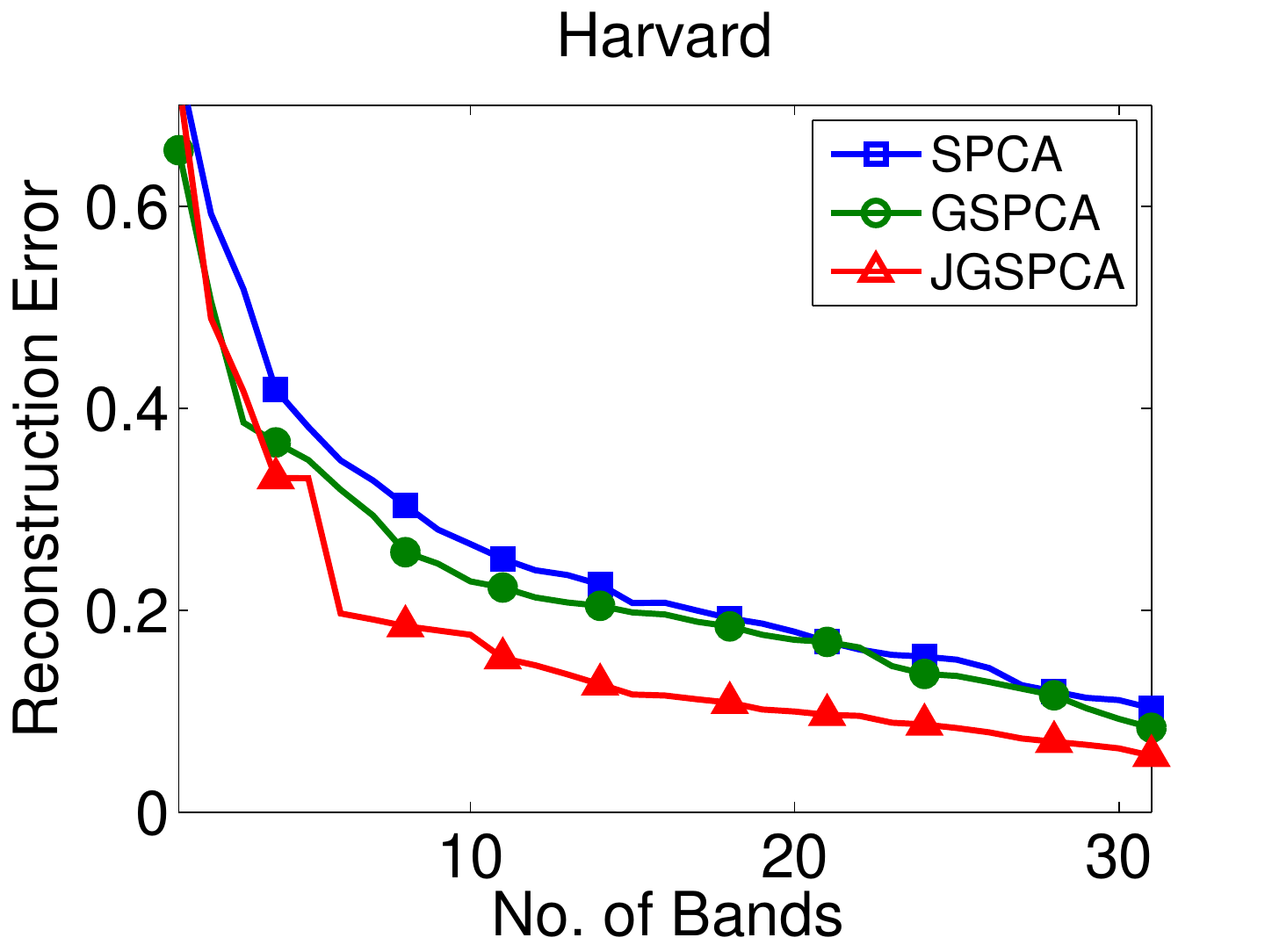} \hfill
\includegraphics[trim = 3pt 2pt 30pt 2pt, clip, width=0.44\linewidth]{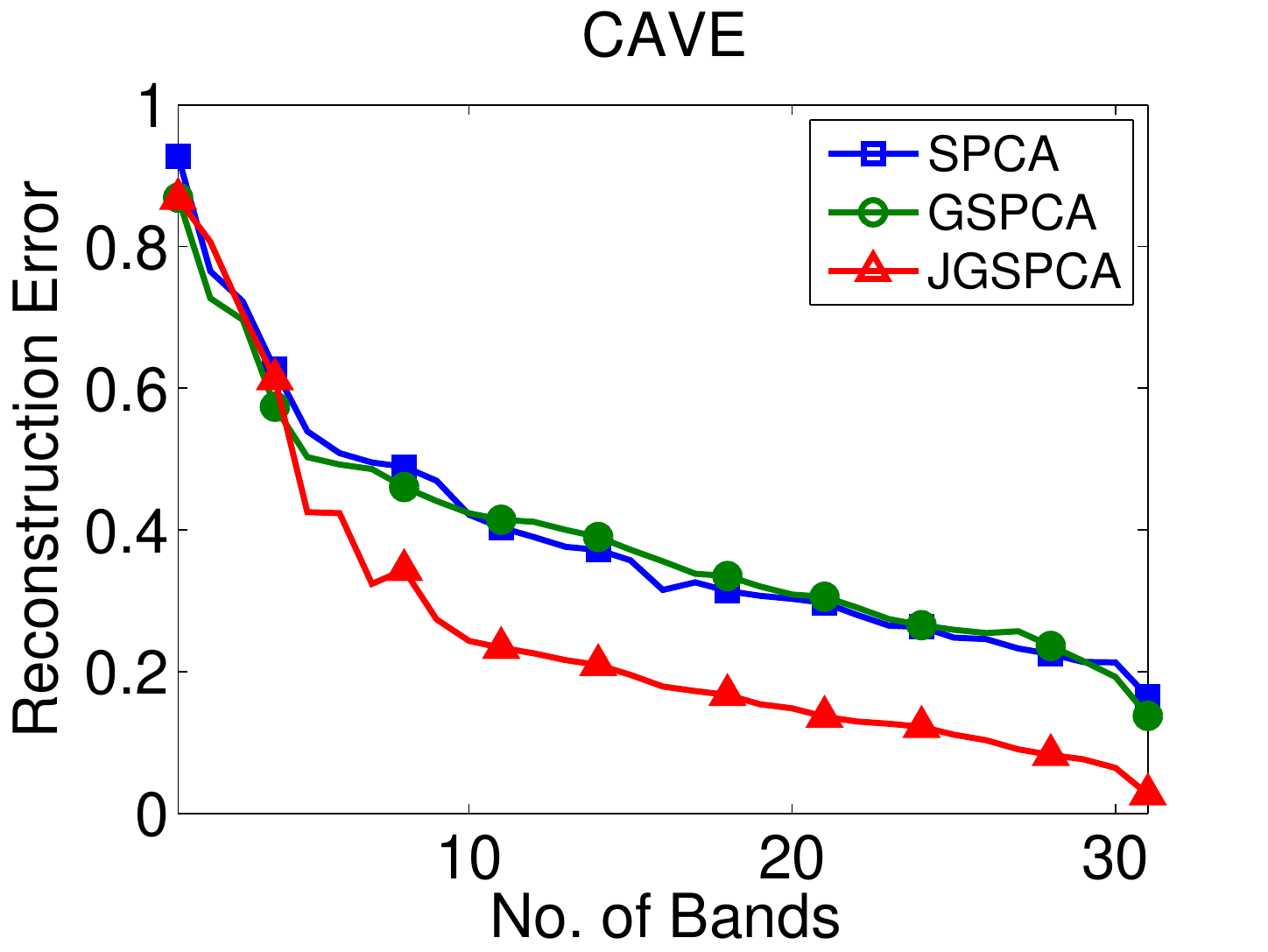}\\ [10pt]
\includegraphics[trim = 3pt 2pt 30pt 2pt, clip, width=0.44\linewidth]{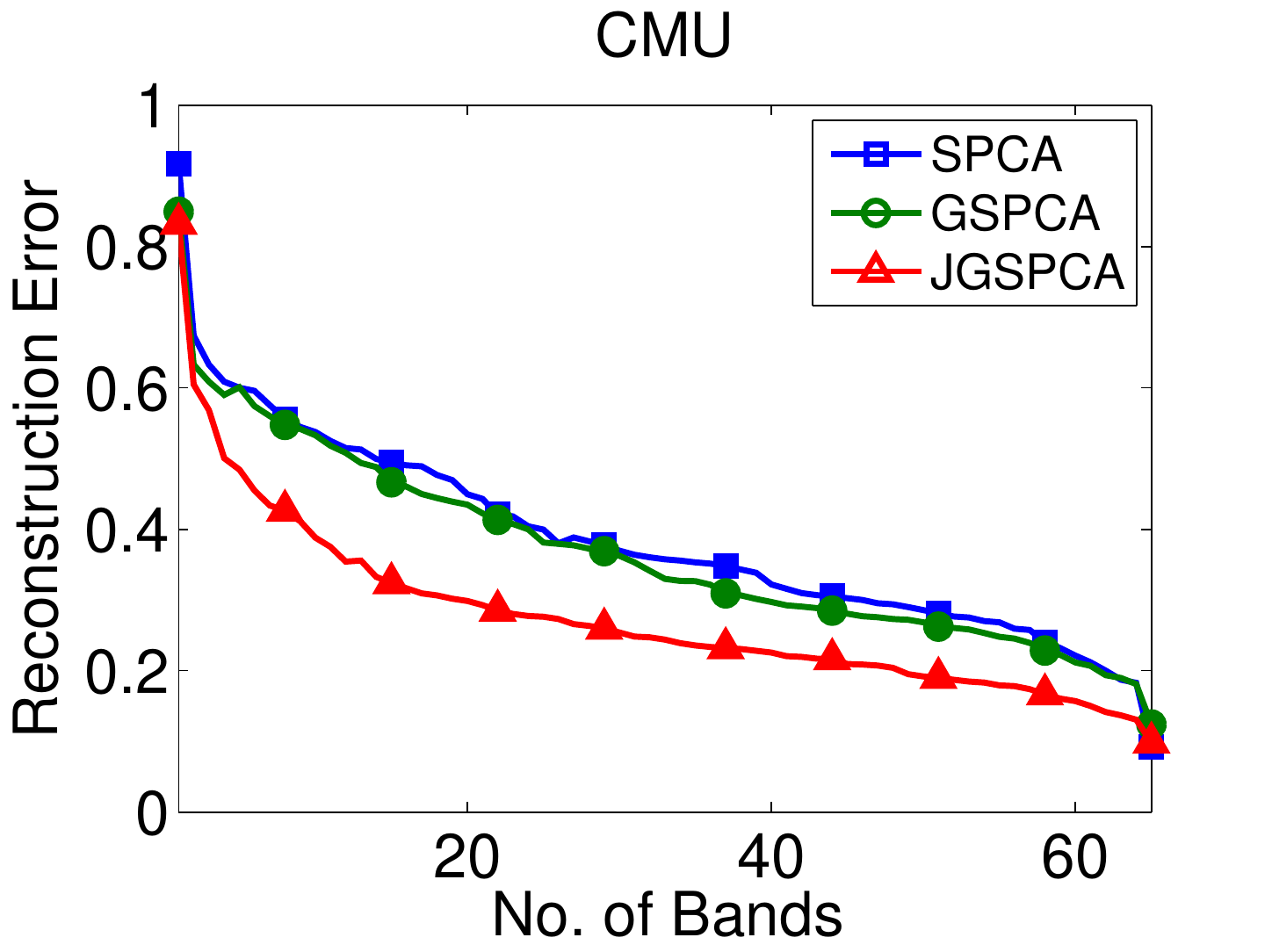}\hfill
\includegraphics[trim = 3pt 2pt 30pt 2pt, clip, width=0.44\linewidth]{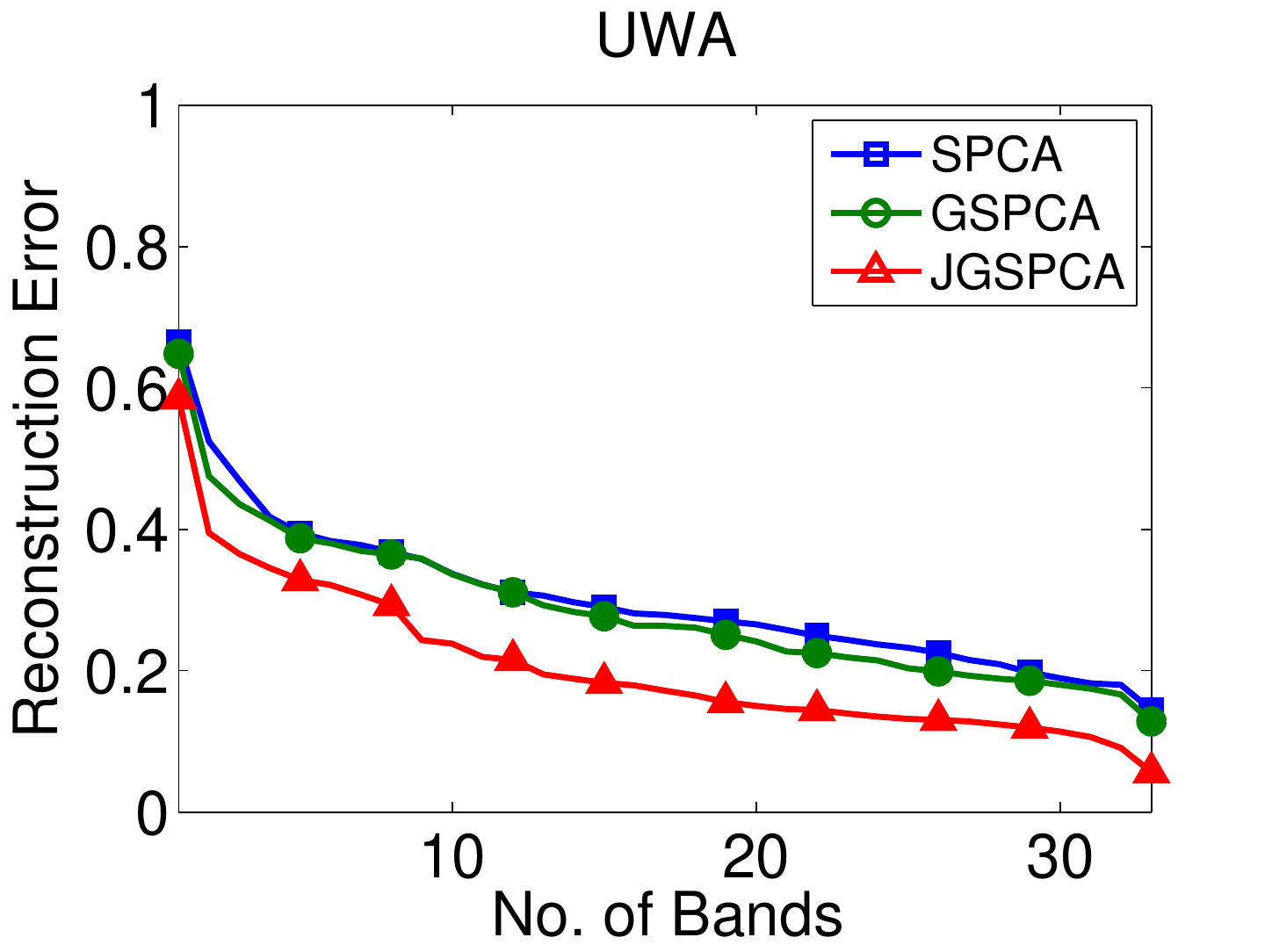}
\caption[Reconstruction errors on all datasets]{Reconstruction errors ($e_r$) on Harvard, CAVE, CMU and UWA datasets}
\label{fig:reconError}
\end{figure}

It is important to note that, in some cases, the first few selected bands are similar regardless of the type of sparsity. Thus, if similar bands are selected, the reconstruction error using those bands may be similar as well. Beyond the first few bands, the proposed JGSPCA is able to identify and select the most informative bands earlier than the other algorithms and hence results in lower reconstruction errors. For instance, the reconstruction error curves on CMU dataset suggest that a number of useful bands are selected by JGSPCA when the number of bands is increased from 1 to 20 which is illustrated by a steep drop in $e_r$ down to 30\%. To reach the same level of $e_r$, SPCA and GSPCA require nearly 45 bands.

The overall trend of reconstruction errors is also related to the variety of objects, and the number of samples used for training in each database. It is difficult to model spatio-spectral variation of complex objects (such as those in CAVE database) with a few bands with a limited training data. On the other hand, the faces are a particular class of objects and can be reconstructed by only a few bands. Moreover, because the image noise is not modeled it is highly unlikely to achieve zero reconstruction error, which is in turn a benefit of sparse modeling techniques. This can be observed in all graphs where reconstruction error exits even when using all the bands.

\begin{table}[!h]
\caption[The number of bands required to achieve a specific reconstruction error]{The number of bands required to achieve a specific reconstruction error. Lower number indicates the superiority of a method in selecting informative bands.}
\label{tab:recError}
\footnotesize
\centering
\begin{tabular}[b]{|l|c|c|c|} 
\multicolumn{4}{c}{Harvard} \\ \hline
\multirow{2}{*}{Method}  & \multicolumn{3}{c|}{$e_r$ (\%)} \\ \cline{2-4}
        &  50\% &  30\% &  10\% \\ \hline
SPCA    &    5  &  16   &   31  \\ \hline
GSPCA   &    5  &  14   &   31  \\ \hline
JGSPCA  &    3  &   6   &   21  \\ \hline
\end{tabular} \hfill
\begin{tabular}[b]{|l|c|c|c|} 
\multicolumn{4}{c}{CAVE} \\ \hline
\multirow{2}{*}{Method}  & \multicolumn{3}{c|}{$e_r$ (\%)} \\ \cline{2-4}
        &  50\% &  30\% &  10\% \\ \hline
SPCA    &   7   &  21   &  31   \\ \hline
GSPCA   &   5   &  21   &  31   \\ \hline
JGSPCA  &   5   &   7   &  26   \\ \hline
\end{tabular} \\
\begin{tabular}[b]{|l|c|c|c|} 
\multicolumn{4}{c}{CMU} \\ \hline
\multirow{2}{*}{Method}  & \multicolumn{3}{c|}{$e_r$ (\%)} \\ \cline{2-4}
        &  50\% &  30\% &  10\% \\ \hline
SPCA    &    14 &  46   &   65  \\ \hline
GSPCA   &    13 &  39   &   65  \\ \hline
JGSPCA  &     4 &  20   &   65  \\ \hline
\end{tabular} \hfill
\begin{tabular}[b]{|l|c|c|c|} 
\multicolumn{4}{c}{UWA} \\ \hline
\multirow{2}{*}{Method} & \multicolumn{3}{c|}{$e_r$ (\%)} \\ \cline{2-4}
        &  50\% &  30\% &  10\% \\ \hline
SPCA    &   2   &   14  &   33  \\ \hline
GSPCA   &   2   &   13  &   33  \\ \hline
JGSPCA  &   1   &    8  &   31  \\ \hline
\end{tabular}
\end{table}

Table~\ref{tab:recError} provides the number of bands required by a model to limit the reconstruction error within an upper bound. As we are interested in achieving low reconstruction error, we restrict to 50\%, 30\% and 10\% error marks for observation. With a limited number of bands, the errors are too high in reconstruction of the hyperspectral data. When more bands are added, JGSPCA clearly selects fewer and better bands to achieve the same reconstruction error mark compared to SPCA and GSPCA. For extremely low reconstruction errors, all methods require roughly the large number of bands, whereas JGSPCA requires comparatively fewer bands.

\begin{figure}[!h]
\footnotesize
\subfigure[A scene from Harvard dataset.]{
\begin{minipage}[c]{0.48\linewidth} \centering
\fbox{\includegraphics[trim = 40pt 86pt 30pt 90pt, clip, width=1\linewidth]{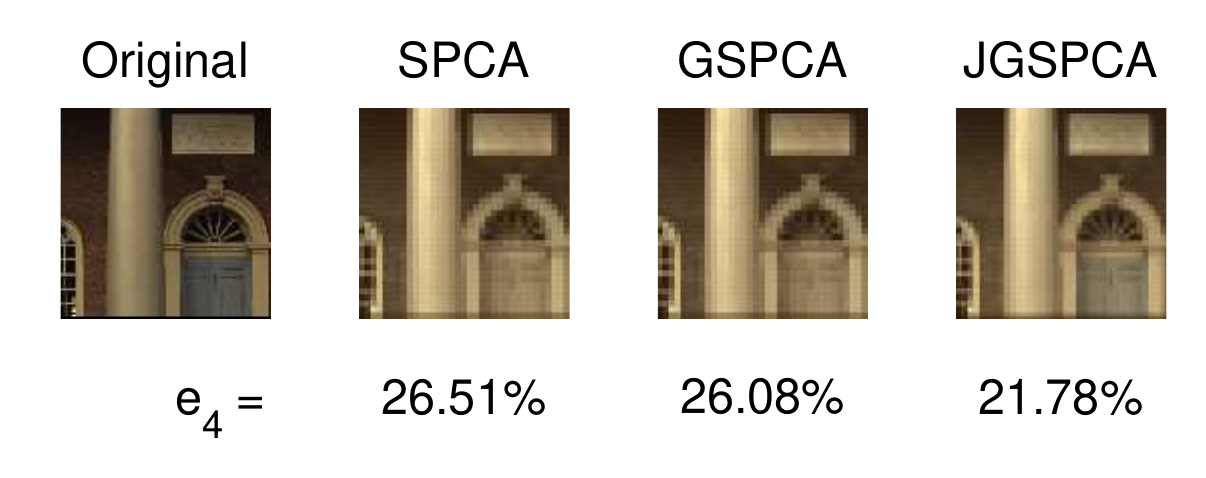}} \\
\fbox{\includegraphics[trim = 40pt 86pt 30pt 90pt, clip, width=1\linewidth]{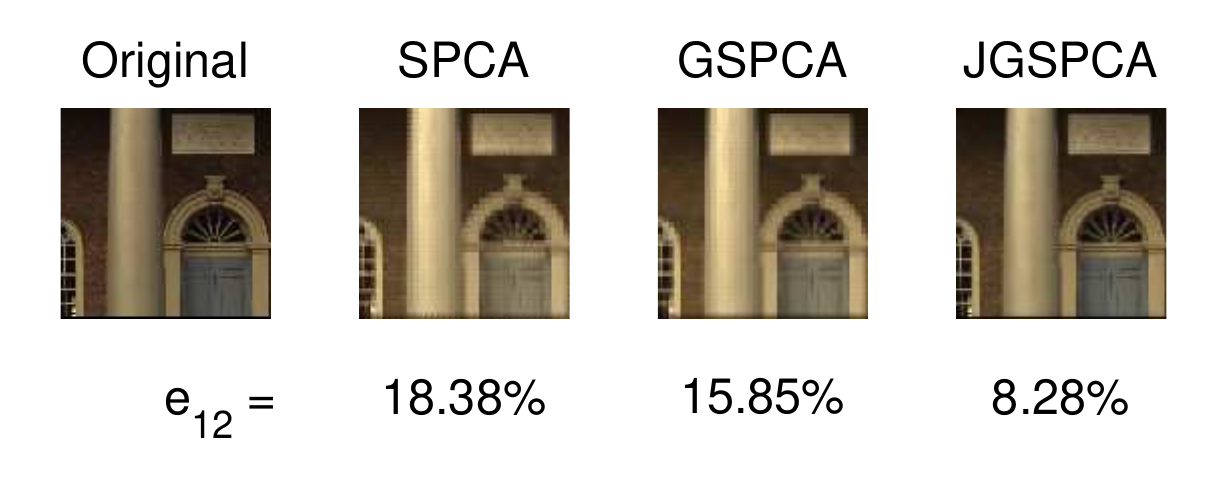}}
\end{minipage}} \vspace{5pt}
\subfigure[A scene from CAVE dataset.]{
\begin{minipage}[c]{0.48\linewidth} \centering
\fbox{\includegraphics[trim = 40pt 86pt 30pt 90pt, clip, width=1\linewidth]{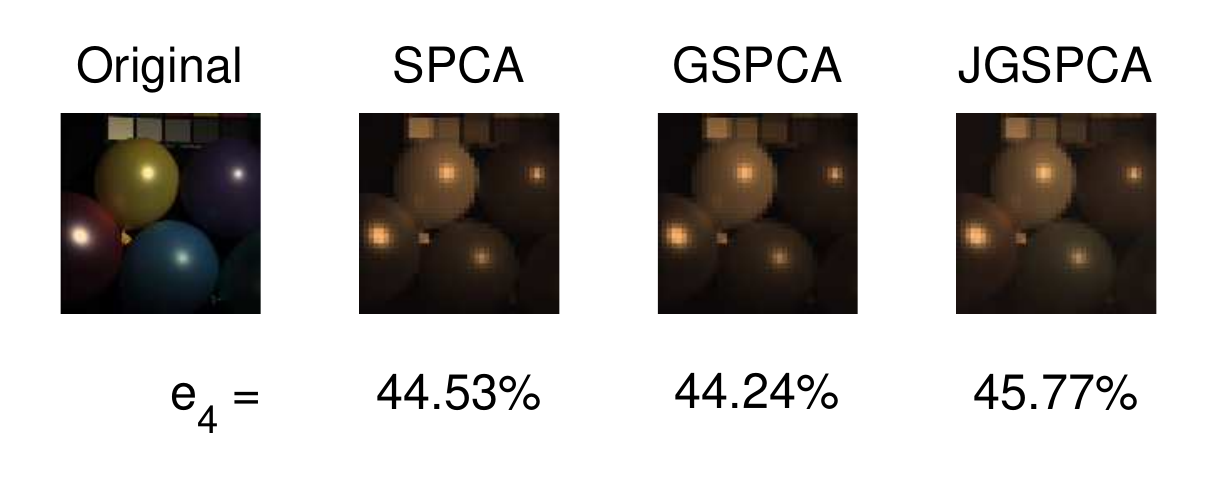}} \\
\fbox{\includegraphics[trim = 40pt 86pt 30pt 90pt, clip, width=1\linewidth]{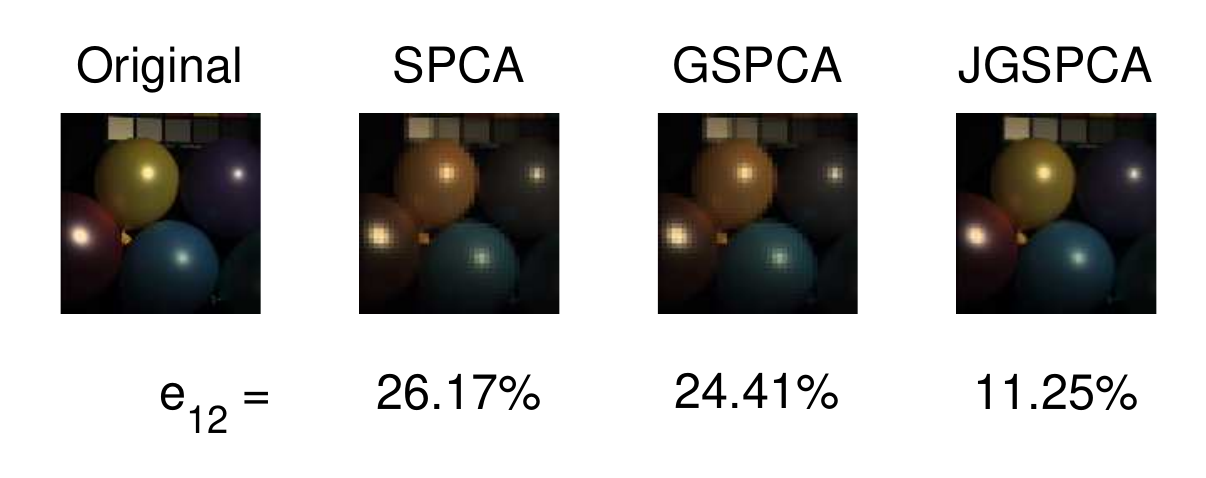}}
\end{minipage}}  \vspace{5pt}\\
\subfigure[A face from UWA dataset.]{
\begin{minipage}[c]{0.48\linewidth} \centering
\fbox{\includegraphics[trim = 40pt 86pt 30pt 90pt, clip, width=1\linewidth]{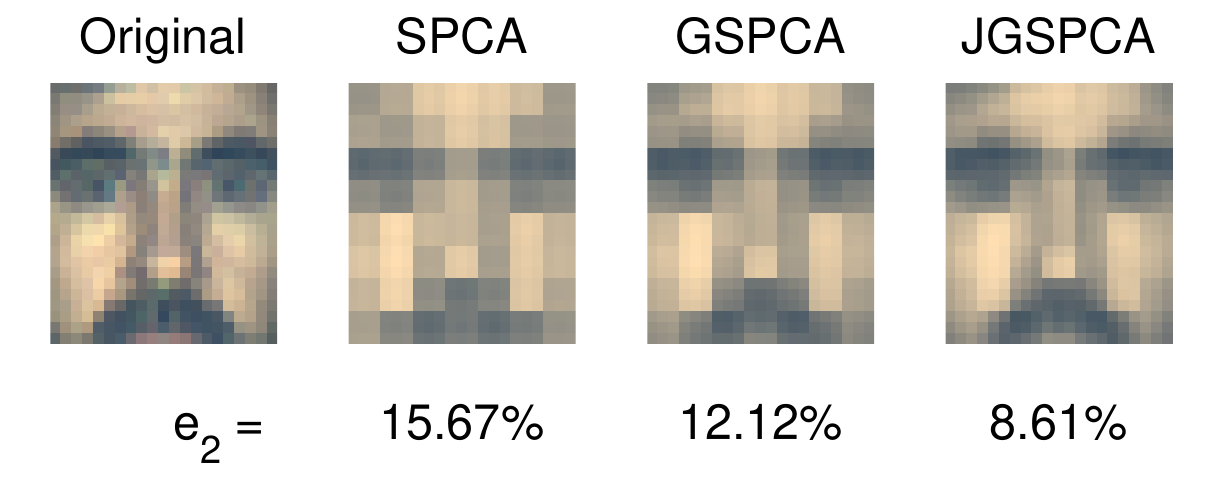}} \\
\fbox{\includegraphics[trim = 40pt 86pt 30pt 90pt, clip, width=1\linewidth]{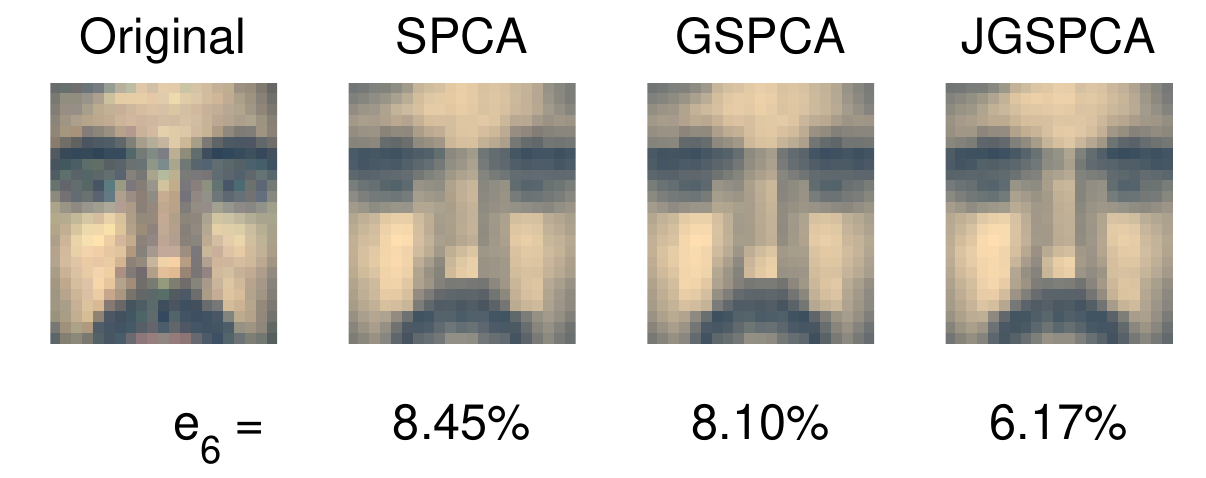}}
\end{minipage}} \vspace{5pt}
\subfigure[A face from CMU dataset.]{
\begin{minipage}[c]{0.48\linewidth} \centering
\fbox{\includegraphics[trim = 40pt 86pt 30pt 90pt, clip, width=1\linewidth]{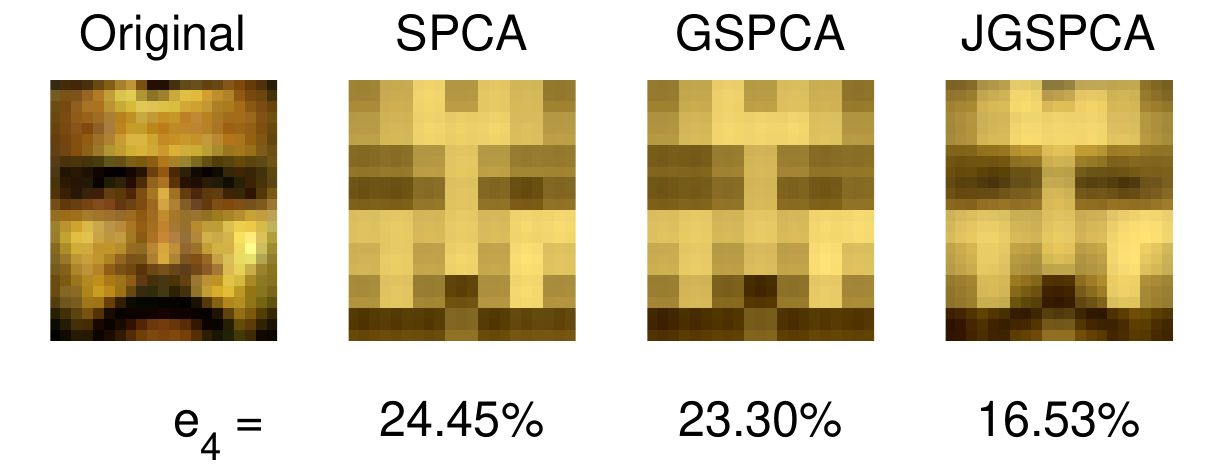}} \\
\fbox{\includegraphics[trim = 40pt 86pt 30pt 90pt, clip, width=1\linewidth]{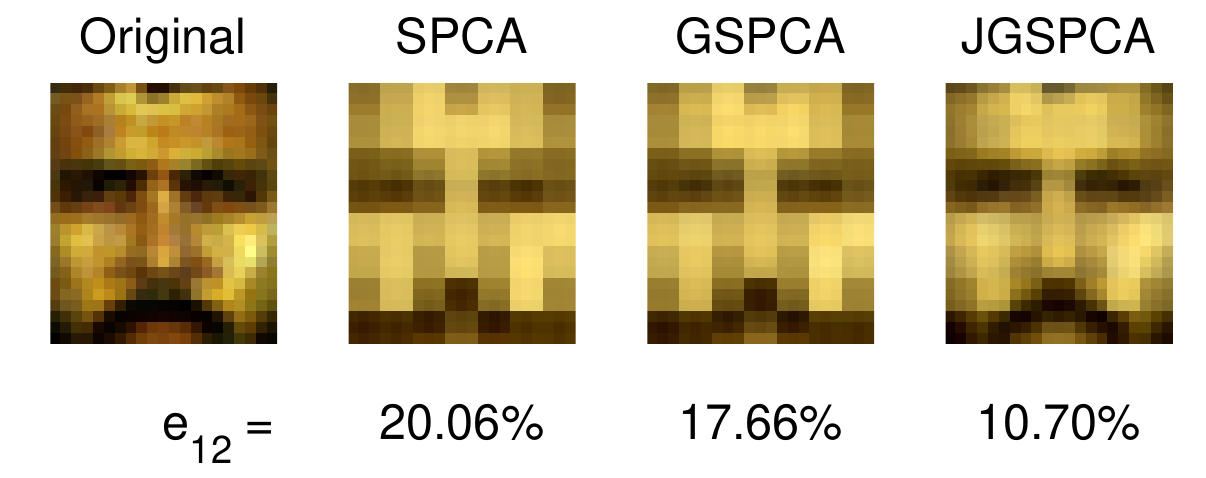}}
\end{minipage}} \vspace{5pt}
\caption[Compressed sensing results of hyperspectral images]{Compressed sensing results of hyperspectral images (rendered as RGB). The results are shown for the same number of bands used for reconstruction of the hyperspectral image using SPCA, GSPCA and JGSPCA. The original images are rendered using all bands of the hyperspectral images. The differences are numerically and visually appreciable in all examples.}
\label{fig:reconResults}
\end{figure}

Figure~\ref{fig:reconResults} shows compressive sensing of two example images using SPCA, GSPCA and JGSPCA methods. The proposed JGSPCA exhibits significantly lower reconstruction errors which can be visually appreciated. The difference is more obvious when using small number of bands for compressed sensing. Overall, on all databases, JGSPCA performs the best, followed by GSPCA and SPCA in compressed hyperspectral imaging.

\subsection{Hyperspectral Face Recognition}

In this experiment, we compare the compressively sensed hyperspectral images using different algorithms for a recognition task. We expect a compressive sensing algorithm to achieve high recognition accuracy by sensing minimum number of bands. We evaluate our proposed JGSPCA algorithm for band selection in hyperspectral face recognition and compare it to SPCA and GSPCA. In order to understand the purpose of this experiment, following points need due consideration
\begin{enumerate}
\item We use several widely accepted classification methods to evaluate the trend of recognition accuracy against compressive sensing of hyperspectral face images. Any other state-of-the-art algorithm may perform better than the chosen baseline algorithms, however the trend is expected to be similar.
\item We assume that the bands that are informative for class separation are the bands that are informative to explanation of the data which is the default criteria in PCA. More relevant criteria such as sparse LDA~\cite{clemmensen2011sparse,merchante2012efficient} are expected to be more supportive of this assumption which can be explored in future.
\end{enumerate}

A model is learned using a single hyperspectral image per subject in the training set which makes the gallery. All remaining hyperspectral images which comprise the test set, serve as the probes. A test hyperspectral image cube is compressively sensed (reconstructed by learned model) and used for classification. Consider a training set $\mathbf{X}$ and test set $\mathbf{Z}$, where each row is a hyperspectral face image. The compressive sensing performance of the $r^{\textrm{th}}$ learned model $\mathcal{M}_r$ in terms of recognition accuracy is computed as
\begin{equation}
a_r =  \textrm{classify}(\mathcal{M}_r,\mathbf{X},\mathbf{Z})~.
\end{equation}

We perform face recognition using Nearest Neighbor (NN), EigenFaces~\cite{turk1991eigenfaces}, Support Vector Machine (SVM)~\cite{guo2000face} and Sparse Representation-based Classification (SRC)~\cite{wright2009robust}. The recognition accuracies from each algorithm are averaged by 3-fold cross validation.

\begin{landscape}
\begin{figure}
\footnotesize
\centering
\includegraphics[trim = 3pt 2pt 30pt 2pt, clip, width=0.24\linewidth]{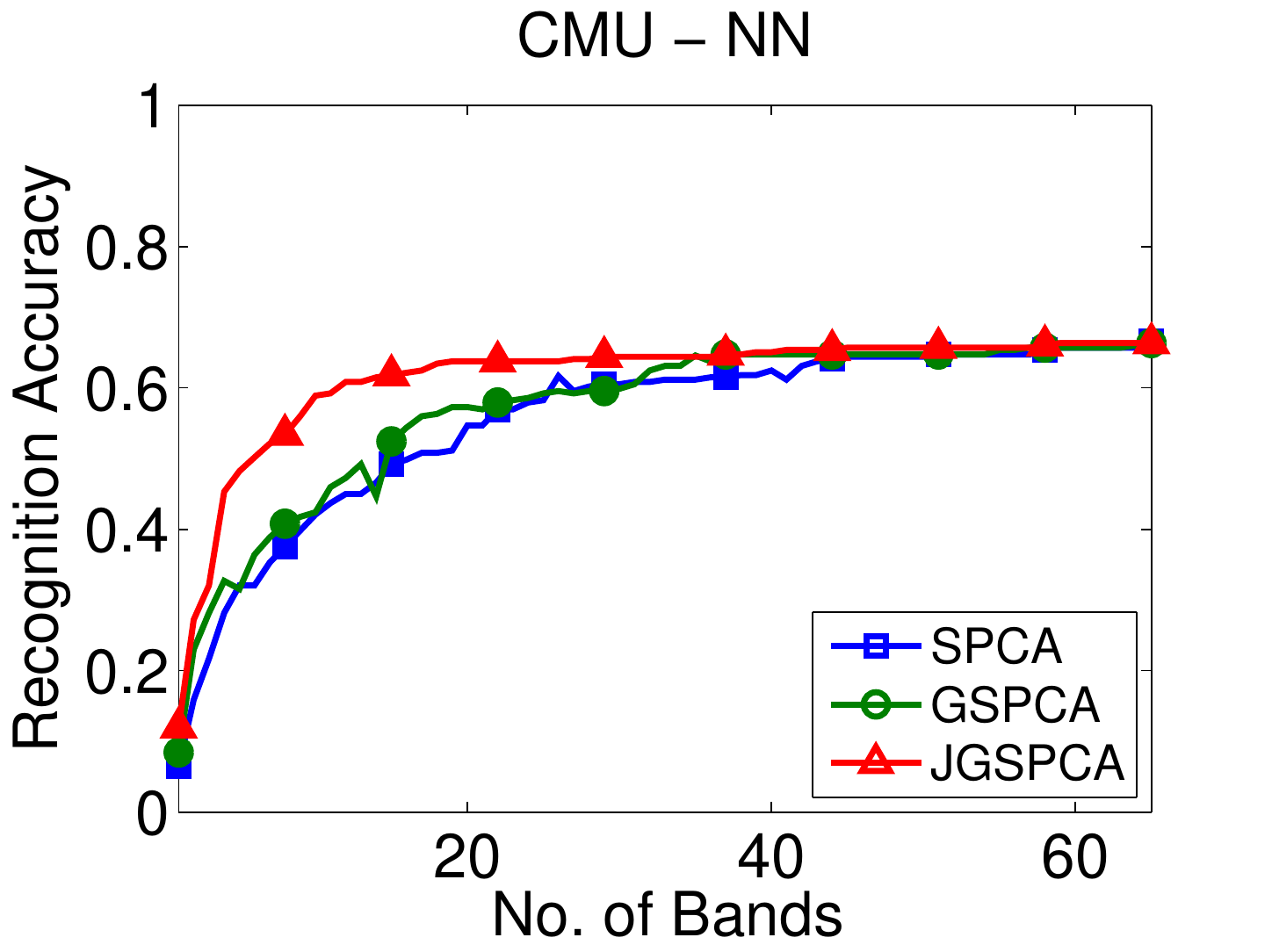}
\includegraphics[trim = 3pt 2pt 30pt 2pt, clip, width=0.24\linewidth]{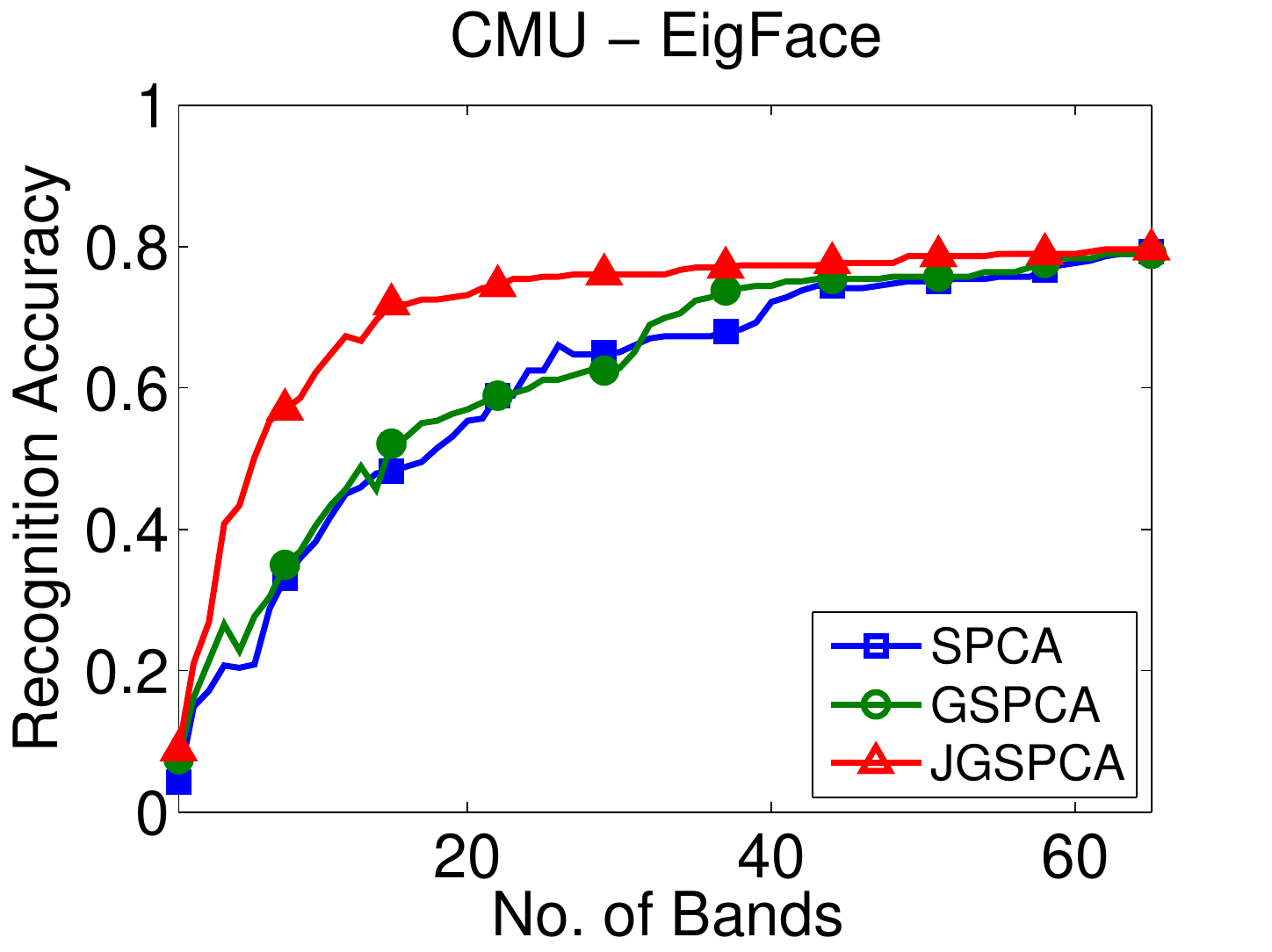}
\includegraphics[trim = 3pt 2pt 30pt 2pt, clip, width=0.24\linewidth]{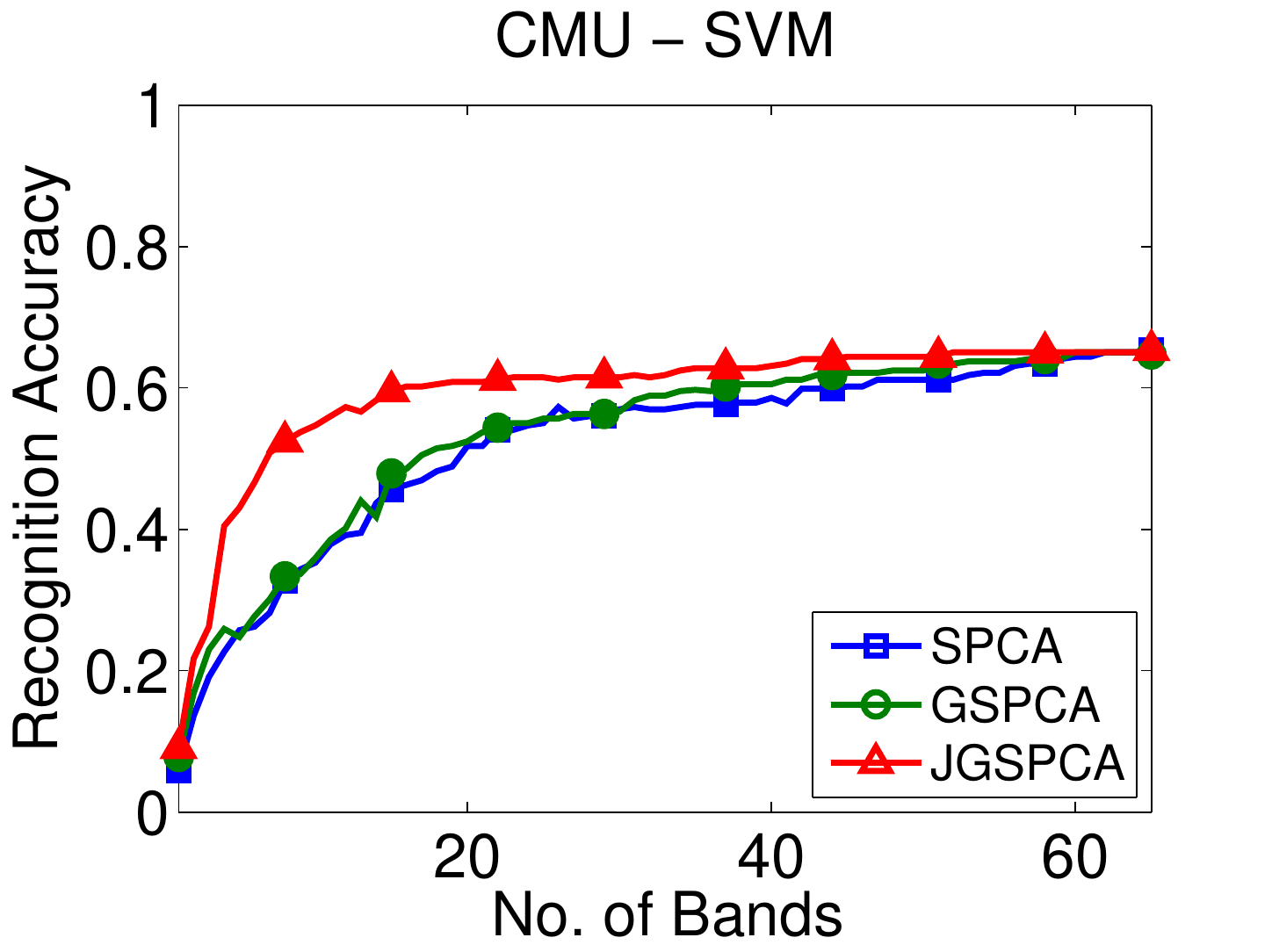}
\includegraphics[trim = 3pt 2pt 30pt 2pt, clip, width=0.24\linewidth]{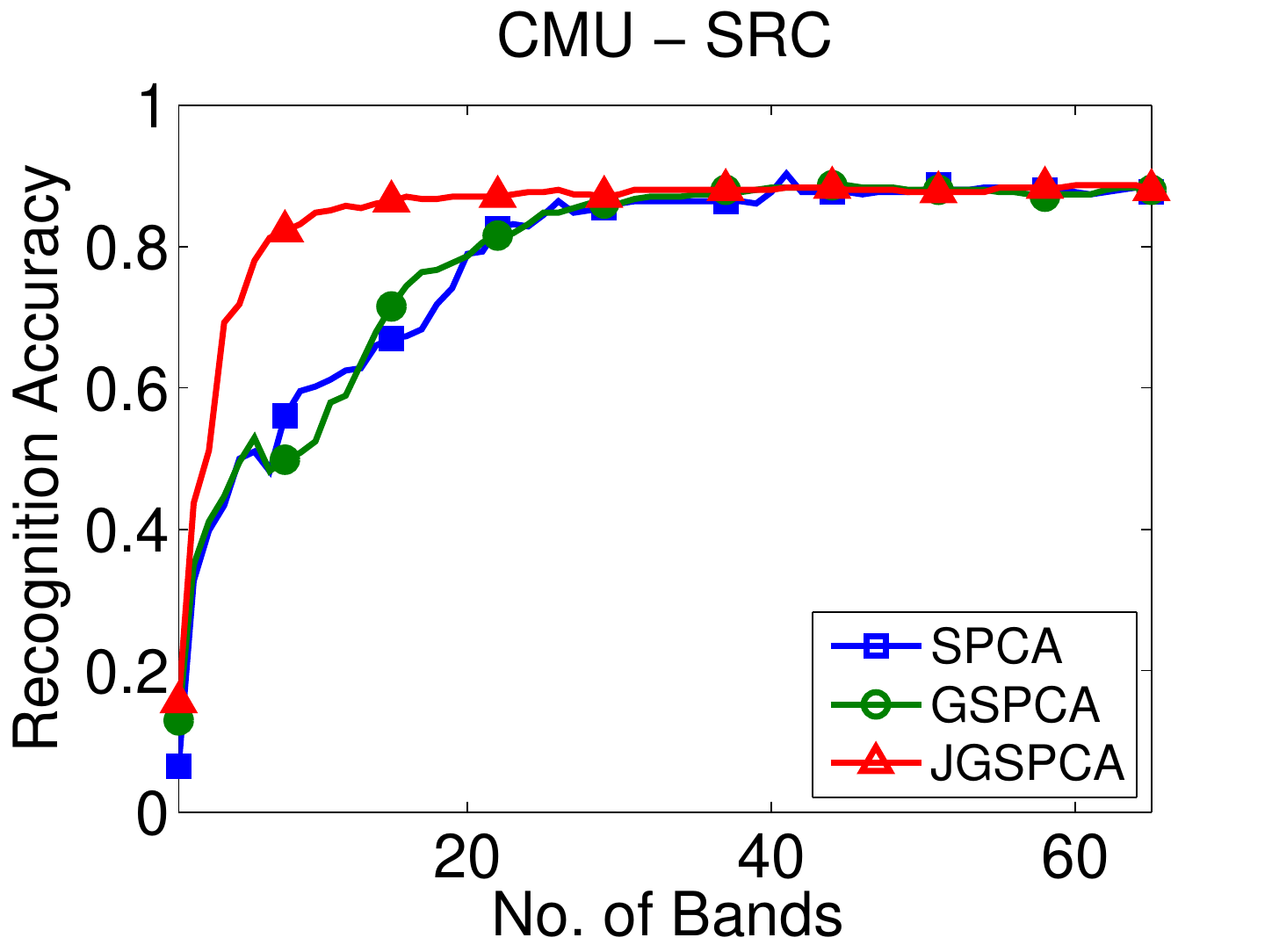}\\ [10pt]
\includegraphics[trim = 3pt 2pt 30pt 2pt, clip, width=0.24\linewidth]{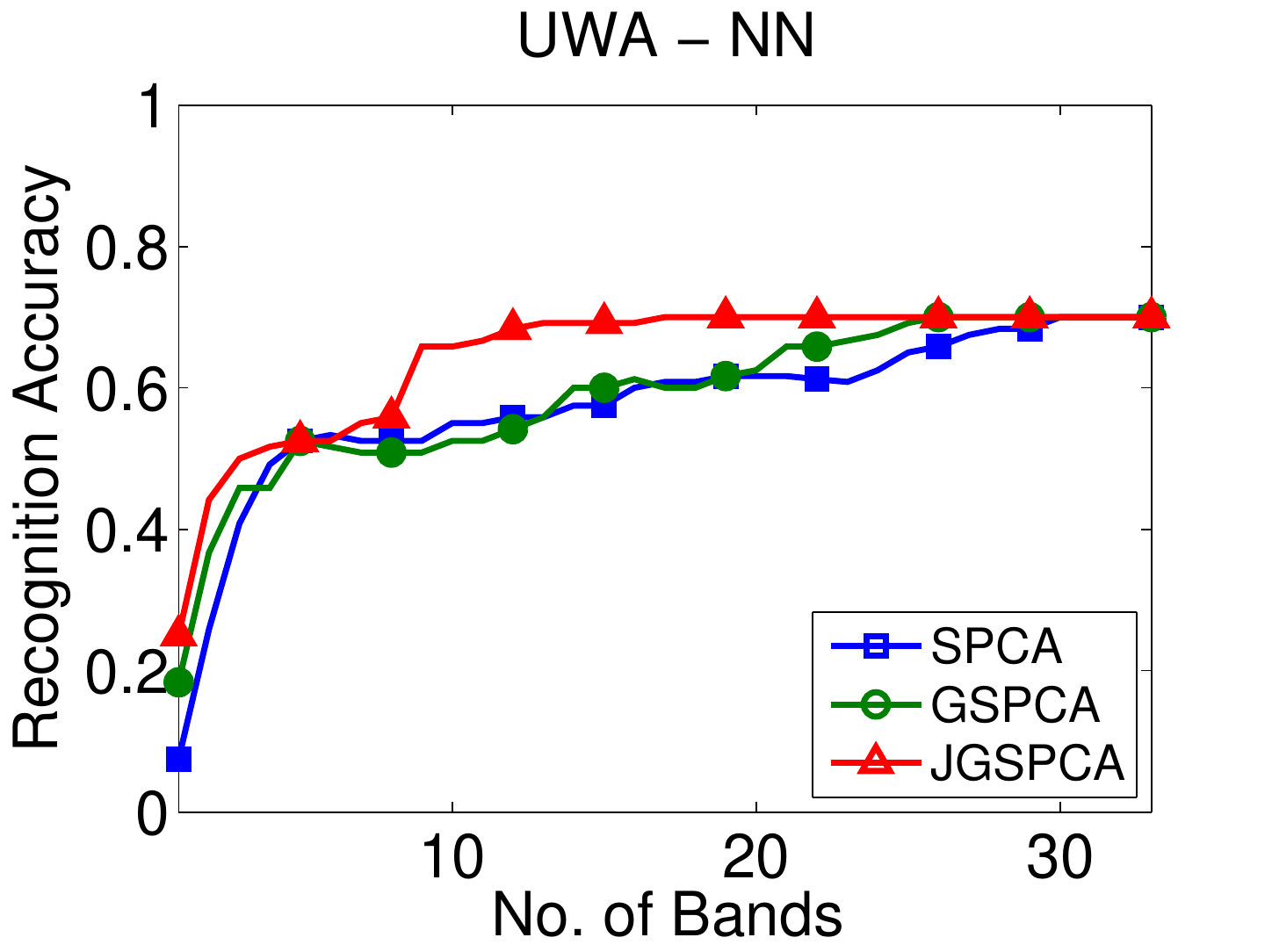}
\includegraphics[trim = 3pt 2pt 30pt 2pt, clip, width=0.24\linewidth]{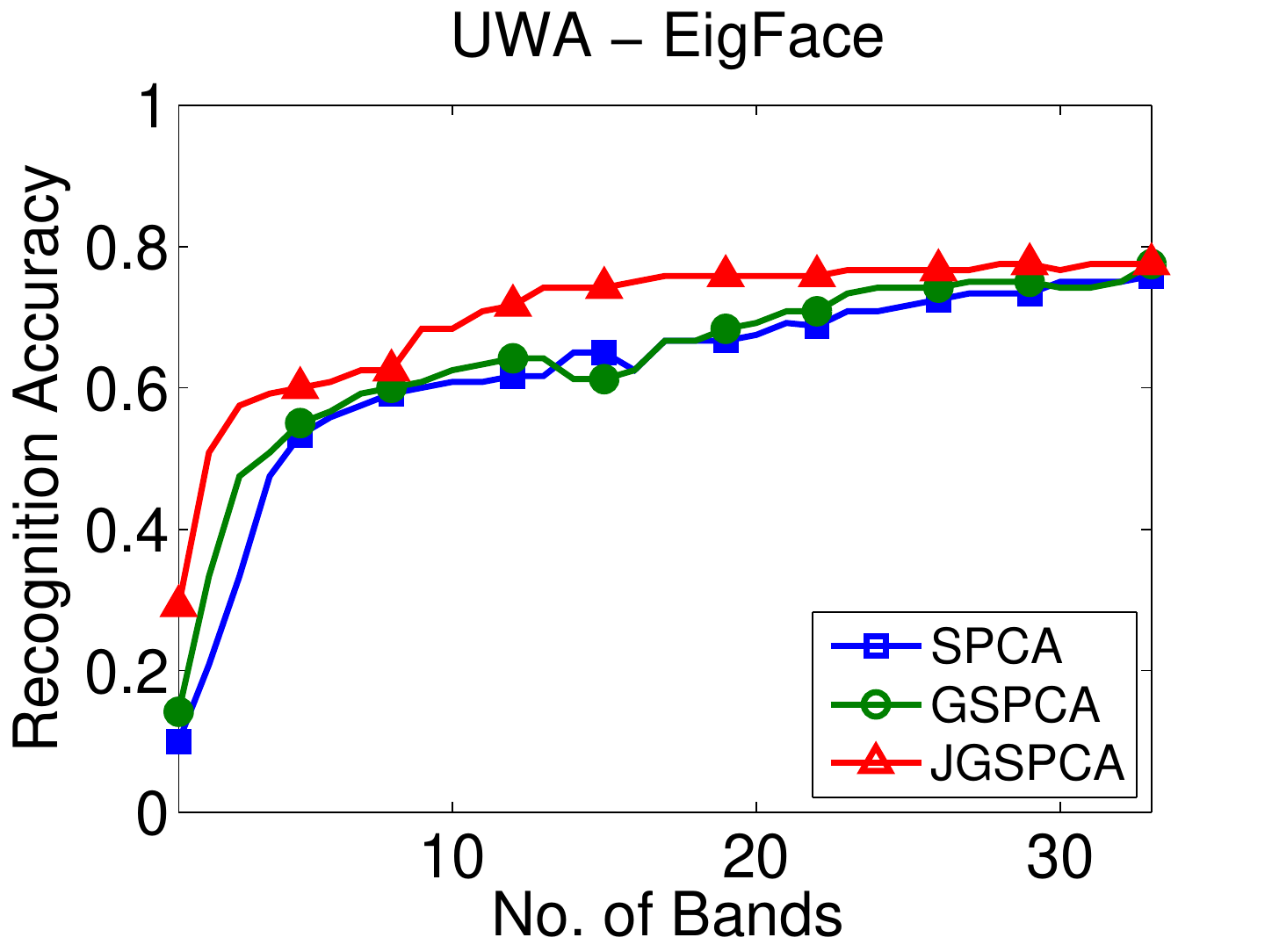}
\includegraphics[trim = 3pt 2pt 30pt 2pt, clip, width=0.24\linewidth]{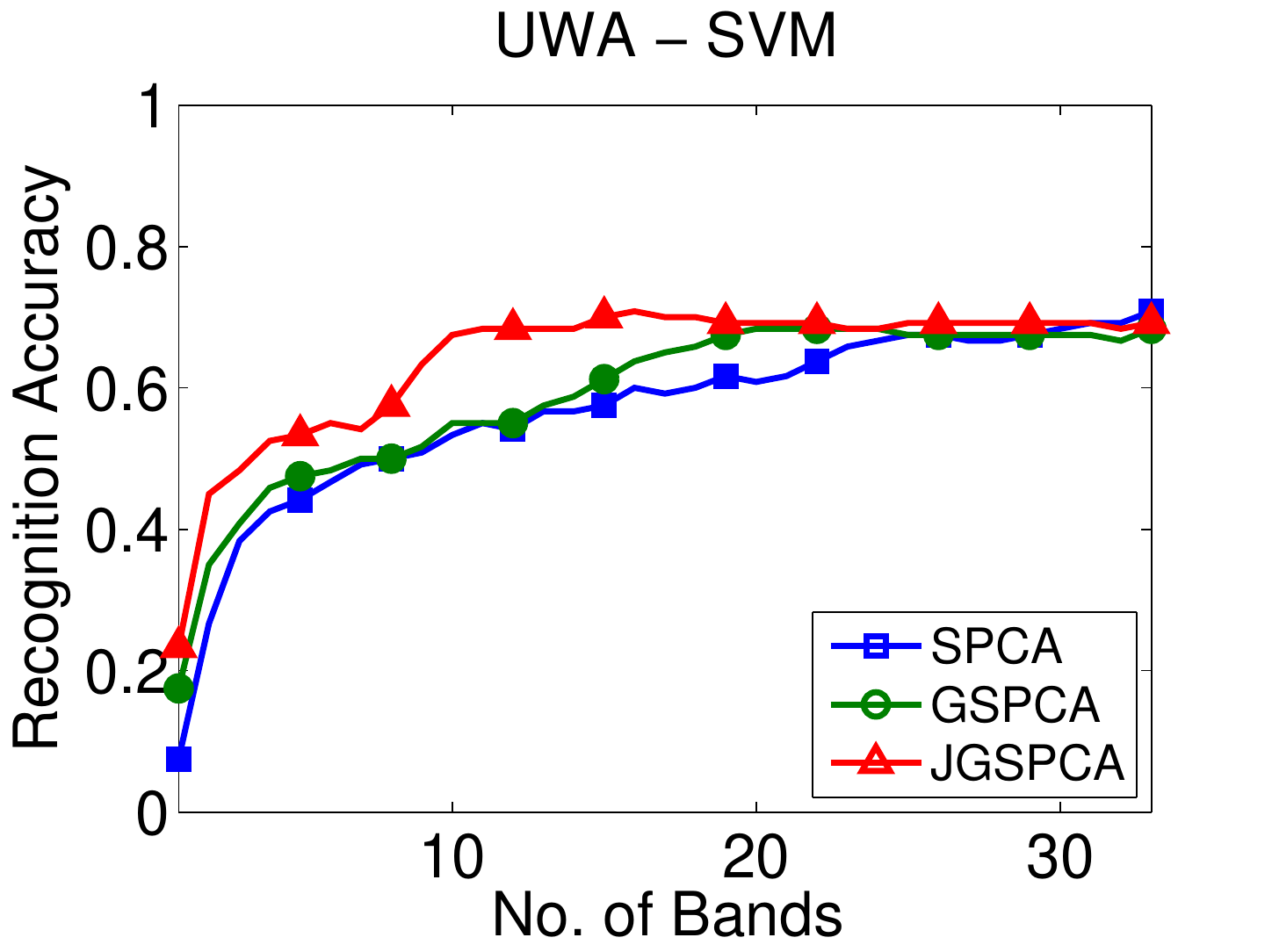}
\includegraphics[trim = 3pt 2pt 30pt 2pt, clip, width=0.24\linewidth]{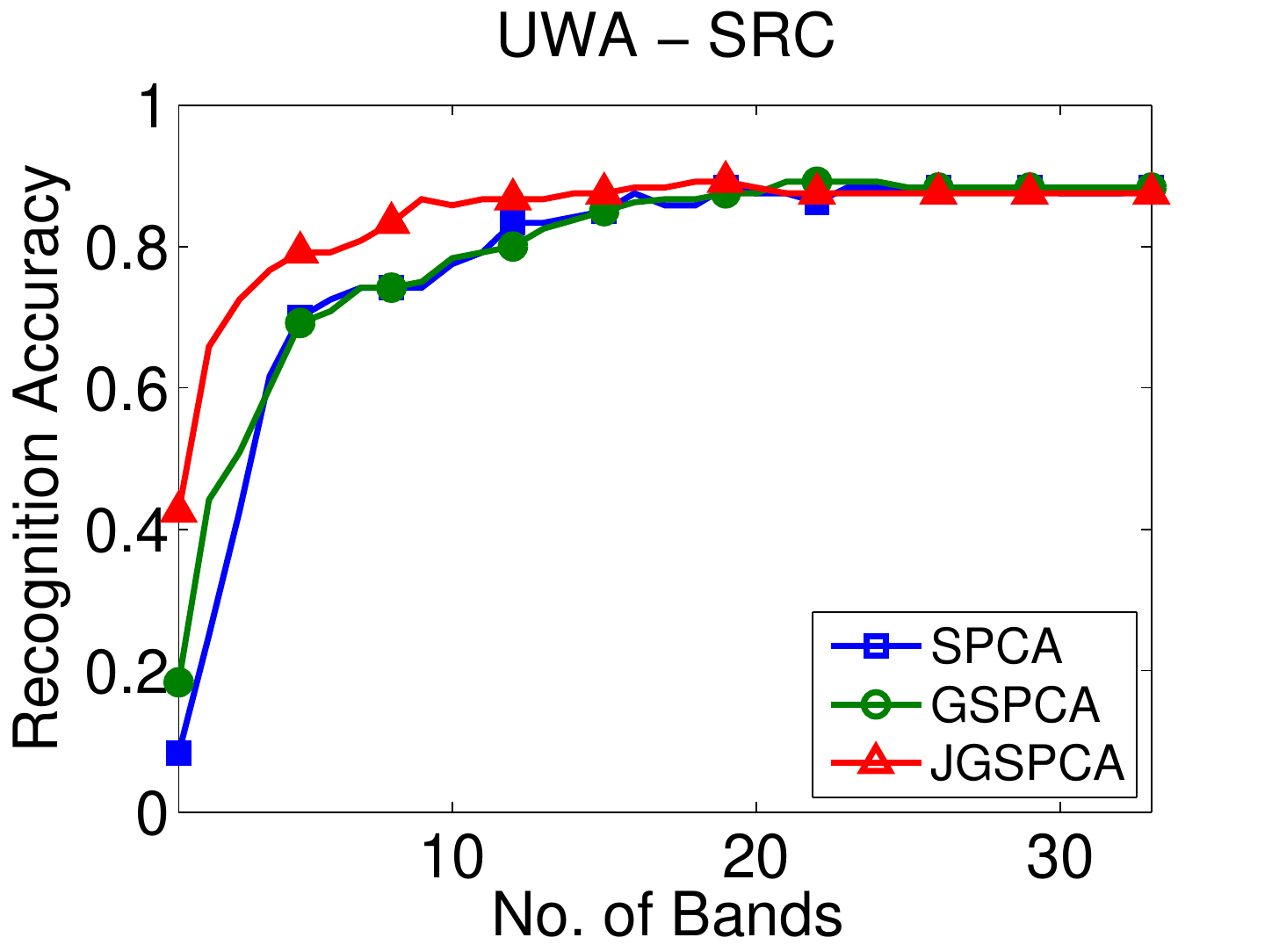}
\caption[Recognition accuracy ($a_r$) versus number of selected bands]{Recognition accuracy ($a_r$) versus number of selected bands on CMU and UWA face datasets. The JGSPCA demonstrates consistently higher recognition accuracy compared to SPCA and GSPCA.}
\label{fig:recogRate}
\end{figure}
\end{landscape}

Figure~\ref{fig:recogRate} shows the recognition accuracy against the number of bands used for reconstruction of test hyperspectral images. It can be easily observed that JGSPCA consistently achieves higher recognition accuracy with fewer bands compared to SPCA and GSPCA on both databases. The consistency of the trend can be observed among different recognition algorithms on the same database. In order to numerically analyze the recognition performance through compressed sensing, we tabulate the number of bands required to achieve a certain recognition accuracy mark. The accuracies are averaged over all recognition algorithms after scaling each algorithm's recognition accuracies between [0,1]. In Table~\ref{tab:recRate}, we are interested in achieving higher recognition accuracies, therefore, we only observe 50\%, 70\% and 90\% accuracy marks. It can be observed that the proposed JGSPCA algorithm achieves higher recognition accuracy by sensing only a few bands compared to SPCA and GSPCA. This implicitly indicates the ability of JGSPCA to select more informative bands for a recognition task.

\begin{table}
\caption[The number of bands required to achieve a specific recognition accuracy]{The number of bands required to achieve a specific recognition accuracy. Lower number indicates the superiority of a method in selecting informative bands.}
\label{tab:recRate}
\footnotesize
\centering
\begin{tabular}[b]{|l|c|c|c|} 
\multicolumn{4}{c}{CMU} \\ \hline
\multirow{2}{*}{Method}  & \multicolumn{3}{c|}{$a_r$ (\%)} \\ \cline{2-4}
        &  50\% &  70\% &  90\% \\ \hline
SPCA    &    9 &  18   &   39 \\ \hline
GSPCA   &    9 &  15   &   32 \\ \hline
JGSPCA  &    4 &   6   &   14 \\ \hline
\end{tabular}
\hfill
\begin{tabular}[b]{|l|c|c|c|} 
\multicolumn{4}{c}{UWA} \\ \hline
\multirow{2}{*}{Method} & \multicolumn{3}{c|}{$a_r$ (\%)} \\ \cline{2-4}
        &  50\% &  70\% &  90\% \\ \hline
SPCA    &   4   &   6  &   23 \\ \hline
GSPCA   &   3   &   6  &   19 \\ \hline
JGSPCA  &   2   &   3  &    9 \\ \hline
\end{tabular}
\end{table}

\section{Conclusion}
\label{sec:conc-JGSPCA}

We presented a Joint Group Sparse PCA algorithm which addresses the problem of finding a few \emph{groups of features} that \emph{jointly} capture most of the variation in the data. Unlike other sparse formulations of PCA, for which all features might still be needed for reconstructing the data, the presented approach requires only a few features to represent the whole data. This property makes the presented formulation most suitable for compressed sensing, in which the main goal is to measure only a few features that capture most significant information. The efficacy of our approach has been demonstrated by experiments on several real-world datasets of hyperspectral images. The proposed methodology is well adaptable to scenarios where the features can be implicitly or explicitly categorized into groups. 

\chapter[Joint Sparse Principal Component Analysis]{Joint Sparse PCA\\for Hyperspectral Ink Mismatch Detection} 

\label{Chapter6} 


Natural and man-made materials exhibit a characteristic response to incident light. As discussed in Chapter~\ref{Chapter1}, humans are metameric to certain colors, i.e.~they are unable to distinguish between two apparently similar colors~\cite{gegenfurtner2003cortical} due to the trichromatic nature of the human visual system. For instance, two blue inks with substantially different spectral responses might look identical to the naked eye. When a document is manipulated with the intention of forgery or fraud, the modifications are often done in such a way that they are hard to catch with a naked eye. In handwritten documents, the forger not only tries to emulate the handwriting of the original writer, but also uses a pen that has a visually similar ink compared to the rest of the note. Hence, analysis of inks is of critical importance in questioned document examination.

The outcome of ink analysis potentially leads to the determination of forgery, fraud, backdating and ink age. Of these, one of the most important tasks is to discriminate between different inks which we term as \emph{ink mismatch detection}. There are two main approaches to distinguish inks, destructive and non-destructive examination. Chemical analysis such as Thin Layer Chromatography (TLC)~\cite{aginsky1993forensic} belongs to the category of destructive testing and can separate a mixture of inks into its constituents. The separation of inks is achieved via capillary action which is a common practice in chemical analysis. There are a few drawbacks to this approach. First, TLC is destructive which means that the originality of the sample is compromised after each repetition of the examination, which is often forbidden by law in the context of forensic case work as it effectively destroys the evidence. Furthermore, the procedure is time consuming because the sample needs to be placed for a certain amount of time before any noticeable differences can be observed in the chromatograph.

An alternative non-destructive approach is to employ spectral imaging to differentiate apparently similar inks. Spectral imaging captures subtle differences in the inks which is valuable for mismatch detection as shown in Figure~\ref{fig:ink-images}. A hyperspectral image (HSI) is a series of discrete narrow-band images in the electro-magnetic spectrum. In contrast to a three channel RGB image, an HSI captures finer detail of a scene in the spectral dimension. Hyperspectral imaging has recently emerged as an efficient non-destructive tool for detection and identification of forensic traces as well as enhancement and restoration of historical documents~\cite{joo2011visual,hedjam2013historical,edelman2012hyperspectral}. It has found good use in forensics for bloodstain analysis, latent print analysis and questioned document examination~\cite{ChemImage}. High fidelity spectral information is very useful, especially, where it is required to distinguish between inks or determine the age of a writing or document.

\begin{figure}[t]
\centering
{\scriptsize
\begin{tabular}{cccccc}
& RGB & 400nm & 520nm & 640nm & 700nm \\
ink 1 &
\fbox{\includegraphics[trim = 514pt 923pt 685pt 762pt, clip, width=0.10\linewidth]{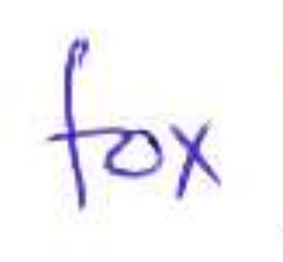}} &
\fbox{\includegraphics[trim = 302pt  20pt 402pt 415pt, clip, width=0.10\linewidth]{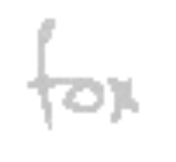}} &
\fbox{\includegraphics[trim = 302pt  20pt 402pt 415pt, clip, width=0.10\linewidth]{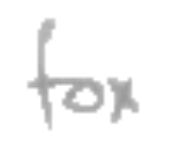}} &
\fbox{\includegraphics[trim = 302pt  20pt 402pt 415pt, clip, width=0.10\linewidth]{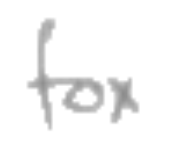}} &
\fbox{\includegraphics[trim = 302pt  20pt 402pt 415pt, clip, width=0.10\linewidth]{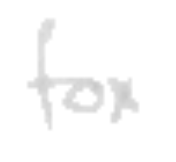}} \\ [2pt]
ink 2 &
\fbox{\includegraphics[trim = 486pt 1363pt 714pt 323pt, clip, width=0.10\linewidth]{chapter_6/Scan1-150}} &
\fbox{\includegraphics[trim = 285pt 310pt 419pt 125pt, clip, width=0.10\linewidth]{chapter_6/MSV_01_01_wbg}} &
\fbox{\includegraphics[trim = 285pt 310pt 419pt 125pt, clip, width=0.10\linewidth]{chapter_6/MSV_01_13_wbg}} &
\fbox{\includegraphics[trim = 285pt 310pt 419pt 125pt, clip, width=0.10\linewidth]{chapter_6/MSV_01_25_wbg}} &
\fbox{\includegraphics[trim = 285pt 310pt 419pt 125pt, clip, width=0.10\linewidth]{chapter_6/MSV_01_31_wbg}}
\end{tabular}}
\caption[Discrimination of inks offered by spectral imaging]{The images highlight the discrimination of inks at different wavelengths offered by spectral imaging. A word written in two different blue inks is shown in this example. Observe that the two inks appear similar at short wavelength and gradually appear different at longer wavelengths.}
\label{fig:ink-images}
\end{figure}

Brauns and Dyer~\cite{brauns2006fourier} developed a hyperspectral imaging system for forgery detection in potentially fraudulent documents in a non-destructive manner. They prepared written documents with blue, black and red inks and later introduced alterations with a different ink of the same color. They used fuzzy c-means clustering to sort ink spectra  into different groups. In Fuzzy clustering, it is possible for an ink spectra to be a member of more than one cluster (or ink) in terms of the degree of association. Their sample data comprised of only two inks for each color. They qualitatively showed that the inks can be separated into two different classes. Their imaging system was based on an interferometer which relies on moving parts for frequency tuning and therefore slows the acquisition process. The small number of inks, absence of quantitative results and slow imaging process collectively limit the applicability of the system to practical ink mismatch detection.

A relatively improved hyperspectral imaging system for the analysis of historical documents in archives was developed by Padaon et al.~\cite{padoan2008quantitative}. It comprised of a CCD camera and tunable light sources. The system provided 70 spectral bands from near-UV through visible to near IR range (365-1100nm in 10nm steps). They highlighted various applications of hyperspectral document imaging including, monitoring paper and ink aging, document enhancement, and distinguishing between inks. The use of narrowband tunable light source may reduce the chances of damage to a document due to excessive heat generated by a strong broadband white light source. However, its extremely slow acquisition time (about 15 minutes)~\cite{klein2008quantitative} consequently results in extended exposure to the light source. Therefore, the benefit gained by a tunable light source may be nullified. Moreover, extended acquisition time limits the productivity of the system in terms of the number of documents that can be processed in a given time. Our proposed system captures hyperspectral images in only a fraction of that time using a \emph{tunable filter}. An electronically tunable filter is fast, precise and has no moving parts.



Hyperspectral document imaging systems from Foster \& Freeman~\cite{FnF} and ChemImage~\cite{ChemImage} are in common use. In these devices, the examiner needs to select a suspected portion of the note for ink mismatch analysis. Above all, an examiner has to search through hundreds of combinations of the different wavelengths to visually identify the differences in inks, which is laborious. For instance, to analyze a 33 band hyperspectral image in an exhaustive search, the total number of band combinations is of the order of $\approx10^{10}$, which is not feasible in time critical scenarios. This procedure is not required in our proposed \emph{automatic} document analysis approach. Moreover, our approach is a quantitative instead of subjective analysis~\cite{hammond2007validation}.

Since, hyperspectral images are densely sampled along the spectral dimension, the neighboring bands are highly correlated. This redundancy in the data makes hyperspectral images a good candidate for sparse representation as well as feature selection~\cite{chakrabarti2011statistics}. Note that acquisition of all bands is time consuming and limits the number of documents that can be scanned in a given time. Moreover, the resulting data is huge and some bands with low energy contain significant system noise. Therefore, it is desirable to select the most informative subset of bands, thereby reducing the acquisition time, and increases the accuracy by getting rid of the noisy bands.

We propose Joint Sparse PCA (JSPCA) that computes a PCA basis by explicitly removing the non-informative bands. The joint sparsity ensures that all basis vectors share the same sparsity structure whereas the complete hyperspectral data can be represented by a sparse linear combination of the bands. We demonstrate the Joint Sparse Band Selection (JSBS) algorithm for hyperspectral ink mismatch detection. We experimentally show that the selected bands yield only fewer combinations to analyze, yet they are informative for ink mismatch detection.



The rest of this chapter is organized as follows. In Section~\ref{sec:method}, we present the proposed ink mismatch detection methodology, the JSBS and SFBS algorithms. In Section~\ref{sec:database} we describe the database specifications, acquisition and normalization. Section~\ref{sec:experiments} provides details of the experimental setup, evaluation protocol, and analysis of the results. The conclusions are presented in Section~\ref{sec:conclusion}.

\section{Ink Mismatch Detection}
\label{sec:method}

Ink mismatch detection is based on the fact that the same inks exhibit similar spectral responses whereas different inks are spectrally dissimilar~\cite{khan2013hyperspectral}. We assume that the spectral responses of the inks are independent of the writing styles of different subjects (which is a spatial characteristic). Thus, unlike works that identify hand writings by the ink-deposition traces~\cite{brink2012writer}, our work solely focuses on the spectral responses of inks for ink discrimination. In the proposed ink mismatch detection framework, the initial objective is to segment handwritten text from the paper. The next task is to select features (bands) from the ink spectra by the proposed band selection technique. Finally, the class membership of each ink pixel is determined by clustering of the ink spectral responses using selected features.


\subsection{Handwritten Text Segmentation}
\label{sec:segment}

Consider a three dimensional hyperspectral image $\mathbf{I} \in \mathbb{R}^{x \times y \times p}$, where $(x,y)$ are the number of pixels in spatial dimension and $p$ is the number of bands in the spectral dimension. The objective is to compute a binary mask $\mathbf{M} \in \mathbb{R}^{x \times y}$ which associates each pixel to the foreground or background. The ink pixels (text) make up the foreground and the blank area of the page is the background.

\begin{figure}[t]
\centering
\subfigure[Document image]{\label{fig:raw-image} \fbox{\includegraphics[width=0.44\linewidth]{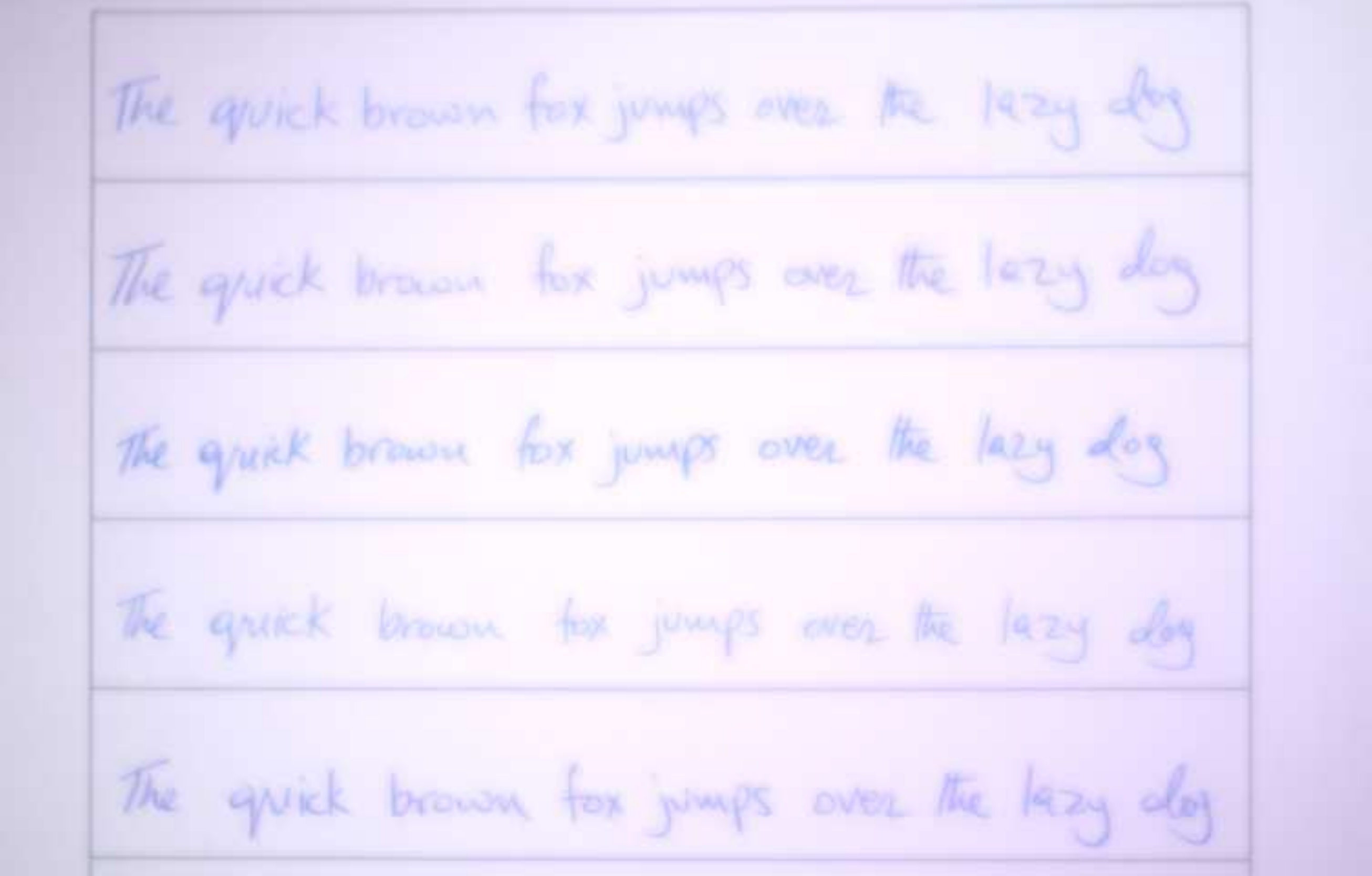}}}
\subfigure[Single Band]{\label{fig:band-640}\fbox{\includegraphics[width=0.44\linewidth]{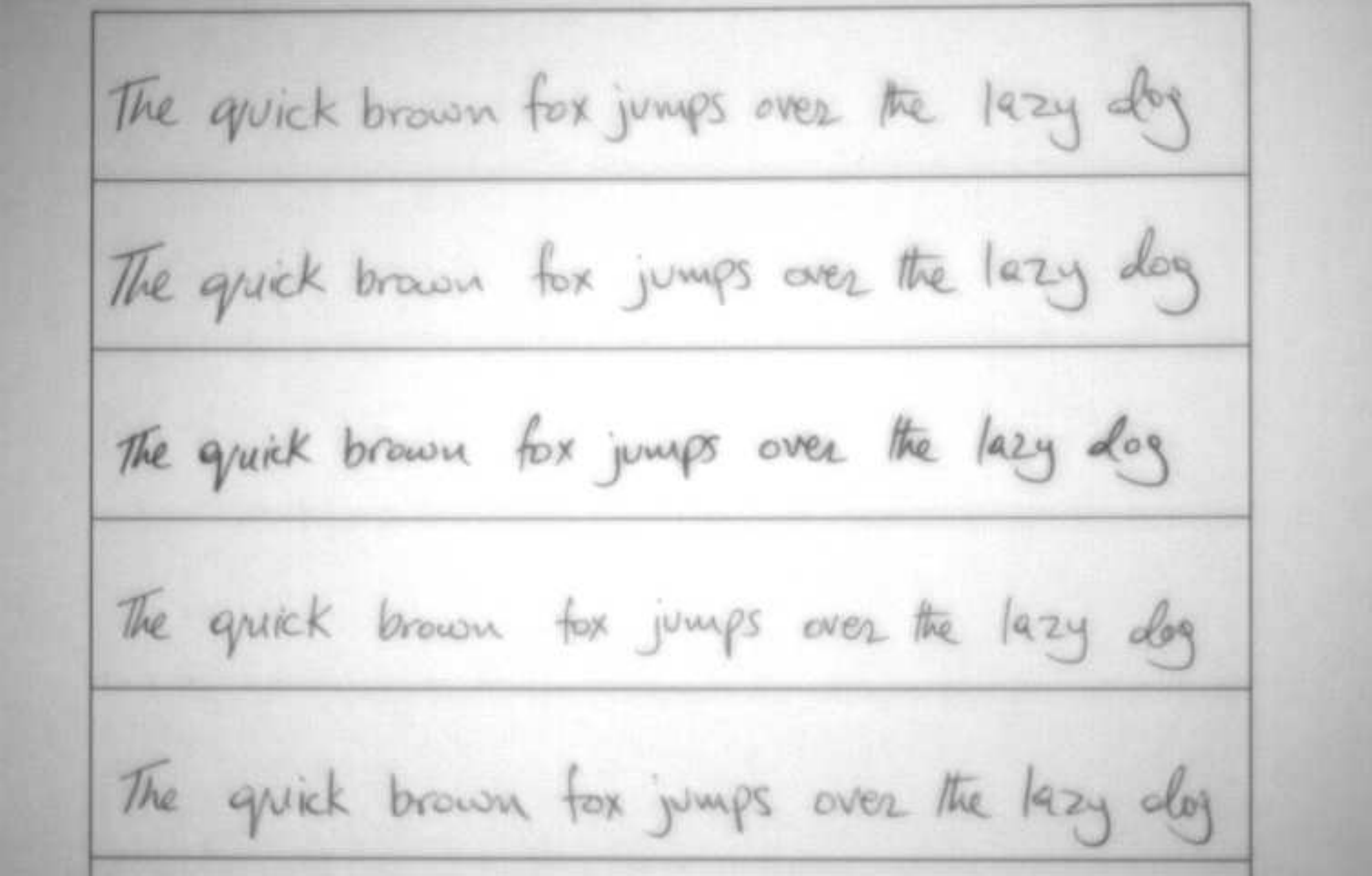}}}\\
\subfigure[Otsu's method]{\label{fig:res-otsu}\fbox{\includegraphics[width=0.44\linewidth]{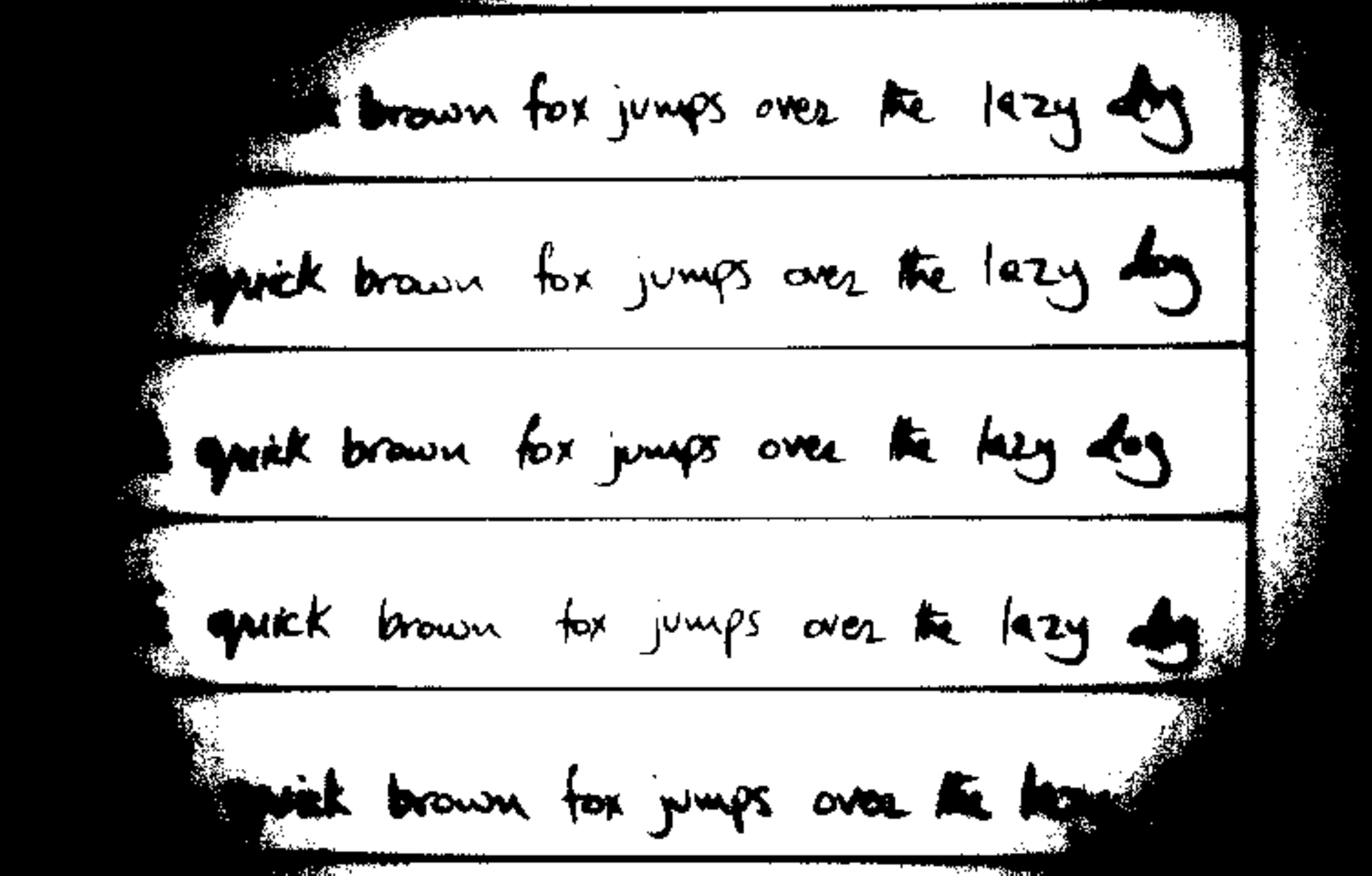}}}
\subfigure[Sauvola's method]{\label{fig:res-sauvola}\fbox{\includegraphics[width=0.44\linewidth]{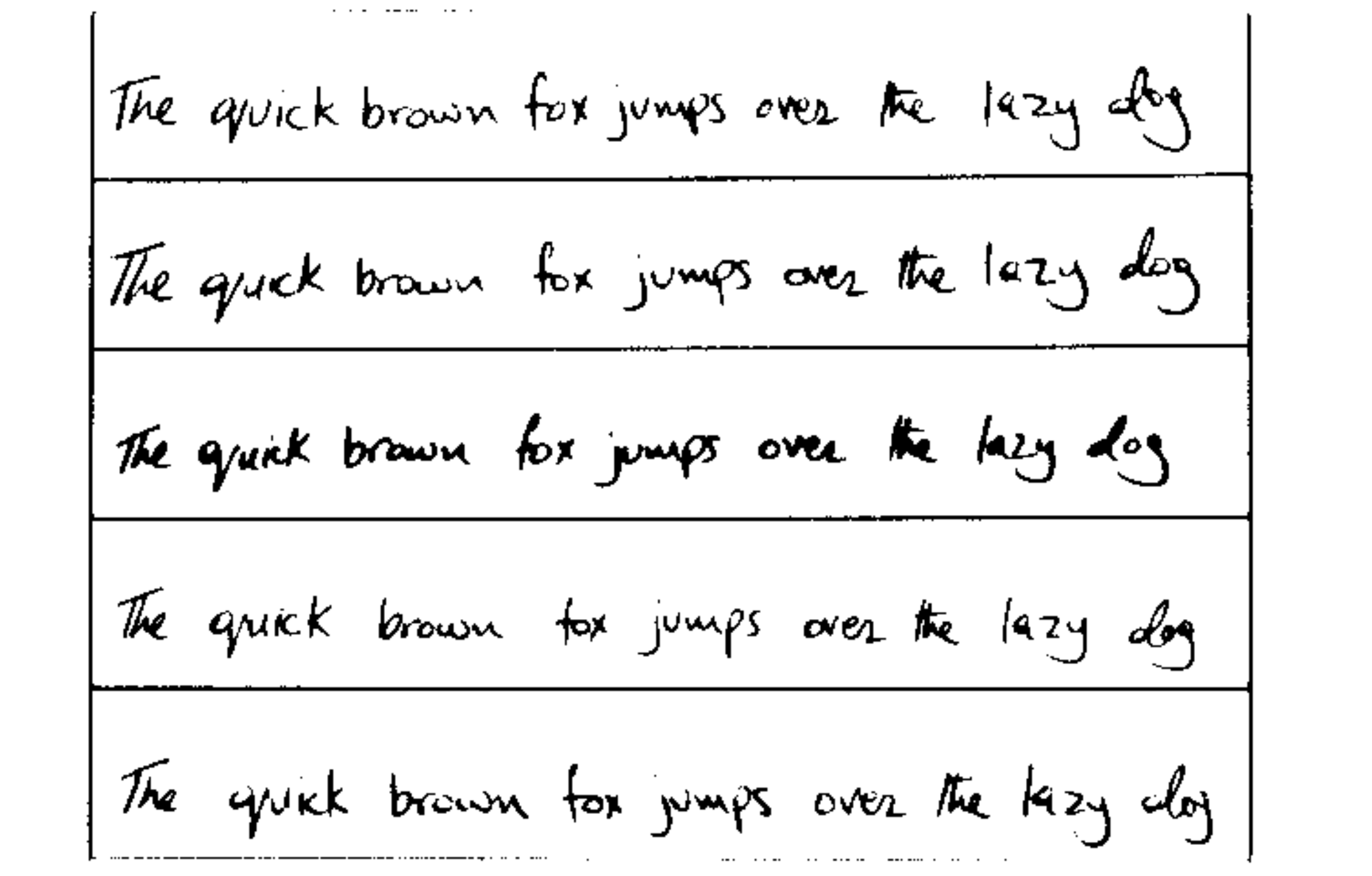}}}
\caption[Hyperspectral document image binarization]{Hyperspectral document image binarization. Notice the high energy in the center of the paper compared to the edges. Local image thresholding is far superior to global image thresholding because of invariance to non-uniform illumination.}
\label{fig:threshold}
\end{figure}

There can be different ways in which the handwritten text can be segmented from the blank paper area. One way is to individually classify the spectral response vectors into ink and non-ink pixels. However, since the spectral responses are modulated by variable illumination, the results may not be optimum. Moreover, it is not possible to train a classifier to cater for all possible kinds of inks from different colored papers.

A better strategy is to binarize spectral bands by well known image thresholding techniques. A global image thresholding method, such as the Otsu~\cite{otsu1975threshold} is ineffective because of the non-uniform illumination over the document (Figure~\ref{fig:raw-image}). A local image thresholding method such as Sauvola~\cite{sauvola2000adaptive} with an efficient integral image based implementation~\cite{shafait2008efficient} more effectively deals with such illumination variations. The Sauvola's method generates a binary mask according to
\begin{equation}
\label{eq:sauvola}
M_{ij} =
\begin{dcases}
1, & \text{if $I_{ij} > \mu_{ij}\left(1+\kappa\left(\frac{\sigma_{ij}}{r-1}\right)\right)$}\\
0, & \text{otherwise}
\end{dcases}
\end{equation}
where
\begin{align*}
\mu_{ij} & = \frac{1}{w^2} \sum_{a=-\frac{w-1}{2}}^{\frac{w+1}{2}} \sum_{b=-\frac{w-1}{2}}^{\frac{w+1}{2}} I(i+a,j+b,c) \\
\sigma_{ij} & = \sqrt{\frac{1}{w^2} \sum_{a=-\frac{w-1}{2}}^{\frac{w+1}{2}} \sum_{b=-\frac{w-1}{2}}^{\frac{w+1}{2}} I(i+a,j+b,c)- \mu_{ij}}
\end{align*}
where ($\mu_{ij},\sigma_{ij}$) are the mean and standard deviation of a $w\times w$ patch centered at pixel ($i,j$). The factors $\kappa$ and $r$ jointly scale the standard deviation term between ($0,1$). The value of $r$ is fixed to the maximum possible standard deviation of the patch which is 128 for an 8-bit image. We empirically found that a patch size $w \times w = 32 \times 32$ and $\kappa=0.15$ give good segmentation results.

Independently applying thresholding to the $p$ bands will result in $p$ different masks which requires reduction to a single mask. Since the intensity of foreground pixels in different spectral bands is variable, the binarization results of each band are not identical. A simple solution is to merge multiple band binarization masks (e.g.~\cite{gatos2008improved}) to get a single binary mask.

An alternative is to apply threshold to a single representative band and propagate the same mask to all the bands. This approach is only applicable if all bands are spatially aligned (which is true in case of stationary document images). We resort to the latter strategy since choosing a representative band for binarization is straightforward. For instance the band with the highest contrast (640nm band, $c=25$) can be chosen (Figure~\ref{fig:band-640}). We observed that selecting any other band in the range [620nm,660nm] resulted in the same mask for various colored inks. Figure~\ref{fig:res-otsu}-\ref{fig:res-sauvola} show the results of Otsu and Sauvola binarization methods. Observe that the Sauvola's method performs better, whereas the Otsu's method returns an inaccurate mask due to non-uniform illumination.

\subsection{Sequential Forward Band Selection}
\label{sec:SFBS}

Feature selection aims to find the feature subset that maximizes a certain performance criteria, generally accuracy~\cite{molina2002feature}. In this section, we aim to find a subset of bands which maximizes ink mismatch detection accuracy (defined later in Section~\ref{sec:IMD}) by a Sequential Forward Band Selection (SFBS) technique. The procedure is described in Algorithm~\ref{alg:SFBS}. Let $\mathbf{X}={[\mathbf{x}_1,\mathbf{x}_2,\ldots,\mathbf{x}_n]}^{\intercal}$ be the matrix of $n$ normalized spectral response vectors, $\mathbf{x} \in \mathbb{R}^{p}$, each corresponding to one ink pixel. In the first step, the mismatch detection accuracy of each of the $p$ bands is computed individually. The band with the highest accuracy is added to the selected band set, $\mathcal{S}$ and removed from the remaining band set $\mathcal{R}$. From the remaining $p-1$ bands in $\mathcal{R}$, one band at a time is combined with $\mathcal{S}$ and the accuracy is observed on the new set. If the accuracy increases, then the added band which maximized accuracy combined with the previously selected bands is retained in $\mathcal{S}$ and removed from $\mathcal{R}$. This process continues until adding another band to $\mathcal{S}$ reduces accuracy. In case all bands are added to $\mathcal{S}$ (which is a rare occurrence), the algorithm will automatically converge.

\begin{algorithm}[h]  
\caption{Sequential Forward Band Selection}          
\label{alg:SFBS}                           
\begin{algorithmic}                    
\Require $\mathbf{X} \in \mathbb{R}^{n \times p}, \mathbf{y} \in \mathbb{Z}^n$
\State $\textbf{Initialize:}~ q \gets 1,$ converge $\gets$ \texttt{false}, $\mathcal{S} \gets \emptyset, \mathcal{R} \gets \{1,2,...,p\}$
\For{$j=1$ to $p$} \Comment{Step 1: find best individual band}
\State $a_j \gets \texttt{IMDaccuracy}(\mathbf{x}_j,\mathbf{y})$
\EndFor
\State $a^{\star} \gets \underset{j}{\max} (a_j),~j^{\star} \gets \underset{j}{\arg \max} (a_j)$ \Comment{best accuracy and band index}
\State $\mathcal{S} \gets \mathcal{S} \cup j^{\star},\; \mathcal{R} \gets \mathcal{R} \setminus j^{\star}$ \Comment{update band sets}
\While {$q\le p \vee \neg \textrm{converge}$} \Comment {Step 2: sequentially add remaining bands}
\For {$k=1$ to $p-q$} \Comment{loop over remaining bands}
\State $\mathcal{T} \gets \mathcal{S} \cup \mathcal{R}(k)$ \Comment{add $k^{\textrm{th}}$ band to the temporary set}
\State $a_k \gets \texttt{IMDaccuracy}(\mathbf{X}_{\centerdot\mathcal{T}},\mathbf{y})$ \Comment{Section~\ref{sec:IMD}}
\EndFor
\State $b^{\star} \gets \underset{k}{\max} (a_k),\;k^{\star} \gets \underset{k}{\arg \max} (a_k)$
\If {$b^{\star}>a^{\star}$} \Comment{Step 3: check if accuracy improved}
\State $\mathcal{S} \gets \mathcal{S} \cup k^{\star},\; \mathcal{R} \gets \mathcal{R} \setminus k^{\star}$ \Comment{update band sets}
\State $q \gets q+1, a^{\star} \gets b^{\star}$
\Else \Comment{accuracy did not improve}
\State{converge $\gets$ \texttt{true}}
\EndIf
\EndWhile
\Ensure $\mathcal{S} \in \mathbb{Z}^q$
\end{algorithmic}
\end{algorithm}

\subsection{Joint Sparse Band Selection}
\label{sec:JSBS}


In Chapter~\ref{Chapter5}, we discussed a way to obtain sparse PCA basis by imposing the $\ell_0$ constraint upon the regression coefficients (basis vectors).
\begin{equation}
\label{eq:SPCACriterion}
\underset{\hat{\mathbf{A}},\hat{\mathbf{B}}}{\arg \min} {\|\mathbf{X}-\mathbf{XBA}^{\!\intercal}\|}_F^2+ \lambda\sum_{j=1}^{k}{\|\bm{\beta}_j\|}_0 \qquad \text{subject to}\, \mathbf{A}^{\!\intercal}\mathbf{A}=\mathbf{I}_k~,
\end{equation}
where $\mathbf{B} \in \mathbb{R}^{p \times k}$ corresponds to the required sparse basis $\{\bm{\beta}_1,\bm{\beta}_2,\ldots,\bm{\beta}_k\}$. When $\lambda$ is large, most coefficients of $\bm{\beta}_j$ will shrink to zero, resulting in sparsity as shown in Figure~\ref{fig:pat1}.

\begin{figure}[!h]
\footnotesize
\centering
\subfigure[SPCA Basis]{\label{fig:pat1}
\begin{minipage}[b]{0.3\linewidth}
\centering
\includegraphics[trim = 232pt 32pt 240pt 2pt, clip, width=0.85\linewidth]{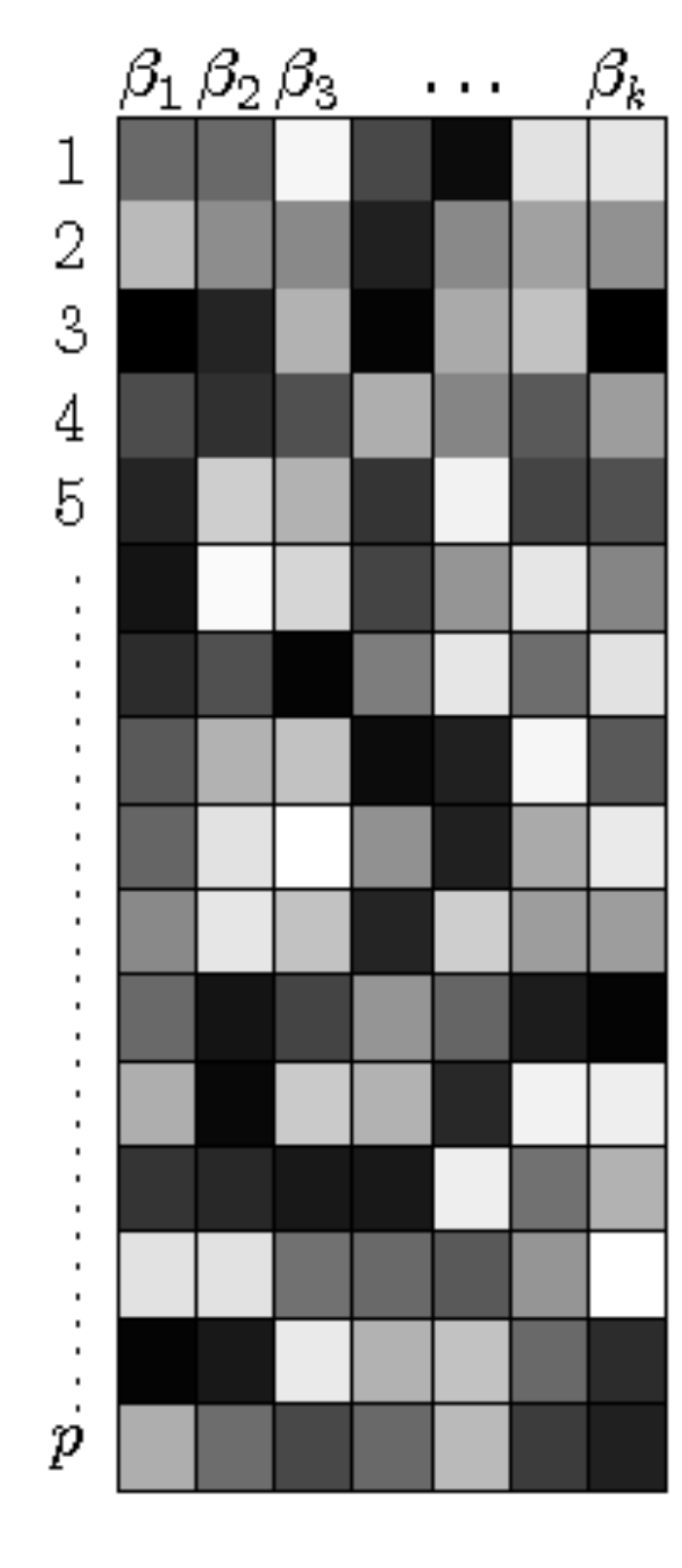}
\end{minipage}}
\subfigure[JSPCA Basis]{\label{fig:pat4}
\begin{minipage}[b]{0.3\linewidth}
\centering
\includegraphics[trim = 232pt 32pt 240pt 2pt, clip, width=0.85\linewidth]{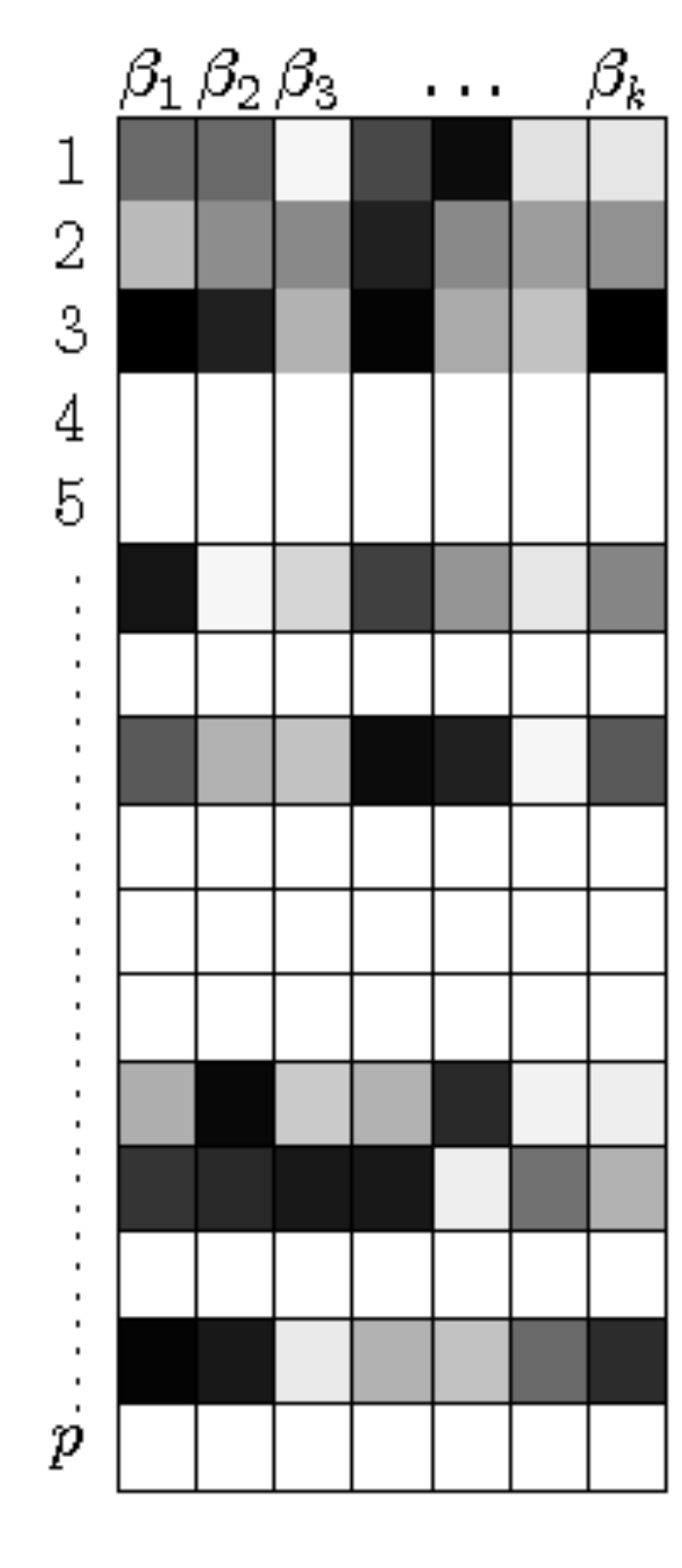}
\end{minipage}}
\caption[Basis vectors sparsity patterns]{This example illustrates basis computed on a pseudo-random data consisting of $p$ bands. Each column is a basis vector $\beta_j$. Dark rectangles are non-zero coefficients. The simple sparsity is unconstrained with respect to the basis and therefore may end up using all bands. Joint sparsity penalizes the rows of the basis and explicitly uses a few bands.}
\label{fig:patterns}
\end{figure}


Since each row of $\mathbf{B}$ corresponds to a particular band, the sparsity within each row of $\mathbf{B}$ may be used to find the relative importance of the bands. If \emph{all} coefficients in a row turn out to be zero, the corresponding band will become irrelevant to the computation of the basis vectors. We therefore suggest a special type of sparsity which enforces most rows of $\mathbf{B}$ to be zero, while the other rows may contain all non-zero coefficients. This type of sparsity among the basis vectors may be called \emph{joint sparsity} (Figure~\ref{fig:pat4}). Note that the existing variants of sparse PCA~\cite{zou2006sparse,jolliffe2003modified,d2007direct,shen2008sparse,journee2010generalized} do not account for this type of joint sparsity, and therefore cannot be used for band selection. We propose a joint sparse PCA formulation which enforces the $\ell_{2,0}$ matrix norm on $\mathbf{B}$ instead of the previously used $\ell_{0}$ vector norm.

\begin{equation}
\label{eq:JSPCACriterion-ell20} \underset{\hat{\mathbf{A}},\hat{\mathbf{B}}}{\arg \min} {\|\mathbf{X}-\mathbf{XBA}^{\!\intercal}\|}_F^2+\lambda\|{\|\bm{\beta}^i\|_2}\|_0 \qquad \text{subject to}\, \mathbf{A}^{\!\intercal}\mathbf{A}=\mathbf{I}_k
\end{equation}
where $\bm{\beta}^i$ is the $i^\textrm{th}$ row of $\mathbf{B}$ ($1<i<p$). The minimization of $\ell_0$-norm over the $\ell_2$-norm of the rows of $\mathbf{B}$ will force most of the rows to be null (for a sufficiently high value of $\lambda$). Although use of $\ell_{2,0}$ penalty gives a joint sparse basis, it makes the minimization non-convex and its solution is NP-hard. In general, the $\ell_0$ norm minimization is relaxed to $\ell_{1}$ norm~\cite{tibshirani1996regression} to reach an approximate solution. Therefore, we solve the following approximation to the joint sparse PCA formulation
\begin{equation}
\label{eq:JSPCACriterion-ell21} \underset{\hat{\mathbf{A}},\hat{\mathbf{B}}}{\arg \min} {\|\mathbf{X}-\mathbf{XBA}^{\!\intercal}\|}_F^2+\lambda\sum_{i=1}^{p} \|\bm{\beta}^i\|_2 \qquad \text{subject to}\, \mathbf{A}^{\!\intercal}\mathbf{A}=\mathbf{I}_k~.
\end{equation}
where the regularization term is often called $\ell_{2,1}$ norm of a matrix which has deemed useful for multi-task learning~\cite{liu2009multi} and feature selection~\cite{zhao2010efficient}.

Although, the above formulation ensures a joint sparse PCA basis, simultaneous optimization of $\mathbf{A}$ and $\mathbf{B}$ makes the problem non-convex. A locally convex solution of~(\ref{eq:JSPCACriterion-ell21}) can be obtained by iteratively minimizing $\mathbf{A}$ and $\mathbf{B}$. The regularization loss term ${\|\mathbf{X}-\mathbf{XBA}^{\!\intercal}\|}_F^2$ in~(\ref{eq:JSPCACriterion-ell21}) is equivalent to minimizing ${\|\mathbf{XA}-\mathbf{XB}\|}_F^2$ given $\mathbf{A}^{\!\intercal}\mathbf{A}=\mathbf{I}_k$, the proof of which has been presented in Chapter~\ref{Chapter5}. A closed form solution for optimizing with respect to $\mathbf{A}$ can be obtained by computing a reduced rank procrustes rotation~\cite{zou2006sparse} which is also proven in Chapter~\ref{Chapter5}.

The alternating minimization process is repeated until convergence or until a specified number of iterations is reached. The obtained joint sparse basis $\mathbf{B}$ has exactly $q<p$ non-zero rows. A reduced band index set $\mathcal{R}$ is computed which contains the indices of the non-zero rows of $\mathbf{B}$. Then all possible subsets $\mathcal{T}$ of the reduced band set $\mathcal{R}$ are tested for ink mismatch detection. This is done by obtaining a reduced data matrix ($\mathbf{X}_{\centerdot\mathcal{T}}$) by taking the columns of $\mathbf{X}$ indexed by each set $\mathcal{T}$ and computing mismatch detection accuracy. The subset $\mathcal{T}$ which maximizes accuracy is chosen as the selected band set $\mathcal{S}$. Algorithm~\ref{alg:JSBS} summarizes the procedure for Joint Sparse Band Selection (JSBS).


\begin{algorithm}[h]  
\caption{Joint Sparse Band Selection}          
\label{alg:JSBS}                           
\begin{algorithmic}                    
\Require $\mathbf{X} \in \mathbb{R}^{n \times p}, \mathbf{y} \in \mathbb{R}^n, \lambda, \epsilon, i_{\textrm{max}}$
\State $\textbf{Initialize:}~ i \gets 1, \textrm{converge}\gets\texttt{false}, \mathcal{S} \gets \emptyset, \mathcal{R} \gets \{1,2,...,p\}$
\State $\mathbf{USV}^{\intercal} \gets \mathbf{X}$ \Comment {SVD of X}
\State $\mathbf{A} \gets \mathbf{V}_{\centerdot \{1:k\}}, \mathbf{B} \gets \mathbf{0} $ \Comment{Initialize $\mathbf{A}$ with $1^\textrm{st}$ $k$ basis vectors and $\mathbf{B}$ with zeros}
\While {$i\le i_{\textrm{max}} \vee \neg \textrm{converge}$}
\State $\hat{\mathbf{B}} \gets \underset{\mathbf{B}}{\min} \|\mathbf{XA}-\mathbf{XB}\|_F^2 + \lambda {\|\mathbf{B}\|}_{2,1}$ \Comment{Step 1: Find $\mathbf{B}$ given $\mathbf{A}$}
\State $\dot{\mathbf{U}} \dot{\mathbf{S}} \dot{\mathbf{V}}^{\intercal} \gets {\mathbf{X}}^{\intercal} \mathbf{X}\hat{\mathbf{B}}$ \Comment{Step 2: Find $\mathbf{A}$ given $\mathbf{B}$}
\State $\hat{\mathbf{A}} \gets \dot{\mathbf{U}}\dot{\mathbf{V}}^{\intercal}$
\If {$\|\mathbf{B}-\hat{\mathbf{B}}\|_F < \epsilon $} \Comment {Step 3: Check convergence}
\State $\textrm{converge} \gets \texttt{true}$
\Else
\State $\mathbf{B} \gets \hat{\mathbf{B}}, \mathbf{A} \gets \hat{\mathbf{A}}$ \Comment {Update $\mathbf{A},\mathbf{B}$}
\State $i \gets i+1$
\EndIf
\EndWhile
\State $\mathcal{R} \gets \{j \in \mathcal{R}: \|\bm{\beta}^j\|_2=0\}$ \Comment {Step 4: Obtain reduced subset of bands}
\State $\mathcal{S} \gets \underset{\mathcal{T}}{\arg \max} \texttt{ IMDaccuracy}(\mathbf{X}_{\centerdot\mathcal{T}},\mathbf{y}), \forall \, \mathcal{T}\subseteq \mathcal{R}$ \Comment{Section~\ref{sec:IMD}}
\Ensure $\mathcal{S} \in \mathbb{Z}^q$
\end{algorithmic}
\end{algorithm}

\subsection{Ink Mismatch Detection Accuracy Computation}
\label{sec:IMD}

In Section~\ref{sec:SFBS} and Section~\ref{sec:JSBS}, we proposed two algorithms for band selection from ink spectral responses. Using the selected bands, a reduced data matrix $\mathbf{Z}$ is used for ink mismatch detection. We cluster the pixels belonging to $g$ inks using k-means algorithm~\cite{jain1999data}. K-means minimizes the squared Euclidean error between the cluster centroid and its members by the following criteria
\begin{equation}
\underset{\hat{\mathcal{C}}}{\arg \min} \sum_{i=1}^{g} \sum_{\mathbf{z}^j \in \hat{\mathcal{C}}_i} {\|\mathbf{z}^j-\bar{\mathbf{z}}_i\|}^2~,
\end{equation}
where $\|.\|^2$ is the squared error between the cluster member $\mathbf{z}^j \in \mathbb{R}^{q}$ and its centroid $\bar{\mathbf{z}}_i$. $\mathbf{z}^j$ is the $j^{th}$ row of matrix $\mathbf{Z}$ which is in the $i^\textrm{th}$ cluster $\hat{\mathcal{C}}_i$. The total number of clusters is $g$ which relates to the number of mixed inks in $\mathbf{Z}$.

Let $\mathbf{Y} \in \mathbb{R}^{n \times g}$ be the ground truth class indicator matrix such that
\begin{equation}
Y_{ij} =
\begin{cases}
1& \text{if $\mathbf{z}^j \in \mathcal{C}_i$}\\
0& \text{otherwise}
\end{cases}
\end{equation}

where $\mathcal{C}_i$ is the ground truth cluster of ink $i$. Also, let $\hat{\mathbf{Y}} \in \mathbb{R}^{n \times g}$ be the class indicator matrix predicted by k-means clustering
\begin{equation}
\hat{Y}_{ij} =
\begin{cases}
1& \text{if $\mathbf{z}^j \in \hat{\mathcal{C}}_i$}\\
0& \text{otherwise}
\end{cases}
\end{equation}

The mismatch detection accuracy of $i^\textrm{th}$ ink class is defined as the number of correctly labeled pixels of $i^\textrm{th}$ ink divided by the number of pixels labeled with $i^\textrm{th}$ ink in either the ground truth labels $\mathbf{y}_i$ or the predicted labels $\hat{\mathbf{y}}_i$~\cite{pascal-voc-2012}. The mismatch detection accuracy is computed as
\begin{equation}
\textrm{accuracy} = \underset{g!}{\max} \frac{1}{g} \sum_{i=1}^{g} \frac{T_i}{T_i + F_i + N_i}~,
\end{equation}
where
\begin{align*}
T_i & = \mathbf{y}_i \wedge \hat{\mathbf{y}}_i \rightarrow \text{no. of correctly labeled pixels of the $i^\textrm{th}$ ink}\\
F_i & = \mathbf{y}_i'\wedge \hat{\mathbf{y}}_i \rightarrow \text{no. of pixels incorrectly labeled as $i^\textrm{th}$ ink}\\
N_i & = \mathbf{y}_i \wedge \hat{\mathbf{y}}_i'\rightarrow \text{no. of incorrectly labeled pixels of $i^\textrm{th}$ ink}
\end{align*}
It is important to note that according to this evaluation metric, the accuracy of a random guess (e.g.~in a two class problem) will be 1/3. This is different to common classification accuracy metrics where the accuracy of a random guess is 1/2. This is because the metric additionally penalizes false negatives which is crucially important in mismatch detection problem.


\section{Writing Ink Hyperspectral Image Database}
\label{sec:database}

Traditionally, document analysis revolves around monochromatic or trichromatic (RGB) imaging, often captured by scanners.
Cameras have now emerged as an alternative to scanners for capturing document images. However, the focus has remained on mono-/tri-chromatic imaging. We outline and discuss the key components of our hyperspectral document imaging system, which offers new challenges and perspectives. We discuss the issues of filter transmittance and spatial/spectral non-uniformity of the illumination and propose possible solutions via pre and post-processing. This section provides an overview of our hyperspectral document imaging system and presents our approach for tackling major challenges specific to hyperspectral document imaging.

\begin{figure}[h]
\footnotesize
\centering
\includegraphics[trim = 4pt 0pt 5pt 0pt, clip,width=0.5\linewidth]{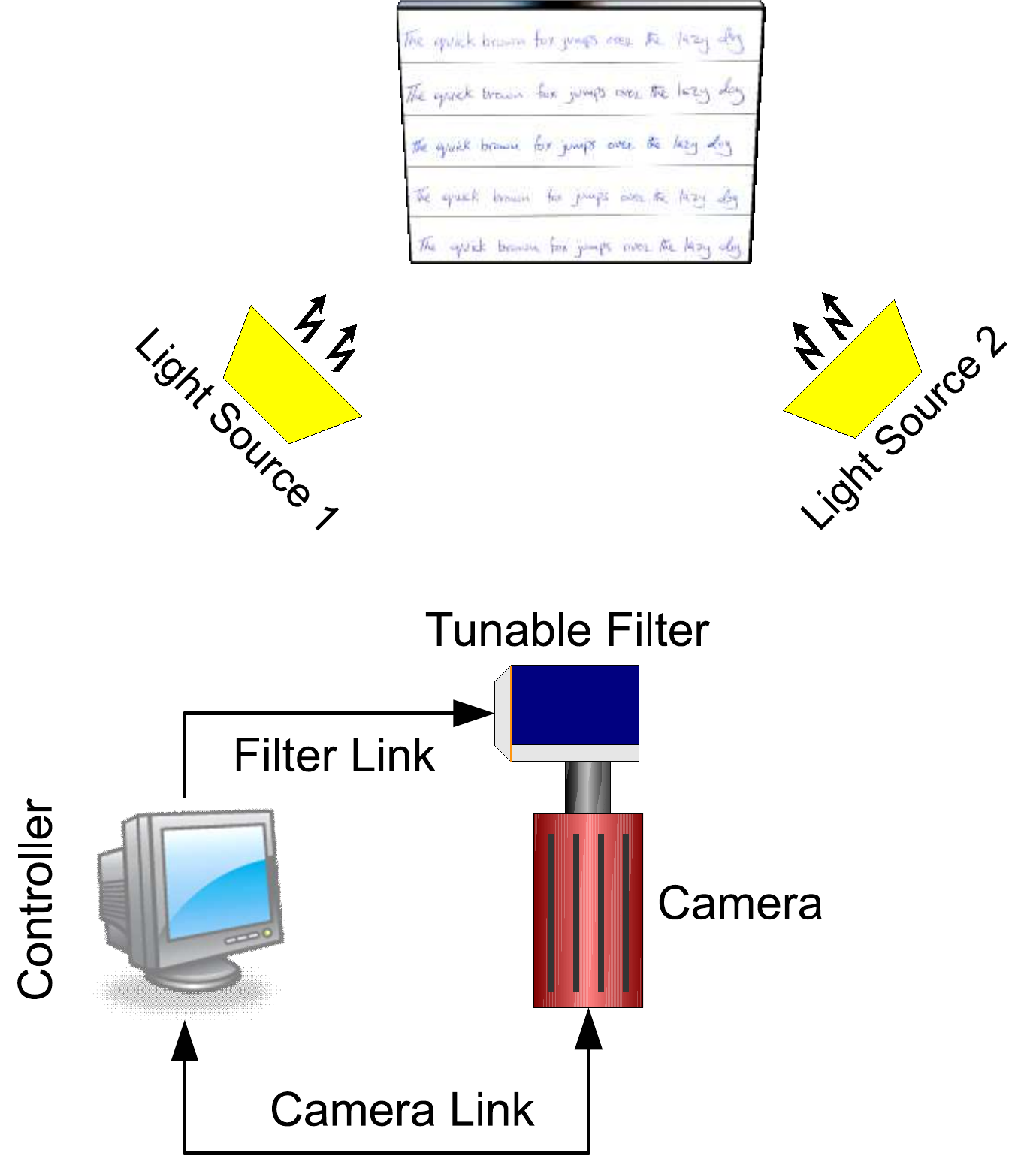}
\caption[Hyperspectral document image acquisition setup]{An illustration of the proposed hyperspectral document image acquisition setup.}
\label{fig:hardware-setup}
\end{figure}

\subsection{Acquisition Setup}
\label{sec:acquisition-setup}


We used the variable exposure imaging setup described in Chapter~\ref{Chapter3}, with halogen illumination over the document. The hyperspectral image of a document is captured in the 400-720nm range at steps of 10 nm which results in a 33 band hyperspectral image. It takes less than 5 seconds to sequentially capture 33 bands of a hyperspectral image. We also collected RGB scanned images at resolutions of 150 and 300 dpi using a flatbed scanner. These RGB images provide baseline information for comparison with HSI. For a fair comparison between RGB and HSI, it is important that their spatial resolutions are similar so that any performance differences can be attributed to the spectral dimension. This is the main reason for selecting low resolution in RGB scans. The RGB images still have the advantage of a flatbed scanning system, i.e.~the illumination is uniformly distributed over the imaging surface.

\subsection{Database Specifications}
\label{sec:data-specs}

We prepared a dataset comprising a total of 70 hyperspectral images of a hand-written note in 10 different inks by 7 subjects. All subjects were instructed to write the sentence, \emph{`The quick brown fox jumps over the lazy dog'}, once in each ink on a white paper. The pens included 5 varieties of blue ink and 5 varieties of blank ink. It was ensured that the pens came from different manufacturers while the inks still appeared visually similar. All efforts were made to avoid prolonged exposure to ambient/daylight by keeping the samples under cover in dark. This is because different inks are likely to undergo a transformation in their spectral properties induced by light. Such an occurrence would, although favor distinguish two different inks, but would bias our analysis. Moreover, all samples were collected from the subjects in one session so that their effective age is the same. 

\subsection{Spectral Normalization}
\label{sec:spectral-norm}

Common illumination sources generate lower spectral power in shorter wavelengths as compared to longer wavelengths. An observation of the LCTF transmittance in Figure~\ref{fig:filter-trans} suggests that the amount of light transmitted is a function of the wavelength such that, the shorter the wavelength $\lambda$, the higher the transmittance and vice versa. Extremely low filter transmittance in (400nm-450nm) results in insufficient energy at the imaging sensor. Finally, the sensor quantum efficiency is also variable with respect to the wavelength as shown in Figure~\ref{fig:basler-QE}. Typically, each band of a hyperspectral image is captured with a fixed exposure time. If imaging is done with fixed exposure setting for each band, the captured hyperspectral image looks as shown in Figure~\ref{fig:sample-cons} which has low energy for the bands corresponding to the blue region of the spectrum.

Since each band is sequentially captured in our imaging technique, we vary exposure time before the acquisition is triggered for each band. This enables compensation for spectral non-uniformity due to filter transmittance and illumination. The exposure times are pre-computed for each band in a calibration step such that a white patch has a flat (uniform) spectral response measured at the sensor. For spectral response calibration, the white patch of a color checker is utilized as a reference. Using variable exposure time in Figure~\ref{fig:camera-exp}, high energy is captured in all bands as shown in Figure~\ref{fig:sample-auto}.

\subsection{Spatial Normalization}
\label{sec:spatial-norm}
In hyperspectral document imaging, the use of a nearby illumination source induces a scalar field over the target image. This means that there is a spatially non-uniform variation in illumination energy. The result is that the pixels near the center of the image are brighter (have higher energy) as compared to the pixels farther away towards the edges. This effect can be seen in Figure~\ref{fig:raw-image}.
Let $\mathbf{I}_{ij}$ be the spectral response at the image pixel $(i,j)$. It can be reasonably assumed here that the non-uniformity in illumination is only a function of pixel coordinates $(i,j)$ and does not depend on the spectral dimension. This assumption will hold for each $(i,j)$ as long as $\mathbf{I}_{ij}$ is not saturated. Hence, normalizing the spectral response vector at each pixel to a unit magnitude will largely compensate for the effect of non-uniform illumination intensity.
\begin{equation}
\label{eq:normalization}
\hat{\mathbf{I}}_{ij} = \frac{\mathbf{I}_{ij}}{\|\mathbf{I}_{ij}\|}~.
\end{equation}


\section{Experiments and Analysis of Results}
\label{sec:experiments}


The UWA writing ink hyperspectral image dataset contains handwritten notes in blue and black inks. It is highly unlikely to perform ink mismatch detection between different colored inks, as they can be easily distinguished by a visual examination. Therefore, we choose to independently perform ink mismatch detection experiments for blue and black inks. Mixed ink handwritten notes are produced from single ink notes by taking individual words of each ink in equal proportion. If the number of mixed inks is unknown, $g$ can theoretically lie in the range $[1,n]$. For the sake of this analysis, we fix $g=2$, i.e., we assume that there are two possible inks in the image. This is a practical assumption in questioned document examination where the original note is written with one ink pen and is suspected to be forged by a second ink pen. In our analysis, five different inks, taken two at a time, results in ten ink combinations, for blue and black color each. In the following experiments, c$_{_{ij}}$ will denote the combination of ink $i$ with ink $j$ such that $i=1,...,5$ and $j=i+1,...,5$. The mismatch detection accuracy is averaged over all samples for each ink combination c$_{_{ij}}$.

\begin{landscape}
\begin{figure}
\footnotesize
\centering
\subfigure[]{\label{fig:illum-SPD}\includegraphics[width=0.3\linewidth]{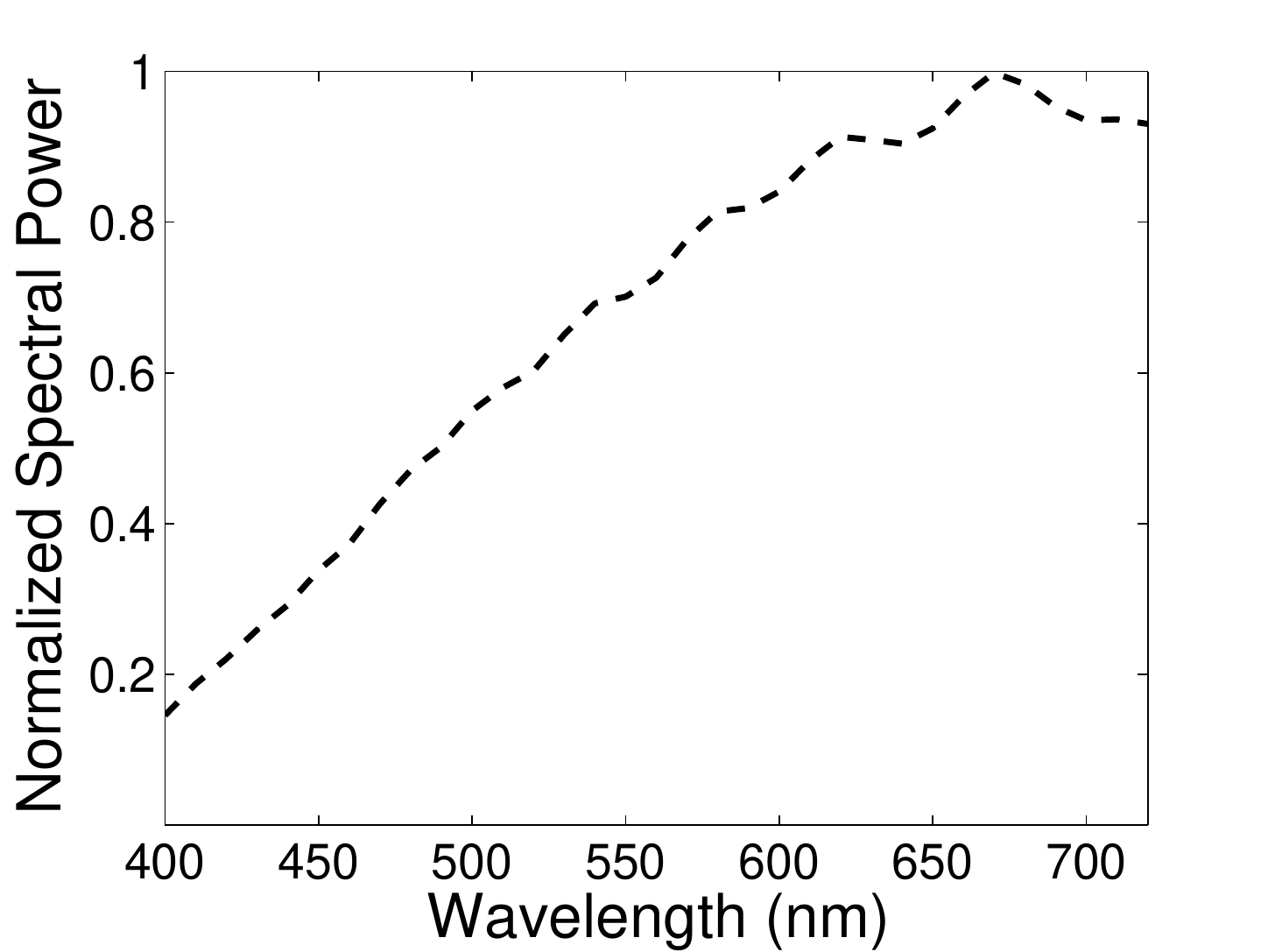}}
\subfigure[]{\label{fig:filter-trans}\includegraphics[trim = 25pt 30pt 5pt 30pt, clip, width=0.33\linewidth]{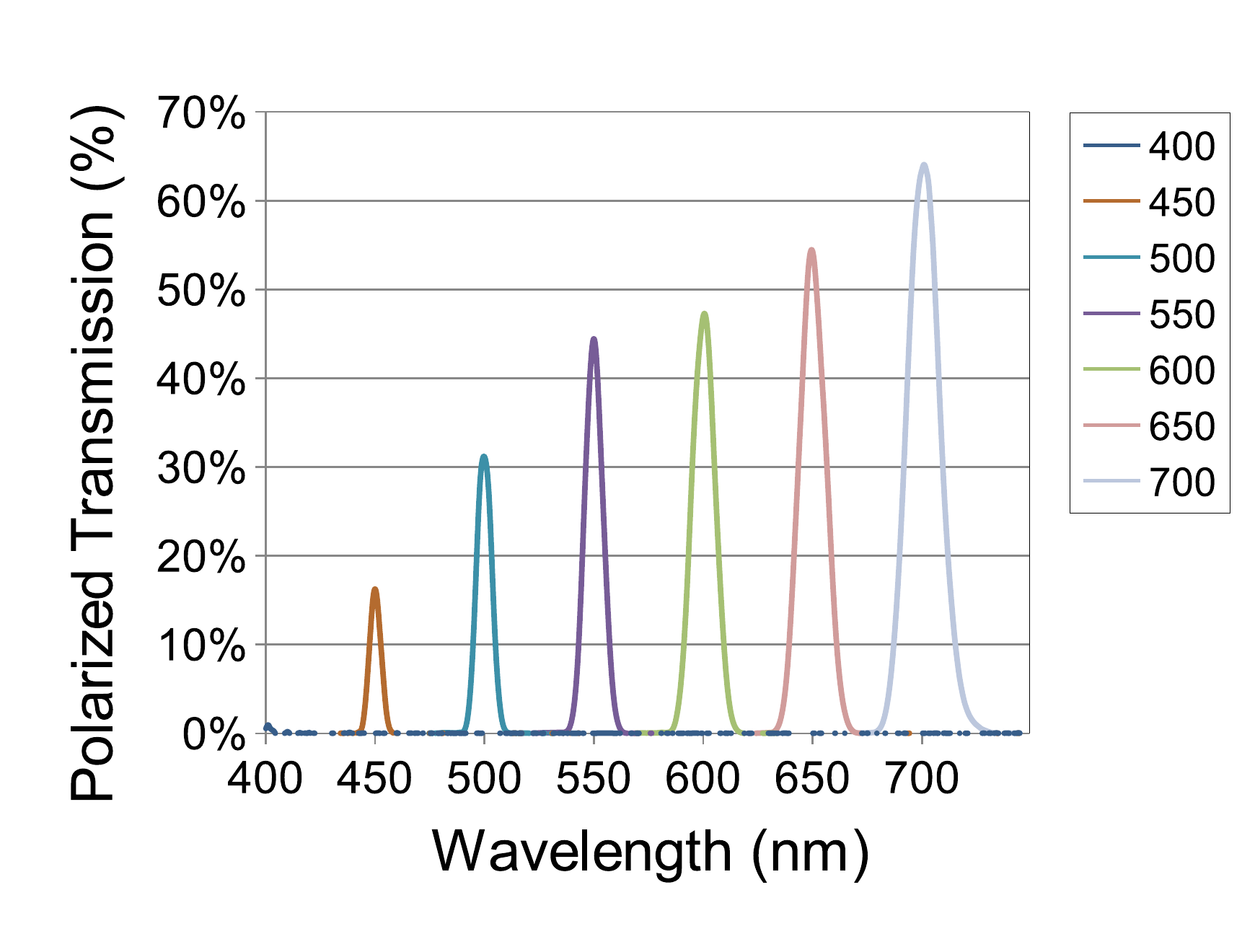}}
\subfigure[]{\label{fig:basler-QE}\includegraphics[width=0.3\linewidth]{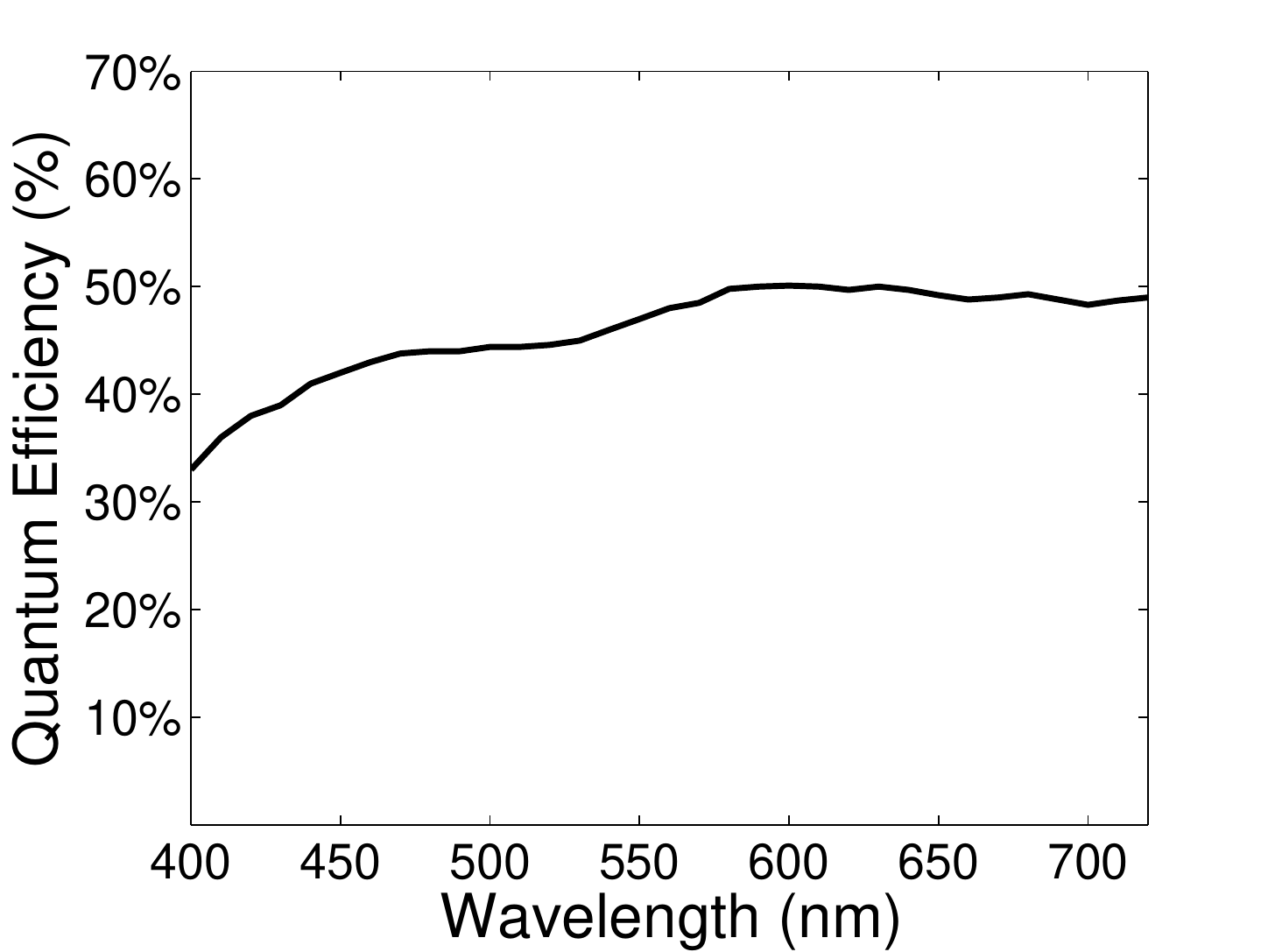}}\\
\subfigure[]{\label{fig:sample-cons}\includegraphics[width=0.34\linewidth]{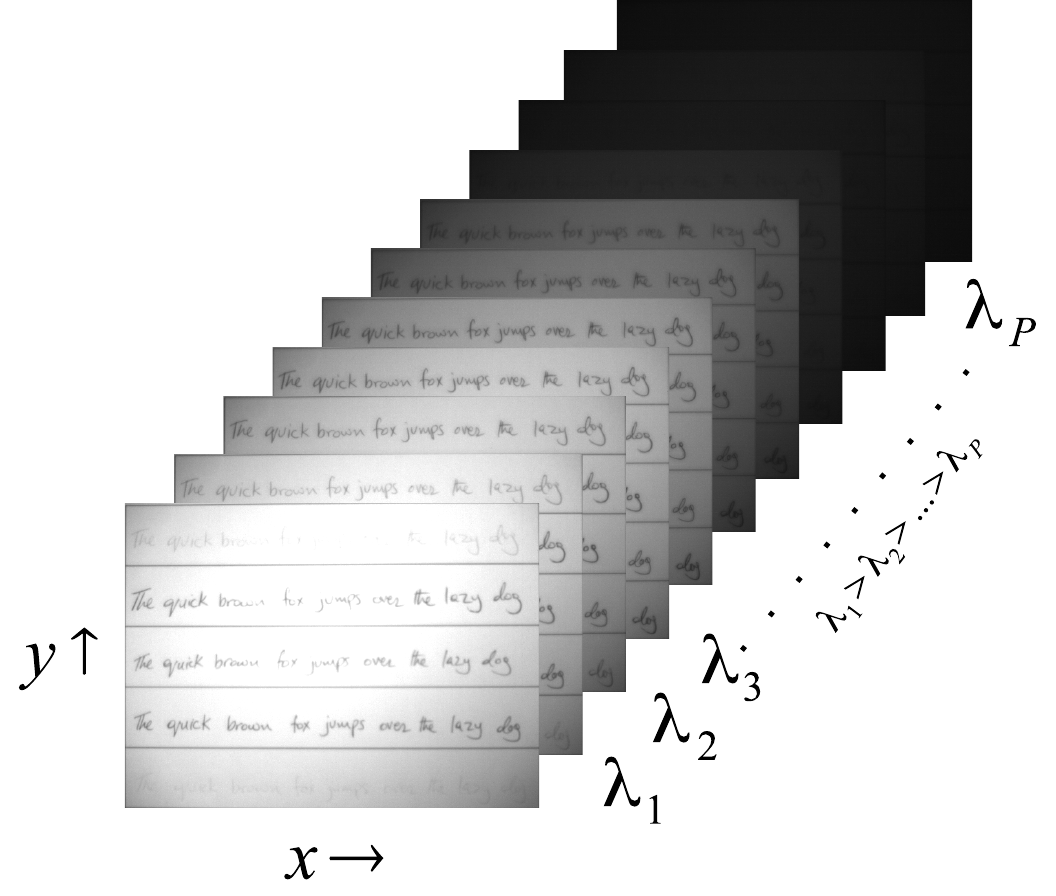}}
\subfigure[]{\label{fig:camera-exp}\includegraphics[width=0.3\linewidth]{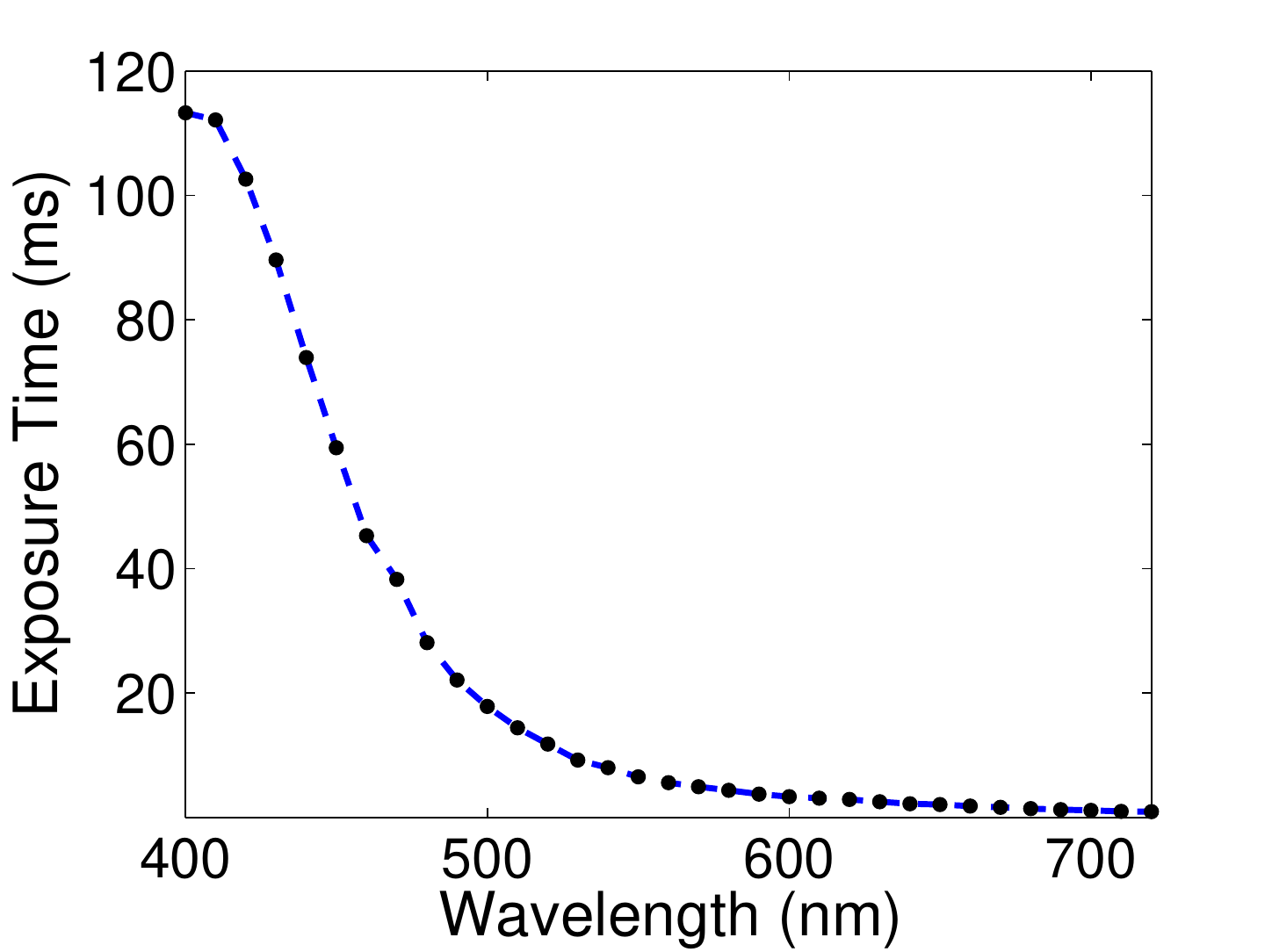}}
\subfigure[]{\label{fig:sample-auto}\includegraphics[width=0.34\linewidth]{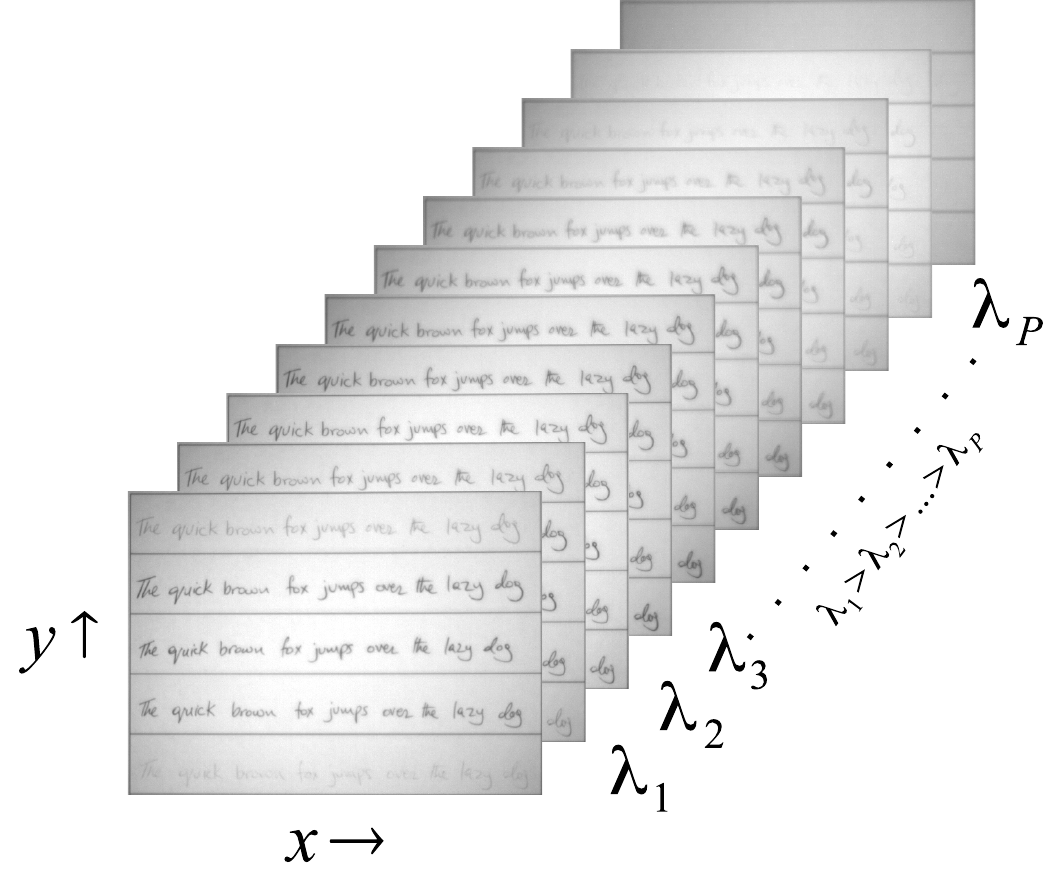}}
\caption[Fixed and variable exposure document imaging]{(a) Spectral power distribution of the illuminant. (b) Transmission functions of the LCTF (only shown for seven different center wavelengths for conciseness). (c) Quantum efficiency of the CCD sensor (d) Image captured with fixed exposure time. (e) Exposure time as a function of wavelength. (f) Image captured with variable exposure time. Observe that the illumination power, filter transmission and sensor quantum efficiency is compensated by variable exposure time.}
\label{fig:syst-response}
\end{figure}
\end{landscape}

\begin{landscape}
\begin{figure}
\centering
\begin{tabular}{cc}
Blue Ink Handwritten Note    & Black Ink Handwritten Note \\
\fbox{\includegraphics[trim = 0pt 5pt 0pt 12pt, clip, width=0.47\linewidth]{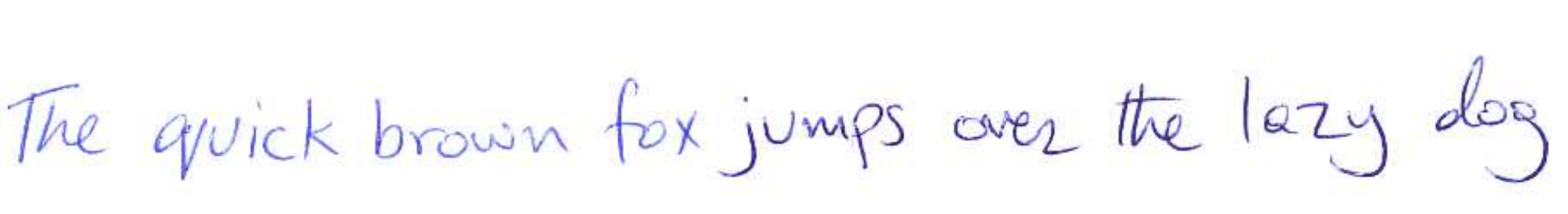}} &
\fbox{\includegraphics[trim = 0pt 0pt 0pt 15pt, clip, width=0.47\linewidth]{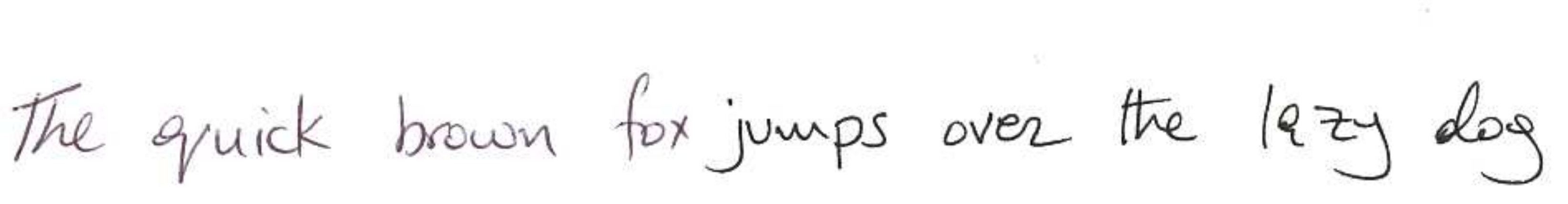}} \\ [2pt]
True Ink Map            & True Ink Map \\
\fbox{\includegraphics[trim = 0pt 5pt 0pt 12pt, clip, width=0.47\linewidth]{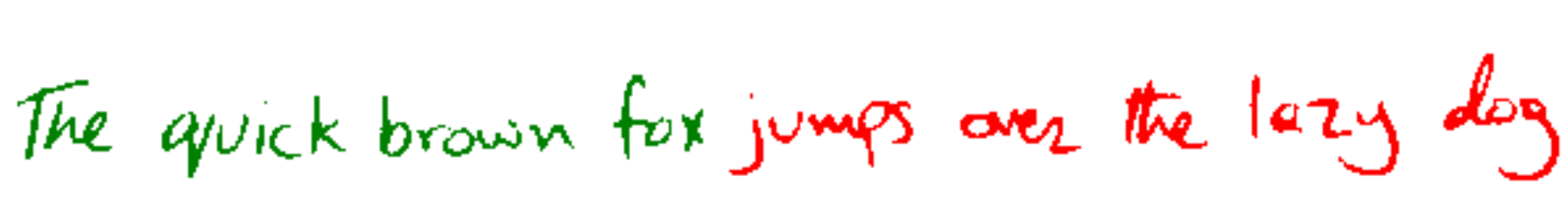}} &
\fbox{\includegraphics[trim = 0pt 0pt 0pt 15pt, clip, width=0.47\linewidth]{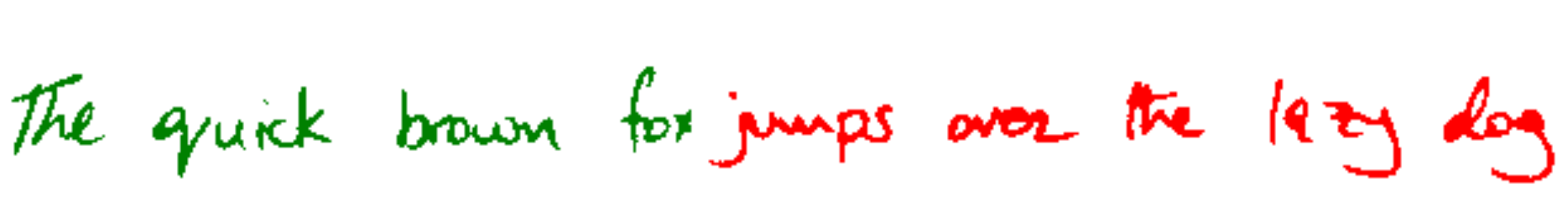}} \\ [2pt]
HS image (raw)          & HS image (raw) \\
\fbox{\includegraphics[trim = 0pt 5pt 0pt 12pt, clip, width=0.47\linewidth]{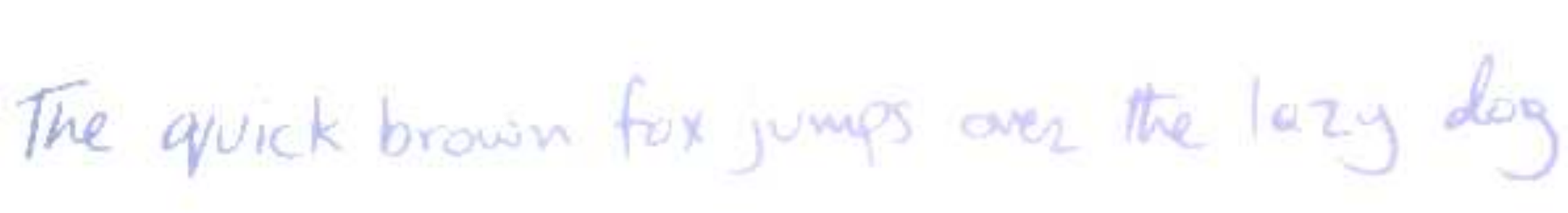}} &
\fbox{\includegraphics[trim = 0pt 0pt 0pt 15pt, clip, width=0.47\linewidth]{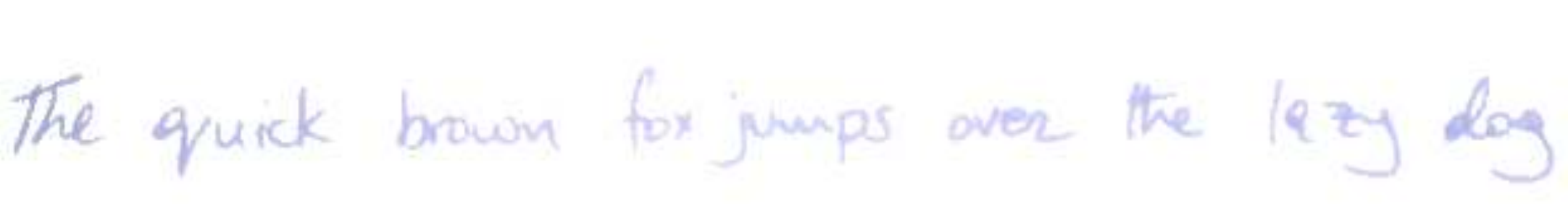}} \\ [2pt]
Result (raw)            & Result (raw) \\
\fbox{\includegraphics[trim = 0pt 5pt 0pt 12pt, clip, width=0.47\linewidth]{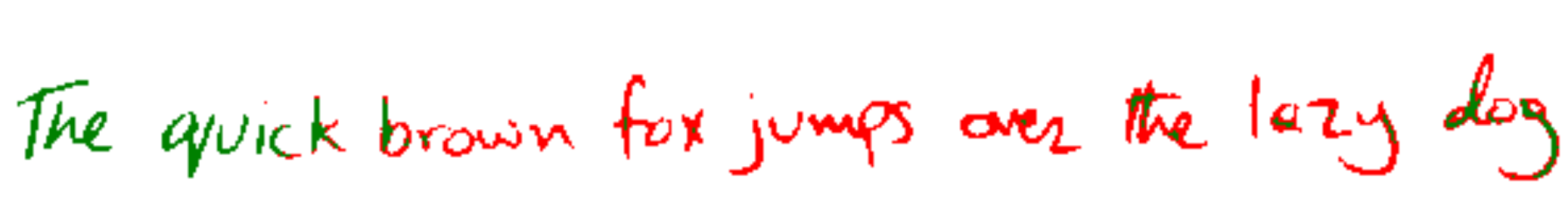}} &
\fbox{\includegraphics[trim = 0pt 0pt 0pt 15pt, clip, width=0.47\linewidth]{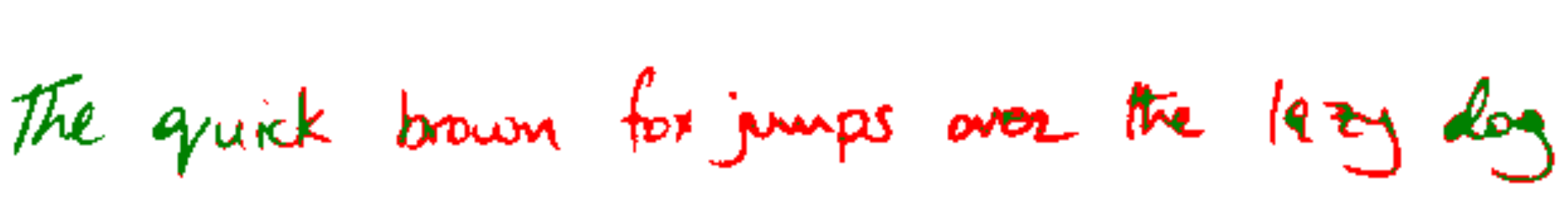}} \\ [2pt]
HS image (norm)         & HS image (norm) \\
\fbox{\includegraphics[trim = 0pt 5pt 0pt 12pt, clip, width=0.47\linewidth]{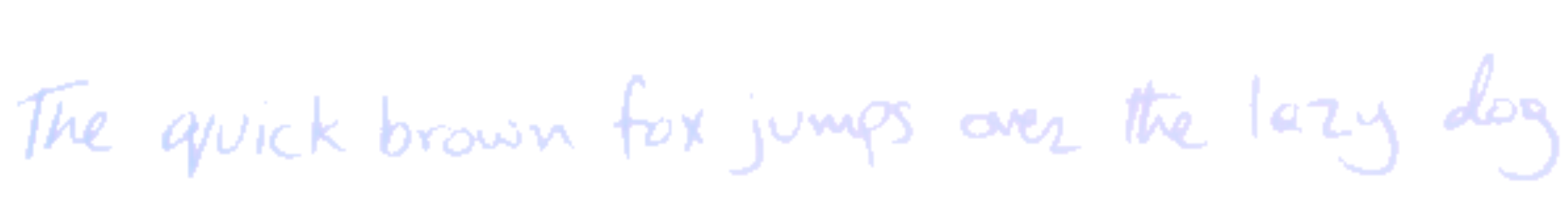}} &
\fbox{\includegraphics[trim = 0pt 0pt 0pt 15pt, clip, width=0.47\linewidth]{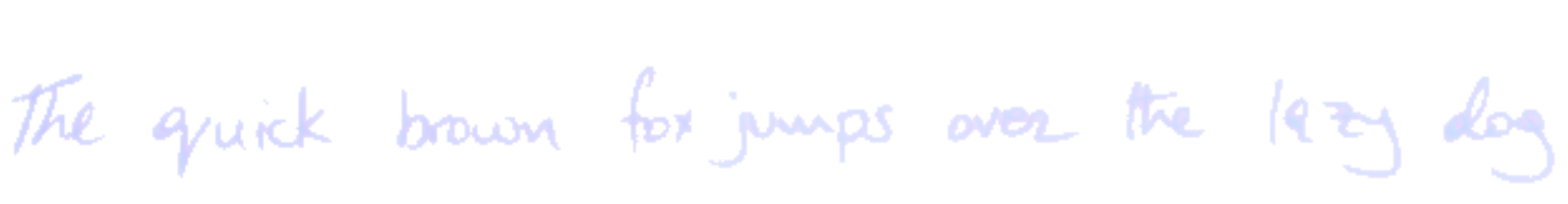}} \\ [2pt]
Result (norm)           & HS image (norm) \\
\fbox{\includegraphics[trim = 0pt 5pt 0pt 12pt, clip, width=0.47\linewidth]{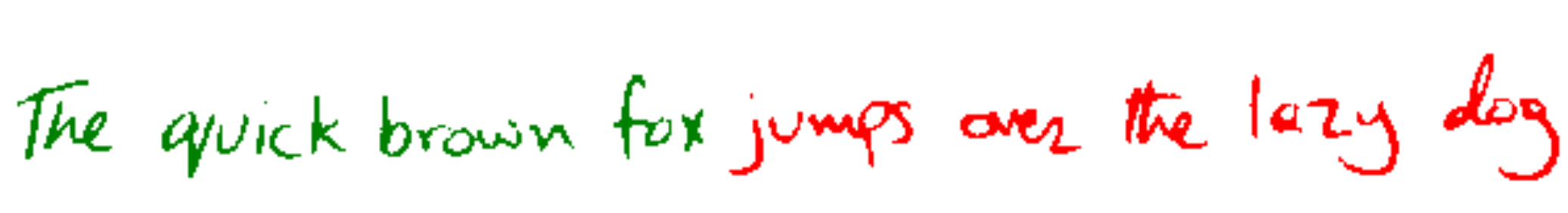}} &
\fbox{\includegraphics[trim = 0pt 0pt 0pt 15pt, clip, width=0.47\linewidth]{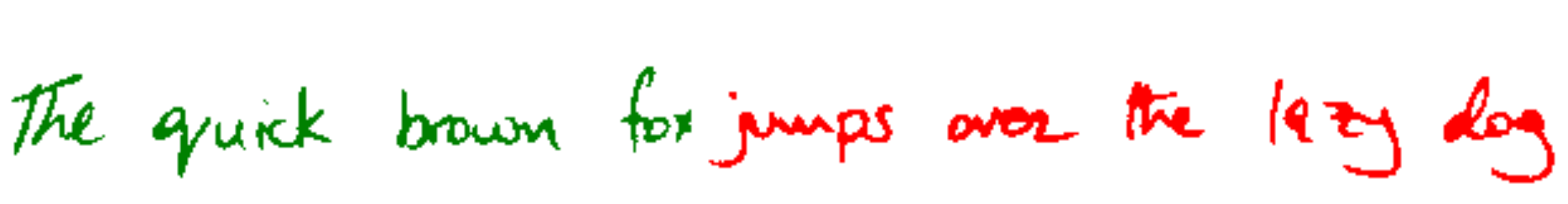}}
\end{tabular}
\caption[Effect of illumination normalization on ink mismatch detection]{An illustration of ink mismatch detection on a blue ink and a black ink image, acquired using adaptive exposure. The ground truth ink pixels are labeled in pseudo colors (red: ink 1, green: ink 2). The spatially non-uniform illumination pattern can be observed in raw HS images, with high energy in the center and low towards the edges. Normalization removes the illumination bias and greatly improves mismatch detection accuracy.}
\label{fig:qual-norm}
\end{figure}
\end{landscape}


We begin by analyzing the efficacy of the proposed hyperspectral document image illumination normalization. Figure~\ref{fig:qual-norm} shows two example handwritten notes in blue and black inks. The images are made by mixing samples of ink 1 and ink 2 of blue and black inks, separately. The original images are shown in RGB for clarity. The ground truth images are labeled in different colors to identify the constituent inks in the note. The spatial non-uniformity of the illumination can be observed from the center to the edges. The mismatch detection results on raw images indicate that the clustering is biased by the illumination intensity, instead of the ink color. After normalization of the raw HS images, it is evident that the illumination variation is highly suppressed. This results in an accurate mismatch detection result that closely follows the ground truth.



We now evaluate ink mismatch detection using RGB images and the effect of different scanning resolutions. Figure~\ref{fig:DPI} shows the average accuracy at two different resolutions for all ink combinations. For most of the blue ink combinations, the preferred choice of resolution is 300 dpi which is superior to 150 dpi in most ink combinations. Interestingly, for black ink combinations, no conclusive evidence is available to support any resolution as the accuracy is within the range of (0.3,0.4). This indicates that RGB images do not carry enough information to differentiate black inks and its accuracy is close to random guess. As a final choice, we resort the 300 dpi RGB images for further comparisons.

\begin{figure}[h]
\centering
\includegraphics[trim = 7pt 3pt 30pt 0pt, clip, width=0.49\linewidth]{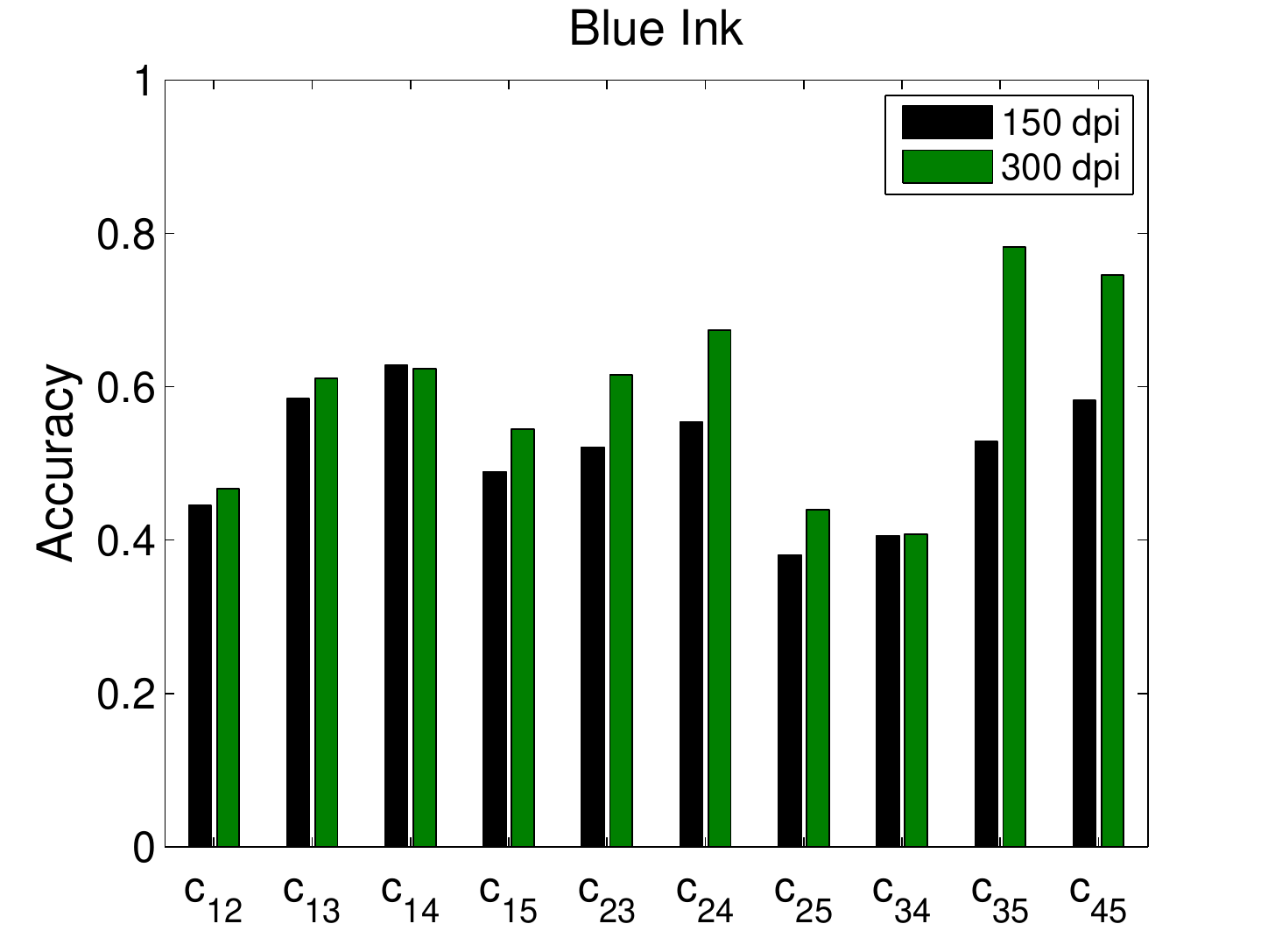}
\includegraphics[trim = 7pt 3pt 30pt 0pt, clip, width=0.49\linewidth]{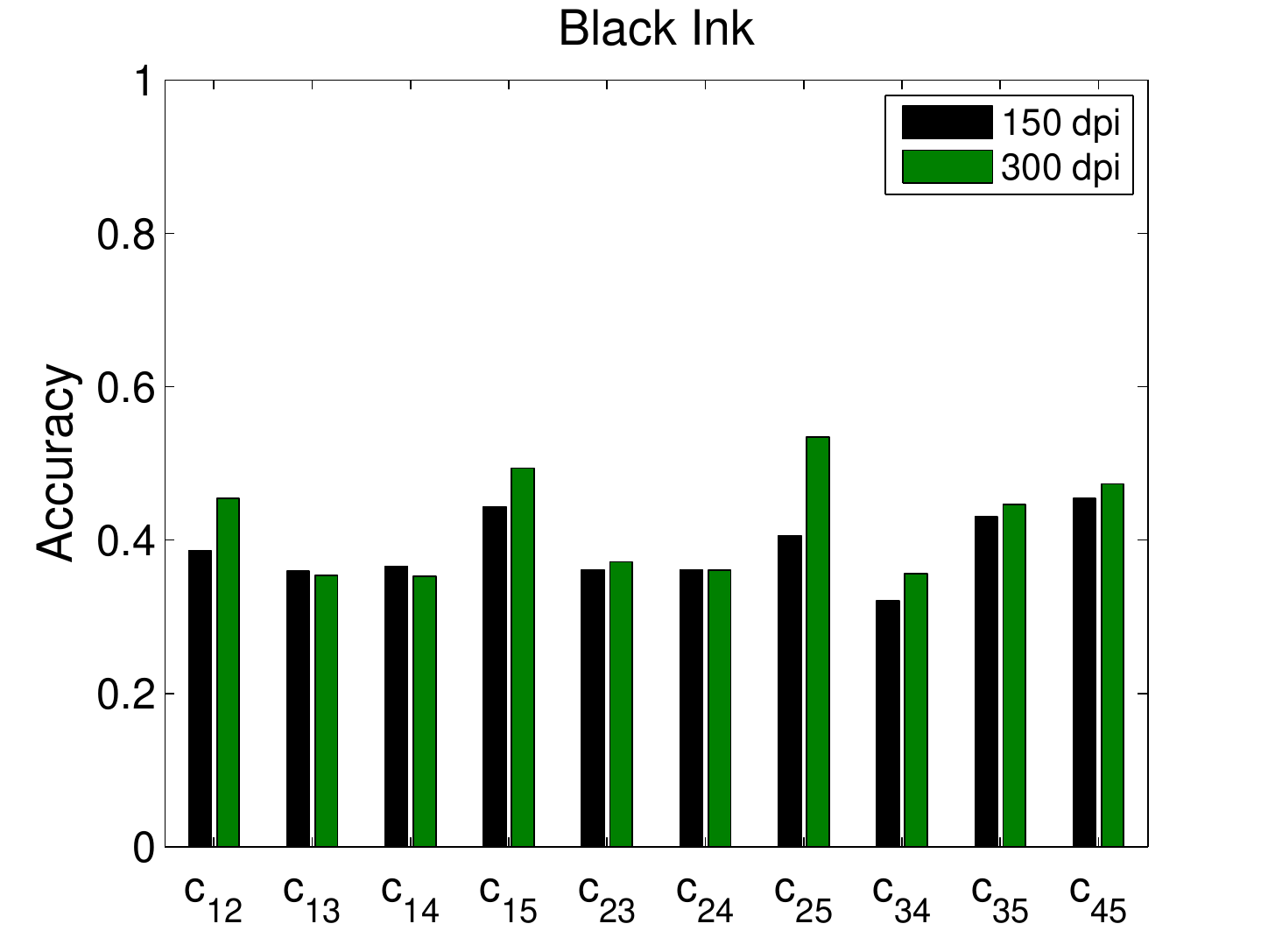}
\caption[Effect of spatial resolution on ink mismatch detection from RGB images]{The effect of spatial resolution on ink mismatch detection from RGB images.}
\label{fig:DPI}
\end{figure}

We now analyze how hyperspectral images (HSI) can be beneficial in ink mismatch detection compared to RGB images. Initially, we use all the bands of hyperspectral image for the analysis. Later, we show how feature selection improves the task of ink mismatch detection. The mismatch detection accuracies of HSI and RGB data are compared in Figure~\ref{fig:RHH}. It can be seen that HSI significantly outperforms RGB in separating most blue ink combinations. This is evident in accurate clustering for ink combinations c$_{_{12}}$, c$_{_{14}}$, c$_{_{25}}$, c$_{_{35}}$ and c$_{_{45}}$ of the blue inks. In the case of black inks, ink 1 is most distinguishable from the other inks resulting in highly accurate mismatch detection for all of its combinations c$_{_{1j}}$ using HSI. As seen previously, RGB images are insufficient for black inks mismatch detection. However, it can be seen that for a few ink combinations, even HSI seems insufficient for discrimination. These results invite further exploration of HSI to find the most informative bands.

\begin{figure}[h]
\centering
\includegraphics[trim = 7pt 3pt 30pt 0pt, clip, width=0.49\linewidth]{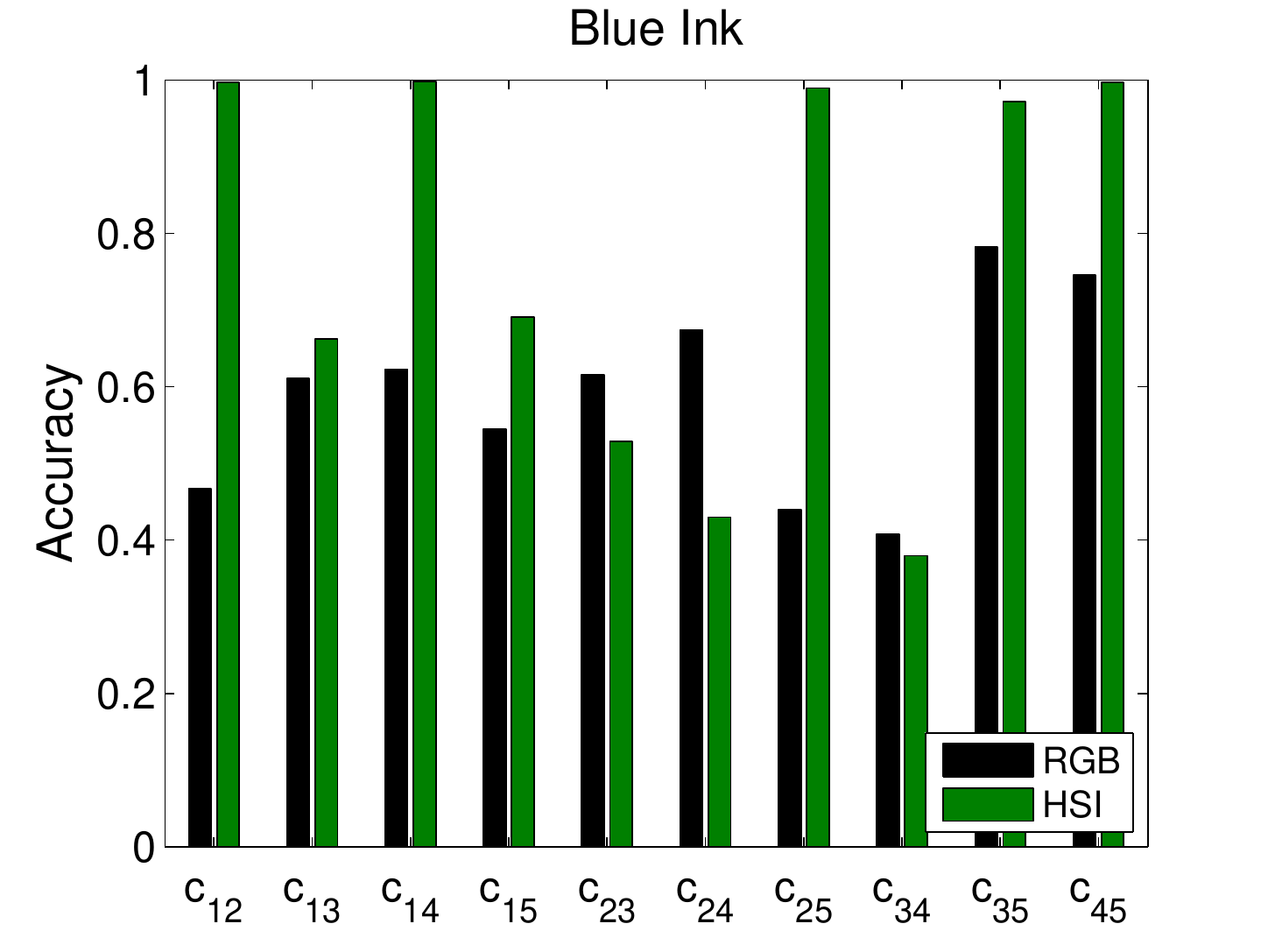}
\includegraphics[trim = 7pt 3pt 30pt 0pt, clip, width=0.49\linewidth]{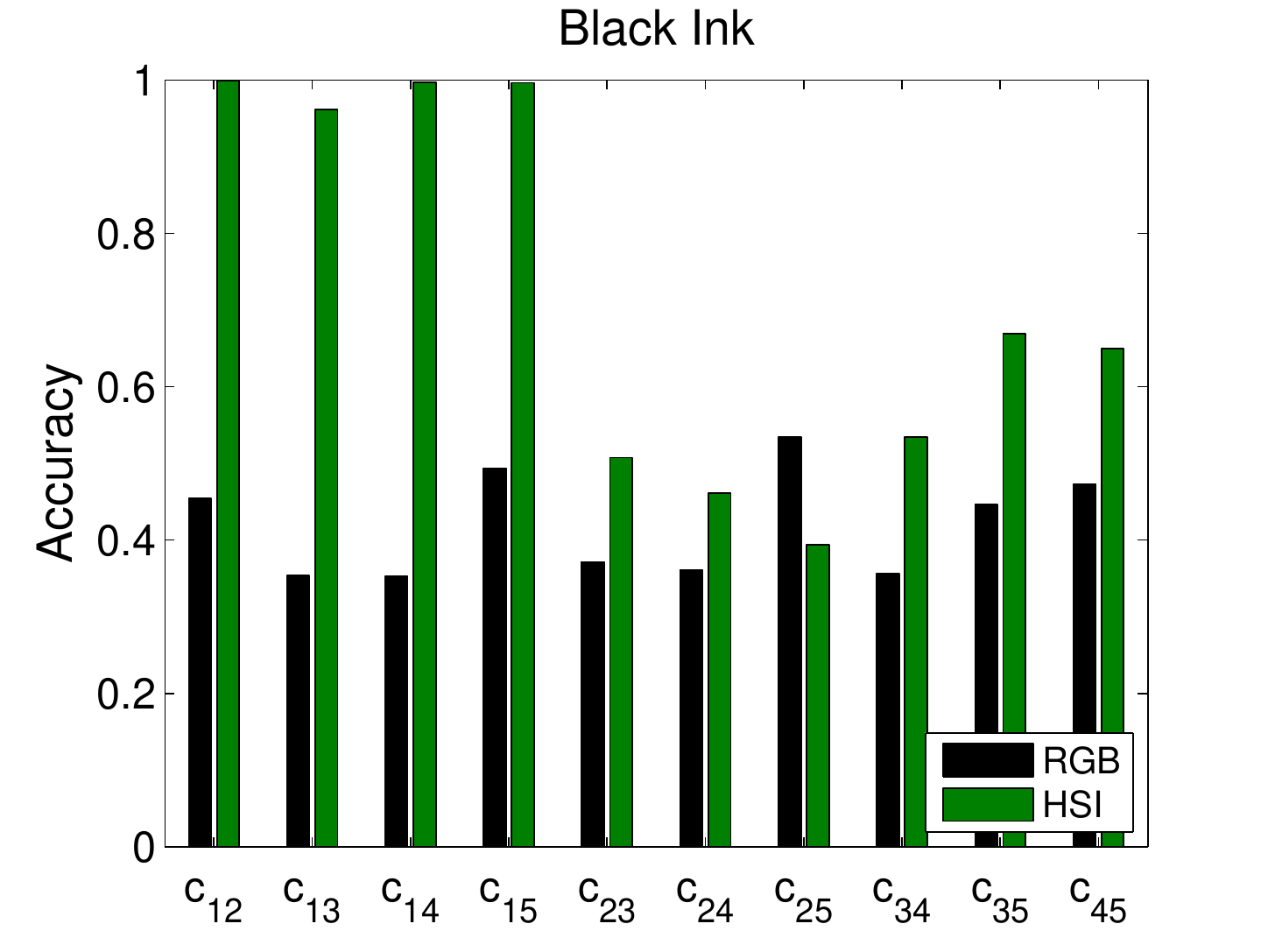}
\caption{Comparison of RGB and HSI image based mismatch detection accuracy.}
\label{fig:RHH}
\end{figure}

Figure~\ref{fig:prop} shows the effect of varying ink proportions on mismatch detection accuracy. Generally, the sentences are composed of 9 words in the database. The ink proportion is varied by mixing words of different inks in different proportion. It is observed that for the lowest proportion, i.e.~a single word in a different ink is least distinguishable. This can be noticed by the drop in accuracy for ink proportions 1:8 and 8:1 in most ink combinations. For highly distinguishable ink-pairs, the trend of accuracy in relation to ink proportion is generally symmetric. Other, less distinguishable ink combinations display a skewed trend. It can be interpreted in the following manner. For a given ink, a minimum quantity of samples are required to distinguish it from another specific ink. For example, given the blue ink combination c$_{_{35}}$, one word of ink 3 is highly distinguishable in a note mainly written with ink 5. Conversely, one word of ink 5 is indistinguishable in a note written in ink 3. This brings forward another research direction of disproportional ink mismatch detection which is practically possible in real life forensic analysis. This is inherently an unbalanced clustering problem, and hence more sophisticated clustering algorithms would be needed to resolve it. In this work, we restrict to equal ink proportions for the rest of the experiments.

\begin{figure}[h]
\centering
\includegraphics[trim = 0pt 40pt 30pt 0pt, clip, width=0.49\linewidth]{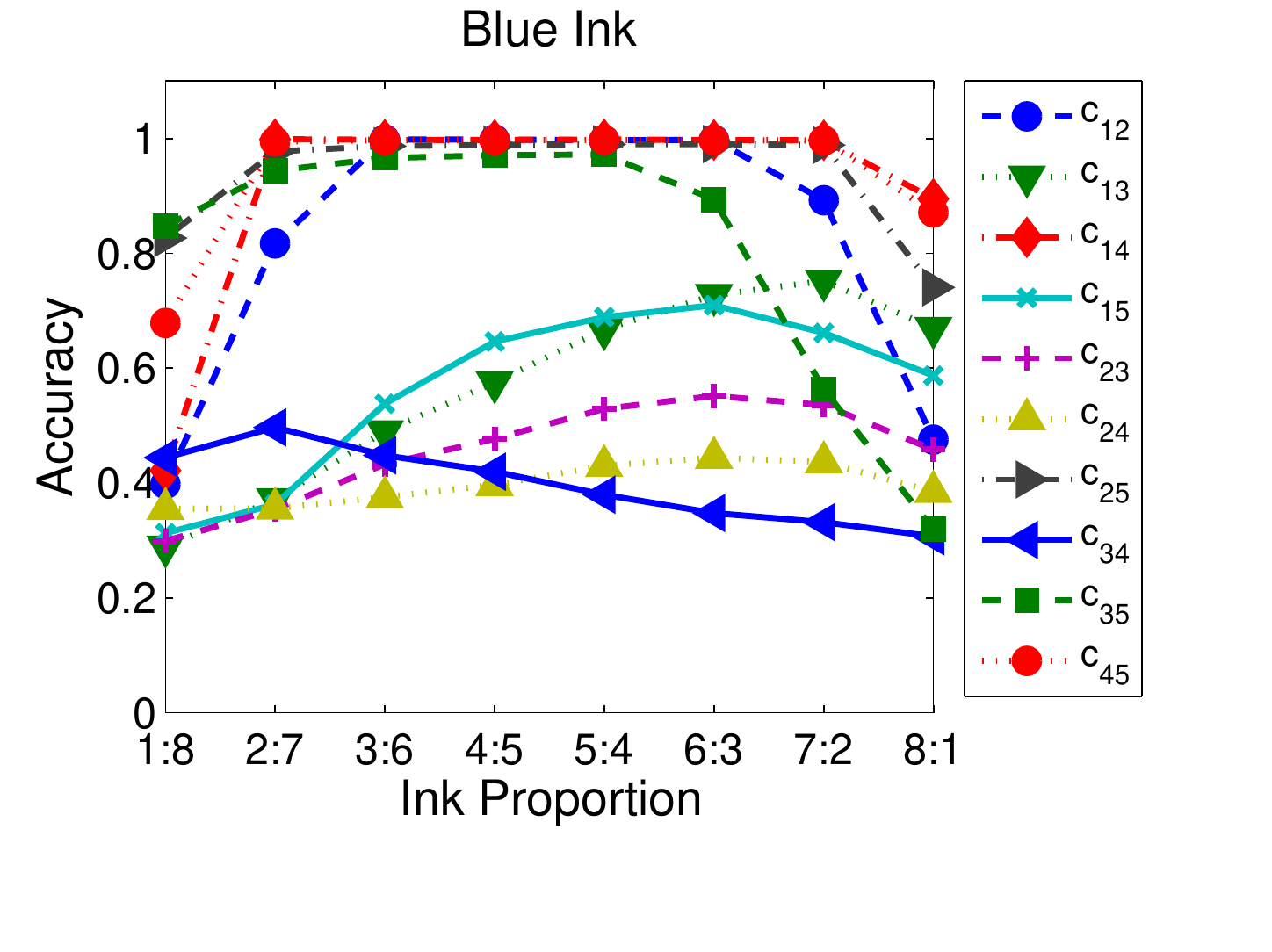}
\includegraphics[trim = 0pt 40pt 30pt 0pt, clip, width=0.49\linewidth]{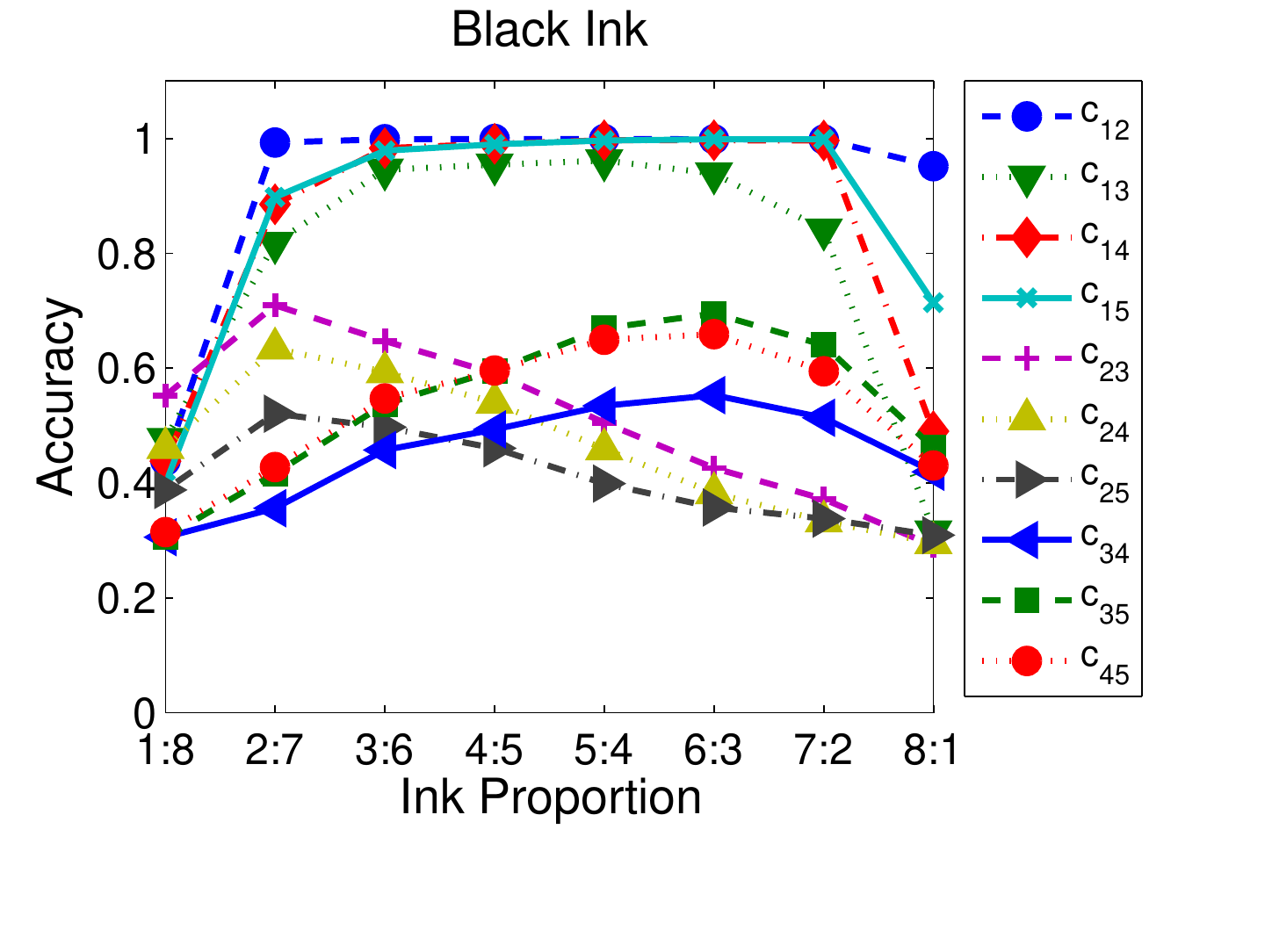}
\caption{The effect of varying ink proportion on mismatch detection accuracy.}
\label{fig:prop}
\end{figure}



To study whether different inks become more distinguishable in different regions of the spectrum, we plot the average normalized spectra of all blue and black inks, in Figure~\ref{fig:spec}. These graphs are the outcome of average spectral response of each ink over all samples in the database. Observe that the ink spectra are more similar in some ranges than in other ranges. It is likely that these inks are better distinguished in different bands in the visible spectrum. In order to evaluate the contribution of sub-ranges to ink discrimination, we divide the hyperspectral data and perform separate experiments in each sub-range. We divide the visible spectrum into three empirical ranges, named as low-visible (400nm-500nm), mid-visible (510nm-590nm) and high-visible range (600nm-720nm). These ranges roughly correspond to the blue, green and red colors and have been empirically selected because no clear sub-categorization of the visible spectrum, is defined in the literature. A close analysis of variability of the ink spectra in these ranges reveals that most of the differences are present in the high-visible range, followed by mid-visible and low-visible ranges.

\begin{figure}[t]
\centering
\includegraphics[trim = 0pt 0pt 30pt 0pt, clip, width=0.49\linewidth]{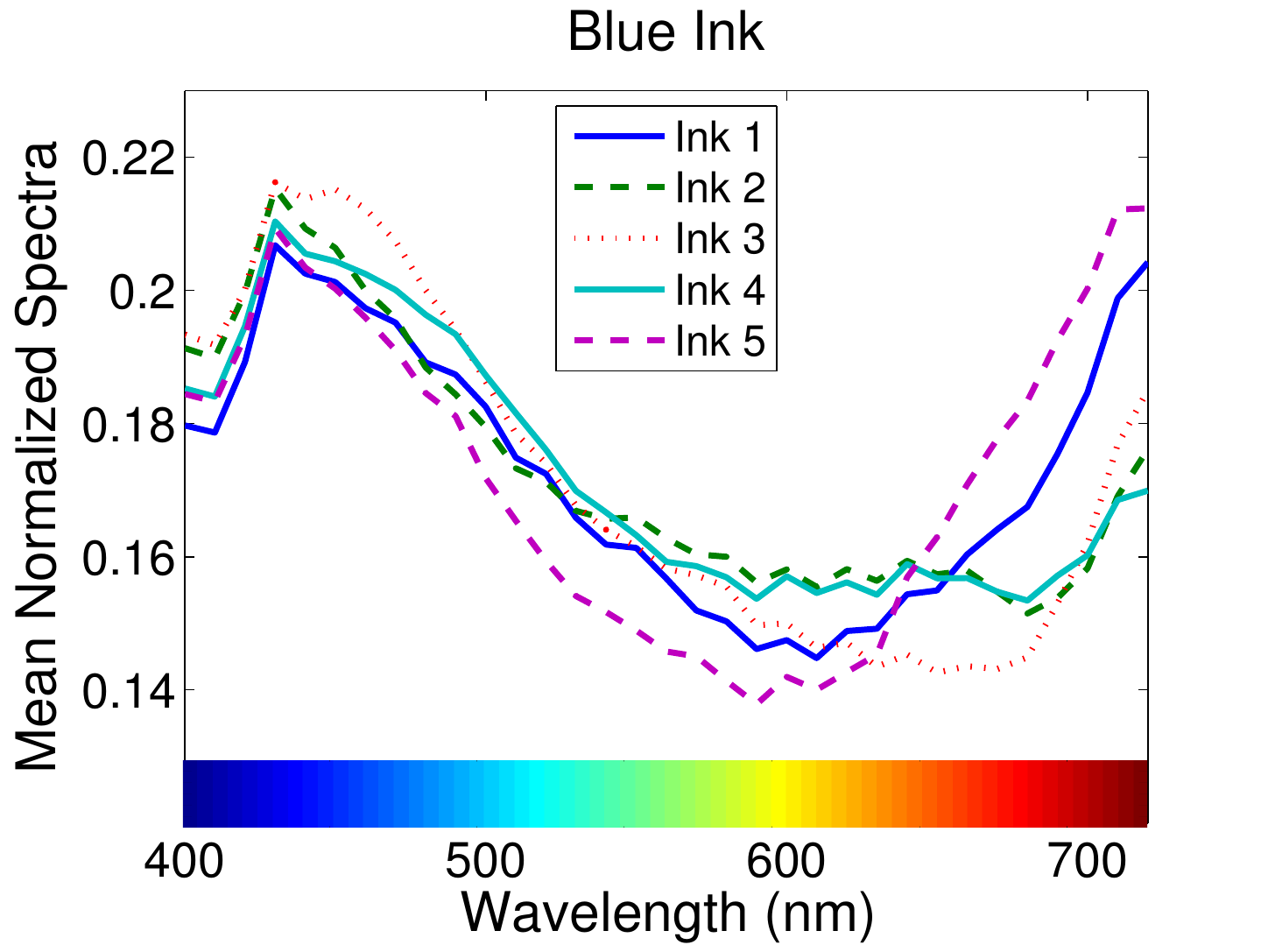}
\includegraphics[trim = 0pt 0pt 30pt 0pt, clip, width=0.49\linewidth]{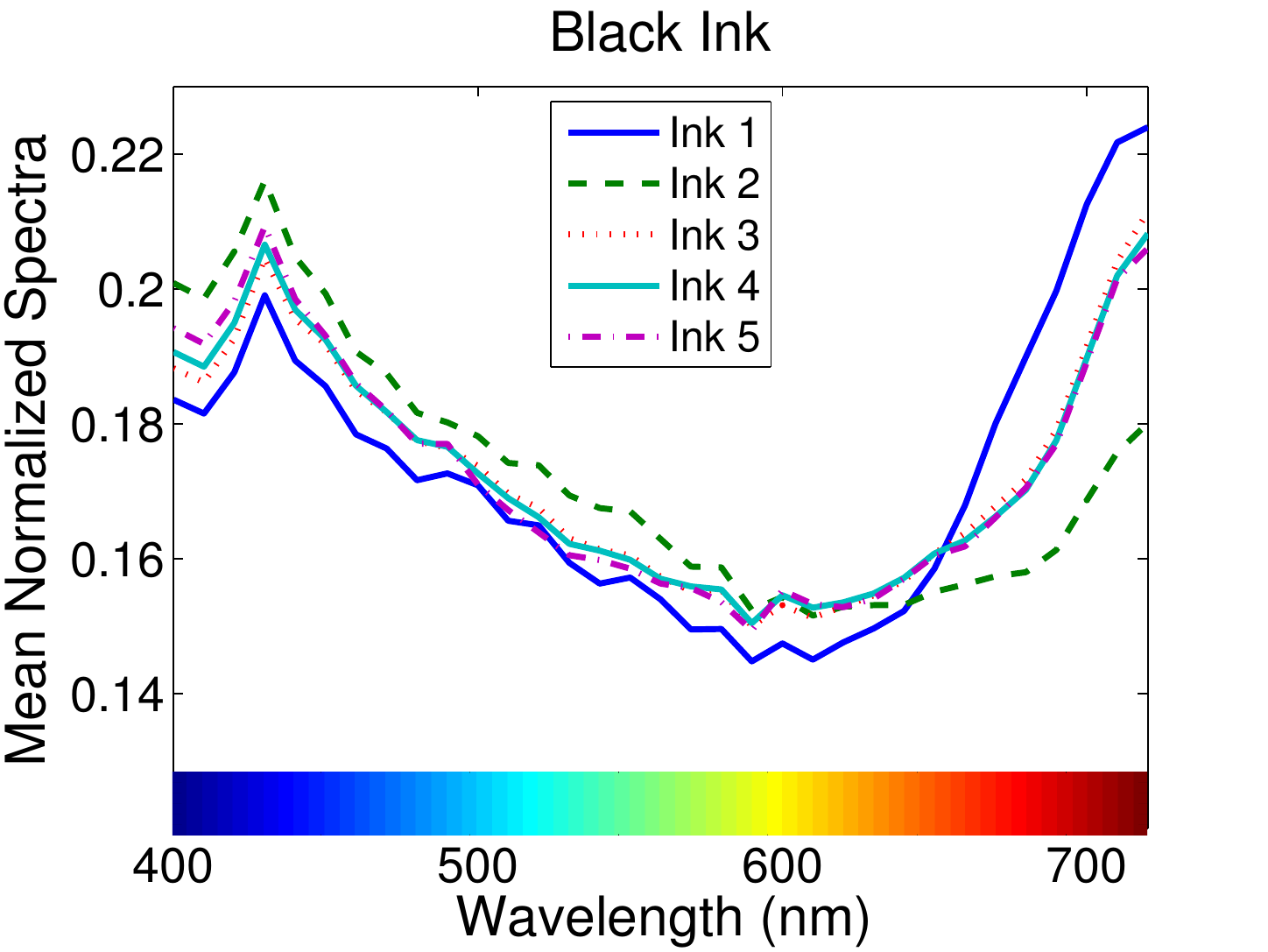}
\caption[Spectra of the blue and black inks under analysis]{Spectra of the blue and black inks under analysis. Note that at some ranges the ink spectra are more distinguished than others.}
\label{fig:spec}
\end{figure}

Figure~\ref{fig:BGR} shows the results of separate experiments in low-visible, mid-visible and high-visible range. Note that for most of the ink combinations, the high-visible range is the most accurate, followed by the mid-visible and the low-visible range respectively. Observe that the black ink combinations c$_{_{34}}$, c$_{_{35}}$ and c$_{_{45}}$ are more distinguished in the low-visible range. This trend can be related back to Figure~\ref{fig:spec} wherein the black inks 3, 4 and 5 are more similar in the high-visible range and dissimilar in the low-visible range.

\begin{figure}[h]
\centering
\includegraphics[trim = 7pt 3pt 30pt 0pt, clip, width=0.49\linewidth]{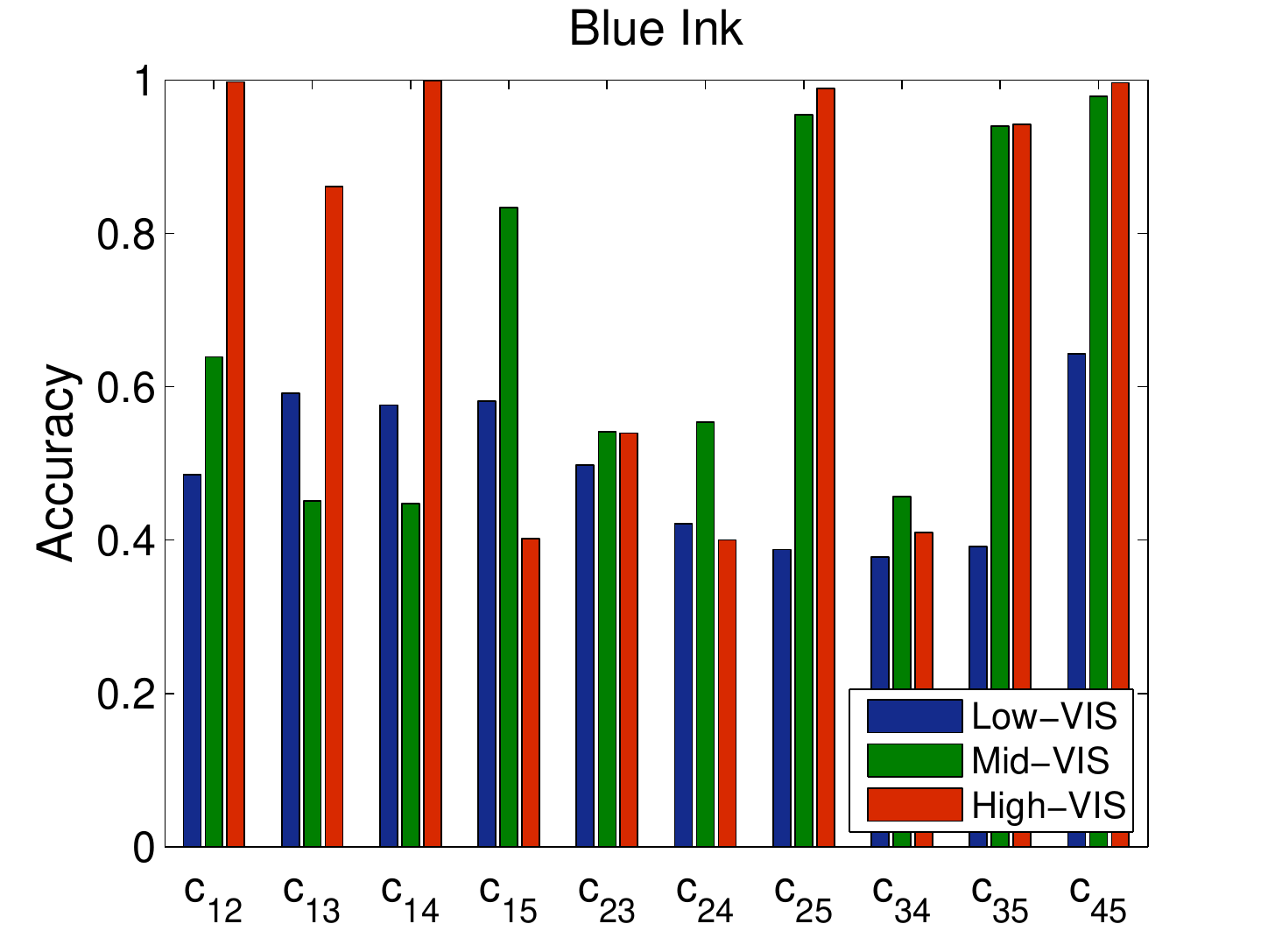}
\includegraphics[trim = 7pt 3pt 30pt 0pt, clip, width=0.49\linewidth]{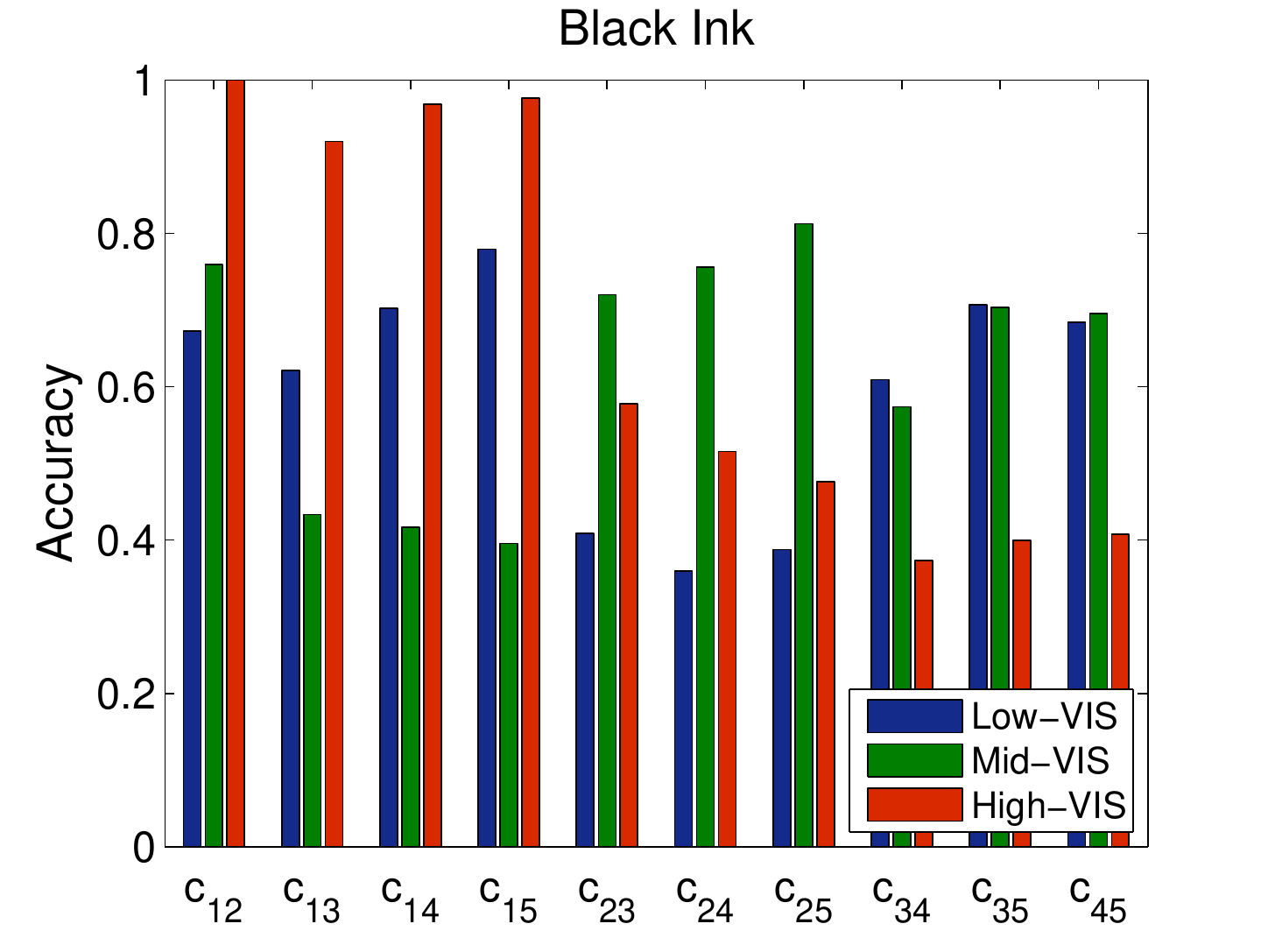}
\caption[HSI wavelength range analysis]{HSI wavelength range analysis. Observe that the high-VIS range performs better than the mid-VIS and low-VIS ranges.}
\label{fig:BGR}
\end{figure}

After ink mismatch detection using all bands of the HSI, we now extend our approach by band selection. In a 10 fold cross validation experiment, a leave two inks out strategy is adopted to avoid bias of the selected features towards particular inks. In each fold, bands are selected from three ink combinations and tested on one ink combination while leaving the remaining six ink combinations that contain either of the two test inks. For example, if the test ink combination is c$_{_{12}}$, all ink combinations containing either ink-1 or ink-2 (i.e. c$_{_{13}}$,c$_{_{14}}$,c$_{_{15}}$,c$_{_{23}}$,c$_{_{24}}$,c$_{_{25}}$) are left out and the bands are selected from (c$_{_{34}}$,c$_{_{35}}$,c$_{_{45}}$). The same protocol is adopted for all test ink combinations.


We now analyze ink mismatch detection using Joint Sparse Band Selection (JSBS) and Sequential Forward Band Selection (SFBS). Table~\ref{tab:res-SFBS-JSBS-blue} gives the bands selected from training data by each technique and their corresponding accuracies on test data for blue inks. The bands selected by JSBS result in higher accuracy except for ink combination c$_{_{15}}$ and c$_{_{24}}$. We can further dissect this result by relating the selected bands to the spectra of ink-1 and ink-5 in Figure~\ref{fig:spec}. It is evident that c$_{_{15}}$ is more differentiable in mid/high visible range. The bands selected by SFBS belong to all ranges whereas JSBS selects bands solely from the high visible range which is the reason for its lower accuracy on c$_{_{15}}$. The spectra of ink-2 and ink-4 is highly similar in the mid/high visible range but slightly dissimilar in the low visible range from which neither technique selected a band. Overall, the average accuracy of JSBS (86.2\%) is better than SFBS (82.1\%) in blue inks mismatch detection.

\begin{table}[h]
\caption[Selected bands and accuracies in blue ink handwritten notes]{Selected bands and mismatch detection accuracies in blue ink handwritten notes.}
\label{tab:res-SFBS-JSBS-blue}
\centering
{\scriptsize
\begin{tabular}{|c|lc|lc|} \hline
\multirow{2}{*}{Fold}                   & \multicolumn{2}{c|}{SFBS} &       \multicolumn{2}{c|}{JSBS}   \\ \cline{2-5}
            & Selected Bands            &       Acc.(\%)   &  Selected Bands           &  Acc. (\%)     \\ \hline
c$_{_{12}}$ & $720,530,460,560        $ & $        99.5 $  &  $610,620,700,710 $       & $\mathbf{99.9}$\\ \hline 
c$_{_{13}}$ & $720,560,490,550,510    $ & $        98.4 $  &  $610,710,720     $       & $\mathbf{99.6}$\\ \hline 
c$_{_{14}}$ & $720,490,550            $ & $\mathbf{99.9}$  &  $620,710         $       & $\mathbf{99.9}$\\ \hline 
c$_{_{15}}$ & $720,490,550,560,510    $ & $\mathbf{81.4}$  &  $660,710         $       & $        72.2 $\\ \hline 
c$_{_{23}}$ & $690,520,630,500,620,470$ & $        50.9 $  &  $720             $       & $\mathbf{56.4}$\\ \hline 
c$_{_{24}}$ & $700,520,500,720        $ & $\mathbf{56.1}$  &  $590,710         $       & $        43.5 $\\ \hline 
c$_{_{25}}$ & $720,550,630,560        $ & $        98.6 $  &  $640,710         $       & $\mathbf{99.1}$\\ \hline 
c$_{_{34}}$ & $690,520,620,500        $ & $        41.2 $  &  $550,610,720     $       & $\mathbf{92.8}$\\ \hline 
c$_{_{35}}$ & $720,490,560,550,510    $ & $        96.1 $  &  $410,620,700     $       & $\mathbf{99.1}$\\ \hline 
c$_{_{45}}$ & $720,490,550            $ & $        99.7 $  &  $400,660,720     $       & $\mathbf{99.8}$\\ \hline 
\multicolumn{2}{|r|}{mean}              & $        82.1 $  & \multicolumn{1}{r|}{mean} & $\mathbf{86.2}$\\ \hline
\end{tabular}}
\end{table}
Table~\ref{tab:res-SFBS-JSBS-black} gives the selected bands and accuracies for black inks. The JSBS consistently outperforms SFBS as it is more accurate for all ink combinations. The lower accuracy of both techniques on combinations arising from ink 3, 4 and 5 is imminent from their highly similar spectra in Figure~\ref{fig:spec}. Interestingly a single band proves sufficient to differentiate inks 3, 4 and 5 from ink 2. Altogether, the average accuracy of JSBS (88.1\%) is much better than SFBS (79.8\%) in black inks mismatch detection.

\begin{table}[h]
\caption[Selected bands and accuracies in black ink handwritten notes]{Selected bands and mismatch detection accuracies in black ink handwritten notes.}
\label{tab:res-SFBS-JSBS-black}
\centering
{\scriptsize
\begin{tabular}{|c|lc|lc|} \hline
\multirow{2}{*}{Fold}   & \multicolumn{2}{c|}{SFBS}              &      \multicolumn{2}{c|}{JSBS}              \\ \cline{2-5}
            & Selected Bands                            & Acc.(\%) & Selected Bands          & Acc.(\%)    \\ \hline
c$_{_{12}}$ & $520,700,530,720,450,500        $ & $\mathbf{100}$ & $710                    $ & $\mathbf{100}$   \\ \hline 
c$_{_{13}}$ & $530,720,460                            $ & $84.1$ & $410,650,680,700        $ & $\mathbf{98.7}$  \\ \hline 
c$_{_{14}}$ & $720,440,520,710,560                    $ & $95.3$ & $400,410,680,700,710    $ & $\mathbf{99.5}$  \\ \hline 
c$_{_{15}}$ & $720,460,520,530,680                    $ & $99.1$ & $420,520,640,650,680,710$ & $\mathbf{99.7}$  \\ \hline 
c$_{_{23}}$ & $700,520,530,460,720,500,470            $ & $81.4$ & $720                    $ & $\mathbf{90.6}$  \\ \hline 
c$_{_{24}}$ & $700,520,500,460,610,530,680,720,440,510$ & $74.3$ & $720                    $ & $\mathbf{87.4}$  \\ \hline 
c$_{_{25}}$ & $700,440,520,500,710,460,550            $ & $79.2$ & $720                    $ & $\mathbf{84.1}$  \\ \hline 
c$_{_{34}}$ & $710,440,550,560                        $ & $58.1$ & $520,700                $ & $\mathbf{68.2}$  \\ \hline 
c$_{_{35}}$ & $710,440,560,550,430,570,720,420        $ & $66.2$ & $430,710,720            $ & $\mathbf{83.4}$  \\ \hline 
c$_{_{45}}$ & $710,440,560,450                        $ & $60.3$ & $420,690,700,710,720    $ & $\mathbf{70.1}$  \\ \hline 
\multicolumn{2}{|r|}{mean}                              & $79.8$ & \multicolumn{1}{r|}{mean} & $\mathbf{88.1}$  \\ \hline
\end{tabular}}
\end{table}
In summary, the SFBS roughly selects bands from each of the low/mid/high visible ranges. The JSBS selects bands solely from high visible range or a combination of low/high or mid/high visible ranges (with few exceptions). This means that JSBS selects more informative bands from complementary ranges resulting in higher accuracy. Moreover, the JSBS consistently selects fewer (or equal) number of bands compared to SFBS for all ink combinations, despite being more accurate, as shown in Figure~\ref{fig:FFS-JSFS}. Thus on average, JSBS selects half as many bands as SFBS while resulting in higher average accuracy. Our band selection gives an insight into how a customized multispectral imaging device with a smaller number of bands can be designed for ink mismatch detection. Bianco et al.~\cite{bianco2012multispectral} developed one such multispectral imaging device by combining an imaging sensor with a mechanical filter wheel. They empirically selected six different filters for the prototype device. Such devices may hugely benefit from the findings of the proposed band selection study in the selection of an optimal combination of filters.

\begin{figure}[h]
\centering
\includegraphics[trim = 7pt 3pt 30pt 0pt, clip, width=0.49\linewidth]{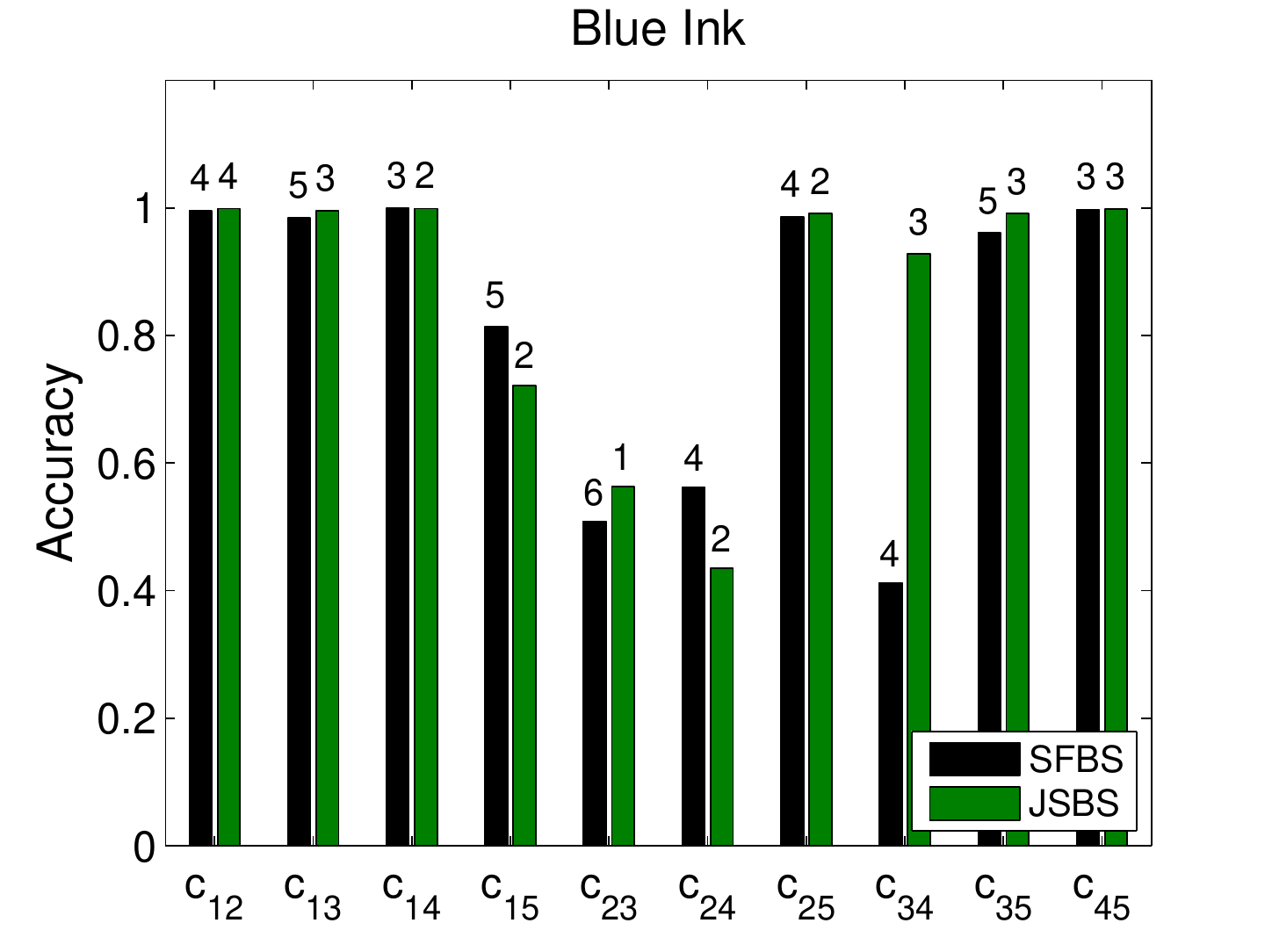}
\includegraphics[trim = 7pt 3pt 30pt 0pt, clip, width=0.49\linewidth]{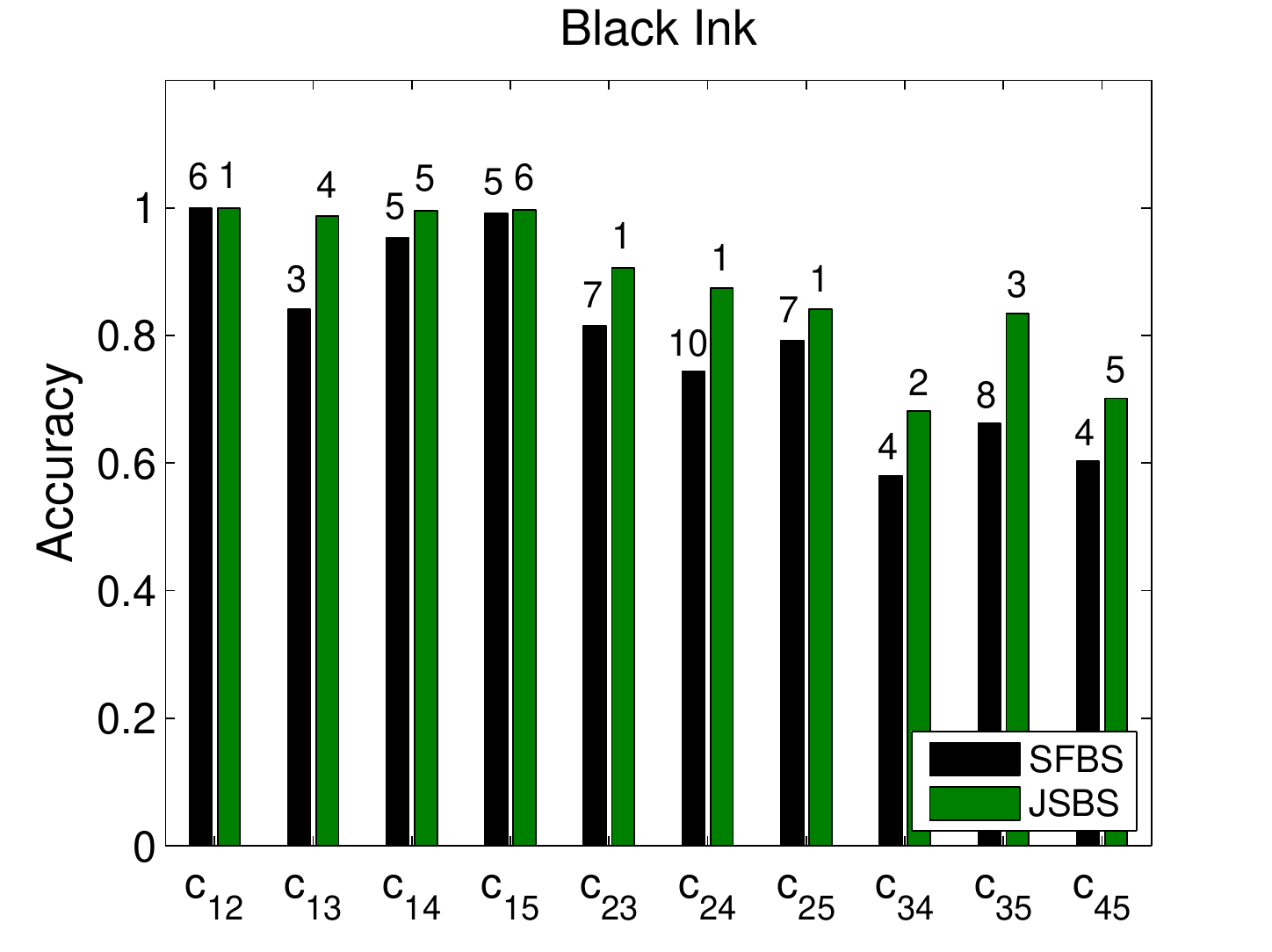}
\caption[Comparison of SFBS and JSBS techniques for ink mismatch detection]{Comparison of SFBS and JSBS techniques for ink mismatch detection. The integers over each bar indicate the selected number of bands. Observe in most cases, JSBS selects less number of bands despite providing better accuracy.}
\label{fig:FFS-JSFS}
\end{figure}

Finally, we qualitatively analyze ink mismatch detection results on example images of blue and black ink combinations. In Figure~\ref{fig:qual-fs}, the original images shown are a combination of blue inks (c$_{_{34}}$) and black inks (c$_{_{45}}$), respectively. We also show original RGB images for better visual analysis. The clustering based on RGB images is unable to group similar ink pixels into the same clusters. Instead, a closer look reveals that all the ink pixels are falsely grouped into one cluster whereas most of the boundary pixels are grouped into another cluster. This means that RGB is not sufficient to discriminate inks in these examples. Mismatch detection based on HSI using all bands also struggles in separating the inks. The result of SFBS is slightly different from HSI-All, indicating that the selected bands are ineffective.

Finally, we see how the accuracy is improved by using only the bands selected by JSBS. The selected bands exhibit a clear advantage over using all the bands or bands selected by SFBS. It can be seen that the majority of the ink pixels are correctly grouped according to ground truth. Mismatch detection of black inks is still a more difficult task compared to blue inks, but much improved in comparison to SFBS. One way to further improve the few mis-classified pixels is to further classify on a word-by-word basis. Recall that currently the spectral responses of inks are separated on a pixel-by-pixel basis, i.e.~without taking the spatial context into account. A word-by-word classification would require learning the possible spatial patterns of forgeries, followed by classification of each word/character as authentic or forged.

\section{Conclusion}
\label{sec:conclusion}

Hyperspectral document imaging has immense potential for forensic document examination. We demonstrated the benefit of hyperspectral imaging in automated ink mismatch detection. The non-informative bands were reduced by the proposed joint sparse band selection technique based on joint sparse PCA. Accurate ink mismatch detection was achieved using joint sparse band selection compared to using all features or using a subset of features selected by sequential forward band selection. We hope that the promising results presented in this work along with the exciting new challenges would trigger more research efforts in the direction of automated hyperspectral document analysis. Our newly developed writing ink hyperspectral image database is publicly available for research.

\begin{landscape}
\begin{figure}
\centering
\begin{tabular}{cc}
Blue Ink Handwritten Note & Black Ink Handwritten Note \\
\fbox{\includegraphics[trim = 0pt 0pt 0pt 20pt, clip, width=0.47\linewidth]{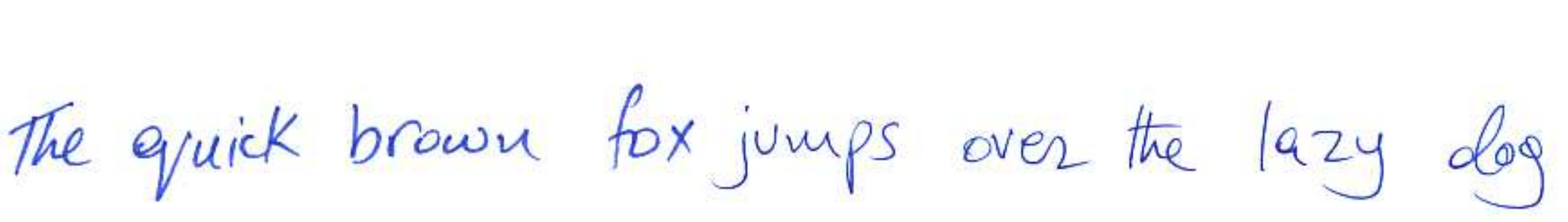}} &
\fbox{\includegraphics[trim = 0pt 0pt 0pt 15pt, clip, width=0.47\linewidth]{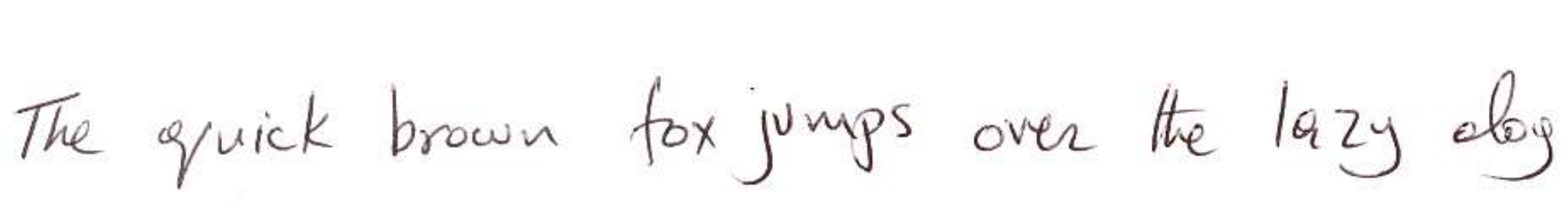}} \\ [2pt]
True Ink Map & True Ink Map \\
\fbox{\includegraphics[trim = 0pt 0pt 0pt 20pt, clip, width=0.47\linewidth]{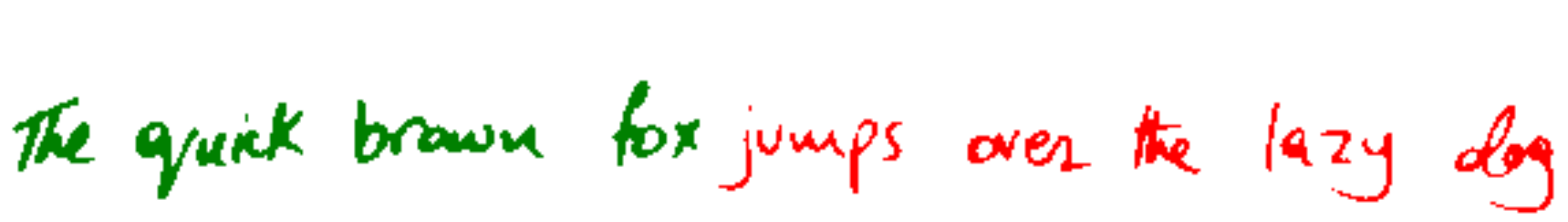}} &
\fbox{\includegraphics[trim = 0pt 0pt 0pt 15pt, clip, width=0.47\linewidth]{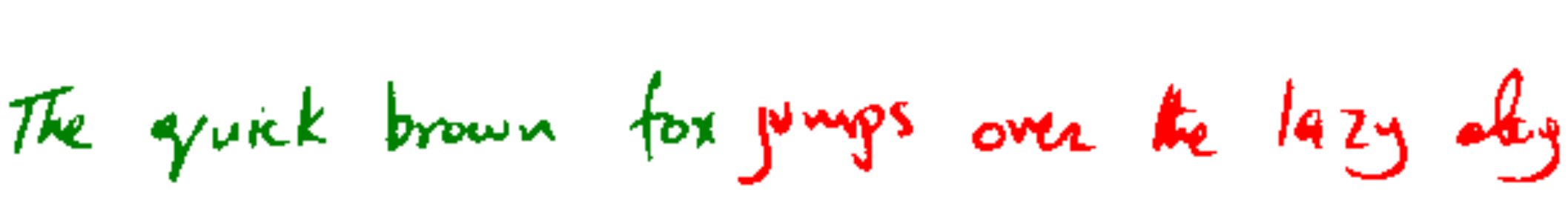}} \\ [2pt]
Result (RGB)        & Result (RGB) \\
\fbox{\includegraphics[trim = 0pt 0pt 0pt 20pt, clip, width=0.47\linewidth]{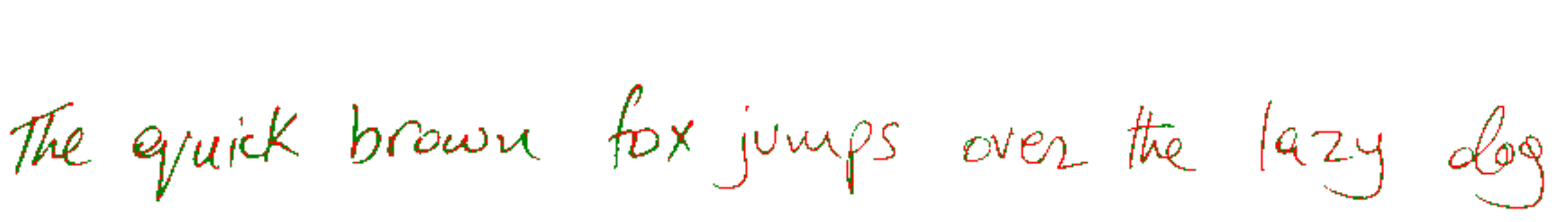}} &
\fbox{\includegraphics[trim = 0pt 0pt 0pt 15pt, clip, width=0.47\linewidth]{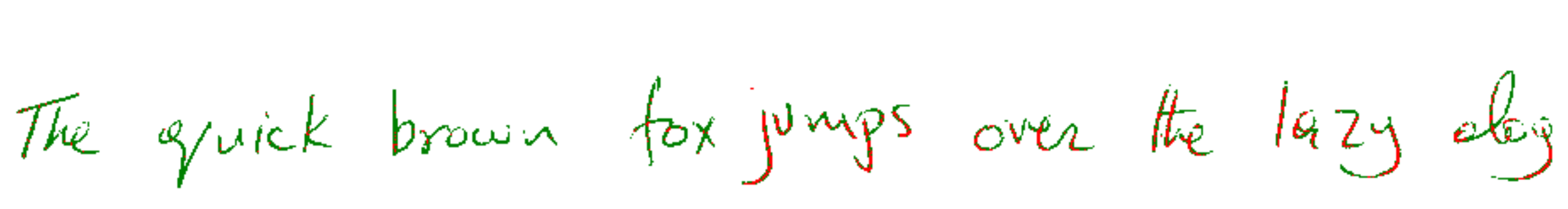}} \\ [2pt]
Result (HSI-All)    & Result (HSI-All) \\
\fbox{\includegraphics[trim = 0pt 0pt 0pt 20pt, clip, width=0.47\linewidth]{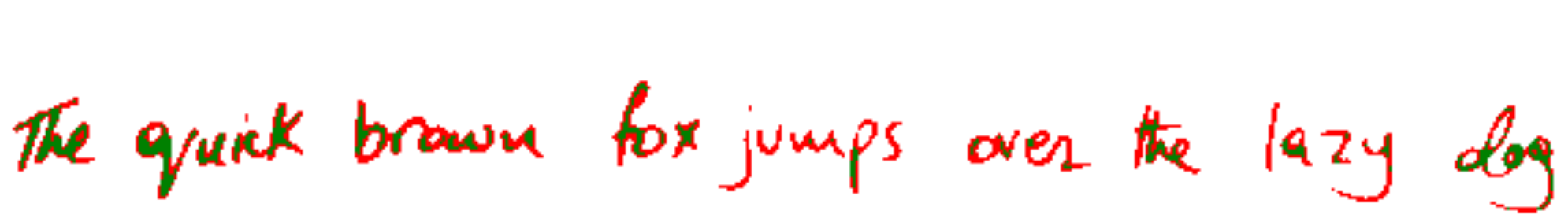}} &
\fbox{\includegraphics[trim = 0pt 0pt 0pt 15pt, clip, width=0.47\linewidth]{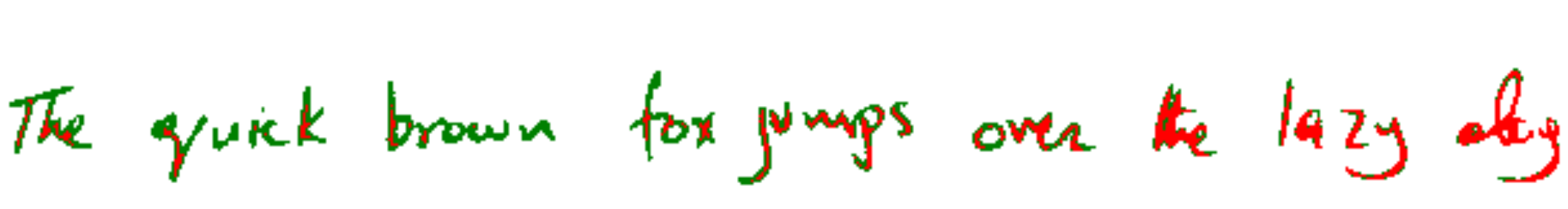}} \\ [2pt]
Result (HSI-SFBS)   & Result (HSI-SFBS) \\
\fbox{\includegraphics[trim = 0pt 0pt 0pt 20pt, clip, width=0.47\linewidth]{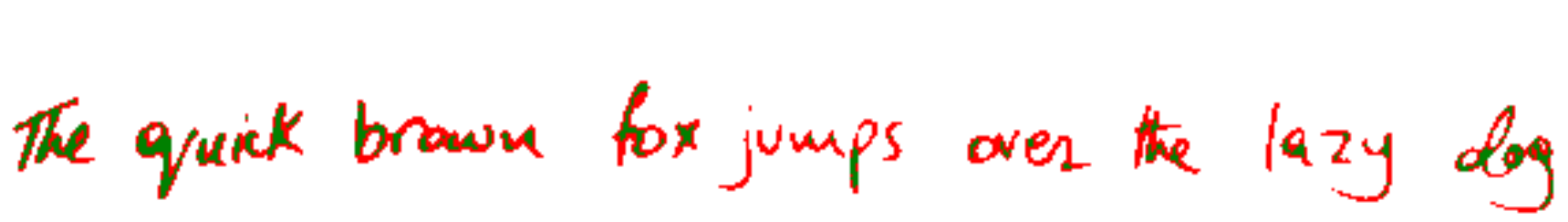}} &
\fbox{\includegraphics[trim = 0pt 0pt 0pt 15pt, clip, width=0.47\linewidth]{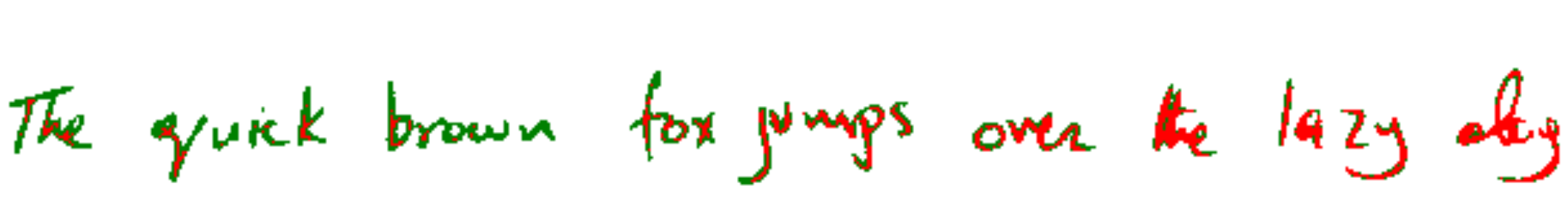}} \\ [2pt]
Result (HSI-JSBS)   & Result (HSI-JSBS) \\
\fbox{\includegraphics[trim = 0pt 0pt 0pt 20pt, clip, width=0.47\linewidth]{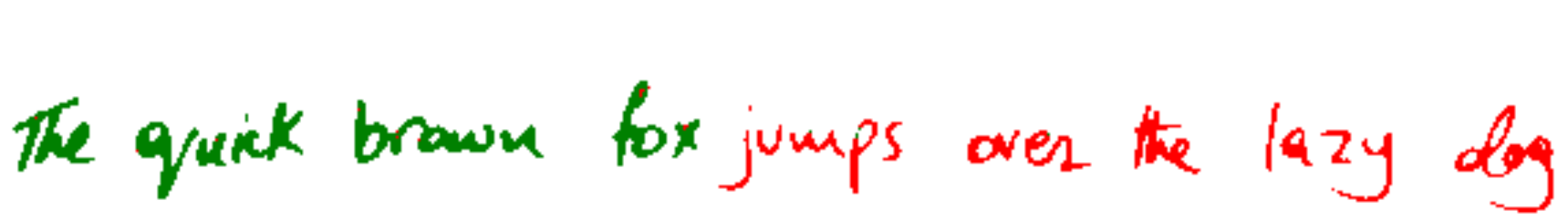}} &
\fbox{\includegraphics[trim = 0pt 0pt 0pt 15pt, clip, width=0.47\linewidth]{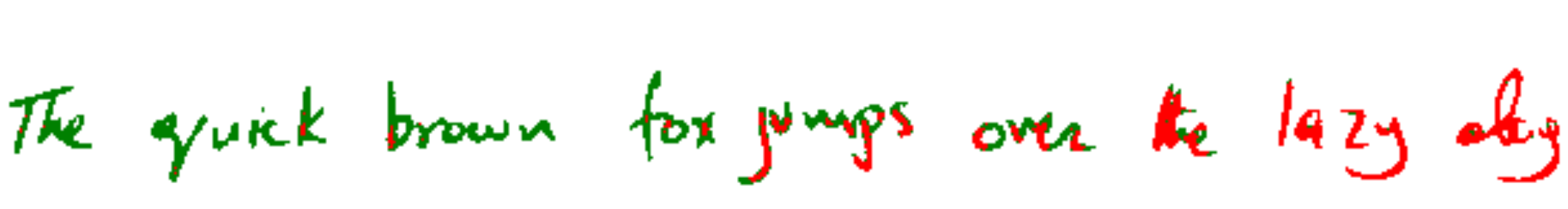}}
\end{tabular}
\caption[Effect of band selection on ink mismatch detection]{Example test images. We purposefully selected two hard cases so that the capability of RGB and HSI based ink mismatch detection is visually appreciable.}
\label{fig:qual-fs}
\end{figure}
\end{landscape}


\chapter{Hyperspectral Palmprint Recognition} 

\label{Chapter7} 


The information present in a human palm has an immense amount of potential for biometric recognition. Information visible to the naked eye includes the principal lines, the wrinkles and the fine ridges which form a unique pattern for every individual~\cite{zhang2012comparative}. These superficial features can be captured using standard imaging devices. High resolution scanners capture the fine ridge pattern of a palm which is generally employed for latent palmprint identification in forensics~\cite{jain2009latent}. The principal lines and wrinkles acquired with low resolution sensors are suitable for security applications like user identification or authentication~\cite{zhang2003online}.

Additional information present in the human palm is the subsurface vein pattern which is indifferent to the palm lines. The availability of such complementary features (palm lines and veins) allows for increased discrimination between individuals. Such features cannot be easily acquired by a standard imaging sensor. Infrared imaging can capture subsurface features due to its capability to penetrate the human skin. 



Figure~\ref{fig:palms} shows a hyperspectral image of a palm as a series of bands at different wavelengths. Observe the variability of information in a palm from shorter to longer wavelengths. The information in nearby bands is highly redundant and the features vary gradually across the spectrum. Due to the large number of bands in a hyperspectral palmprint image, band selection is inevitable for efficient palmprint recognition. The selected bands must be informative for representation to facilitate development of a realtime multispectral palmprint recognition system with a few bands. Moreover, extraction of line-like features from representative bands of a hyperspectral palmprint requires a robust feature extraction scheme which is a challenging task given the multi-modal nature of palm.

Another desirable characteristic of a futuristic biometric system is its non-invasive nature. Contact devices, restrict the hand movement but raise user acceptability issues due to hygiene. On the other hand, biometrics that are acquired with non-contact sensors are user friendly and socially more acceptable~\cite{jain2004introduction}, but introduce challenge of rotation, scale and translation (RST) variation. The misalignments caused by movement of the palm can be recovered by reliable, repeatable landmark detection and subsequent ROI extraction.

\clearpage
In this chapter, an end-to-end framework for hyperspectral palmprint recognition is proposed~\cite{khan2011contour}. The key contributions of this work are
\begin{itemize}
  \item A reliable technique for ROI extraction from non-contact palmprint images.
  \item A robust multidirectional feature encoding for multispectral palmprint representation.
  \item An efficient hash table scheme for compact storage and matching of multi-directional features.
\end{itemize}

\begin{figure}[t]
\renewcommand{\thesubfigure}{\relax}
\subfigure[510nm]{\includegraphics[width=0.09\linewidth]{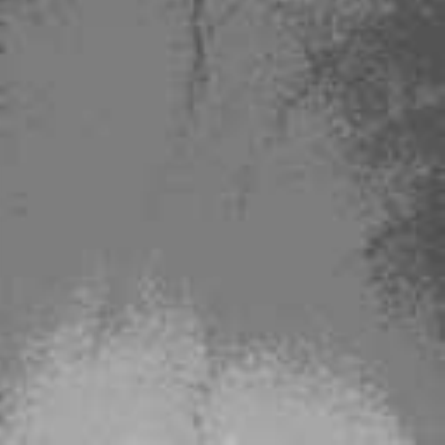}}
\subfigure[540nm]{\includegraphics[width=0.09\linewidth]{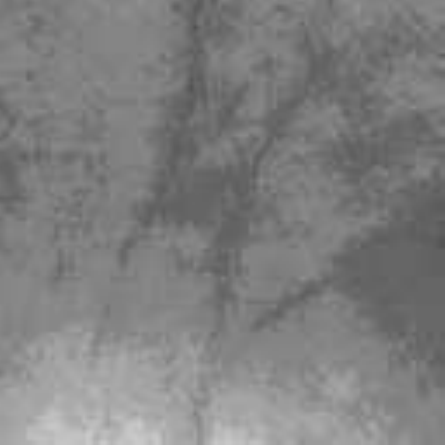}}
\subfigure[570nm]{\includegraphics[width=0.09\linewidth]{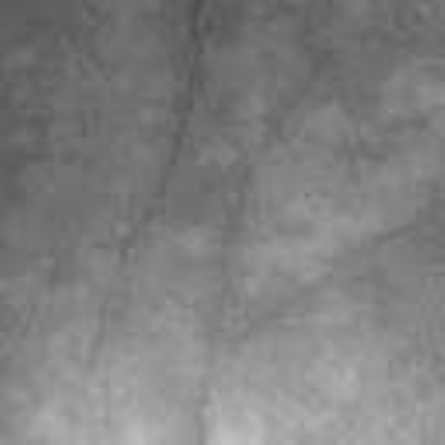}}
\subfigure[600nm]{\includegraphics[width=0.09\linewidth]{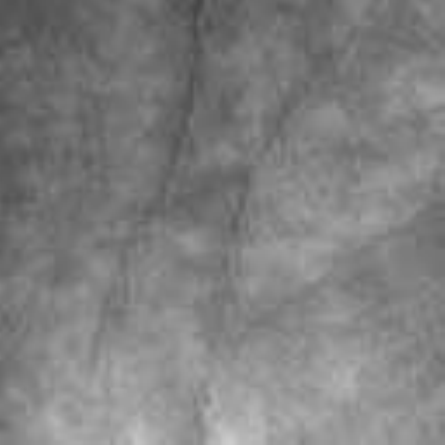}}
\subfigure[630nm]{\includegraphics[width=0.09\linewidth]{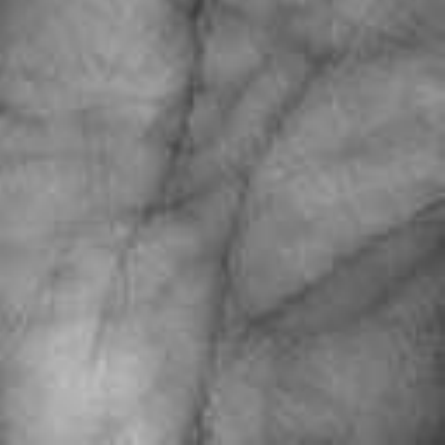}}
\subfigure[660nm]{\includegraphics[width=0.09\linewidth]{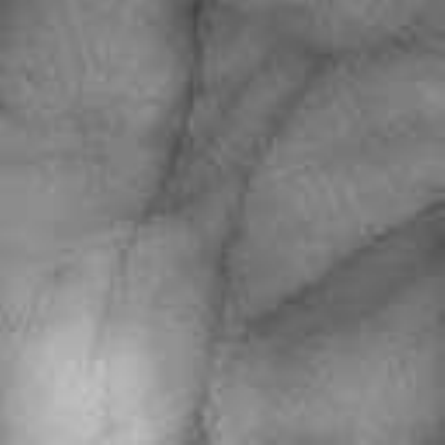}}
\subfigure[690nm]{\includegraphics[width=0.09\linewidth]{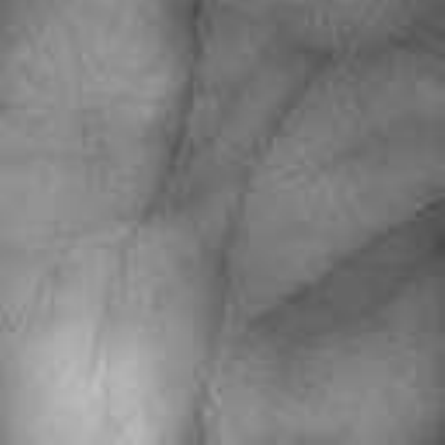}}
\subfigure[720nm]{\includegraphics[width=0.09\linewidth]{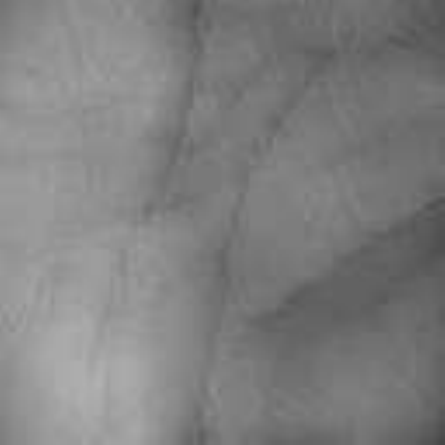}}
\subfigure[750nm]{\includegraphics[width=0.09\linewidth]{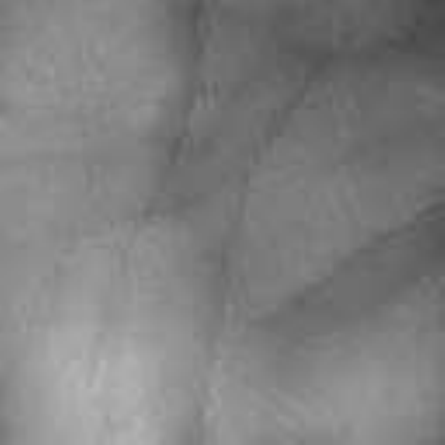}}
\subfigure[780nm]{\includegraphics[width=0.09\linewidth]{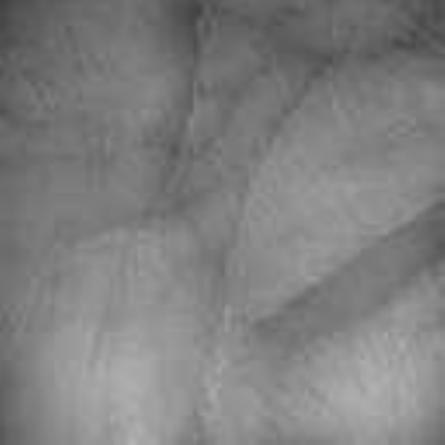}} \\ [9pt]
\subfigure[810nm]{\includegraphics[width=0.09\linewidth]{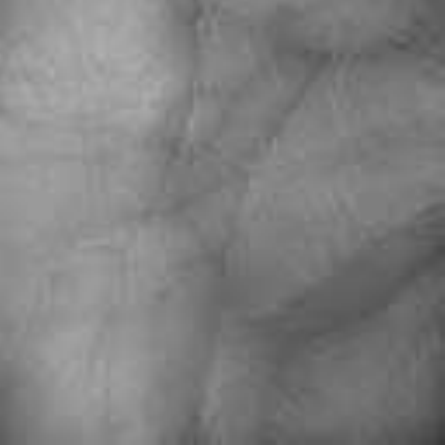}}
\subfigure[840nm]{\includegraphics[width=0.09\linewidth]{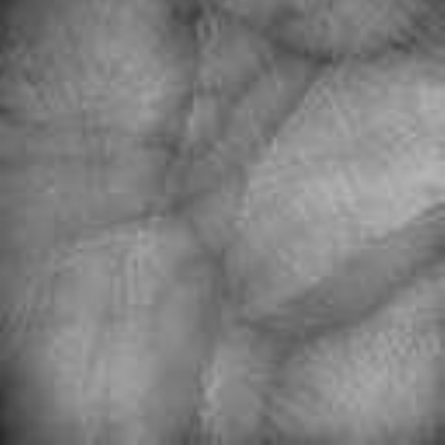}}
\subfigure[870nm]{\includegraphics[width=0.09\linewidth]{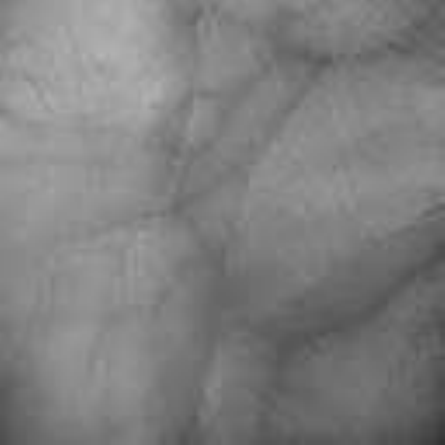}}
\subfigure[900nm]{\includegraphics[width=0.09\linewidth]{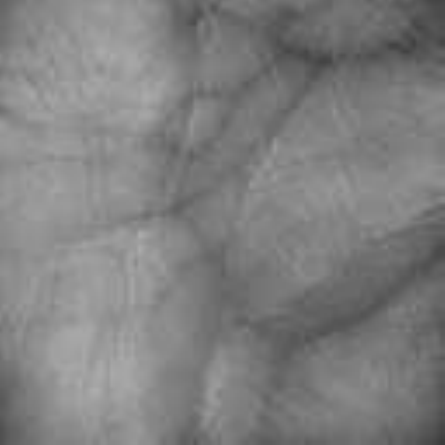}}
\subfigure[930nm]{\includegraphics[width=0.09\linewidth]{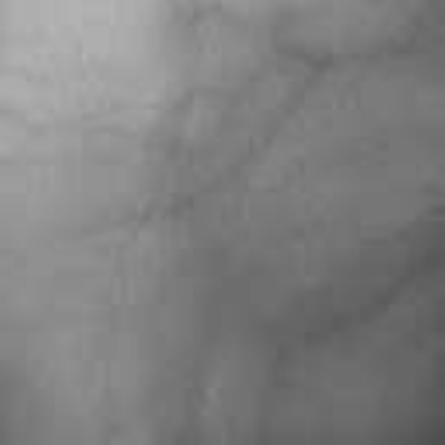}}
\subfigure[960nm]{\includegraphics[width=0.09\linewidth]{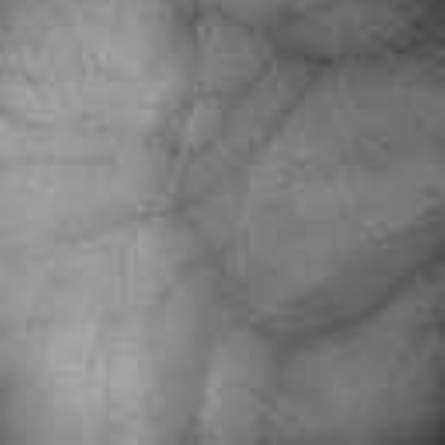}}
\subfigure[990nm]{\includegraphics[width=0.09\linewidth]{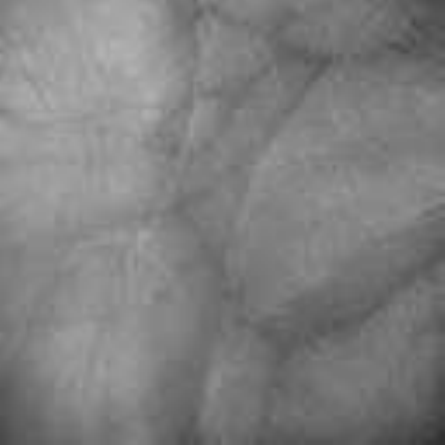}}
\subfigure[1020nm]{\includegraphics[width=0.09\linewidth]{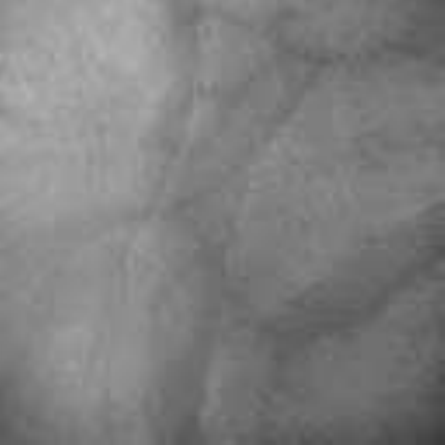}}
\subfigure[1050nm]{\includegraphics[width=0.09\linewidth]{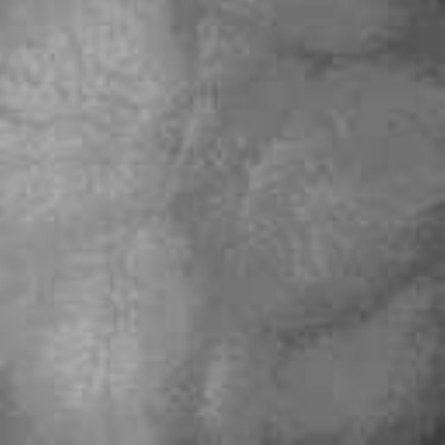}}
\subfigure[1080nm]{\includegraphics[width=0.09\linewidth]{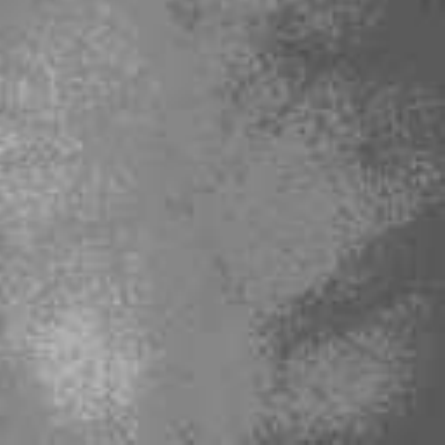}}
\caption[Examples of palmprint features in multiple bands]{Examples of palmprint features in multiple bands of a hyperspectral image (510 to 1080nm with 30nm steps. The first and last few bands are highly corrupted with system noise and barely capture features of the palm.}
\label{fig:palms}
\end{figure}

\section{Related Work}
\label{sec:work}

In the past decade, biometrics such as the iris~\cite{boyce2006multispectral}, face~\cite{di2010studies,pan2003face} and fingerprint~\cite{rowe2007multispectral} have been investigated using multispectral images for improved accuracy. Recently, there has been an increased interest in multispectral palmprint recognition~\cite{meraoumia2013efficient,mittal2012rank,luo2012multispectral,mistani2011multispectral,kekre2011palmprint} \cite{tahmasebi2011novel,zhang2011online,kisku2010multispectral,guo2010feature,zhang2010online} \cite{xu2010multispectral,han2008multispectral,hao2008multispectral,hao2007comparative}. In general, palmprint recognition approaches can be categorized into line-like feature detectors, subspace learning methods and texture based coding techniques~\cite{kong2009survey}. These three categories are not mutually exclusive and their combinations are also possible. Line detection based approaches commonly extract palm lines using edge detectors. Huang et al.~\cite{huang2008palmprint} proposed a palmprint verification technique based on principal lines. The principal palm lines were extracted using a modified finite Radon transform and a binary edge map was used for representation. However, recognition based solely on palm lines proved insufficient due to their sparse nature and the possibility of different individuals to have highly similar palm lines~\cite{zhang2011online}. Although, line detection can extract palm lines effectively, it may not be equally useful for the extraction of palm veins due to their low contrast and broad structure.

A subspace projection captures the global characteristics of a palm by projecting to the most varying (in case of PCA) or the most discriminative (in case of LDA) dimensions. Subspace projection methods include eigenpalm~\cite{lu2003palmprint}, which globally projects palm images to a PCA space, or fisherpalm~\cite{wu2003fisherpalms} which projects to an LDA space. However, the finer local details are not well preserved and modeled by such subspace projections. Wang et al.~\cite{wang2007fusion} fused palmprint and palmvein images and proposed the \emph{Laplacianpalm} representation. Unlike the eigenpalm~\cite{lu2003palmprint} or the fisherpalm~\cite{wu2003fisherpalms}, the \emph{Laplacianpalm} representation attempts to preserve the local characteristics as well while projecting onto a subspace. Xu et al.~\cite{xu2010multispectral} represented multispectral palmprint images as quaternion and applied quaternion PCA to extract features. A nearest neighbor classifier was used for recognition using quaternion vectors. The quaternion model did not prove useful for representing multispectral palm images and demonstrated low recognition accuracy compared to the state-of-the-art techniques. The main reason is that subspaces learned from misaligned palms are unlikely to generate accurate representation of each identity.

Orientation codes extract and encode the orientation of lines and have shown state-of-the-art performance in palmprint recognition~\cite{zhang2012comparative}. Examples of orientation codes include the Competitive Code (CompCode)~\cite{kong2004competitive}, the Ordinal Code (OrdCode)~\cite{sun2005ordinal} and the Derivative of Gaussian Code (DoGCode)~\cite{wu2006palmprint}. In the generic form of orientation coding, the response of a palm to a bank of directional filters is computed such that the resulting directional subbands correspond to specific orientations of line. Then, the dominant orientation index from the directional subbands is extracted at each point to form the orientation code. CompCode~\cite{kong2004competitive} employs a directional bank of Gabor filters to extract the orientation of palm lines. The orientation is encoded into a binary code and matched directly using the Hamming distance. The OrdCode~\cite{sun2005ordinal} emphasizes the ordinal relationship of lines by comparing mutually orthogonal filter pairs to extract the feature orientation at a point. The DoGCode~\cite{wu2006palmprint} is a compact representation which only uses vertical and horizontal gaussian derivative filters to extract feature orientation. Orientation codes can be binarized for efficient storage and fast matching unlike other representations which require floating point data storage and computations. Another important aspect of multispectral palmprints is the combination of different bands which has been investigated with data, feature, score and rank-level fusion.

\clearpage

Multispectral palmprints are fused at image level generally using multi-resolution transforms, such as Wavelet and Curvelet. Han et al.~\cite{han2008multispectral} used a three level Wavelet fusion strategy for combining multispectral palmprint images. After fusion, CompCode was used for feature extraction and matching. Their results showed that the Wavelet fusion of multispectral palm images is only useful for blurred source images. Hao et al.~\cite{hao2008multispectral} used various image fusion techniques and the OLOF representation for multispectral palmprint recognition. The best recognition performance was achieved when the Curvelet transform was used for band fusion. Kisku et al.~\cite{kisku2010multispectral} proposed Wavelet based band fusion and Gabor Wavelet feature representation for multispectral palm images. To reduce the dimensionality, feature selection was performed using the Ant Colony Optimization (ACO) algorithm~\cite{dorigo1997ant} and classified by normalized correlation and SVM. However, the Gabor Wavelet based band fusion could not improve palmprint recognition performance compared to the Curvelet fusion with OLOF~\cite{hao2008multispectral}. Kekre et al.~\cite{kekre2011palmprint} proposed a hybrid transform by kronecker product of DCT and Walsh transforms, which better describes the energy in the local regions of a multispectral palm. A subset of the regions was selected by comparison with the mean energy map and stored as features for matching. It is observed that fusion of multispectral palmprints is cumbersome due to multimodal nature of palm. A single fused palm image is a compromise of the wealth of complementary information present in different bands, which results in below par recognition performance.

Fusion of spectral bands has been demonstrated at feature level. Luo et al.~\cite{luo2012multispectral} used feature level band fusion for multispectral palmprints. Specifically, a modification of CompCode was combined with the original CompCode and features from the pair of less correlated bands were fused. The results indicated an improvement over image level fusion and were comparable to match-score level fusion. Zhou and Kumar~\cite{zhou2010contactless} encoded palm vein features by enhancement of vascular patterns and using the Hessian phase information. They showed that a combination of various feature representations can be used for achieving improved performance based on palmvein images. Mittal et al.~\cite{mittal2012rank} investigated fuzzy and sigmoid features for multispectral palmprints and a rank-level fusion of scores using various strategies. It was observed that a nonlinear fusion function at rank-level was effective for improved recognition performance. Tahmasebi et al.~\cite{tahmasebi2011novel} used Gabor kernels for feature extraction from multispectral palmprints and a rank-level fusion scheme for fusing the outputs from individual band comparisons. One drawback of rank-level fusion is that it assigns fixed weights to the rank outputs of spectral bands, which results in sub-optimal performance.

Zhang et al.~\cite{zhang2010online} compared palmprint matching using individual bands and reported that the red band performed better than the near infrared, blue and green bands. A score level fusion of these bands achieved superior performance compared to any single band. Another joint palmline and palmvein approach for multispectral palmprint recognition was proposed by Zhang et al.~\cite{zhang2011online}. They designed separate feature extraction methodologies for palm line and palm vein and later used score level fusion for computing the final match. The approach yielded promising results, albeit at the cost of increased complexity. A comparison of different fusion strategies indicates that a score level fusion of multispectral bands is promising and most effective compared to a data, feature or rank-level fusion.

It is worth mentioning that a simple extension of the existing palmprint representations to multispectral palmprints may not fully preserve the features that appear in different bands. For example, the representation may not be able to extract both lines and vein features from different bands. Moreover, existing research suggests that a score level fusion of multispectral bands is promising compared to a data level fusion using multi-resolution transforms.

In this chapter, we propose Contour Code, a novel orientation and binary hash table based encoding for robust and efficient multispectral palmprint recognition. Unlike existing orientation codes, which apply a directional filter bank directly to a palm image, we propose a two stage filtering approach to extract only the robust directional features. We develop a unified methodology for the extraction of multispectral (the line and vein) features. The Contour Code is binarized into an efficient hash table structure that only requires indexing and summation operations for simultaneous one-to-many matching with an embedded score level fusion of multiple bands. The Contour Code is compared to three existing methods~\cite{kong2004competitive}\cite{sun2005ordinal}\cite{wu2006palmprint} in palmprint recognition.


\section{Region of Interest Extraction}
\label{sec:roi_ext}

The extraction of a reliable ROI from contact free palmprint images is a major challenge and involves a series of steps. To extract ROI from a palmprint, it is necessary to define some reference \emph{landmarks} from within the palm which can normalize its relative movement. The landmark detection must be accurate and repeatable to ensure that the same ROI is extracted from different planar views of a palm. Among the features commonly used in hand geometry analysis, the valleys between the fingers are a suitable choice for landmarks due to their invariance to hand movement.

\subsection{Preprocessing}

The input band of a hand image is first thresholded to get a binary image (see  Figure~\ref{fig:hand_gs} and~\ref{fig:hand_tr}). Smaller objects, not connected to the hand, that appear due to noise are removed using binary pixel connectivity. Morphological closing is carried out with a square structuring element to normalize the contour of the hand. Finally, any holes within the hand pixels are filled using binary hole filling. These operations ensure accurate and successful landmarks localization. The resulting preprocessed binary image $\mathbf{B}$ is shown in Figure~\ref{fig:hand_bw}.

\begin{figure}[h]
\begin{center}
\begin{minipage}[b]{0.86\linewidth}
\subfigure[]{\label{fig:hand_gs}\includegraphics[trim = 80pt 5pt 30pt 0pt, clip, width=0.315\linewidth]{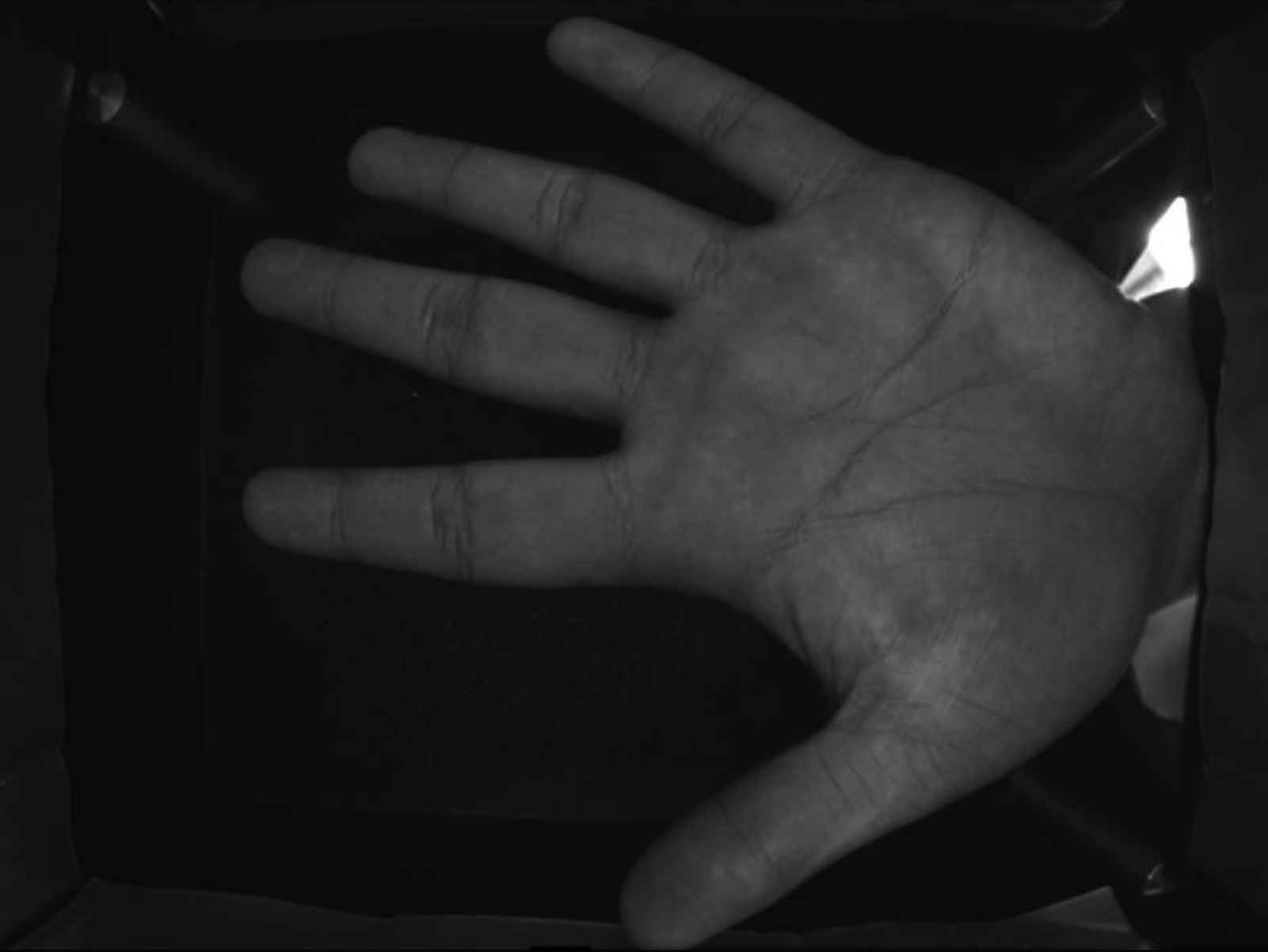}}\hspace{0.05pt}
\subfigure[]{\label{fig:hand_tr}\includegraphics[trim = 80pt 5pt 30pt 0pt, clip, width=0.315\linewidth]{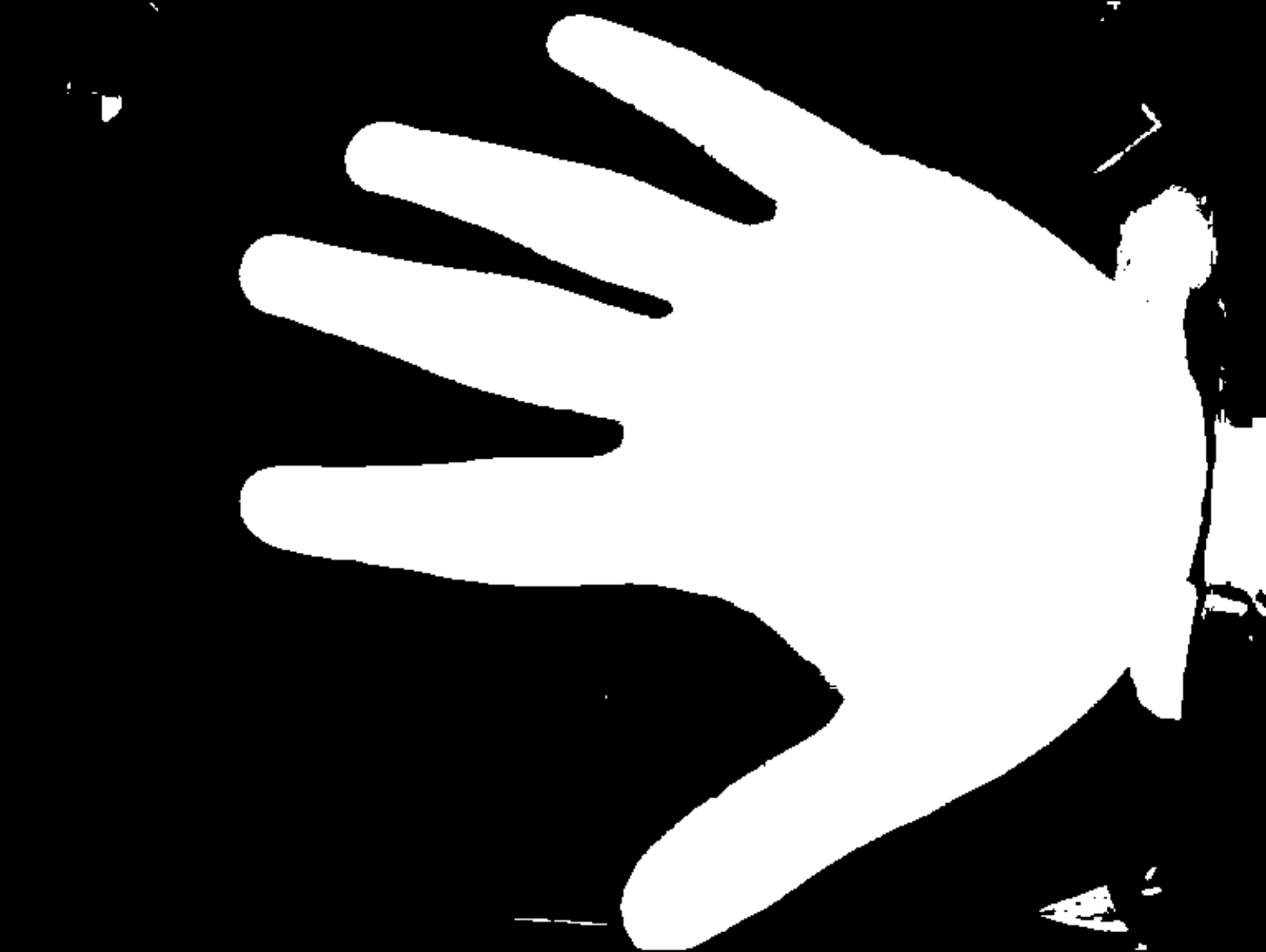}}\hspace{0.05pt}
\subfigure[]{\label{fig:hand_bw}\includegraphics[trim = 80pt 5pt 30pt 0pt, clip, width=0.315\linewidth]{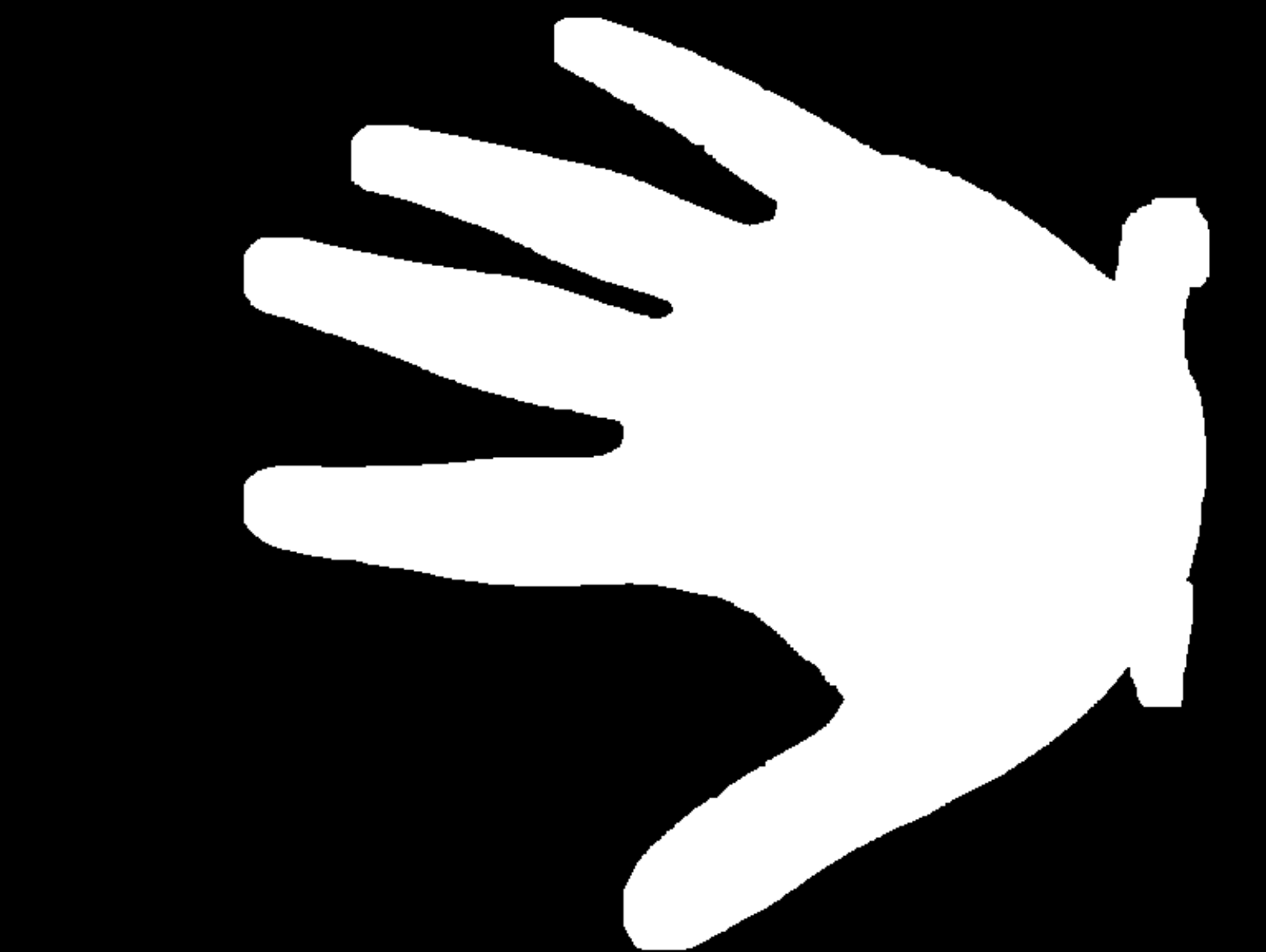}}\\
\end{minipage}
\end{center}
\caption[Hand image preprocessing]{Hand image preprocessing. (a) Input image. (b) Thresholded image with noise. (c) Preprocessed image with reduced noise.}
\label{fig:hand_pp}
\vspace{-10pt}
\end{figure}

\subsection{Localization of Landmarks}

Given $\mathbf{B}$, in which the foreground pixels correspond to the hand, the localization proceeds as follows. In a column wise search (Figure~\ref{fig:hand_mp}), the binary discontinuities, i.e. the edges of the fingers are identified. From the edges, the gaps between fingers are located in each column and the mid points of all the finger gaps are computed. Continuing inwards along the valley of the fingers, the column encountered next to the last column should contain hand pixels (Figure~\ref{fig:hand_kpc}). The search is terminated when four such valley closings are recorded. This column search succeeds when the hand rotation is within $\pm90^{\circ}$ in the image plane.

\begin{figure}[h]
\begin{center}
\begin{minipage}[b]{0.8\linewidth}
\subfigure[]{\label{fig:hand_mp}\includegraphics[trim = 80pt 5pt 20pt 0pt, clip, width=0.31\linewidth]{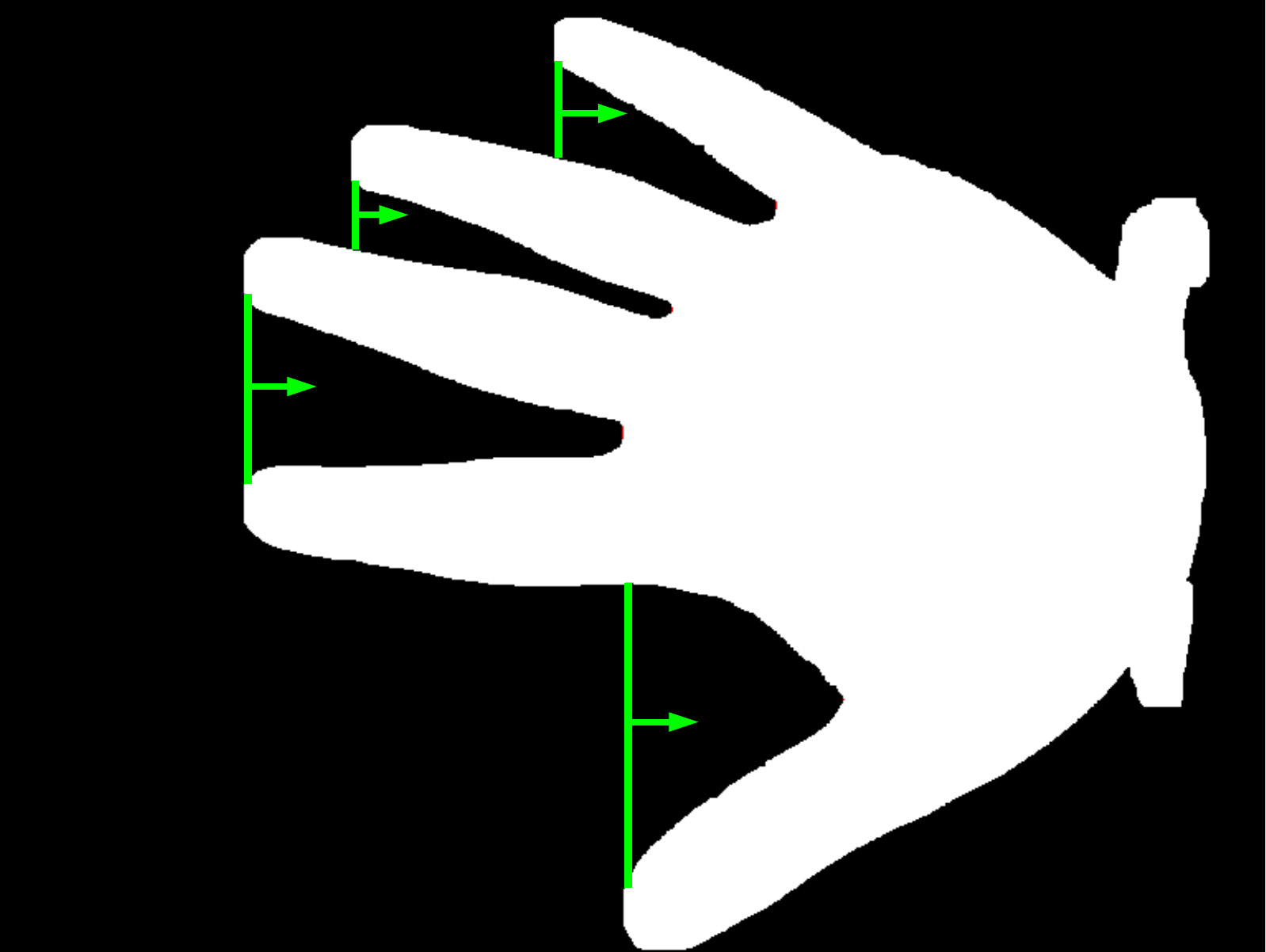}}\hspace{0.05pt}
\subfigure[]{\label{fig:hand_kpc}
\begin{minipage}[b]{0.16\linewidth}
\fbox{\includegraphics[trim = 264pt 264pt 177pt 70pt, clip, width=1\linewidth]{chapter_7/hand_kpc}}\\
\fbox{\includegraphics[trim = 222pt 230pt 215pt 106pt, clip, width=1\linewidth]{chapter_7/hand_kpc}}\\
\fbox{\includegraphics[trim = 209pt 183pt 233pt 151pt, clip, width=1\linewidth]{chapter_7/hand_kpc}}
\end{minipage}}\hspace{0.05pt}
\subfigure[]{\label{fig:hand_mid}\fbox{\includegraphics[trim = 80pt 130pt 150pt 0pt, clip, width=0.315\linewidth]{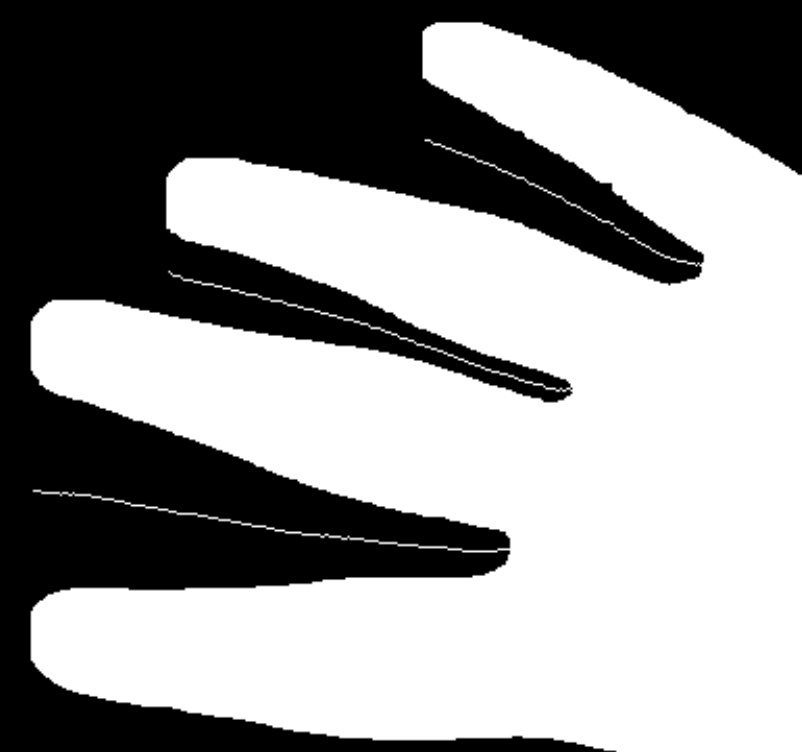}}}\hspace{0.05pt}
\subfigure[]{\label{fig:hand_kpf}
\begin{minipage}[b]{0.16\linewidth}
\fbox{\includegraphics[trim = 264pt 264pt 177pt 70pt, clip, width=1\linewidth]{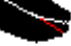}}\\
\fbox{\includegraphics[trim = 222pt 230pt 215pt 106pt, clip, width=1\linewidth]{chapter_7/hand_kpf}}\\
\fbox{\includegraphics[trim = 209pt 183pt 233pt 151pt, clip, width=1\linewidth]{chapter_7/hand_kpf}}
\end{minipage}}\hspace{0.05pt}
\end{minipage}
\end{center}
\caption[Landmarks localization in a hand image]{Landmarks localization in a hand image acquired with a non-contact sensor. (a) Initialization of search for mid points (green). (b) Search termination (red). (c) Located mid points. (d) Polynomial extrapolation of selected mid points (red) to find the landmarks.}
\label{fig:hand_lm}
\vspace{-5pt}
\end{figure}

The mid points corresponding to the four valleys are clustered and the one corresponding to the index-thumb valley is discarded. Due to the natural contour of the fingers, it can be observed that the mid points do not follow a linear trend towards the valley as shown in Figure~\ref{fig:hand_mid}. Therefore, a second order polynomial is fitted to the mid points of each valley excluding the last few points which tend to deviate from the path towards the landmark. The estimated polynomial is then used to predict the last few mid points (10\%) towards the valley as shown in Figure~\ref{fig:hand_kpf}. The final location of a landmark is the last encountered background pixel in the sequence of extrapolated points. The procedure is summarized in Algorithm~\ref{alg}.

\begin{algorithm}[h]  
\caption{Localization of Landmarks}          
\label{alg}                           
\begin{algorithmic}                    
\Require $\mathbf{B}$ \Comment{binary Image}
\Ensure $\mathbf{P} \in \mathbb{R}^{3\times 2}$ \Comment{landmark coordinates}
\State $\textbf{Initialize:}~ \mathcal{M} \gets \emptyset$ \Comment{set of mid point coordinates in $\mathbf{B}$}
\State $l\gets0, c\gets2$ \Comment{landmark counter, start search from column `2'}
\While{$l<4$} \Comment{maximum of four landmarks}
\State $\mathcal{I} \gets \{i | \quad B_{i,c} = 1\}$ \Comment{foreground pixel locations in column $\mathbf{C}$}
\State $\mathcal{D} \gets \{j | \quad \mathcal{I}^{j+1}-\mathcal{I}^j > 1\}$ \Comment{discontinuities in column $\mathbf{C}$}
\For{$u=1$ to $|\mathcal{D}|$}
\State $x_1 \gets \mathcal{I}^{\mathcal{D}^u}, x_2 \gets \mathcal{I}^{\mathcal{D}^{u+1}}, x \gets \frac{x_1+x_2}{2}$ \Comment {mid point location}
\State $\mathcal{M} \gets \mathcal{M} \cup (x,c)$
\If{$|\mathbf{B}_c^{\centerdot x_1+1:x_2-1} \vee \mathbf{B}_{c-1}^{\centerdot x_1+1:x_2-1}|=\sum \mathbf{B}_c^{\centerdot x_1+1:x_2-1}$}
\State $l\gets l+1$ \Comment{valley closed}
\EndIf
\EndFor
\State $c\gets c+1$
\EndWhile
\State $\mathcal{P} \gets$ \textsc{cluster}$(\mathcal{M},l)$ \Comment{group midpoints into $l$ clusters}
\For{$j=1$ to $3$}
\State $\zeta_j\gets$ \textsc{fitpoly}($\mathcal{P}_j^{\centerdot 90\%},2$) \Comment{fit $2^\textrm{nd}$ order polynomial to midpoints of $j^\textrm{th}$ landmark}
\State $\mathcal{P}_j^{\centerdot 10\%} \gets$ \textsc{evalpoly}($\mathcal{P}_j^{\centerdot 10\%},\zeta_j$) \Comment{extrapolate midpoints using $\zeta_j$}
\State $\mathbf{P}_j(x,y) \gets \underset{x,y}{\arg}(\underset{y}{\arg \max} I_{\mathcal{P}_j(x,y)}=0)$ \Comment{coordinates of the last background pixel}
\EndFor
\end{algorithmic}
\end{algorithm}

\begin{figure}[h]
\centering
\begin{minipage}[b]{0.8\linewidth}
\subfigure[]{\label{fig:roi_scale}\includegraphics[trim = 80pt 10pt 20pt 0pt, clip, width=0.49\linewidth]{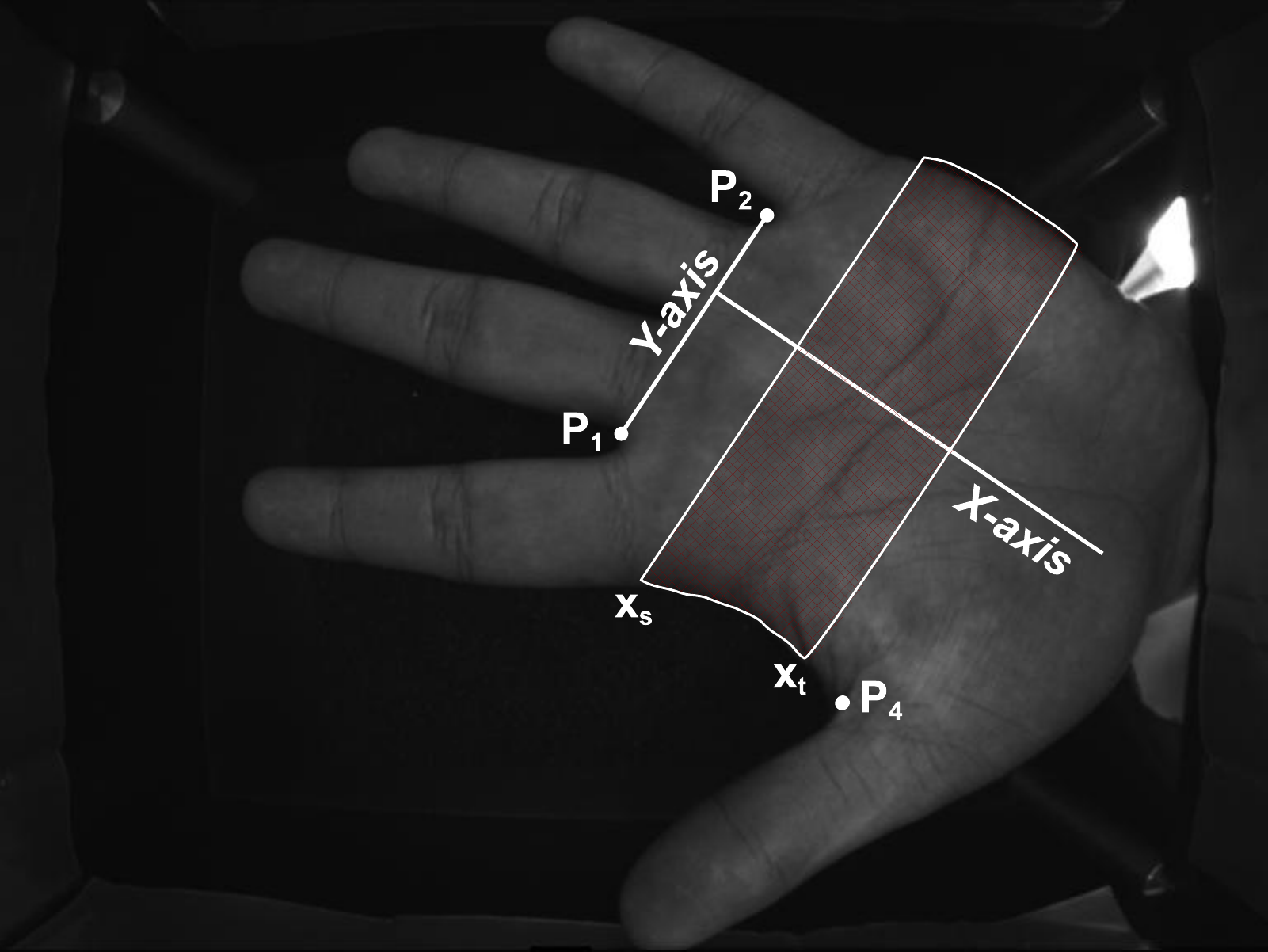}}\hspace{1pt}
\subfigure[]{\label{fig:roi_extract}\includegraphics[trim = 80pt 10pt 20pt 0pt, clip, width=0.49\linewidth]{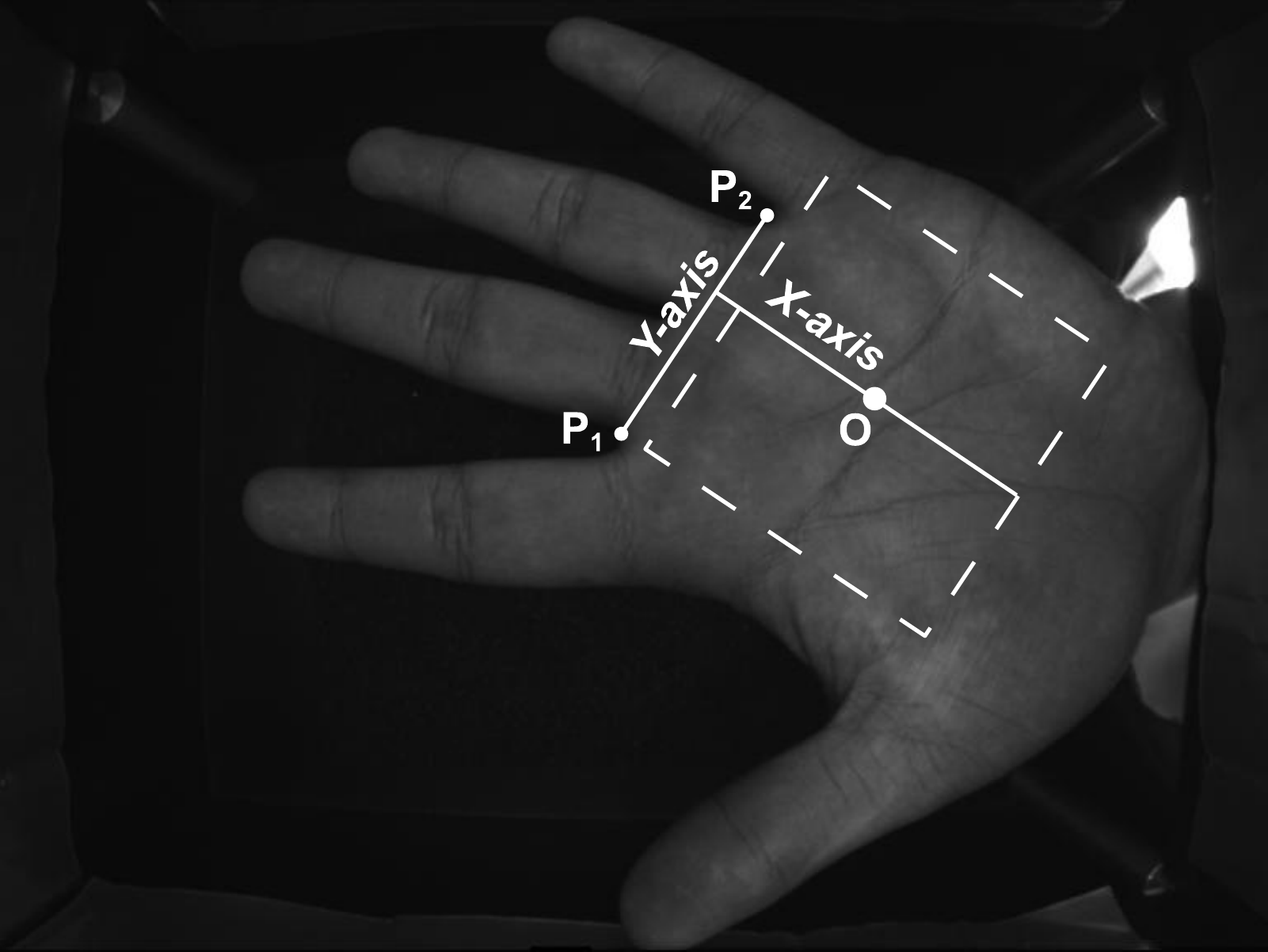}}
\end{minipage}
\caption[ROI extraction based on detected landmarks]{ROI extraction based on detected landmarks. The points $P_1,P_2$ define the \emph{Y-axis}. The \emph{X-axis} is orthogonal to the \emph{Y-axis} at 2/3 distance from $P_1$ to $P_2$. (a) Average palm width $\bar{w}$ is computed in the range $(x_{s},x_{t})$ (b) The distance of origin $O$ of the ROI is proportional to the average palm width found in (a).}
\label{fig:roi_all}
\vspace{-5pt}
\end{figure}

\subsection{ROI Extraction}

Using the three landmarks, an RST invariant ROI can be extracted (see Figure~\ref{fig:roi_all}). The landmarks ($P_1$,$P_2$) form a reference \emph{Y-axis}. We fix the \emph{X-axis} at two-thirds of the distance from $P_1$ to $P_2$ so that the ROI is centered over the palm. The automatically estimated palm width $\bar{w}$ serves as the scale identifier to extract a scale invariant ROI. To compute $\bar{w}$, we find the average width of the palm from point $x_{s}$ to $x_{t}$. To keep $x_{s}$ from being very close to the fingers and affected by their movement we set a safe value of $x_{s}=P_1(x)+(P_2(y)-P_1(y))/3$. Moreover, we set $x_{t}=P_{4}(x)-(P_2(y)-P_1(y))/12$. Note that the scale computation is not sensitive to the values of $x_{s}$ and $x_{t}$ as the palmwidth is averaged over the range $(x_{s},x_{t})$.

In our experiments, the ROI side length is scaled to $70\%$ of $\bar{w}$ and extracted from the input image using an affine transformation. The same region is extracted from the remaining bands of the multispectral palm.

\subsection{Inter-band Registration}
Since the bands of the multispectral images were sequentially acquired, minor hand movement can not be ruled out. Therefore, an inter-band registration of the ROIs based on the maximization of \emph{Mutual Information} is carried out. The approach has shown to be effective for registering multispectral palmprints~\cite{hao2008multispectral}. Since, there is negligible rotational and scale variation within the consecutive bands, we limit the registration search space to only $\pm2$ pixels translations along both dimensions. The registered ROI of each band is then downsampled to $32\times32$ pixels using bicubic interpolation. This resampling step has several advantages. First, it suppresses the inconsistent lines and noisy regions. Second, it reduces the storage requirement for the final Contour Code. Third, it significantly reduces the time required for the extraction of features and matching.

\section{Contour Code}
\label{sec:cnt_code}

\subsection{Multidirectional Feature Encoding}

The nonsubsampled contourlet transform (NSCT)~\cite{da2006nonsubsampled} is a multi-directional expansion with improved directional frequency localization properties and efficient implementation compared to its predecessor, the contourlet transform~\cite{do2005contourlet}. It has been effective in basic image processing applications such as image denoising and enhancement. Here, we exploit the directional frequency localization characteristics of the NSCT for multidirectional feature extraction from palmprint images.

An ROI of a band $\mathbf{I} \in \mathbb{R}^{m\times n}$ is first convolved with a nonsubsampled bandpass pyramidal filter $(\mathbf{F}_{p})$ which captures the details in the palm at a single scale as shown in Figure~\ref{fig:ContourCode}. This filtering operation allows only the robust information in a palm to be passed on to the subsequent directional decomposition stage.
\begin{equation}
\bm{\rho} = \mathbf{I} \ast \mathbf{F}_{p}~.
\end{equation}
The band pass filtered component $(\bm{\rho}) \in \mathbb{R}^{m \times n}$ of the input image is subsequently processed by a nonsubsampled directional filter bank $(\mathbf{F}_{d})$, comprising $2^{k}$ directional filters.
\begin{equation}
\bm{\Psi} = \bm{\rho} \ast \mathbf{F}_{d}^{i}~,
\end{equation}
where $\bm{\Psi} \in \mathbb{R}^{m \times n \times 2^k}$ is the set of directional subbands. Each directional subband covers an angular region of $\pi/2^{k}$ radians. For example, a third order directional filter bank ($k=3$) will result in $8$ directional filtered subbands. The combination of pyramidal and directional filters determine the capability to capture line like features. We perform a detailed experimental analysis of pyramidal-directional filters in Section~\ref{sec:filter}.

\begin{figure}[h]
\centering
\includegraphics[width=0.7\linewidth]{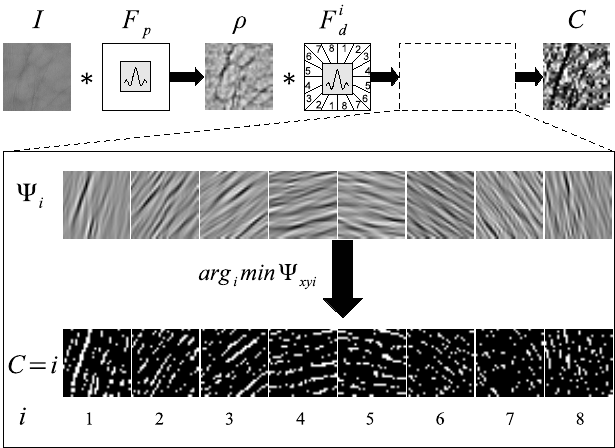}
\caption[Extraction of the Contour Code representation]{Extraction of the Contour Code representation. $\bm{\rho}$ is the pyramidal bandpass subband and $\bm{\Psi}_{i} \in \mathbb{R}^{m \times n}$ are the bandpass directional subbands. The images $\mathbf{C}=i$ represent the dominant points existing in the $i^{th}$ directional subband. The intensities in $\mathbf{C}$ correspond to $i=1,2,\ldots,8$ (from dark to bright).}
\label{fig:ContourCode}
\end{figure}

Generally, both the line and vein patterns appear as dark intensities in a palm and correspond to a negative filter response. The orientation of a feature is determined by the coefficient corresponding to the minimum peak response among all directional subbands at a specific point. Let $\Psi_{xyi}$ denote the coefficient at point $(x,y)$ in the $i^{th}$ directional subband where $i=1,2,3,\ldots,2^{k}$. We define a rule similar to the competitive rule~\cite{kong2004competitive}, to encode the dominant orientation at each $(x,y)$.
\begin{equation}
C_{xy}=\underset{i}{\arg \min}\,(\Psi_{xyi})~,
\end{equation}
where $\mathbf{C}$ is the Contour Code representation of $\mathbf{I}$. Similarly, $\mathbf{C}$ is computed for all bands of the multispectral image of a palm. An example procedure for a single band of a palm is shown in Figure~\ref{fig:ContourCode}.

\begin{landscape}
\begin{figure}[h]
\centering
\includegraphics[width=1\linewidth]{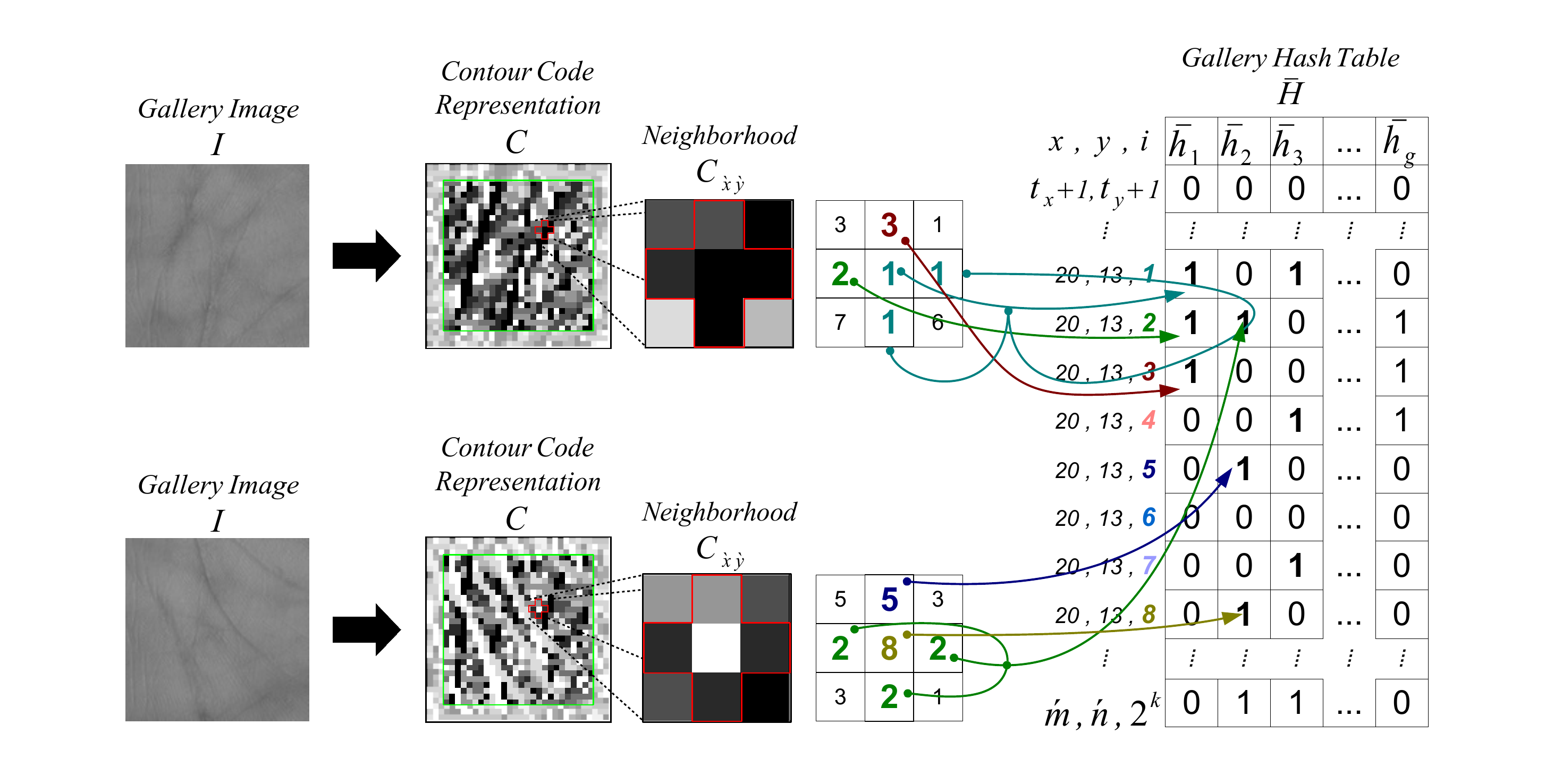}
\caption[Illustration of binary hash table encoding]{Illustration of binary hash table encoding. A gallery image $\mathbf{I}$ is first converted to its Contour Code representation $\mathbf{C}$. We show how a single point $C_{\acute{x}\acute{y}}$ is encoded into the binary hash table based on its \emph{z-connected} neighborhood (z=4 in this example). For $\bar{\mathbf{h}}_{1}$, bins 1,2 and 3 are encoded as 1 and the rest as 0. Whereas for $\bar{\mathbf{h}}_{2}$, bins 2, 5 and 8 are encoded as 1 and the remaining are set as 0.}
\label{fig:hash_pop}
\end{figure}
\end{landscape}

\subsection{Binary Hash Table Encoding}

A code with $2^{k}$ orientations requires a minimum of $k$ bits for encoding. However, we binarize the Contour Code using $2^{k}$ bits to take advantage of a fast binary code matching scheme. Unlike other methods which use the Angular distance~\cite{kong2004competitive} or the Hamming distance~\cite{hao2008multispectral}, we propose an efficient binary hash table based Contour Code matching. Each column of the hash table refers to a palm's binary hash vector derived from its Contour Code representation. Within a column, each hash location $(x,y)$ has $2^{k}$ bins corresponding to each orientation. We define a hash function so that each point can be mapped to the corresponding location and bin in the hash table. For a pixel $(x,y)$, the hash function assigns it to the $i^{th}$ bin according to
\begin{equation}
\label{eq:hash}
H^{xyi} =
\begin{dcases}
1, & i=C_{xy}\\
0, & \textrm{otherwise}
\end{dcases}
\end{equation}
where $\mathbf{H}$ is the binarized form of a Contour Code representation $\mathbf{C}$.

Since hands are naturally non-rigid, it is not possible to achieve a perfect 1-1 correspondences between all the orientations of two palm images taken at different instances of time. It is, therefore, intuitive to assign multiple orientations to a hash location $(x,y)$ based on its neighborhood. Therefore, the hash function is blurred so that it assigns for each hash bin in $\bar{\mathbf{H}}$ with all the orientations at $(x,y)$ and its neighbors.
\begin{equation}
\bar{H}^{xyi} =
\begin{dcases}
1, & i \in C_{\acute{x}\acute{y}}\\
0, & \textrm{otherwise}
\end{dcases}
\end{equation}
where ($\acute{x},\acute{y}$) is the set of $(x,y)$ and its \emph{z-connected} neighbors. The extent of the blur neighborhood $(z)$, determines the flexibility/rigidity of matching. Less blurring will result in a small number of points matched, however, with high confidence. On the other hand, too much blur will result in a large number of points matched but with low confidence.

\begin{figure}[h]
\centering
\subfigure[]{\label{fig:blur-no}\includegraphics[width=0.49\linewidth]{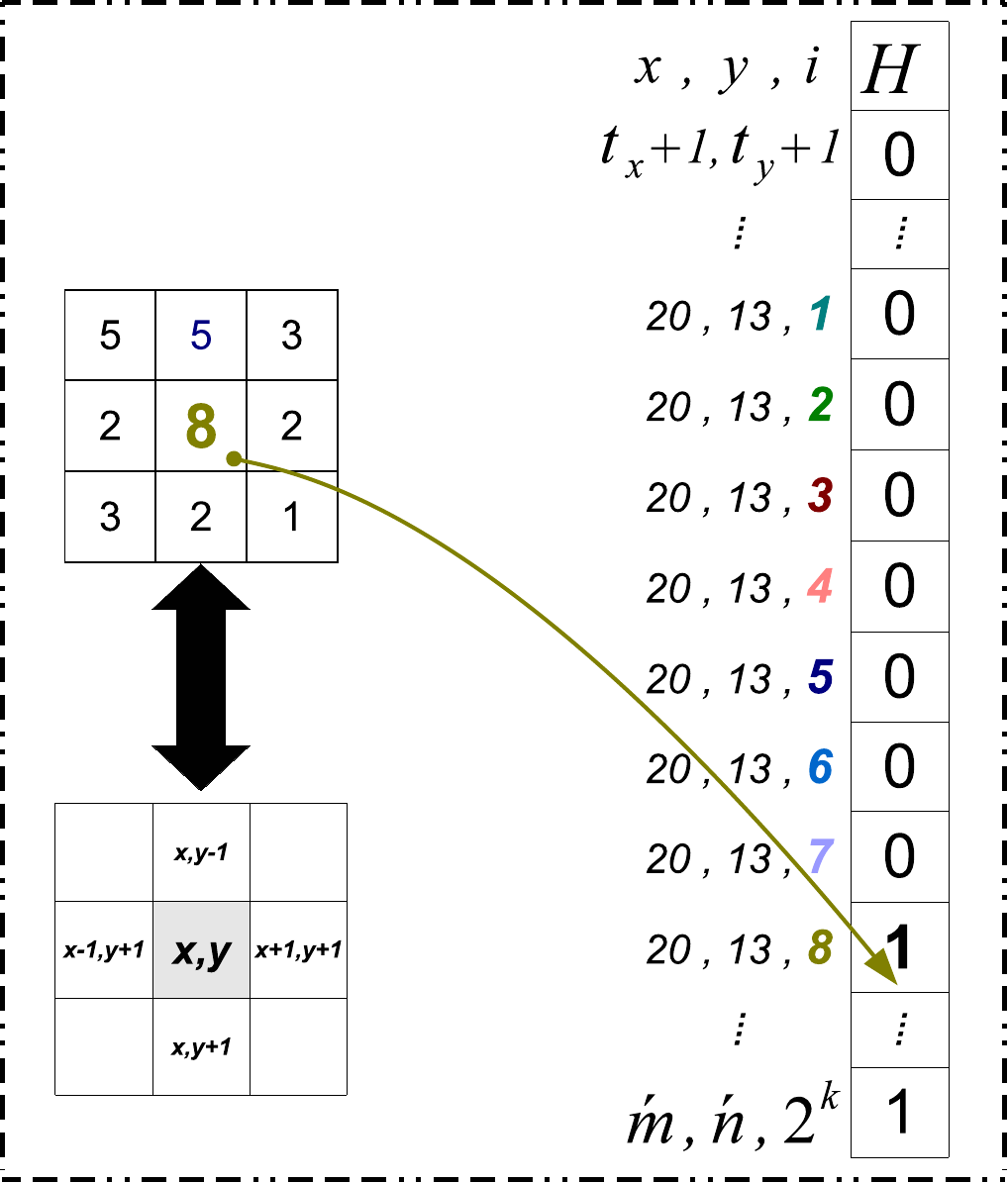}}\hspace{1pt}
\subfigure[]{\label{fig:blur-plus}\includegraphics[width=0.49\linewidth]{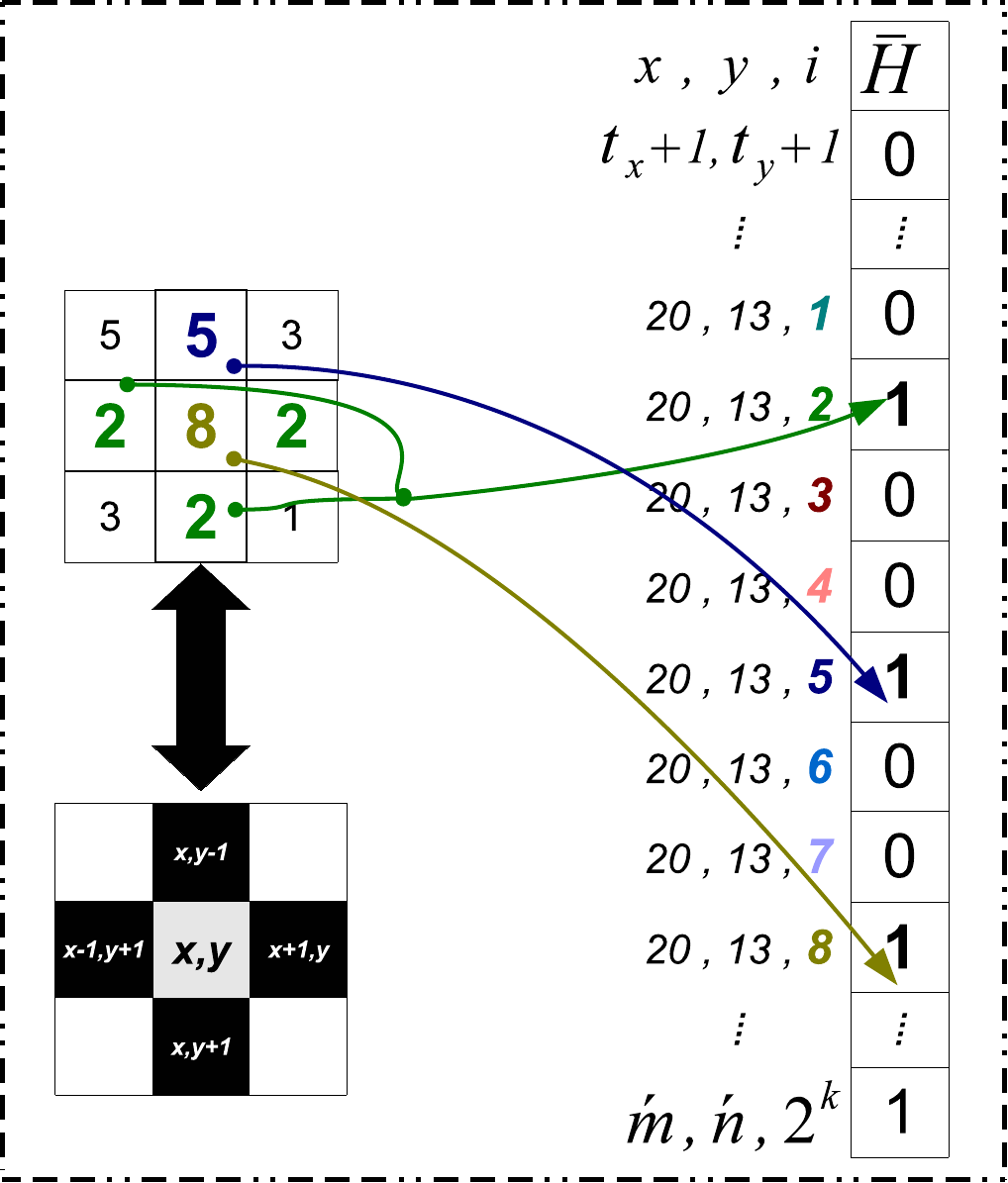}}
\caption[Binary hash table encoding different forms of blurring]{Binary hash table encoding with (a) no blurring, (b) a \emph{4-connected} neighborhood blurring.}
\label{fig:blur-eg}
\end{figure}

Since, the hash table blurring depends on a certain neighborhood as opposed to a single pixel, it robustly captures crossover line orientations. A detailed example of Contour Code binarization using the blurred hash function is given in Figure~\ref{fig:hash_pop}. Figure~\ref{fig:blur-eg} illustrates hash table encoding without blurring and with a {\em 4-connected} neighborhood blurring. The discussion of an appropriate blur neighborhood is presented later in Section~\ref{sec:params}.

\begin{figure}[!h]
\centering
\includegraphics[width=1\linewidth]{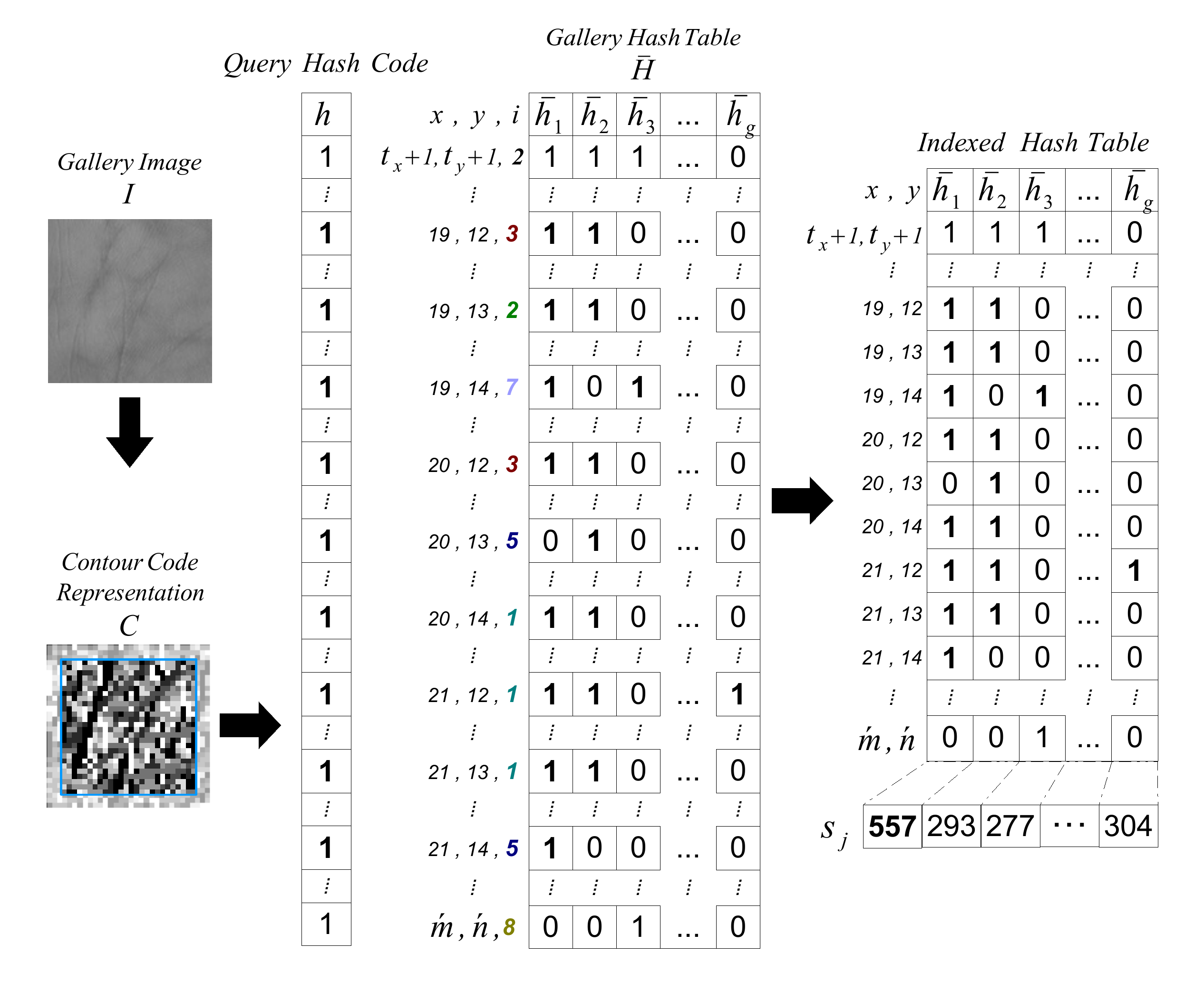}
\caption[Illustration of the Contour Code matching]{Illustration of the Contour Code matching. A query image is first transformed into its binary encoding without blurring to get $\mathbf{h}$, which is subsequently used to index the rows of the gallery hash table $\bar{\mathbf{H}}$. The columns of indexed hash table are summed up to obtain the match scores with the gallery hash table. In the above example, the query image is best matched to $\bar{\mathbf{h}}_{1}$ as it has more $1$s in the indexed hash table resulting in $s_{1}$ to be the maximum match score.}
\label{fig:hash_mtch}
\end{figure}

\subsection{Matching}

The binary hash table facilitates simultaneous one-to-many matching for palmprint identification. Figure~\ref{fig:hash_mtch} illustrates the process. The Contour Code representation $\mathbf{C}$ of a query image is first converted to the binary form $\mathbf{h}$ using equation~\eqref{eq:hash}. No blurring is required now, since it has already been performed offline on all the gallery images. The match scores $(\mathbf{s} \in \mathbb{R}^g)$ between query hash code $\mathbf{h}\in \mathbb{R}^{2^kmn}$ and gallery hash table $\bar{\mathbf{H}} \in \mathbb{R}^{2^kmn \times g}$, is computed as
\begin{align}
\label{eq:mtch}
\mathcal{X} &= \{xyi | \quad \underset{xyi}{\arg} H^{xyi}=1\}\\
s_j&=\|\bar{\mathbf{H}}_j^{\mathcal{X}}\|_0
\end{align}
where ${\|.\|}_0$ is the $\ell_{0}$-norm which is the number of non-zero hash entries in $\bar{\mathbf{H}}$ for all bin locations where $H^{xyi}=1$. The hash table, after indexing, produces a relatively sparse binary matrix which can be efficiently summed. Since, the $\ell_{0}$-norm of a binary vector is equivalent to the summation of all the vector elements, equation \eqref{eq:mtch} can be rewritten as
\begin{equation}
s_j=\sum \{\bar{\mathbf{H}}_j^{\mathcal{X}}\}~.
\end{equation}

\subsubsection{Translated Matches}
Apart from the blurring, which caters for minor orientation location errors within an ROI, during matching, the query image is shifted $\pm t_{x}$ pixels horizontally and $\pm t_{y}$ pixels vertically to cater for the global misalignment between different ROIs. The highest score among all translations is considered as the final match. The class of a query palm is determined by the gallery image $n$ corresponding to the best match.
\begin{equation}
\textrm{class}=\underset{j}{\arg \max} (S_{xyj})~,
\end{equation}
where $\mathbf{S} \in \mathbb{R}^{g\times 2t_x+1 \times 2t_y+1}$ is the matching score matrix of a query image with all gallery images and under all possible translations.

We present two variants of matching and report results for both in Section~\ref{sec:exp}. When the bands are translated in synchronization, it is referred to as \emph{Synchronous Translation Matching} (denoted by {ContCode-STM}) and when the bands are independently translated to find the best possible match, it is referred to as \emph{Asynchronous Translation Matching} (denoted by {ContCode-ATM}).

Due to translated matching, the overlapping region of two matched images reduces to ($\acute{m}=m-2t_{x}$,$\acute{n}=n-2t_{y}$) pixels. Therefore, only the central $\acute{m}\times\acute{n}$ region of a Contour Code is required to encode in the hash table, further reducing the storage requirement. Consequently during matching, a shifting window of $\acute{m}\times\acute{n}$ of the query image is matched with the hash table. We selected $t_{x},t_{y}=3$, since no improvement in match score was observed for translation beyond $\pm3$ pixels. The indexing part of the matching is independent of the database size and only depends on the matched ROI size $(\acute{m}\times\acute{n})$. Additionally, we vertically stack the hash table entries of all the bands to compute the aggregate match score in a single step for {ContCode-STM}. Thus the score level fusion of bands is embedded within the hash table matching. In the case of {ContCode-ATM}, each band is matched separately and the resulting scores are summed. Since, a match score is always an integer, no floating point comparisons are required for the final decision.

In a verification (1-1 matching) scenario, the query palm may be matched with the palm samples of the claimed identity only. Thus the effective width of the hash table becomes equal to the number of palm samples of the claimed identity but its height remains the same $(2^{k}\acute{m}\acute{n})$.

\section{Experiments}
\label{sec:exp}


\subsection{Databases}
\label{sec:data}

We used the PolyU-MS\footnote{\tiny{PolyU Multispectral Palmprint Database}~ \url{http://www.comp.polyu.edu.hk/~biometrics/MultispectralPalmprint/MSP.htm}}, PolyU-HS\footnote{\tiny{PolyU Hyperspectral Palmprint Database}~ \url{http://www4.comp.polyu.edu.hk/~biometrics/HyperspectralPalmprint/HSP.htm}} and CASIA-MS\footnote{\tiny{CASIA Multispectral Palmprint Database}~\url{http://www.cbsr.ia.ac.cn/MS_Palmprint_Database.asp}} palmprint databases in our experiments. All databases contain low resolution ($<$150 \emph{dpi}) palm images stored as 8-bit grayscale images per band. Several samples of each subject were acquired in two different sessions. Detailed specifications of the databases are given in Table~\ref{tab:databases}.
\begin{table}[h]
\caption{Specifications of the PolyU-MS, PolyU-HS and CASIA-MS databases.}
\label{tab:databases}
\centering
\begin{tabular}{|l|c|c|c|} \hline
Database                        &   PolyU-MS    &   PolyU-HS    &   CASIA-MS        \\ \hline
Sensor                          &   contact 	&   contact     &  non-contact      \\
Identities                      &   500         &   380         &   200             \\
Samples per identity            &   12          &   11-14       &   6               \\
Total samples                   &   6000        &   5240        &   1200            \\
Bands per sample                &   4           &   69          &   6               \\
\multirow{2}{*}{Wavelength(nm)} &   470, 525,   &   420-1100    &   460, 630, 700,  \\
                                &   660, 880    & (10 nm steps) &   850, 940, White \\ \hline
\end{tabular}
\end{table}

The PolyU-HS database was collected with the aim to find the minimum number of bands required for designing a multispectral palmprint recognition system rather than utilizing the complete set of hyperspectral bands. We reduced the number of bands of the PolyU-HS database from 69 to 4 using the compressed hyperspectral imaging method proposed in Chapter~\ref{Chapter5}. We observed that any additional bands did not significantly improve palmprint recognition accuracy. The four most informative bands were 770, 820, 910 and 920nm.

\subsection{ROI Extraction Accuracy}
\label{sec:roi_acc}

The ROIs in the PolyU-MS and PolyU-HS database are already extracted according to~\cite{zhang2003online}. The CASIA-MS database was acquired using a non-contact sensor and contains large RST variations. The accuracy of automatic landmark localization was determined by comparison with a manual identification of landmarks for 200 multispectral images in the CASIA-MS database. Manual selections were averaged over the six bands to minimize human error. In this section, the displacement, rotation and scale variation of the automatically detected landmarks from the manually marked ground is computed and analyzed.

We define the localization error as the Chebyshev distance between the manually selected and the automatically extracted landmarks. The localization error $(e_{\ell})$, between two landmarks is computed as
\begin{equation}
e_{\ell}=\max(|x-\bar{x}|,|y-\bar{y}|)\times \hat{S}~,
\end{equation}
where $(x,y)$, $(\bar{x},\bar{y})$ correspond to the manually and automatically identified landmark coordinates respectively. The $e_{\ell}$ is 1 if the located landmark falls within the first 8 neighboring pixels, 2 if it falls in the first 24 neighboring pixels and so on. Due to scale variation, the $e_{\ell}$ between different palm images could not be directly compared. For example, a localization error of 5 pixels in a close-up image may correspond to only a 1 pixel error in a distant image. To avoid this, we normalize $e_{\ell}$ by determining it at the final size of the ROI. The normalization factor $\hat{S}$, is the side length of ROI at the original scale, $0.7\times \bar{w}$ divided by the final side length of the ROI ($m=32$).

The absolute rotation error $e_{\theta}$ between two palm samples is defined as
\begin{equation}
e_{\theta}=|\theta-\bar{\theta}|~,
\end{equation}
where, $\theta$ and $\bar{\theta}$ correspond to the angle of rotation of a palm computed from manual and automatic landmarks respectively. Finally, the percentage scale error $e_{s}$ is defined as
\begin{equation}
e_{s}=\left(\max\left(\frac{w}{\bar{w}},\frac{\bar{w}}{w}\right)-1\right)\times100 \%~,
\end{equation}
where $w$ is the manually identified palm width averaged over three measurements, while $\bar{w}$ is automatically calculated.

\begin{figure}[h]
\begin{center}
\subfigure[]{\label{fig:cum_trans}\includegraphics[trim = 0pt 5pt 50pt 20pt, clip, width=0.32\linewidth]{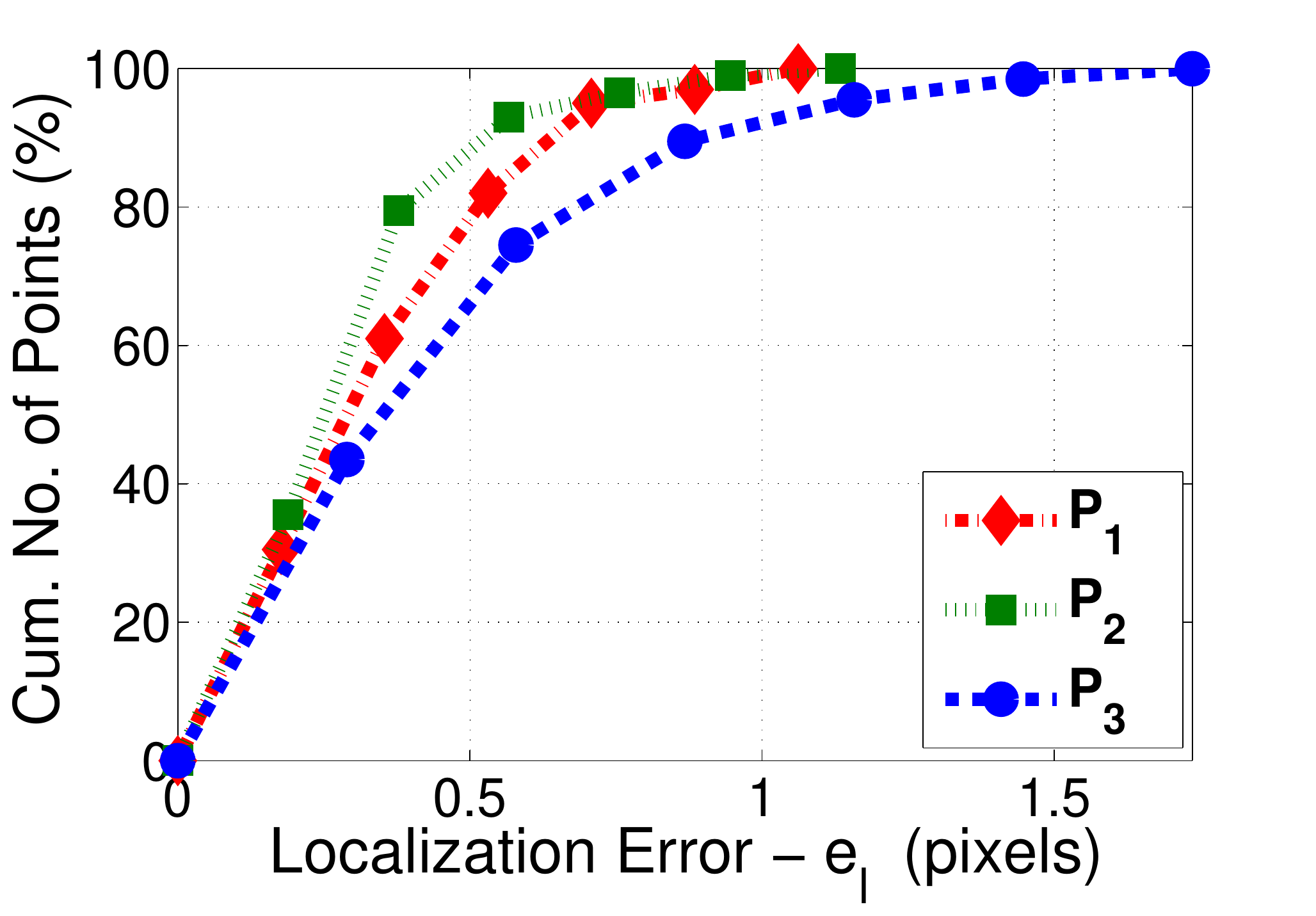}}\hspace{1pt}
\subfigure[]{\label{fig:cum_rot}\includegraphics[trim = 0pt 5pt 50pt 20pt, clip, width=0.32\linewidth]{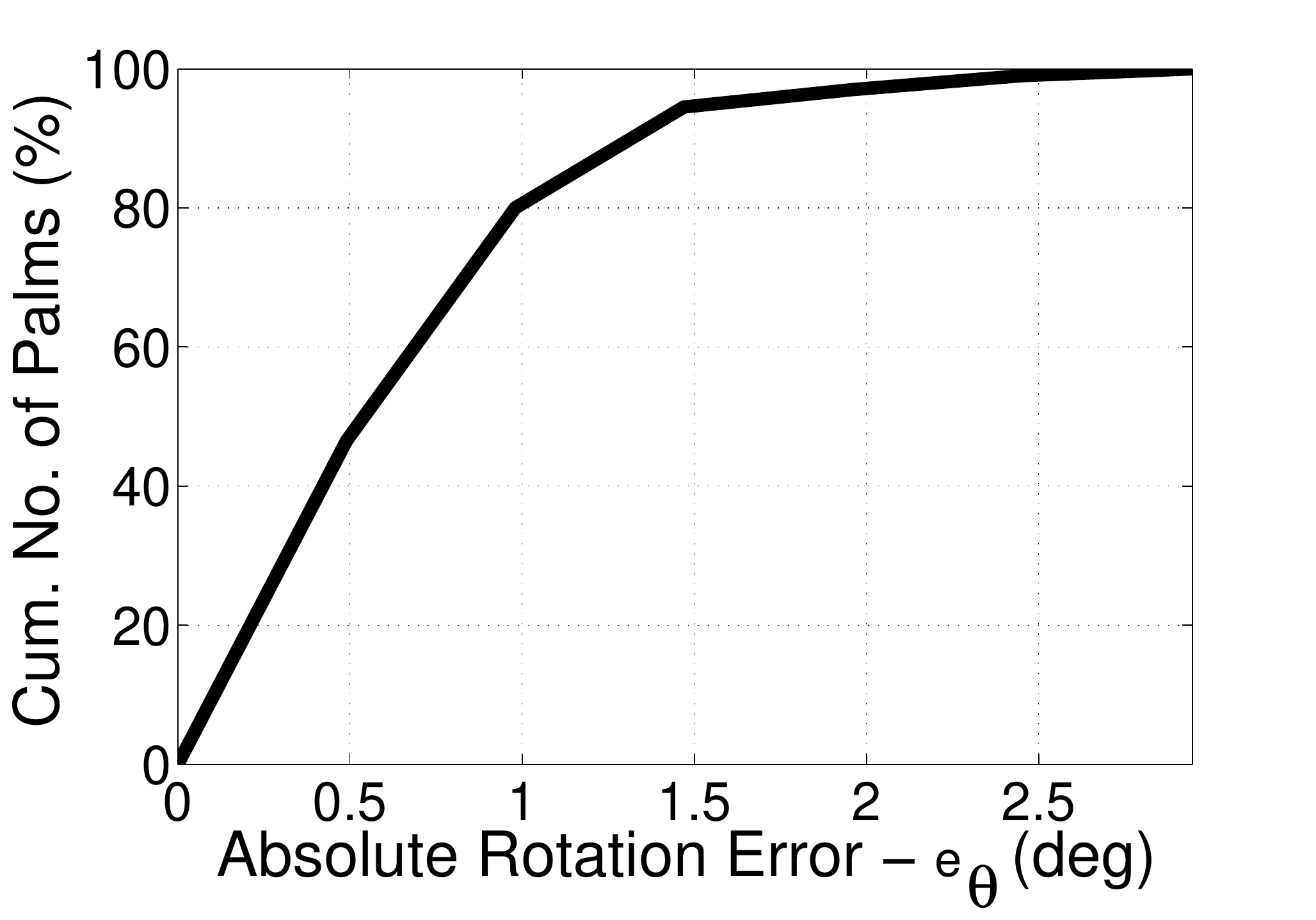}}\hspace{1pt}
\subfigure[]{\label{fig:cum_scale}\includegraphics[trim = 0pt 5pt 50pt 20pt, clip, width=0.32\linewidth]{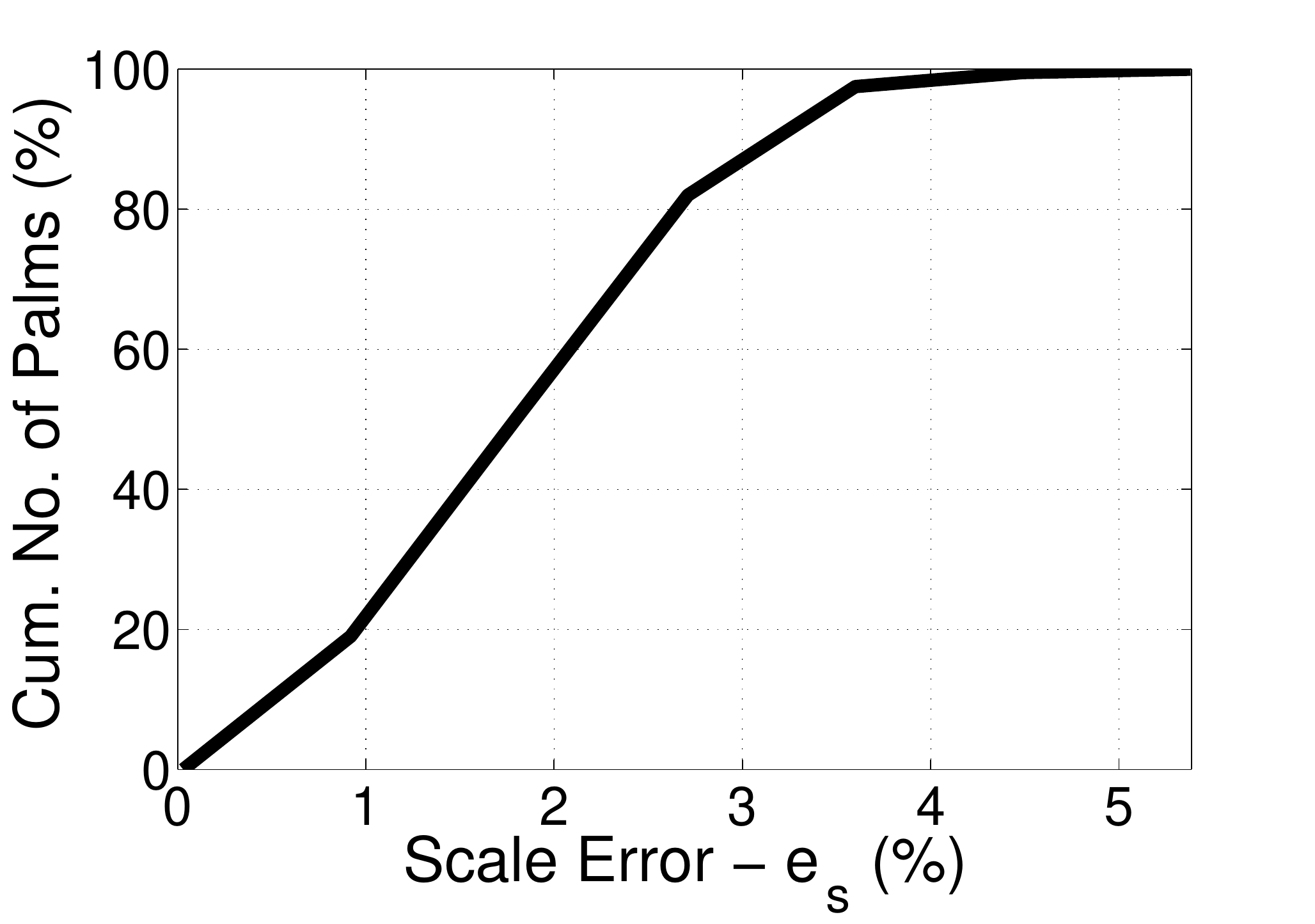}}
\end{center}
\caption[Evaluation of ROI extraction accuracy]{Evaluation of ROI extraction accuracy. (a) Cumulative percentage of landmarks within a localization error $(L_{e})$. (b) Cumulative percentage of palm samples within an absolute rotation error $(\theta_{e})$. (c) Cumulative percentage of palm samples within a scale error $(S_{e})$.}
\label{fig:cum_tsr}
\end{figure}

Figure~\ref{fig:cum_trans} shows the error, in pixels, of landmarks $P_1,P_2,P_3$. It can be observed that the three landmarks are correctly located within an $e_{\ell}\le2$ for all of the samples. It is important to emphasize that $P_1,P_2$ which are actually used in the ROI extraction are more accurately localized ($e_{\ell}\simeq$ 1 pixel) compared to $P_3$. Thus, our ROI extraction is based on relatively reliable landmarks. Figure~\ref{fig:cum_rot} shows the absolute rotation error of the automatically extracted ROIs. All the ROIs were extracted within a $e_{\theta} \le 3^{\circ}$. The proposed ROI extraction is robust to rotational variations given the overall absolute rotational variation, $(\mu_{\theta},\sigma_{\theta})\!=\!(20.7^{\circ},7.24^{\circ})$ of the CASIA-MS database. The proposed technique is able to estimate the palm scale within an error of $\pm5.4\%$ (Figure~\ref{fig:cum_scale}).

\begin{figure}[!h]
\centering
\begin{minipage}[b]{0.8\linewidth}
\subfigure{\includegraphics[trim = 130pt 10pt 30pt 0pt, clip,width=0.235\linewidth]{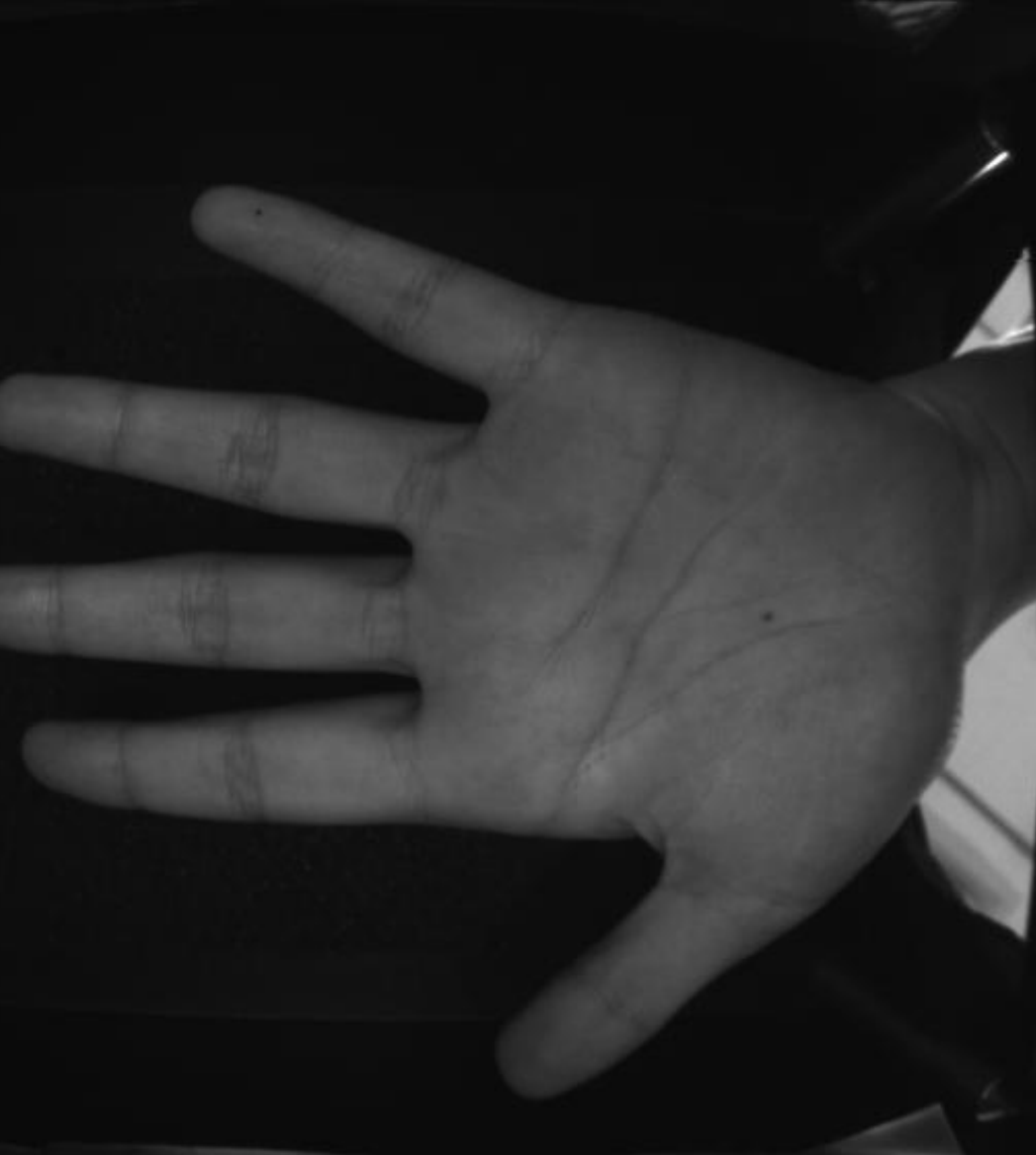}}\hspace{0.1pt}
\subfigure{\includegraphics[trim = 130pt 10pt 30pt 0pt, clip,width=0.235\linewidth]{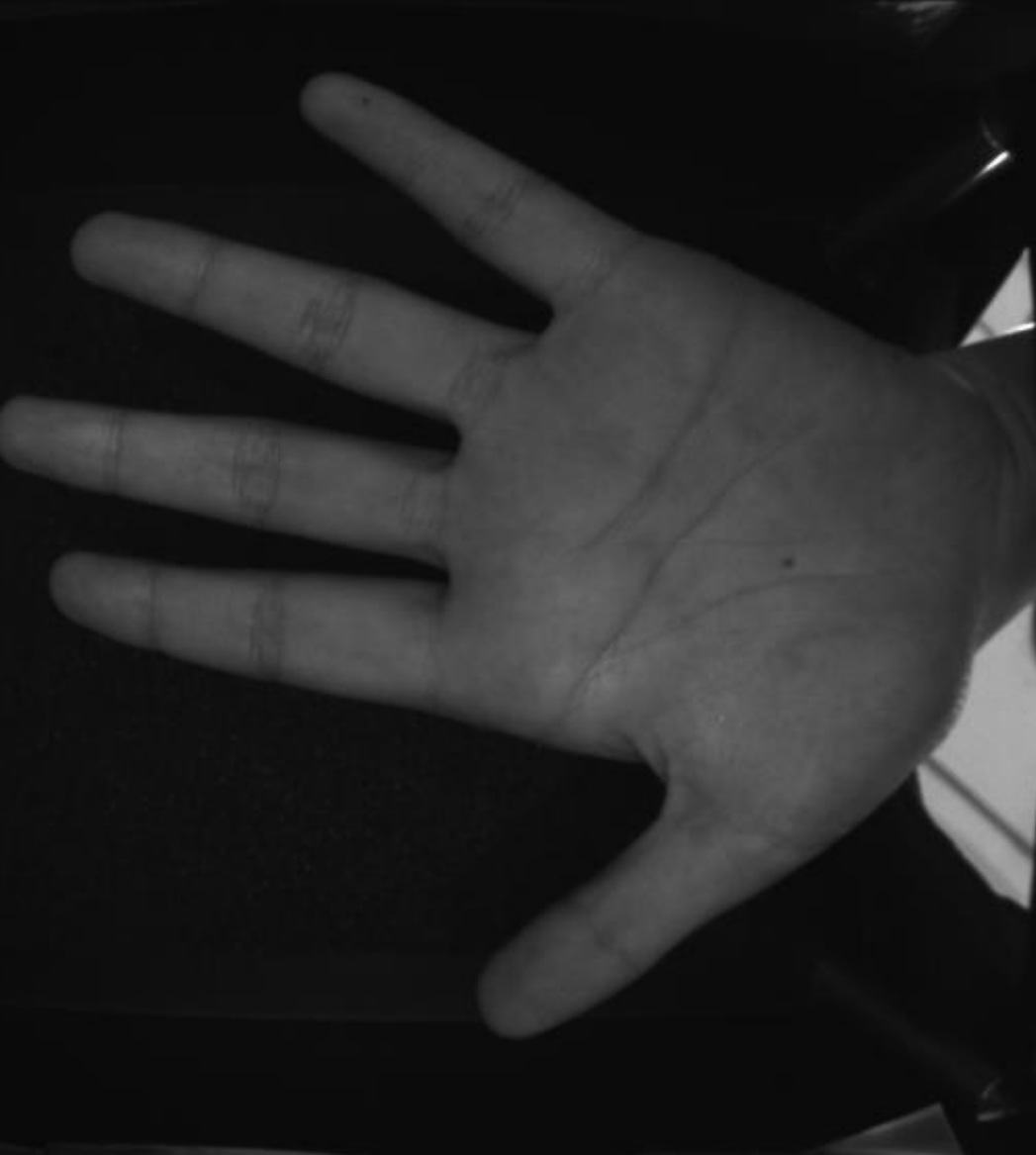}}\hspace{0.1pt}
\subfigure{\includegraphics[trim = 130pt 10pt 30pt 0pt, clip,width=0.235\linewidth]{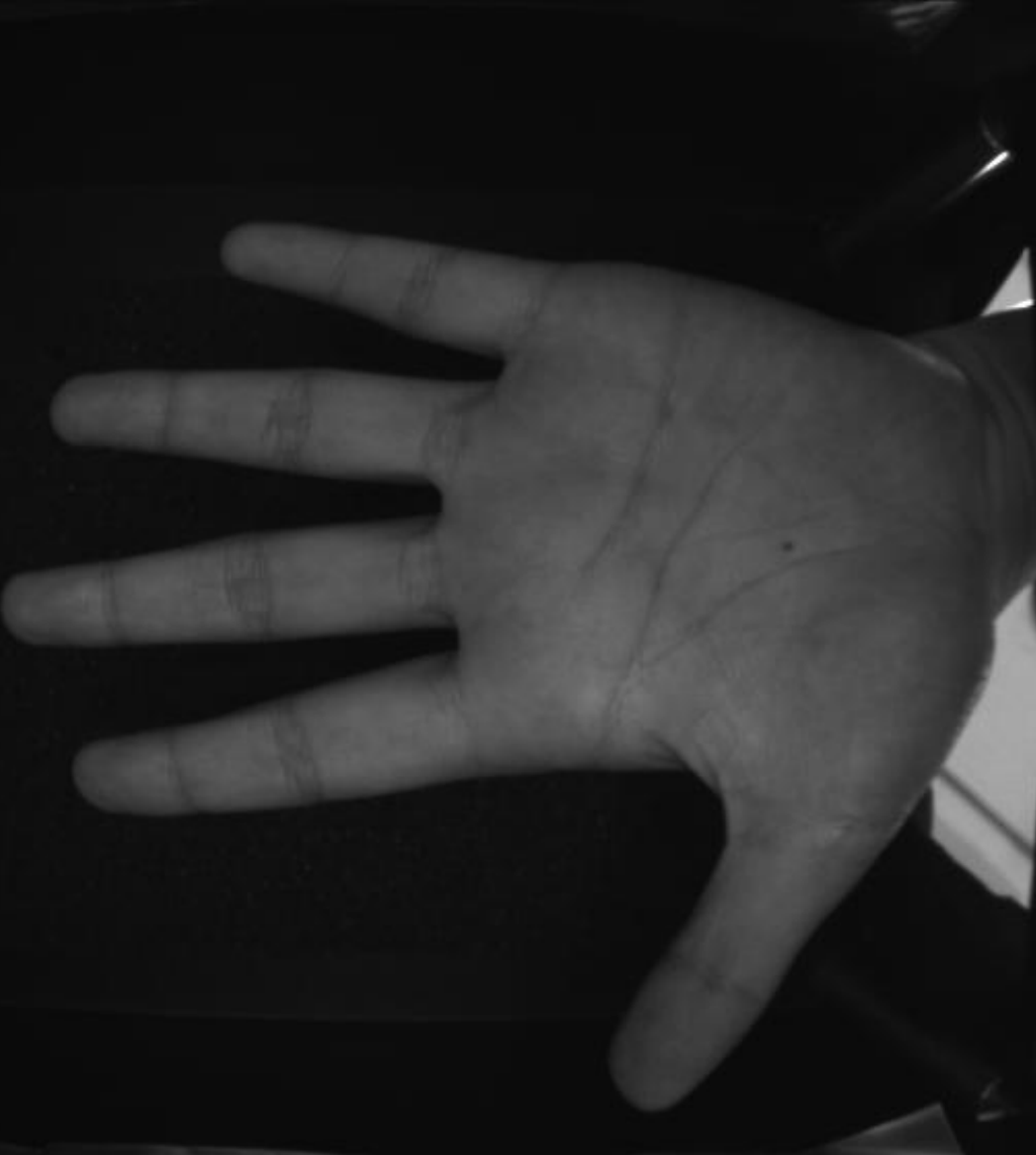}}\hspace{0.1pt}
\subfigure{\includegraphics[trim = 130pt 10pt 30pt 0pt, clip,width=0.235\linewidth]{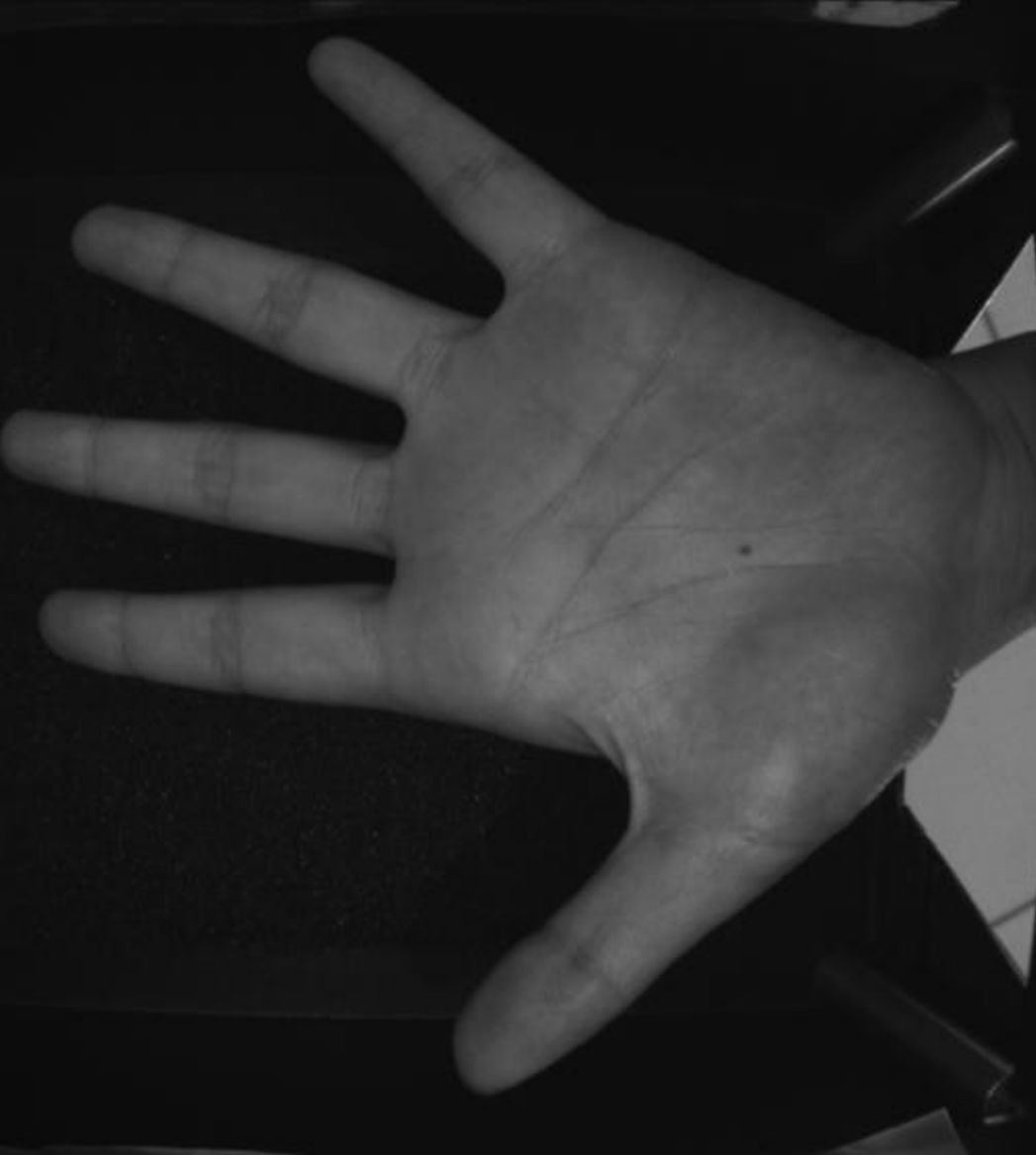}}\\
\subfigure{\includegraphics[width=0.235\linewidth]{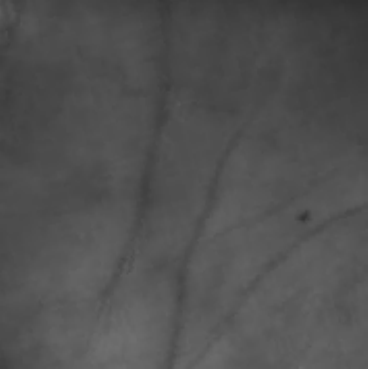}}\hspace{0.1pt}
\subfigure{\includegraphics[width=0.235\linewidth]{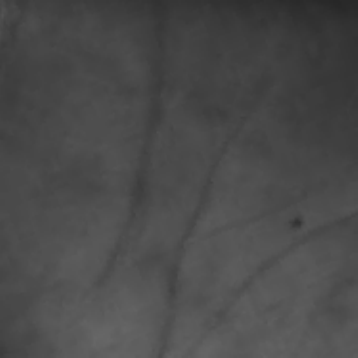}}\hspace{0.1pt}
\subfigure{\includegraphics[width=0.235\linewidth]{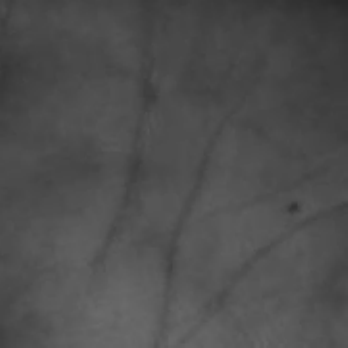}}\hspace{0.1pt}
\subfigure{\includegraphics[width=0.235\linewidth]{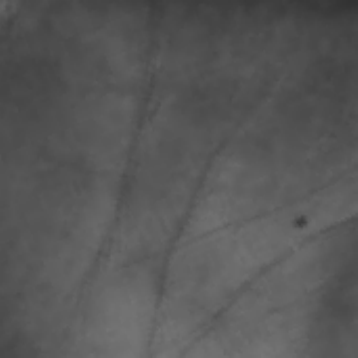}}
\end{minipage}
\caption[Extracted ROIs for a variety of hand movements]{Extracted ROIs for a variety of hand movements between different samples of a person. Despite the movements in the hand and fingers, the ROIs are consistently extracted from the same region.}
\label{fig:ROIs}
\end{figure}

ROI extraction from palm images acquired with non-contact sensors must overcome the challenges of RST variation. The proposed technique addresses these challenges except for excessive out-of-plane rotations or palm deformations. Figure~\ref{fig:fail} shows examples of images that resulted in a low match score due to out-of-plane rotation and excessive deformation of the second (probe) palm. We noted that these two anomalies are a major source of error in matching. In such cases, additional information is required to correct the out-of-plane rotation error and remove the non-rigid deformations. One solution is to use 3D information~\cite{li2010efficient,zhang2009palmprint}. However, the discussion of such techniques is out of the scope of this work.

\begin{figure}[!h]
\begin{center}
\begin{minipage}[b]{0.23\linewidth}
\includegraphics[trim = 64pt 0pt 0pt 0pt, clip, width=1\linewidth]{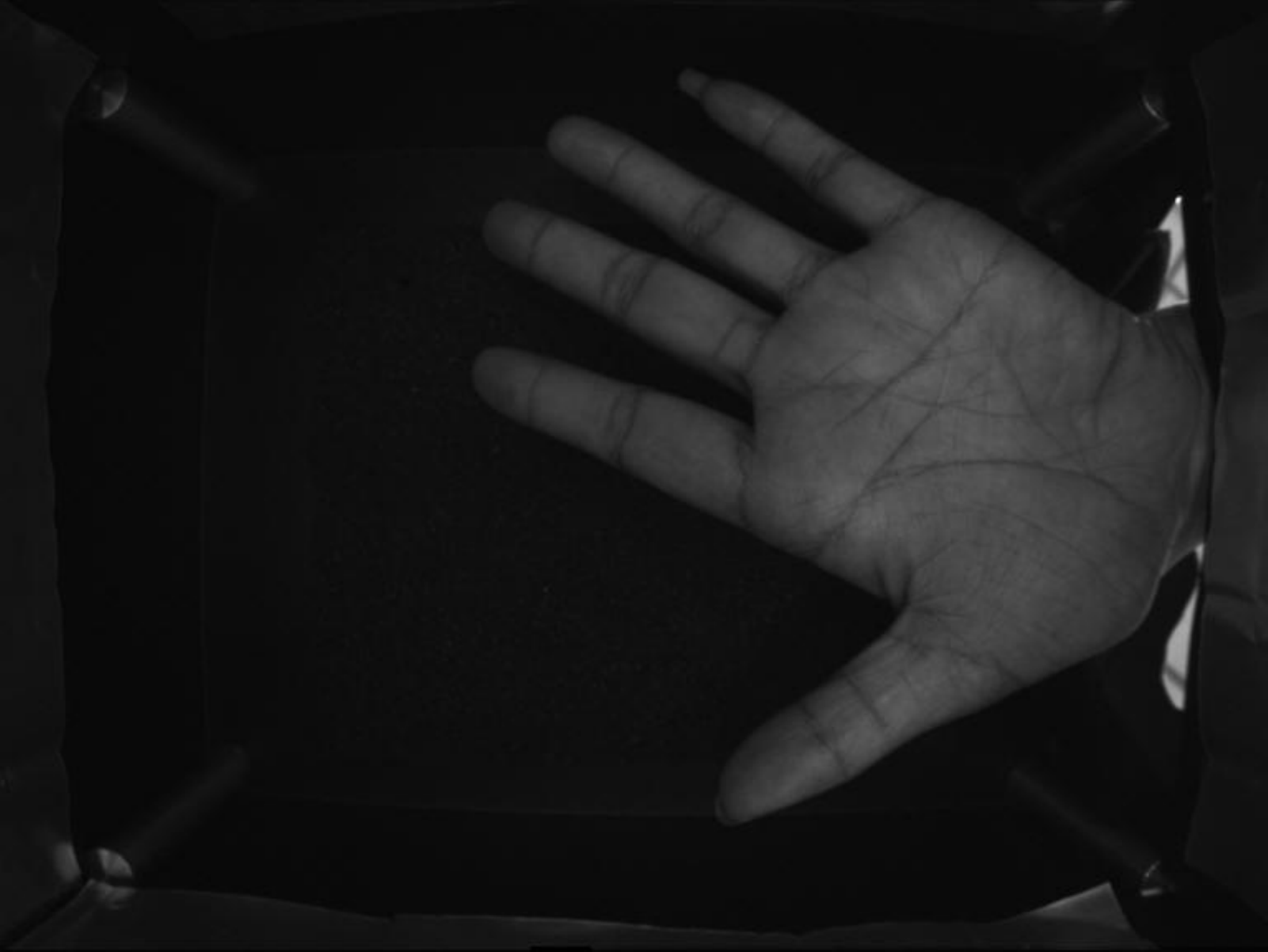}\\
\subfigure[]{\includegraphics[trim = 64pt 0pt 0pt 0pt, clip, width=1\linewidth]{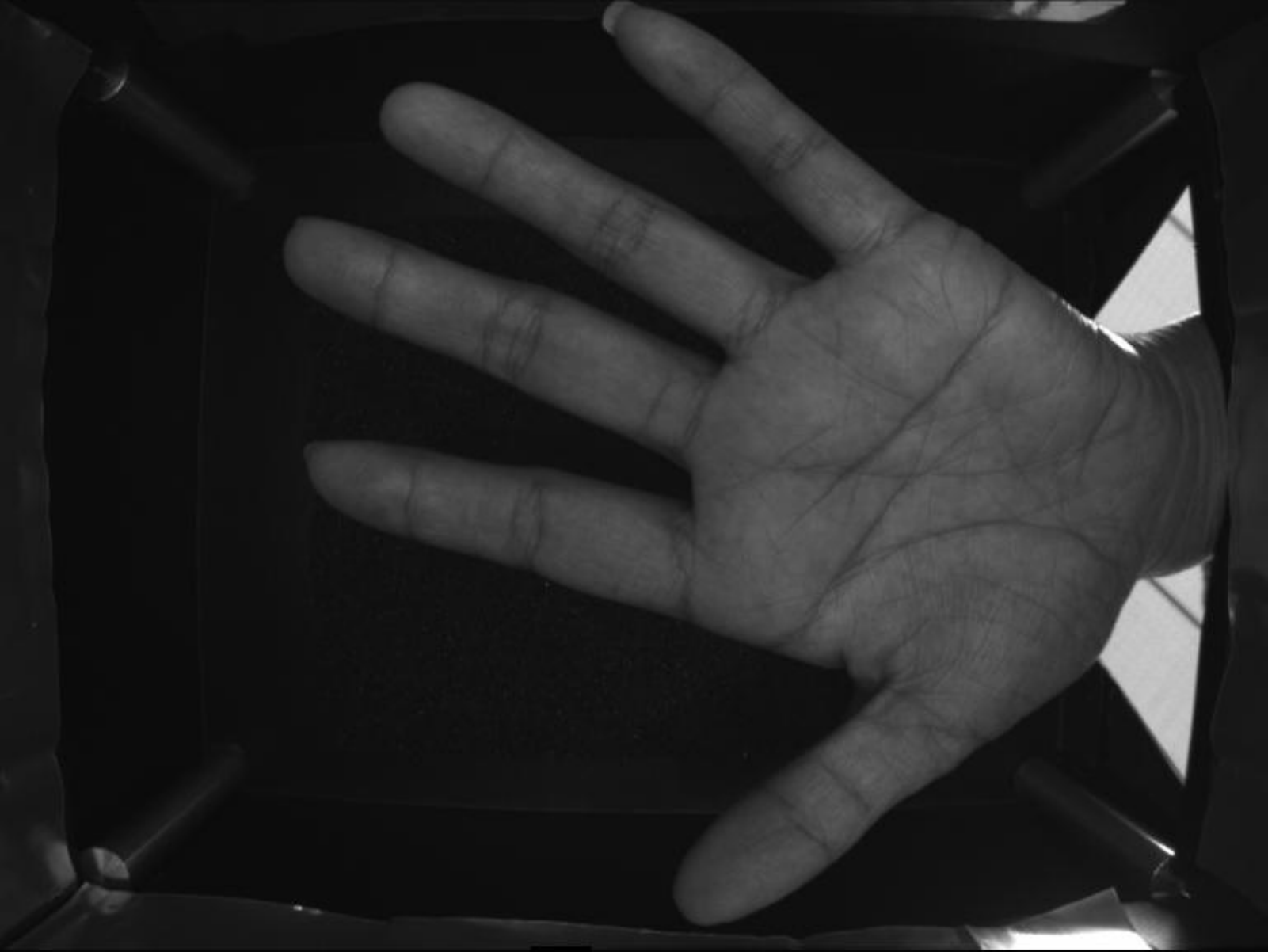}}
\end{minipage}\hspace{0.5pt}
\begin{minipage}[b]{0.23\linewidth}
\includegraphics[trim = 0pt 0pt 12pt 0pt, clip, width=1\linewidth]{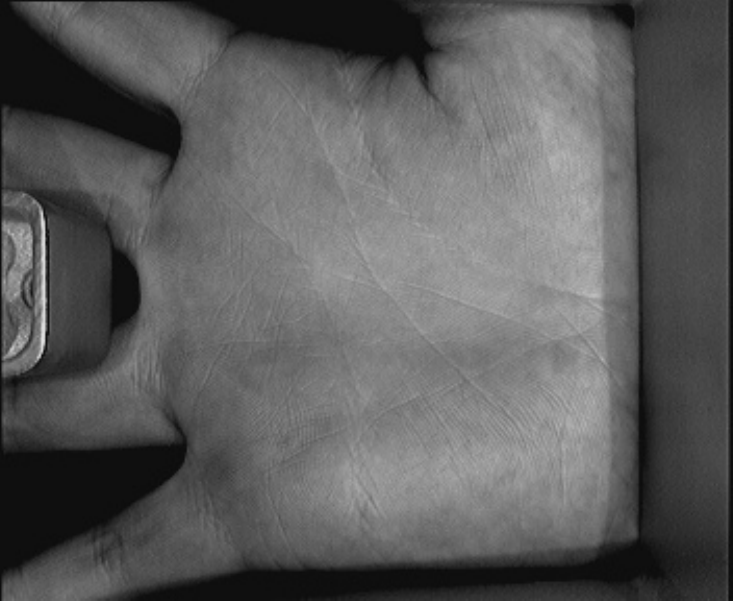}\\
\subfigure[]{\includegraphics[trim = 0pt 0pt 12pt 0pt, clip, width=1\linewidth]{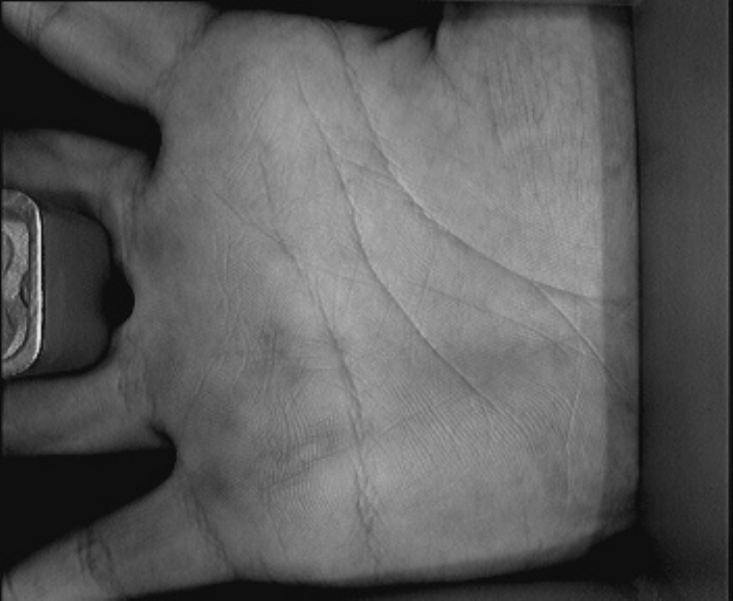}}
\end{minipage}
\end{center}
\caption[Examples of improper hand presentation]{Examples of improper hand presentation. (a) A palm correctly presented to a non-contact sensor (top), and out of plane rotated (bottom), resulting in incorrect estimation of the palm scale. (b) A palm correctly presented to a contact sensor (top), and deformed due to excessive pressure (bottom). This may result in a different ROI and with inconsistent spacing between the palm lines.}
\label{fig:fail}
\end{figure}

\subsection{Parameter Analysis}
\label{sec:params}
In this section, we examine the effects of various parameters including, ROI size, number of Contour Code orientations, hash table blurring neighborhood and pyramidal-directional filter pair combination. All other parameters are kept fixed during the analysis of a particular parameter. These experiments are performed on a sample subset (50\%) of the PolyU-MS database comprising equal proportion of images from the $1^{st}$ and $2^{nd}$ session. The same optimal parameters were used for the PolyU-HS and CASIA-MS database. For the purpose of analysis, we use the {ContCode-ATM} technique.

\begin{figure}[h]
\centering
\subfigure[]{\label{fig:params-dim}\includegraphics[trim = 0pt 0pt 10pt 5pt, clip, width=0.3\linewidth]{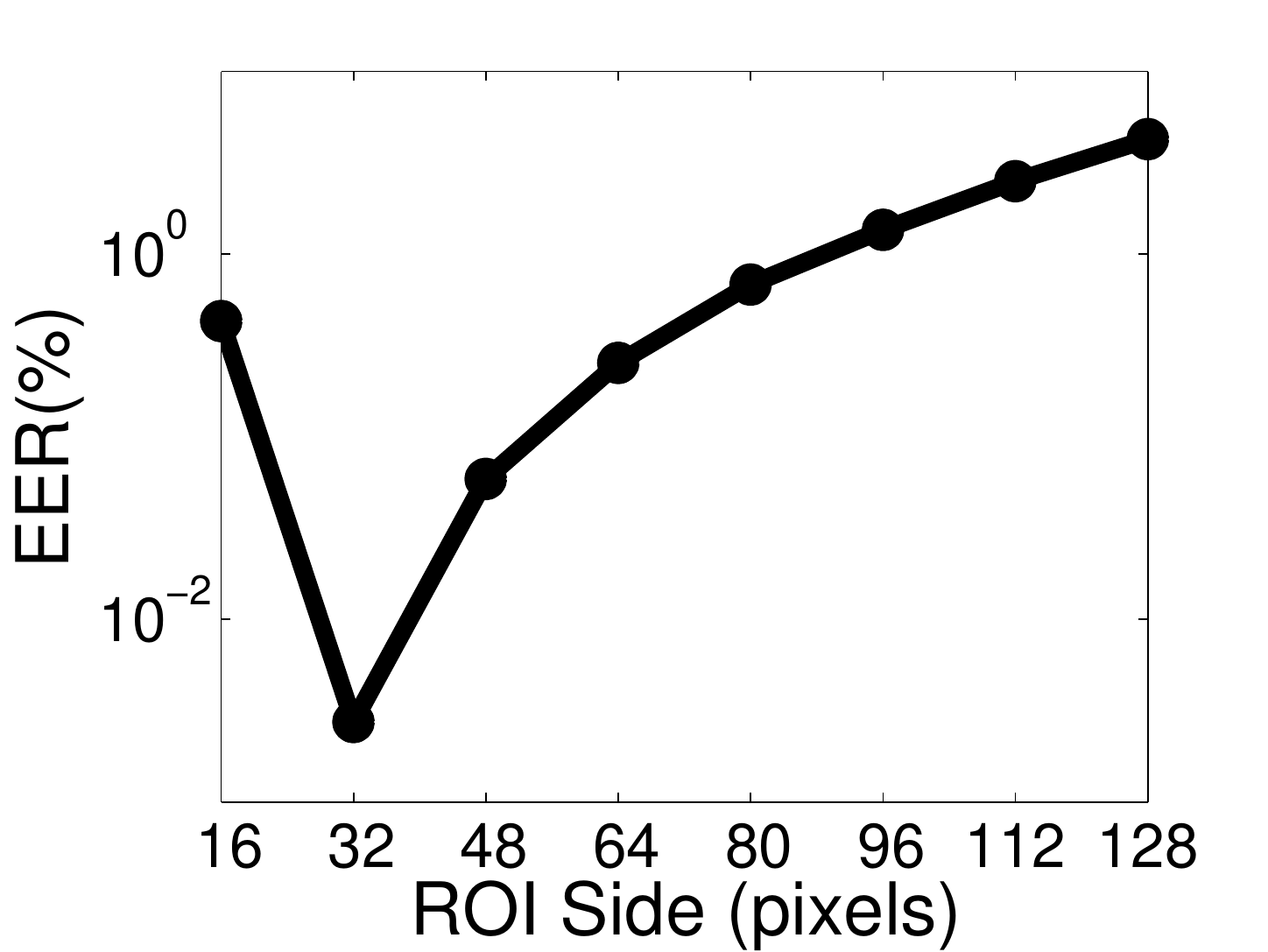}}
\subfigure[]{\label{fig:params-orient}\includegraphics[trim = 0pt 0pt 10pt 5pt, clip, width=0.3\linewidth]{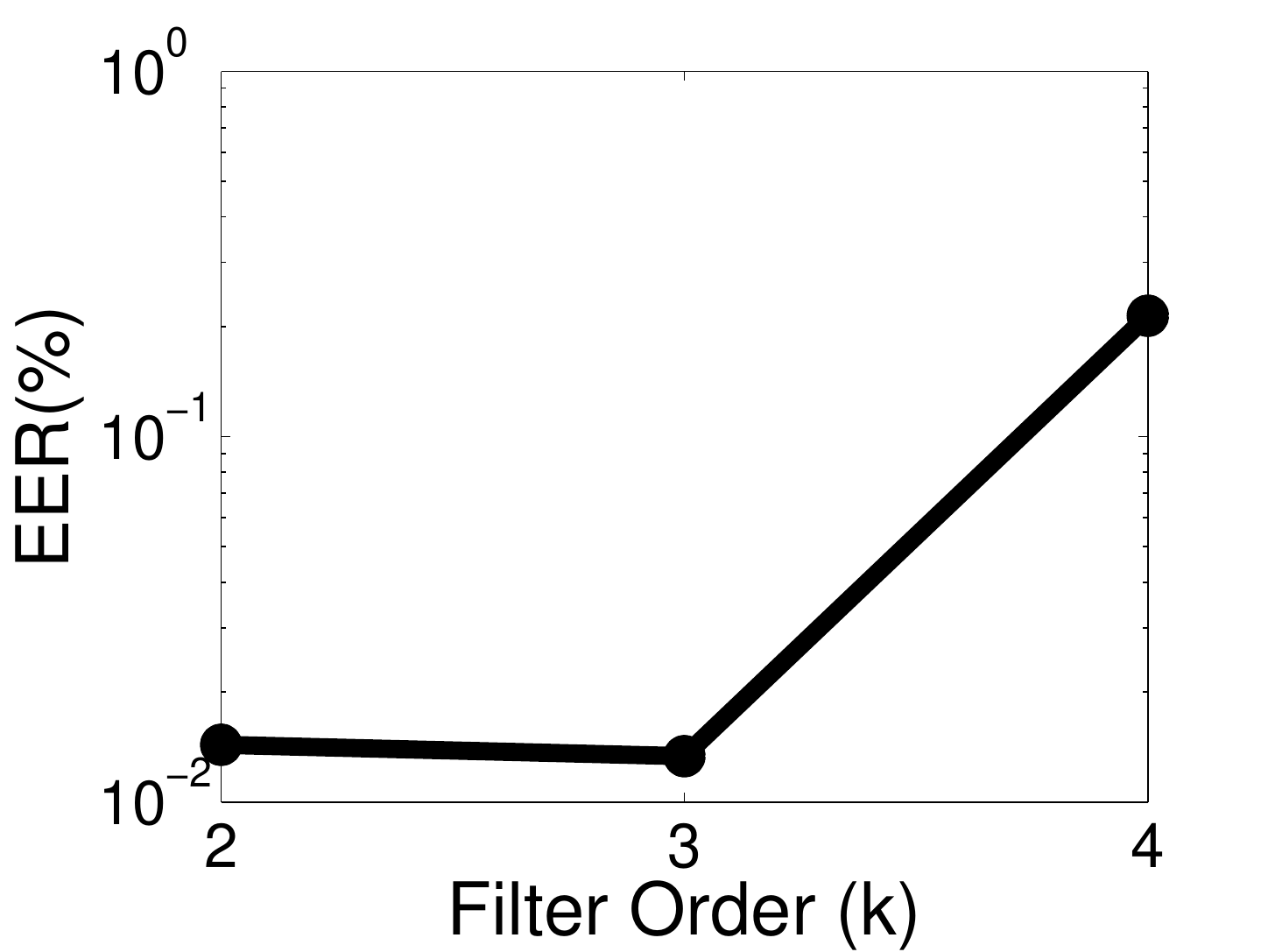}}
\subfigure[]{\label{fig:params-blur}\includegraphics[trim = 0pt 0pt 10pt 5pt, clip, width=0.3\linewidth]{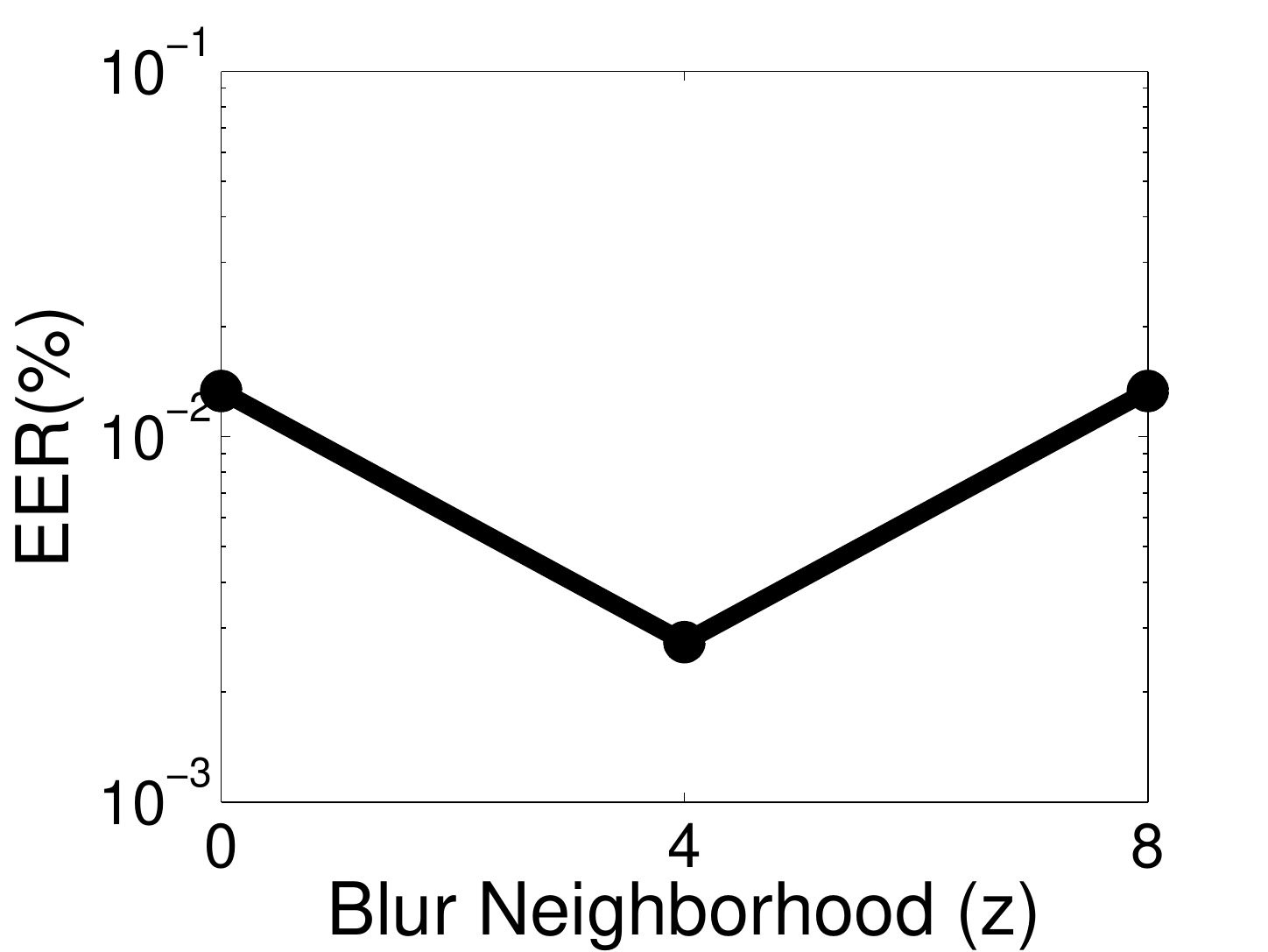}}
\caption[Analysis of parameters]{Analysis of parameters (a) ROI dimensions $(m,n)$ (b) Filter Order $(k)$ (c) Blur Neighbourhood}
\label{fig:params}
\end{figure}

\subsubsection{ROI Dimension}
Palmprint features have varying length and thickness. A larger ROI size may include unnecessary detail and unreliable features. A smaller ROI, on the other hand, may leave out important discriminative information. We empirically find the best ROI size by testing square $(m=n)$ ROI regions of side ${16,32,48,\ldots,128}$. The results in Figure~\ref{fig:params-dim} show that a minimum EER is achieved at 32 after which the EER increases. We use an ROI of $(m,n)=(32,32)$ in all our remaining experiments. Note that a peak performance at such a small ROI is in fact favorable for the efficiency of the Contour Code representation and subsequent matching.

\subsubsection{Orientation Fidelity}
The orientations of the dominant feature directions at points are quantized into a certain number of bins in the Contour Code. A greater number of bins offers better fidelity but also increases the size of the Contour Code representation and its sensitivity to noise. The minimum number of orientation bins that can achieve maximum accuracy is naturally a preferred choice. The orientations are quantized into $2^{k}$ bins, where $k$, the order of the directional filter determines the number of orientations. Figure~\ref{fig:params-orient} shows the results of our experiments for $k={2,3,4}$. We observe that $k=3$ (i.e.~$2^3=8$ directional bins) minimizes the EER. Although, there is a small improvement from $k=2$ to $k=3$, we prefer the latter as it provides more orientation fidelity which supports the process of hash table blurring. Therefore, we set $k=3$ in all our remaining experiments.

\subsubsection{Hash Table Blur Neighborhood}
Hash table blurring is performed to cater for small misalignments given that the palm is not a rigid object. However, too much blur can result in incorrect matches. We analyzed three different neighborhood types for blurring. The results are reported in Figure~\ref{fig:params-blur} which show that the lowest error rate is achieved with the \emph{4-connected} blur neighborhood. This neighborhood is, therefore, used in all the following experiments.

\subsubsection{Pyramidal-Directional Filter Pair}
\label{sec:filter}

An appropriate pyramidal and directional filter pair combination is critical to robust feature extraction. It is important to emphasize that the ability of a filter to capture robust line like features in a palm should consequently be reflected in the final recognition accuracy. Hence, we can regard the pyramidal-directional filter combination with the lowest EER as the most appropriate.

\begin{figure}[h]
\centering
\subfigure{\includegraphics[trim = 60pt 10pt 60pt 5pt, clip,width=0.19\linewidth]{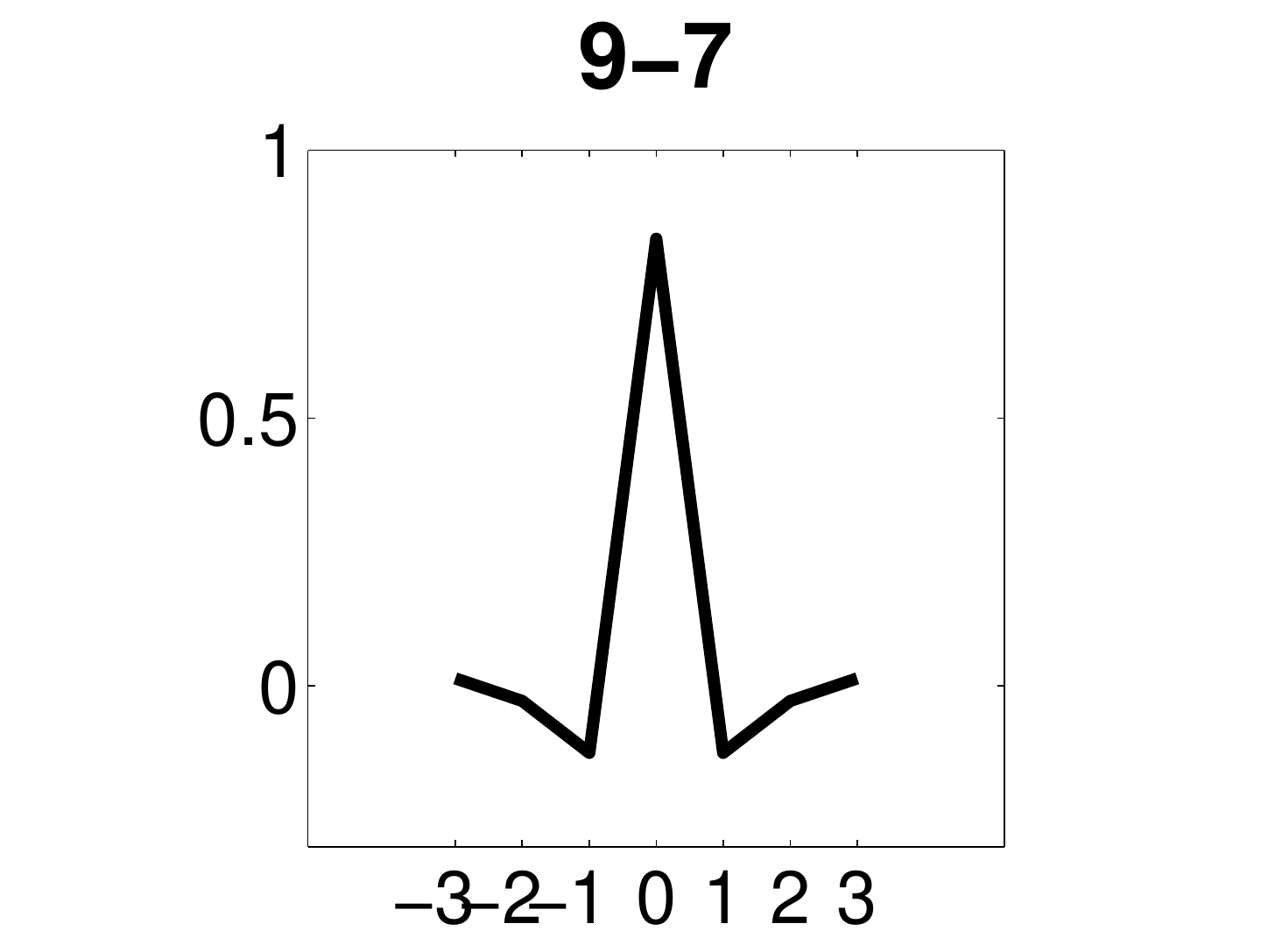}}
\subfigure{\includegraphics[trim = 60pt 10pt 60pt 5pt, clip,width=0.19\linewidth]{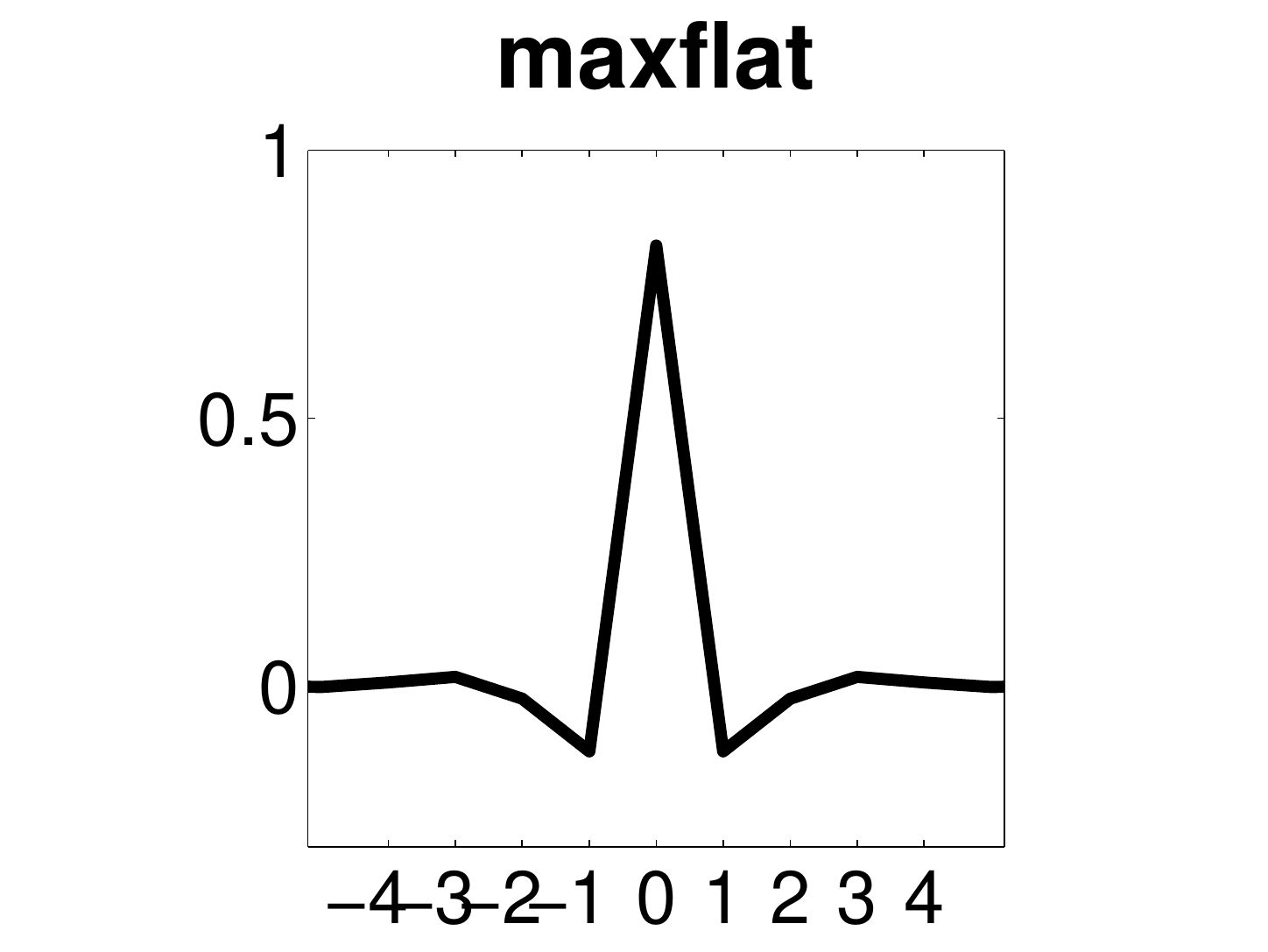}}
\subfigure{\includegraphics[trim = 60pt 10pt 60pt 5pt, clip,width=0.19\linewidth]{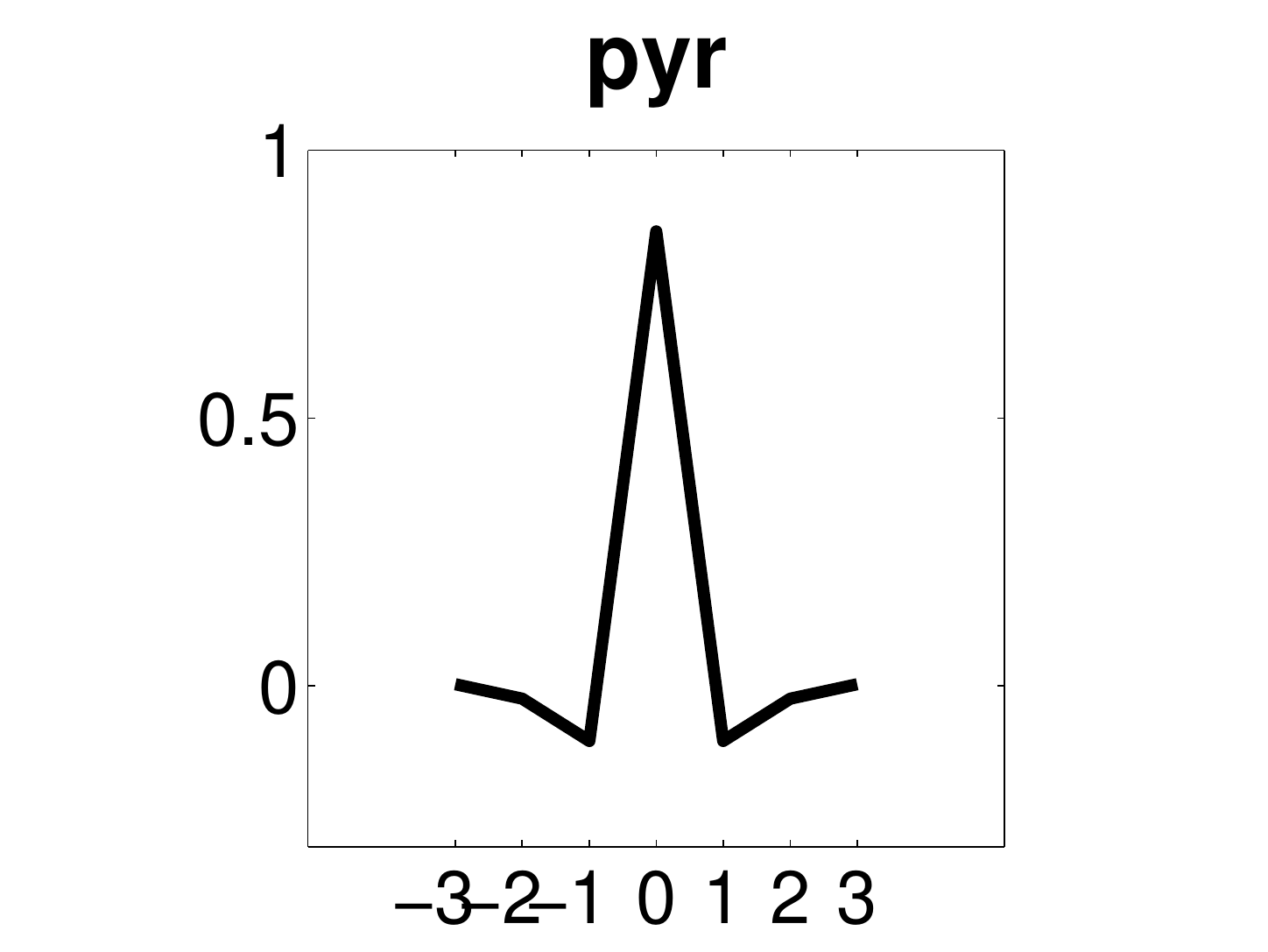}}
\subfigure{\includegraphics[trim = 60pt 10pt 60pt 5pt, clip,width=0.19\linewidth]{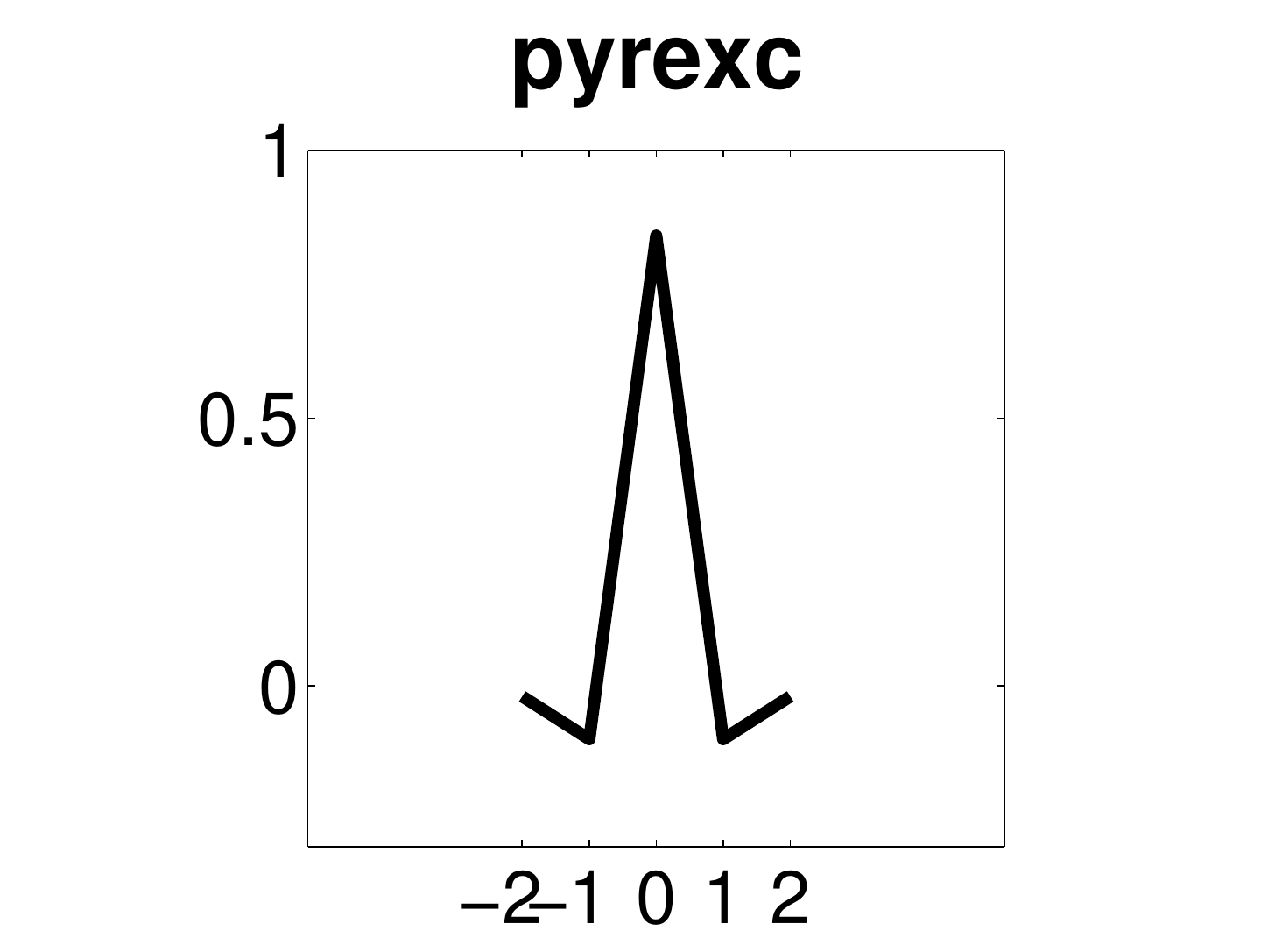}}
\caption[NSCT pyramidal highpass filters]{NSCT pyramidal highpass filters~\cite{da2008geometrical}. Planar profile of the 2D filters are shown for better visual comparison. (a) Filters from 9-7 1-D prototypes. (b) Filters derived from 1-D using maximally flat mapping function with 4 vanishing moments. (c) Filters derived from 1-D using maximally flat mapping function with 2 vanishing moments. (d) Similar to pyr but exchanging two highpass filters.}
\label{fig:pfilters}
\end{figure}
\begin{figure}[h]
\centering
\subfigure{\includegraphics[trim = 60pt 10pt 60pt 5pt, clip,width=0.19\linewidth]{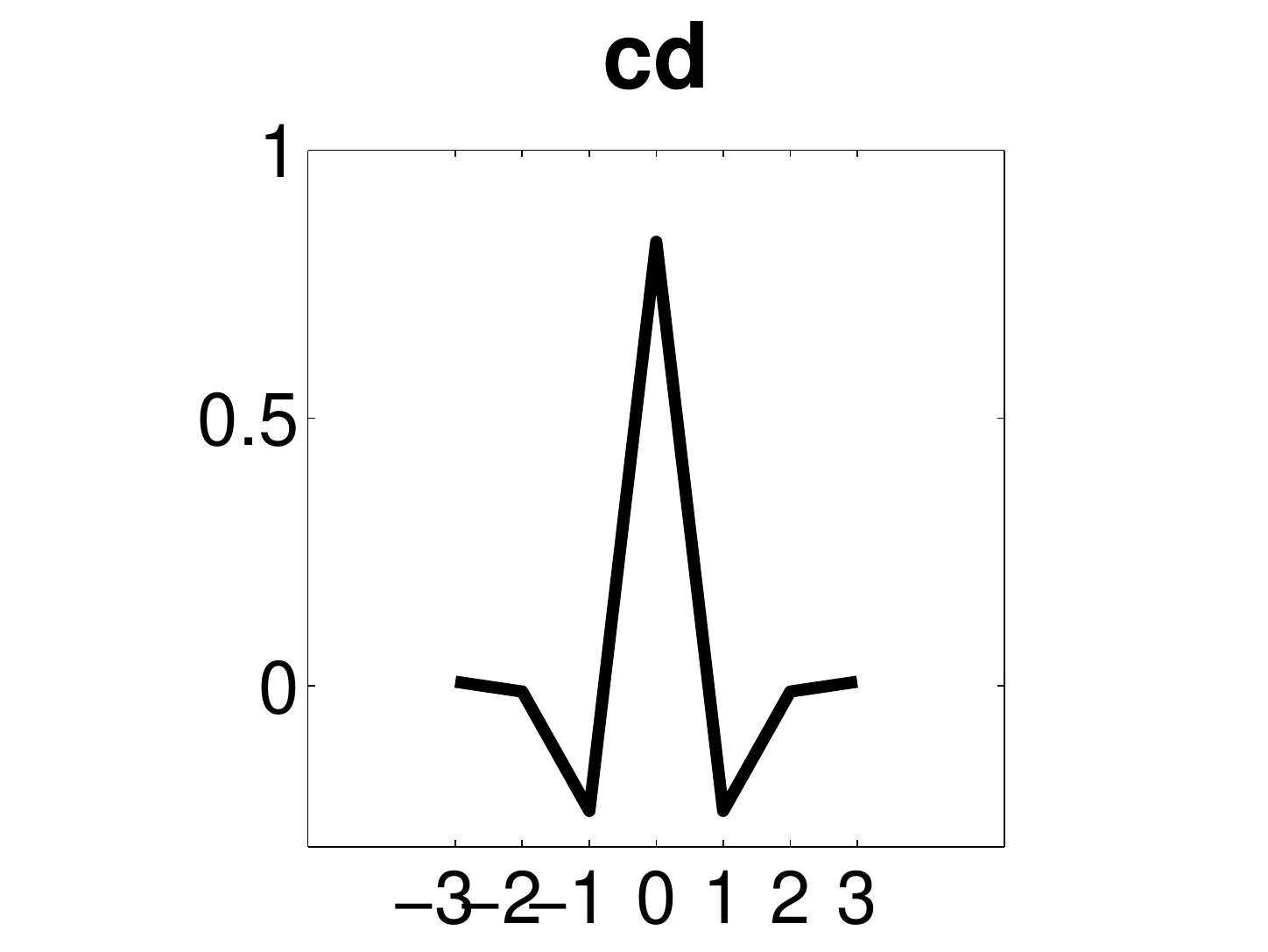}}
\subfigure{\includegraphics[trim = 60pt 10pt 60pt 5pt, clip,width=0.19\linewidth]{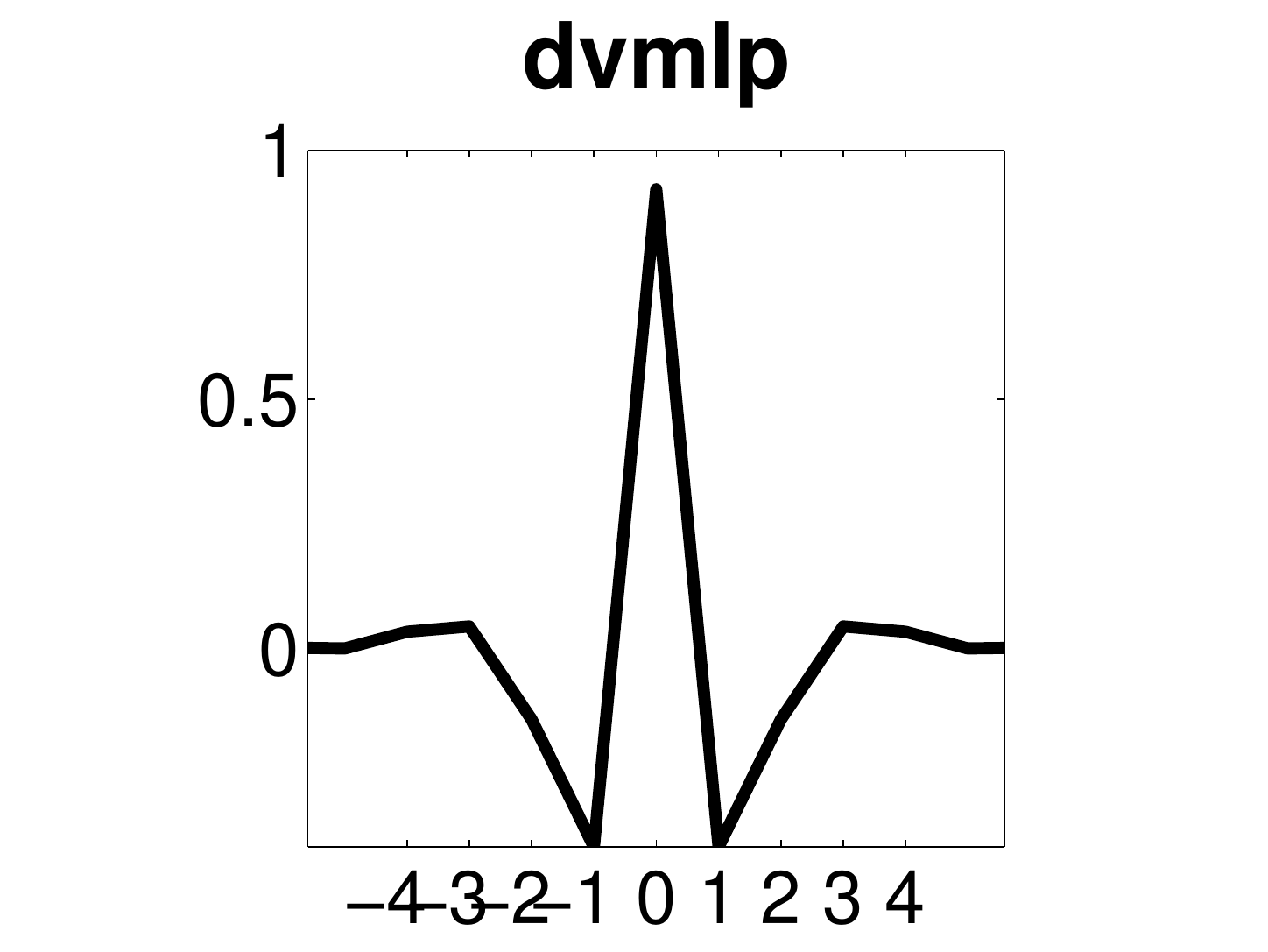}}
\subfigure{\includegraphics[trim = 60pt 10pt 60pt 5pt, clip,width=0.19\linewidth]{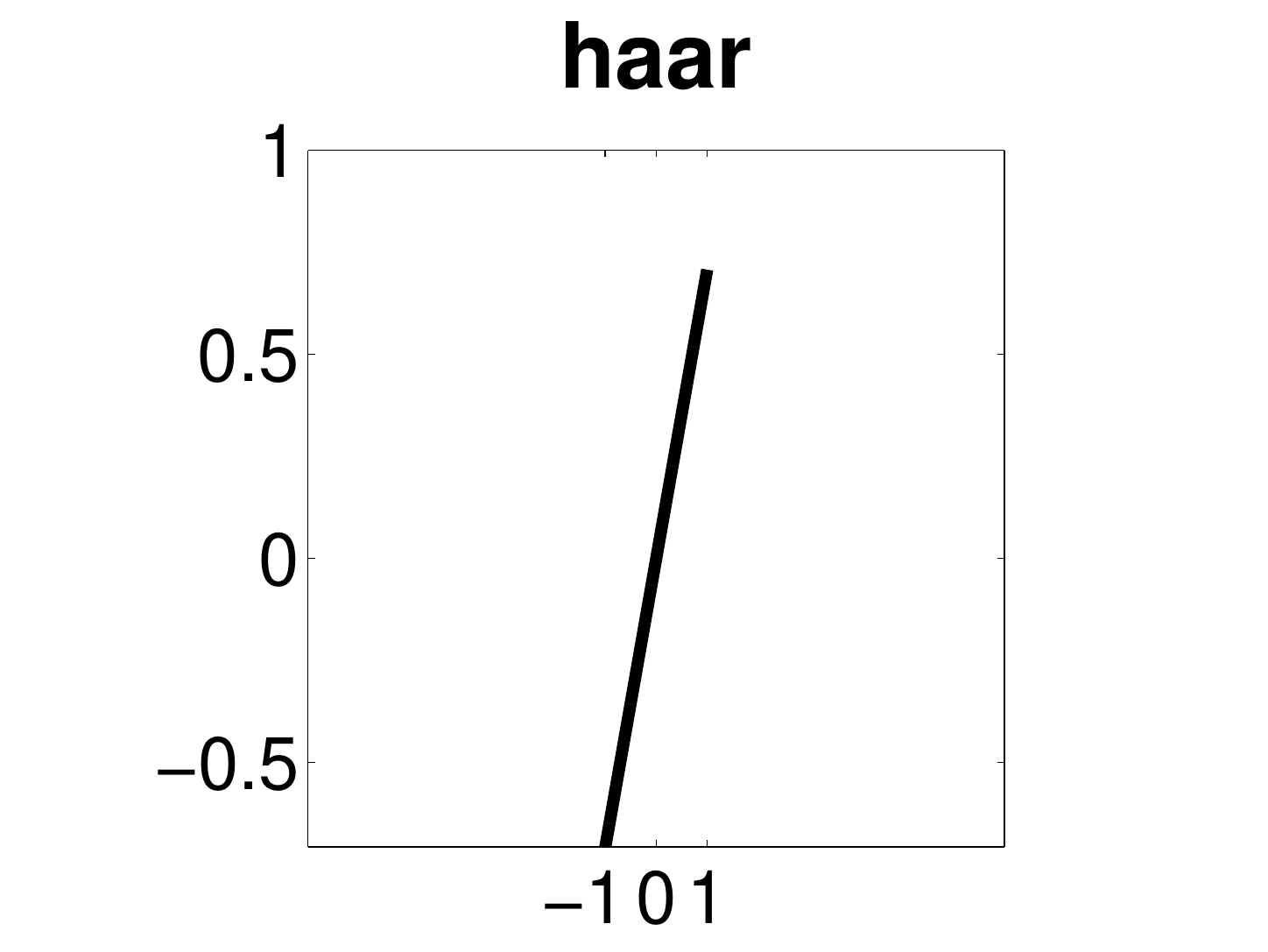}}
\subfigure{\includegraphics[trim = 60pt 10pt 60pt 5pt, clip,width=0.19\linewidth]{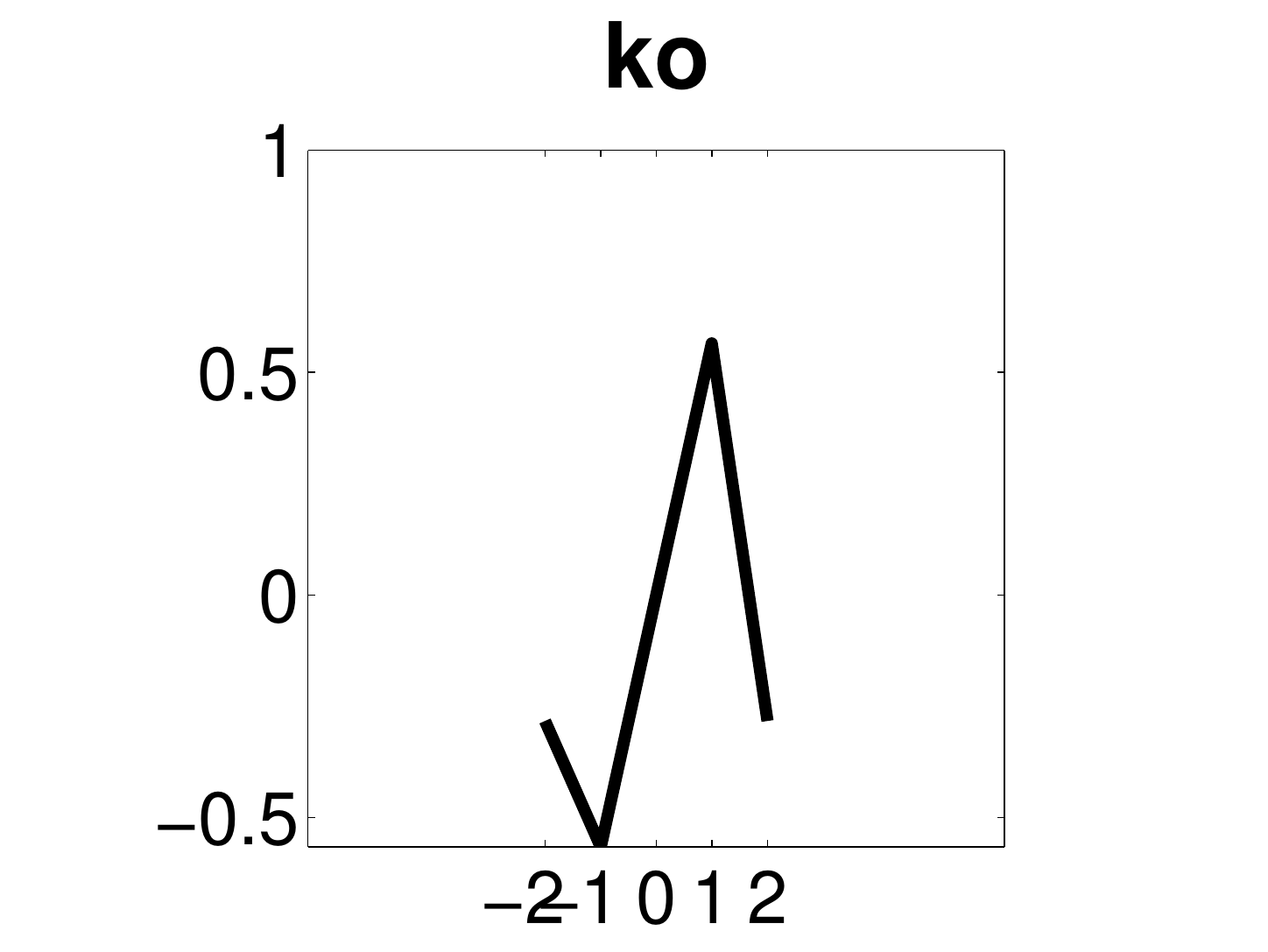}}
\subfigure{\includegraphics[trim = 60pt 10pt 60pt 5pt, clip,width=0.19\linewidth]{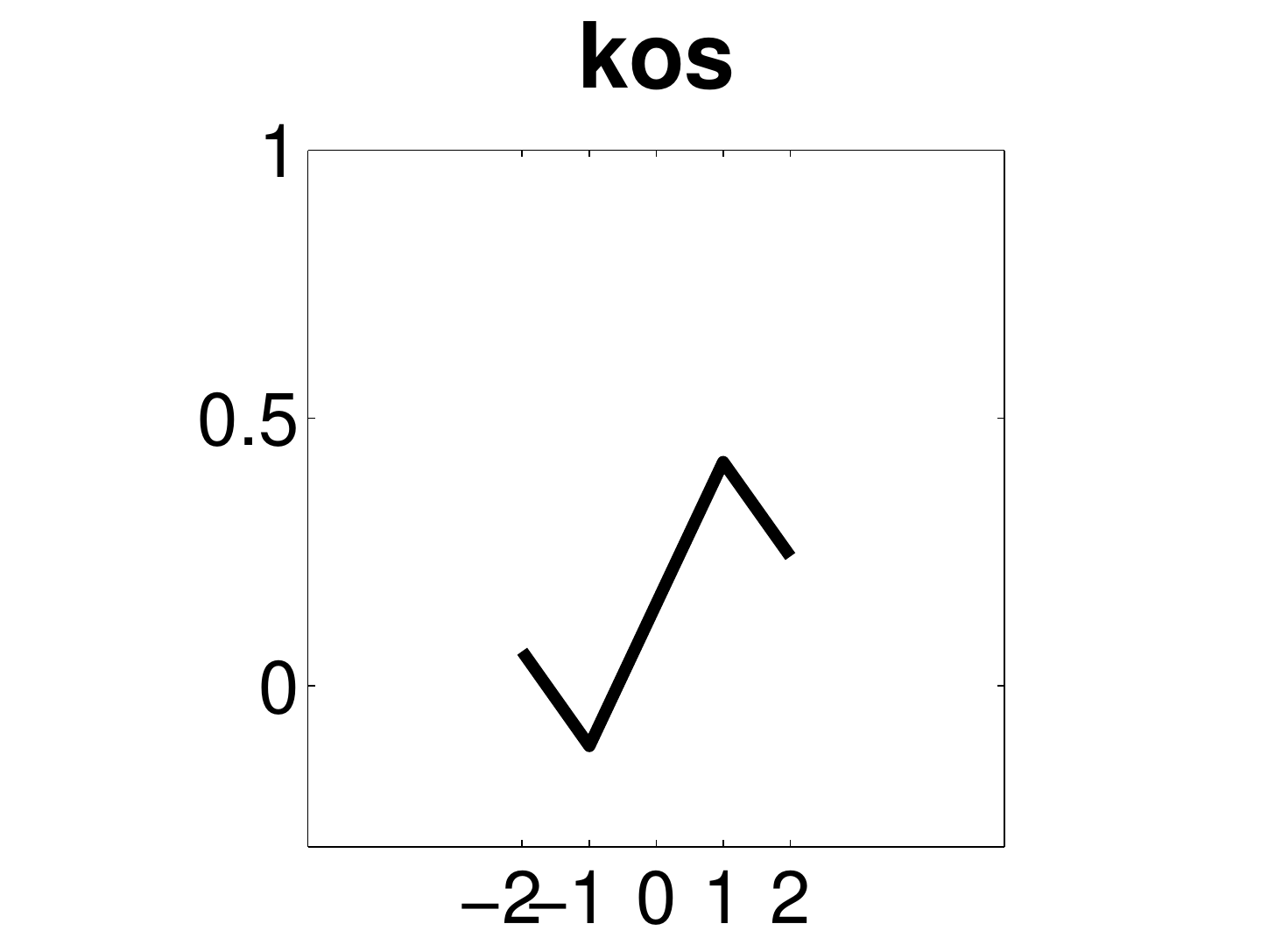}}\\
\subfigure{\includegraphics[trim = 60pt 10pt 60pt 5pt, clip,width=0.19\linewidth]{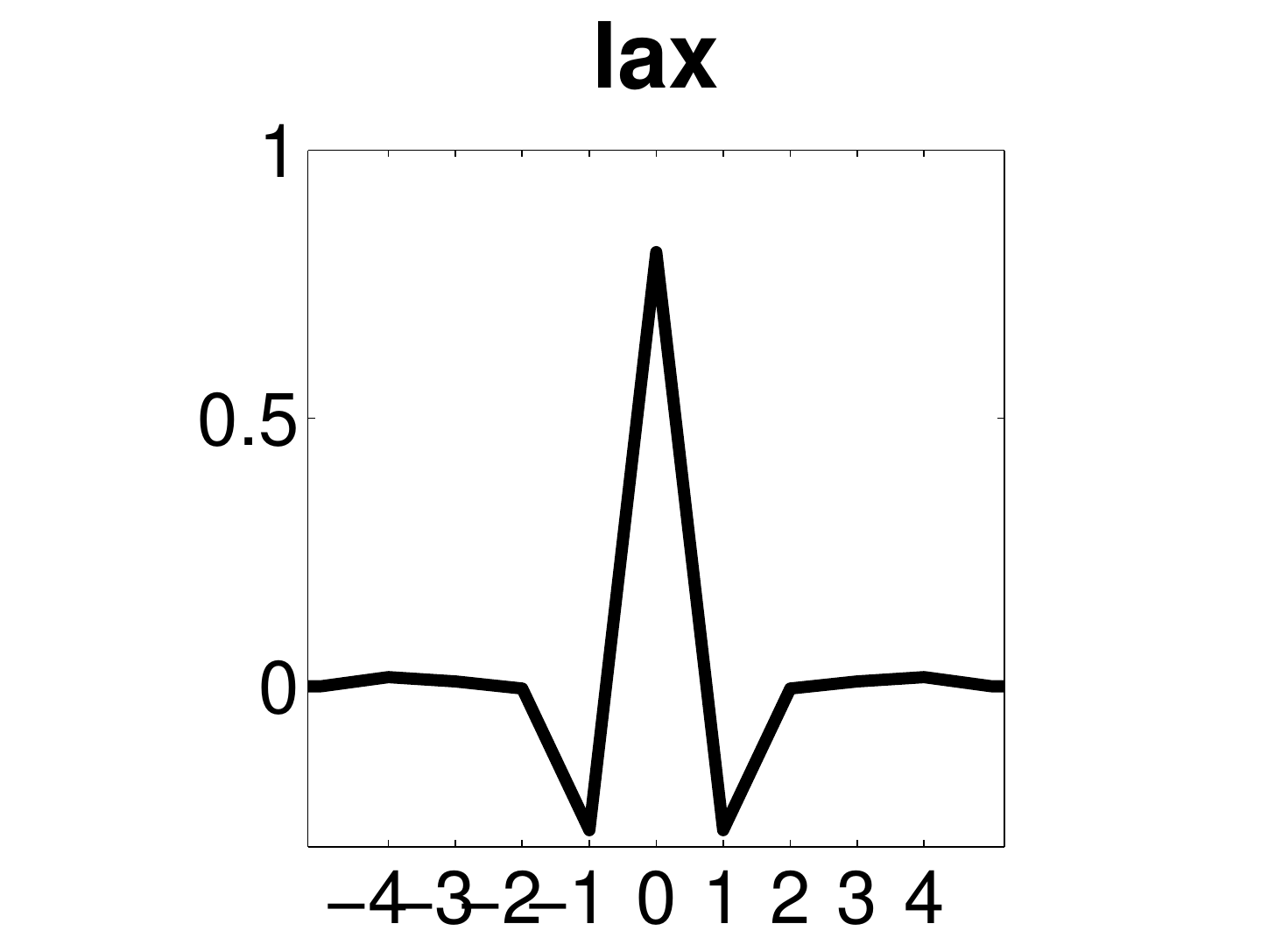}}
\subfigure{\includegraphics[trim = 60pt 10pt 60pt 5pt, clip,width=0.19\linewidth]{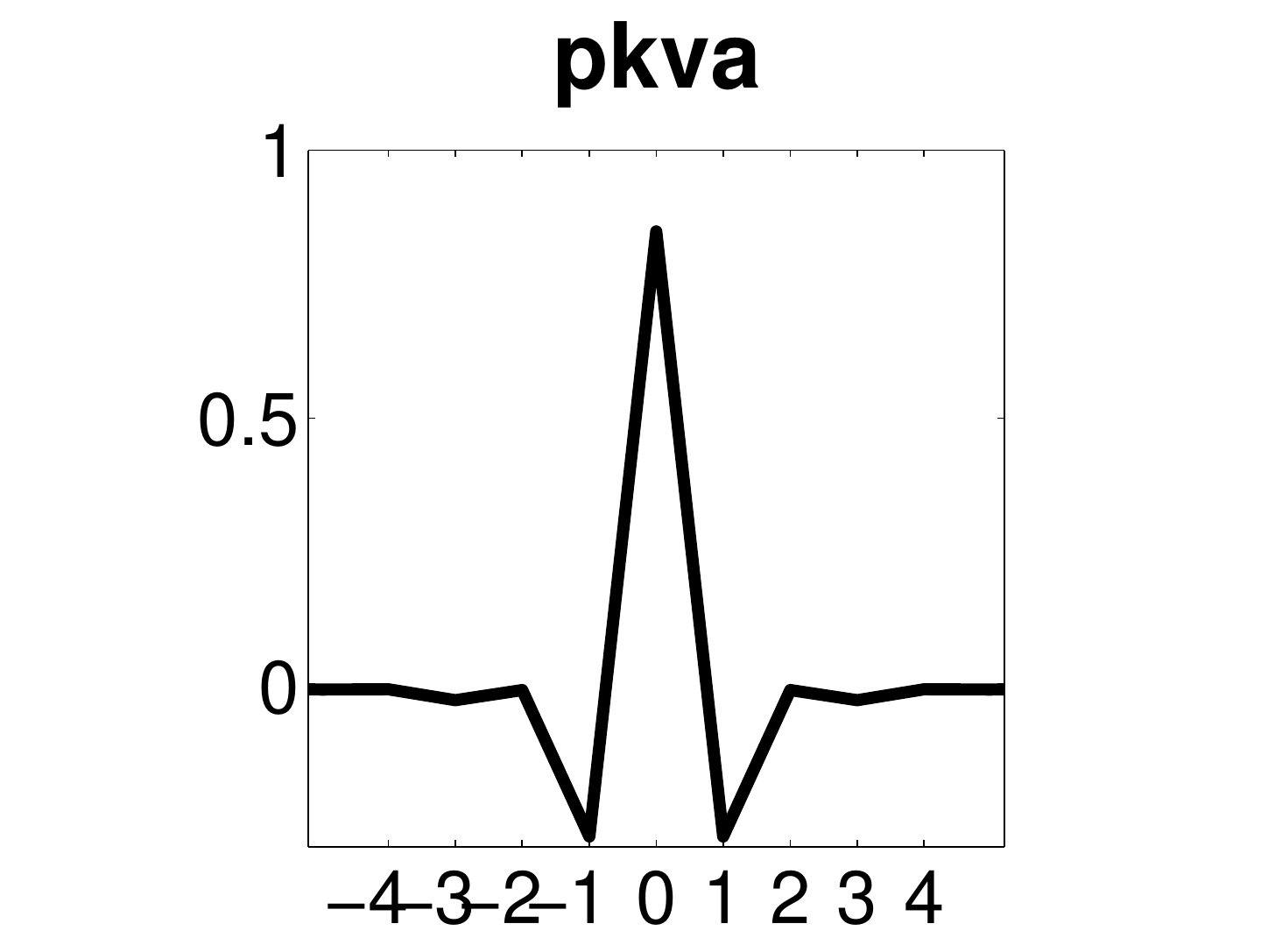}}
\subfigure{\includegraphics[trim = 60pt 10pt 60pt 5pt, clip,width=0.19\linewidth]{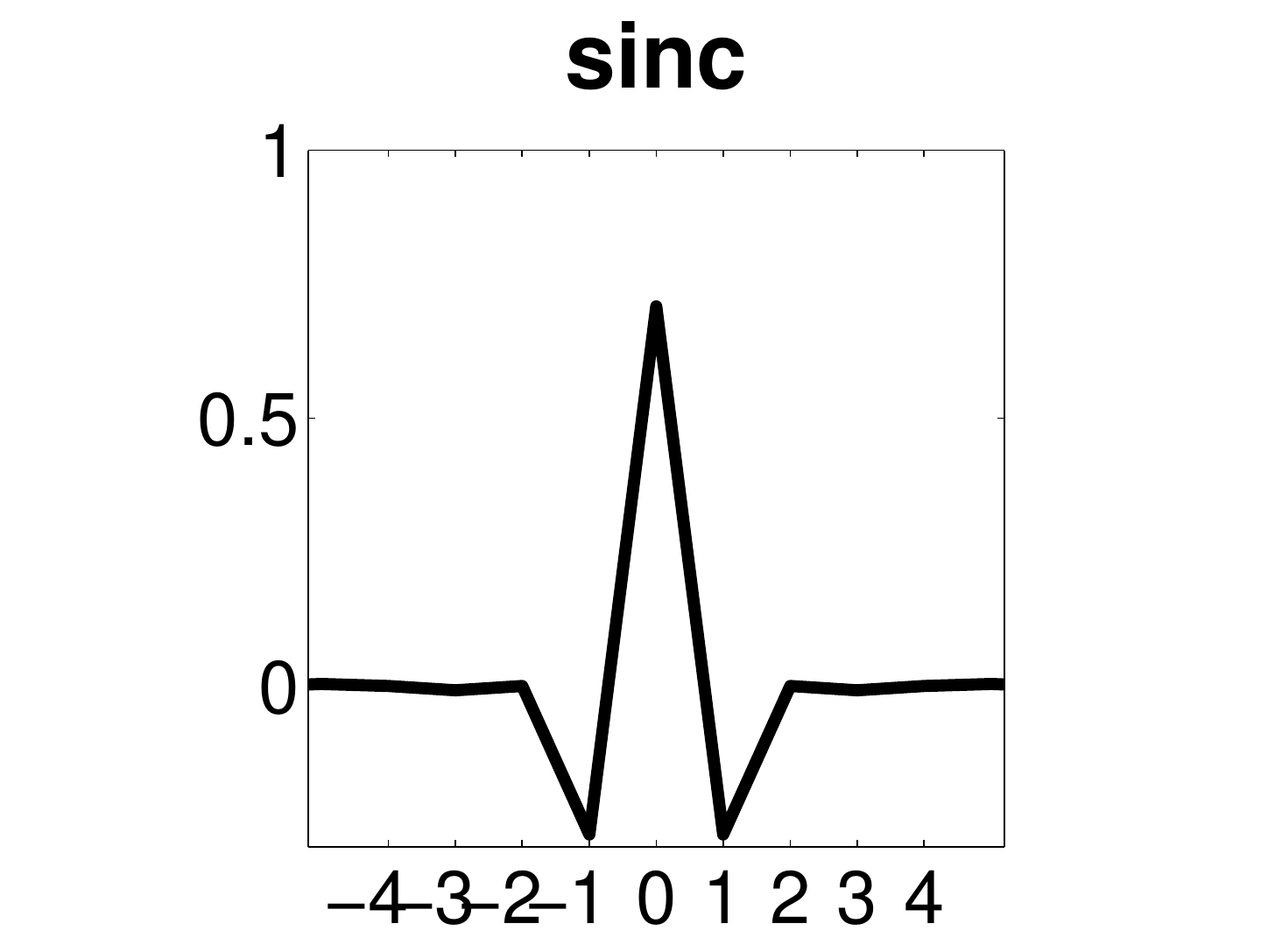}}
\subfigure{\includegraphics[trim = 60pt 10pt 60pt 5pt, clip,width=0.19\linewidth]{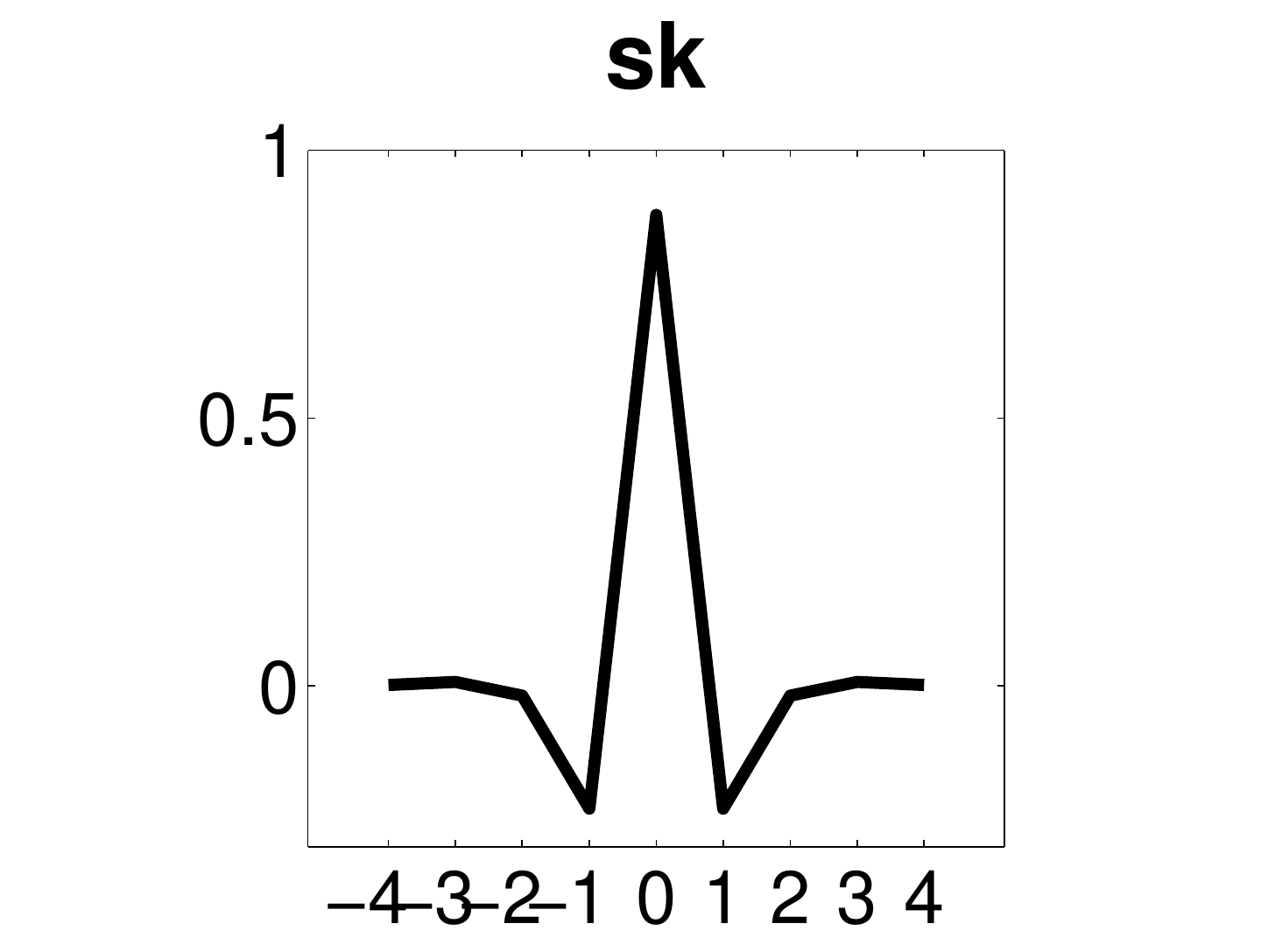}}
\subfigure{\includegraphics[trim = 60pt 10pt 60pt 5pt, clip,width=0.19\linewidth]{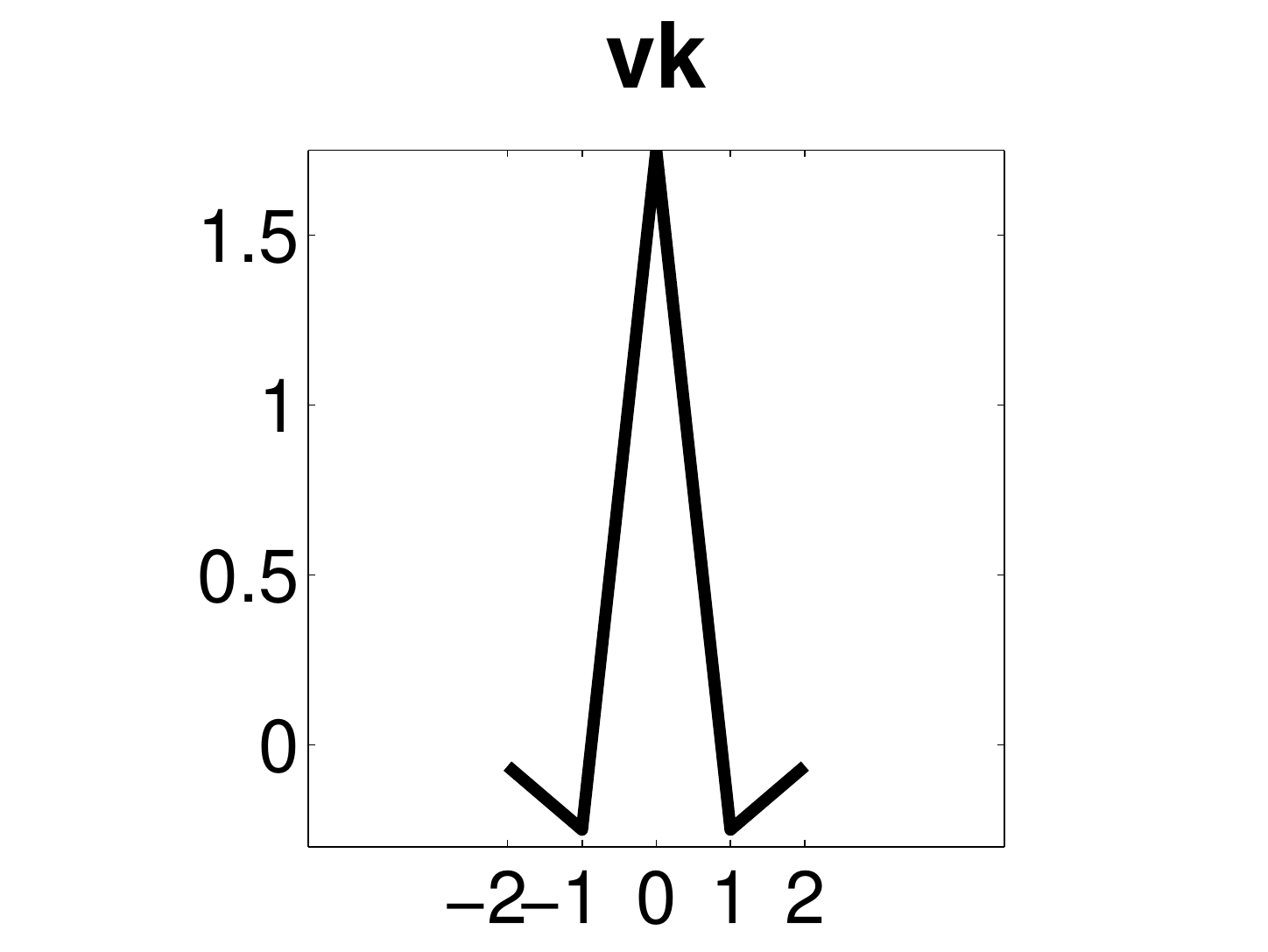}}
\caption[NSCT directional highpass filters]{NSCT directional highpass filters. (a) cd: 7 and 9 McClellan transformed by Cohen and Daubechies~\cite{cohen1992biorthogonal}. (b) dvmlp: regular linear phase biorthogonal filter with 3 dvm~\cite{da2005bi}. (c) haar: the `Haar' filter~\cite{haar1910theorie}. (d) ko: orthogonal filter from Kovacevic (e) kos: smooth `ko' filter. (f) lax: $17\times17$ by Lu, Antoniou and Xu~\cite{lu1998direct}. (g) pkva: ladder filters by Phong et al.~\cite{phoong1995new}. (h) sinc: ideal filter. (i) sk: $9\times9$ by Shah and Kalker~\cite{shah1994ladder}. (j) vk: McClellan transform of filter from the VK book~\cite{vetterli1995wavelets}.}
\label{fig:dfilters}
\end{figure}

We tested all possible pairs of $4$ pyramidal and $10$ directional filters available in the Nonsubsampled Contourlet Toolbox \footnote{The Nonsubsampled Contourlet Toolbox \url{http://www.mathworks.com/matlabcentral/fileexchange/10049}}. These filters exhibit characteristic response to line-like features. The pyramidal filters used at the first stage are shown in Figure~\ref{fig:pfilters} and the directional filters used at the second stage are shown in Figure~\ref{fig:dfilters}. The EER is normalized by dividing with the maximum EER in all filter combinations. As shown in Figure~\ref{fig:Filt_PolyU_MS}, the \emph{pyrexc} filter consistently achieves lower EER compared to the other three pyramidal decomposition filters for most combinations of the directional filters. Moreover, the \emph{sinc} filter exhibits the lowest EER which shows its best directional feature capturing characteristics over a broad spectral range. Based on these results, we select the \emph{pyrexc-sinc} filter combination for Contour Code representation in all experiments. Close inspection of the pyramidal filters shows that {\em pyr} and {\em pyrexc} are quite similar to each other. Similarly, {\em sinc}, {\em pkva} and {\em sk} directional filters have approximately the same shape. Therefore, it is not surprising to see that all combinations of these two pyramidal filters with the three directional filters give significantly better performance compared to the remaining combinations.

\begin{landscape}
\begin{figure}[h]
\centering
\includegraphics[trim = 20pt 100pt 38pt 80pt, clip,width=0.8\linewidth]{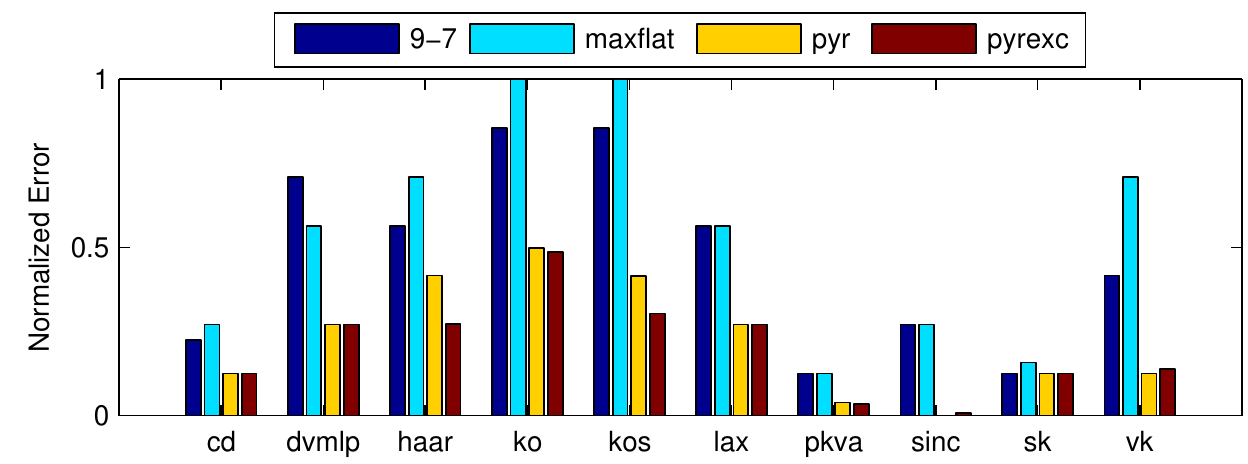}
\caption[Normalized EERs for various pyramidal-directional filter pairs]{Normalized EERs for various pyramidal-directional filter pairs. The \emph{9-7} and \emph{maxflat} pyramidal filters can be safely ruled out due to their relatively poor performance overall. The \emph{pyrexc} filter performs equal or better than the \emph{pyr} filter for atleast 8 out of the 10 combinations with the directional filters. Thus, the \emph{pyrexc} filter is an obvious choice for the pyramidal filter stage. For the directional filter stage, the \emph{sinc} filter clearly outperforms all other directional filters, particularly in combination with the \emph{pyrexc} pyramidal filter. From this analysis, the \emph{sinc-pyrexc} filter pair is used in the Contour Code representation.}
\label{fig:Filt_PolyU_MS}
\end{figure}
\end{landscape}

\subsection{Verification Experiments}
\label{sec:roc}

Verification experiments are performed on PolyU-MS, PolyU-HS and CASIA-MS databases. In all cases, we follow the protocol of~\cite{zhang2012comparative}, where session based experiments are structured to observe the recognition performance. Our evaluation comprises five verification experiments to test the proposed technique. The experiments proceed by matching

\noindent
\emph{Exp.1:} individual bands of palm irrespective of the session (all vs.~all).\\
\emph{Exp.2:} multispectral palmprints acquired in the $1^{\textrm{st}}$ session.\\
\emph{Exp.3:} multispectral palmprints acquired in the $2^{\textrm{nd}}$ session.\\
\emph{Exp.4:} multispectral palmprints of the $1^{\textrm{st}}$ session to the $2^{\textrm{nd}}$ session.\\
\emph{Exp.5:} multispectral palmprints irrespective of the session (all vs.~all).

%
%

In all cases, we report the ROC curves, which depict False Rejection Rate (FRR) versus the False Acceptance Rate (FAR). We also summarize the Equal Error Rate (EER), and the Genuine Acceptance Rate (GAR) at 0.1\% FAR and compare performance of the proposed Contour Code with the CompCode~\cite{kong2004competitive}, OrdCode~\cite{sun2005ordinal} and DoGCode~\cite{wu2006palmprint}. Note that we used our implementation of these methods as their code is not publicly available. Unless otherwise stated, we use a \emph{4-connected} blur neighborhood for gallery hash table encoding and the matching is performed in ATM mode followed by score-level fusion of bands.

%
%

\subsubsection{Experiment 1}

compares the relative discriminant capability of individual bands. We compare the performance of individual bands of all databases using {ContCode-ATM}. Figure~\ref{fig:roc_Exp1} shows the ROC curves of the individual bands and Table~\ref{tab:res_Exp1} lists their EERs. In the PolyU-MS database, the 660nm band gives the best performance indicating the presence of more discriminatory features. A logical explanation could be that the 660nm wavelength partially captures both the line and vein features making this band relatively more discriminative. In the PolyU-HS database, the 820nm and 920nm have the lowest errors followed by 770nm and 900nm. In CASIA-MS database, the most discriminant information is present in the 460nm, 630nm and 940nm bands which are close competitors.


\subsubsection{Experiment 2}

analyzes the variability in the palmprint data acquired in the $1^{\textrm{st}}$ session. Figure~\ref{fig:roc_Exp2} compares the ROC curves of {ContCode-ATM} with three other techniques on the all databases. It is observable that the CompCode and the OrdCode show intermediate performance close to {ContCode-ATM}. The DoGCode exhibits a drastic degradation of accuracy implying its inability to sufficiently cope with the variations of PolyU-HS and CASIA-MS data. Overall, the CompCode and ContCode-ATM perform better on all databases while the latter performs the best.

\begin{landscape}
\begin{figure}[t]
\centering
\begin{minipage}[b]{1\linewidth}
\subfigure{\label{fig:roc_PolyU_MS_Exp1}\includegraphics[trim = 0pt 0pt 10pt 0pt, clip, width=0.32\linewidth]{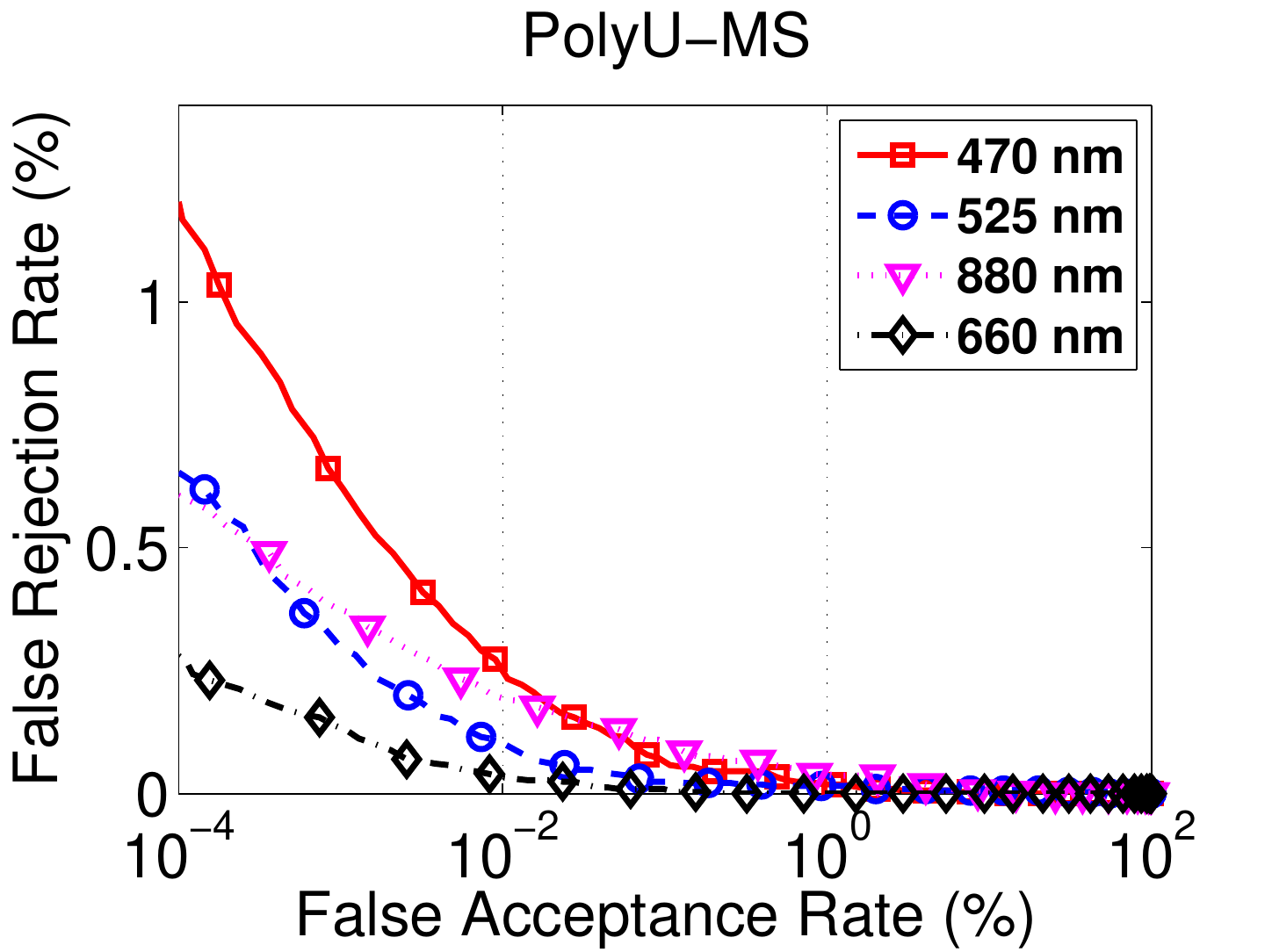}}
\subfigure{\label{fig:roc_PolyU_HS_Exp1}\includegraphics[trim = 0pt 0pt 10pt 0pt, clip, width=0.32\linewidth]{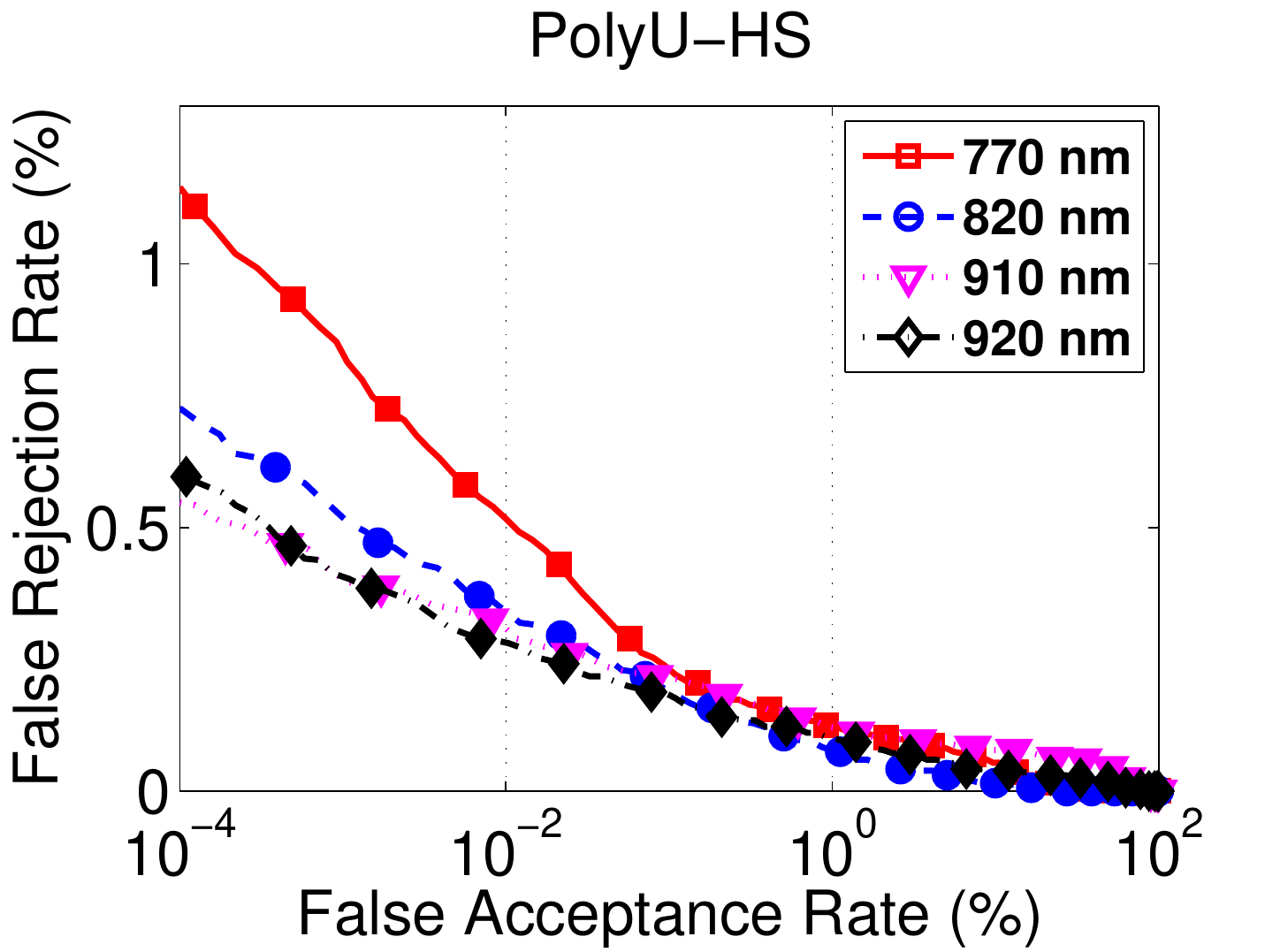}}
\subfigure{\label{fig:roc_CASIA_MS_Exp1}\includegraphics[trim = 0pt 0pt 10pt 0pt, clip, width=0.32\linewidth]{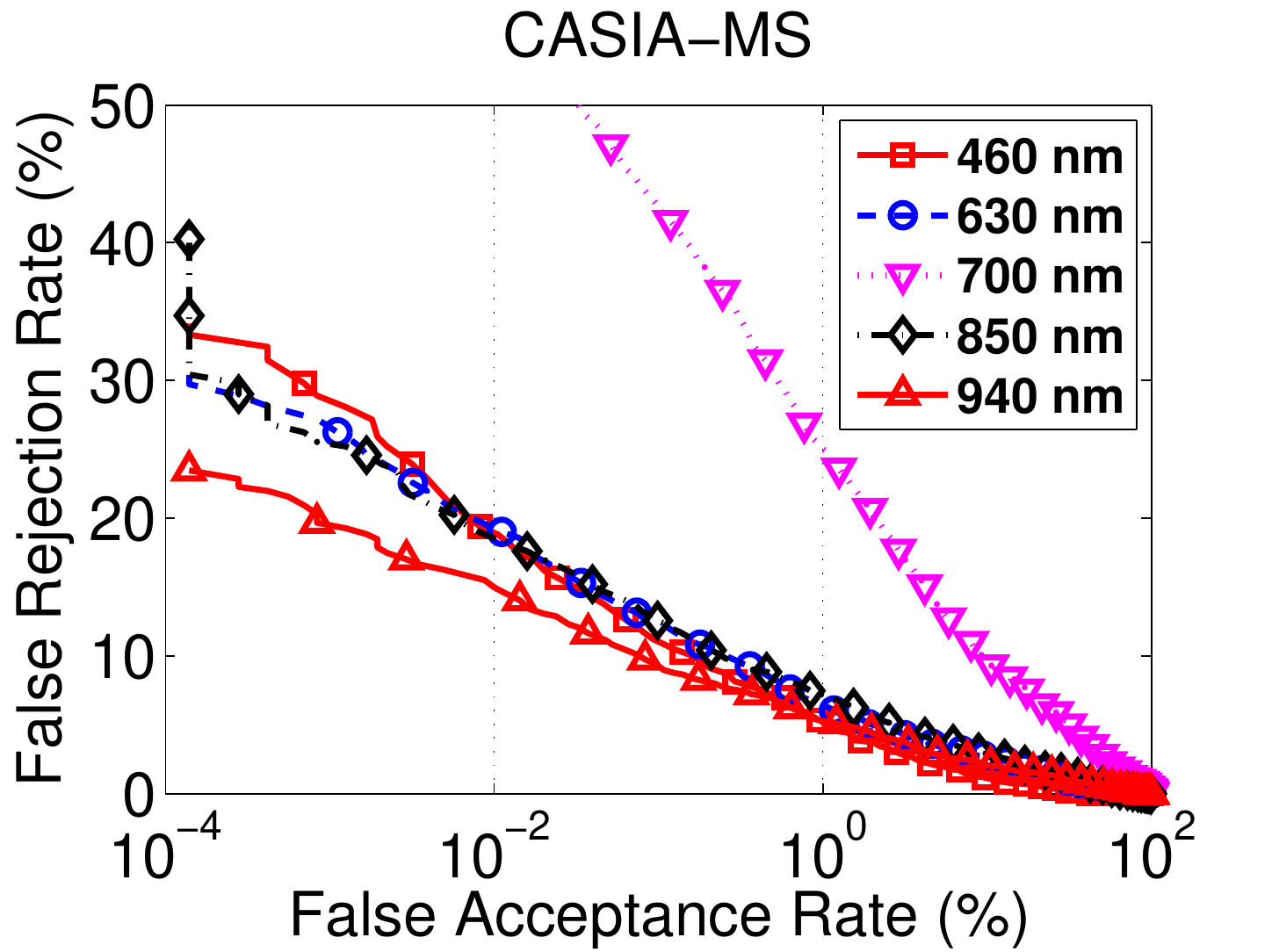}}
\end{minipage}
\caption{Exp.1: ROC curves of {ContCode-ATM} on individual bands}
\label{fig:roc_Exp1}
\end{figure}
\begin{table}[h]
\caption{Individual band performance of {ContCode-ATM}}
\label{tab:res_Exp1}
\footnotesize
\centering
\begin{tabular}[t]{|l|cc|}
\multicolumn{3}{c}{PolyU-MS} \\ \noalign{\smallskip}\hline
Band    & GAR(\%)   &  EER(\%)  \\ \hline
470 nm  & 99.94     &   0.0784  \\
525 nm  & 99.98     &   0.0420  \\
660 nm  & 99.99     &   0.0242  \\
880 nm  & 99.90     &   0.1030  \\ \hline
\end{tabular}
\begin{tabular}[t]{|l|cc|}
\multicolumn{3}{c}{PolyU-HS} \\ \noalign{\smallskip}\hline
Band    & GAR(\%)   &  EER(\%)  \\ \hline
770 nm  & 99.77     &   0.1876  \\
820 nm  & 99.81     &   0.1579  \\
900 nm  & 99.79     &   0.1931  \\
920 nm  & 99.82     &   0.1559  \\ \hline
\end{tabular}
\begin{tabular}[t]{|l|cc|}
\multicolumn{3}{c}{CASIA-MS} \\ \noalign{\smallskip}\hline
Band    & GAR(\%)   &  EER(\%)  \\ \hline
460 nm  & 88.95     &   2.9246  \\
630 nm  & 87.79     &   3.9065  \\
700 nm  & 57.35     &   9.7318  \\
850 nm  & 87.45     &   4.1398  \\
940 nm  & 90.73     &   3.4769  \\ \hline
\end{tabular}
\end{table}
\end{landscape}

\begin{landscape}
\begin{figure}[t]
\centering
\begin{minipage}[b]{1\linewidth}
\subfigure{\label{fig:roc_PolyU_MS_Exp2}\includegraphics[trim = 0pt 0pt 10pt 0pt, clip, width=0.32\linewidth]{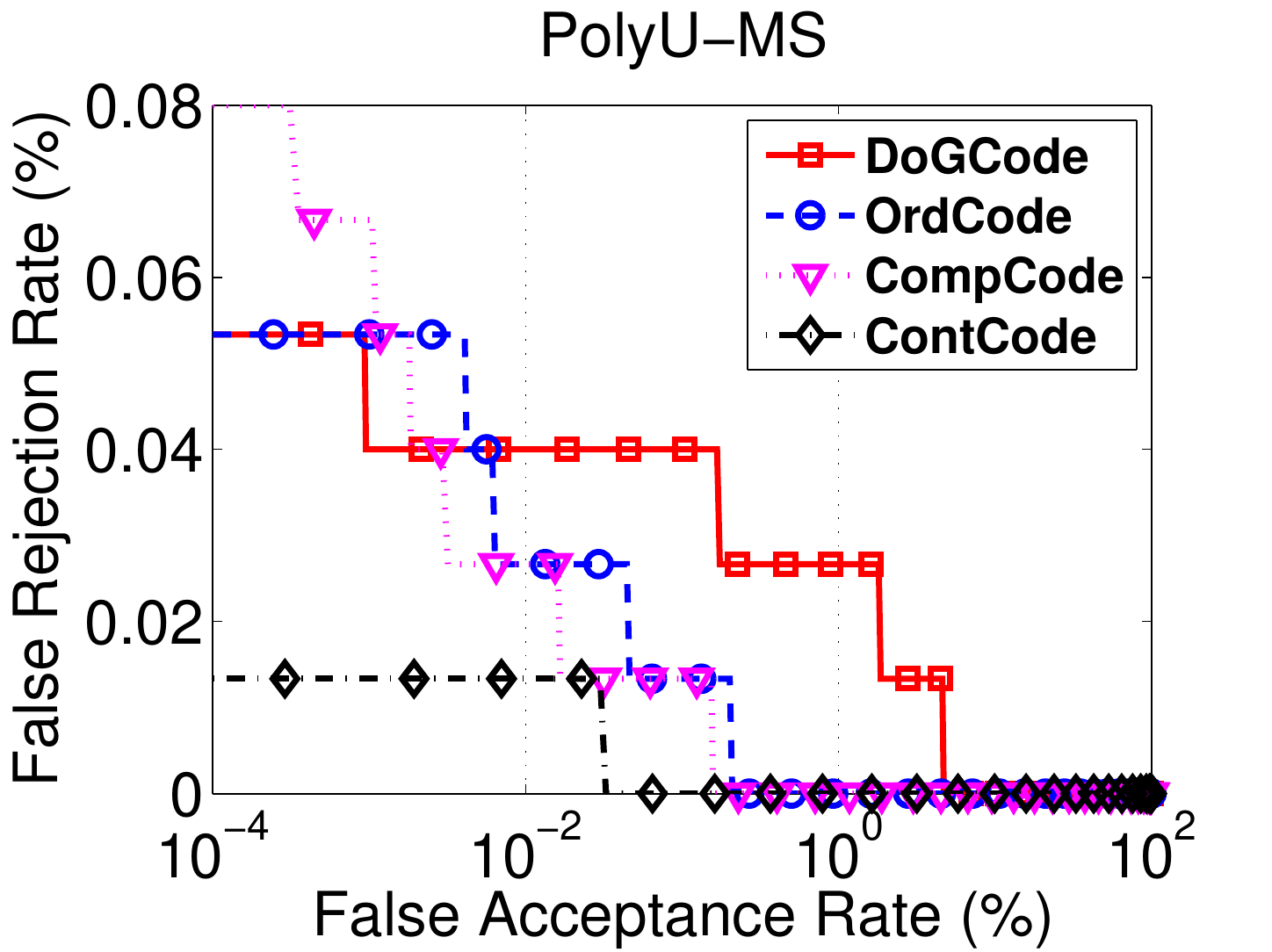}}
\subfigure{\label{fig:roc_PolyU_HS_Exp2}\includegraphics[trim = 0pt 0pt 10pt 0pt, clip, width=0.32\linewidth]{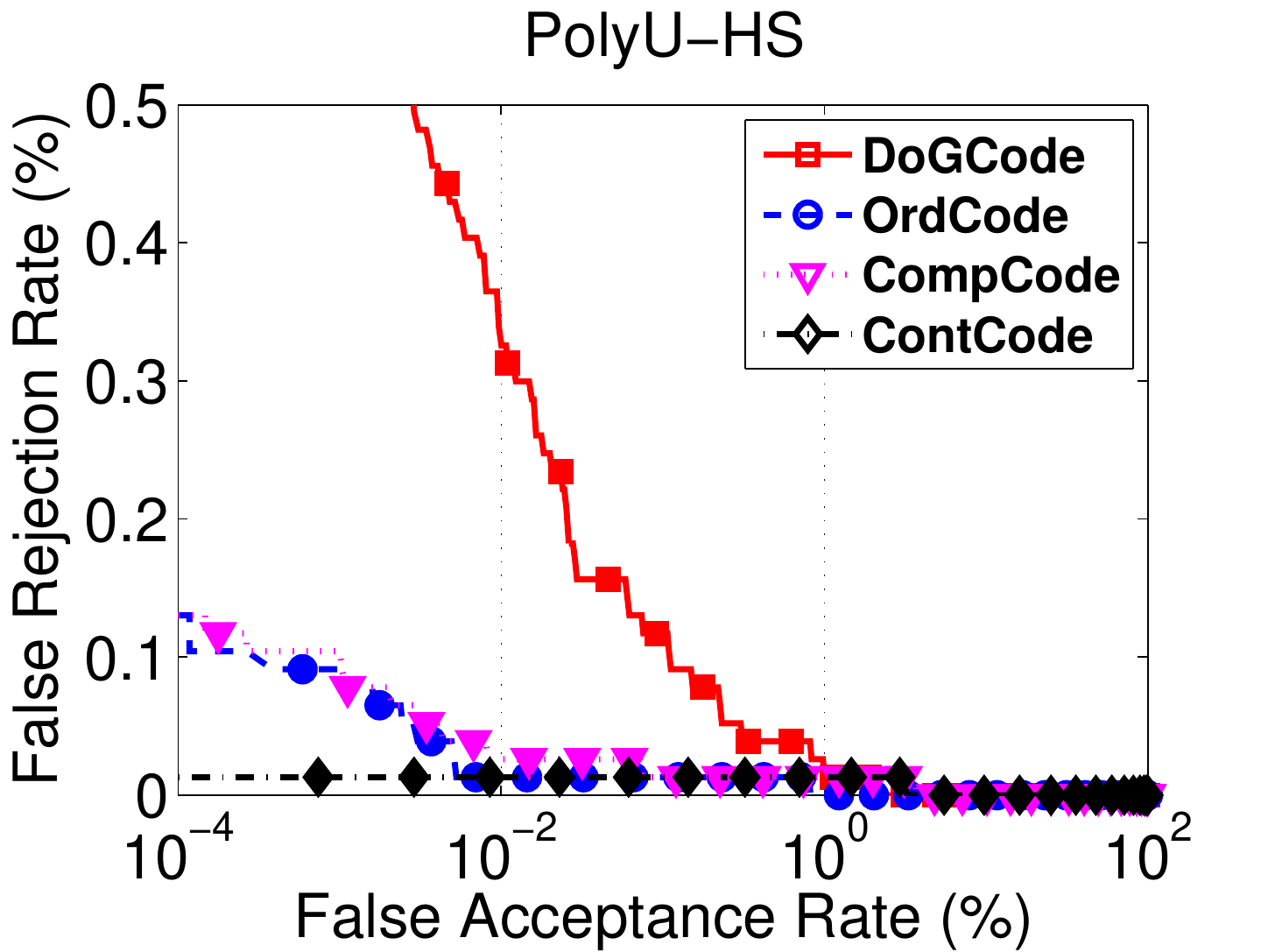}}
\subfigure{\label{fig:roc_CASIA_MS_Exp2}\includegraphics[trim = 0pt 0pt 10pt 0pt, clip, width=0.32\linewidth]{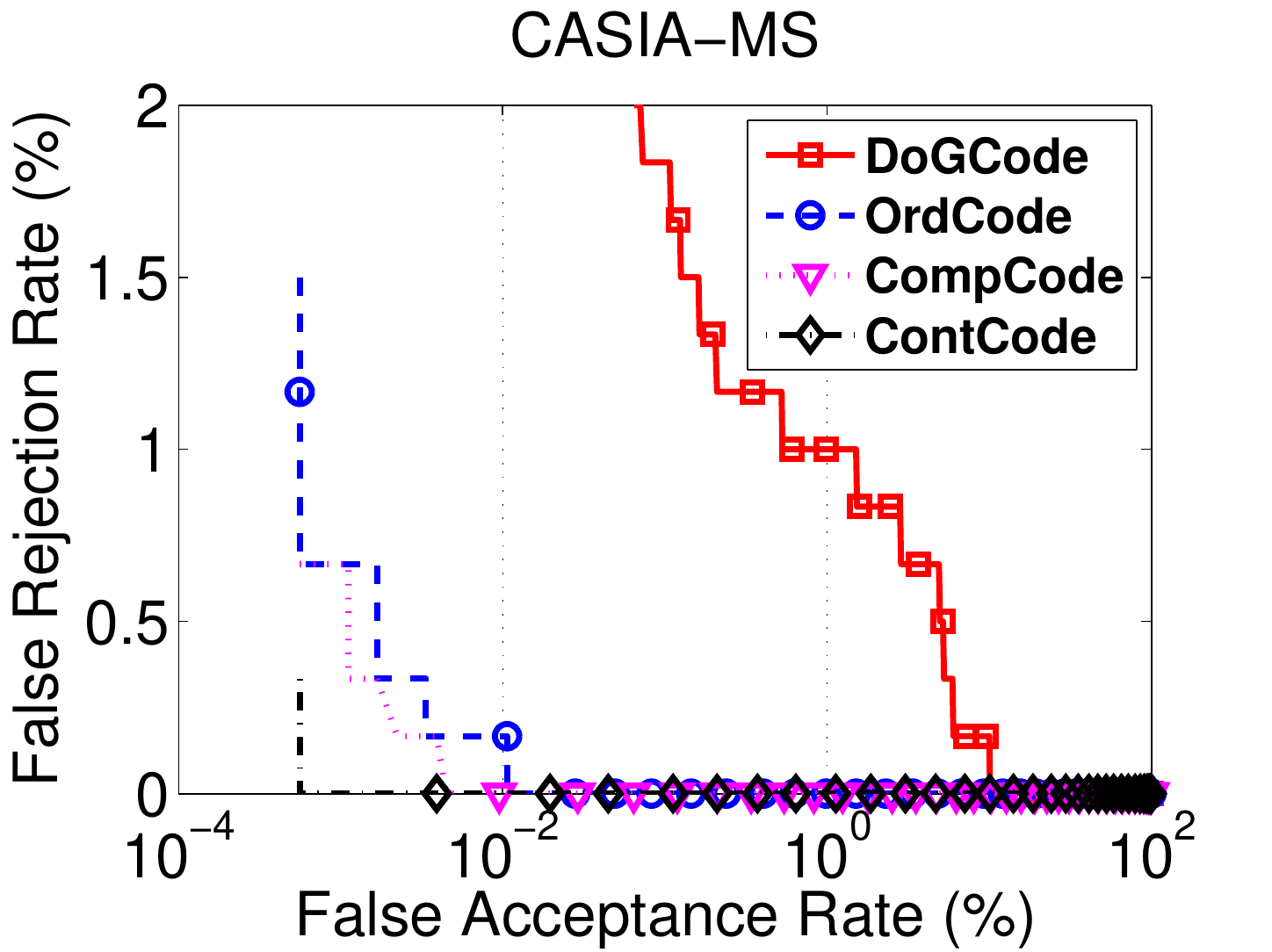}}
\end{minipage}
\caption{Exp.2: Matching palmprints of $1^{\textrm{st}}$ session.}
\label{fig:roc_Exp2}
\end{figure}
\begin{figure}[t]
\centering
\begin{minipage}[b]{1\linewidth}
\subfigure{\label{fig:roc_PolyU_MS_Exp3}\includegraphics[trim = 0pt 0pt 10pt 0pt, clip, width=0.32\linewidth]{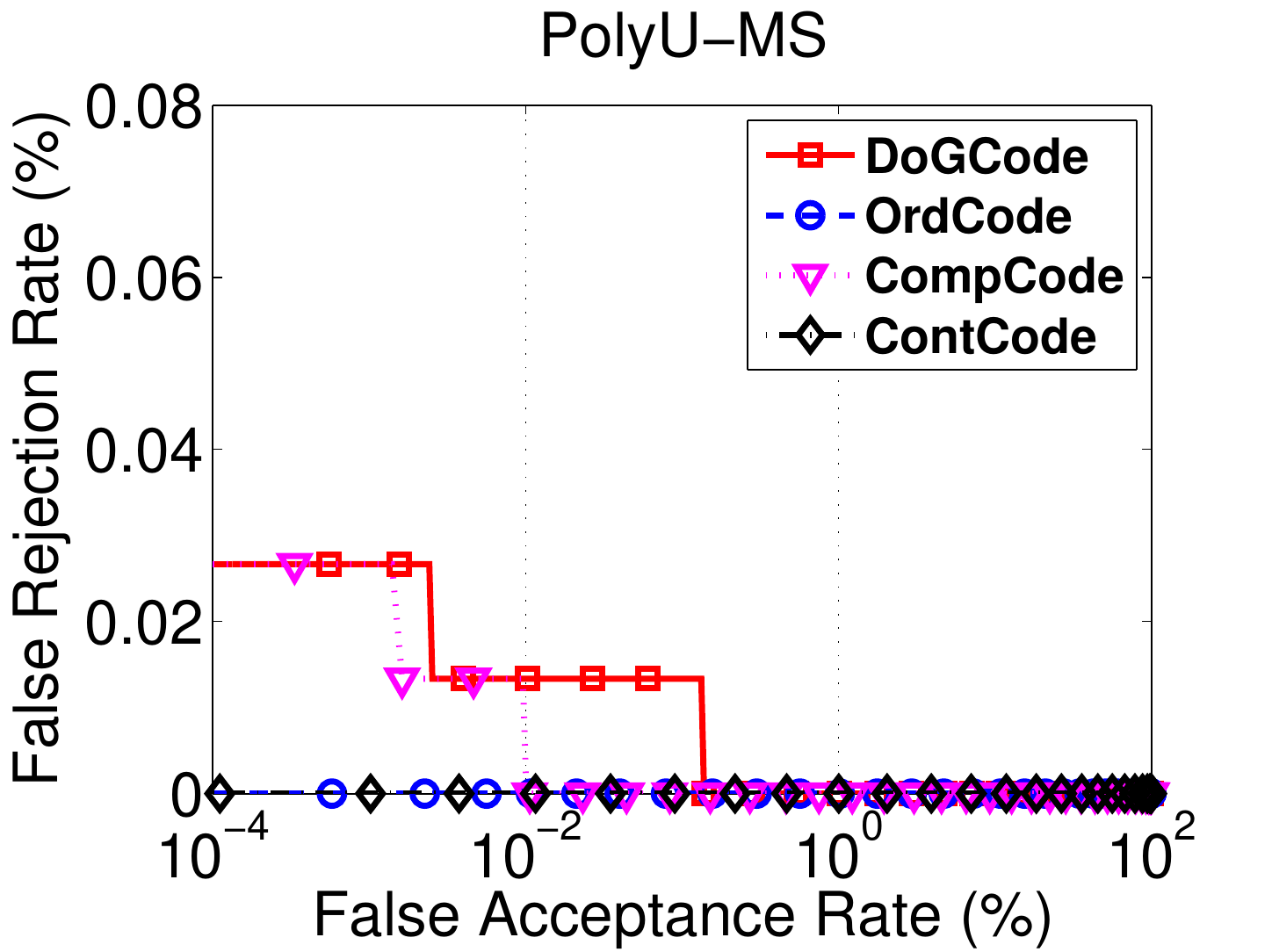}}
\subfigure{\label{fig:roc_PolyU_HS_Exp3}\includegraphics[trim = 0pt 0pt 10pt 0pt, clip, width=0.32\linewidth]{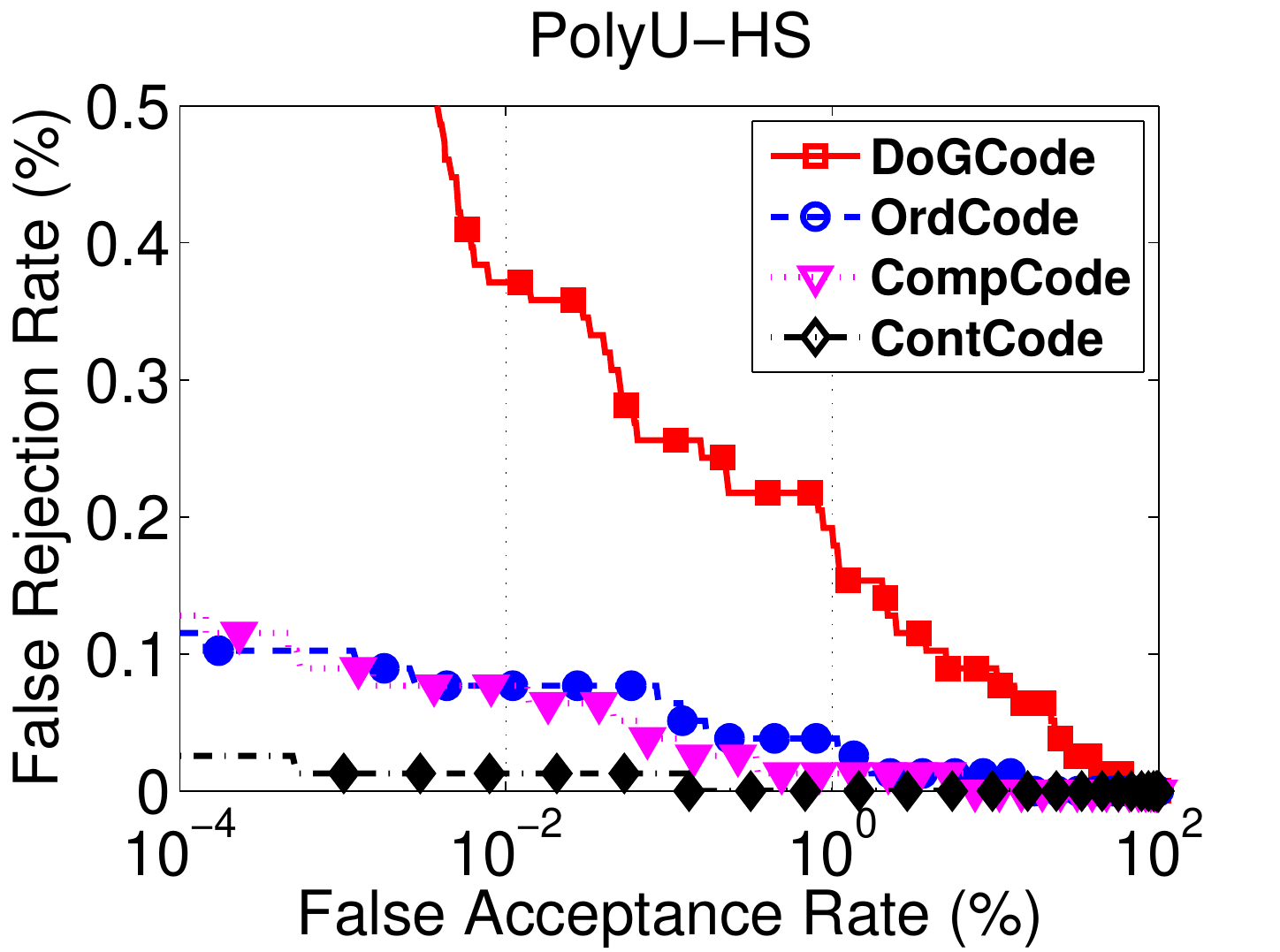}}
\subfigure{\label{fig:roc_CASIA_MS_Exp3}\includegraphics[trim = 0pt 0pt 10pt 0pt, clip, width=0.32\linewidth]{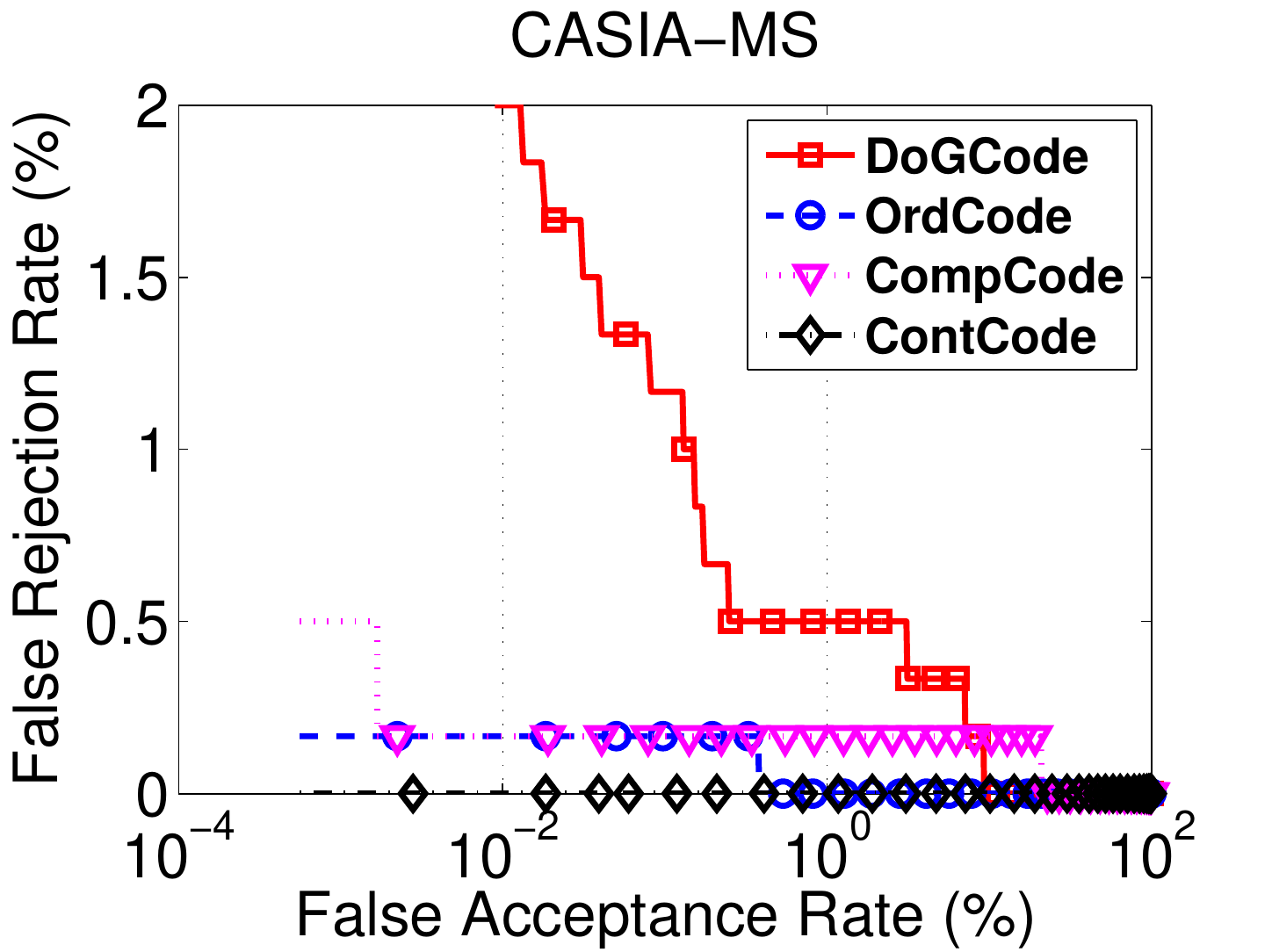}}
\end{minipage}
\caption{Exp.3: Matching palmprints of $2^{\textrm{nd}}$ session.}
\label{fig:roc_Exp3}
\end{figure}
\end{landscape}

\subsubsection{Experiment 3}

analyzes the variability in the palmprint data acquired in the $2^{\textrm{nd}}$ session. This allows for a comparison with the results of \emph{Exp.2} to analyze the intra-session variability. Therefore, only the palmprints acquired in the $2^{\textrm{nd}}$ session are matched. Figure~\ref{fig:roc_Exp3} compares the ROC curves of the {ContCode-ATM} with the other techniques. The small improvement in verification performance on the images of $2^{\textrm{nd}}$ session can be attributed to the better quality of images and increased user familiarity with the acquisition system. A similar trend is observed in verification performance of all techniques as in \emph{Exp.2}.


\subsubsection{Experiment 4}

is designed to mimic a real life verification scenario where variation in image quality or subject behavior over time exists. This experiment analyzes the inter-session variability of multispectral palmprints. Therefore, all images from the $1^{\textrm{st}}$ session are matched to all images of the $2^{\textrm{nd}}$ session. Figure~\ref{fig:roc_Exp4} compares the ROC curves of the {ContCode-ATM} with three other techniques. Note that the performance of other techniques is relatively lower for this experiment compared to \emph{Exp.2} and \emph{Exp.3} because this is a difficult scenario due to the intrinsic variability in image acquisition protocol and the human behavior over time. However, the drop in performance of {ContCode-ATM} is the minimum. Therefore, it is fair to deduce that {ContCode-ATM} is relatively robust to the image variability over time.

\subsubsection{Experiment 5}

evaluates the overall verification performance by combining images from both sessions allowing a direct comparison to existing techniques whose results are reported on the same databases. All images in the database are matched to all other images, irrespective of the acquisition session which is commonly termed as an ``all versus all'' experiment. Figure~\ref{fig:roc_Exp5} compares the ROC curves of the {ContCode-ATM} with three other techniques in the all versus all scenario. Similar to the previous experiments, the {ContCode-ATM} consistently outperforms all other techniques.

The results of \emph{Exp.2} to \emph{Exp.5} are summarized in Table~\ref{tab:res_Exp} for the all databases. The {ContCode-ATM} consistently outperforms the other methods in all experiments. Moreover, CompCode is consistently the second best performer except for very low FAR values in \emph{Exp.4} and \emph{Exp.5} on the PolyU-MS database (see Figure~\ref{fig:roc_Exp4} and Figure~\ref{fig:roc_Exp5}). It is interesting to note the OrdCode performs better than the DoGCode on all databases.

\begin{landscape}
\begin{figure}[t]
\centering
\begin{minipage}[b]{1\linewidth}
\subfigure{\label{fig:roc_PolyU_MS_Exp4}\includegraphics[trim = 0pt 0pt 10pt 0pt, clip, width=0.32\linewidth]{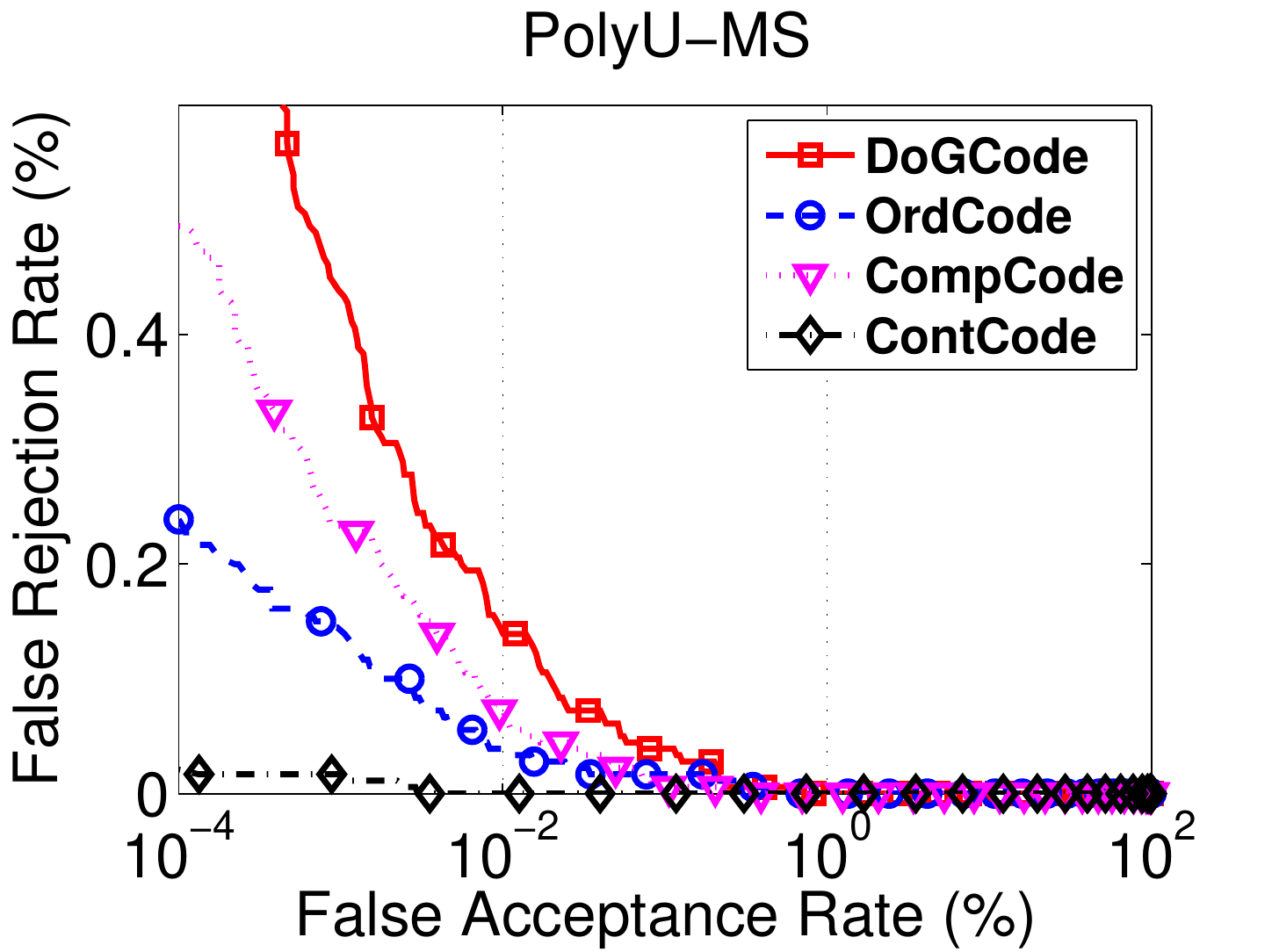}}
\subfigure{\label{fig:roc_PolyU_HS_Exp4}\includegraphics[trim = 0pt 0pt 10pt 0pt, clip, width=0.32\linewidth]{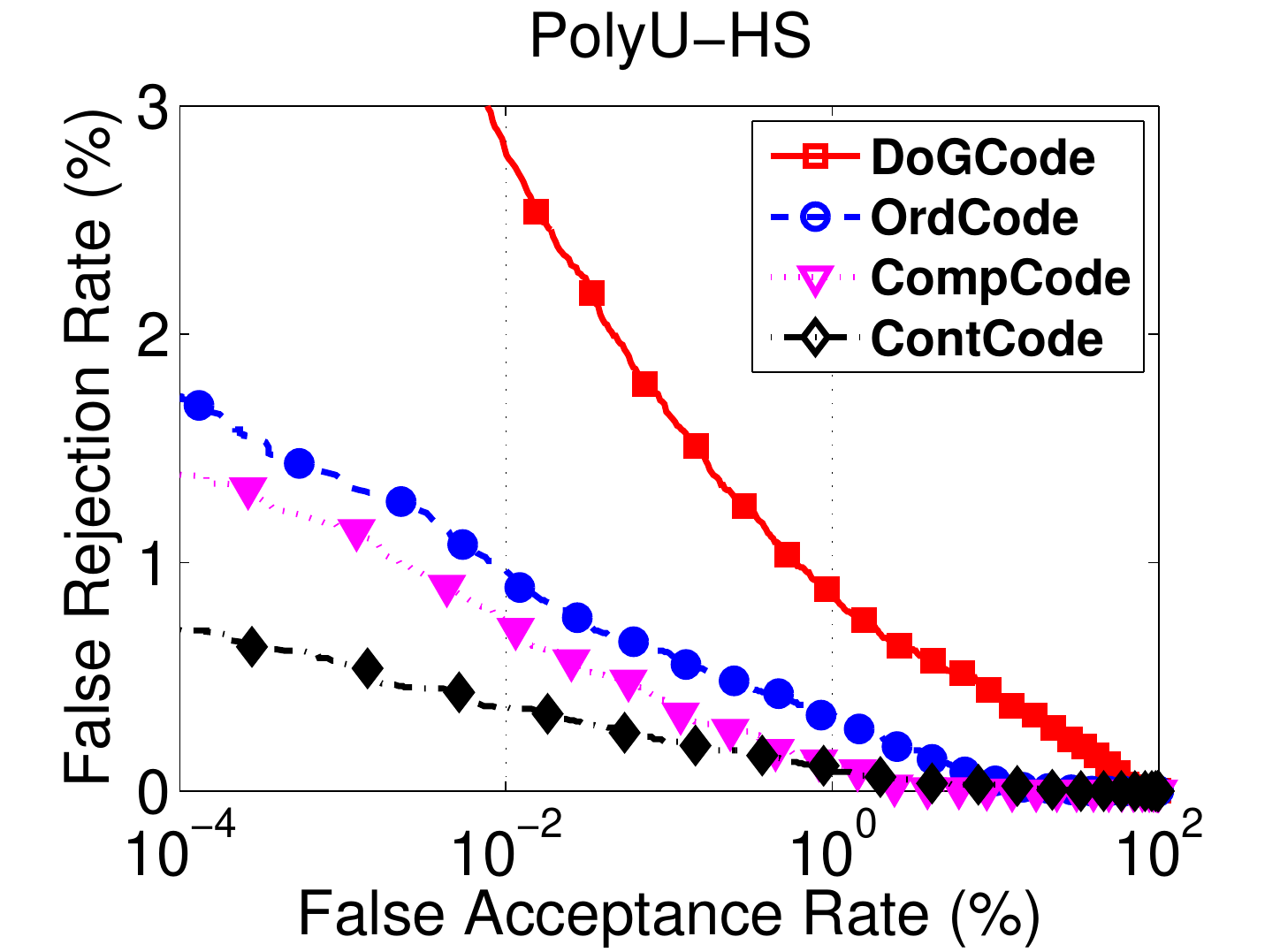}}
\subfigure{\label{fig:roc_CASIA_MS_Exp4}\includegraphics[trim = 0pt 0pt 10pt 0pt, clip, width=0.32\linewidth]{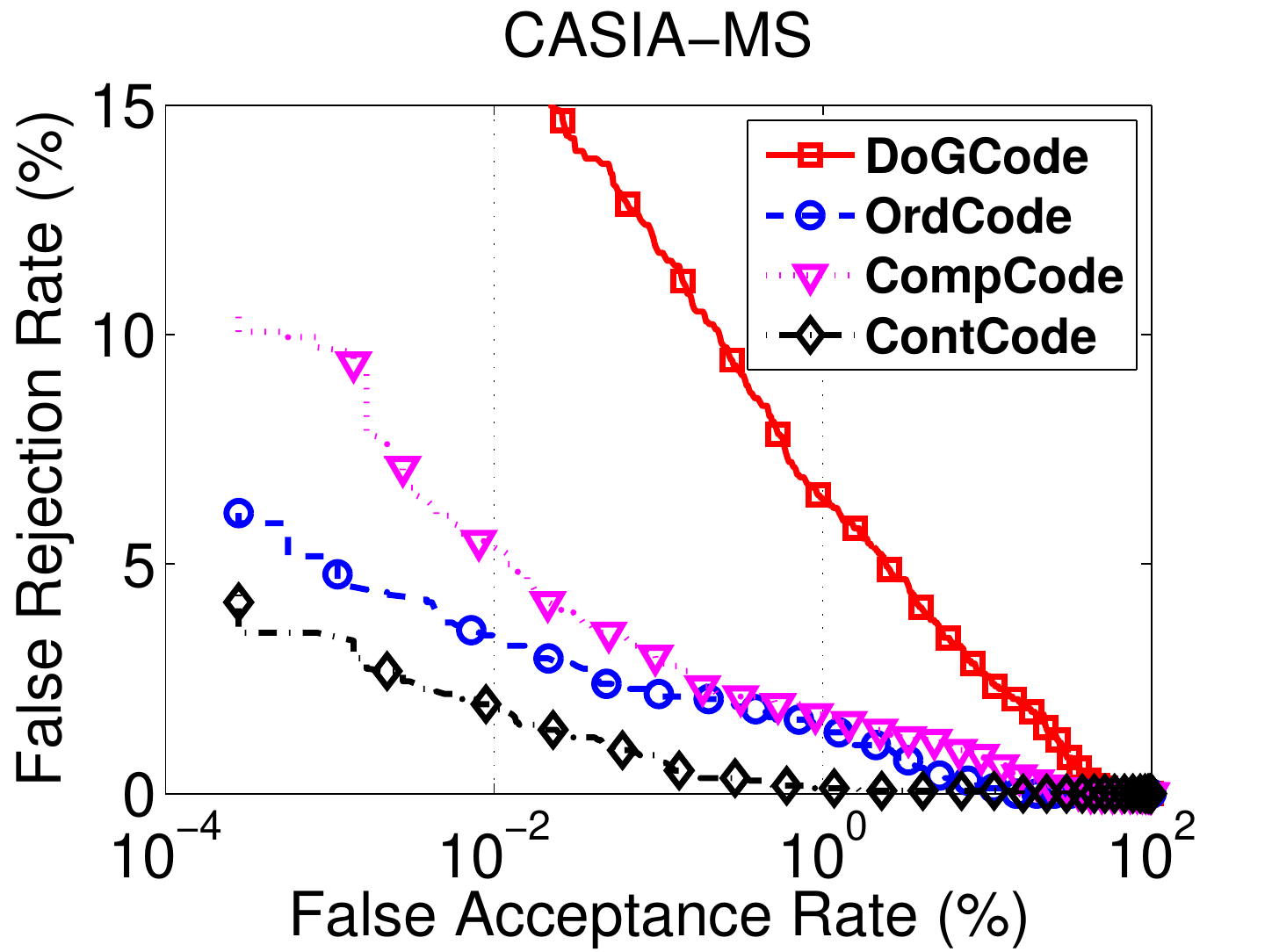}}
\end{minipage}
\caption[Exp.4: Matching palmprints of the $1^{\textrm{st}}$ session to the $2^{\textrm{nd}}$ session]{Exp.4: Matching palmprints of the $1^{\textrm{st}}$ session to the $2^{\textrm{nd}}$ session. The verification performance is low relative to \emph{Exp.2} and \emph{Exp.3}.}
\label{fig:roc_Exp4}
\end{figure}
\begin{figure}[t]
\centering
\begin{minipage}[b]{1\linewidth}
\subfigure{\label{fig:roc_PolyU_MS_Exp5}\includegraphics[trim = 0pt 0pt 10pt 0pt, clip, width=0.32\linewidth]{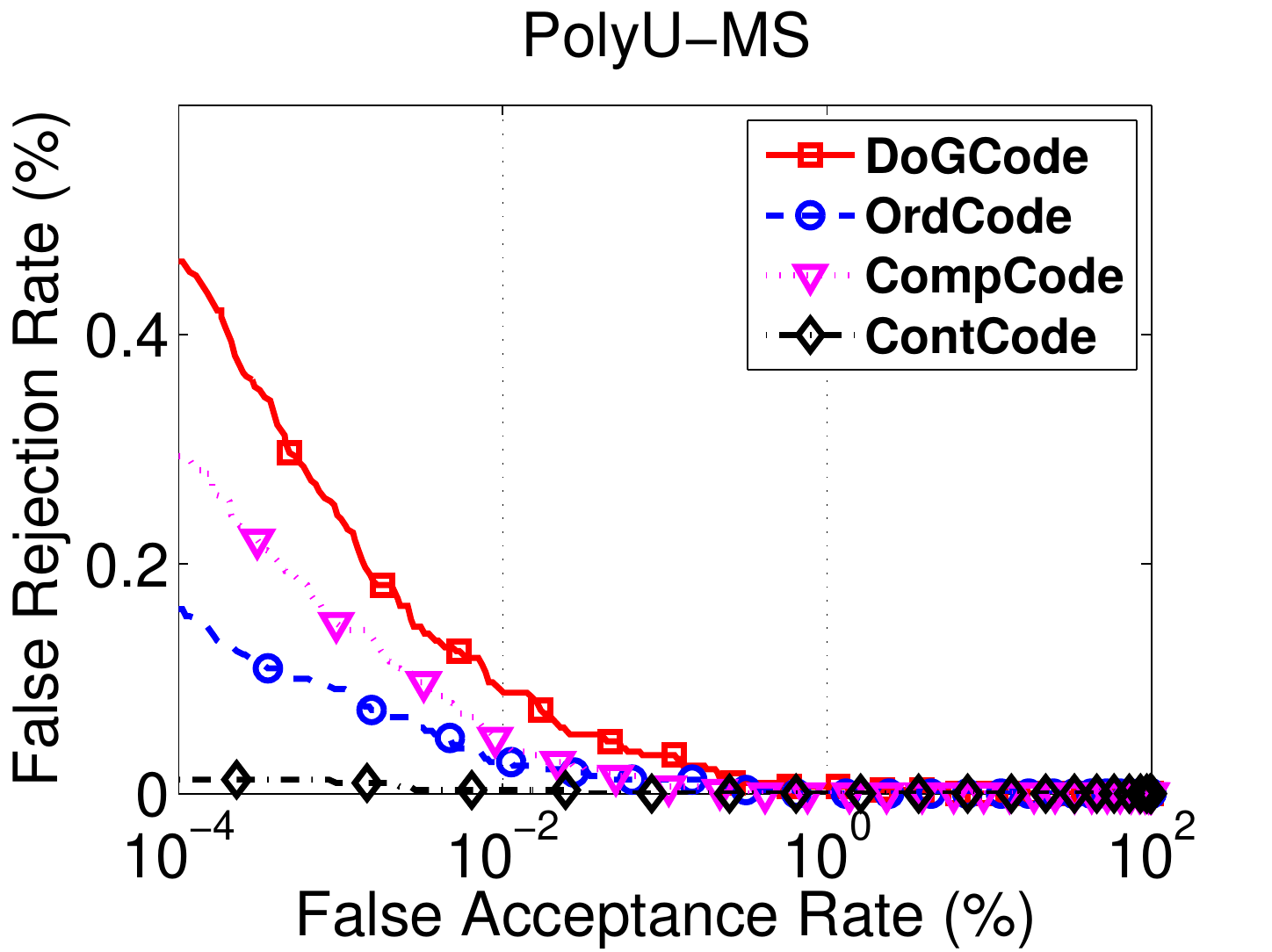}}
\subfigure{\label{fig:roc_PolyU_HS_Exp5}\includegraphics[trim = 0pt 0pt 10pt 0pt, clip, width=0.32\linewidth]{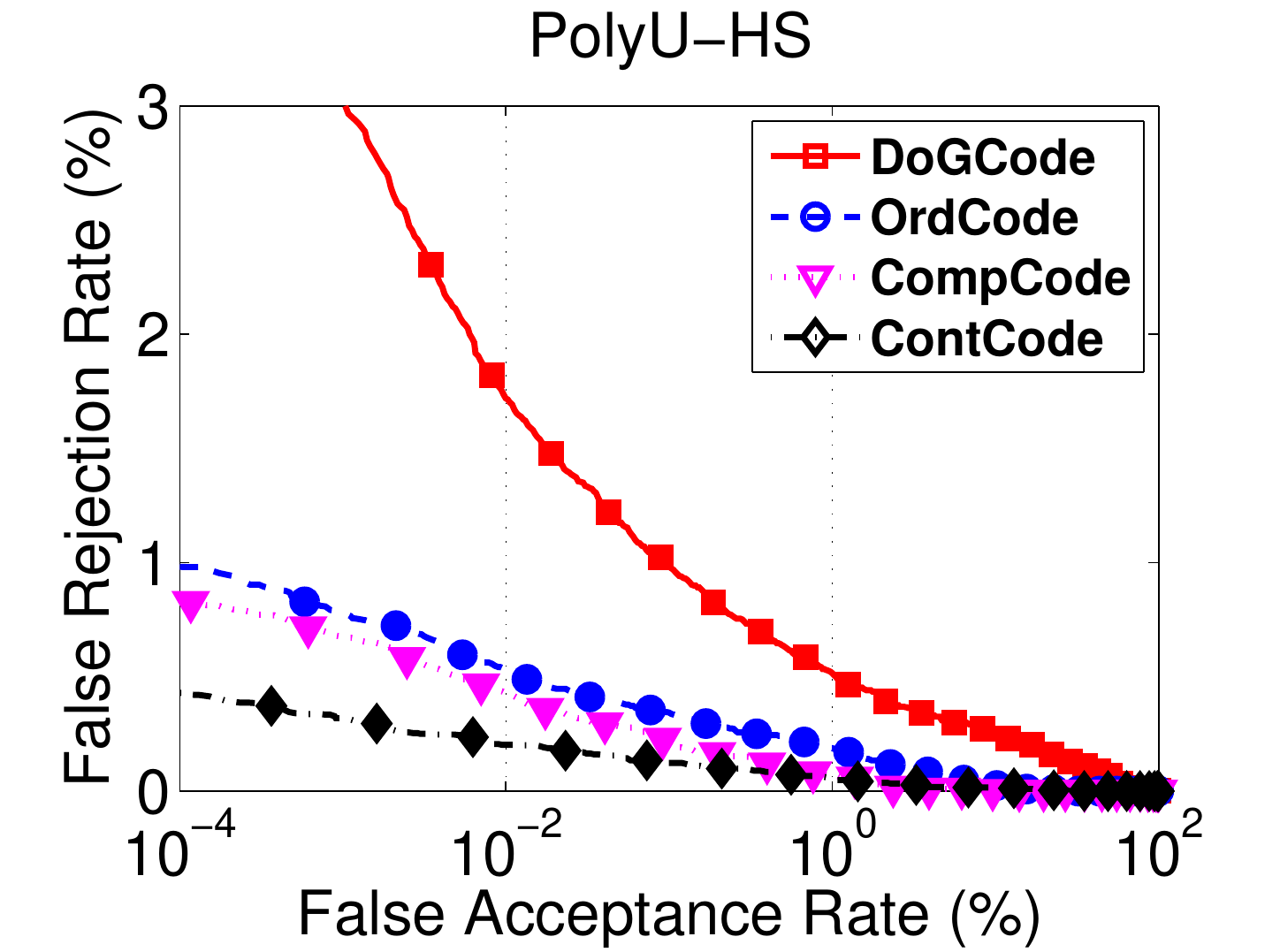}}
\subfigure{\label{fig:roc_CASIA_MS_Exp5}\includegraphics[trim = 0pt 0pt 10pt 0pt, clip, width=0.32\linewidth]{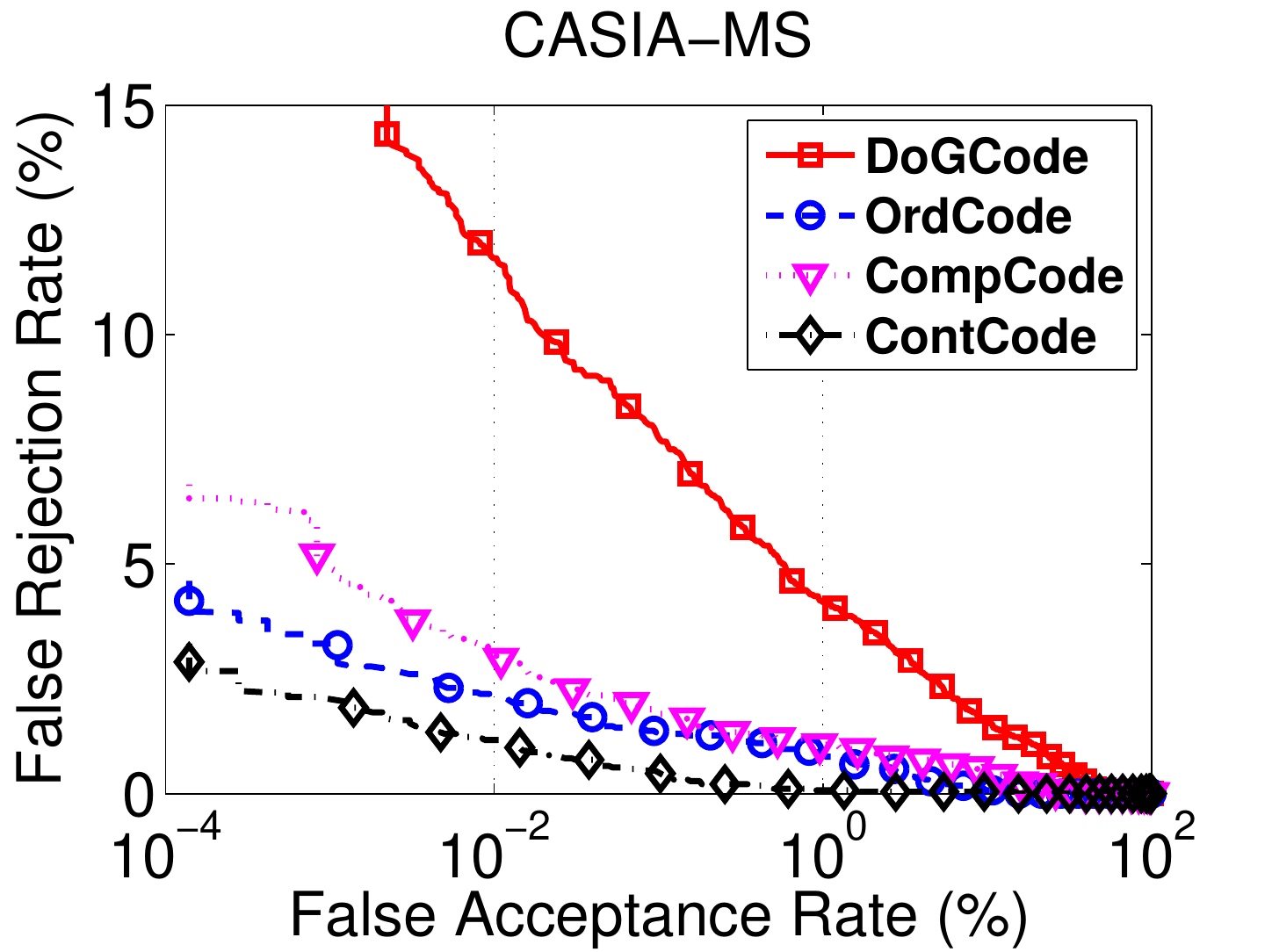}}
\end{minipage}
\caption{Exp.5: Matching palmprints irrespective of the acquisition session.}
\label{fig:roc_Exp5}
\end{figure}
\end{landscape}

\begin{table*}[h]
\caption{Summary of verification results for \emph{Exp.2} to \emph{Exp.5}}
\footnotesize
\label{tab:res_Exp}
\centering
\begin{tabular}{|l|c|c|c|c|c|} \hline
\multicolumn{2}{|c|}{} & \multicolumn{4}{c|}{PolyU-MS} \\ \cline{3-6}
\multicolumn{2}{|c|}{} & {DoGCode} & {OrdCode}  & {CompCode}  & {ContCode-ATM}  \\ \hline
\multirow{2}{*}{Exp.2}  & EER(\%)  & 0.0400 & 0.0267 & 0.0165 & 0.0133  \\
                        & GAR(\%)  & 99.96  & 99.99  & 99.99  & 100.00  \\ \hline
\multirow{2}{*}{Exp.3}  & EER(\%)  & 0.0133 & 0      & 0.0098 & 0       \\
                        & GAR(\%)  & 99.99  & 100.00 & 100.00 & 100.00  \\ \hline
\multirow{2}{*}{Exp.4}  & EER(\%)  & 0.0528 & 0.0247 & 0.0333 & 0.0029  \\
                        & GAR(\%)  & 99.96  & 99.98  & 99.99  & 100.00  \\ \hline
\multirow{2}{*}{Exp.5}  & EER(\%)  & 0.0455 & 0.0212 & 0.0263 & 0.0030  \\
                        & GAR(\%)  & 99.97  & 99.99  & 99.99  & 100.00  \\ \hline

\multicolumn{2}{|c|}{} & \multicolumn{4}{c|}{PolyU-HS} \\ \cline{3-6}
\multicolumn{2}{|c|}{}  & {DoGCode} & {OrdCode}  & {CompCode}  & {ContCode-ATM}  \\ \hline
\multirow{2}{*}{Exp.2}  & EER(\%)  & 0.1084 & 0.0130 & 0.0261 & 0.0130 \\
                        & GAR(\%)  & 99.88  & 99.99  & 99.99  & 99.99  \\ \hline
\multirow{2}{*}{Exp.3}  & EER(\%)  & 0.2269 & 0.0768 & 0.0512 & 0.0128 \\
                        & GAR(\%)  & 99.74  & 99.94  & 99.96  & 99.99 \\ \hline
\multirow{2}{*}{Exp.4}  & EER(\%)  & 0.8912 & 0.4318 & 0.2623 & 0.1847 \\
                        & GAR(\%)  & 98.35  & 99.40  & 99.62  & 99.78  \\ \hline
\multirow{2}{*}{Exp.5}  & EER(\%)  & 0.5965 & 0.2599 & 0.1669 & 0.1213 \\
                        & GAR(\%)  & 99.02  & 99.66  & 99.78  & 99.86  \\ \hline

\multicolumn{2}{|c|}{} & \multicolumn{4}{c|}{CASIA-MS} \\ \cline{3-6}
\multicolumn{2}{|c|}{} & {DoGCode} & {OrdCode}  & {CompCode}  & {ContCode-ATM}  \\ \hline
\multirow{2}{*}{Exp.2}  & EER(\%)  & 1.000  & 0.1667 & 0.0140 & 0       \\
                        & GAR(\%)  & 98.00  & 99.67  & 100.00 & 100.00  \\ \hline
\multirow{2}{*}{Exp.3}  & EER(\%)  & 0.6667 & 0.1667 & 0.1667 & 0.0011  \\
                        & GAR(\%)  & 98.50  & 99.83  & 99.83  & 100.00  \\ \hline
\multirow{2}{*}{Exp.4}  & EER(\%)  & 3.8669 & 1.2778 & 0.6667 & 0.2778  \\
                        & GAR(\%)  & 87.70  & 97.39  & 97.72  & 99.61   \\ \hline
\multirow{2}{*}{Exp.5}  & EER(\%)  & 2.8873 & 0.8667 & 0.4993 & 0.2000  \\
                        & GAR(\%)  & 92.01  & 98.37  & 98.60  & 99.76   \\ \hline

\end{tabular}
\end{table*}

We also compare performance of the Contour Code with that of the other methods reported in the literature. The results can be directly compared because all methods have reported EERs for the ``all versus all'' scenario. Table~\ref{tab:res_PolyU_MS_all} compares the EERs of various methods on the PolyU-MS database. The proposed {ContCode-STM} achieves an EER of $0.0061\%$ whereas in the case of {ContCode-ATM} the EER reduces to $0.0030\%$ on the PolyU-MS database. The proposed {ContCode-ATM} thus achieves a $75\%$ reduction in EER compared to nearest competitor, the CompCode. The error rates are extremely low (0.06 in 1000 chance of error) and indicate the viability of using multispectral palmprints in a high security system.

\begin{table}[h]
\caption{Comparative performance on the PolyU-MS database.}
\footnotesize
\label{tab:res_PolyU_MS_all}
\centering
\begin{tabular}{|l|c|c|} \hline
\multirow{2}{*}{Method}                                     & \multicolumn{2}{c|}{EER (\%)} \\ \cline{2-3}
                                                            & \emph{No Blur}&\emph{With Blur}\\ \hline
*Palmprint and Palmvein~\cite{zhang2011online}              &   0.0158      &       -       \\
*CompCode-wavelet fusion~\cite{han2008multispectral}        &   0.0696      &       -       \\
*CompCode-feature-level fusion~\cite{luo2012multispectral}  &   0.0151      &       -       \\
*CompCode-score-level fusion~\cite{zhang2010online}         &\textbf{0.0121}&       -       \\
OrdCode~\cite{sun2005ordinal}                               &   0.0248      &   0.0212      \\
DoGCode~\cite{wu2006palmprint}                              &   0.0303      &   0.0455      \\
{ContCode-STM}                                              &   0.0182      &   0.0061      \\
{ContCode-ATM}                                              &   0.0182      &\textbf{0.0030}\\ \hline
\end{tabular}
\\
{\scriptsize *Results taken from published papers using the ``all versus all'' protocol.}
\end{table}

Table~\ref{tab:res_PolyU_HS_all} provides the EERs of all methods with and without blur on PolyU-HS database. It can be observed that the other algorithms do not favor blur in matching. However, the proposed {ContCode-ATM} always improves with the introduction of blur. The EER reduction in case of {ContCode-ATM} is nearly 5\% from its nearest competitor CompCode without blur. Notice that the difference in EER among all algorithms is relatively small compared to that of these algorithms on PolyU-MS database.

\begin{table}[h]
\caption{Comparative performance on the PolyU-HS database.}
\footnotesize
\label{tab:res_PolyU_HS_all}
\centering
\begin{tabular}{|l|c|c|} \hline
\multirow{2}{*}{Method}             & \multicolumn{2}{c|}{EER (\%)}     \\ \cline{2-3}
                                    & \emph{No Blur} &\emph{With Blur}  \\ \hline
DoGCode~\cite{wu2006palmprint}      &  0.4679        &  0.5965          \\
OrdCode~\cite{sun2005ordinal}       &  0.1751        &  0.2599          \\
CompCode~\cite{kong2004competitive} & \textbf{0.1267}&  0.1669          \\
{ContCode-STM}                      &  0.1522        &  0.1448          \\
{ContCode-ATM}                      &  0.1371        &\textbf{0.1213}   \\ \hline
\end{tabular}
\end{table}

Table~\ref{tab:res_CASIA_MS_All} compares the EERs of various methods reported on the CASIA-MS database. Both variants of the Contour Code outperform other methods and achieve an EER reduction of 60\% compared to the nearest competitor, CompCode. The EER of Contour Code on the CASIA-MS database is higher relative to the PolyU-MS database because the former was acquired using a non-contact sensor. Interestingly, the performance of DoGCode and OrdCode deteriorates with blurring. While in the case of CompCode and ContCode, the performance improves with the introduction of blur. However, the improvement in the proposed ContCode is much larger. Another important observation is that the ATM mode of matching always performs better than the STM mode and is thus the preferable choice for multispectral palmprint matching.

\begin{table}[h]
\caption{Comparative performance on the CASIA-MS database.}
\footnotesize
\label{tab:res_CASIA_MS_All}
\begin{center}
\begin{tabular}{|l|c|c|} \hline
\multirow{2}{*}{Method}                                 & \multicolumn{2}{c|}{EER (\%)}     \\ \cline{2-3}
                                                        & \emph{No Blur}    &\emph{With Blur}\\ \hline
*Wavelet fusion with ACO~\cite{kisku2010multispectral}  &   3.125           &   -           \\
*Curvelet fusion with OLOF~\cite{hao2008multispectral}  &   0.50$\dag$      &   -           \\
OrdCode~\cite{sun2005ordinal}                           &   0.5667          &   0.8667      \\
DoGCode~\cite{wu2006palmprint}                          &   1.9667          &   2.8873      \\
CompCode~\cite{kong2004competitive}                     &   0.8667          &   0.4993      \\
{ContCode-STM}                                          &   0.6279          &   0.2705      \\
{ContCode-ATM}                                          &   \textbf{0.4333} &\textbf{0.2000}\\ \hline
\end{tabular}
\end{center}
\scriptsize{*Results taken from published papers.\\
$^\dag$This result was reported on a database of 330 hands whereas only a subset of 200 hands has been made public}
\end{table}

\subsection{Identification Experiments}
\label{sec:cmc}
We perform identification experiments using 5-fold cross validation and report Cumulative Match Characteristics (CMC) curves and rank-1 identification rates. In each fold, we randomly select one multispectral palmprint image per subject to form the gallery and treat all the remaining images as probes. So, identification is based on a single multispectral image in the gallery for any probe subject. The identification rates are then averaged over the five folds. This protocol is followed in all databases.

The CMC curves, for comparison with three other techniques, on all databases are given in Figure~\ref{fig:cmc} and the identification results are summarized in Table~\ref{tab:res_ident}. The {ContCode-ATM} achieved an average of $99.88\%$ identification rate on the CASIA-MS database 99.91\% on PolyU-HS database and $100\%$ on the PolyU-MS database. The proposed {ContCode-ATM} clearly demonstrates better identification performance in comparison to state-of-the-art techniques.

\begin{landscape}
\begin{figure}[t]
\centering
\begin{minipage}[b]{1\linewidth}
\subfigure{\label{fig:cmc_PolyU_MS}\includegraphics[trim = 0pt 0pt 10pt 0pt, clip, width=0.32\linewidth]{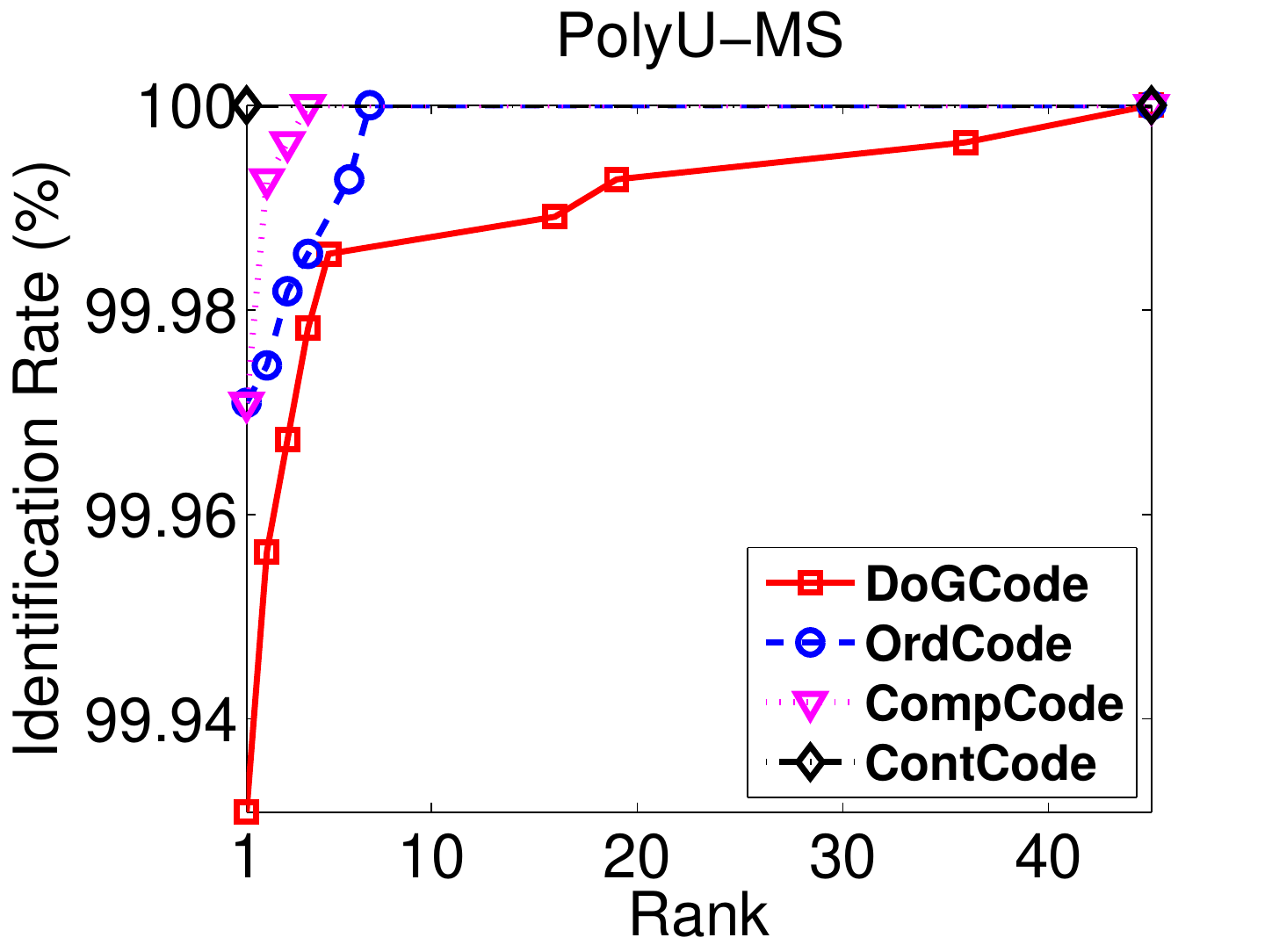}}
\subfigure{\label{fig:cmc_PolyU_HS}\includegraphics[trim = 0pt 0pt 10pt 0pt, clip, width=0.32\linewidth]{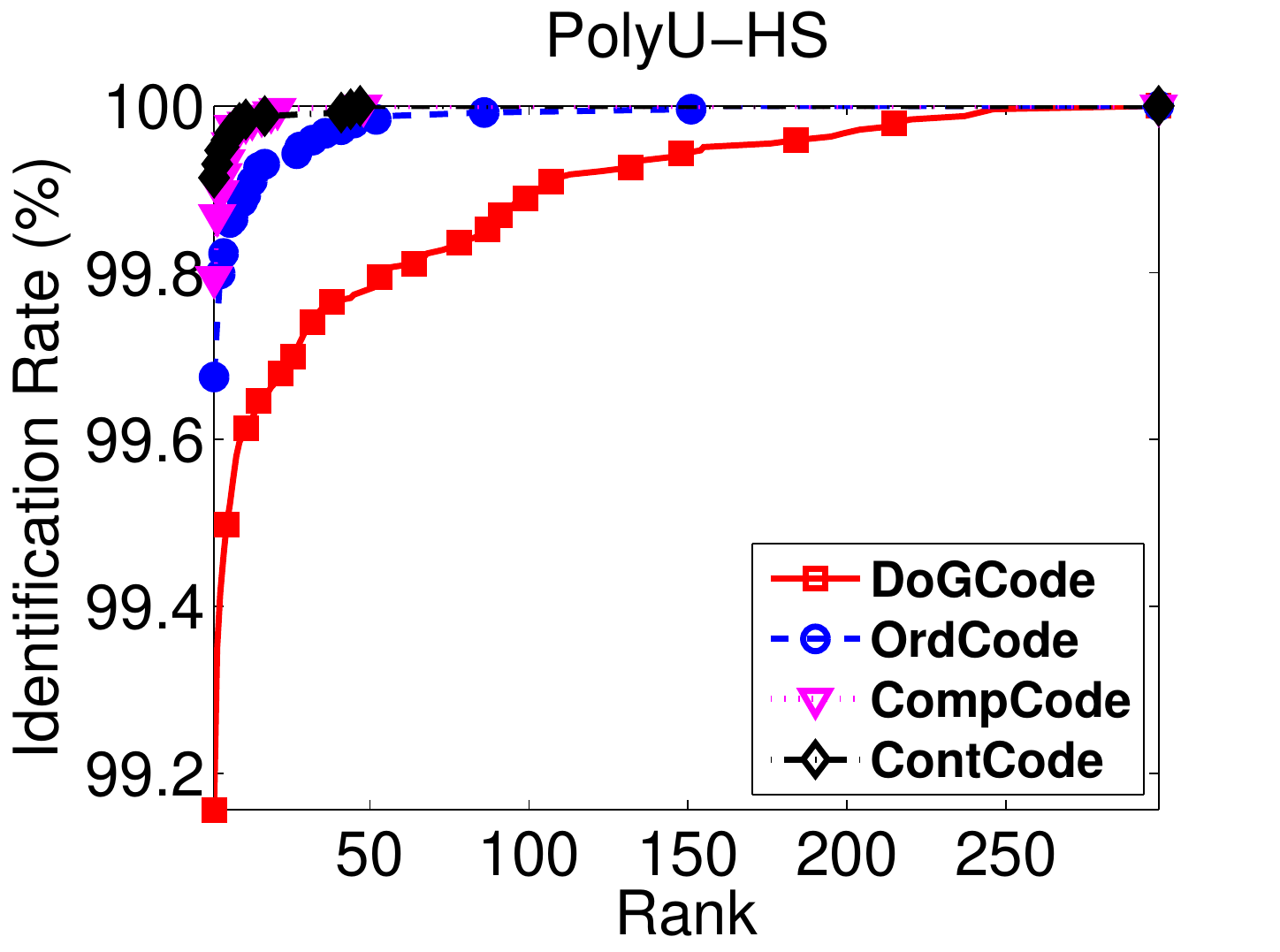}}
\subfigure{\label{fig:cmc_CASIA_MS}\includegraphics[trim = 0pt 0pt 10pt 0pt, clip, width=0.32\linewidth]{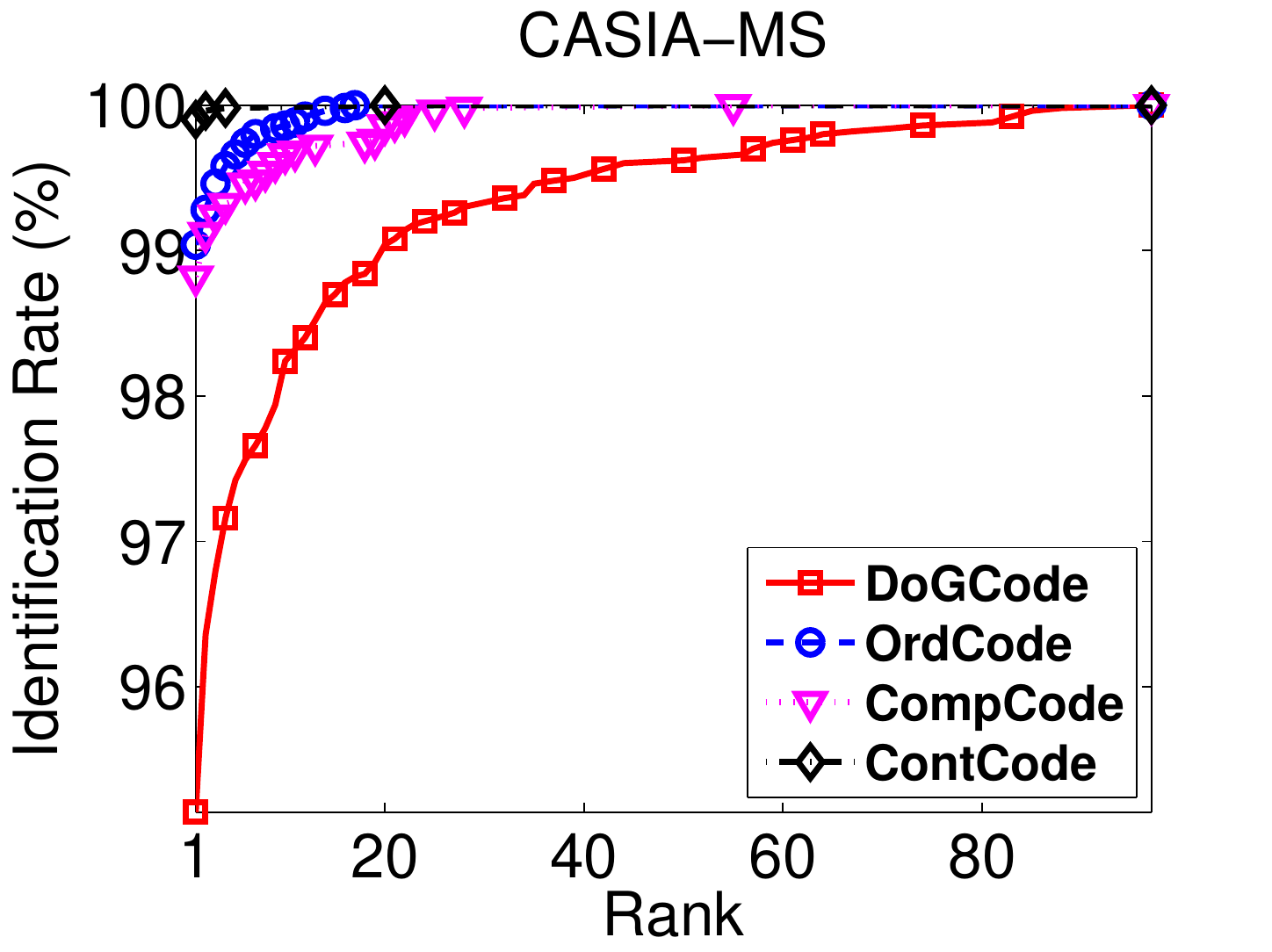}}
\end{minipage}
\caption{CMC curves for the identification experiment on all databases.}
\label{fig:cmc}
\end{figure}
\begin{table}[h]
\caption[Recognition rates on the PolyU-MS, PolyU-HS and CASIA-MS datasets]{Comparison of rank-1 recognition rates and standard deviations on the CASIA-MS, PolyU-HS and PolyU-MS databases. Recognition rates are averaged over 5-folds.}
\footnotesize
\label{tab:res_ident}
\centering
\begin{tabular}{|l|c|c|c|} \hline
Method                              &   PolyU-MS (\%) &   PolyU-HS (\%)   &   CASIA-MS (\%)  \\ \hline
OrdCode~\cite{sun2005ordinal}       &   99.93$\pm$0.05&   99.16$\pm$0.23  &   99.02$\pm$0.11 \\
DoGCode~\cite{wu2006palmprint}      &   99.97$\pm$0.04&   99.67$\pm$0.06  &   95.08$\pm$0.75 \\
CompCode~\cite{kong2004competitive} &   99.97$\pm$0.03&   99.79$\pm$0.07  &   99.52$\pm$0.11 \\
{ContCode-ATM}                      &   100.0$\pm$0   &   99.91$\pm$0.04  &   99.88$\pm$0.08 \\ \hline
\end{tabular}
\end{table}
\end{landscape}

\subsection{Efficiency}

The computational complexity of matching is critical in practical identification scenarios because databases can be quite large. The binary Contour Code matching has been designed to carry out operations that are of low computational complexity. It comprises an indexing part whose complexity is independent of the database size. It depends on the Contour Code size and is, therefore, fixed for a given size. The indexing operation results in a relatively sparse binary matrix whose column wise summation can be efficiently performed. Summing a column of this matrix calculates the match score with an embedded score level fusion of the multispectral bands of an individual (in STM mode). A MATLAB implementation on a $2.67$ GHz machine with 8 GB RAM can perform over $70,000$ matches per second per band (using a single CPU core). The Contour Code extraction takes only $43$ms per band. In terms of memory requirement, the Contour Code takes only 676 bytes per palm per band.

\section{Conclusion}
\label{sec:conc}

We presented Contour Code, a novel multidirectional representation and binary hash table encoding for robust and efficient multispectral palmprint recognition. An automatic technique was designed for the extraction of a region of interest from palm images acquired with noncontact sensors. Unlike existing methods, we reported quantitative results of ROI extraction by comparing the automatically extracted ROIs with manually extracted ground truth. The Contour Code exhibited robust multispectral feature capturing capability and consistently outperformed existing state-of-the-art techniques in various experimental setups using PolyU-MS, PolyU-HS and CASIA-MS palmprint databases. Binary encoding of the Contour Code in a hash table facilitated simultaneous matching to the database and score level fusion of the multispectral bands in a single step (in STM mode), with no requirement for score normalization before fusion. The Contour Code is a generic orientation code for line-like features and can be extended to other biometric traits including fingerprints and finger-knuckle prints. All MATLAB codes of this work are available at \url{www.sites.google.com/site/zohaibnet/Home/codes}. 

\chapter{Conclusion} 

\label{Chapter8} 



This thesis proposed hyperspectral imaging and analysis methods for sparse reconstruction and recognition from hyperspectral images. A method was proposed to adaptively recover the spectral reflectance from an LCTF based hyperspectral imaging system. An automatic exposure adjustment technique was proposed for compensating the bias of various optical factors involved in the system. Experiments were performed on an in house developed and a publicly available database of a variety of objects in simulated and real illumination conditions. It was observed that the identification of the illuminant a priori, is particularly useful for estimating illuminant sources with a smooth spectral power distribution. The findings also suggest that automatic exposure adjustment based imaging followed by color constancy improves spectral reflectance recovery under different illuminations.

A self similarity based descriptor was proposed for cross spectral alignment of hyperspectral images. The proposed descriptor was designed to be robust to the spectral variation between bands and sensitive to the spatial misalignment. Experiments were conducted on hyperspectral face images with inter-band misalignment due to movement of subjects. The results indicated significant reduction in the mean registration error by the proposed self similarity based registration.

An important contribution of the thesis was the Joint Group Sparse PCA algorithm which addressed the problem of finding informative bands from spatio-spectral data where pixels are spatially related. The efficacy of proposed approach was demonstrated by experiments on four hyperspectral image datasets. The results showed that the proposed method outperforms Sparse PCA and Group Sparse PCA algorithms by achieving lower reconstruction error in compressed hyperspectral imaging and higher accuracy in hyperspectral face recognition. This thesis also proposed a Joint Sparse PCA algorithm for band selection from spectral only data which has no spatial relationship. Accurate ink mismatch detection was achieved by Joint Sparse Band Selection compared to using all features or using a subset of features selected by Sequential Forward Band Selection.

A novel multidirectional representation and binary encoding technique (Contour Code) was presented for robust and efficient hyperspectral palmprint recognition. An automatic technique was designed for the extraction of a region of interest from palm images acquired with non-contact sensors. The proposed method exhibits robust multispectral feature capturing capability and consistently outperformed existing state-of-the-art techniques in recognition experiments on multiple palmprint databases. The Contour Code is a generic orientation code for line-like features and can be extended to other biometric traits including fingerprints and finger-knuckle prints.

\section{Future Work}

Recovery of spectral reflectance is of crucial importance to a spectral imaging system. The proposed spectral reflectance recovery method has been shown to be effective for single illumination source in the same scene. In real world scenarios, the statistics of a scene may change from one spatial location to another because of the presence of multiple illuminations at different locations. In such scenarios, it is expected that spectral reflectance recovery may be possible with localized methods which is a promising direction of future work.

The proposed Joint Group Sparse PCA methodology is well adaptable to scenarios where the features can be implicitly or explicitly categorized into non-overlapping groups. The case of overlapping groups where features may be related to more than one group is still unresolved. Moreover, extension to discriminative group sparse problems will require a discriminative criterion to be incorporated into the formulation. Joint Group Sparse LDA is one promising direction for future work.

Hyperspectral ink mismatch detection is a promising direction in forensic document analysis. There is further room for improvement in hard to separate ink combinations, which encourages further research. Another interesting problem that requires deeper investigation is the case of highly disproportionate mixtures of inks. It is expected that the results presented in this thesis will trigger more research efforts in the direction of automated hyperspectral document analysis. The collected database is publicly available for research. 

\bibliography{Zohaib_PhD_Thesis}

\providecommand{\bysame}{\leavevmode\hbox to3em{\hrulefill}\thinspace}
\begin{thebibliography}{100}

\bibitem{aginsky1993forensic}
V.~Aginsky,  `Forensic examinaton of ``slightly soluble'' ink pigments using
  thin-layer chromatography', \emph{Journal of Forensic Sciences} \textbf{38}
  (1993), 1131--1131.

\bibitem{alboon2008flat}
S.~A. Alboon and R.~G. Lindquist,  `Flat top liquid crystal tunable filter
  using coupled fabry-perot cavities', \emph{Optics Express} \textbf{16}
  (2008), no.~1, 231--236.

\bibitem{baraniuk2007compressive}
R.~G. Baraniuk,  `Compressive sensing', \emph{IEEE Signal Processing Magazine}
  \textbf{24} (2007), no.~4, 118--121.

\bibitem{barnard2002data}
K.~Barnard, L.~Martin, B.~Funt, and A.~Coath,  `A data set for color research',
  \emph{Color Research \& Application} \textbf{27} (2002), no.~3, 147--151.

\bibitem{baronti1997principal}
S.~Baronti, A.~Casini, F.~Lotti, and S.~Porcinai,  `Principal component
  analysis of visible and near-infrared multispectral images of works of art',
  \emph{Chemometrics and Intelligent Laboratory Systems} \textbf{39} (1997),
  no.~1, 103--114.

\bibitem{bianco2012multispectral}
G.~Bianco, F.~Bruno, and M.~Muzzupappa,  `Multispectral data cube acquisition
  of aligned images for document analysis by means of a filter-wheel camera
  provided with focus control', \emph{Journal of Cultural Heritage} \textbf{14}
  (2013), no.~3, 190--200.

\bibitem{bianco2010automatic}
S.~Bianco, G.~Ciocca, C.~Cusano, and R.~Schettini,  `Automatic color constancy
  algorithm selection and combination', \emph{Pattern Recognition} \textbf{43}
  (2010), no.~3, 695--705.

\bibitem{bianco2012color}
S.~Bianco and R.~Schettini,  `Color constancy using faces', in \emph{Proc.
  International Conference on Computer Vision and Pattern Recognition} (IEEE,
  2012), 65--72.

\bibitem{boyce2006multispectral}
C.~Boyce, A.~Ross, M.~Monaco, L.~Hornak, and X.~Li,  `Multispectral iris
  analysis: {A} preliminary study', in \emph{Proc. Computer Vision and Pattern
  Recognition Workshops} (IEEE, 2006), 51.

\bibitem{brauns2006fourier}
E.~B. Brauns and R.~B. Dyer,  `Fourier transform hyperspectral visible imaging
  and the nondestructive analysis of potentially fraudulent documents',
  \emph{Applied Spectroscopy} \textbf{60} (2006), no.~8, 833--840.

\bibitem{brill1990image}
M.~H. Brill,  `Image segmentation by object color: {A} unifying framework and
  connection to color constancy', \emph{JOSA A} \textbf{7} (1990), no.~10,
  2041--2047.

\bibitem{brink2012writer}
A.~Brink, J.~Smit, M.~Bulacu, and L.~Schomaker,  `Writer identification using
  directional ink-trace width measurements', \emph{Pattern Recognition}
  \textbf{45} (2012), no.~1, 162--171.

\bibitem{buchsbaum1980spatial}
G.~Buchsbaum,  `A spatial processor model for object colour perception',
  \emph{Journal of the Franklin Institute} \textbf{310} (1980), no.~1, 1--26.

\bibitem{burns1996analysis}
P.~D. Burns and R.~S. Berns,  `Analysis of multispectral image capture', in
  \emph{Proc. 4th IS\&T/SID Color Imaging Conference} (Society for Imaging
  Science and Technology, 1996), 19--22.

\bibitem{calcagni2011multispectral}
A.~Calcagni, J.~Gibson, I.~Styles, E.~Claridge, and F.~Orihuela-Espina,
  `Multispectral retinal image analysis: {A} novel non-invasive tool for
  retinal imaging', \emph{Eye} \textbf{25} (2011), no.~12, 1562--1569.

\bibitem{chakrabarti2011statistics}
A.~Chakrabarti and T.~Zickler,  `Statistics of real-world hyperspectral
  images', in \emph{Proc. International Conference on Computer Vision and
  Pattern Recognition} (IEEE, 2011), 193--200.

\bibitem{chang2008multispectral}
H.~Chang, A.~Koschan, M.~Abidi, S.~G. Kong, and C.-H. Won,  `Multispectral
  visible and infrared imaging for face recognition', in \emph{Proc. Computer
  Vision and Pattern Recognition Workshops} (IEEE, 2008), 1--6.

\bibitem{chatfield2009efficient}
K.~Chatfield, J.~Philbin, and A.~Zisserman,  `Efficient retrieval of deformable
  shape classes using local self-similarities', in \emph{Proc. International
  Conference on Computer Vision Workshops} (IEEE, 2009), 264--271.

\bibitem{ChemImage}
ChemImage (\url{http://www.chemimage.com/}).

\bibitem{chi2010multi}
C.~Chi, H.~Yoo, and M.~Ben-Ezra,  `Multi-spectral imaging by optimized wide
  band illumination', \emph{International Journal of Computer Vision}
  \textbf{86} (2010), no.~2-3, 140--151.

\bibitem{clemmensen2011sparse}
L.~Clemmensen, T.~Hastie, D.~Witten, and B.~Ersb{\o}ll,  `Sparse discriminant
  analysis', \emph{Technometrics} \textbf{53} (2011), no.~4, 406--413.

\bibitem{clemmensen2010comparison}
L.~H. Clemmensen, M.~E. Hansen, and B.~K. Ersb{\o}ll,  `A comparison of
  dimension reduction methods with application to multi-spectral images of sand
  used in concrete', \emph{Machine Vision and Applications} \textbf{21} (2010),
  no.~6, 959--968.

\bibitem{cohen1992biorthogonal}
A.~Cohen, I.~Daubechies, and J.-C. Feauveau,  `Biorthogonal bases of compactly
  supported wavelets', \emph{Communications on Pure and Applied Mathematics}
  \textbf{45} (1992), no.~5, 485--560.

\bibitem{comelli2008portable}
D.~Comelli, G.~Valentini, A.~Nevin, A.~Farina, L.~Toniolo, and R.~Cubeddu,  `A
  portable {UV}-fluorescence multispectral imaging system for the analysis of
  painted surfaces', \emph{Review of Scientific Instruments} \textbf{79}
  (2008), no.~8, 086112.

\bibitem{da2008geometrical}
A.~da~Cunha, \emph{Geometrical representation, processing, and coding of visual
  information} (Ph.D. thesis, University of Illinois at Urbana-Champaign,
  United States, 2008).

\bibitem{da2005bi}
A.~L. Da~Cunha and M.~N. Do,  `Bi-orthogonal filter banks with directional
  vanishing moments [image representation applications]', in \emph{Proc.
  International Conference on Acoustics, Speech, and Signal Processing},
  \textbf{4} (IEEE, 2005), 553--556.

\bibitem{da2006nonsubsampled}
A.~L. Da~Cunha, J.~Zhou, and M.~N. Do,  `The nonsubsampled contourlet
  transform: {Theory}, design, and applications', \emph{IEEE Transactions on
  Image Processing} \textbf{15} (2006), no.~10, 3089--3101.

\bibitem{d2007direct}
A.~d'Aspremont, L.~El~Ghaoui, M.~I. Jordan, and G.~R. Lanckriet,  `A direct
  formulation for sparse {PCA} using semidefinite programming', \emph{SIAM
  Review} \textbf{49} (2007), no.~3, 434--448.

\bibitem{davidson1975iterative}
E.~R. Davidson,  `The iterative calculation of a few of the lowest eigenvalues
  and corresponding eigenvectors of large real-symmetric matrices',
  \emph{Journal of Computational Physics} \textbf{17} (1975), no.~1, 87--94.

\bibitem{denes2002hyperspectral}
L.~Denes, P.~Metes, and Y.~Liu,  `Hyperspectral face database', \emph{Tech.
  Report} \textbf{CMU-RI-TR-02-25} (Robotics Institute, Pittsburgh, PA, October
  2002).

\bibitem{descour1995computed}
M.~Descour and E.~Dereniak,  `{Computed-Tomography Imaging Spectrometer:
  Experimental} calibration and reconstruction results', \emph{Applied Optics}
  \textbf{34} (1995), no.~22, 4817--4826.

\bibitem{di2010studies}
W.~Di, L.~Zhang, D.~Zhang, and Q.~Pan,  `Studies on hyperspectral face
  recognition in visible spectrum with feature band selection', \emph{IEEE
  Transactions on Systems, Man and Cybernetics, Part A: Systems and Humans}
  \textbf{40} (2010), no.~6, 1354--1361.

\bibitem{do2005contourlet}
M.~N. Do and M.~Vetterli,  `{The Contourlet Transform: An} efficient
  directional multiresolution image representation', \emph{IEEE Transactions on
  Image Processing} \textbf{14} (2005), no.~12, 2091--2106.

\bibitem{donoho2006compressed}
D.~L. Donoho,  `Compressed sensing', \emph{IEEE Transactions on Information
  Theory} \textbf{52} (2006), no.~4, 1289--1306.

\bibitem{dorigo1997ant}
M.~Dorigo and L.~M. Gambardella,  `{Ant Colony System: A} cooperative learning
  approach to the traveling salesman problem', \emph{IEEE Transactions on
  Evolutionary Computation} \textbf{1} (1997), no.~1, 53--66.

\bibitem{du2009prism}
H.~Du, X.~Tong, X.~Cao, and S.~Lin,  `A prism-based system for multispectral
  video acquisition', in \emph{Proc. International Conference on Computer
  Vision} (IEEE, 2009), 175--182.

\bibitem{duarte2005joint}
M.~F. Duarte, S.~Sarvotham, M.~B. Wakin, D.~Baron, and R.~G. Baraniuk,
  \emph{Joint sparsity models for distributed compressed sensing}, in
  \emph{Workshop on Signal Processing with Adaptative Sparse Structured
  Representations} (2005).

\bibitem{edelman2012hyperspectral}
G.~Edelman, E.~Gaston, T.~van Leeuwen, P.~Cullen, and M.~Aalders,
  `Hyperspectral imaging for non-contact analysis of forensic traces',
  \emph{Forensic Science International} \textbf{223} (2012), 28--39.

\bibitem{erives2006automatic}
H.~Erives and G.~J. Fitzgerald,  `Automatic sub-pixel registration for a
  tunable hyperspectral imaging system', in \emph{Proc. International
  Geoscience and Remote Sensing Symposium} (IEEE, 2006), 952--955.

\bibitem{erives2007automated}
H.~Erives, S.~Teare, and G.~J. Fitzgerald,  `An automated nonrigid registration
  for a tunable hyperspectral imaging system', in \emph{Fourier Transform
  Spectroscopy} (Optical Society of America, 2007), JWA20.

\bibitem{pascal-voc-2012}
M.~Everingham, L.~Van~Gool, C.~K.~I. Williams, J.~Winn, and A.~Zisserman,
  \emph{The {PASCAL} {V}isual {O}bject {C}lasses {C}hallenge} (2012).

\bibitem{fang2012multi}
Z.~Fang, C.~Cao, W.~Jiang, W.~Ji, M.~Xu, and S.~Lu,  `Multi-spectral image
  inter-band registration technology research', in \emph{Proc. International
  Geoscience and Remote Sensing Symposium} (IEEE, 2012), 4287--4290.

\bibitem{fauch2010recovery}
L.~Fauch, E.~Nippolainen, V.~Teplov, and A.~A. Kamshilin,  `Recovery of
  reflection spectra in a multispectral imaging system with light emitting
  diodes', \emph{Optics Express} \textbf{18} (2010), no.~22, 23394--23405.

\bibitem{finlayson2004shades}
G.~Finlayson and E.~Trezzi,  `Shades of gray and colour constancy', in
  \emph{Twelfth Color Imaging Conference: Color Science and Engineering
  Systems, Technologies, and Applications} (SPIE, 2004), 37--41.

\bibitem{fiorentin2009multispectral}
P.~Fiorentin, E.~Pedrotti, and A.~Scroccaro,  `A multispectral imaging device
  for monitoring of colour in art works', in \emph{Proc. International
  Instrumentation and Measurement Technology Conference} (IEEE, 2009),
  356--360.

\bibitem{forsyth1990novel}
D.~Forsyth,  `A novel algorithm for color constancy', \emph{International
  Journal of Computer Vision} \textbf{5} (1990), no.~1, 5--35.

\bibitem{FnF}
foster +~freeman (\url{http://www.fosterfreeman.com/index.php}).

\bibitem{friedman2010note}
J.~Friedman, T.~Hastie, and R.~Tibshirani, \emph{A note on the group lasso and
  a sparse group lasso} (2010).

\bibitem{fu2011discriminant}
Z.~Fu and A.~Robles-Kelly,  `Discriminant absorption-feature learning for
  material classification', \emph{IEEE Transactions on Geoscience and Remote
  Sensing} \textbf{49} (2011), no.~5, 1536--1556.

\bibitem{gat2000imaging}
N.~Gat,  `Imaging spectroscopy using tunable filters: {A} review', in
  \emph{AeroSense 2000} (International Society for Optics and Photonics, 2000),
  50--64.

\bibitem{gatos2008improved}
B.~Gatos, I.~Pratikakis, and S.~J. Perantonis,  `Improved document image
  binarization by using a combination of multiple binarization techniques and
  adapted edge information', in \emph{Proc. International Conference on Pattern
  Recognition} (2008), 1--4.

\bibitem{gegenfurtner2003cortical}
K.~R. Gegenfurtner,  `Cortical mechanisms of colour vision', \emph{Nature
  Reviews Neuroscience} \textbf{4} (2003), no.~7, 563--572.

\bibitem{gevers2000color}
T.~Gevers, H.~Stokman, and J.~Van De~Weijer,  `Color constancy from
  hyper-spectral data', in \emph{Proc. British Machine Vision Conference}
  (BMVA, 2000), 292--301.

\bibitem{gijsenij2011computational}
A.~Gijsenij, T.~Gevers, and J.~Van De~Weijer,  `Computational color constancy:
  {Survey} and experiments', \emph{IEEE Transactions on Image Processing}
  \textbf{20} (2011), no.~9, 2475--2489.

\bibitem{gijsenij2009perceptual}
A.~Gijsenij, T.~Gevers, and M.~P. Lucassen,  `Perceptual analysis of distance
  measures for color constancy algorithms', \emph{JOSA A} \textbf{26} (2009),
  no.~10, 2243--2256.

\bibitem{golbabaee2012hyperspectral}
M.~Golbabaee and P.~Vandergheynst,  `Hyperspectral image compressed sensing via
  low-rank and joint-sparse matrix recovery', in \emph{Proc. International
  Conference on Acoustics, Speech and Signal Processing} (IEEE, 2012),
  2741--2744.

\bibitem{golz2002influence}
J.~Golz and D.~I. MacLeod,  `Influence of scene statistics on colour
  constancy', \emph{Nature} \textbf{415} (2002), no.~6872, 637--640.

\bibitem{gorman2010generalization}
A.~Gorman, D.~W. Fletcher-Holmes, A.~R. Harvey, et~al.,  `Generalization of the
  {Lyot} filter and its application to snapshot spectral imaging', \emph{Optics
  Express} \textbf{18} (2010), no.~6, 5602--5608.

\bibitem{grbovic2012sparse}
M.~Grbovic, C.~R. Dance, and S.~Vucetic, \emph{Sparse principal component
  analysis with constraints}, in \emph{Proc. AAAI Conference on Artificial
  Intelligence} (2012).

\bibitem{gregoris2004multispectral}
D.~Gregoris, S.~Yu, and F.~Teti,  `Multispectral imaging of ice', in
  \emph{Proc. Canadian Conference on Electrical and Computer Engineering},
  \textbf{4} (IEEE, 2004), 2051--2056.

\bibitem{gu2013efficient}
L.~Gu, A.~Robles-Kelly, and J.~Zhou,  `Efficient estimation of reflectance
  parameters from imaging spectroscopy', \emph{IEEE Transactions on Image
  Processing} \textbf{22} (2013), no.~9, 3648--3663.

\bibitem{guo2000face}
G.~Guo, S.~Z. Li, and K.~Chan,  `Face recognition by support vector machines',
  in \emph{Proc. International Conference on Automatic Face and Gesture
  Recognition} (IEEE, 2000), 196--201.

\bibitem{guo2010principal}
J.~Guo, G.~James, E.~Levina, G.~Michailidis, and J.~Zhu,  `Principal component
  analysis with sparse fused loadings', \emph{Journal of Computational and
  Graphical Statistics} \textbf{19} (2010), no.~4, 930--946.

\bibitem{guo2010feature}
Z.~Guo, L.~Zhang, and D.~Zhang,  `Feature band selection for multispectral
  palmprint recognition', in \emph{Proc. International Conference on Pattern
  Recognition} (IEEE, 2010), 1136--1139.

\bibitem{haar1910theorie}
A.~Haar,  `Zur theorie der orthogonalen funktionensysteme', \emph{Mathematische
  Annalen} \textbf{69} (1910), no.~3, 331--371.

\bibitem{hammond2007validation}
D.~L. Hammond,  `Validation of lab color mode as a nondestructive method to
  differentiate black ballpoint pen inks', \emph{Journal of Forensic Sciences}
  \textbf{52} (2007), no.~4, 967--973.

\bibitem{han2008multispectral}
D.~Han, Z.~Guo, and D.~Zhang,  `Multispectral palmprint recognition using
  wavelet-based image fusion', in \emph{Proc. International Conference on
  Signal Processing} (IEEE, 2008), 2074--2077.

\bibitem{hao2007comparative}
Y.~Hao, Z.~Sun, and T.~Tan,  `Comparative studies on multispectral palm image
  fusion for biometrics', in \emph{Proc. Asian Conference on Computer Vision}
  (Springer, 2007), 12--21.

\bibitem{hao2008multispectral}
Y.~Hao, Z.~Sun, T.~Tan, and C.~Ren,  `Multispectral palm image fusion for
  accurate contact-free palmprint recognition', in \emph{Proc. International
  Conference on Image Processing} (IEEE, 2008), 281--284.

\bibitem{hashimoto2011multispectral}
N.~Hashimoto, Y.~Murakami, P.~A. Bautista, M.~Yamaguchi, T.~Obi, N.~Ohyama,
  K.~Uto, and Y.~Kosugi,  `Multispectral image enhancement for effective
  visualization', \emph{Optics Express} \textbf{19} (2011), no.~10, 9315--9329.

\bibitem{healey1994global}
G.~Healey and D.~Slater,  `Global color constancy: {Recognition} of objects by
  use of illumination-invariant properties of color distributions', \emph{JOSA
  A} \textbf{11} (1994), no.~11, 3003--3010.

\bibitem{hedjam2013historical}
R.~Hedjam and M.~Cheriet,  `Historical document image restoration using
  multispectral imaging system', \emph{Pattern Recognition} \textbf{46} (2013),
  no.~8, 2297--2312.

\bibitem{hordley2004re}
S.~Hordley and G.~Finlayson,  `Re-evaluating colour constancy algorithms', in
  \emph{Proc. International Conference on Pattern Recognition} (2004), 76--79.

\bibitem{huang2008palmprint}
D.-S. Huang, W.~Jia, and D.~Zhang,  `Palmprint verification based on principal
  lines', \emph{Pattern Recognition} \textbf{41} (2008), no.~4, 1316--1328.

\bibitem{huang2011learning}
J.~Huang, T.~Zhang, and D.~Metaxas,  `Learning with structured sparsity',
  \emph{The Journal of Machine Learning Research} \textbf{12} (2011),
  3371--3412.

\bibitem{jacob2009group}
L.~Jacob, G.~Obozinski, and J.-P. Vert,  `Group lasso with overlap and graph
  lasso', in \emph{Proc. International Conference on Machine Learning} (ACM,
  2009), 433--440.

\bibitem{jain2009latent}
A.~K. Jain and J.~Feng,  `Latent palmprint matching', \emph{IEEE Transactions
  on Pattern Analysis and Machine Intelligence} \textbf{31} (2009), no.~6,
  1032--1047.

\bibitem{jain1999data}
A.~K. Jain, M.~N. Murty, and P.~J. Flynn,  `Data clustering: {A} review',
  \emph{ACM Computing Surveys} \textbf{31} (1999), no.~3, 264--323.

\bibitem{jain2004introduction}
A.~K. Jain, A.~Ross, and S.~Prabhakar,  `An introduction to biometric
  recognition', \emph{IEEE Transactions on Circuits and Systems for Video
  Technology} \textbf{14} (2004), no.~1, 4--20.

\bibitem{jenatton2010proximal}
R.~Jenatton, J.~Mairal, G.~Obozinski, and F.~Bach,  `Proximal methods for
  sparse hierarchical dictionary learning', in \emph{Proc. International
  Conference on Machine Learning} (Omnipress, 2010), 1--8.

\bibitem{jenatton2010structured}
R.~Jenatton, G.~Obozinski, and F.~Bach, \emph{Structured sparse principal
  component analysis}, in \emph{International Conference on Artificial
  Intelligence and Statistics} (2010).

\bibitem{jolliffe2005principal}
I.~Jolliffe, \emph{Principal component analysis} (Wiley Online Library, 2005).

\bibitem{jolliffe2003modified}
I.~T. Jolliffe, N.~T. Trendafilov, and M.~Uddin,  `A modified principal
  component technique based on the lasso', \emph{Journal of Computational and
  Graphical Statistics} \textbf{12} (2003), no.~3, 531--547.

\bibitem{joo2011visual}
S.~Joo~Kim, F.~Deng, and M.~S. Brown,  `Visual enhancement of old documents
  with hyperspectral imaging', \emph{Pattern Recognition} \textbf{44} (2011),
  no.~7, 1461--1469.

\bibitem{journee2010generalized}
M.~Journ{\'e}e, Y.~Nesterov, P.~Richt{\'a}rik, and R.~Sepulchre,  `Generalized
  power method for sparse principal component analysis', \emph{Journal of
  Machine Learning Research} \textbf{11} (2010), 517--553.

\bibitem{katravsnik2013method}
J.~Katra{\v{s}}nik, F.~Pernu{\v{s}}, and B.~Likar,  `A method for
  characterizing illumination systems for hyperspectral imaging', \emph{Optics
  Express} \textbf{21} (2013), no.~4, 4841--4853.

\bibitem{kekre2011palmprint}
H.~Kekre, T.~Sarode, R.~Vig, A.~Pranay, I.~Aashita, and B.~Saurabh,  `Palmprint
  identification using {Kronecker product of DCT and Walsh Transforms} for
  multi-spectral images', in \emph{Proc. International Conference on Hand-Based
  Biometrics} (IEEE, 2011), 1--7.

\bibitem{khan2011contour}
Z.~Khan, A.~Mian, and Y.~Hu,  `{Contour Code: Robust} and efficient
  multispectral palmprint encoding for human recognition', in \emph{Proc.
  International Conference on Computer Vision} (IEEE, 2011), 1935--1942.

\bibitem{khan2013hyperspectral}
Z.~Khan, F.~Shafait, and A.~Mian,  `Hyperspectral imaging for ink mismatch
  detection', in \emph{Proc. International Conference on Document Analysis and
  Recognition} (IEEE, 2013), 877--881.

\bibitem{kise2010multispectral}
M.~Kise, B.~Park, G.~W. Heitschmidt, K.~C. Lawrence, and W.~R. Windham,
  `Multispectral imaging system with interchangeable filter design',
  \emph{Computers and Electronics in Agriculture} \textbf{72} (2010), no.~2,
  61--68.

\bibitem{kisku2010multispectral}
D.~R. Kisku, P.~Gupta, J.~K. Sing, and C.~Hwang,  `Multispectral palm image
  fusion for person authentication using ant colony optimization', in
  \emph{Proc. International Workshop on Emerging Techniques and Challenges for
  Hand-Based Biometrics} (IEEE, 2010), 1--7.

\bibitem{klein2008quantitative}
M.~E. Klein, B.~J. Aalderink, R.~Padoan, G.~De~Bruin, and T.~A. Steemers,
  `Quantitative hyperspectral reflectance imaging', \emph{Sensors} \textbf{8}
  (2008), no.~9, 5576--5618.

\bibitem{kong2009survey}
A.~Kong, D.~Zhang, and M.~Kamel,  `A survey of palmprint recognition',
  \emph{Pattern Recognition} \textbf{42} (2009), no.~7, 1408--1418.

\bibitem{kong2004competitive}
A.-K. Kong and D.~Zhang,  `Competitive coding scheme for palmprint
  verification', in \emph{Proc. International Conference on Pattern
  Recognition} (IEEE, 2004), 520--523.

\bibitem{land1974retinex}
E.~H. Land, \emph{The retinex theory of color vision} (Science Center, Harvard
  University, 1974).

\bibitem{land1977retinex}
\bysame, \emph{The retinex theory of color vision} (Scientific America., 1977).

\bibitem{lee2012subspace}
K.~Lee, Y.~Bresler, and M.~Junge,  `Subspace methods for joint sparse
  recovery', \emph{IEEE Transactions on Information Theory} \textbf{58} (2012),
  no.~6, 3613--3641.

\bibitem{li2011evaluating}
B.~Li, W.~Xiong, W.~Hu, and O.~Wu,  `Evaluating combinational color constancy
  methods on real-world images', in \emph{Proc. International Conference on
  Computer Vision and Pattern Recognition} (2011), 1929--1936.

\bibitem{li2010efficient}
W.~Li, L.~Zhang, D.~Zhang, G.~Lu, and J.~Yan,  `Efficient joint {2D} and {3D}
  palmprint matching with alignment refinement', in \emph{Proc. International
  Conference on Computer Vision and Pattern Recognition} (IEEE, 2010),
  795--801.

\bibitem{liu2009multi}
J.~Liu, S.~Ji, and J.~Ye,  `Multi-task feature learning via efficient
  $\ell_{2,1}$-norm minimization', in \emph{Proc. Twenty-Fifth Conference on
  Uncertainty in Artificial Intelligence} (AUAI Press, 2009), 339--348.

\bibitem{lu2003palmprint}
G.~Lu, D.~Zhang, and K.~Wang,  `Palmprint recognition using eigenpalms
  features', \emph{Pattern Recognition Letters} \textbf{24} (2003), no.~9,
  1463--1467.

\bibitem{lu1998direct}
W.-S. Lu, A.~Antoniou, and H.~Xu,  `A direct method for the design of {2-D}
  nonseparable filter banks', \emph{IEEE Transactions on Circuits and Systems
  II: Analog and Digital Signal Processing} \textbf{45} (1998), no.~8,
  1146--1150.

\bibitem{luo2012multispectral}
N.~Luo, Z.~Guo, G.~Wu, and C.~Song,  `Multispectral palmprint recognition by
  feature level fusion', in \emph{Recent Advances in Computer Science and
  Information Engineering} (Springer, 2012), 427--432.

\bibitem{maintz1998survey}
J.~Maintz and M.~A. Viergever,  `A survey of medical image registration',
  \emph{Medical Image Analysis} \textbf{2} (1998), no.~1, 1--36.

\bibitem{mairal2010network}
J.~Mairal, R.~Jenatton, G.~Obozinski, and F.~Bach,  `Network flow algorithms
  for structured sparsity', \emph{preprint arXiv:1008.5209} (2010).

\bibitem{maloney1986color}
L.~T. Maloney and B.~A. Wandell,  `Color constancy: {A} method for recovering
  surface spectral reflectance', \emph{JOSA A} \textbf{3} (1986), no.~1,
  29--33.

\bibitem{martin2009retinal}
P.~R. Martin,  `Retinal color vision in primates', in \emph{Encyclopedia of
  Neuroscience} (Springer, 2009), 3497--3501.

\bibitem{mathews2008design}
S.~A. Mathews,  `Design and fabrication of a low-cost, multispectral imaging
  system', \emph{Applied Optics} \textbf{47} (2008), no.~28, F71--F76.

\bibitem{meinshausen2006high}
N.~Meinshausen and P.~B{\"u}hlmann,  `High-dimensional graphs and variable
  selection with the lasso', \emph{The Annals of Statistics} \textbf{34}
  (2006), no.~3, 1436--1462.

\bibitem{meraoumia2013efficient}
A.~Meraoumia, S.~Chitroub, and A.~Bouridane,  `An efficient palmprint
  identification system using multispectral and hyperspectral imaging', in
  \emph{Modeling Approaches and Algorithms for Advanced Computer Applications}
  (Springer, 2013), 155--164.

\bibitem{merchante2012efficient}
L.~Merchante, Y.~Grandvalet, and G.~Govaert,  `An efficient approach to sparse
  linear discriminant analysis', in \emph{Proc. International Conference on
  Machine Learning} (Omnipress, 2012), 1167--1174.

\bibitem{mian2012hyperspectral}
A.~Mian and R.~Hartley,  `Hyperspectral video restoration using optical flow
  and sparse coding', \emph{Optics Express} \textbf{20} (2012), no.~10,
  10658--10673.

\bibitem{mistani2011multispectral}
S.~A. Mistani, S.~Minaee, and E.~Fatemizadeh,  `Multispectral palmprint
  recognition using a hybrid feature', \emph{preprint arXiv:1112.5997} (2011).

\bibitem{mittal2012rank}
N.~Mittal, M.~Hanmandlu, J.~Grover, and R.~Vijay,  `Rank-level fusion of
  multispectral palmprints', \emph{International Journal of Computer
  Applications} \textbf{38} (2012), no.~2, 33--37.

\bibitem{mohan2008agile}
A.~Mohan, R.~Raskar, and J.~Tumblin,  `{Agile Spectrum Imaging: Programmable}
  wavelength modulation for cameras and projectors', \emph{Computer Graphics
  Forum} \textbf{27} (2008), no.~2, 709--717.

\bibitem{molina2002feature}
L.~C. Molina, L.~Belanche, and {\`A}.~Nebot,  `Feature selection algorithms:
  {A} survey and experimental evaluation', in \emph{Proc. International
  Conference on Data Mining} (2002), 306--313.

\bibitem{murakami2008color}
Y.~Murakami, K.~Fukura, M.~Yamaguchi, and N.~Ohyama,  `Color reproduction from
  {low-SNR} multispectral images using spatio-spectral {Wiener} estimation',
  \emph{Optics express} \textbf{16} (2008), no.~6, 4106--4120.

\bibitem{nascimento2005psychophysical}
S.~M. Nascimento, D.~H. Foster, and K.~Amano,  `Psychophysical estimates of the
  number of spectral-reflectance basis functions needed to reproduce natural
  scenes', \emph{JOSA A} \textbf{22} (2005), no.~6, 1017--1022.

\bibitem{natarajan1995sparse}
B.~K. Natarajan,  `Sparse approximate solutions to linear systems', \emph{SIAM
  Journal on Computing} \textbf{24} (1995), no.~2, 227--234.

\bibitem{otsu1975threshold}
N.~Otsu,  `A threshold selection method from gray-level histograms',
  \emph{Automatica} \textbf{11} (1975), no.~285-296, 23--27.

\bibitem{padoan2008quantitative}
R.~Padoan, T.~A. Steemers, M.~Klein, B.~Aalderink, and G.~de~Bruin,
  \emph{Quantitative hyperspectral imaging of historical documents: {Technique}
  and applications}, in \emph{Proc. International Conference on NDT of Art}
  (2008).

\bibitem{pan2003face}
Z.~Pan, G.~Healey, M.~Prasad, and B.~Tromberg,  `Face recognition in
  hyperspectral images', \emph{IEEE Trans. on Pattern Analysis and Machine
  Intelligence} \textbf{25} (2003), no.~12, 1552--1560.

\bibitem{park2007multispectral}
J.-I. Park, M.-H. Lee, M.~D. Grossberg, and S.~K. Nayar,  `Multispectral
  imaging using multiplexed illumination', in \emph{Proc. International
  Conference on Computer Vision} (IEEE, 2007), 1--8.

\bibitem{pelagotti2008multispectral}
A.~Pelagotti, A.~Del~Mastio, A.~De~Rosa, and A.~Piva,  `Multispectral imaging
  of paintings', \emph{IEEE Signal Processing Magazine} \textbf{25} (2008),
  no.~4, 27--36.

\bibitem{phoong1995new}
S.-M. Phoong, C.~W. Kim, P.~Vaidyanathan, and R.~Ansari,  `A new class of
  two-channel biorthogonal filter banks and wavelet bases', \emph{IEEE
  Transactions on Signal Processing} \textbf{43} (1995), no.~3, 649--665.

\bibitem{pluim2003mutual}
J.~P. Pluim, J.~A. Maintz, and M.~A. Viergever,  `Mutual-information-based
  registration of medical images: {A} survey', \emph{IEEE Transactions on
  Medical Imaging} \textbf{22} (2003), no.~8, 986--1004.

\bibitem{poger2001multispectral}
S.~Poger and E.~Angelopoulou,  `Multispectral sensors in computer vision',
  \emph{Tech. Report} \textbf{CS-2001-3} (Stevens Institute of Technology,
  2001).

\bibitem{pratt1974correlation}
W.~K. Pratt,  `Correlation techniques of image registration', \emph{IEEE
  Transactions on Aerospace and Electronic Systems} (1974), no.~3, 353--358.

\bibitem{rowe2005multispectral}
R.~K. Rowe, K.~Nixon, and S.~Corcoran,  `Multispectral fingerprint biometrics',
  in \emph{Proc. IEEE SMC Information Assurance Workshop} (IEEE, 2005), 14--20.

\bibitem{rowe2007multispectral}
R.~K. Rowe, U.~Uludag, M.~Demirkus, S.~Parthasaradhi, and A.~Jain,  `A
  multispectral whole-hand biometric authentication system', in \emph{Proc.
  Biometrics Symposium} (IEEE, 2007), 1--6.

\bibitem{rueckert1999nonrigid}
D.~Rueckert, L.~I. Sonoda, C.~Hayes, D.~L. Hill, M.~O. Leach, and D.~J. Hawkes,
   `Nonrigid registration using free-form deformations: {Application} to breast
  {MR} images', \emph{IEEE Transactions on Medical Imaging} \textbf{18} (1999),
  no.~8, 712--721.

\bibitem{sauvola2000adaptive}
J.~Sauvola and M.~Pietik{\"a}inen,  `Adaptive document image binarization',
  \emph{Pattern Recognition} \textbf{33} (2000), no.~2, 225--236.

\bibitem{shafait2008efficient}
F.~Shafait, D.~Keysers, and T.~M. Breuel,  `Efficient implementation of local
  adaptive thresholding techniques using integral images', in \emph{Proc.
  Document Recognition and Retrieval XV} (International Society for Optics and
  Photonics, 2008), 681510--681510--6.

\bibitem{shah1994ladder}
I.~A. Shah and T.~A. Kalker,  `On ladder structures and linear phase conditions
  for bi-orthogonal filter banks', in \emph{Proc. International Conference on
  Acoustics, Speech, and Signal Processing}, \textbf{3} (IEEE, 1994), 181--184.

\bibitem{shechtman2007matching}
E.~Shechtman and M.~Irani,  `Matching local self-similarities across images and
  videos', in \emph{Proc. International Conference on Computer Vision and
  Pattern Recognition} (IEEE, 2007), 1--8.

\bibitem{shen2008sparse}
H.~Shen and J.~Z. Huang,  `Sparse principal component analysis via regularized
  low rank matrix approximation', \emph{Journal of Multivariate Analysis}
  \textbf{99} (2008), no.~6, 1015--1034.

\bibitem{shen2007reflectance}
H.-L. Shen, P.-Q. Cai, S.~Shao, and J.~H. Xin,  `Reflectance reconstruction for
  multispectral imaging by adaptive {Wiener} estimation', \emph{Optics Express}
  \textbf{15} (2007), no.~23, 15545--15554.

\bibitem{shen2007improved}
H.-L. Shen, J.~H. Xin, and S.~Shao,  `Improved reflectance reconstruction for
  multispectral imaging by combining different techniques', \emph{Optics
  Express} \textbf{15} (2007), no.~9, 5531--5536.

\bibitem{shen2008optimal}
H.-L. Shen, H.-G. Zhang, J.~H. Xin, and S.-J. Shao,  `Optimal selection of
  representative colors for spectral reflectance reconstruction in a
  multispectral imaging system', \emph{Applied Optics} \textbf{47} (2008),
  no.~13, 2494--2502.

\bibitem{smith1931cie}
T.~Smith and J.~Guild,  `The {CIE} colorimetric standards and their use',
  \emph{Transactions of the Optical Society} \textbf{33} (1931), no.~3, 73.

\bibitem{stone2002blind}
H.~S. Stone and R.~Wolpov,  `Blind cross-spectral image registration using
  prefiltering and fourier-based translation detection', \emph{IEEE
  Transactions on Geoscience and Remote Sensing} \textbf{40} (2002), no.~3,
  637--650.

\bibitem{sun2010low}
R.~Sun, M.~B. Bouchard, S.~A. Burgess, A.~J. Radosevich, and E.~M. Hillman,  `A
  low-cost, portable system for high-speed multispectral optical imaging', in
  \emph{Biomedical Optics} (Optical Society of America, 2010), BTuD41.

\bibitem{sun2005ordinal}
Z.~Sun, T.~Tan, Y.~Wang, and S.~Z. Li,  `Ordinal palmprint representation for
  personal identification', in \emph{Proc. International Conference on Computer
  Vision and Pattern Recognition} (IEEE, 2005), 279--284.

\bibitem{tahmasebi2011novel}
A.~Tahmasebi, H.~Pourghasem, and H.~Mahdavi-Nasab,  `A novel rank-level fusion
  for multispectral palmprint identification system', in \emph{Proc.
  International Conference on Intelligent Computation and Bio-Medical
  Instrumentation} (IEEE, 2011), 208--211.

\bibitem{taxt1994multispectral}
T.~Taxt and A.~Lundervold,  `Multispectral analysis of the brain using magnetic
  resonance imaging', \emph{IEEE Transactions on Medical Imaging} \textbf{13}
  (1994), no.~3, 470--481.

\bibitem{thai2002invariant}
B.~Thai and G.~Healey,  `Invariant subpixel material detection in hyperspectral
  imagery', \emph{IEEE Transactions on Geoscience and Remote Sensing}
  \textbf{40} (2002), no.~3, 599--608.

\bibitem{thodberg1996review}
H.~H. Thodberg,  `A review of bayesian neural networks with an application to
  near infrared spectroscopy', \emph{IEEE Transactions on Neural Networks}
  \textbf{7} (1996), no.~1, 56--72.

\bibitem{tibshirani1996regression}
R.~Tibshirani,  `Regression shrinkage and selection via the lasso',
  \emph{Journal of the Royal Statistical Society. Series B (Methodological)}
  (1996), 267--288.

\bibitem{turk1991eigenfaces}
M.~Turk and A.~Pentland,  `Eigenfaces for recognition', \emph{Journal of
  Cognitive Neuroscience} \textbf{3} (1991), no.~1, 71--86.

\bibitem{van2007edge}
J.~Van De~Weijer, T.~Gevers, and A.~Gijsenij,  `Edge-based color constancy',
  \emph{IEEE Transactions on Image Processing} \textbf{16} (2007), no.~9,
  2207--2214.

\bibitem{van2007using}
J.~Van De~Weijer, C.~Schmid, and J.~Verbeek,  `Using high-level visual
  information for color constancy', in \emph{Proc. International Conference on
  Computer Vision} (2007), 1--8.

\bibitem{vetterli1995wavelets}
M.~Vetterli and J.~Kova{\v{c}}evi{\'c}, \emph{Wavelets and subband coding},
  \textbf{87} (Prentice Hall PTR Englewood Cliffs, New Jersey, 1995).

\bibitem{wang2003detection}
C.~Wang, C.~Chen, Y.~Chung, S.~Yang, P.~Chung, C.~Yang, and C.~Chang,
  `Detection of spectral signatures in multispectral {MR} images for
  classification', \emph{IEEE Transactions on Medical Imaging} \textbf{22}
  (2003), no.~1, 50--61.

\bibitem{wang2007fusion}
J.-G. Wang, W.-Y. Yau, A.~Suwandy, and E.~Sung,  `Fusion of palmprint and palm
  vein images for person recognition based on laplacianpalm feature', in
  \emph{Proc. International Conference on Computer Vision and Pattern
  Recognition} (IEEE, 2007), 1--8.

\bibitem{wiemker1997color}
R.~Wiemker, \emph{The color constancy problem in multispectral remote
  sensing-{On} the impact of surface orientation on spectral signatures} (Ph.D.
  thesis, 1997).

\bibitem{woods1932effect}
H.~Woods, H.~H. Steinour, and H.~R. Starke,  `Effect of composition of
  {Portland} cement on heat evolved during hardening', \emph{Industrial \&
  Engineering Chemistry} \textbf{24} (1932), no.~11, 1207--1214.

\bibitem{wright2009robust}
J.~Wright, A.~Y. Yang, A.~Ganesh, S.~S. Sastry, and Y.~Ma,  `Robust face
  recognition via sparse representation', \emph{IEEE Transactions on Pattern
  Analysis and Machine Intelligence} \textbf{31} (2009), no.~2, 210--227.

\bibitem{wu2006palmprint}
X.~Wu, K.~Wang, and D.~Zhang,  `Palmprint texture analysis using derivative of
  gaussian filters', in \emph{Proc. International Conference on Computational
  Intelligence and Security} (IEEE, 2006), 751--754.

\bibitem{wu2003fisherpalms}
X.~Wu, D.~Zhang, and K.~Wang,  `Fisherpalms based palmprint recognition',
  \emph{Pattern Recognition Letters} \textbf{24} (2003), no.~15, 2829--2838.

\bibitem{xu2010multispectral}
X.~Xu and Z.~Guo,  `Multispectral palmprint recognition using quaternion
  principal component analysis', in \emph{Proc. International Workshop on
  Emerging Techniques and Challenges for Hand-Based Biometrics} (IEEE, 2010),
  1--5.

\bibitem{yasuma2010generalized}
F.~Yasuma, T.~Mitsunaga, D.~Iso, and S.~K. Nayar,  `{Generalized Assorted Pixel
  camera: Postcapture} control of resolution, dynamic range, and spectrum',
  \emph{IEEE Transactions on Pattern Analysis and Machine Intelligence}
  \textbf{19} (2010), no.~9, 2241--2253.

\bibitem{yuan2006model}
M.~Yuan and Y.~Lin,  `Model selection and estimation in regression with grouped
  variables', \emph{Journal of the Royal Statistical Society: Series B
  (Statistical Methodology)} \textbf{68} (2006), no.~1, 49--67.

\bibitem{zass2006nonnegative}
R.~Zass and A.~Shashua,  `Nonnegative sparse {PCA}', in \emph{Advances in
  Neural Information Processing Systems} (2006), 1561--1568.

\bibitem{zawada2003image}
D.~G. Zawada,  `Image processing of underwater multispectral imagery',
  \emph{IEEE Journal of Oceanic Engineering} \textbf{28} (2003), no.~4,
  583--594.

\bibitem{zhang2011online}
D.~Zhang, Z.~Guo, G.~Lu, L.~Zhang, Y.~Liu, and W.~Zuo,  `Online joint palmprint
  and palmvein verification', \emph{Expert Systems with Applications}
  \textbf{38} (2011), no.~3, 2621--2631.

\bibitem{zhang2010online}
D.~Zhang, Z.~Guo, G.~Lu, L.~Zhang, and W.~Zuo,  `An online system of
  multispectral palmprint verification', \emph{IEEE Transactions on
  Instrumentation and Measurement} \textbf{59} (2010), no.~2, 480--490.

\bibitem{zhang2003online}
D.~Zhang, W.-K. Kong, J.~You, and M.~Wong,  `Online palmprint identification',
  \emph{IEEE Transactions on Pattern Analysis and Machine Intelligence}
  \textbf{25} (2003), no.~9, 1041--1050.

\bibitem{zhang2009palmprint}
D.~Zhang, G.~Lu, W.~Li, L.~Zhang, and N.~Luo,  `Palmprint recognition using
  {3-D} information', \emph{IEEE Transactions on Systems, Man, and Cybernetics,
  Part C: Applications and Reviews} \textbf{39} (2009), no.~5, 505--519.

\bibitem{zhang2012comparative}
D.~Zhang, W.~Zuo, and F.~Yue,  `A comparative study of palmprint recognition
  algorithms', \emph{ACM Computing Surveys} \textbf{44} (2012), no.~1,
  2:1--2:37.

\bibitem{zhao2013development}
H.~Zhao, P.~Zhou, Y.~Zhang, Z.~Wang, and S.~Shi,  `Development of a dual-path
  system for band-to-band registration of an acousto-optic tunable filter-based
  imaging spectrometer', \emph{Optics Letters} \textbf{38} (2013), no.~20,
  4120--4123.

\bibitem{zhao2010efficient}
Z.~Zhao, L.~Wang, and H.~Liu, \emph{Efficient spectral feature selection with
  minimum redundancy}, in \emph{Proc. AAAI Conference on Artificial
  Intelligence} (2010).

\bibitem{zheng2011three}
Z.~Zheng, G.~Yang, H.~Li, and X.~Liu,  `Three-stage {Fabry--Perot} liquid
  crystal tunable filter with extended spectral range', \emph{Optics Express}
  \textbf{19} (2011), no.~3, 2158--2164.

\bibitem{zhou2010contactless}
Y.~Zhou and A.~Kumar,  `Contactless palm vein identification using multiple
  representations', in \emph{Proc. International Conference on Biometrics:
  Theory Applications and Systems} (IEEE, 2010), 1--6.

\bibitem{zitova2003image}
B.~Zitova and J.~Flusser,  `Image registration methods: {A} survey',
  \emph{Image and Vision Computing} \textbf{21} (2003), no.~11, 977--1000.

\bibitem{zou2006sparse}
H.~Zou, T.~Hastie, and R.~Tibshirani,  `Sparse principal component analysis',
  \emph{Journal of Computational and Graphical Statistics} \textbf{15} (2006),
  no.~2, 265--286.

\bibitem{zou2006adaptive}
H.~Zou,  `The adaptive lasso and its oracle properties', \emph{Journal of the
  American Statistical Association} \textbf{101} (2006), no.~476, 1418--1429.

\end{thebibliography}

\end{document}